\documentclass[hdr, twoside, final]{unswbook}
\usepackage{mystyle} 
\makeatletter
\fancypagestyle{noHeading}{
        \fancyhead{}
        \renewcommand{\headrulewidth}{0pt}
}
\thesistitle{Advanced Algorithms of Collision Free  Navigation and Flocking for Autonomous UAVs} 
\thesisauthor{Taha Elmokadem}
\thesisdate{October 2021}

\DeclareGraphicsExtensions{.pdf,.jpeg,.png,.eps,.svg}
\graphicspath{{images/}}

\begin{document}
    
    \renewcommand{\bibname}{References}
    \frontmatter  
    \maketitle
    

\chapter{Abstract}

Unmanned aerial vehicles (UAVs) have become very popular for many military and civilian applications including in agriculture, construction, mining, environmental monitoring, etc. A desirable feature for UAVs is the ability to navigate and perform tasks autonomously with least human interaction. This is a very challenging problem due to several factors such as the high complexity of UAV applications, operation in harsh environments, limited payload and onboard computing power and highly nonlinear dynamics. Therefore, more research is still needed towards developing advanced reliable control strategies for UAVs to enable safe navigation in unknown and dynamic environments. This problem is even more challenging for multi-UAV systems where it is more efficient to utilize information shared among the networked vehicles. Therefore, the work presented in this report contributes towards the state-of-the-art in UAV control for safe autonomous navigation and motion coordination of multi-UAV systems. The first part of this report deals with single-UAV systems. Initially, a hybrid navigation framework is developed for autonomous mobile robots using a general 2D nonholonomic unicycle model that can be applied to different types of UAVs, ground vehicles and underwater vehicles considering only lateral motion. Then, the more complex problem of three-dimensional (3D) collision-free navigation in unknown/dynamic environments is addressed. To that end, advanced 3D reactive control strategies are developed adopting the sense-and-avoid paradigm to produce quick reactions around obstacles. A special case of navigation in 3D unknown confined environments (i.e. tunnel-like) is also addressed. General 3D kinematic models are considered in the design which makes these methods applicable to different UAV types in addition to underwater vehicles. Moreover, different implementation methods for these strategies with quadrotor-type UAVs are also investigated considering UAV dynamics in the control design. Practical experiments and simulations were carried out to analyze the performance of the developed methods. The second part of this report addresses safe navigation for multi-UAV systems. Distributed motion coordination methods of multi-UAV systems for flocking and 3D area coverage are developed. These methods offer good computational cost for large-scale systems. Simulations were performed to verify the performance of these methods considering systems with different sizes.

    \tableofcontents




\nomenclature{2D}{Two Dimensional}
\nomenclature{3D}{Three Dimensional}
\nomenclature{APF}{Artificial Potential Field}
\nomenclature{AUV}{Autonomous Underwater Vehicle}
\nomenclature{BLOS}{Beyond-Line-of-Sight}
\nomenclature{CPU}{Central Processing Unit}
\nomenclature{DOF}{Degrees of freedom}
\nomenclature{EKF}{Extended Kalman Filter}
\nomenclature{ESC}{Electronic Speed Controller}
\nomenclature{ESDF}{Euclidean Signed Distance Field}
\nomenclature{FCS}{Flight Control System}
\nomenclature{FCU}{Flight Control Unit}
\nomenclature{FOV}{Field of View}
\nomenclature{GCS}{Ground Control Station}
\nomenclature{GNSS}{Global Navigation Satellite System}
\nomenclature{GPS}{Global Positioning System}
\nomenclature{IMU}{Inertial Measurement Unit}
\nomenclature{LiDAR}{Light Detection and Ranging}
\nomenclature{LOS}{Line-of-Sight}
\nomenclature{MAV}{Micro Aerial Vehicle}
\nomenclature{MPC}{Model Predictive Control}
\nomenclature{MIQP}{Mixed-Integer Quadratic Program}
\nomenclature{MWSN}{Mobile Wireless Sensor Network}
\nomenclature{NCS}{Networked Control System}
\nomenclature{NED}{North, East, Down}
\nomenclature{NMPC}{Nonlinear Model Predictive Control}
\nomenclature{NSB}{Null-Space-based Behavioral}
\nomenclature{OCP}{Optimal Control Problem}
\nomenclature{PDB}{Power Distribution Board}
\nomenclature{QP}{Quadratic Program}
\nomenclature{RGB-D}{Red, Green, Blue, Depth}
\nomenclature{ROS}{Robot Operating System}
\nomenclature{RRT}{Rapidly-exploring Random Tree}
\nomenclature{RTK}{Real-Time Kinematic}
\nomenclature{SITL}{Software-in-the-Loop}
\nomenclature{SLAM}{Simultaneous Localization and Mapping}
\nomenclature{SMC}{Sliding Mode Control}
\nomenclature{UAS}{Unmanned Aerial System}
\nomenclature{UAV}{Unmanned Aerial Vehicle}
\nomenclature{UGV}{Unmanned Ground Vehicle}
\nomenclature{VTOL}{Vertical Take-off and Landing}

\printnomenclature

    
    \mainmatter
    \pagestyle{fancy}
        \fancyhf{}
        \fancyhead[LE]{\leftmark}
        \fancyhead[RO]{\rightmark}
        \fancyfoot[C]{\thepage}
        \renewcommand{\headrulewidth}{1pt}
        \setcounter{secnumdepth}{3}

    \chapter{Introduction}\label{cha:introduction}

This report deals with control problems related to safe navigation of unmanned aerial vehicles (UAVs) including single- and multi-vehicle systems.
Over the past decades, developments of UAV technologies have allowed them to be increasingly deployed in many applications extending into the civilian domain after being traditionally used mostly in military missions at early stages of development.
Nowadays, we see UAVs being utilized in various fields due to their low cost and agility.
Examples of fields and applications where UAVs have become popular tools include agriculture, mining, construction, engineering geology, archeology, surveying, inspection, autonomous firefighting, photography and many more.
Many of these applications still rely on manual operation or teleoperation due to reliability concerns as there are many challenges arise when designing fully autonomous solutions that require least to no human interaction.
This highlights the importance of research in advancing such technology to eliminate the human factor from the loop.
That is, it is very motivating to develop UAVs to a stage where they can carry out a complete mission without any human interaction operating in a fully autonomous mode.
This will greatly improve productivity, save costs and lives.

A key component for a UAV to operate autonomously is how it can navigate or fly safely to reach some targeted locations by only making decisions based on observations from onboard sensors.
This general problem has attracted great interest in the past years not only for UAVs but also for unmanned ground vehicles (UGVs), autonomous underwater vehicles (AUVs), and other mobile robots.
The level of complexity in software and hardware components design increases with the required level of autonomy as well as the complexity of the overall designated task.
Even operation in different environments with various challenging conditions may affect the overall system design in terms of the suitable sensors to use and how robust navigation algorithms should be.
Other challenges arise when developing UAVs are related to available technologies in computing, power, electronics and communications.
For example, some factors to consider when designing miniature UAVs are limited flight time (battery life), payload capacity and available onboard computational resources.
These are some of the factors that make developing navigation strategies for UAVs very challenging.
This motivates us to develop more advanced methods to cope with such limitations and to fully utilize the emerging advances in related technologies.

\section{Research Problem \& Objectives}\label{ch1:problem} %

The general research problem considered in this work is how to safely navigate UAVs utilizing their ability to perform 3D maneuvers in order to perform various tasks autonomously.
In other terms, it is required to develop navigation control strategies with 3D collision-avoidance capabilities to achieve certain global motion objective(s).
This broad problem is targeted in this report through tackling different subproblems including navigation in \textit{unknown}, dynamic and tunnel-like environments.
Referring to an environment as \textit{unknown} means that no previous knowledge or a map is available about the environment, and the vehicle can only build its understanding of the environment through information interpreted from onboard sensors.
Thus, it is a more complex and challenging problem than navigating in known environments.
Additionally, we also consider this problem for multi-vehicle systems where it is possible to achieve the safety navigation goal more efficiently through information exchange among networked vehicles.
Such information can be utilized through the development of distributed motion coordination control strategies.

The considered navigation subproblems address the following general research questions:
\begin{itemize}
	\item How to autonomously guide a UAV to reach some desired target position while avoiding collisions with surroundings when navigating in unknown and dynamic environments?
	\item How to autonomously guide a UAV to progressively advance through 3D unknown tunnel-like environments while avoiding collisions with its boundaries?
	\item How can a multi-vehicle system navigate safely as a group with no collisions among its vehicles to achieve some global objectives such as making a certain geometric formation and achieving consensus when moving to goal regions?
	\item How can a multi-vehicle system perform coverage tasks autonomously in 3D environments with guaranteed collision avoidance?
\end{itemize}
The complete description of each subproblem is provided in more detail in each chapter where they are addressed.

One of the main objectives is to develop control strategies with low computational cost adopting a \textit{"Sense \& React"} or \textit{"Sense \& Avoid"} paradigm.
Hence, they can be more effective in unknown and dynamic environments to provide quick motion decisions compared to many of the existing search-based and optimization-based methods.
This makes the developed methods more favorable in situations where it is more desired to dedicate more computational resources to other components within the overall autonomous stack such as computer vision, localization, mapping, etc.
The low computational cost also makes these approaches suitable for high-speed applications and miniature UAVs with limited onboard computing power.
Additionally, we aim to develop some of the ideas at a higher level so that they can be applicable to different UAV types in addition to other vehicles with the ability to move in 3D spaces such as AUVs.

\section{Research Approach}

Novel methods are proposed in this report to address the aforementioned problems. %
The development process include conceptual development, rigorous mathematical analysis, extensive simulations and experimental evaluations.
Observations made from simulations and practical implementations are used to further refine and improve the developed methods.
In simulations, we follow two approaches. 
General simulations using MATLAB are carried out based on a general kinematic model and some abstract sensing models to validate the overall concept.
Furthermore, Software-in-the-Loop (SITL) simulations are performed using complete models (i.e. including vehicle dynamics) of quadrotor-type UAVs through a physics engine.
In this case, the control methods are implemented in C++ and/or Python utilizing the Robot Operating System (ROS) middleware for efficient implementaion of the full autonomous stack which can be deployed directly to real systems.
SITL simulations help evaluating the computational performance since the same production code is applied in real-time to a simulation engine with a setup similar to the actual hardware.
This also provides an easier way of considering different practical sensing models. 

The practical evaluation is also done through several experiments in the Autonomous Systems Testing Laboratory at UNSW with different quadrotor UAVs.
Even though some of the proposed approaches have not been fully evaluated through experiments, practical aspects has been considered in their development based on insights from flights when testing the other approaches.
Examples of the used vehicles in some of the experiments are shown in \cref{fig:ch1:quadrotors}.

To achieve the highlighted objectives in the previous section, reactive-based approaches are adopted in the design process of many of the proposed methods.
This provides solutions with low-computational cost and quick reactions at the expense of being less optimal in some scenarios.
Furthermore, the overall ideas are developed at a high level using general kinematic models of UAVs treating them as single point moving in a three-dimensional (3D) workspace so that these methods can also be applicable to other vehicle types navigating in 3D such as AUVs.
Further possible ways of implementation and low-level control designs are also proposed for quadrotor UAVs to address potential practical concerns.

\begin{figure}[!htb]
	\centering
	\begin{adjustbox}{minipage=\linewidth,scale=0.75}
		\begin{subfigure}[t]{0.46\textwidth}
			\centering
			\includegraphics[width=\linewidth]{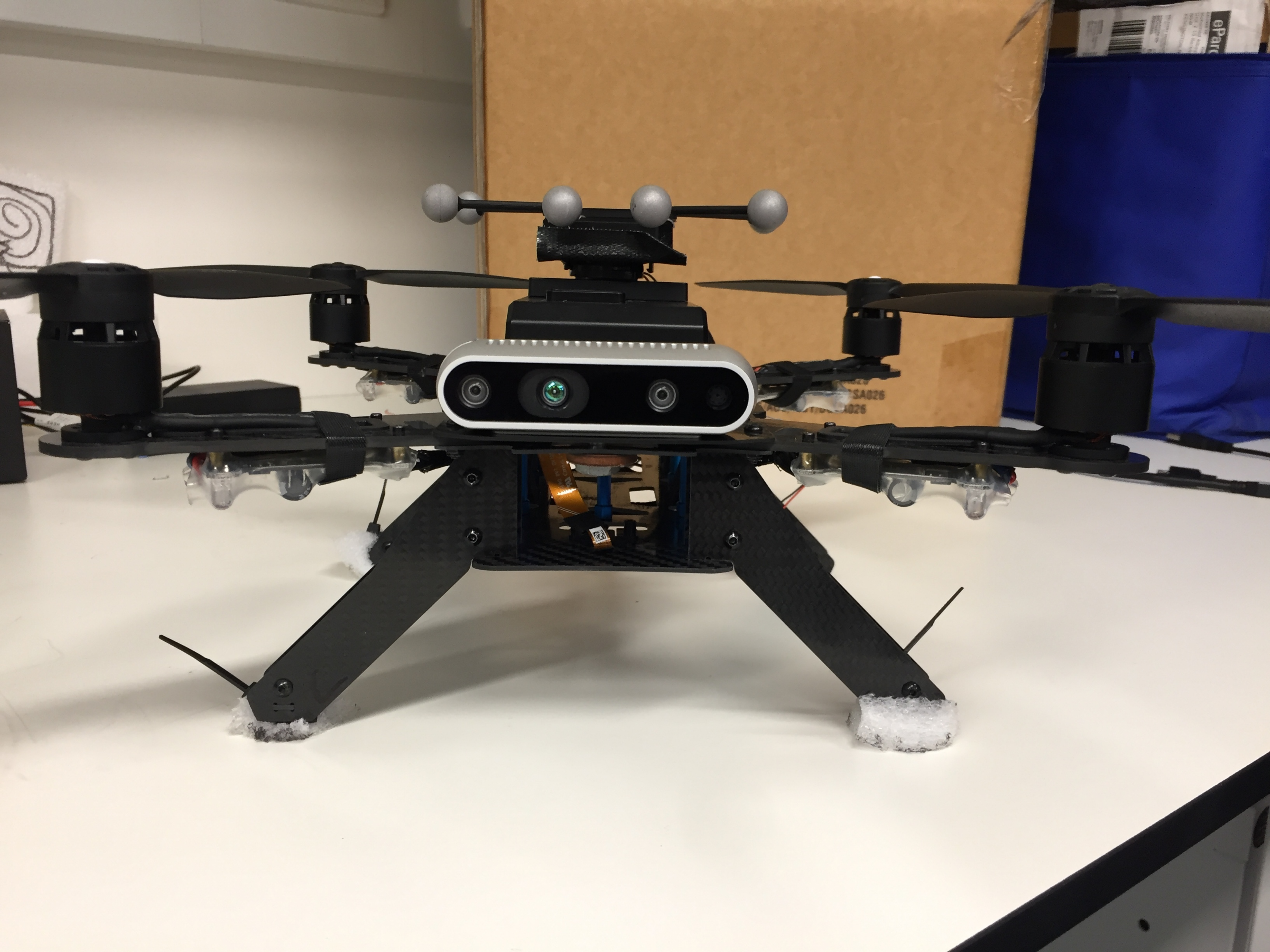} 
			\caption{}
			\label{fig:ch1:quadrotor1}
		\end{subfigure}
		\hfill
		\begin{subfigure}[t]{0.46\textwidth}
			\centering
			\includegraphics[width=\linewidth]{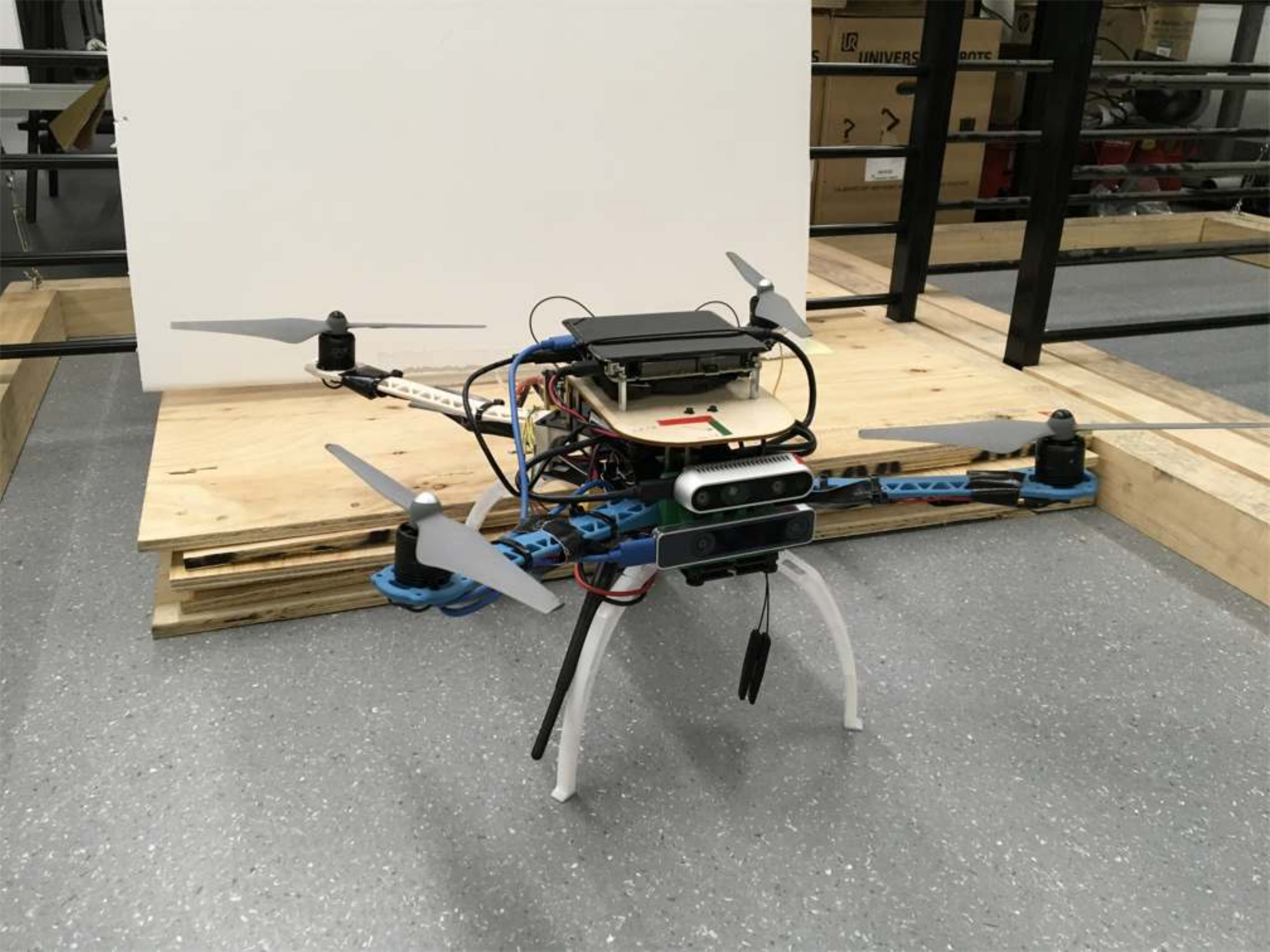} 
			\caption{}
			\label{fig:ch1:quadrotor2}
		\end{subfigure}
		
		\caption{Quadrotors used for experimental evaluation}
		\label{fig:ch1:quadrotors}
	\end{adjustbox}
\end{figure}

\section{Report Outline}

This report is organized into two main parts.
Part~I deals with navigation problems of single-UAV systems which includes chapters~\ref{cha:methods_hybrid}-\ref{cha:tunnel_navigation}, and part~II addresses problems related to motion coordination control strategies for multi-UAV systems which are covered in chapters~\ref{cha:flocking_control}-\ref{cha:coverage_control}.
Each chapter is presented in a complete manner containing problem formulation, proposed control strategy, mathematical analysis and validation through simulations and/or experimental results.
These chapters can be briefly described as follows:
\begin{itemize}%
	\item \textbf{Chapter~\ref{cha:litreivew}} %
	presents an overview of the components required for safe autonomous navigation of UAVs which is also common with other mobile unmanned robots such as ground and underwater vehicles.
	It also provides a general literature review of recent research addressing motion control, 3D collision avoidance and perception related to UAVs.

	\item \textbf{Chapter~\ref{cha:methods_hybrid}} %
	suggests a hybrid navigation approach which can act as a base framework to achieve reliable collision-free motions.
	This framework combines reactive collision avoidance with global path planning methods to provide better solutions in partially-unknown environments.
	As a general framework, it is presented considering only 2D movements for simplicity; however, it can be extended to adopt the 3D reactive methods developed in the subsequent chapters.
	
	\item \textbf{Chapter~\ref{cha:methods_reactive3D}} %
	proposes a novel 3D reactive navigation strategy for UAVs with obstacle avoidance capabilities.
	The developed approach can utilize the UAVs full capabilities in doing 3D maneuvers for obstacle avoidance in contrast to many of the existing 2D reactive methods.
	The method is developed using a general 3D nonholonomic kinematic model which is applicable to different UAV types in addition to autonomous underwater vehicles.
	
	\item \textbf{Chapter~\ref{cha:reactive_impl}} %
	extends the work done in chapter~\ref{cha:methods_reactive3D} by proposing control laws for quadrotor UAVs based on the differential-flatness property of quadrotor dynamics and the sliding mode control technique.
	Implementation details and experimental results are also provided in this chapter.
	
	\item \textbf{Chapter~\ref{cha:deforming_approach}} %
	also proposes a different 3D collision-free navigation method adopting the concept of real-time path deformation.
	This method relies on light processing of sensors measurements which make it considered as reactive to provide quick responses around unknown and dynamic obstacles.
	
	\item \textbf{Chapter~\ref{cha:tunnel_navigation}} %
	proposes a novel 3D tunnel navigation strategy with light computational complexity.
	The strategy generates motion commands directly based on sensors observations without the need for accurate localization.
	It can produce 3D maneuvers to navigate in very complex 3D tunnel-like environments in contrast to many of the available reactive methods which consider only 2D motions that are applicable to certain tunnel shapes.
	Simulation and experimental validations are provided with detailed technical discussion about some practical aspects. 
	
	\item \textbf{Chapter~\ref{cha:flocking_control}} %
	deals with the flocking problem related to motion coordination of multi-vehicle systems.
	It proposes a distributed control approach to ensure that each vehicle within the system can achieve four main objectives which are: avoiding collisions with other vehicles, avoiding collisions with obstacles, maintaining its position within the group to achieve some formation, and reaching a goal region in consensus with the other vehicles.

	\item \textbf{Chapter~\ref{cha:coverage_control}} %
	tackles a motion coordination problem with different global objective which is to perform coverage tasks in 3D environments using multi-vehicle systems.
	Coverage control methods with different control laws are proposed considering a general kinematic model that is applicable to different UAV types.
	The methods are also extended to quadrotor UAVs where control laws based on their dynamical models are developed.
	
	\item \textbf{Chapter~\ref{cha:conclusion}} %
	concludes the work presented in this report highlighting the main contributions made towards the state-of-the-art in these areas.
	It also outlines current ongoing research and potential directions for future work.
\end{itemize}

    \chapter{Literature Review\label{cha:litreivew}}

This chapter provides an overview of recent developments in the field of unmanned aerial vehicles related to autonomous navigation.
Since collision avoidance is a very critical component, a great part of this chapter focus on advanced methods capable of producing three-dimensional (3D) avoidance maneuvers.
The work presented in this chapter was published in \cite{elmokadem2021towards}

\section{Introduction}

Unmanned aerial vehicles (UAVs) have evolved greatly over the past decades with prevalent use in military and civilian applications such as search \& rescue \cite{goodrich2008supporting}, wireless sensor networks and the Internet of Things (IoT) \cite{li2018wireless,huang2018towards}, remote sensing \cite{pajares2015overview}, surveillance and monitoring \cite{savkin2019method,huang2020algorithm,savkin2020navigation}, 3D mapping \cite{nex2014uav}, objects grasping and aerial manipulation \cite{korpela2012mm,ruggiero2018aerial}, underground mines exploration \cite{li2020autonomous}, etc.
Challenges in developing UAVs keep increasing as the complexity of their tasks increases  especially with the aim of pushing towards fully autonomous operation (i.e. with least human interaction).
Moreover, many applications require UAVs to autonomously operate in unknown and dynamic environments in which they need to completely rely on onboard sensors to understand the environment they navigate in and complete their tasks efficiently.
The autonomous navigation problem can generally be defined as the vehicle's ability to reach a goal location while avoiding collisions with surroundings without human interaction.
This is a very challenging problem as it is important to achieve safe navigation to avoid causing damages or injuries.
Limitations on available technologies related to UAV add more complexities to the development of autonomous navigation methods in order to ensure reliability and robustness compared with unmanned ground vehicles (UGVs) and autonomous underwater vehicles (AUVs).
Examples of such are limitations on sensing capabilities, allowed payload capacity, flight time, energy consumption, communication, actuation and control effort.
Developing efficient and advanced motion control methods plays a critical role in minimizing the effect of these factors.
For example, adopting complex bio-inspired flying behaviors such as perching and maneuvering on surfaces can help extending mission flight time \cite{roderick2017touchdown}.

Many researchers have contributed towards addressing the navigation problem for UAVs.
This overview aims at surveying the developments made in the past ten years towards achieving fully autonomous operations.
Some key approaches developed earlier than the considered time frame is also reported for the sake of completion.
General definitions and research areas are also provided for new researchers interested in this field.
Additionally, a list of key open-source projects is provided which may aid in quick development and deployment of new approaches as part of a complete autonomous stack. %

A great focus of this review is dedicated to the more complex problem of three-dimensional (3D) obstacle avoidance utilizing the full maneuvering capabilities of UAVs. 
Given the fact that many of the existing algorithms are developed considering general 3D kinematic models, they are applicable to vehicles moving in 3D including different UAV types and autonomous underwater vehicles (AUVs).
Similarly, some of the general approaches developed for AUVs are also reported here given that they are applicable to UAVs.
Planar approaches usually consider flights at a fixed altitude to simplify the obstacle avoidance problem.
These approaches may fail with the increased complexity of the environments where UAVs are needed; hence, utilizing 3D avoidance maneuvers is more desirable.
However, some planar approaches are also reported here where they can potentially inspire extensions to more general 3D methods.

This chapter is organized as follows.
A general overview of existing UAV types, classifications, autonomous navigation paradigms, and control structures is given in \cref{sec:ch2:types}.
Next, many motion planning and obstacle avoidance techniques are surveyed in \cref{sec:ch2:planning}.
After that, \cref{sec:ch2:control} presents different control methods used for UAVs along with adopted dynamical models for different UAV types.
Brief information about existing localization and mapping techniques is provided in \cref{sec:ch2:localization}.
Additionally, some open-source projects and useful tools for UAV development are provided in \cref{sec:ch2:open_source}.
Finally, concluding remarks are made in \cref{sec:ch2:conclusion}.

\section{UAV Types, Autonomy \& System Architectures}\label{sec:ch2:types}

\subsection{UAV Types}

UAVs can be classified based on several factors such as size, mean takeoff weight, control configuration, autonomy level, etc.
For example, classifications of UAVs based on size according to the Australian Civil Aviation Safety Authority (CASA) are:
\begin{itemize}[noitemsep,topsep=0pt]
	\item Micro: less than $250g$
	\item Very Small: $0.25-2Kg$
	\item Small: $2-25Kg$
	\item Medium: $25-150Kg$
	\item Large: More than $150Kg$
\end{itemize}
Large UAVs are mainly used in tactical missions and military applications; for more detailed classifications related to military use, see \cite{valavanis2015handbook}.
Based on control configurations, UAVs can be categorized into (see \cref{fig:ch2:UAV_types}):
\begin{itemize}[noitemsep,topsep=0pt]
	\item single-rotor \cite{cai2005design,cai2008systematic,godbolt2013experimental,cai2013design}: helicopter 
	\item multi-rotor \cite{mahony2012multirotor,phang2014systematic,segui2014novel,verbeke2014design,kamel2018voliro,rashad2020fully}: tricopter, quadrotor, hexacopter, etc.
	\item fixed-wing \cite{shkarayev2007introduction,keane2017small,zhao2020structural}
	\item hybrid \cite{cetinsoy2012design,ozdemir2014design,ke2018design,chipade2018systematic}
	\item flapping wings \cite{gerdes2012review,karasek2014robotic,gerdes2014robo,hassanalian2017novel,icsbitirici2017design,holness2018characterizing,yousaf2020recent}: Ornithopters and Entomopters
\end{itemize}
Single-rotor aerial vehicles such as helicopters have not been utilized much as UAV platforms.
Multi-rotors on the other hand have become the most popular choice in most civilian applications when it comes to maneuverability.
Multi-rotors such as quadrotors, hexacopters and octocopters with fixed-pitch rotors share similar dynamical model for control.
However, quadrotors are cheaper, faster and highly maneuverable while hexacopters and octocopters can offer better flight stability, fault-tolerance and heavier payload.
Multi-rotors with fixed-pitch rotors are underactuated systems where it is not possible to completely control all degrees of freedom.
There have been recent advances in developing omnidirectional tilt-rotor UAVs which are fully actuated in 6DOF such as \cite{rashad2020fully,kaufman2014design,kamel2018voliro,allenspach2020design}. 

Multi-rotors in general lie under the category of vertical-takeoff-and-landing (VTOL) vehicles with the ability to hover in place.
On contrary, fixed-wing UAVs are horizontal-takeoff-and-landing (HTOL) vehicles, and they cannot hover at a certain position due to nonholonomic constraints.
Instead, they have to loiter around areas of interest.
However, fixed-wing UAVs have advantages such as long endurance and simpler mechanical structure compared to multi-rotors.
Hybrid UAVs combine both configurations of fixed-wings and multi-rotors utilizing the advantages of both such as vertical takeoff \& landing, hovering and long endurance flights. 
However, these vehicles are still under development, and more research is needed for a reliable control especially when switching between flight modes.

Another type of UAVs are those with flapping wings inspired from birds (Ornithopters) and insects (Entomopters).
They are still under development due to their complex dynamics and anticipated power problems \cite{gerdes2012review}. %
Recently, new bio-inspired hybrid unmanned vehicles have also been proposed to handle navigation in different domains such as underwater-aerial vehicles \cite{stewart2018design,stewart2019dynamic,tan2019design} and aerial-ground vehicles \cite{kalantari2014modeling,mulgaonkar2016flying,yamada2017development,sabet2019rollocopter,atay2021spherical}.

\begin{figure}[!htb]
	\centering
	\begin{adjustbox}{minipage=\linewidth,scale=1.0}
		\begin{subfigure}[t]{0.48\textwidth}
			\centering
			\includegraphics[clip, width=\linewidth]{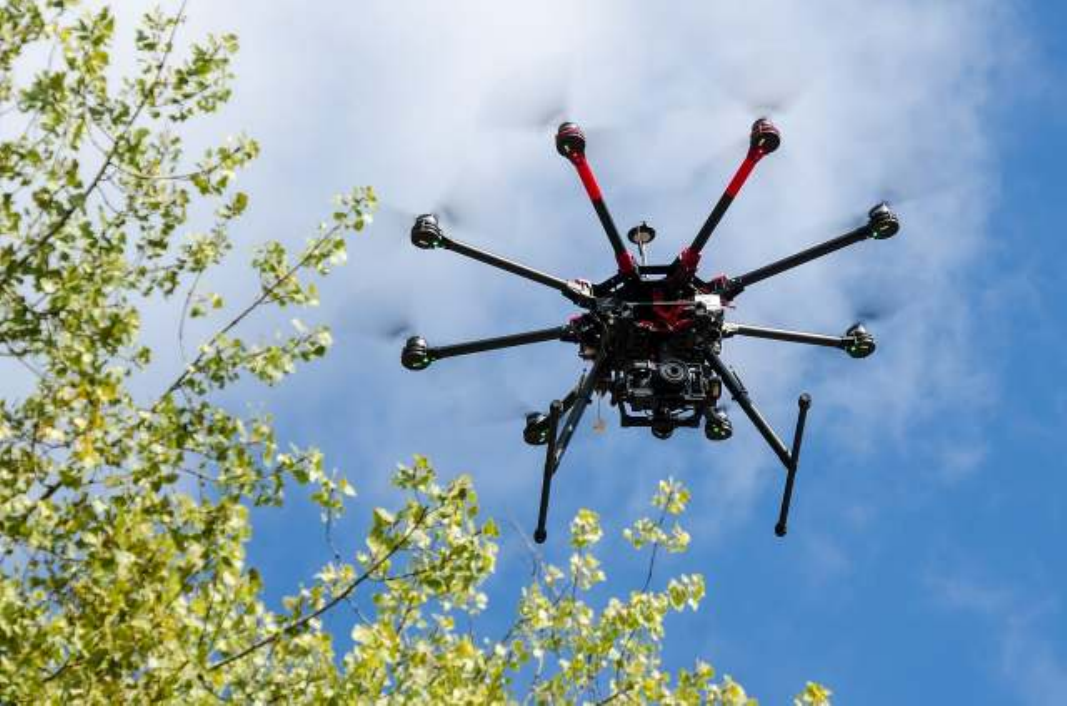} 
			\caption{Multirotor (Hexacopter)}
		\end{subfigure}
		\hfill
		\begin{subfigure}[t]{0.48\textwidth}
			\centering
			\includegraphics[clip, width=\linewidth]{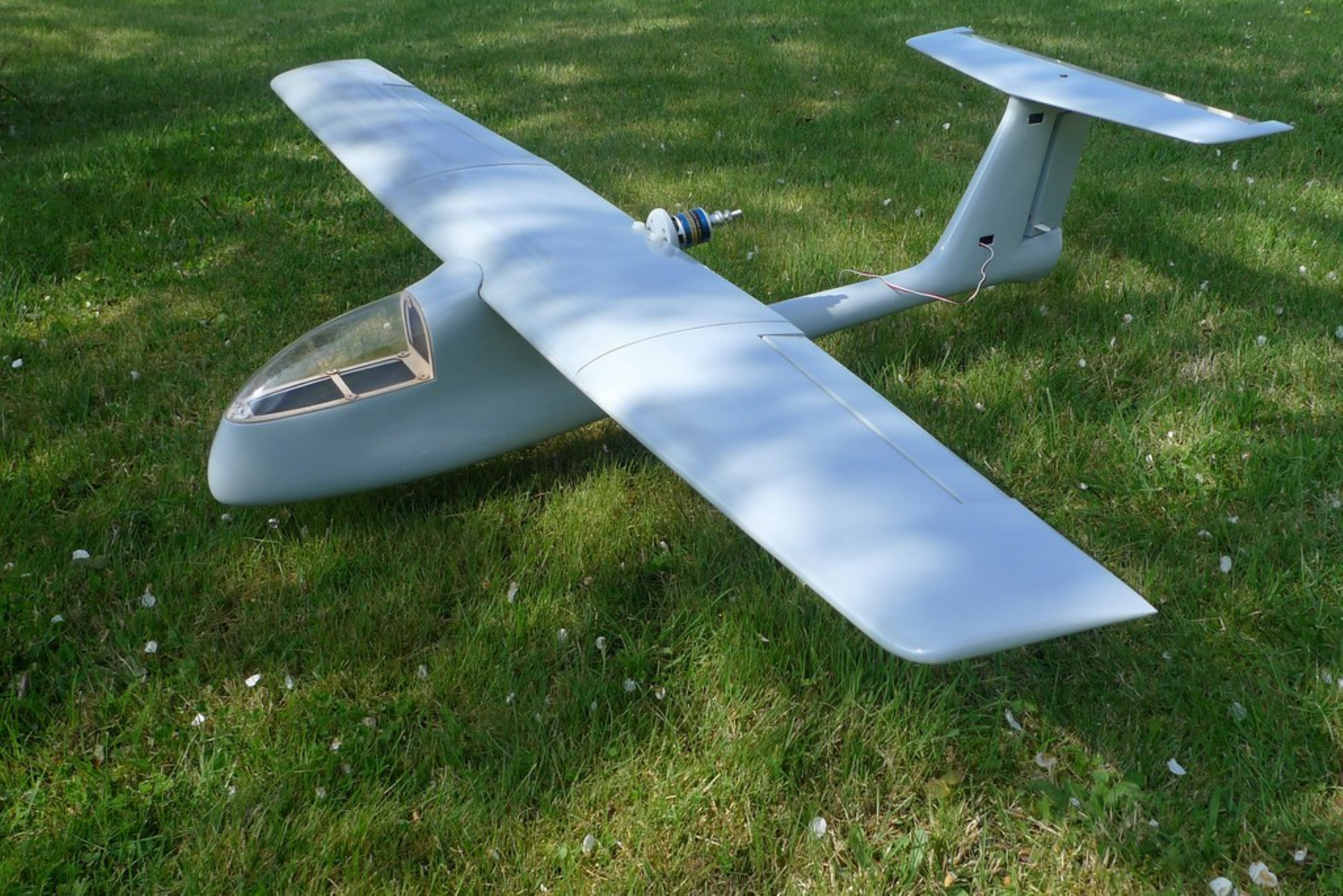}
			\caption{Fixed-Wing}
		\end{subfigure}

		\begin{subfigure}[t]{0.48\textwidth}
			\centering
			\includegraphics[clip, width=\linewidth]{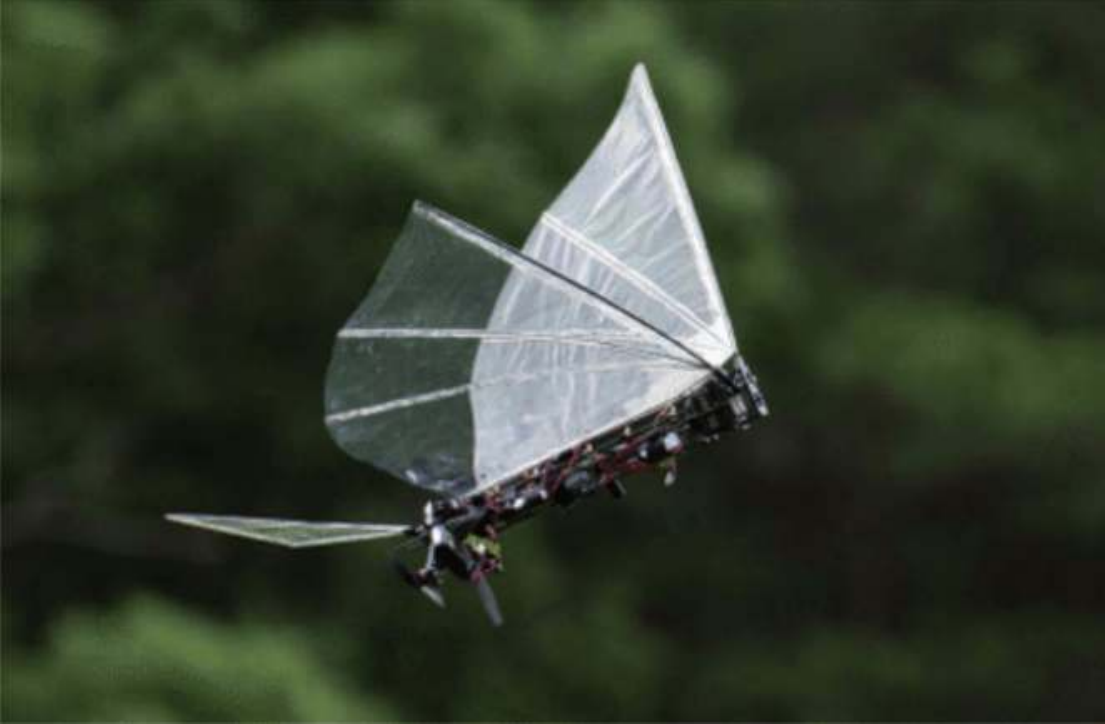} 
			\caption{Ornithopter flapping-wing UAV (Robo Raven) \cite{holness2018characterizing}}
		\end{subfigure}
		\hfill
		\begin{subfigure}[t]{0.48\textwidth}
			\centering
			\includegraphics[clip, width=\linewidth]{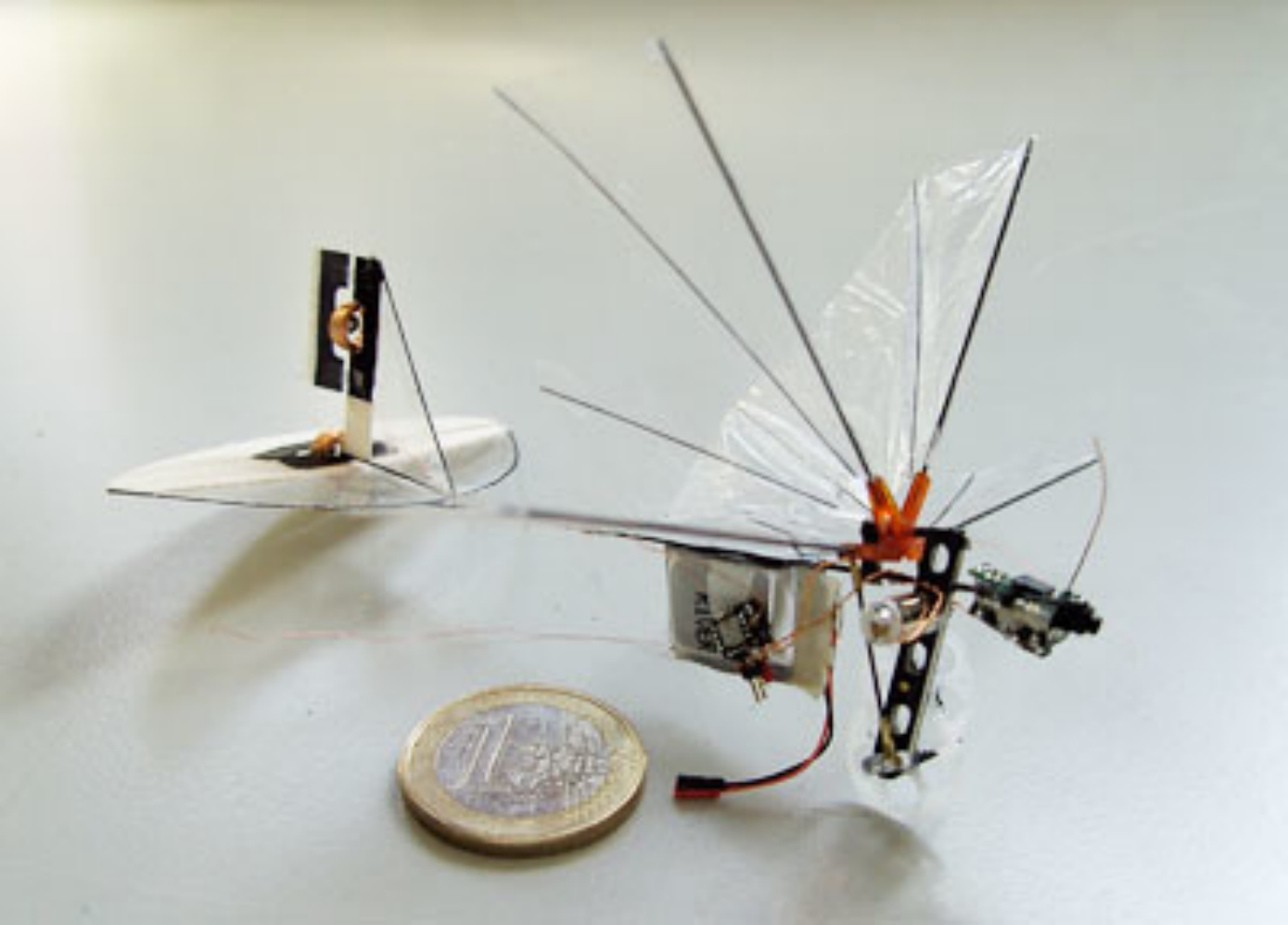}
			\caption{Entomopter flapping-wing UAV (DelFly Micro)}
		\end{subfigure}
		\end{adjustbox}
		\caption{Different UAV types based on control configurations}
		\label{fig:ch2:UAV_types}
\end{figure}

\subsection{Autonomy Levels}

Being completely able to carry out missions/tasks with least human interaction is an ultimate goal for unmanned aerial vehicles.
Different levels of autonomy can be achieved towards that goal depending on the complexity of tasks and whether a fully autonomous solution exists or not for that specific application.
These levels can be described based on the UAV mode of operation according to the National Institute of Standards and Technology (NIST) as follows \cite{huang2004autonomy}:
\begin{itemize}[noitemsep,topsep=0pt]
	\item \textbf{Fully autonomous:} UAV can carry out a delegated task/mission without human interaction where all decision making are made onboard based on sensors observations adapting to operational and environmental changes. 
	\item \textbf{Semi-autonomous:} A human operator is needed for high-level mission  planning and to interact during the movement when some decisions are needed that the UAV is not capable of making.
	The vehicle can maintain autonomous operation in between these interactions.
	For example, an operator can provide a list of waypoints to guide the vehicle where it can manage to move safely towards these positions with obstacle avoidance capability.
	\item \textbf{Teleoperated:} The remote operator relies on onboard sensors feedback to move the vehicle either by directly sending control commands or intermediate goals with no obstacle avoidance capabilities. This mode can be used in Beyond-Line-of-Sight (BLOS) applications.
	\item \textbf{Remotely controlled:} A remote pilot is needed to manually control the UAV without sensors feedback which can be used in Line-of-Sight (LOS) applications.
\end{itemize}

\subsection{Towards Fully Autonomous Operations}

Developing a fully autonomous UAV is a very challenging and complex problem.
A modular approach for both hardware and software architectural design is commonly adopted in the literature by most existing autonomous UAVs for a simpler and fault-tolerant solution.

On the hardware level, a UAV in a simplest form consists of a frame, a propulsion system and a \textit{Flight Control System (FCS)}.
The UAV size and propulsion system can be designed to support the needed payload and flight time as per mission requirements.
A propulsion system consists of a power source (ex. batteries, fuel cells, micro-diesels and/or micro gas turbines), electronic speed controllers (ESCs), DC Brushless motors, propellers and/or control surfaces (ailerons, flaps, elevators, and rudders).

A \textit{flight control system} is simply an embedded system consisting of the autopilot, avionics and other hardware directly related to flight control \cite{valavanis2015handbook}.
For example, main sensors critical to flight control include inertial measurement unit (IMU), barometers/altimeters, and GNSS (for outdoor use).
A computing unit (ex. a microcontroller), with real-time constraints, is usually used to implement the autopilot logic for a reliable and fault-tolerant flight control.
FCS is responsible for computing low-level control commands, estimating the vehicles states (altitude, attitude, velocity, etc.) based on sensors data, logging critical information for post-flight analysis and interfacing with higher level components either by wired connection or through other communication links.
Having a FCS is enough to allow teleoperation navigation mode where a remote operator can directly send waypoints and/or control commands.
It is also possible to achieve semi-autonomous operations in simple environments where reactive control methods with low computational cost are implemented within the autopilot to provide basic collision avoidance capabilities.

For more complex tasks/missions, it is required to have an onboard computer with higher processing power, namely a mission computer, to achieve fully autonomous operations given that a UAV with proper size and power is used.
In this structure, the mission computer usually implements the high-level mission and motion planning by relying on information interpreted from high-bandwidth sensory data in addition to running required processes with expensive computational cost.
It can also have its own communication link with a Ground Control Station (GCS) to stream high-bandwidth data such as images and depth point clouds.

Different kind of sensors can be used for advanced perception and planning which depends on the mission requirements, UAV available payload and power, and environmental conditions. 
Examples of commonly used sensors are cameras (monocular, RGBD, thermal, hyperspectral, etc.), range sensors (LiDAR, RADAR, ultrasonic) and other task-specific sensors.
A summary of the hardware and software components used with UAVs is shown in \cref{fig:ch2:uav_components}, and an example hexacopter is presented in \cref{fig:ch2:uav_setup} showing the system components.

The software architecture of the autonomous stack implemented on the mission computer typically consists of several processes/modules running in parallel and a messaging middleware is used to interchange messages between processes on the mission computer or with other computers on the same network (for example, in multi-UAV systems).
Some of these modules are related to the mobility aspects that can ensure a safe navigation which can be common among most UAV systems and other autonomous mobile robots.
Other modules would implement logic that is application-specific such that the UAV can autonomously perform the delegated task.
For example, in fire-fighting applications, a UAV is needed to autonomously locate and extinguish fires which requires additional modules to be included within the autonomous stack based on computer vision and extinguisher control mechanism.
In many remote sensing applications, the main task could be only collecting data whether images or information from other onboard sensors to be analyzed post-flight.
Mobility-related modules are the core components needed to ensure collision-free navigation in all applications.
By considering only the mobility-related components, popular modular structure for autonomous navigation is adopted in the literature which consists of the following modules/subsystems (\cref{fig:ch2:uav_modular_design}):
\begin{itemize}[noitemsep,topsep=0pt]
	\item Perception
	\item Localization \& Mapping
	\item Motion Planning \& Obstacle Avoidance
	\item Control
\end{itemize}
This modular approach of addressing the navigation problem offers a flexible expandable design with fault-tolerance.
However, other possible designs can also be seen for less complex tasks or for vehicles with very limited resources by coupling control and planning without the need for localization and mapping in a reactive fashion as will be shown in the next section.

\begin{figure}[!htb]
	\centering
	\includegraphics[clip, width=\linewidth]{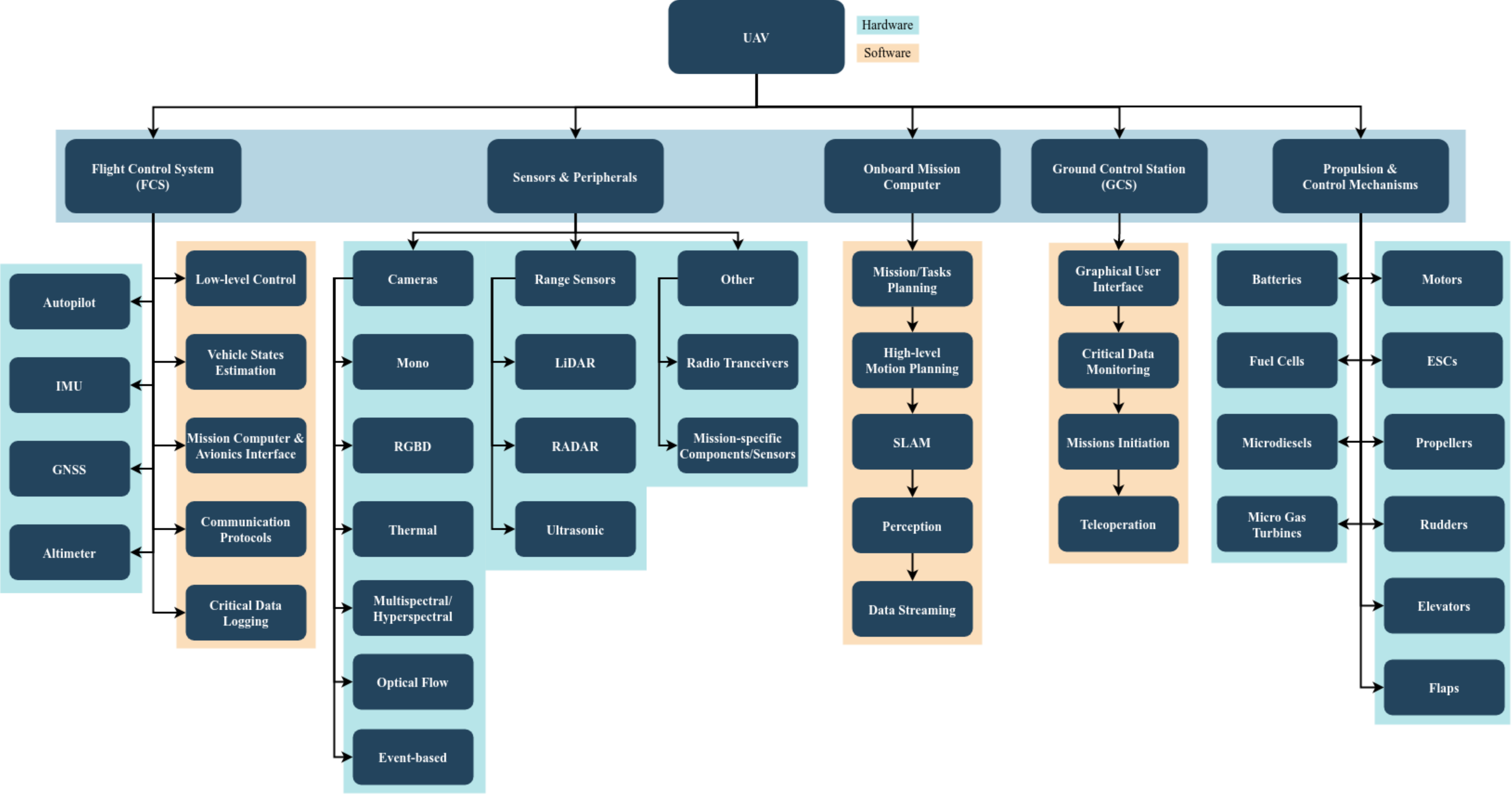}
	\caption{System architecture including hardware and software components commonly used with UAVs}
	\label{fig:ch2:uav_components}
\end{figure}

\begin{figure}[!htb]
	\centering
	\includegraphics[clip, width=\linewidth]{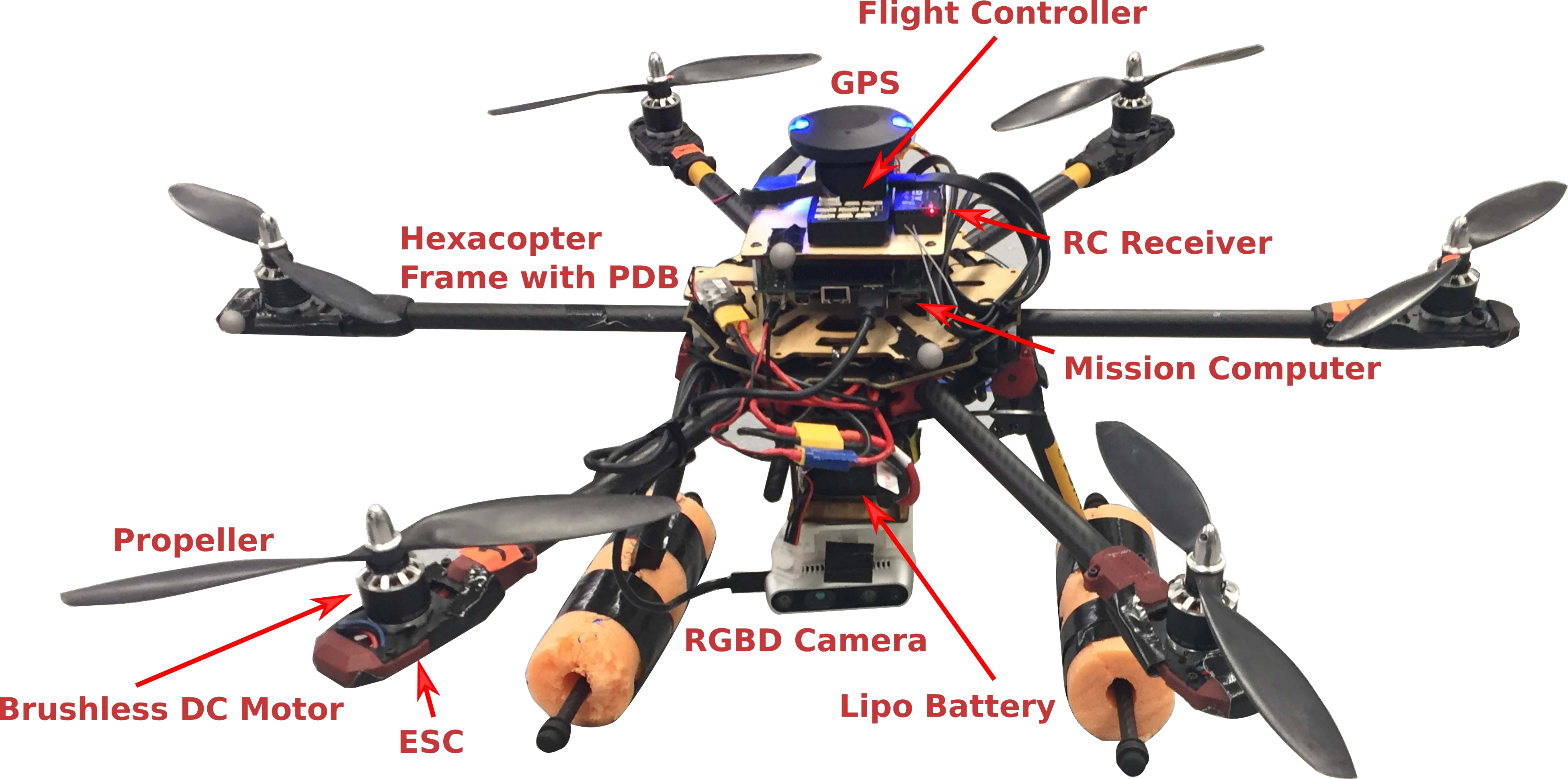}
	\caption{Example UAV setup of a hexacopter UAV type with FCS, mission computer and an RGBD sensor}
	\label{fig:ch2:uav_setup}
\end{figure}

\begin{figure}[!htb]
	\centering
	\includegraphics[clip, width=0.25\linewidth]{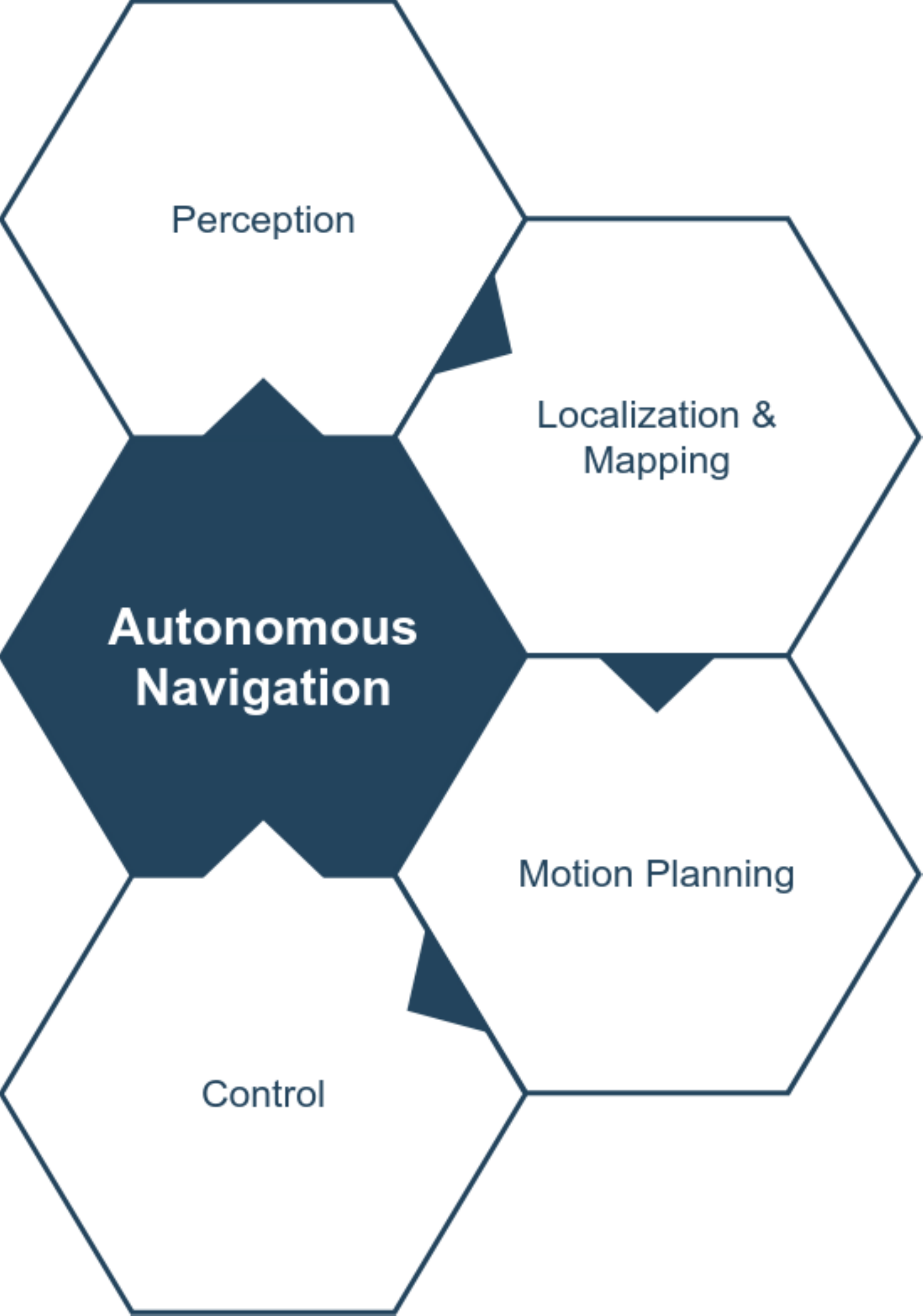}
	\caption{Modular software structure for UAV navigation stack}
	\label{fig:ch2:uav_modular_design}
\end{figure}

\section{UAV Navigation Techniques} \label{sec:ch2:planning}

A crucial part for autonomous navigation is to ensure that the vehicle can move while avoiding collisions with its surroundings.
This is a general problem in robotics which can be addressed by motion planning or reactive control. %
Generally, the motion planning problem can roughly be described as trying to find collision-free trajectories between initial and final \textit{configurations} while satisfying some kinematic and dynamic constraints.
A \textit{configuration} in this case refers to the position and orientation of a mobile robot where a \textit{configuration space} is the set of all possible configurations.
The dimension of the configuration space equals the number of controllable degrees of freedom.
For example, planning motions for quadrotors can be done in a space of their 3D position coordinates and heading (yaw) angle while motions for omnidirectional (fully actuated) UAVs can be planned considering all translational and rotational states (6DOF).

In a decoupled approach, the UAV control system can execute motions planned by a high level system, namely a motion planner, where these plans need to be feasible and safe (i.e. collision-free).
In other implementations, the motion planning can be coupled with the control system design where reactive control laws are developed to directly generate obstacle avoidance maneuvers based on sensors measurements.
Some refer to those in the literature in loose terms as obstacle/collision avoidance methods.
The term collision avoidance is mostly used by the UAV research society in referring to avoiding collisions with other cooperative or non-cooperative aerial vehicles (i.e. dynamic obstacles) sharing the same flight space while the term obstacle avoidance may be used more often in indoors, industrial and urban environments where the flight space is filled with other static/dynamic obstacles.
That is, high-altitude flights commonly adopt the collision avoidance terminology and low-altitude flights may use the more general obstacle avoidance term.
This terminology is also adopted more often in multi-UAV systems to differentiate between methods that only consider collision avoidance among the vehicles within the system to those that also consider obstacle avoidance in obstacle-filled environments.

\subsection{Navigation Paradigms}

Existing navigation techniques for autonomous mobile robots in general can be classified into \textit{deliberative} (global planning), \textit{sensor-based} (local planning) or \textit{hybrid} (see \cref{fig:ch2:Nav_paradigms}).
\textit{Deliberative} approaches require a complete knowledge of the environment represented as a map.
Global path planning methods can then be used to search for safe and optimal paths.
Classical path planning algorithms can be categorized into: 
\begin{itemize}[noitemsep,topsep=0pt]
	\item Search-based (ex. Djikstera, $A^*$, $D^*$, etc.)
	\item Potential Field (ex. navigation function, wavefront planner, etc.)
	\item Geometric (ex. cell decomposition, generalized Voronoi diagrams, visibility graphs, etc.)
	\item Sampling-based (ex. PRM, RRT, RRT*, FMT, BIT, etc.)
	\item Optimization-based (PSO, genetic algorithms, etc.)
\end{itemize}
These methods can find optimal paths if one exists at the expense of requiring full knowledge about the environment which is not suitable in unknown and dynamic environments.
For more detailed information about such planning methods, the reader is referred to \cite{lavalle2006planning}.

On the other hand, \textit{sensor-based} methods rely directly on current sensors measurements or a short history of the sensors observations (i.e. a local map) to plan safe paths in real-time.
The planning horizon can typically be very short for some period ahead of time or it could be done at each control update cycle in a receding horizon fashion.
A very special class of such methods is reactive approaches where sensors measurements are coupled to control actions either directly \cite{hoy2015algorithms} or after light processing \cite{tobaruela2017reactive}.
Sensor-based methods offer solutions with great computational performance which makes them favorable for navigation problems in unknown and dynamic environments.
These methods do not generate optimal solutions as they do not utilize the information acquired about the environment during the motion.
However, it is common to sacrifice optimality for computation speed especially when considering micro UAVs with fast dynamics and limited computing power.
Sensor-based methods are also prone to getting stuck sometimes due to local minimum.

\textit{Hybrid} approaches combine both deliberative and sensor-based methods to generate a more advanced navigation behavior benefiting from the advantages of both classes.
It relies on low-latency local planning or reactive control to handle unknown and dynamic obstacles while using a high-level global planning method to guide the vehicle utilizing accumulated knowledge about the environment.

\begin{figure}[!htb]
	\centering
	\includegraphics[clip, width=0.75\linewidth]{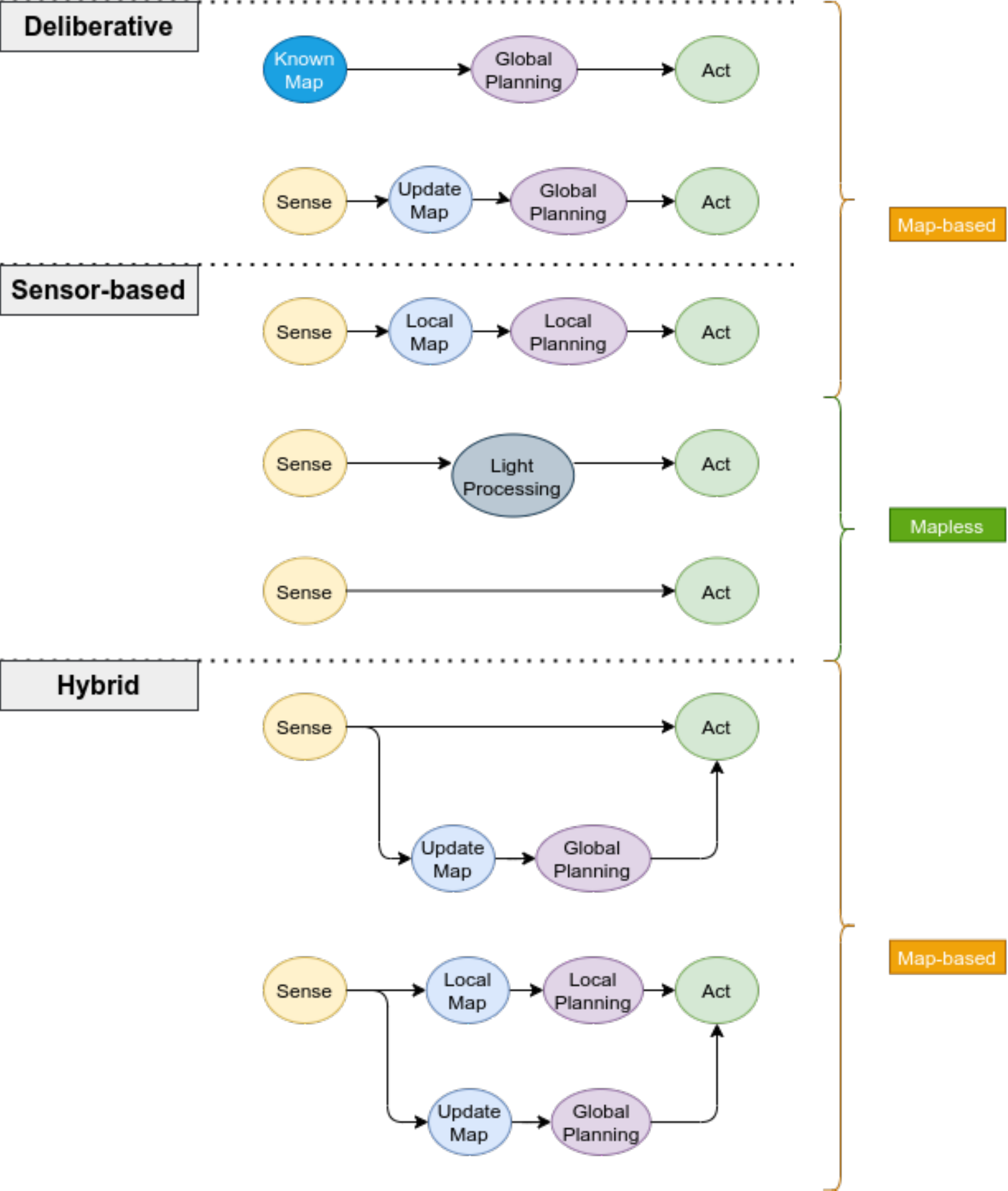}
	\caption{Different paradigms adopted for autonomous navigation from a high-level prospective}
	\label{fig:ch2:Nav_paradigms}
\end{figure}

\subsection{Map-based vs Mapless}

Navigation methods can alternatively be classified into \textit{map-based} or \textit{mapless} approaches \cite{desouza2002vision,bonin2008visual}.
This classification highlights the computational complexity and whether they rely on accurate localization and mapping or not.

\textit{Map-based} strategies require a local (or global) map representation of the environment which can be provided before navigation starts (deliberative approaches) or it can be built during navigation based on sensors measurements (some sensor-based approaches). 
Safe paths can then be found using local/global planning algorithms based on either \textit{metric} or \textit{topological} maps.
Therefore, such methods are demanding in terms of computational resources, planning time and memory requirements which is highly dependent on the environment size and its complexity.
Nevertheless, local map-based methods are very commonly used with UAVs to generate locally optimal solutions ought to technological advances where it is possible to have mini light-weight computers with high processing power onboard.

On contrary, \textit{mapless} strategies (reactive methods) rely directly on sensors measurements to make motion decisions without the need for maintaining global maps and accurate localization (except when using GNSS).
Hence, control actions can be directly coupled with either visual clues from image segmentation, optical flow or features tracking in subsequent frames in vision-based methods \cite{bonin2008visual} or interpreted information from range sensors and 3D point clouds such as relative-distance to obstacles, gaps or bounding objects.
These methods offer the best computational complexity for obstacle avoidance as control is coupled with planning through light processing of sensors data which can provide very quick reflex-like reactions to obstacles.
Some of the challenges when developing purely reactive navigation methods is the possibility to get stuck in local minimums, and limited field-of-view (FOV) may affect the overall performance.
Also, fast reactions to obstacles achieved by reactive methods come at the cost of generating non-optimal solutions in some cases due to the fact that they do not utilize information about previously sensed obstacles.

\subsection{Overall Navigation Control Structure}

Form a control prospective, different structures were adopted in the literature to deal with the high complexity of the navigation problem.
As mentioned before, the most common structure is based on decoupling planning and control due to its simplicity in design.
One can categorize the existing methods into seven different control structures as shown in \cref{fig:ch2:Nav_control}.
Structures I-III show the general decoupled approach where motion planning and control are decoupled while structure IV is used by reactive approaches which directly couple planning and control.
Structures V-VII correspond to hybrid approaches which can be a combination of structures I-IV.

In decoupled approaches, some motion planning methods simplify the problem by subdividing it into two stages.
The first stage simply tries to find a collision-free geometric path satisfying kinematic constraints.
Constraints can be considered directly in the planning algorithm, or the whole process can be further decomposed into finding a safe path first ignoring such constraints then applying path smoothing techniques to satisfy the kinematic constraints.
Then, it is followed by a trajectory generation stage to obtain feasible trajectories satisfying dynamic constraints.
Other approaches tackled this problem by directly planning trajectories using optimization-based methods which is a harder problem to solve.

In order to differentiate between different motion planning algorithms, the difference between \textit{path planning} and \textit{trajectory planning/generation} should be understood.
\textit{Path planning} is the process of finding a geometric collision-free path between starting and end positions without a timing law.
In \textit{trajectory planning}, a timing law is associated with the planned collision-free geometric path represented as a \textit{trajectory} which includes information about higher derivatives (i.e. velocity, acceleration, etc.).
Trajectories are mostly planned to satisfy dynamic constraints which can then be passed to a control system adopting a trajectory tracking control design.
One of the simplified approaches for trajectory planning is by combining a path planning algorithm with a trajectory generation method.
For example, a path planner could be used to generate a smooth geometric path which is then passed to a trajectory planner to generate a feasible trajectory characterized by position, velocity and acceleration satisfying some dynamic constraints.

In the following subsections, we will survey recent works adopting local motion planning or reactive paradigms in accordance with the considered control structures.

\begin{figure}[!htb]
	\centering
	\includegraphics[clip, width=0.75\linewidth]{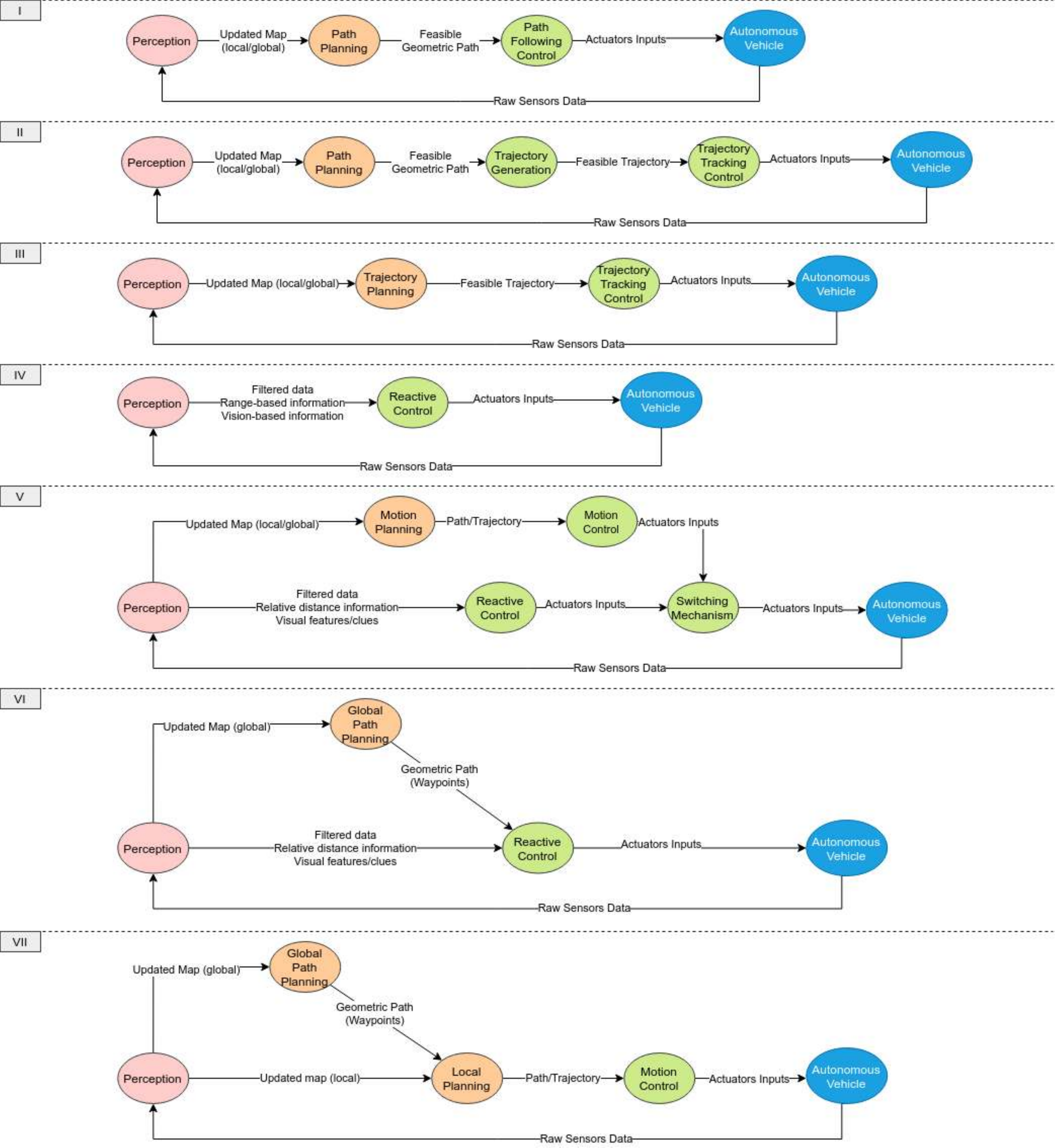}
	\caption{Different autonomous navigation control structures}
	\label{fig:ch2:Nav_control}
\end{figure}

\subsection{Local Path Planning}
A number of existing methods treat the problem through applying path planning algorithms locally to find feasible geometric paths assuming a general 2D/3D kinematic model.
Examples of these methods include sampling-based \cite{Georges2017,yang2010efficient,lin2017sampling,schmid2020efficient}, graph-based \cite{liu2016high,sanchez2019real} and optimization-based \cite{miller20113d,chen2016uav,roberge2012comparison}.
These methods are developed at a high level considering only kinematic constraints assuming a low-level path following controller exists to execute the planned paths while satisfying the dynamic constraints similar to control structure I.
They can also be combined with a trajectory generation method similar to structure II.

Adopting sampling-based methods helps addressing the high dimensionality problem of the 3D search space to generate collision-free paths in real-time which was considered in some of these works.
In \cite{yang2010efficient}, a planning algorithm was proposed for rotary-wing UAVs.
It decouples the motion planning problem into two stages, namely path planning and path smoothing, which is a common approach to simplify the problem especially when nonholonomic constraints needs to be satisfied (ex. for fixed-wing UAVs); for example, see \cite{roberge2018fast,sahingoz2014generation}.
A sampling-based planning algorithm, namely RRR, was adopted to search for collision-free paths followed by a path smoothing algorithm such that the smoothed path can satisfy curvature continuity and nonholonomic constraints.
An analytical solution for the adopted path smoothing algorithm was also presented in \cite{yang2010analytical} considering smoothing of 3D paths.
An explicit path-following model predictive control (MPC) was used in \cite{yang2010efficient} to ensure that the vehicle can track the planned paths, and it was formulated based on a linear model of the motion with no constraints.
Another real-time path planning algorithm was suggested in \cite{Georges2017} based on chance-constrained rapidly exploring random trees (CC-RRT) for safe navigation in 2D constrained and dynamic environments.
The motion planning relies on a proposed clustering-based trajectory prediction to model and predict future behavior of dynamic obstacles.
This motion prediction algorithm combines Gaussian processes (GP) with the sampling-based algorithm RRT-Reach to cope with GP shortcomings such as the high computational cost.
Another RRT variant, namely Closed-Loop RRT, was used in \cite{lin2017sampling} to handle navigation in 3D dynamic environments.
In \cite{schmid2020efficient}, a sampling-based approach was adopted in an informative path planning framework where the goal is to generate safe paths that can maximize the information gathered during movement which is important in exploring unknown environments. 

Some other works formulated the 3D path planning problem as an optimal control problem such as \cite{miller20113d,chen2016uav}.
The authors of \cite{miller20113d} formulated the optimal control problem in 2D to satisfy time and risk constraints as the 3D optimal control problem would be harder to solve.
Then, a 3D path was approximated in a final stage based on a terrain height map.
On contrarily, the method in \cite{chen2016uav} presented a path planner based on a 3D optimal control problem formulation where a model based on artificial potential field (APF) is used.  
Other optimization-based methods considered parallel genetic algorithm and particle swarm optimization as in \cite{roberge2012comparison}.

\subsection{Local Trajectory Planning}

A more popular approach in addressing the local planning problem for UAVs is through planning feasible trajectories to further satisfy dynamical constraints and optimality of path smoothness with respect to higher derivatives enabling high-speed and aggressive flights.
Generating smooth trajectories is important for high-speed applications to avoid sudden changes in actuators' accelerations and mechanical vibrations problems \cite{gasparetto2015path}.
Therefore, it can be seen from the literature that control structures II-III are commonly used for aggressive maneuvers whether by combining path planning and trajectory generation as in \cite{mellinger2011minimum,richter2016polynomial,Faculty2016,Oleynikova2016,liu2016high,liu2017planning,watterson2015safe,liu2017robust,spedicato2017minimum} or by direct trajectory planning as in \cite{ryll2019efficient,tordesillas2019fastrap,tordesillas2020mader,tordesillas2021panther,chen2021computationally,ye2020tgk,bucki2020rectangular,ji2020mapless,quan2020eva,lee2021autonomous}.

A trajectory generation method for quadrotors was suggested in \cite{mellinger2011minimum} to find minimum-snap trajectories between specified keyframes provided by a high-level planner with corridor-like constraints representing convex decompositions of free space.
This idea was adopted in several research works such as \cite{Faculty2016,liu2016high,liu2017planning,Mohta2018,watterson2015safe}.
The work \cite{Faculty2016} formulated the trajectory generation as a mixed-integer optimization problem to generate minimum-jerk polynomial trajectories constrained to convex collision-free regions with other constraints on velocity and acceleration. %
The authors have also proposed a way to generate the safe convex regions using Iterative Regional Inflation by Semi-definite programming (IRIS) which was initially proposed in \cite{deits2015computing}.

Similarly, a real-time trajectory generation method was proposed in \cite{liu2016high} for quadrotors suggesting another way of determining such safe convex regions.  %
It relies on online built voxel maps and short range planning algorithm where it uses $A^*$ search method to find a safe path in a discretized graph representation of the voxel map.
The generated path is then inflated using a set of polygons specifying the collision-free regions around the path resulting in corridor-like constraints.
This approach was further developed in \cite{liu2017planning} to provide a more robust and efficient solution which was implemented in \cite{Mohta2018} showing a complete system for autonomous flights of multi-rotors in GPS-denied indoors environments.
A minimum-jerk trajectory is then computed similar to the approach in \cite{watterson2015safe} where a convex optimization problem is formulated by confining the trajectory spline segments to be within specified flight corridors with constraints to ensure the continuity of the trajectory splines.
This approach avoids the more complex non-convex problem formulation that results when considering the trajectory planning problem with constraints corresponding to collisions with obstacles.

The works \cite{liu2017planning,watterson2015safe} adopt a receding horizon planning paradigm to plan trajectories over finite time intervals with safe stopping policies in case of planning failure.
The works \cite{liu2016high} and \cite{watterson2015safe} adopt a short-range planning paradigm where a set of candidate goals within the current sensing FOV are used for trajectory planning until the global goal is reached.
In contrast to expressing collision-free constraints as convex decomposition of free space, \cite{tordesillas2020mader} suggested a different approach to efficiently handle dynamic and cluttered environments by using planes to represent the separation between the polyhedral representations of each trajectory segment.

Another optimization-based method was suggested in \cite{richter2016polynomial} as an extension to \cite{mellinger2011minimum} by formulating the minimum-snap trajectory generation problem as an unconstrained quadratic program (QP).
This trajectory generation can be combined with a 3D kinematic planner to generate safe geometric paths where the authors have considered the RRT* planner in their implementation.
Additional iterative steps are needed if the generated trajectories were found in collision where the optimization problem is repeatedly resolved using safe intermediate waypoints until a collision-free trajectory is obtained.

In contrast to optimization-based trajectory generation where dynamic constraints are considered in the optimization problem, motion primitives were considered as a simpler computationally efficient way to generate collision-free trajectories in 3D in some works such as \cite{mueller2015computationally,paranjape2015motion,lopez2017aggressive,tordesillas2019real,ryll2019efficient,gonzalez2020autonomous}.
Motion primitives offer a light-weight algebraic solution to the problem which can then be checked for dynamic constraints violation.
The low-computational cost of such methods allows for high-speed and aggressive movements since it is possible to quickly search over a large number of motion primitives to achieve a certain goal \cite{mueller2015computationally}.
Motion and sensing uncertainty were also considered in some methods at planning time such as \cite{gonzalez2020autonomous}.

Generally, considering dynamic constraints and constraints due to collisions with obstacles in the planning problem makes it harder to solve in real-time causing potential convergence problems.
This is known as kinodynamic planning which is a motion planning problem in a higher dimensional space with differential and obstacle constraints \cite{lavalle2001randomized}.
Some approaches have considered this idea rather than decoupling the path planning and trajectory generation such as \cite{liu2017robust,spedicato2017minimum,lindqvist2020nonlinear}. 
The work \cite{liu2017robust} addressed the trajectory planning problem as a 3D Optimal Control Problem (OCP) with soft obstacle avoidance constraints on a non-convex quadratic optimization problem.
To reduce the computational burden of solving the formulated OCP, constraints based on a reduced number of obstacles, the most threatening ones, were considered.
In \cite{spedicato2017minimum}, trajectory planning and control of quadrators in constrained environments was achieved through a formulation as a minimum-time optimal control problem with several constraints on states and inputs, and it was based on the full 6DOF dynamical model.
The general problem was reformulated using a change of coordinates and state-input constraints relaxation to reduce the high computational complexity of the original constrained problem.

The motion planning problem for multi-rotors among dynamic obstacles was tackled in \cite{lindqvist2020nonlinear} at the control level using a nonlinear model predictive controller (NMPC) based on a cost function in terms of the tracking error, input cost and input smoothness cost.
Addressing path planning using a pure NMPC structure is challenging as it is computationally expensive to solve nonconvex optimization problems in real-time.
Therefore, \cite{lindqvist2020nonlinear} considered a new solver for such nonlinear nonconvex problems known as Proximal Averaged Newton for Optimal Control (PANOC) \cite{sathya2018embedded,stella2017simple} to make the solution more appealing.
There exist an open-source implementation of this solver which is OpEn (Optimization Engine) \cite{sopasakis2020open}.
A similar approach was also considered in \cite{mansouri2020unified}.

Formulating the 3D trajectory planning as a Quadratic Program (QP) was also considered in \cite{Oleynikova2016,tordesillas2019fastrap,ye2020tgk}.
In \cite{Oleynikova2016}, the optimization-based method was proposed to generate locally optimal safe trajectories for multirotor UAVs using high-order polynomial splines. %
The optimization problem was formulated to minimize costs related to higher order derivatives of the trajectory (ex. snap) and collisions with the environment. 
The objective function computes collision costs using an Euclidean Signed Distance Field (ESDF) function with a voxel-based 3D local map of the environment.
The optimization problem was formulated as an unconstrained quadratic program (QP) so that it can be solved in real-time.
The work \cite{tordesillas2019fastrap} adopted a mixed-integer quadratic program formulation allowing the solver to choose the trajectory interval allocation, and the time allocation is found by a line search algorithm initialized with a heuristic computed from the previous replanning iteration.
Another kinodynamic planner for quadrotors was introduced in \cite{ye2020tgk} using a sampling-based method in combination with an additional optimization-based stage using a sequence of QPs to refine the smoothness and continuity of the obtained trajectory.

Recently, there have also been some growing interest in the field of perception-aware trajectory planning considering perception constraints in the planning problem.
The developed methods in this area takes into account perception quality to minimize state estimation uncertainty \cite{zhang2018perception} which can be done by keeping specific objects/features in the vehicle's sensing FOV \cite{tordesillas2021panther}.
Examples of such methods can be seen in \cite{zhang2018perception,tordesillas2021panther,Falanga2018pampc,murali2019perception,spasojevic2020perception,sheckells2016optimal}.

\subsection{Reactive Methods}

Most of the existing reactive methods are developed at a higher-level considering different abstractions of UAV 2D/3D kinematic models with velocities/accelerations as control inputs. 
Collision avoidance can be ensured rigorously for some of these methods under certain technical assumptions \cite{hoy2015algorithms} in contrast to other motion planning methods.
For example, the design may rely on assumptions made about obstacles (shape, size, velocity profile, etc.), environment (static or dynamic) and sensing capabilities (vision-based, distance-based, FOV, range, etc.).
Many of the existing reactive methods are planar which can generally be applied to various types of mobile robots including UAVs moving at a fixed altitude; examples of such methods include \cite{toibero2009stable,teimoori2010biologically,matveev2011method,matveev2012real,savkin2013simple,matveev2015safe,choi2017two,mcguire2017efficient,matveev2015globally}.
Adopting these methods for vehicles that can navigate in 3D, such as UAVs and AUVs, becomes less efficient.
Therefore, there has been a growing interest in developing 3D reactive navigation methods which will be the main focus in this section in addition to some of the 2D vision-based approaches sufficiently suitable for UAVs in some applications.

A number of geometric-based reactive collision avoidance methods focused on non-cooperative scenarios (i.e. dynamic environments) for fixed-wing UAVs or vehicles with nonholonomic constraints adopting the idea of collision cones such as \cite{mujumdar2011reactive,wang2018strategy,lin2020fast,Belkhouche2012,belkhouche2017reactive,wiig20203d}.
Many of these approaches use linear or nonlinear guidance laws to align the velocity vector (i.e. controlling heading and flight path angles) in a certain direction while keeping a constant relative distance to the obstacle to avoid collisions.
The work \cite{mujumdar2011reactive} proposed two guidance laws for collision avoidance in static and dynamic environments based on collision cones where the vehicle is guided to track the surface of a safety sphere around the obstacle.
Similarly, the works \cite{wang2018strategy,lin2020fast} adopted collision cones to safely guide fixed-wing UAVs in 3D dynamic environments.
In \cite{Belkhouche2012}, a 3D reactive navigation law was proposed based on relative kinematics between the vehicle and obstacles decoupled into horizontal and vertical planes.  %
Obstacles were modeled as spheres, and collision cones were used for obstacle avoidance.
This method was further developed in \cite{belkhouche2017reactive} where a reactive optimal approach was suggested for motion planning in dynamic environments. %

A different implementation of collision cones was done in \cite{wiig20203d} for AUVs; however, the same idea can be applied to UAVs as well.
No assumptions were made about the obstacle shape; however, obstacles were modeled as spheres for the mathematical development, and it was only assumed that the collision cone to the obstacle can be interpreted from sensors measurements.
This method relied on maintaining a constant avoidance angle from a nearby obstacle while ensuring a minimum relative distance is achieved.
The same problem was addressed differently in \cite{wu2021obstacle} where a new nature-inspired 3D obstacle avoidance method for AUVs were developed based on concepts from fluid dynamics.

Another class of 3D reactive methods modified the Velocity Obstacle (VO) approach to allow navigation in dynamic environments such as \cite{yang20133d,tan2020three}.
In \cite{yang20133d}, the proposed method relied on decoupling the 3D motion to achieve constant relative bearing and elevation in both the horizontal and vertical planes simultaneously.
It was assumed that the desired relative bearing and elevation with respect to the non-cooperative vehicle can be estimated using onboard cameras.
Also, \cite{tan2020three} proposed an improvement to the Velocity Obstacle (VO) method to handle 3D static and dynamic environments.

Artificial potential field was also considered in some approaches to handle navigation in dynamic environments as in \cite{zhu20163d,roussos20103d,santos2017novel}.
The approaches \cite{zhu20163d,roussos20103d} developed modified APF methods for 3D nonholonomic vehicles while the work \cite{santos2017novel} designed an APF reactive controller for quadrotors.
The approach in \cite{santos2017novel} combines obstacle avoidance control law based on artificial potential field with a trajectory tracking control law using on a null-space-based scheme on the kinematic level where the obstacle avoidance input has the higher priority.
A dynamic controller was then proposed to generate low-level input to ensure that velocities generated by the kinematic controller can be tracked.

The authors of \cite{hrabar2011reactive} suggested a different 3D navigation approach for rotorcraft UAVs where an escape waypoint is determined whenever an obstacle is detected.
Obstacle detection was done by extending a cylindrical safety volume from the UAV position along the movement direction in 3D local map representation of the environment.
The escape waypoint is determined by performing a search through a set of concentric ellipsoids around the detected obstacles by iteratively incrementing the ellipses radii until a safe escape point is found.
Due to the low complexity of the algorithm, it belongs to the reactive class.

In \cite{nguyen2018real}, a computationally-light approach was suggested through real-time deformations of a predefined 3D path based on the intersection between two 3D surfaces determined according to the free space and obstacles.
Either one or both surfaces are modified in the presence of obstacles such that the intersection between the two surfaces provides a path around the obstacle.
To that end, proper functions need to be carefully chosen to represent the obstacle where the authors considered a Gaussian function whose parameters require proper tuning.
A path following controller was also proposed based on multirotor full dynamical model where a cascaded approach for control was adopted for position and attitude.
This was further implemented in \cite{iacono2018path} where a depth camera was used to detect obstacles.
Another 3D reactive method adopting the idea of real-time deformable paths around dynamic obstacles was also proposed in \cite{elmokadem2020control}.

A number of reactive methods considers vision-based structure such as \cite{Oleynikova2015,bucki2020rectangular,ji2020mapless,lee2021autonomous,potena2019joint}.
In \cite{Oleynikova2015}, a vision-based reactive approach was proposed for quadrotor MAVs based on embedded stereo vision. %
Obstacles are detected from stereo images based U-V disparity maps.
A short-term local map is built for planning purposes representing approximations of detected obstacles as ellipsoids.
Hence, no accurate odometery is needed since no global map is built.
The obstacle avoidance algorithm is mainly 2D to find the shortest path along obstacles' edges.
On the other hand, the works \cite{bucki2020rectangular,ji2020mapless,lee2021autonomous} proposed 3D mapless vision-based trajectory planning methods using depth images which can be considered reactive as the planning horizon becomes very short.
A different vision-based 3D reactive method was proposed in \cite{potena2019joint} based on NMPC for quadrotors navigating in dynamic environments.

Some other methods relied on LiDAR sensors such as \cite{mansouri2020unified} which combined 3D collision avoidance with control in a nonlinear model predictive control scheme considering both dynamic and geometric constraints at the same time.
It adopted a mapless approach by relying on a subspace clustering method applied to 3D point clouds obtained directly from a 3D LiDAR sensor.

Concepts from machine learning were also considered recently in some reactive methods to address obstacle avoidance problems for UAVs. %
However, these methods are more computationally expensive than other reactive methods, and there are still concerns related to how guaranteed a collision avoidance is as the performance relies on how good the training/learning stage is.
Also, many of the existing approaches consider only generating motion decisions/policies in 2D without utilizing the full maneuverability of UAVs.
Most of these methods are based on deep reinforcement learning \cite{ross2013learning,zhang2015geometric,wang2017autonomous,ma2018saliency,singla2019memory,walker2019deep,yan2019towards,wang2019autonomous} and deep neural networks \cite{padhy2018deep,dionisio2018deep,dai2020automatic,back2020autonomous,lee2021deep,yang2019fast,wang2020uav,sanket2020evdodgenet}.

\begin{table*}[!htb]
	\centering
	\resizebox{\columnwidth}{!}{%
		\begin{tabular}{| c | c | c | c | c | c | c |} 
			\hline
			Ref. & Control Structure & Local Motion Planning & Model & Dynamic Environment \\ 
			\hline
			\cite{Georges2017} & I/II & sampling-based path planning & 2D Kinematics (nonholonomic) & $\checkmark$ \\ 
			\hline
			\cite{yang2010efficient} & I/II & sampling-based path planning & 3D Single-rotor Dynamics &  \\ 
			\hline
			\cite{lin2017sampling} & I/II & sampling-based path planning  & 3D Kinematics (nonholonomic) & $\checkmark$  \\ 
			\hline
			\cite{schmid2020efficient} & I/II & sampling-based path planning & 3D Kinematics (holonomic) &  \\ 
			\hline
			\cite{sanchez2019real} & I/II & graph-based path planning  & 3D Kinematics (holonomic) & $\checkmark$ \\ 
			\hline
			\cite{miller20113d,roberge2012comparison} & I/II & optimization-based path planning & 3D Kinematics &  \\ %
			\hline
			\cite{chen2016uav} & I/II & optimization-based path planning  & 3D Quadrotor Dynamics &  \\ 
			\hline
			\cite{mellinger2011minimum,Faculty2016,liu2016high,liu2017planning,watterson2015safe,Mohta2018} & II/III & optimization-based trajectory generation using QP with corridor-like constraints & 3D Quadrotor Dynamics &  \\ 
			\hline
			\cite{tordesillas2019fastrap,ye2020tgk,tordesillas2019real} & III & optimization-based trajectory planning using QP  & 3D Dynamics (acceleration/jerk input) &  \\ 
			\hline
			\cite{richter2016polynomial,quan2020eva} & III & optimization-based trajectory planning using unconstrained QP & 3D Quadrotor Dynamics &  \\ 
			\hline
			\cite{Oleynikova2016,liu2017robust,spedicato2017minimum} & III & optimization-based trajectory planning with obstacles constraints & 3D Quadrotor Dynamics &  \\ 
			\hline
			\cite{ryll2019efficient,chen2021computationally,mueller2015computationally,lopez2017aggressive,zhang2018perception} & III & motion primitives & 3D Quadrotor Dynamics &  \\ 
			\hline
			\cite{gonzalez2020autonomous} & III & motion primitives  & 3D Kinematics (holonomic) &  \\ 
			\hline
			\cite{paranjape2015motion} & III & motion primitives & 3D Kinematics (nonholonomic) &  \\ 
			\hline
			\cite{tordesillas2020mader,tordesillas2021panther} & III & perception-aware trajectory planning  & 3D Dynamics (jerk input) & $\checkmark$ \\ 
			\hline
			\cite{zhang2018perception,Falanga2018pampc,murali2019perception,spasojevic2020perception,sheckells2016optimal} & III & perception-aware trajectory planning  & 3D Quadrotor Dynamics &  \\
			\hline
			\cite{lindqvist2020nonlinear,mansouri2020unified} & III/IV & non-convex optimization with obstacles constraints using NMPC  & 3D Quadrotor Dynamics & $\checkmark$ \\ 
			\hline
			\cite{bucki2020rectangular,ji2020mapless,lee2021autonomous} & III/IV & mapless vision-based trajectory planning using depth images & 3D Dynamics (jerk input) &  \\ 
			\hline
			\cite{wang2018strategy,mujumdar2011reactive,lin2020fast,wiig20203d,Belkhouche2012,belkhouche2017reactive} & IV & Geometric-based (collision cones) reactive control & 3D Kinematics & $\checkmark$ \\ 
			\hline
			\cite{yang20133d,tan2020three} & IV & reactive control based on Velocity Obstacle (VO) & 3D Kinematics & $\checkmark$ \\ 
			\hline
			\cite{zhu20163d,roussos20103d,santos2017novel} & IV & reactive control based on artificial potential field & 3D Kinematics (nonholonomic)/Quadrotor Dynamics & $\checkmark$ \\ 
			\hline
			\cite{wu2021obstacle} & IV & nature-inspired reactive control & 3D Kinematics (nonholonomic) &  \\ 
			\hline
			\cite{Oleynikova2015} & IV & vision-based reactive control & 2D Kinematics &  \\ 
			\hline
			\cite{nguyen2018real,iacono2018path,elmokadem2020control} & IV & real-time path deformation (reactive) & 3D Quadrotor Dynamics & $\checkmark$ \\ 
			\hline
			\cite{potena2019joint} & IV & vision-based reactive control based on NMPC & 3D Quadrotor Dynamics & $\checkmark$ \\ 
			\hline
			\cite{ross2013learning,zhang2015geometric,wang2017autonomous,ma2018saliency,singla2019memory,walker2019deep} & IV & deep reinforcement learning & 2D Kinematics &  \\ 
			\hline
			\cite{yan2019towards} & IV & deep reinforcement learning & 2D Kinematics & $\checkmark$ \\ 
			\hline
			\cite{wang2019autonomous} & IV & deep reinforcement learning & 3D Kinematics &  \\ 
			\hline
			\cite{padhy2018deep,dionisio2018deep,dai2020automatic,back2020autonomous,lee2021deep} & IV & deep neural networks & 2D Kinematics &  \\ 
			\hline
			\cite{yang2019fast,wang2020uav} & IV & deep neural networks & 3D Kinematics &  \\ 
			\hline
			\cite{sanket2020evdodgenet} & IV & deep neural networks & 3D Kinematics & $\checkmark$ \\ 
			\hline
		\end{tabular}
	}
	\caption{A summary of surveyed local motion planning methods for UAVs}
	\label{tab:ch2:motion_planning}
\end{table*}

\section{UAV Modeling \& Control} \label{sec:ch2:control}

\subsection{Modeling}

For control design and simulation purposes, it is required to have a valid mathematical model that can express the UAV motion.
Generally, such model consists of two main parts which are \textit{kinematics} and \textit{dynamics}.
\textit{Kinematic} equations are mainly derived to represent the geometrical aspects of the motion in 3D spaces through defining translation and rotation relationships between different coordinate frames.
\textit{Dynamics} can be obtained through the application of Newton laws for a moving rigid body to derive linear and angular momentum equations. %
Application of Newton laws requires an inertial reference frame $\mathcal{I}$ to be defined.
On the other hand, analyzing forces and torques acting on the vehicle needs to be done with respect to a coordinate frame  attached to the moving vehicle (i.e. a body-fixed frame $\mathcal{B}$). %
Clearly, different UAV types would have some differences in their dynamic equations depending on the actuators configurations and other external forces and torques acting on the vehicle.
For simplicity, the origin of the body-fixed frame is chosen to coincide with the vehicle's center of mass.
Note that there are other coordinate frames that can be used for different purposes for navigation and control such as Earth-Centered, Geodetic and wind coordinate frames.
For more details about these coordinate frames, refer to \cite{valavanis2015handbook}.

A rotation matrix between the inertial and body-fixed coordinate frames can be used to define the attitude/orientation of the UAV.
It is also common to use other representations such as Euler angles (i.e. roll $\phi$, pitch $\theta$ and yaw $\psi$) and quaternions $\bm{q}\in \R^4$.
Quaternions are more computationally efficient and do not have the gimbal lock problem while Euler angles are easier to understand physically and can be decoupled into separate degrees of freedom under some assumptions for simplicity.

Let the Euler's angles vector be $\bm{\Phi} = [\phi, \theta, \psi]^T$, and consider a quaternion vector $\bm{q} = [q_1,q_2,q_3,q_4]^T$.
Notice that with Euler angles, usually three rotations are applied in a specific order which can result in different forms for the rotation matrix.
The following is an example considering the rotation order $ZYX$,
\begin{equation}
{}^{\mathcal{I}}_{\mathcal{B}}\mathbf{R}(\bm{\Phi}) = \left[\begin{array}{ccc}
	c_{\theta} c_{\psi} & 
	s_{\phi} s_{\theta} c_{\psi} - c_{\phi} s_{\psi} & 
	c_{\phi} s_{\theta} c_{\psi} + s_{\phi} s_{\psi} \\
	c_{\theta} s_{\psi} & 
	s_{\phi} s_{\theta} s_{\psi} + c_{\phi} c_{\psi} & 
	c_{\phi} s_{\theta} s_{\psi} - s_{\phi} c_{\psi} \\
	-s_{\theta} & 
	s_{\phi} c_{\theta} & 
	c_{\phi} c_{\theta}
\end{array}\right]
\end{equation}
where $c_{\alpha}:=\cos\alpha$, and $s_{\alpha}:=\sin\alpha$.
Note that ${}^{\mathcal{I}}_{\mathcal{B}}\mathbf{R}(\bm{\Phi})$ represents the rotation from the body-fixed frame to the inertial frame.
Furthermore, ${}_{\mathcal{I}}^{\mathcal{B}}\mathbf{R}(\bm{\Phi})={}^{\mathcal{I}}_{\mathcal{B}}\mathbf{R}^{T}(\bm{\Phi})$.

For a velocity vector expressed in the body-fixed frame, it can be transformed to the inertial frame as follows:
\begin{equation}
{}^{\mathcal{I}}\bm{v} = {}^{\mathcal{I}}_{\mathcal{B}}\bm{R}(\bm{\Phi})\ {}^{\mathcal{B}}\bm{v}
\end{equation}
such that ${}^{\mathcal{I}}\bm{v} = [\dot{x},\dot{y},\dot{z}]^T$ and ${}^{\mathcal{B}}\bm{v}=[u,v,w]^T$.
Also, the angular velocity can be transformed from $\mathcal{B}$ to $\mathcal{I}$ as:
\begin{equation}
\bm{\dot{\Phi}} = T(\bm{\Phi}) \bm{\Omega}
\end{equation}
where
\begin{equation}
	T(\bm{\Phi}) = \left[\begin{array}{ccc}
		1 & s_{\phi} t_{\theta} & c_{\phi} t_{\theta} \\
		0 & c_{\phi} & -s_{\phi} \\
		0 & s_{\phi}/c_{\theta} & c_{\phi}/c_{\theta}
	\end{array}\right]
\end{equation}
with $t_{\theta}:=\tan\theta$.
The gimbal lock problem can be seen clearly from $T(\bm{\Phi})$ where a singularity occurs when $\theta = \pm 90^o$.
Such a problem does not exist when using quaternions for kinematic modelling.

Hence, the general model for a UAV is given by:
\begin{eqnarray}
\bm{\dot{p}} &=& {}_{\mathcal{B}}^{\mathcal{I}}\bm{R}(\bm{\Phi}) \ {}^{\mathcal{B}}\bm{v} \label{equ:ch2:model1} \\
{}^{\mathcal{B}}\bm{\dot{v}} &=& \frac{\bm{F}}{m} - \bm{\Omega}\times{}^{\mathcal{B}}\bm{v} \label{equ:ch2:model2} \\
{}_{\mathcal{B}}^{\mathcal{I}}\bm{\dot{R}} &=& {}_{\mathcal{B}}^{\mathcal{I}}\bm{R} \bm{\Omega} \label{equ:ch2:model3} \\
\bm{\dot{\Omega}} &=& \bm{I}^{-1} \Big(\bm{M} - \bm{\omega}\times\bm{I}\bm{\Omega}\Big) \label{equ:ch2:model4} 
\end{eqnarray}
where $\bm{p},{}^{\mathcal{I}}\bm{v}\in \R^3$ are the position and linear velocity expressed in the inertial frame, $\Omega\in \R^3$ is the angular velocity defined in the body-fixed frame, $m \in \R^{+}$ is the UAV mass, and $\bm{I}\in \R^{3\times 3}$ is the inertia matrix.
Furthermore, $\bm{F}\in\R^3$ and $\bm{M}\in\R^3$ correspond to external forces and torques acting on the vehicle. 

Modeling the forces and torques differ based on the UAV type, design and actuators configuration which affects the control system design.
Example of these differences can be seen in the complete models for fixed-pitch multi-rotors \cite{hamel2002dynamic,mellinger2011minimum,faessler2017differential}, variable-pitch multi-rotors \cite{kamel2018voliro,allenspach2020design}, helicopters \cite{godbolt2013experimental}, fixed-wing UAVs \cite{lesprier2015modeling}, flapping-wing UAVs \cite{karasek2014robotic}, etc. 
Some researchers have further extended the UAV modeling considered in the control design to include some added systems such as cable-suspended payload \cite{foehn2017fast,tang2018aggressive}.

\subsection{Low-level Control}

As mentioned earlier, a common approach to handle the navigation problem is by decoupling planning from control.
Thus, a low-level control can be designed independently to track the generated reference paths, trajectories, heading/flight path angles or velocity/acceleration commands.
Typically, control laws are developed to minimize tracking errors by determining required input forces and body torques which can then be mapped into motor and actuator commands depending on the UAV type.
State estimation is a very critical component for feedback control.
Extended Kalman Filter (EKF) is a popular choice in many implementations to provide estimates for the UAV attitude, linear and angular velocities by fusing data from different sensors.
Position can also be estimated by fusing information from a positioning source such as GNSS, visual odometry, external positioning system, etc.

A cascaded approach is very common in different control structures where the attitude dynamics (i.e. \eqref{equ:ch2:model3}-\eqref{equ:ch2:model4}) are decoupled to avoid considering the full nonlinear system dynamics in the control design \cite{bangura2014real}.
A high-bandwidth inner loop attitude controller is used to ensure that the vehicle can accurately track reference attitude or angular velocity commands.
This reduces the control problem to design an outer control loop for the translational dynamics \eqref{equ:ch2:model1}-\eqref{equ:ch2:model2} that can achieve position/velocity tracking by deciding proper laws in terms of thrust, attitude and/or angular velocities.
Several control techniques were adopted in the literature such as PID \cite{Yamasaki2009,godbolt2013experimental}, sliding mode control \cite{Zheng2014}, Lyapunov-based nonlinear control \cite{Ambrosino2009} and model predictive control \cite{yang2013adaptive,bangura2014real,bicego2020nonlinear,li2018development,kamel2017linear}. 

Multi-rotors are the most popular UAV type for many civilian applications due to their simplicity in mechanical design and control.
Therefore, there have been many recent developments in nonlinear control of multi-rotors enabling high-speed navigation \cite{liu2016high,ryll2019efficient}, aggressive flights \cite{mellinger2012trajectory,bry2015aggressive,loianno2016estimation} and aerial manipulation \cite{michael2011cooperative,fink2011planning,sreenath2013dynamics}.

Quadrotor dynamics are differentially-flat which was shown in \cite{mellinger2011minimum} (even under drag effects \cite{faessler2017differential}).
Differential-flatness denotes that all system variables (i.e. states and inputs) can be written in terms of a set of flat outputs (for example, $[x,y,z,\psi]$).
That is, trajectories can be planned in the space of flat outputs, and it ensures that any smooth trajectory with proper bounded derivatives can be tracked.
Hence, several control methods adopted a geometric-based control design utilizing the differential-flatness property such as \cite{mellinger2011minimum,lee2013nonlinear}.
Model predictive control was also considered in \cite{kamel2017linear} include blade flapping and induced drag effects modeled as external disturbances.
Including drag effects and external disturbances in the control design was considered by several other works such as \cite{bangura2014real,omari2013nonlinear,faessler2017differential}.
Some other control designs for fixed-pitch multirotor UAVs were proposed; for example, see \cite{mahony2012multirotor,manjunath2016application,Zheng2014,bicego2020nonlinear,lee2017trajectory,nascimento2019position} and references therein.

Variable-pitch/omni-directional multi-rotors are fully actuated vehicles where translational and rotational degrees of freedom can be decoupled; examples of control methods developed for these vehicles can be found in \cite{kamel2018voliro,allenspach2020design,rashad2020fully}.
Control of Single-rotor UAVs (helicopters) have also been tackled in several research works using the similar cascaded structure.
For example, a PID-based trajectory tracking controller was designed in \cite{godbolt2013experimental} while robust and perfect tracking (RPT) technique was suggested in \cite{cai2013design}.
Control of fixed-wing UAVs followed a similar control structure using decoupled control loops for translational and attitude dynamics.
Control designs for fixed-wing UAVs take into consideration the models nonholonomic kinematic constraints, and many of the existing methods adopt path-following techniques based on guidance laws such as \cite{Yamasaki2009,Ambrosino2009,Sujit2013}.
In \cite{Yamasaki2009}, the control method adopted pure pursuit guidance and a decoupled proportional control for velocity and attitude.
A similar control method was suggested in \cite{Ambrosino2009} based on LOS guidance algorithms and nonlinear control considering wind effects.
Model predictive control was also considered in the path-following control design proposed in \cite{yang2013adaptive}.
Alternatively, \cite{lesprier2015modeling} presented control designs for fixed-wing UAVs based on linear pole placement and nonlinear structured multimodal $H_{\infty}$ synthesis to track a reference air speed and flight path angle.
Control of other UAV types have also attracted some interest in the community developing new control methods for hybrid UAVs \cite{li2018development,atay2021spherical}, flapping-wing UAVs \cite{karasek2014robotic,icsbitirici2017design}, etc.

\section{Simultaneous Localization \& Mapping (SLAM)} \label{sec:ch2:localization}

Localization is trying to determine the vehicle's position given a certain map based on the newly obtained sensors information while mapping is trying to build a map representation of the environment given localization information.
Thus, navigation in unknown environments requires this to be done online simultaneously which is known as simultaneous localization and mapping (SLAM).
Development of SLAM methods is a very active field of research in robotics as the performance of map-based navigation methods rely on SLAM accuracy. %

Existing SLAM methods can be classified as either LiDAR-based or vision-based.
LiDAR-based methods adopt scan matching algorithms, and they offer better accuracy (ex. see \cite{zhang2014loam,hess2016real,koide2018portable,legoloam2018shan,liosam2020shan}).
However, vision-based SLAM methods have become more popular for UAVs due to the lower cost and light weight of cameras compared to LiDARs.
According to \cite{taketomi2017visual}, these can be classified into feature-based \cite{klein2007parallel,mur2015orb,mur2017orb}, direct \cite{engel2014lsd,engel2017direct} or RGB-D camera-based \cite{whelan2015real,naudet2021constrained}.
Feature-based methods relies on detecting and extracting features from an input image to be used for localization which can be challenging in textureless environments.
On contrary, direct methods use the whole image directly offering more robustness at the expense of increased computational cost.
RGB-D camera-based methods combines both image and depth information in its formulation. 
For more detailed overview of SLAM methods, the reader is referred to the following surveys \cite{cadena2016past,taketomi2017visual,lu2018survey}.

\section{Summary of Recent Developments} \label{sec:ch2:summary}

\newcommand{\refsM}{\cite{Georges2017,lin2017sampling,sanchez2019real,miller20113d,chen2016uav,roberge2012comparison,gonzalez2020autonomous,Faculty2016,Oleynikova2016,liu2017planning,watterson2015safe,spedicato2017minimum,ryll2019efficient,tordesillas2019fastrap,tordesillas2020mader,tordesillas2021panther,chen2021computationally,ye2020tgk,bucki2020rectangular,ji2020mapless,quan2020eva,lee2021autonomous,mueller2015computationally,lopez2017aggressive,tordesillas2019real,lindqvist2020nonlinear,zhang2018perception,spasojevic2020perception,mujumdar2011reactive,lin2020fast,yang20133d,zhu20163d,roussos20103d,tan2020three,wiig20203d,wu2021obstacle,Belkhouche2012,belkhouche2017reactive,hrabar2011reactive,iacono2018path,ross2013learning,zhang2015geometric,wang2017autonomous,ma2018saliency,singla2019memory,wang2019autonomous,yan2019towards,walker2019deep,padhy2018deep,dionisio2018deep,dai2020automatic,back2020autonomous,lee2021deep}}
\newcommand{\refsPM}{\cite{liu2016high,Oleynikova2015,yang2019fast,wang2020uav,sanket2020evdodgenet,hoy2012collision,Oleynikova2016}}

\newcommand{\refsME}{\cite{Yang2013a,Oleynikova2017,meera2019obstacle,schmid2020efficient}}

\newcommand{\refsC}{\cite{Yamasaki2009,Ambrosino2009,godbolt2013experimental,cai2013design,yang2013adaptive,lesprier2015modeling,mahony2012multirotor,faessler2017differential,kamel2017linear,manjunath2016application,Sujit2013,Zheng2014,lee2013nonlinear,bicego2020nonlinear,li2018development,atay2021spherical}}

\newcommand{\refsCM}{\cite{yang2010efficient,mellinger2011minimum,richter2016polynomial,liu2017robust,paranjape2015motion,Falanga2018pampc,murali2019perception,sheckells2016optimal,nguyen2018real,santos2017novel,potena2019joint}}

\newcommand{\refsP}{\cite{Holz2013}}

\newcommand{\refsS}{\cite{zhang2014loam,hess2016real,koide2018portable,legoloam2018shan,liosam2020shan,taketomi2017visual,klein2007parallel,mur2015orb,mur2017orb,engel2014lsd,engel2017direct,whelan2015real,naudet2021constrained}}

\newcommand{\refsCPM}{\cite{mansouri2020unified}}

\newcommand{\refsCPSM}{\cite{shen2011autonomous,perez2018architecture,Mohta2018}}

\newcommand{\refsCPSME}{\cite{Bachrach2011}}

\newcommand{\refsCPS}{\cite{Blosch2010}}

\Cref{tab:ch2:summary} summarizes some of the recent contributions made towards developing fully autonomous UAVs in terms of control, perception, SLAM, motion planning and exploration capabilities.

\begin{table*}[!htb]
	\centering
	\resizebox{\columnwidth}{!}{%
		\begin{tabular}{ | c | c| c | c | c | c| c |} 
			\hline
			Reference & Control & Perception & SLAM & Motion Planning & Exploration \\ 
			\hline
			\refsP  &  & $\checkmark$ & &  &  \\ %
			\hline
			\refsM & & & & $\checkmark$ &  \\ %
			\hline
			\refsPM & & $\checkmark$ & & $\checkmark$ &  \\ %
			\hline %
			\refsME & & & & $\checkmark$ & $\checkmark$ \\  %
			\hline
			\refsC  & $\checkmark$ & & & &  \\ %
			\hline
			\refsS  &  & & $\checkmark$ & &  \\ %
			\hline
			\refsCM & $\checkmark$  & & & $\checkmark$ & \\ %
			\hline
			\refsCPM & $\checkmark$  & $\checkmark$ & & $\checkmark$ & \\ %
			\hline
			\refsCPS  & $\checkmark$  & $\checkmark$ &$\checkmark$ &  &  \\ %
			\hline
			\refsCPSM  & $\checkmark$  & $\checkmark$ &$\checkmark$ & $\checkmark$ &  \\ %
			\hline
			\refsCPSME  & $\checkmark$  & $\checkmark$ &$\checkmark$ & $\checkmark$ & $\checkmark$ \\ %
			\hline
		\end{tabular}
	}
	\caption{A summary of some recent developments for UAVs in control, perception, SLAM and motion planning}
	\label{tab:ch2:summary}
\end{table*}

\section{Open-Source Projects}\label{sec:ch2:open_source}

There have been many developments in the field of UAVs in terms of perception, control, SLAM and path planning over the past years.
Implementing a complete autonomous navigation stack would require a large team with different skill-sets in these areas or collaborations among research groups.
Open-source projects contributed by many researchers have made it possible to focus on the development and improvement of a specific navigation component while easily integrating with other components developed by researchers in the community saving a lot of development time.
\Cref{tab:ch2:2} shows a list of some existing open-source projects and tools useful for autonomous UAV research and development.

\begin{table*}[!htb]
	\centering
	\resizebox{0.75\columnwidth}{!}{%
	\begin{tabular}{m{4cm} | m{4cm} | m{6cm} | l } 
		& Name & Description & Source \\ 
		\hline
		 \multirow{6}{*}{Navigation Stack} & Vision-based navigation for MAVs\cite{oleynikova2020open} & provides an open-source system for MAVs based on vision-based sensors including control, sensor fusion, mapping, local and global planning & \makecell{\href{http://github.com/ethz-asl/voxblox}{http://github.com/ethz-asl/voxblox} \\ \href{http://github.com/ethz-asl/rovio}{http://github.com/ethz-asl/rovio} \\ \href{http://github.com/ethz-asl/ethzasl_msf}{http://github.com/ethz-asl/ethzasl\_msf} \\ \href{http://github.com/ethz-asl/odom_predictor}{http://github.com/ethz-asl/odom\_predictor} \\ \href{http://github.com/ethz-asl/maplab}{http://github.com/ethz-asl/maplab} \\ \href{http://github.com/ethz-asl/mav_control_rw}{http://github.com/ethz-asl/mav\_control\_rw}} \\
		\cline{2-4}
		 & PULP-DroNet \cite{palossi2019open} & a deep learning-powered visual navigation engine for nano-UAVs & \href{https://github.com/pulp-platform/pulp-dronet}{https://github.com/pulp-platform/pulp-dronet}  \\
		\hline
		\multirow{4}{*}{LiDAR-based SLAM} & Google’s Cartographer\cite{hess2016real} & provides a real-time SLAM solution in 2D and 3D & \href{https://github.com/cartographer-project/cartographer}{https://github.com/cartographer-project/cartographer}  \\
		\cline{2-4}
		& hdl\_graph\_slam\cite{koide2018portable} & a real-time 6DOF SLAM using a 3D LIDAR & \href{https://github.com/koide3/hdl_graph_slam}{https://github.com/koide3/hdl\_graph\_slam}  \\ 
		\cline{2-4}
		& loam\_velodyne \cite{zhang2014loam} & Laser Odometry and Mapping  & \href{https://github.com/laboshinl/loam_velodyne}{https://github.com/laboshinl/loam\_velodyne}  \\
		\cline{2-4}
		& A-LOAM & Advanced implementation of LOAM  & \href{https://github.com/HKUST-Aerial-Robotics/A-LOAM}{https://github.com/HKUST-Aerial-Robotics/A-LOAM}  \\ 
		\cline{2-4}
		& FLOAM & a faster and optimized version of A-LOAM and LOAM  & \href{https://github.com/wh200720041/floam}{https://github.com/wh200720041/floam}  \\ 
		\hline
		\multirow{14}{*}{Vision-based SLAM} & ORB SLAM \cite{mur2015orb} & a keyframe and feature-based Monocular SLAM & \href{https://openslam-org.github.io/orbslam.html}{https://openslam-org.github.io/orbslam.html}  \\ 
		\cline{2-4}
		& ORB SLAM 2 \cite{mur2017orb} & a real-time SLAM library for Monocular, Stereo and RGB-D cameras & \href{https://github.com/raulmur/ORB_SLAM2}{https://github.com/raulmur/ORB\_SLAM2}  \\
		\cline{2-4}
		& LSD-SLAM \cite{engel2014lsd} & a Large-Scale Direct Monocular SLAM system & \href{https://github.com/tum-vision/lsd_slam}{https://github.com/tum-vision/lsd\_slam}  \\
		\cline{2-4}
		& SVO Semi-direct Visual Odometry \cite{Forster2014ICRA} & a semi-direct monocular visual SLAM & \href{https://github.com/uzh-rpg/rpg_svo}{https://github.com/uzh-rpg/rpg\_svo}  \\
		\cline{2-4}
		& PTAM \cite{klein2007parallel} &  a monocular SLAM system & \href{https://github.com/Oxford-PTAM/PTAM-GPL}{https://github.com/Oxford-PTAM/PTAM-GPL}  \\ 
		\cline{2-4}
		& RTAB-Map \cite{labbe2019rtab,labbe2018long} & RGB-D, Stereo and Lidar Graph-Based SLAM algorithm & \href{http://introlab.github.io/rtabmap}{http://introlab.github.io/rtabmap}  \\ 
		\cline{2-4}
		& ElasticFusion \cite{whelan2016elasticfusion} & Real-time dense visual SLAM system using RGB-D cameras & \href{https://github.com/mp3guy/ElasticFusion}{https://github.com/mp3guy/ElasticFusion}  \\ 
		\cline{2-4}
		& Kintinuous \cite{whelan2015real} & Real-time dense visual SLAM system using RGB-D cameras & \href{https://github.com/mp3guy/Kintinuous}{https://github.com/mp3guy/Kintinuous}  \\ 
		\hline
		\multirow{4}{*}{Motion Planning} & Fast-Planner\cite{zhou2019robust} & a set of planning algorithms for fast flights with quadrotors in complex unknown environments & \href{https://github.com/HKUST-Aerial-Robotics/Fast-Planner}{https://github.com/HKUST-Aerial-Robotics/Fast-Planner} \\
		\cline{2-4}
		& FUEL\cite{zhou2021fuel} & a hierarchical framework for Fast UAV Exploration & \href{https://github.com/HKUST-Aerial-Robotics/FUEL}{https://github.com/HKUST-Aerial-Robotics/FUEL} \\
		\cline{2-4}
		& EGO-Planner & Gradient-based Local Planner for Quadrotors & \href{https://github.com/ZJU-FAST-Lab/ego-planner}{https://github.com/ZJU-FAST-Lab/ego-planner} \\
		\cline{2-4}
		& TopoTraj\cite{zhou2020robust} & a robust planner for quadrotor trajectory replanning based on gradient-based trajectory optimization & \href{https://github.com/HKUST-Aerial-Robotics/TopoTraj}{https://github.com/HKUST-Aerial-Robotics/TopoTraj} \\
		\cline{2-4}
		& toppra\cite{pham2018new} & a library for computing time-optimal trajectories subject to kinematic and dynamic constraints & \href{https://github.com/hungpham2511/toppra}{https://github.com/hungpham2511/toppra} \\
		\cline{2-4}
		& Open Motion Planning Library & a library for sampling-based motion planning algorithms & \href{https://ompl.kavrakilab.org/core/index.html}{https://ompl.kavrakilab.org/core/index.html} \\
		\cline{2-4}
		& AIKIDO & a C++ library for motion planning and decision making problems & \href{https://github.com/personalrobotics/aikido}{https://github.com/personalrobotics/aikido} \\
		\cline{2-4}
		& PathPlanning & a collection of search-based and sampling-based path planners implemented in Python & \href{https://github.com/zhm-real/PathPlanning}{https://github.com/zhm-real/PathPlanning} \\
		\hline
		\multirow{4}{*}{Control} & mav\_control\_rw\cite{kamel2017linear} & Linear and nonlinear MPC controllers for Micro Aerial Vehicles & \href{https://github.com/ethz-asl/mav_control_rw}{https://github.com/ethz-asl/mav\_control\_rw} \\
		\cline{2-4}
		& rpg\_mpc\cite{Falanga2018pampc} & Perception-Aware MPC for quadrotors & \href{https://github.com/uzh-rpg/rpg_mpc}{https://github.com/uzh-rpg/rpg\_mpc} \\
		\cline{2-4}
		& ACADO Toolkit & collection of algorithms for automatic control and dynamic optimization & \href{http://acado.github.io/}{http://acado.github.io/} \\
		\cline{2-4}
		& Control Toolbox & a C++ library for robotics addressing control, estimation and motion planing & \href{https://github.com/ethz-adrl/control-toolbox}{https://github.com/ethz-adrl/control-toolbox} \\
		\cline{2-4}
		& PX4 & an open-source flight control software for UAVs & \href{https://px4.io/}{https://px4.io/} \\
		\cline{2-4}
		& ArduPilot & an open-source flight control software for UAVs & \href{https://ardupilot.org/}{https://ardupilot.org/} \\
		\hline
		\multirow{3}{*}{Perception} & Augmented Autoencoders\cite{Sundermeyer2018ECCV} & 3D object detection pipeline from RGB images & \href{https://github.com/DLR-RM/AugmentedAutoencoder}{https://github.com/DLR-RM/AugmentedAutoencoder} \\
		\cline{2-4}
		& MoreFusion\cite{wada2020morefusion} & a perception pipeline for 6D pose estimations of multi-objects & \href{https://github.com/wkentaro/morefusion}{https://github.com/wkentaro/morefusion} \\
		\cline{2-4}
		& OpenCV & an optimized computer vision library & \href{https://opencv.org/}{https://opencv.org/} \\
		\cline{2-4}
		& Point Cloud Library (PCL) & efficient point cloud processing C++ library & \href{https://pointclouds.org/}{https://pointclouds.org/} \\
		\cline{2-4}
		& cilantro \cite{zampogiannis2018cilantro} & efficient point cloud processing C++ library & \href{https://github.com/kzampog/cilantro}{https://github.com/kzampog/cilantro} \\
		\hline
		\multirow{2}{*}{Simulators} & Gazebo & a robot simulator & \href{http://gazebosim.org/}{http://gazebosim.org/} \\
		\cline{2-4}
		& CoppeliaSim/V-REP & a robot simulator & \href{https://www.coppeliarobotics.com/}{https://www.coppeliarobotics.com/} \\
		\cline{2-4}
		& Webots & a robot simulator & \href{https://cyberbotics.com/}{https://cyberbotics.com/} \\
		\cline{2-4}
		& Hector Quadrotor & provides simulation tools for quadrotors (ROS-based) & \href{http://wiki.ros.org/hector_quadrotor}{http://wiki.ros.org/hector\_quadrotor} \\
		\cline{2-4}
		& RotorS \cite{Furrer2016} & a set of tools to simulate multi-rotors in Gazebo & \href{https://github.com/ethz-asl/rotors_simulator}{https://github.com/ethz-asl/rotors\_simulator} \\
		\hline
		\multirow{3}{*}{General} & Robot Operating System (ROS) & a middleware to facilitate building large robotic applications & \href{https://www.ros.org/}{https://www.ros.org/} \\
		\cline{2-4}
		& Ceres Solver & a C++ library for solving large optimization problems & \href{http://ceres-solver.org/}{http://ceres-solver.org/} \\
		\cline{2-4}
		& g2o & a C++ framework for graph-based nonlinear optimization & \href{https://github.com/RainerKuemmerle/g2o}{https://github.com/RainerKuemmerle/g2o} \\
		\cline{2-4}
		& NLopt & a nonlinear optimization library & \href{https://nlopt.readthedocs.io/en/latest/}{https://nlopt.readthedocs.io/en/latest/} \\
		\cline{2-4}
		& Optimization Engine (OpEn) & a fast solver for optimization problems in robotics & \href{https://nlopt.readthedocs.io/en/latest/}{https://nlopt.readthedocs.io/en/latest/} \\
		\hline
	\end{tabular}
	}
	\caption{Open-source projects and tools for UAV development}
	\label{tab:ch2:2}
\end{table*}

\section{Conclusion}\label{sec:ch2:conclusion}

This chapter presented an overview of recent navigation methods with collision avoidance capability needed to develop fully autonomous UAVs.
The main focus was on 3D collision avoidance strategies where several works related to control, 3D motion planning and SLAM were surveyed.
Moreover, an overview of some existing open-source projects was provided to aid researchers in quickly developing and deploying new technologies for UAVs.

    \renewcommand{\target}{\bm{p}_{goal}}

\part{Navigation Strategies for Single-UAV Systems}
\chapter{A Hybrid Navigation Strategy for Partially-Known and Dynamic Environments\label{cha:methods_hybrid}}

This chapter proposes a hybrid strategy for the navigation of unmanned aerial vehicles considering only lateral motion at a fixed altitude.
It is based on a general 2D kinematic model with nonholonomic constraints that best suits fixed-wing UAVs.
However, the model applies also to multi-rotor UAVs for constant-speed applications in addition to other mobile robots such as unmanned ground vehicles (UGVs) and autonomous underwater vehicles (AUVs).
This strategy is designed to allow autonomous operation in partially-known and dynamic environments.
It can also handle navigation in unknown environments with the requirement of building a map during navigation by means of simultaneous localization and mapping (SLAM).
The overall strategy can be extended to allow for 3D motions when combined with 3D control laws designed based on 3D kinematic models similar to the ones developed in \cref{cha:methods_reactive3D,cha:reactive_impl,cha:deforming_approach}.
The results presented in this chapter were published in \cite{elmokadem2018hybrid}.

\section{Introduction}\label{sec:ch3:Intro}

Unmanned aerial vehicles (UAVs) have emerged in many applications where it is required to do some repetitive tasks in a certain environment.
Autonomous operation is highly desirable in these applications which adds more requirements on the vehicle to achieve safe navigation towards areas of interest.
Navigation methods can be generally classified as global path planning (deliberative), local path planning (sensor-based) and hybrid.
A subset of sensor-based methods include reactive approaches \cite{hoy2015algorithms}.

Global path planning requires an overall knowledge about the environment to produce optimal and efficient paths which can be tracked by the vehicle's control system.
There exist many different techniques to address global planning problems including roadmap methods \cite{lozano1979algorithm,dunlaing1985retraction}, cell decomposition \cite{lozano1983spatial}, potential field \cite{khatib1986real,ge2000new}, Probabilistic Roadmaps (PRM) \cite{kavraki1996probabilistic}, Rapidly-exploring Random Tree (RRT) \cite{lavalle1998rapidly} and optimization-based techniques \cite{dragan2011manipulation,zucker2013chomp,schulman2014motion,li2018path}.
Some of these techniques become computationally challenging when dealing with unknown and dynamic environments since a complete updated map is required a priori.
As a workaround to handle such environments, extensions to some of these approaches were proposed by adding an additional layer to continuously refine the initial path locally around detected obstacles.
This still may be less efficient in highly complex and dynamic environments.

On the other hand, sensor-based methods generate local paths or motion commands in real-time based on a locally observed fraction of the environment interpreted directly from sensors measurements.
Search-based methods such those used for global planning and optimization-based methods can be used with a local map to generate paths locally where computational complexity depends on the selected map size.
On contrary, reactive methods offer better computational solutions by directly coupling sensors observations into control inputs providing quick reactions to perceived obstacles.
Hence, they can be more suitable in unknown and dynamic environments.
Examples of classical reactive methods used in unknown environments are dynamic window \cite{fox1997dynamic} and curvature velocity \cite{simmons1996curvature}.
Other classical examples of reactive methods dealing with dynamic obstacles include collision cones \cite{chakravarthy1998obstacle} and velocity obstacles \cite{fiorini1998motion}.
A class of reactive approaches adopt a boundary following paradigm to circumvent obstructing obstacles; for example, see \cite{bemporad2000sonar,toibero2009stable,teimoori2010biologically,matveev2011method,matveev2012real,savkin2013simple,matveev2015safe}.
The low computational cost of such methods comes  at the expense of being prone to trapping situations.
Some researchers suggested a combination of a randomized behavior with the boundary following approach to escape such situations \cite{savkin2013reactive}.
However, this may sometimes produce very unpredictable motions and even inefficient ones \cite{urdiales2003hybrid} without utilizing previous sensors observation acquired through the motion.

Hybrid strategies tend to address the afromentioned drawbacks by combining both deliberative and reactive approaches for a more efficient navigation behavior in unknown and dynamic environments.
There exists a body of literature on hybrid approaches; for example, see \cite{urdiales2003hybrid,zhu2012new,nieuwenhuisen2016layered,hank2016hybrid,wzorek2017framework,d2017safe,adouane2017reactive} and references therein.
A hybrid approach was suggested in \cite{urdiales2003hybrid} for navigation in dynamic environments which combined a potential field-based local planner with the A* algorithm as a global planner based on a topological map.
Similarly, the work \cite{zhu2012new} adopted the A* algorithm with binary grid maps while using a variant of the Bug algorithms as a reactive component to address navigation in partially unknown environments.
The authors of \cite{nieuwenhuisen2016layered} suggested another hybrid approach  for micro aerial vehicles that uses A* for both deliberative and local planning components where the global path gets refined locally around obstacles through replanning processes.
The hybrid approach presented in \cite{hank2016hybrid} adopted a fuzzy logic-based boundary following technique to implement the reactive layer while an Optimal Reciprocal Collision Avoidance (ORCA) algorithm was used in \cite{wzorek2017framework}.
Sampling-based search methods were also used in some approaches such as \cite{d2017safe} where a global planner based on the Dynamic Rapidly-exploring Random Tree (DRRT) was suggested.
Real-time obstacle avoidance was then dealt with by choosing a best candidate trajectory from a sampled set.
In \cite{adouane2017reactive}, Parallel Elliptic Limit-Cycle approach was adopted to implement both global and local planning components.

This chapter aims to enrich the literature on hybrid navigation strategies for UAVs to allow efficient navigation in partially unknown/dynamic environments.
A general hybrid strategy is presented by combining a global path planning layer with a reactive component using a general 2D kinematic model.
The global path planning layer, based on RRT,  can produce more efficient paths based on the available knowledge about the environment in addition to work as a recovery layer to escape trapping situations.
The sliding mode technique is adopted to implement a boundary following behavior in the reactive layer.
This choice provides quick reactions to obstacles with a cheap computational cost compared to search-based and optimization-based local planners.
To develop a proper hybrid navigation strategy,an implementation of a proper switching mechanism is presented to handle the transition between the two layers.
Overall, the proposed method can overcome the shortcomings of relying purely on a deliberative or a reactive approach.

This chapter is organized as follows.
\Cref{sec:ch3:prob} provides a formulation of the tackled navigation problem.
The suggested hybrid strategy is then presented in \cref{sec:ch3:nav}.
The performance of this approach is confirmed through different simulation scenarios which is shown in \cref{sec:ch3:sim}.
Finally, concluding remarks are made in \cref{sec:ch3:conc}.

\section{Problem Statement}\label{sec:ch3:prob}

A general navigation problem is considered here where a UAV is required to navigate safely in a partially-known environment $\env \subset \R^3$ to reach some goal position $\target \subset \env$ represented in a world coordinate frame $\{W\}$ with obstacle avoidance capability.
The environment $\env$ contains a set of $n$ obstacles with no assumptions made on their shape $\obs = \{\obs_1, \obs_2,\cdots,\obs_n\} \subset \env$.
These obstacles can be categorized into static known obstacles $\obs_k$ and unknown static/dynamic obstacles $\obs_u$ such that $\obs = \obs_k \cup \obs_u$.
The UAV starts with an initial map of the environment containing only information about $\obs_k$ based on a previous knowledge.
Let $\bm{p}(t)= [x(t),y(t),z(t)]^T$ be the UAV Cartesian coordinates expressed in $\{W\}$.
A safety requirement for the UAV is defined as keeping a safe distance $d_{safe}>0$ to all obstacles according to the following:
\begin{equation}\label{equ:ch3:safety}
	d(t) := \min_{\bm{p}^\prime(t) \in\obs}||\bm{p}(t) - \bm{p}^\prime(t)|| \geq d_{safe}\ \forall t
\end{equation}
where  $d(t)$ is the distance to closest obstacle, and $||\cdot||$ is the standard Euclidean norm of a vector.

We consider a unicycle-type kinematic model based only on planar motion in a subspace $\bar{\env} = \{(x,y,z): x,y\in \R, z = L\} \subset \env$ by assuming that the motion is constrained to a fixed altitude $L \in \R$ by a separate control loop.
Hence, the $z$ coordinate will be omitted from now on for brevity.
A description of this model is given by the following equations:
\begin{equation} \label{equ:ch3:model}
	\begin{aligned}
		\dot{x}(t) &= V(t) \cos \theta(t) \\
		\dot{y}(t) &= V(t) \sin \theta(t) \\
		\dot{\theta}(t) &= u(t)
	\end{aligned}
\end{equation}
where the UAV's linear and angular velocities are represented by $V(t)\in \R$ and $ u(t)\in \R$ respectively.
The motion direction is characterized by an angle $\theta(t)\in[-\pi,\pi]$ measured from the $+ve$ $x$-axis of $\{W\}$ in the counter clockwise direction.
The UAV velocities $V(t)$ and $u(t)$ are regarded as control inputs with some upper bounds on linear and angular velocities (denoted by $V_{max},u_{max}>0$ respectively) due to physical limitation.
These constraints can be expressed as follows:
\begin{equation}\label{equ:ch3:constraints}
	\begin{aligned}
		V(t) &\in [0,\ V_{max}], \\
		u(t) &\in [-u_{max},\ u_{max}].
	\end{aligned}
\end{equation}
Notice that we consider only forward motions by not allowing negative values for $V(t)$.
Also, the above model is nonholonomic exhibiting a constraint on the velocity given by:
\begin{equation}
	\dot{x} \sin\theta - \dot{y} \cos\theta = 0
\end{equation}
where the time parameter is removed for brevity.
Moreover, the velocities constraints impose a lower bound $R_{min} \geq 0$ on the turning radius given by:
\begin{equation}\label{equ:ch3:Rmin}
	R_{min} = \frac{V(t)}{u_{max}}
\end{equation}
For constant-speed applications, $V(t)$ is kept constant at some value $\bar{V} > 0$ resulting in $\bar{R}_{min} = \frac{\bar{V}}{u_{max}} > 0$.
Additionally, the following assumptions are made.
\begin{assumption}\label{asm:ch3:sensing}
	The UAV senses a fraction of the surroundings by onboard sensors, and it can determine the distance to closest obstacle as part of its perception system.
	An abstract sensing model is considered where obstacles within a distance of $d_{sensing}>0$ from the UAV can be detected.
\end{assumption}

\begin{assumption}\label{asm:ch3:localization}
	Estimates of the UAV's position $\bm{p}(t)$ and orientation $\theta(t)$ are available.
\end{assumption}

The following statement summarizes the considered problem in this chapter.
\begin{problem}
	Consider a UAV moving at a constant altitude whose planar movement is described by the model \eqref{equ:ch3:model}.
	Under assumptions~\ref{asm:ch3:sensing} and \ref{asm:ch3:localization}, design control laws for the linear and angular velocities $V(t)$ and $u(t)$ to ensure a collision-free navigation through an environment $\env$ to reach a goal position $\bm{p}_{goal}$ by satisfying the safety requirement in \eqref{equ:ch3:safety} and limits constraints in \eqref{equ:ch3:constraints}.
\end{problem}

\begin{remark}
	The model \eqref{equ:ch3:model} is applicable to fixed-wing UAVs, multi-rotor UAVs for constant-speed applications, unmanned ground vehicles and autonomous underwater vehicles.
\end{remark}

\section{Proposed Hybrid Navigation Strategy}\label{sec:ch3:nav}

The architecture of the overall hybrid navigation strategy is shown in \cref{fig:ch3:architecture}.
It consists of a perception subsystem, a low-level control subsystem and a high-level navigation subsystem.
The perception subsystem handles processing onboard sensors measurements to provide meaningful information about the environment through updating a map representation of the environment and providing distance to closest obstacle as required by our navigation strategy.
Then, a low-level control subsystem is used to generate actuators/motors commands to execute high level velocity commands generated by the high-level navigation subsystem.
The design of the high-level navigation subsystem is the main contribution of this work, and the proposed design combines three main layers, namely a \textit{deliberative layer (global planner)}, \textit{a reactive layer} and \textit{an execution layer}.

A global path planning algorithm constitutes the main component of the \textit{deliberative layer} in addition to an updated map provided from the perception subsystem based on the current knowledge of the environment.
The role of this layer is to generate a feasible and safe geometric path based on the current map which is needed initially when a new goal position is assigned and every time a trapping situation is detected as will be described later.
The \textit{reactive layer} provides quick actions to unknown obstacles by directly coupling command velocities into current sensors observations to provide collision free motion around obstacles.
The quick reactions to obstacles provided by this layer comes from the low computational cost of reactive control relative to map-based local planners.
The \textit{execution layer} is then responsible for making decisions on which layer to activate based on some designed switching rules, and it generates the actual command velocities to be executed by the low-level control subsystem.
The activation of each layer results in two navigation modes, namely \textit{path tracking} and \textit{collision avoidance} respectively.

\begin{figure}[!htb]
	\centering
	\includegraphics[width=\linewidth]{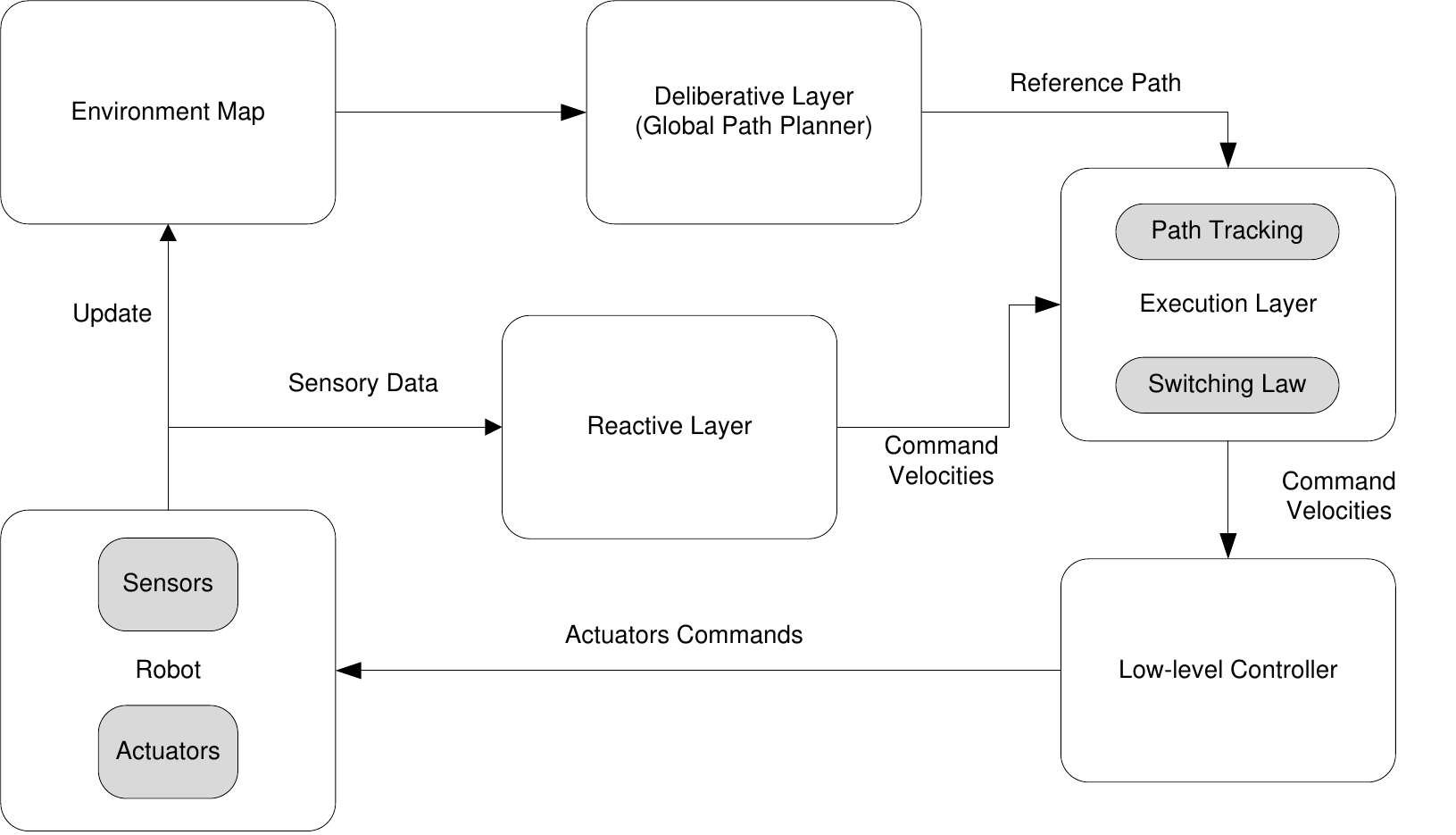} 
	\caption{Architecture of the proposed navigation strategy} \label{fig:ch3:architecture}
\end{figure}

\subsection{Global Planner}

The deliberative layer is based mainly on a global planning algorithm which requires a map representation of the environment.
There exist many algorithms in the literature where the choice of appropriate method can vary according to the UAV sensing and computing capabilities.
Optimality of the planned paths and the computational complexity of the overall algorithm could be key factors in determining which method to consider for a specific application.
Note that it is also possible to adopt more than one path planning algorithm within the deliberative layer where each planner gets activated under certain conditions.  
A typical example for this case is the use of optimization-based algorithms to initially generate optimal paths while implementing a computationally-efficient algorithm to modify the initial path when needed. 

The deliberative layer in our proposed framework is not restricted to a specific planning algorithm.
However, for the sake of validating our navigation strategy, this work will consider the Rapidly-exploring Random Tree (RRT) algorithm as a backbone for implementing the global planner. %
This choice is due to the algorithm's popularity and low computational cost which is favourable since the deliberative layer will be required to perform online replanning in some trapping scenarios as will be shown later.
Also, we follow an optimistic approach in our implementation of the RRT algorithm where unknown space is considered to be unoccupied (i.e. obstacle-free).
Based on this assumption, the reactive layer will be responsible for handling obstacle avoidance whenever an unknown space is found to be occupied. 

RRT is a randomized sampling-based path planning method which can find collision-free paths if exist with a probability that will reach one as the runtime increases (i.e. probabilistically complete).
However, it has been found to provide quick solutions in practice.
The basic concept of an RRT planner can be summarized as follows.%
Let $\graph$ be a search tree (a graph) initialized with an initial \textit{configuration} (defined later). 
An iterative approach is used to extend the search tree through the \textit{configuration space} until a feasible solution is found.
At each iteration, a configuration is sampled either randomly or using some heuristics which can help biasing the growth of the search tree. 
Biasing the sampling process to select the goal configuration with some probability ($0<pb_{goal}<1$) was found to enhance RRT growth.
The algorithm then tries to extend the search tree $\graph$ to the sampled configuration by connecting it to the nearest one within the tree.
Different methods could be used to connect configurations within the configuration space especially when trying to satisfy some constraints.
However, a common approach is to use straight lines to connect two configurations especially when dealing with Euclidean spaces which is considered here.
A collision checker is then used to check the feasibility of each extension based on the available environment map where each feasible extension results in growing the search tree by adding the new configuration.
The sampling process gets repeated iteratively until a feasible path between the initial and goal configurations is found or until a stopping criteria is met (for example, exceeding a predefined planning time limit).

Due to the sampling nature of RRT algorithms, the generated paths are non-optimal in terms of the overall length.
Furthermore, these paths do not satisfy nonholonomic constraints if straight lines were considered when extending the search tree.
Therefore, we follow a common practice by refining the obtained paths through two post-processing stages, namely \textit{pruning} and \textit{smoothing}, as was done in \cite{yang2010efficient,yang2010analytical}.

For the \textit{pruning} stage, redundant waypoints are removed to improve the overall quality of the path.
Let $W=\{w_1,w_2,\cdots,w_k\}$ be a set of waypoints representing a path generated by RRT, and let $W_p$ be the pruned path obtained after this stage.
Redundant waypoints can then be removed using Algorithm~\ref{alg:ch3:pruning} based on \cite{yang2010efficient}.

A \textit{smoothing} algorithm is applied next to the pruned path to ensure that the final path satisfies the vehicle's nonholonomic constraints (minimum turning radius).
To that end, parametric Bezier curves are used to generate continuous-curvature smooth paths following the approach suggested in \cite{yang2010analytical}.

\begin{algorithm}
	\caption{RRT Path Prunning}
	\label{alg:ch3:pruning}
	\begin{algorithmic}
		\State \textbf{Input}: $W=\{w_1,w_2,\cdots,w_k\}$
		\State \textbf{Output}: $W_p$
	\end{algorithmic}
	\begin{algorithmic}[1]
		\State $W_p \gets \{\}$ \Comment \begin{scriptsize}
			$W_p$ is initialised as an empty list
		\end{scriptsize}
		\State $i\gets k$
		\State insert $w_i$ into $W_p$
		\Repeat
		\State $j\gets 0$
		\Repeat
		\State $j\gets j + 1$
		\Until{collisionFree($w_i, w_j$) or $j = i-1$} \Comment \begin{scriptsize}
			Stop when a collision-free segment is found
		\end{scriptsize} 
		\State insert $w_j$ into $W_p$
		\State $i\gets j$
		\Until{i = 1}
	\end{algorithmic}
\end{algorithm}

\subsection{Reactive Layer}

Navigating in highly dynamic environments requires more safety measures as the planned path from the deliberative layer can become unsafe whenever new obstacles are detected especially if they are dynamic.
Hence, the reactive layer generates reflex-like reactions to detected obstacles by navigating around them until it's safe to continue following the previously planned path.

In our implementation, we adopt a reactive control law utilizing sliding mode control technique based on \cite{matveev2011method,hoy2012collision}.
Distance to closest obstacle $d(t)$ is the only information needed to implement this controller which can be obtained from onboard sensors.
Thus, the reactive layer implements the following navigation law:
\begin{equation}\label{equ:ch3:ReactiveLaw}
	\begin{aligned}
		V(t) &= V_{max} \\
		u(t) &= \Gamma\ u_{max}\ \text{sgn}\Big(\dot{d}(t) + \chi(d(t) - d_0)\Big)
	\end{aligned}
\end{equation}
where a constant forward speed is considered, $d_0>d_{safe}>0$ is a desired distance, $\Gamma=\pm1$ determines the avoidance manoeuvre direction (i.e. clockwise or counter clockwise), and $\text{sgn}(\alpha)$ is the signum function.
Also, $\chi(\beta)$ is a saturation function which is defined as:
\begin{equation}
\chi (\beta) = \left\{\begin{array}{cc}
\gamma \beta & if\  |\beta|\leq \delta \\
\delta \gamma\ \text{sgn}(\beta) & otherwise
\end{array}\right.,\ \ \gamma,\delta>0
\end{equation}
for some design parameters $\gamma,\delta>0$.

The navigation law \eqref{equ:ch3:ReactiveLaw} can ensure that the vehicle will maintain a fixed distance $d_0$ while navigating around the nearest obstacle under some assumptions as was mathematically proven in \cite{matveev2011method}.
Once it is safe, the vehicle should continue following the planned path as described next.

\subsection{Execution Layer} \label{subs:ch3:exe}

Decision making for the autonomous navigation stack is managed by the execution layer.
It is responsible for deciding whether it is safe to follow the planned path, to activate the reactive layer or to issue replanning commands.
In general, two navigation modes will be used and a switching mechanism implemented within the execution layer switches between those two modes.
The two modes are: path tracking mode $\mathcal{M}_1$ and obstacle avoidance/reactive mode $\mathcal{M}_2$.

The path tracking mode is initially activated after acquiring a global path.
This mode adopts the Pure Pursuit tracking (PP) algorithm which is known for its stability and simplicity \cite{amidi1991integrated}.
Assuming that a geometric path $\tau \subset \R^2$ is available, the PP algorithm steers the vehicle to follow a virtual target $p_v=(x_v,y_v) \in \tau$ moving on the path.
This target is usually selected as to be at some lookahead distance $L$ away from the closest path point $p_c\in \tau$ to the vehicle's current position.
For more stability, a modified version of the PP algorithm is considered in this work based on \cite{giesbrecht2005path} which suggest using an adaptive lookahead distance instead of a fixed value.
This PP variant provides more stability when the vehicle is further from the planned path which is needed here since the vehicle can sometimes deviate from the planned path when activating the reactive mode.
The following control law is used for path tracking:
\begin{equation}\label{equ:ch3:TrackingLaw}
	\begin{aligned}
		V(t) &= V_{max} \\
		u(t) &= u_{max}\ sgn\Big(\theta_d(t) - \theta(t)\Big) \\
		\theta_d(t) &= \tan^{-1}\frac{y_v(t) - y(t)}{x_v(t) - x(t)}
	\end{aligned}
\end{equation}

On the other hand, the obstacle avoidance mode is simply activating the reactive layer to generate velocity commands to navigate around obstacles as was presented earlier.

The critical role of the execution layer is to handle the switching between the two modes to make sure that the vehicle can safely reach the final goal while avoiding trapping situations when possible.
Assuming that the vehicle initially starts in mode $\mathcal{M}_1$, we consider the following switching rules: \\[0.1cm]
\noindent\textbf{R1:} switch to the reactive mode $\mathcal{M}_2$ when the distance to the closest obstacle $\obs_i$ drops below some threshold distance $C < d_{sensing}$ (i.e. $d_i(t)=C$ and $\dot{d}_i(t)<0$).\\[0.1cm]
\noindent\textbf{R2:} switch to the Path Tracking mode $\mathcal{M}_1$ from $\mathcal{M}_2$ when $d_i(t)>C\ \forall i$ and the vehicle's heading is targeted towards the lookahead goal $p_G$ on the planned path.

Moreover, trapping situations can be detected by the execution layer whenever the reactive controller gets stuck in a local minima.
An example to such scenario can be seen when navigating in maze-like environments, concave obstacles and/or long blocking walls \cite{nakhaeinia2015hybrid}.
The proposed approach to tackle this problem is by issuing a replanning command to the deliberative layer to acquire a new path based on the updated knowledge about the environment.
A description of this approach is given next.

Consider a frontal field of view (FOV) $\alpha\in[\theta-\alpha_0,\theta+\alpha_0]$ where $\alpha_0\leq \frac{\pi}{2}$ achieved by an onboard sensor such as a LIDAR or a depth camera which can be in a form of a point cloud or a discrete array of distances with certain angular resolution.
We characterize all directions or points within the FOV using a mapping function $M(\alpha,t)$ motivated by \cite{savkin2014seeking}.
This function simply maps available sensor measurements at an given time $t$ as follows:
\begin{itemize}
	\item $M(\alpha,t) = 0$ for clear directions 
	\item $M(\alpha,t) = 1$ for directions blocked by an obstacle
\end{itemize}
Hence, trapping situations are detected when $M(\alpha,t)=1,\ \forall \alpha\in[\theta-\alpha_0,\theta+\alpha_0]$.
Whenever this happens, the execution layer triggers a replanning command which will be performed by the deliberative layer.
One possible way is for the vehicle to come to a complete stop or hover in place.
Another possible approach which is considered here is to determine an intermediate \textit{escape goal} $p_{ge}$ for the vehicle to move towards until a new safe path is available.
This helps compensating for the planning time $\tau_p$ which should be less than the estimated time $\triangle \tau$ to reach $p_{ge}$.
In this case, the global planner will consider $p_{ge}$ as the starting position for the replanning process.
Note that the vehicle will be in the path tracking navigation mode $\mathcal{M}_1$ when moving towards $p_{ge}$ and when the new path becomes available.

\begin{remark}
	The RRT planner used in this work considers the $\R^2$ Euclidean space for a better computational cost.
	Therefore, there could be a small deviation from the new planned path when the vehicle reaches $p_{ge}$ if the minimum turning radius is not satisfied depending on the vehicle's heading angle when it reaches  $p_{ge}$.
	This is handled by allowing for some safety margin around $p_{ge}$ when it is decided.
	Howver, there exist variants of the RRT planner that considers non-holonomic constraints within the planning process such as \cite{yang2014spline} that could be used instead.
\end{remark}

\begin{remark}\label{rem:ch3:Escape}
	Escape goals $p_{ge}$ can be generated randomly in a safe direction $\alpha_{ge}$ within the FOV (i.e. $\alpha_{ge} \in [\theta-\alpha_0,\theta+\alpha_0]$) at some distance $l_{ge}$ from the current position.
	Thus, $p_{ge}$ can be computed according to the following:
	\begin{equation}
		p_{ge} = p_{tr} + l_{ge} \left[\begin{array}{c}
			\cos\alpha_{ge} \\ \sin\alpha_{ge}
		\end{array}\right]
	\end{equation}
	where $p_{tr}$ is the position at which the trapping situation is detected.
\end{remark}

\section{Simulations}\label{sec:ch3:sim}

The developed hybrid navigation approach was tested using simulations considering different scenarios including static and dynamic environments.
Some knowledge about these environments were assumed known a priori to show the role of the deliberative layer; however, the developed approach can work well even when no such knowledge is available.
Three simulation cases are presented in this section.
The first two cases deal with unknown static and dynamic obstacles respectively.
The third case shows how the proposed navigation scheme can detect and handle trapping situations. 
We considered a maximum linear speed of $V=1~m/s$ and a maximum angular velocity of $u_{max} = 1.5~rad/s$ in all simulations.
Also, the update rate of the control was set to be $0.1$ seconds.
Description and results of these simulations are presented next.

\subsection{Case I: Unknown Static Obstacles}

In the first simulation case, we consider an environment in which the UAV does some repetitive tasks.
Hence, some knowledge about the environment is known a priori  such walls locations.
However, new static obstacles can be added to the environment at different times which are not known to the vehicle.
A real life example of this scenario is when operating in a warehouse or inside a building.
The  warehouse/building layout can be known in advance or from an initial mapping process while objects can be moved around all the time.

Figure~\ref{fig:ch3:simStatic1} shows the considered environment for this simulation case.
The initial map available to the vehicle includes only information about the two walls (shown in black).
There are also 30 randomly generated unknown obstacles with different sizes (radius between $0.2m$ and $1m$).
Note that the generation process of obstacles rejects generating new obstacles that are very close to available ones to avoid having a cluttered or blocking situation.
A bounding shape can usually be estimated in practice by the vehicle's perception system to represent very close obstacles as a single object. 

It can clearly be seen from Fig~\ref{fig:ch3:simStatic1} that the proposed hybrid navigation strategy can safely guide the vehicle starting from some initial position (green square marker) to reach the goal location (red star marker).
This figure shows both the path generated by the global planner based on the initial knowledge about the environment as well as the actual executed path by the vehicle.
The vehicle can successfully track the planned path whenever it has good clearance from obstacles.
Once an obstacle is detected by the vehicle's sensors, the vehicle switches to the reactive mode $\mathcal{M}_2$ to move around the obstacle.
Then, it goes back to the path tracking mode $\mathcal{M}_1$ whenever it is clear to do so according to the switching rule \textbf{R2} as described in subsection~\ref{subs:ch3:exe}.
Example locations during the motion when such switching from $\mathcal{M}_1$ to $\mathcal{M}_2$ occurs are shown in \Fig{fig:ch3:simStatic2}.
Overall, these results confirms that the proposed method can guide the vehicle safely among unknown static obstacles.

\begin{figure}[h]
	\centering
	\includegraphics[width=0.75\linewidth]{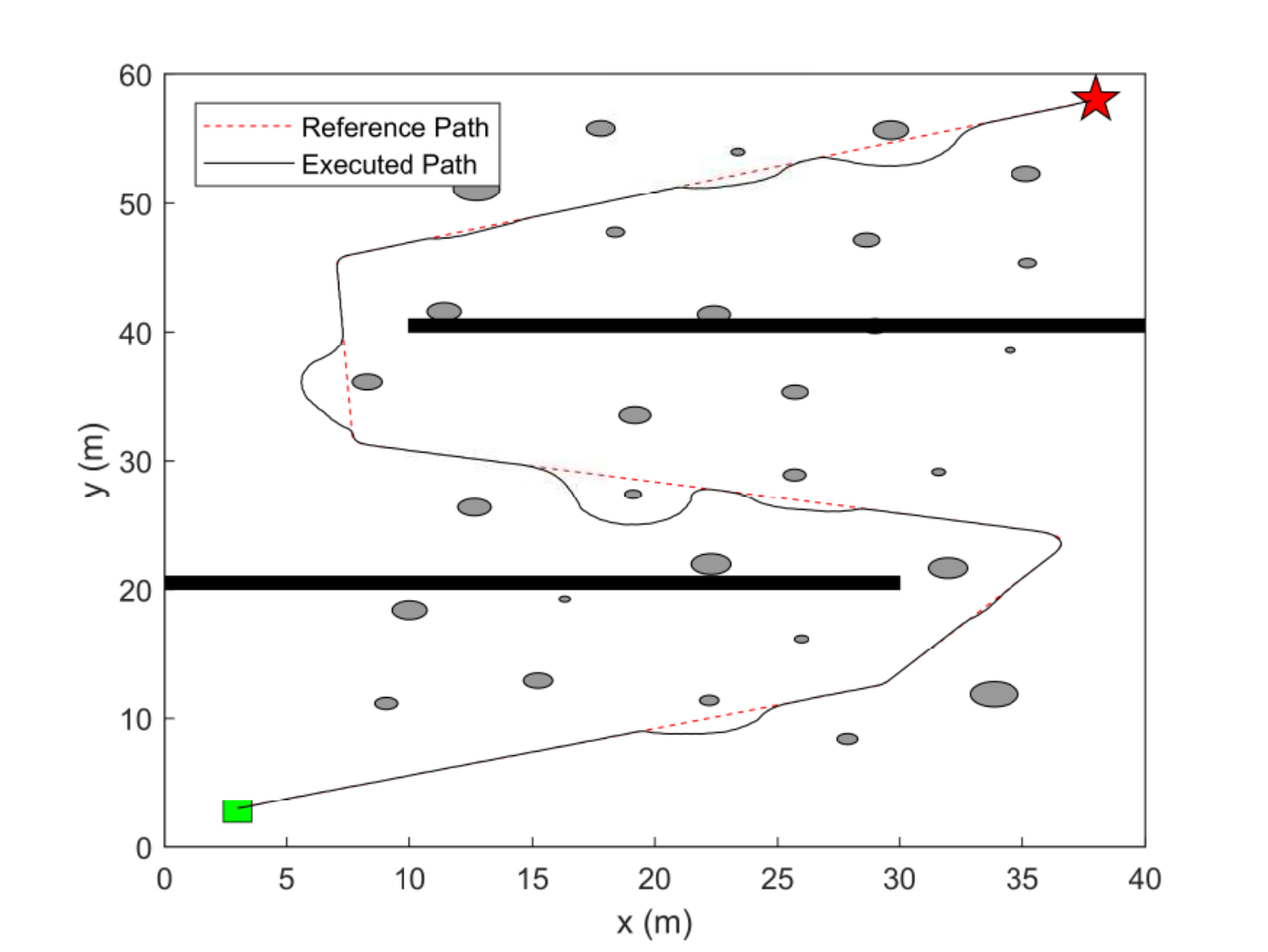} 
	\caption{Simulation Case I: Initially planned path and executed motion (static environment)}\label{fig:ch3:simStatic1}
\end{figure}

\begin{figure}[h]
	\centering
	\begin{adjustbox}{minipage=\linewidth,scale=0.8}
		\begin{subfigure}[t]{0.5\textwidth}
			\centering
			\includegraphics[width=\linewidth]{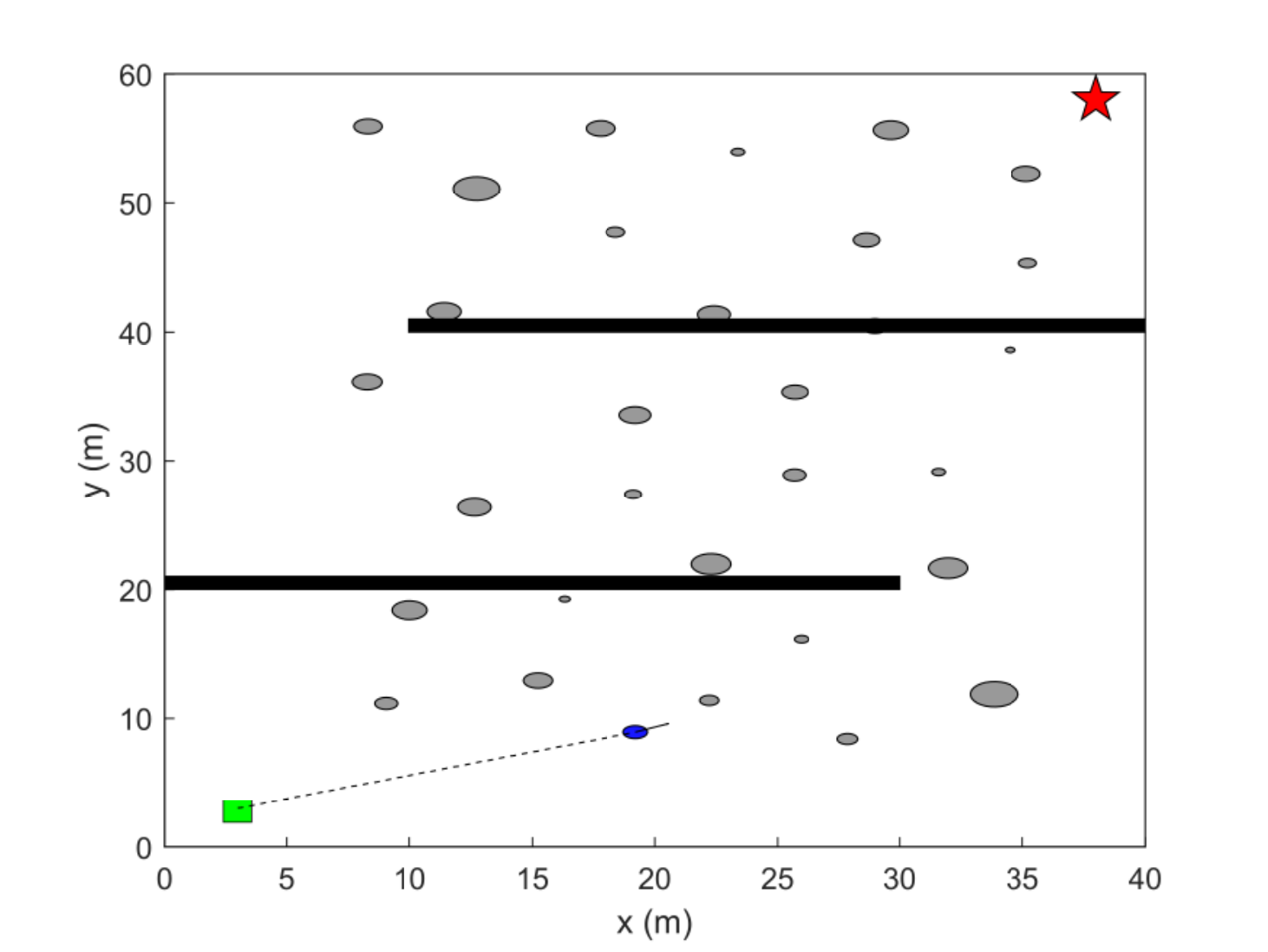} 
			\caption{}
		\end{subfigure}
				\hfill
				\begin{subfigure}[t]{0.5\textwidth}
					\centering
					\includegraphics[width=\linewidth]{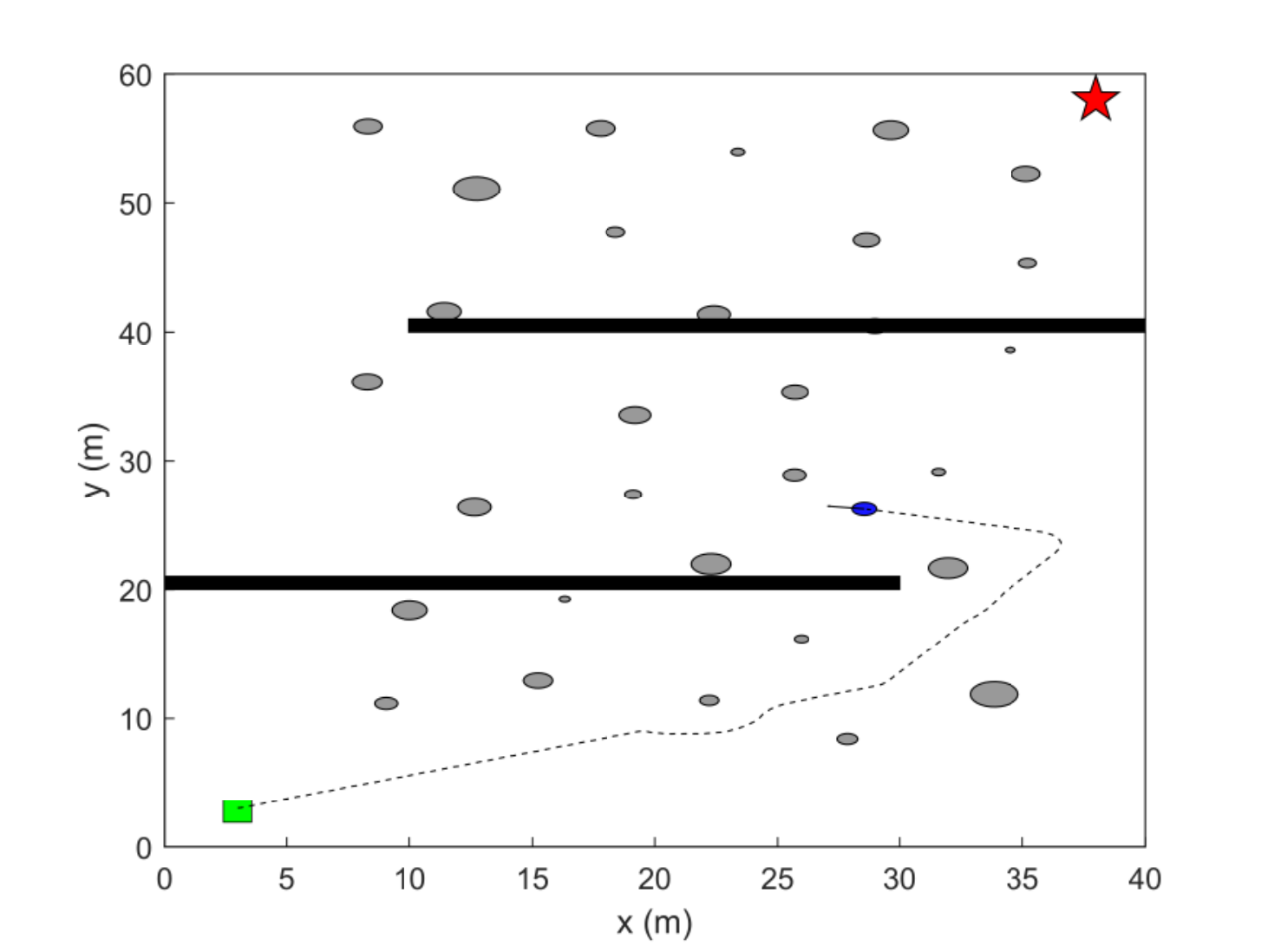} 
					\caption{}
				\end{subfigure}
				
				\begin{subfigure}[t]{0.5\textwidth}
					\centering
					\includegraphics[width=\linewidth]{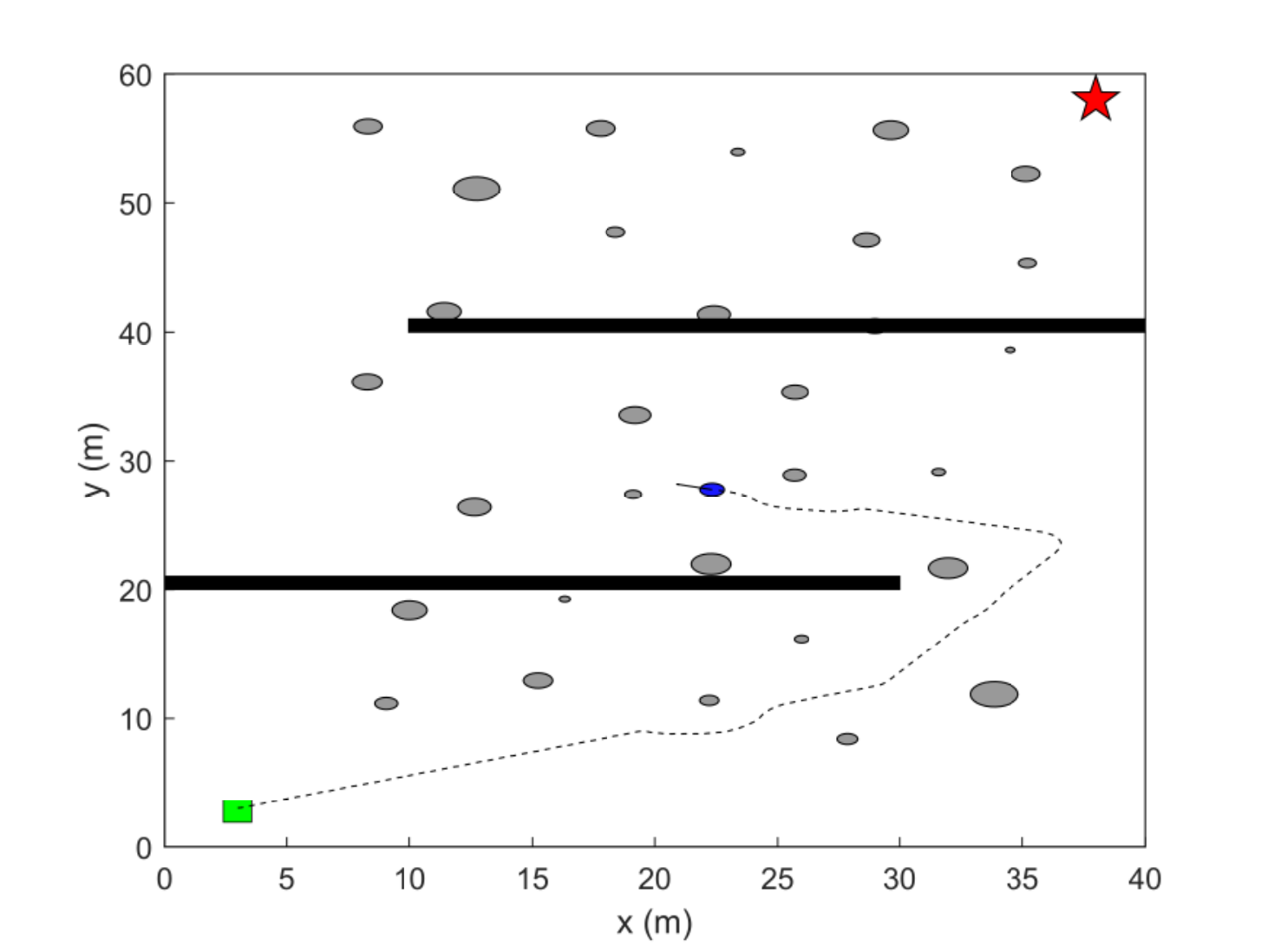} 
					\caption{}
				\end{subfigure}
		\hfill
		\begin{subfigure}[t]{0.5\textwidth}
			\centering
			\includegraphics[width=\linewidth]{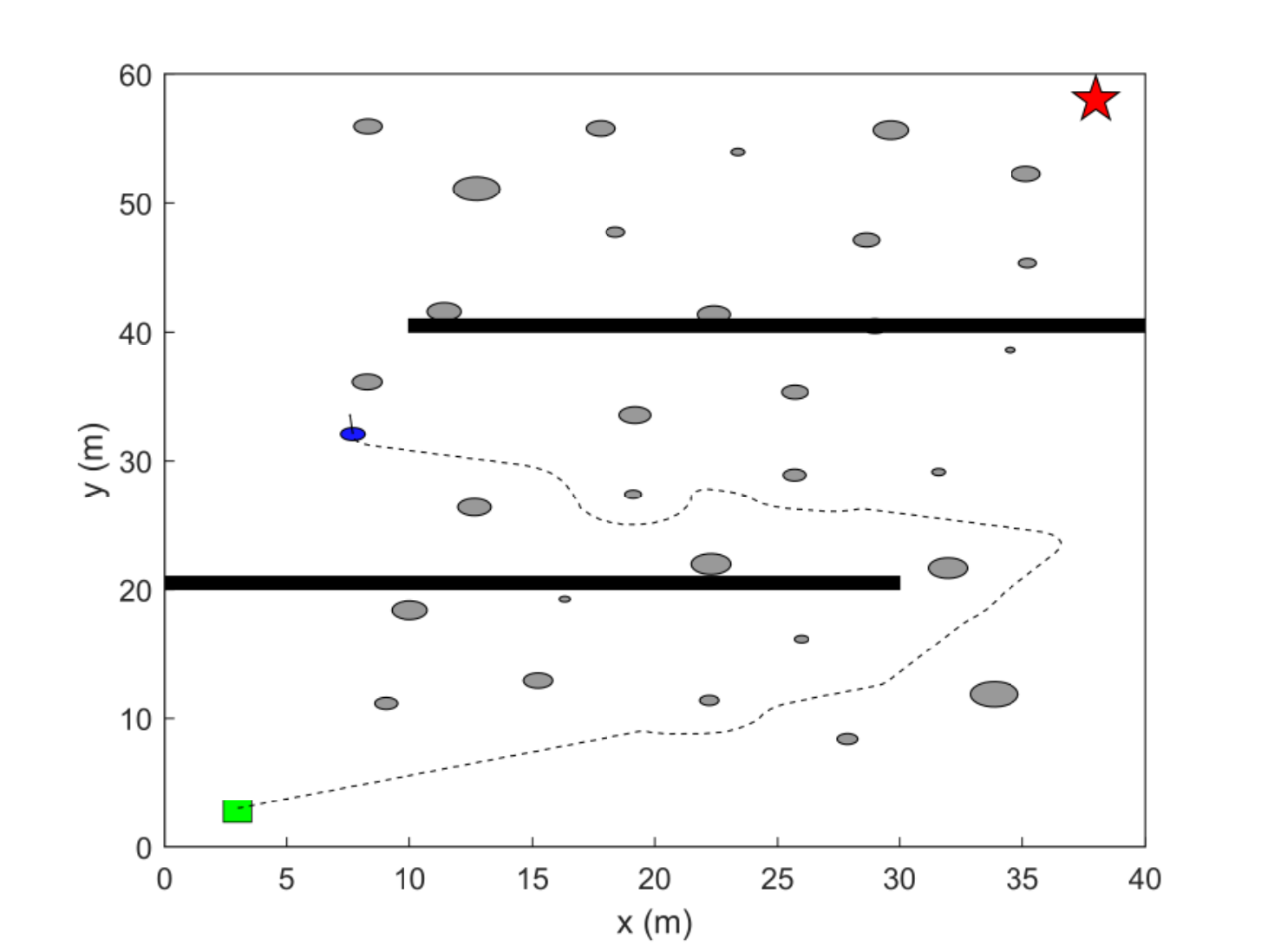} 
			\caption{}
		\end{subfigure}
		
				\begin{subfigure}[t]{0.5\textwidth}
					\centering
					\includegraphics[width=\linewidth]{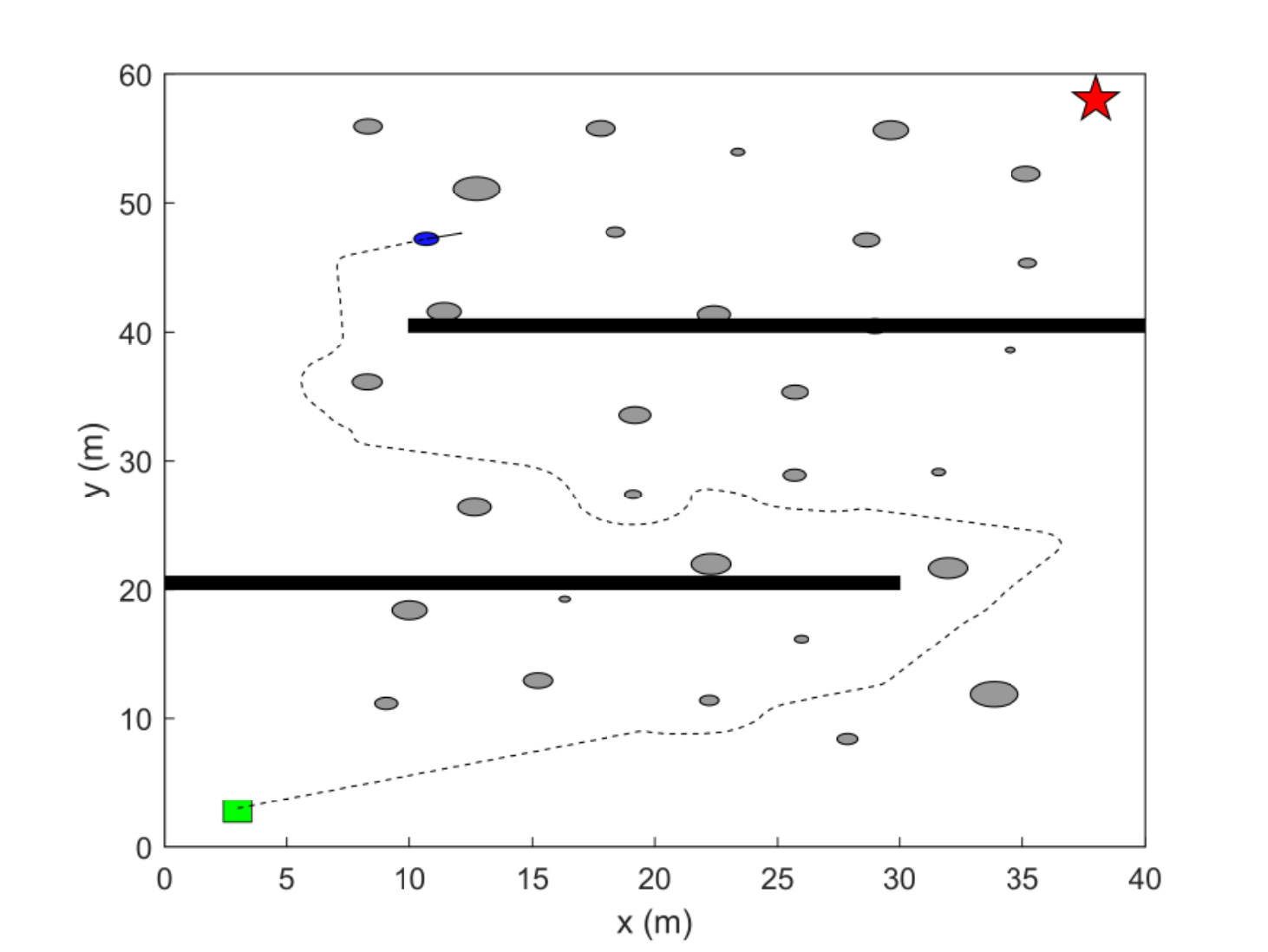} 
					\caption{}
				\end{subfigure}
				\hfill
				\begin{subfigure}[t]{0.5\textwidth}
					\centering
					\includegraphics[width=\linewidth]{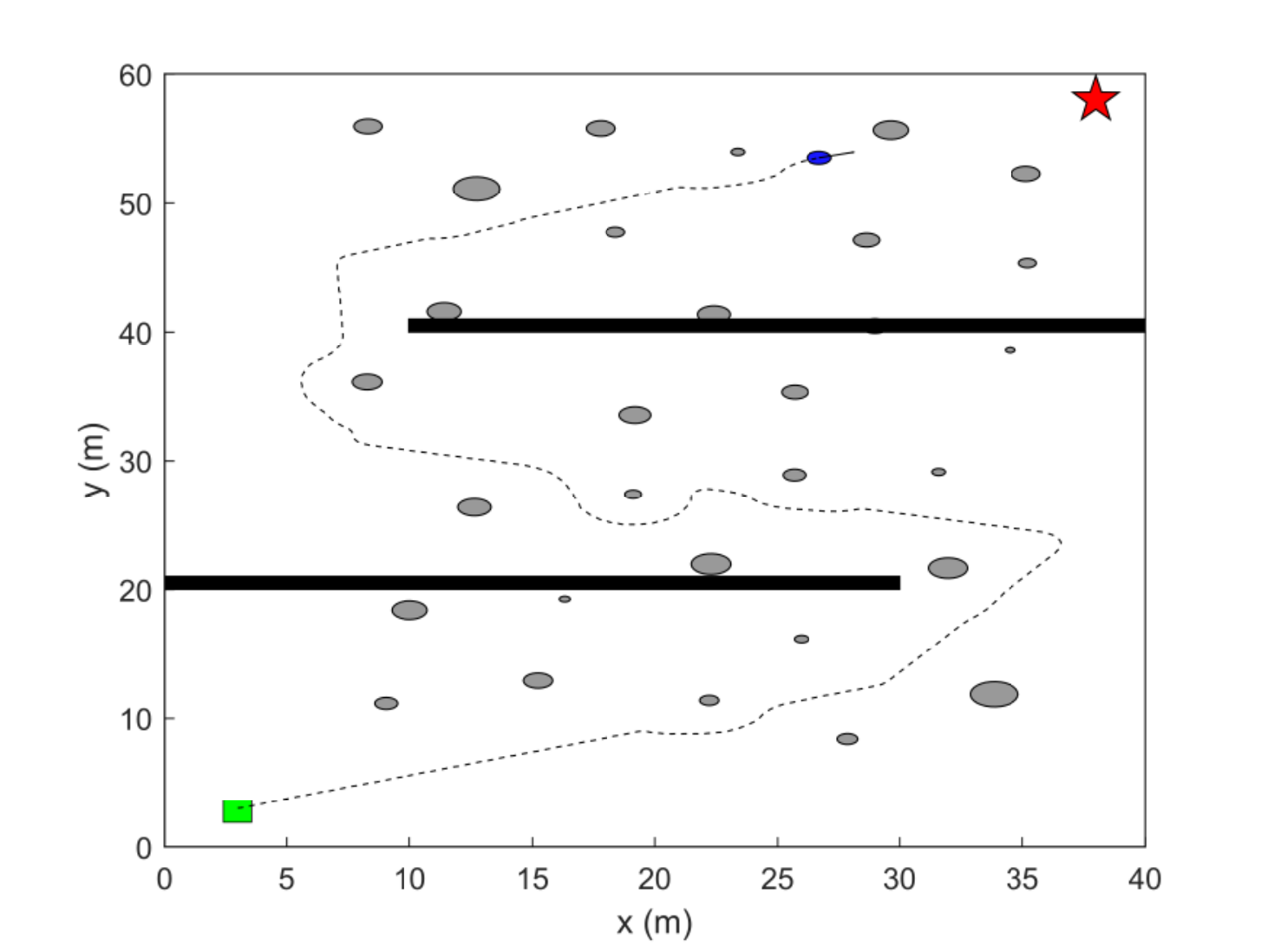}
					\caption{}
				\end{subfigure}
		
		\begin{subfigure}[t]{0.5\textwidth}
			\centering
			\includegraphics[width=\linewidth]{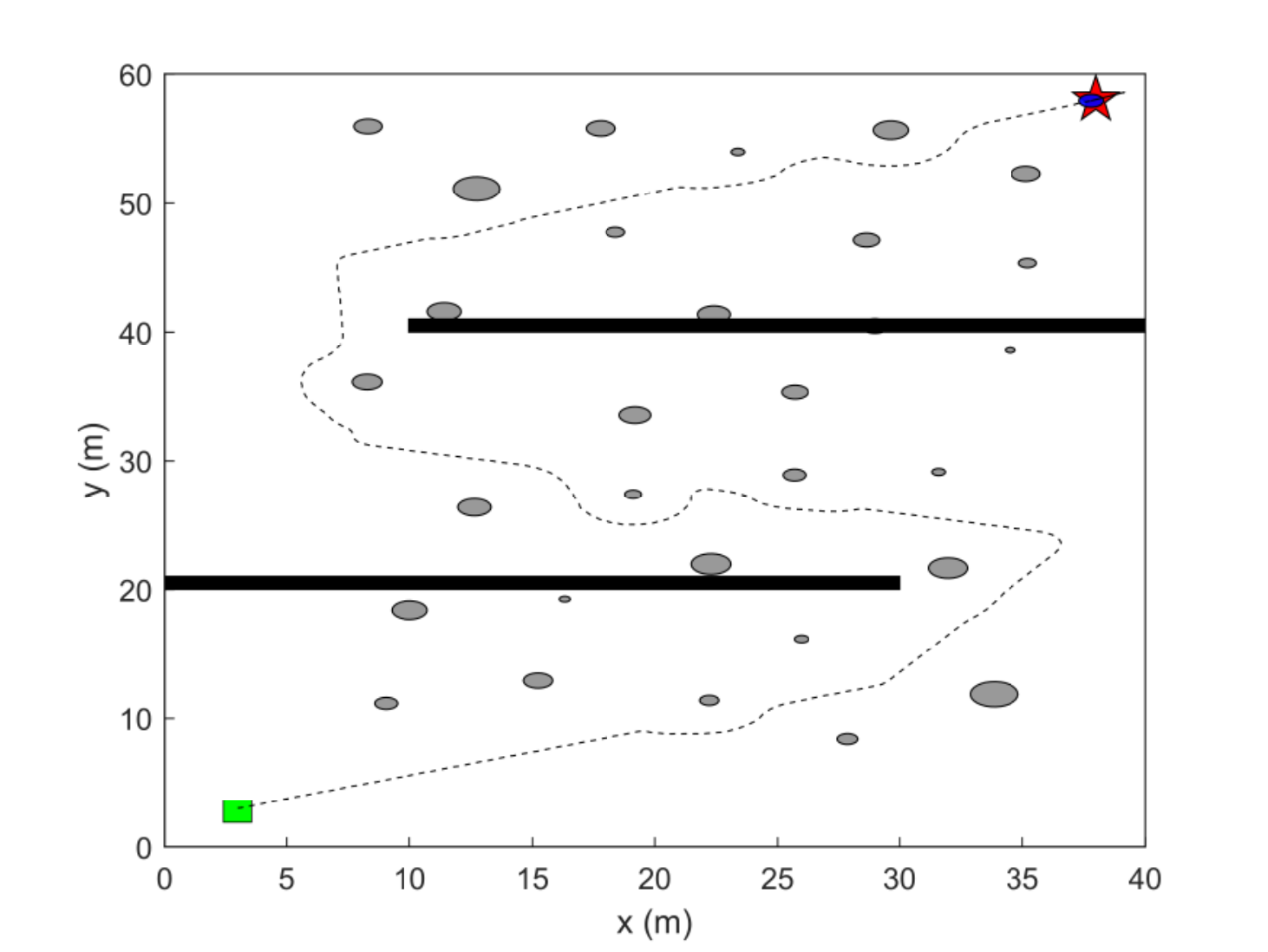}
			\caption{}
		\end{subfigure}
		
		\caption{Simulation Case I: different instances during motion at which switching to reactive mode $\mathcal{M}_2$ is triggered by the execution layer}\label{fig:ch3:simStatic2}
	\end{adjustbox}
\end{figure}

\subsection{Case II: Unknown Dynamic Obstacles}

The second simulation scenario considers the case where there are several unknown moving obstacles which makes the environment highly dynamic.
Similar to previous, some initial knowledge about the environment layout is assumed known which is shown in black in \Fig{fig:ch3:simDynamic1}.
Moreover, there are multiple unknown dynamic obstacles with random sizes between $0.2m$-$1m$ in radius and velocities (less than $1 m/s$) whose trajectories are shown in \Fig{fig:ch3:simDynamic2}.
Note that collisions between different obstacles are not considered in this case for simplicity.

Based on the initial map, the deliberative layer plans a safe path using RRT as explained earlier which is shown as a dashed red line in \Fig{fig:ch3:simDynamic1}.
The vehicle then starts moving in mode $\mathcal{M}_1$ to track the planned (reference) path.
However, due to the highly dynamic nature of the environment, this path becomes unsafe whenever there are obstacles approaching the vehicle as shown in \Fig{fig:ch3:simDynamic3} at different instance during the motion.
In that figure, the vehicle and moving obstacles are represented by blue and grey circular markers respectively with a small arrow denoting the motion direction for each one. 
Each time such a threatening obstacle is detected, the vehicle switches to navigation in the reactive mode $\mathcal{M}_2$ according to the switching rule \textbf{R1}.
This can be seen in Figures~\ref{fig:ch3:simDynamic1}-\ref{fig:ch3:simDynamic2} which show that the vehicle's actual executed path deviates from the planned path at some locations to avoid obstacles.
It is evident from these results that the proposed strategy works well in dynamic environments as well given.
It should be mentioned that the vehicle's maximum velocity much be larger than obstacles' velocities to guarantee safety.
However, the reactive control law can still handle some cases where the obstacles are moving faster the vehicle but with no safety guarantees in some aggressive scenarios.

\begin{figure}[h]
	\centering
	\includegraphics[width=0.7\linewidth]{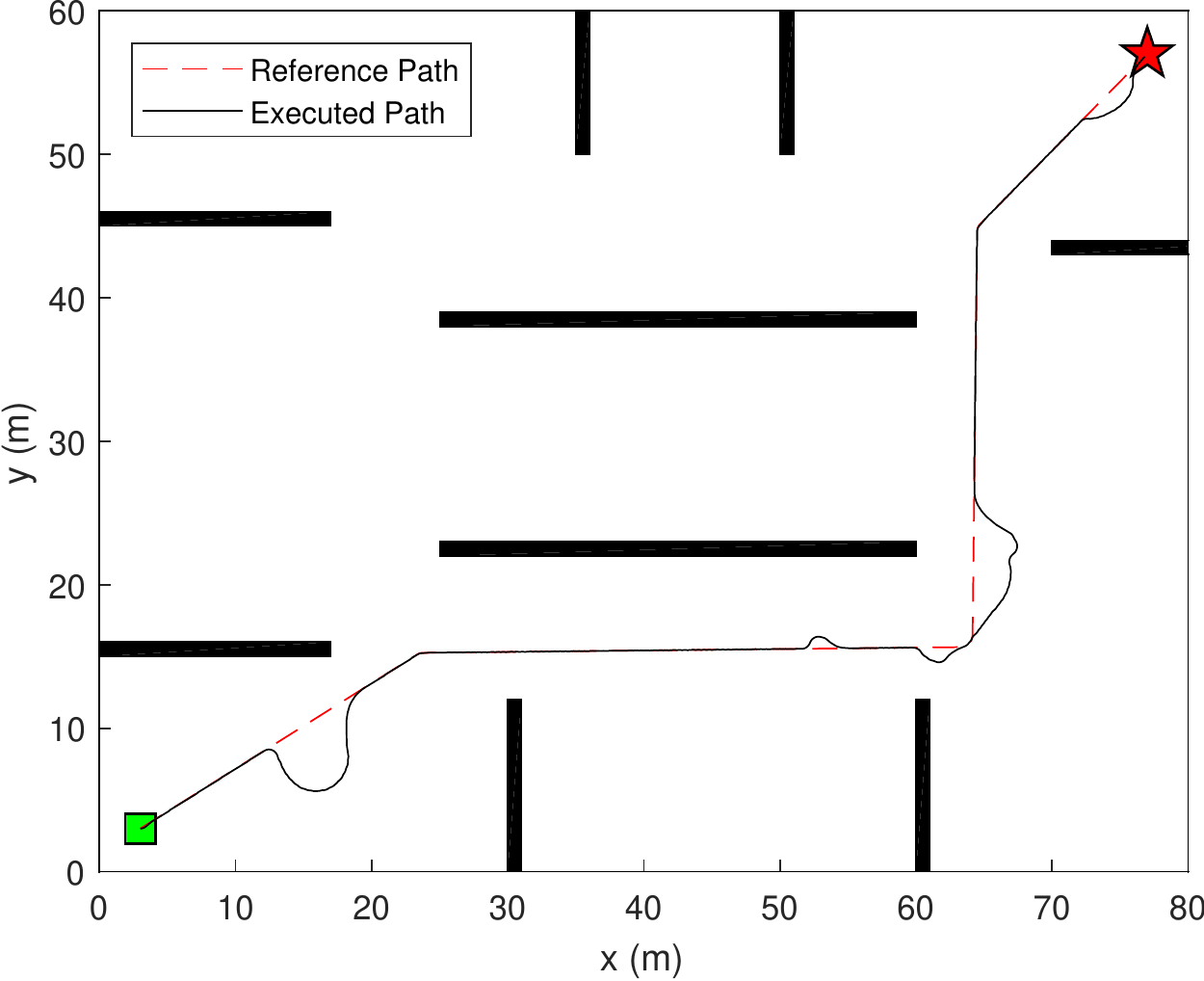} 
	\caption{Simulation Case II: initial planned path and executed motion (dynamic environment)}\label{fig:ch3:simDynamic1}
\end{figure}

\begin{figure}[h]
	\centering
	\includegraphics[width=0.7\linewidth]{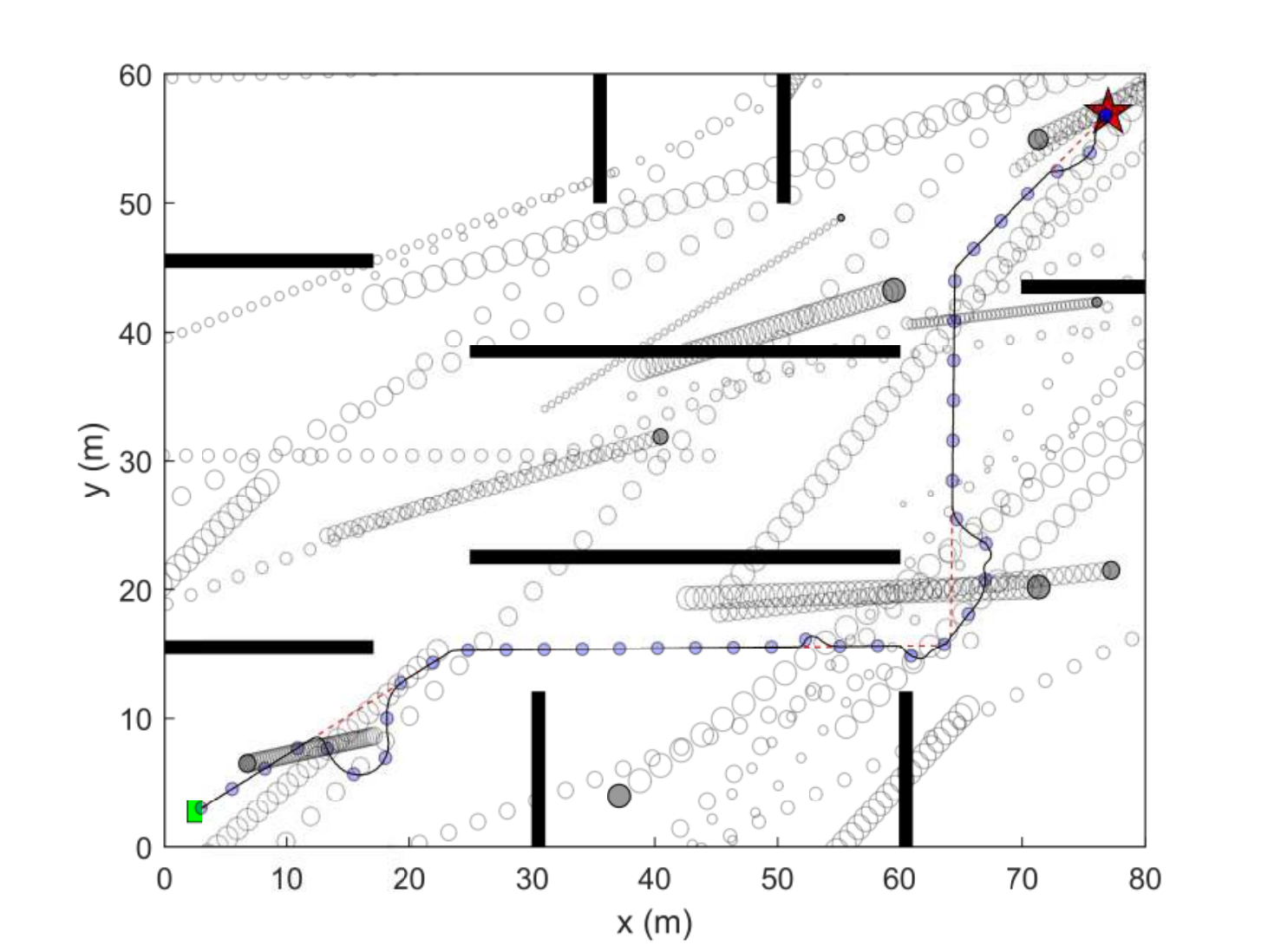} 
	\caption{Simulation Case II: Dynamic obstacles trajectories }\label{fig:ch3:simDynamic2}
\end{figure}

\begin{figure}[h]
	\centering
	\begin{adjustbox}{minipage=\linewidth,scale=1}
		\begin{subfigure}[t]{0.5\textwidth}
			\centering
			\includegraphics[width=\linewidth]{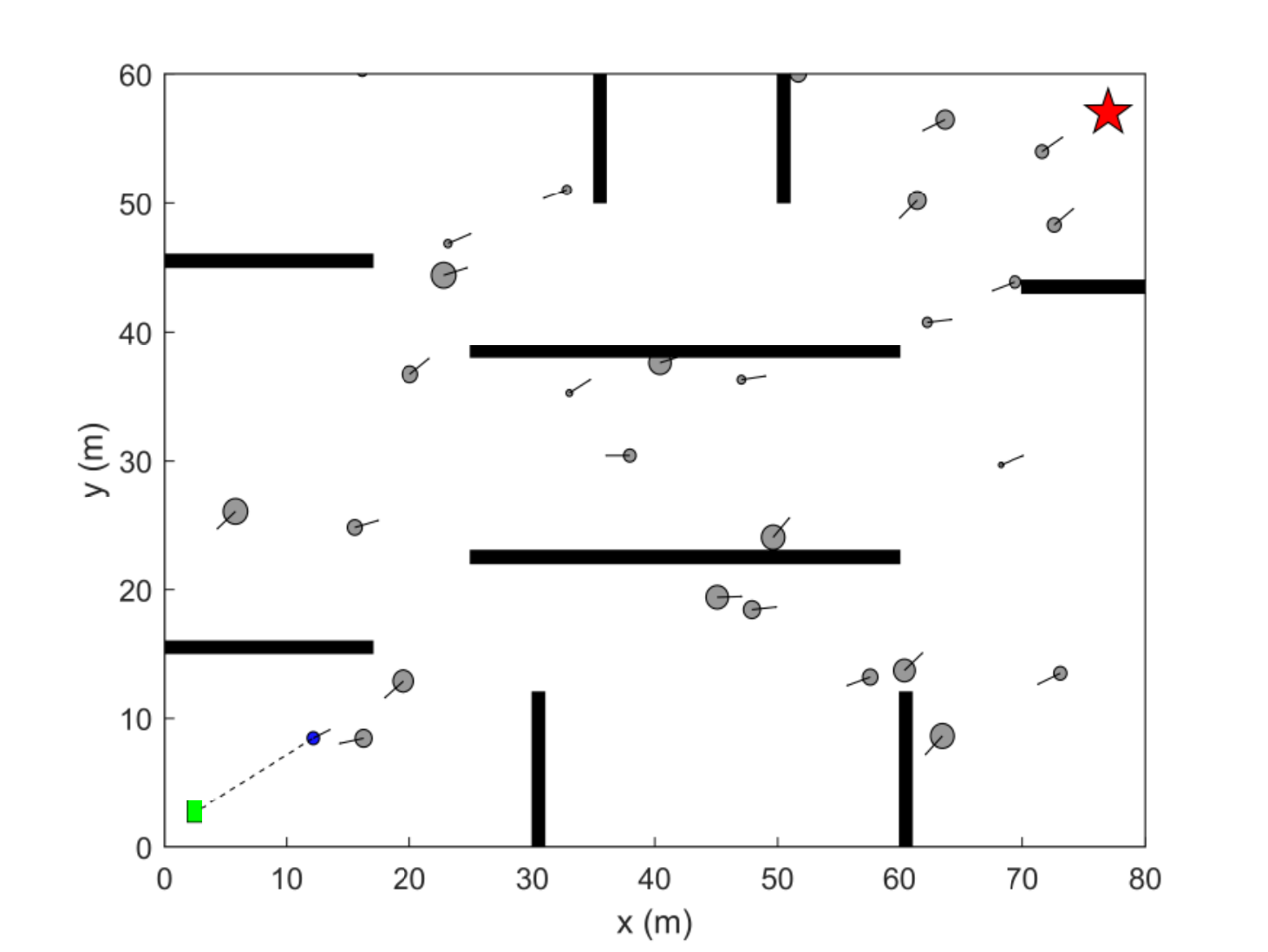} 
			\caption{}
		\end{subfigure}
		\hfill
		\begin{subfigure}[t]{0.5\textwidth}
			\centering
			\includegraphics[width=\linewidth]{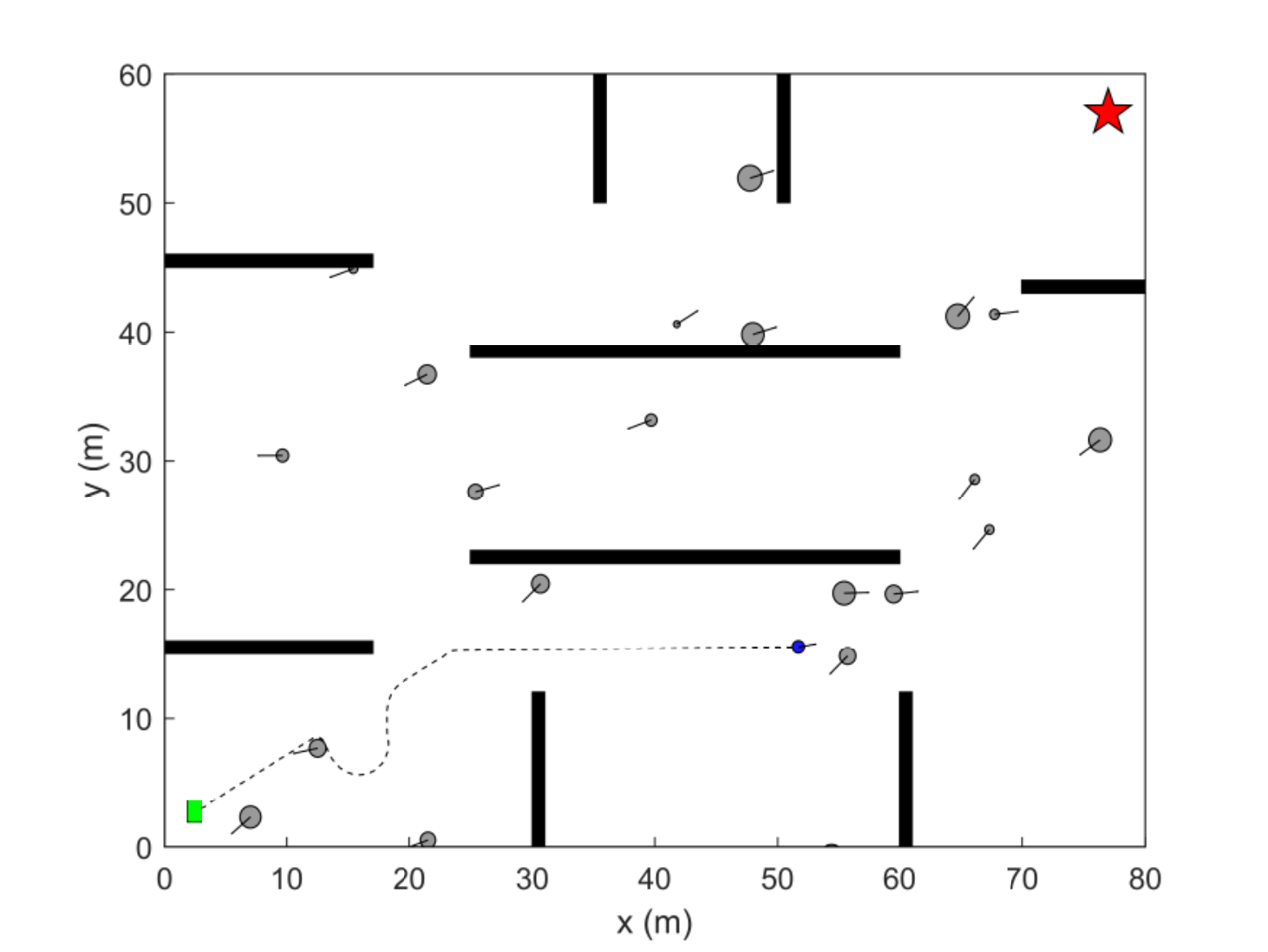} 
			\caption{}
		\end{subfigure}
		
		\begin{subfigure}[t]{0.5\textwidth}
			\centering
			\includegraphics[width=\linewidth]{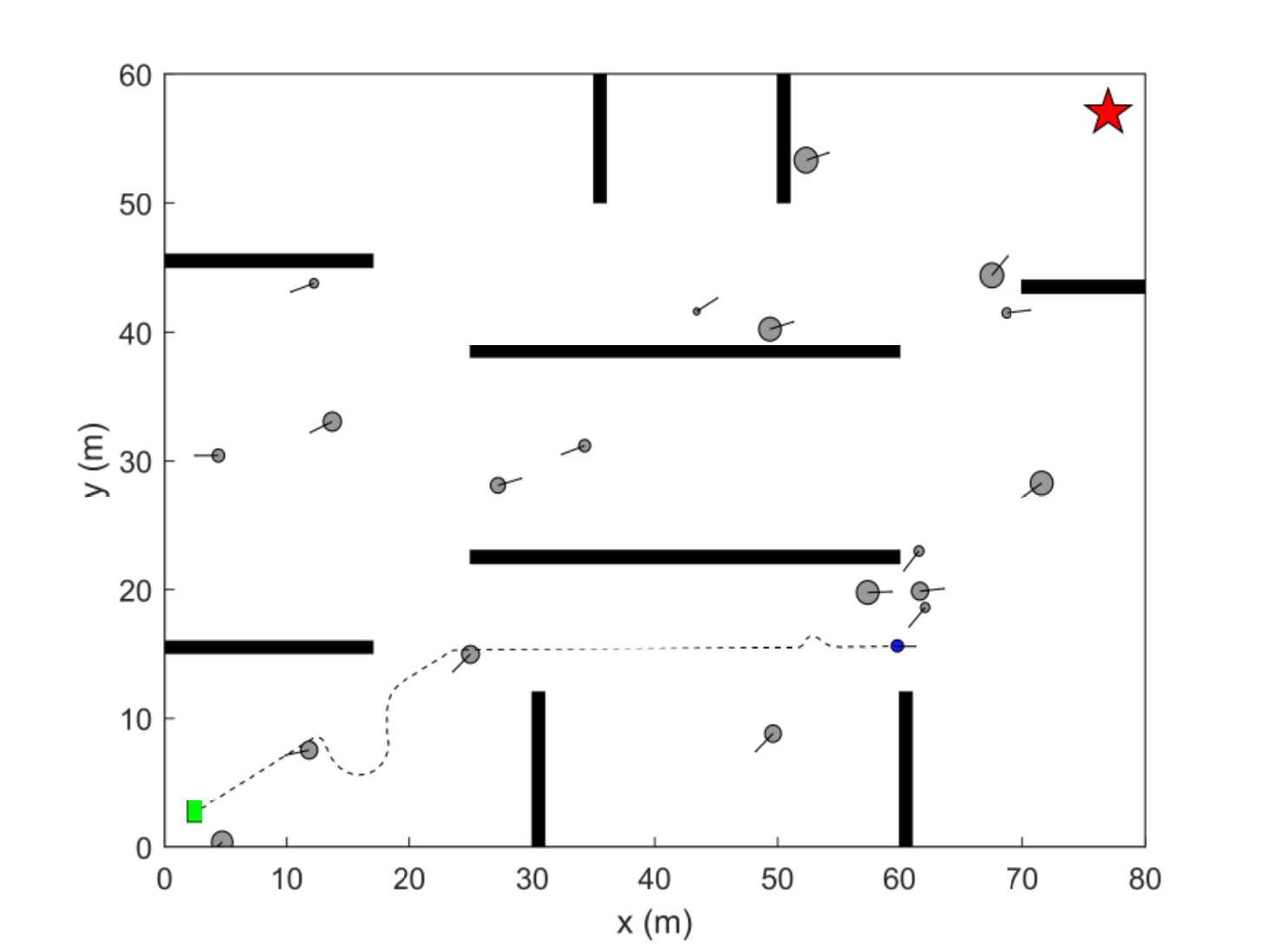} 
			\caption{}
		\end{subfigure}
		\hfill
		\begin{subfigure}[t]{0.5\textwidth}
			\centering
			\includegraphics[width=\linewidth]{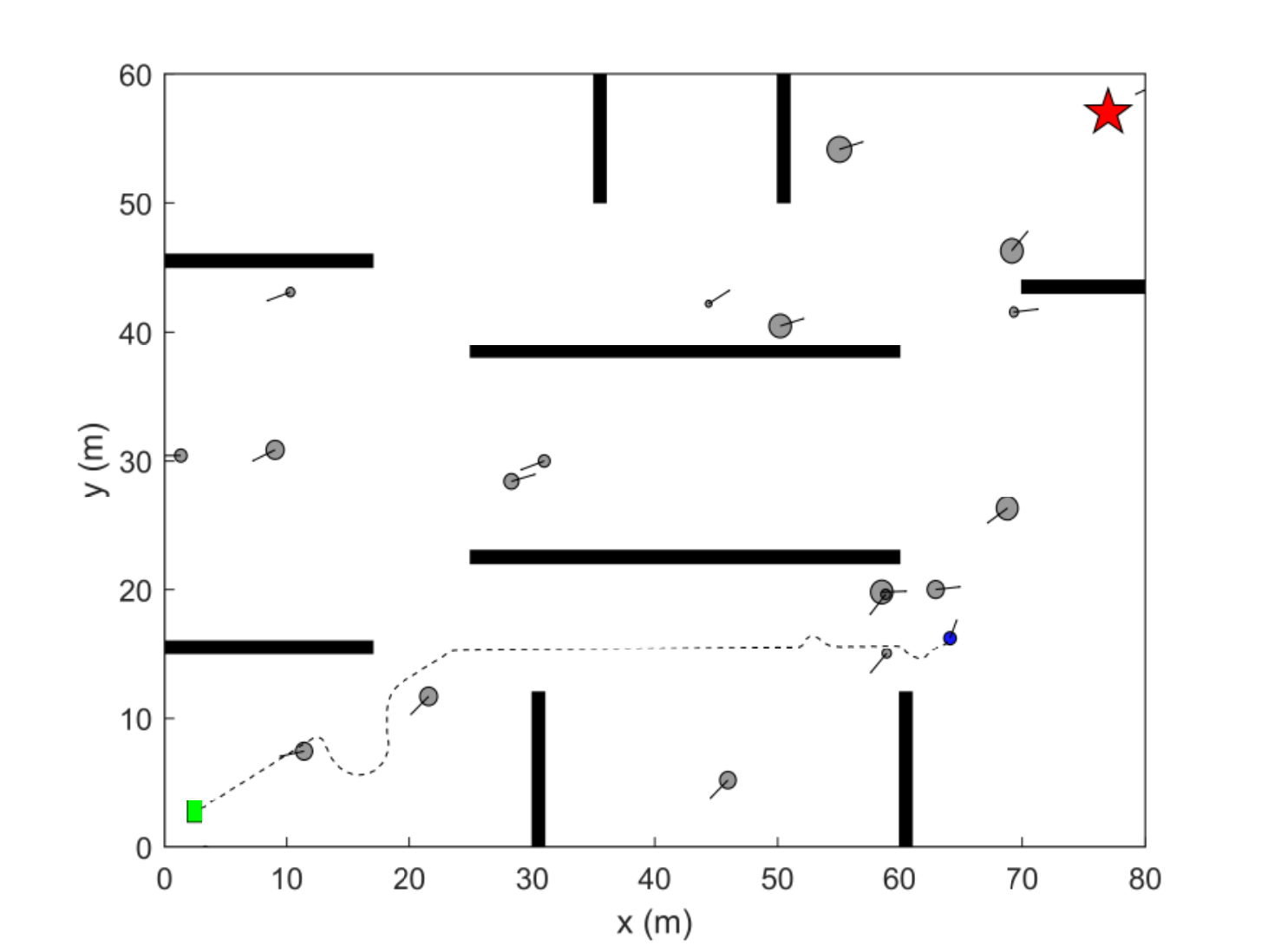} 
			\caption{}
		\end{subfigure}
		
		\begin{subfigure}[t]{0.5\textwidth}
			\centering
			\includegraphics[width=\linewidth]{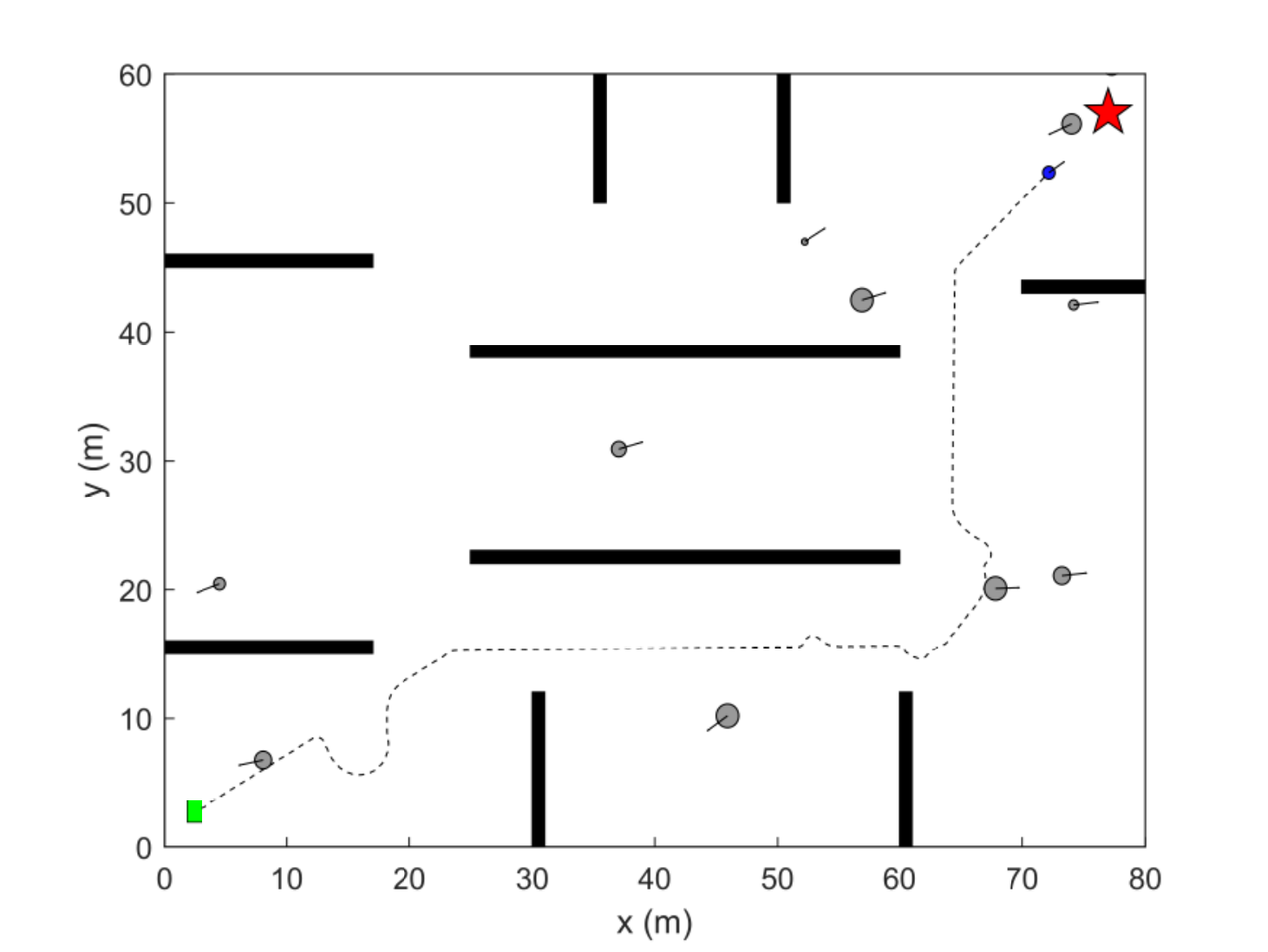}
			\caption{}
		\end{subfigure}
		\hfill
		\begin{subfigure}[t]{0.5\textwidth}
			\centering
			\includegraphics[width=\linewidth]{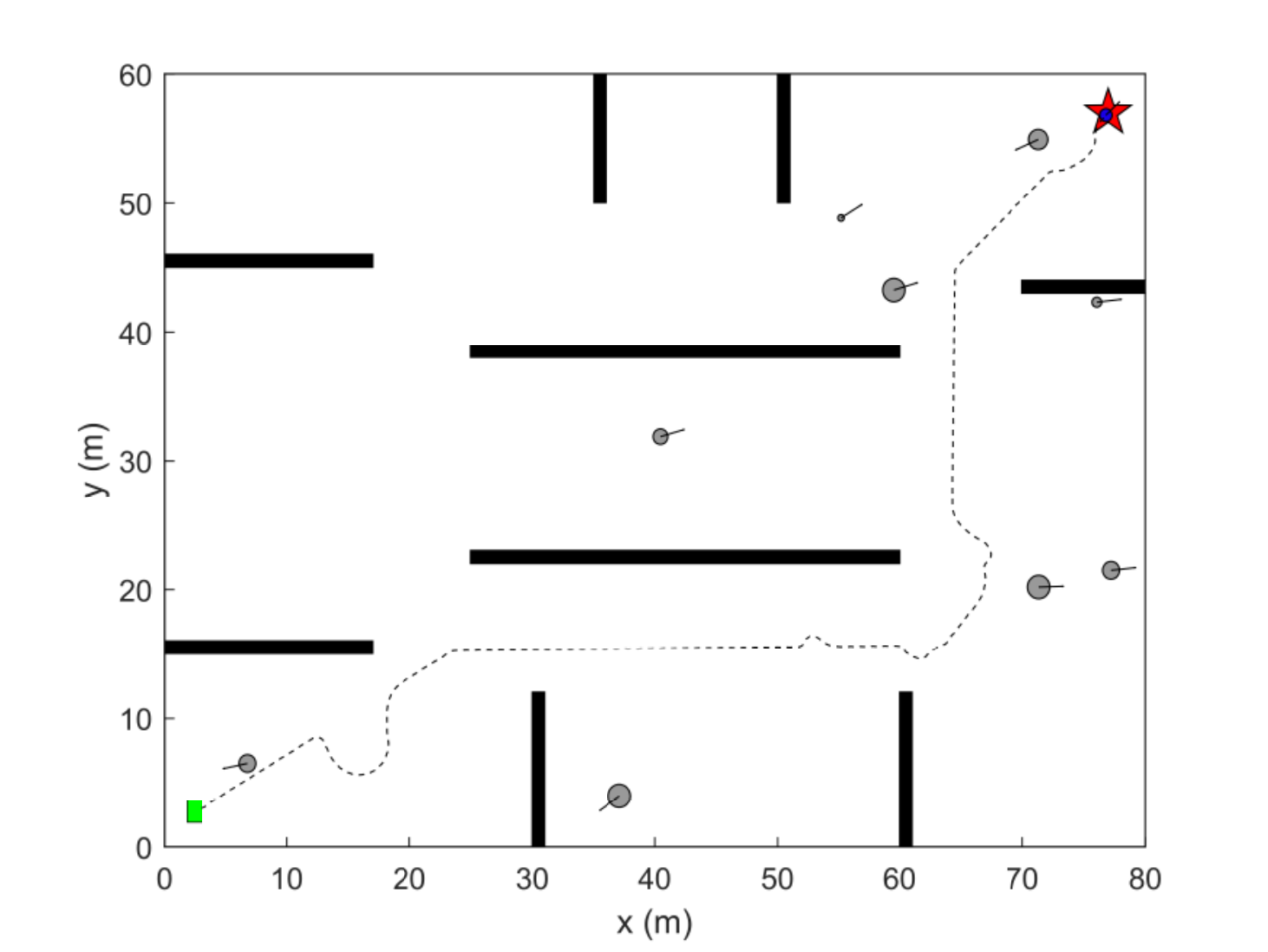}
			\caption{}
		\end{subfigure}
		
		\caption{Simulation Case II: Different instances during motion at which switching to reactive mode triggers}\label{fig:ch3:simDynamic3}
	\end{adjustbox}
\end{figure}

\subsection{Case III: Trapping Situation Scenario}

An additional simulation case was carried out to show the effectiveness of the proposed method in trapping situations.
We considered an environment as shown in \Fig{fig:ch3:simReplan1} where information is initially available regarding obstacles highlighted in black while the grey ones are unknown.
As in previous cases, the vehicle initially plan a path towards the goal location (red star marker).
It is clear from \Fig{fig:ch3:simReplan1} that there are only two ways to reach the goal based on the initial knowledge.
However, one of these two ways is actually obstructed by the large unknown obstacle shown in grey.
We considered one of the simulations where the initial planned path goes through this obstructed way as shown in \Fig{fig:ch3:simReplan1} to show how such scenario is handled by our approach.

As the vehicle tracks the planned path, it reaches a position where most of the FOV is obstructed by a newly detected obstacle (i.e. $M(\alpha,t)=1,\ \forall \alpha\in[\theta-\alpha_0,\theta+\alpha_0]$) as shown in \Fig{fig:ch3:simReplan2} which indicates a trapping situation as was explained in subsection~\ref{subs:ch3:exe}.
Note that the edges of the map are also considered to be non traversable.
Also, we select the design parameter $C$ to be less than the sensing range to avoid switching to the reactive mode $\mathcal{M}_2$ before detecting the trapping situation.
Hence, the execution layer immediately generates an escape goal based on \cref{rem:ch3:Escape} in some direction within $[\theta-\frac{\pi}{2},\theta+\frac{\pi}{2}]$ and at a chosen distance of $3m$.
It also requests a new path from the global planner starting from the escape goal and utilizing the updated map of the environment.
The generated escape goal and new planned path are shown in \Fig{fig:ch3:simReplan3}.
The complete vehicle trajectory is presented in \Fig{fig:ch3:simReplan4} which clearly shows a collision-free motion avoiding other unknown obstacles along the newly planned path.
It is worth mentioning that if a purely reactive method was used in this situation, it would result in an inefficient motion when following the boundary of the obstructing wall in either direction when the vehicle is at the position shown in \Fig{fig:ch3:simReplan2}.
It can be seen the vehicle would take a very long path to reach the goal by just following the walls boundaries.
This shows the importance of considering hybrid navigation strategies as was suggested.

\begin{figure}[h]
	\centering
	\begin{adjustbox}{minipage=\linewidth,scale=1}
		\begin{subfigure}[t]{0.5\textwidth}
			\centering
			\includegraphics[width=\linewidth]{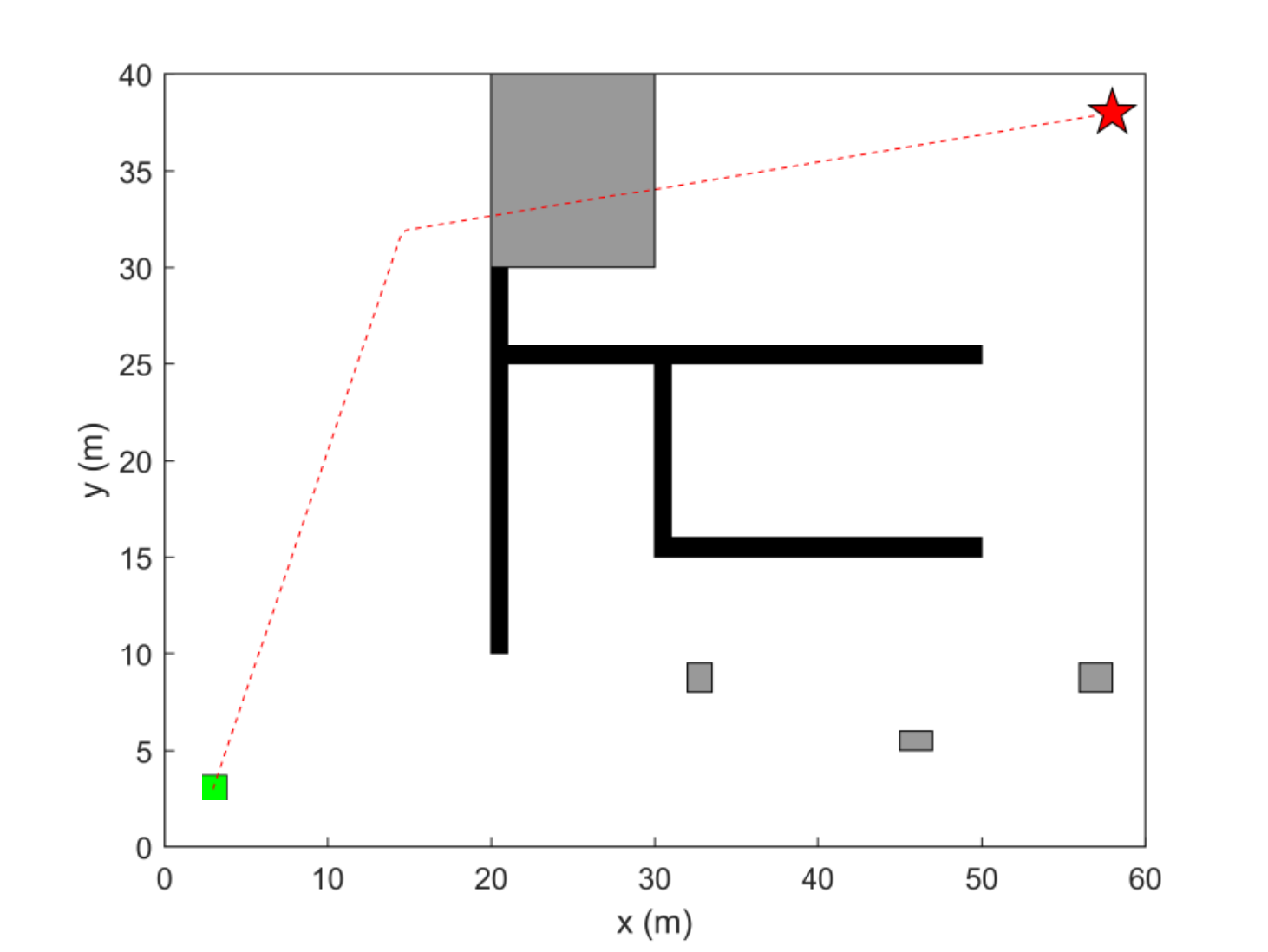} 
			\caption{Initially Planned Path}
			\label{fig:ch3:simReplan1}
		\end{subfigure}
		\hfill
		\begin{subfigure}[t]{0.5\textwidth}
			\centering
			\includegraphics[width=\linewidth]{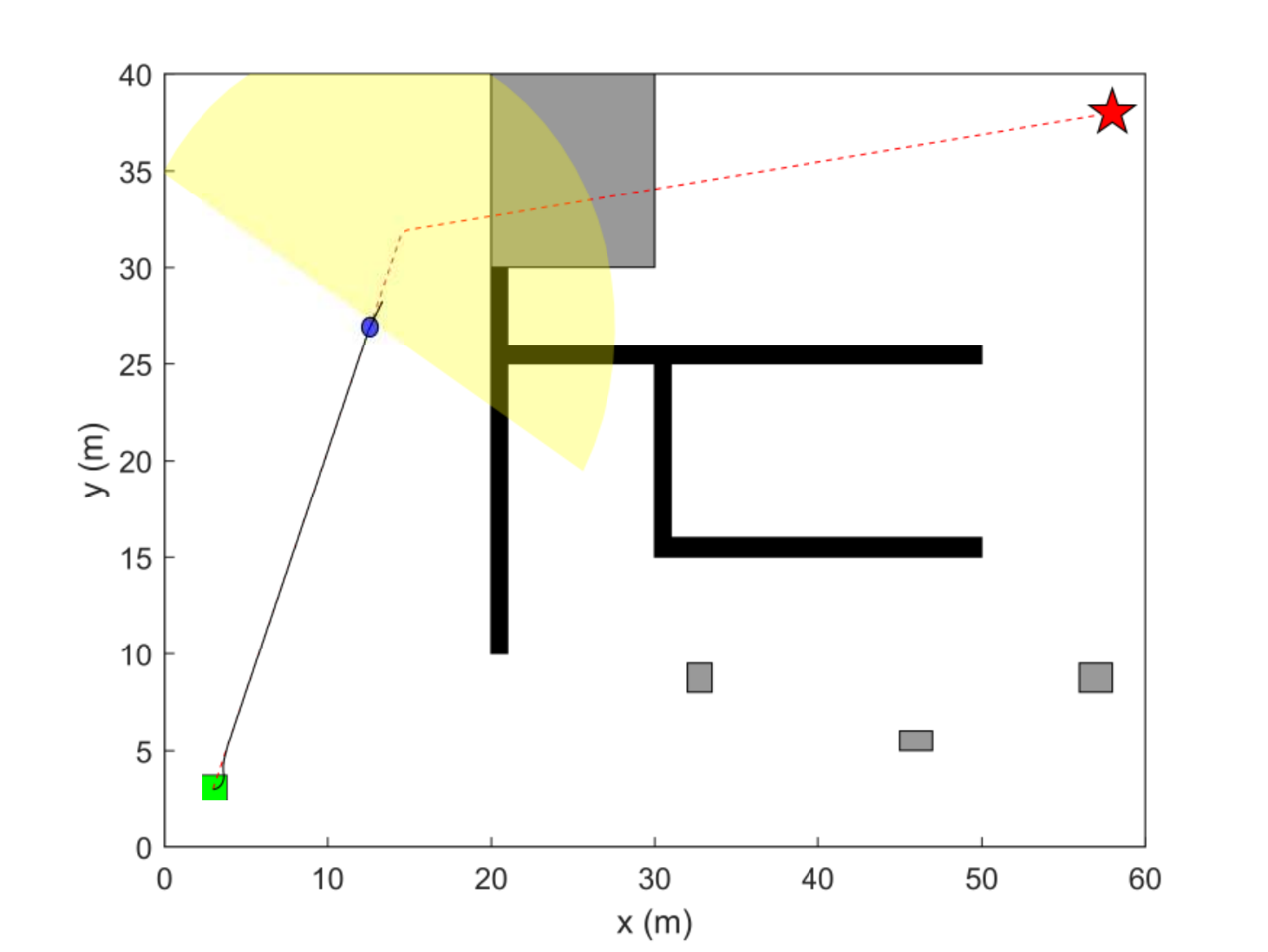} 
			\caption{Trapping Situation Detection}
			\label{fig:ch3:simReplan2}
		\end{subfigure}
		
		\begin{subfigure}[t]{0.5\textwidth}
			\centering
			\includegraphics[width=\linewidth]{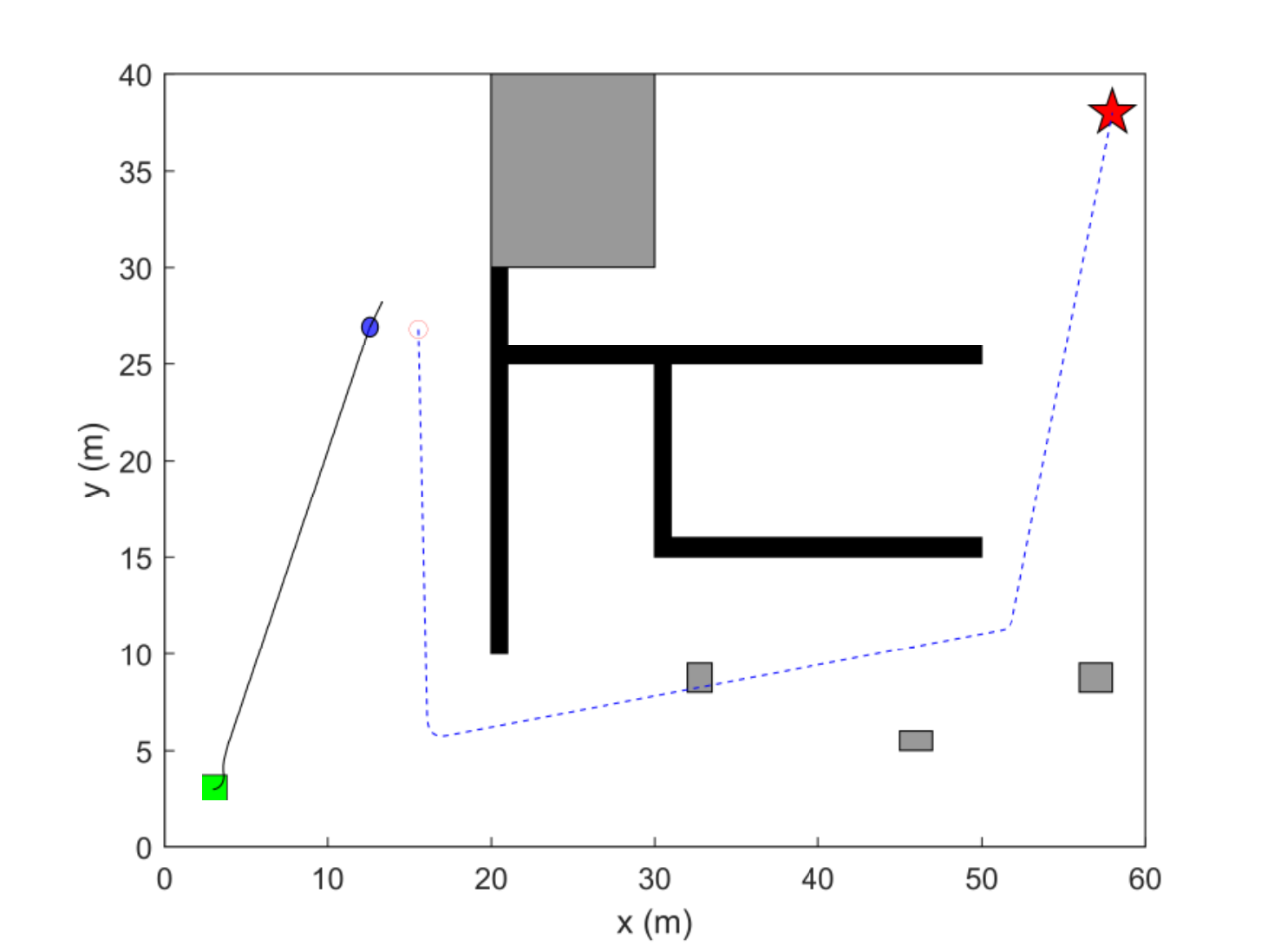} 
			\caption{Escape Goal \& Replanning}
			\label{fig:ch3:simReplan3}
		\end{subfigure}
		\hfill
		\begin{subfigure}[t]{0.5\textwidth}
			\centering
			\includegraphics[width=\linewidth]{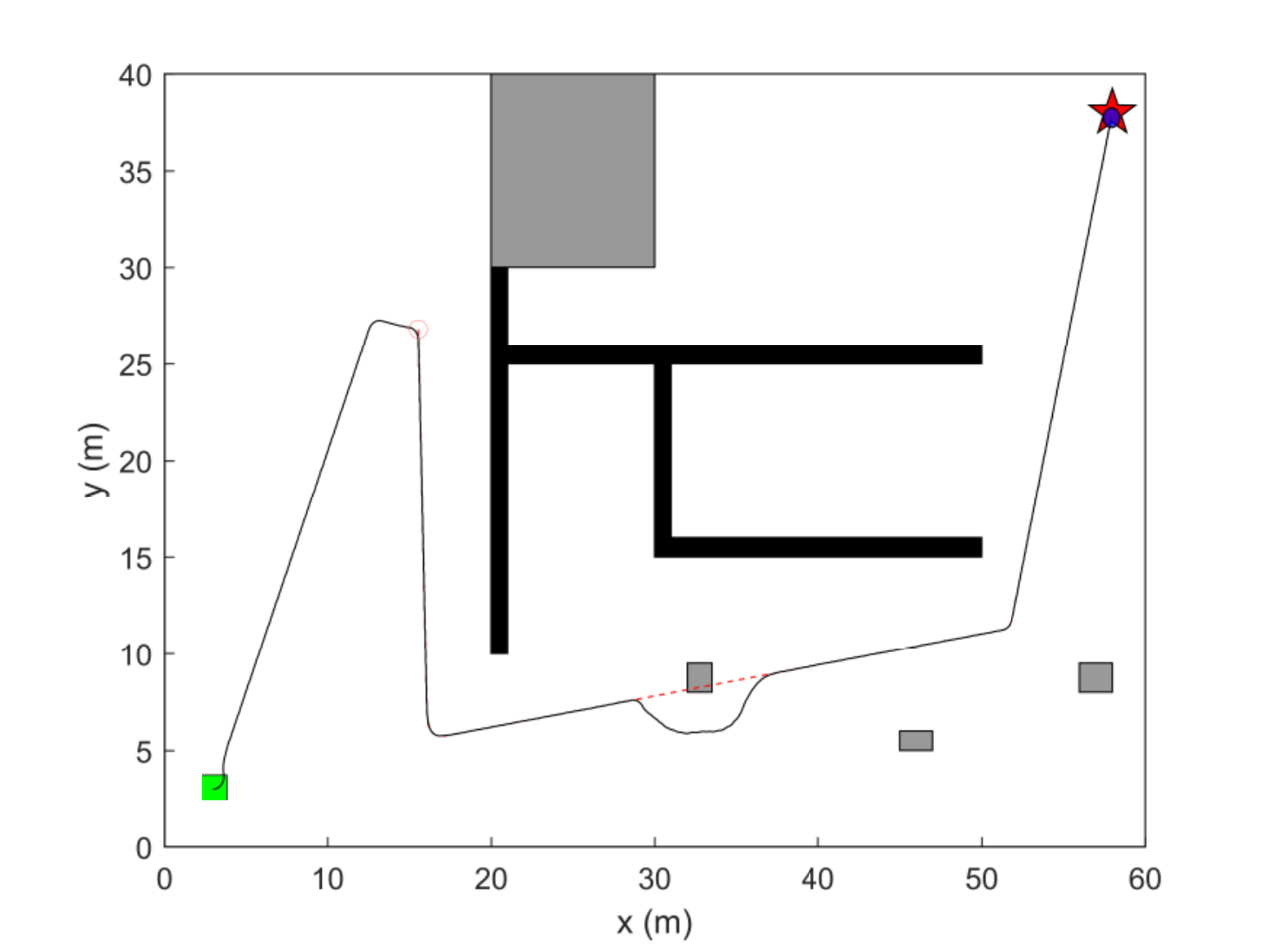} 
			\caption{Complete Trajectory}
			\label{fig:ch3:simReplan4}
		\end{subfigure}
		\caption{Simulation Case III: A trapping situation scenario}\label{fig:ch3:simReplan}
	\end{adjustbox}
\end{figure}

\section{Conclusion}\label{sec:ch3:conc}

A hybrid navigation strategy for UAVs was presented in this chapter.
The problem formulation considered a general 2D nonholonomic kinematic model for UAVs flying at a fixed altitude.
This model is also applicable to ground and underwater vehicles.
A global planner based on rapidly-exploring random tree algorithm was used which works well in three-dimensional spaces as well.
Simulation results confirm that the proposed hybrid navigation method works well in 2D environments with unknown static and dynamic obstacles.
The overall strategy structure can be modified to consider 3D motion by adopting 3D control laws for both path tracking and reactive modes based on a 3D kinematic model.
The following chapters present 3D control laws fr UAVs that can be adopted within this hybrid navigation scheme.

    \chapter{A 3D Reactive Collision-Free Navigation Strategy for Nonholonomic Mobile Robots\label{cha:methods_reactive3D}}

The previous chapter presented a general hybrid navigation framework which was implemented using some of the existing 2D reactive methods.
This chapter starts to treat the tackled navigation problem in this report for UAVs by considering the more complex 3D problem.
One of the novel 3D reactive navigation strategies for UAVs is presented here to also address the first research question in \cref{ch1:problem}.
This computationally-light strategy couples control inputs directly with interpreted information from sensors to address navigation problems in 3D unknown environments.
The work presented in this chapter was published in \cite{elmokadem20183d}.

\section{Introduction}\label{sec:ch4:Intro}

The last few decades have seen a growing trend towards developing autonomous mobile robots including unmanned aerial vehicles and underwater vehicles.
These vehicles have become essential for a wide range of applications such as surveillance, real-time monitoring, search and rescue, border patrolling, reconnaissance, objects inspection, scientific and military missions, etc.
A real challenge in developing autonomous vehicles is the development of safe and reliable navigation systems that can generate safe paths reaching a goal destination within an environment while avoiding collision with obstacles.

A considerable amount of literature has been published on mobile robot navigation with obstacle avoidance.
Generally, approaches used for mobile robot navigation can be classified into two categories: global path planning techniques and sensor-based techniques \cite{hoy2015algorithms}.
Path planning is usually concerned with problems where a full knowledge about the environment is available.
Classical path planning approaches include cell decomposition \cite{lozano1983spatial}, roadmap \cite{lozano1979algorithm,dunlaing1985retraction}, potential field \cite{khatib1986real,ge2000new}, probabilistic approaches \cite{kavraki1996probabilistic,lavalle1998rapidly}, etc.
Alternatively, sensor-based approaches are more suitable for applications that require operation in unknown environments where only sensors data collected locally is available; reactive strategies are subset of sensor-based navigation techniques.
A range of available reactive methods are designed to handle static obstacles only such as the dynamic window approach \cite{fox1997dynamic} and the curvature velocity approach \cite{simmons1996curvature}.
On contrary, other approaches are developed to be suitable for collision avoidance with moving obstacles (for example, see \cite{matveev2015safe,matveev2012real,chakravarthy1998obstacle,fiorini1998motion}).
A survey of available navigation techniques for mobile robots can be found in \cite{hoy2015algorithms}.

Many studies in the field of mobile robot navigation have focused on the planner case of collision avoidance.
For the case of unmanned aerial vehicles and underwater vehicles, it is more efficient to utilize the vehicle's full capability to perform 3D avoidance maneuvers.
Few research works have dealt with the more general and complex case of three-dimensional avoidance. 
Examples of available 3D strategies include \cite{liu2017robust,choi2017two,wang2018strategy,hrabar2011reactive,liu2016high,roussos20083d,chen2013three,nieuwenhuisen2014obstacle,popp2015novel,subramanian2014real,mcguire2017efficient,yang2017obstacle}.
In \cite{liu2017robust,choi2017two}, strategies based on model predictive control are developed to generate safe trajectories, and obstacle avoidance is achieved by solving a local optimization problem with obstacle constraints.
Although these techniques can provide near optimal paths, they have expensive computational cost making them not suitable for mobile robots with limited capabilities.
A collision-free navigation strategy in dynamic environments is proposed in \cite{wang2018strategy}.
This strategy adopts an enlarged vision cone technique assuming that obstacles are covered by spheres, and the direction of possible 3D avoidance maneuvers can only be in one of the two boundary rays of the vision cone. 
A map-based navigation approach is proposed in \cite{hrabar2011reactive} which relies on performing a localized search on 3D occupancy maps to find an escape point.
Experimental results were provided to show the computation tractability of the approach using a certain flying robot; however, the use of this method could still be limited to vehicles with high processing capabilities.
The authors of \cite{liu2016high} attempted to develop a navigation algorithm for quadrotors with limited capabilities.
However, their approach may fail to find a safe trajectory in some scenarios, and they overcome this problem by using a stopping policy to halt the vehicle which is an undesirable behavior for autonomous vehicles navigation.
Research works \cite{roussos20083d,chen2013three,nieuwenhuisen2014obstacle,popp2015novel,subramanian2014real} present navigation methods based on artificial potential field technique.
Such methods suffer from the local minimum problem. 
A vision-based strategy is developed in \cite{mcguire2017efficient} for indoor obstacles avoidance but the technique is limited to low-speed navigation.
Another vision-based approach is proposed in \cite{yang2017obstacle} where obstacle avoidance is managed by predicting safe trajectories from the images obtained by the vehicle's camera based on convolutional neural networks.

This work aims to address the complex problem of 3D navigation in unknown environments with collision avoidance.  
A reactive strategy for nonholonomic mobile robots is developed to tackle this problem.
One of the main features of this strategy is its low computational cost because it requires only information about the distance to nearest obstacles and the heading to target destination.
Thus, it removes the need for heavy computational processing of sensors data required by many of the available strategies. For example, some approaches need to perform computationally exhaustive search algorithms on local maps to construct safe paths.
An equally important feature of this approach is its capability of generating different 3D paths which results in more efficient obstacle avoidance maneuvers when combined with a global path planner.
Also, the strategy does not add restrictions on obstacles' shapes.
Moreover, it is based on the sliding mode control theory which is known for its robustness against disturbances and uncertainties (for example, see \cite{chen2016formation,yu2017distributed,gao2017integral}).

This chapter is organized as follows.
In section~\ref{sec:ch4:problem}, the problem of 3D navigation under consideration is formulated.
Then, the suggested navigation strategy is proposed in section~\ref{sec:ch4:strategy} with some detailed analysis.
This strategy is then validated in section~\ref{sec:ch4:simulation} using computer simulations along with a discussion of the results.  
Finally, this work is concluded in section~\ref{sec:ch4:conclusion}, and potential future work is suggested.

\section{Problem Formulation}\label{sec:ch4:problem}

A nonholonomic mobile robot moving in a 3D environment is considered in this work.
The mathematical model (kinematics) of such robot can be described using its absolute Cartesian coordinates $\mathbf{s} = [x,\ y,\ z]^T$ and its orientation $\vect{a}\in \R^3$ where $||\vect{a}||=1$.
A description of this model is given as follows \cite{matveev20143d},
\begin{equation} \label{equ:ch4:model3D}
\begin{array}{cccc}
\dot{\mathbf{s}} = V \vect{a}, & \dot{\vect{a}} = \vect{u}, & \vect{a} \cdot \vect{u} = 0,  & ||\vect{u}|| \leq \bar{u}.
\end{array}
\end{equation} 
where $V\in [0,\ V_{max}]$ is the linear speed, $\vect{u}$ is the two-degrees-of-freedom control input, and $\bar{u}$ is a given constant that represents the control limits.
The third equation in the above model indicates that the input must always be perpendicular to the heading vector which results in steering-like behavior while maintaining the unity length of $\vect{a} $ as required.
Moreover, the turning radius of the robot is lower bounded by,
\begin{equation*}
R_{min} = \frac{V}{\bar{u}}
\end{equation*}
The above non-holonomic model is applicable to many rigid-body vehicles including aerial and underwater vehicles \cite{matveev20143d}.

We consider a mobile robot $\A$ moving in a 3D environment with $n$ obstacles $\obs_i,\ i=\{1,\cdots,n\}$ whose boundaries $\partial \obs_i$ are piece-wise smooth.
This work aims to develop a 3D navigation strategy with collision avoidance that allows $\A$ to navigate safely reaching a goal destination $G=(x_G,\ y_G,\ z_G)$ within the environment whose location is static and known (i.e. $\mathbf{s}(t)\in \R^3 \setminus \{\obs_1 \cup \obs_2 \cup \cdots \cup \obs_n\}\ \forall t$ and $\mathbf{s}(t_f) = G$ for some $t_f>0$).

\section{Navigation Strategy in 3D}\label{sec:ch4:strategy}

\subsection{Strategy}

This section presents the proposed navigation algorithm for mobile robots in 3D environments by developing ideas from \cite{matveev2011method,matveev2017method}.

Given any nearby obstacle $\obs_i$, let $\vect{T}_i$ be a vector from the robot's coordinates tangent to $\partial\obs_i$ at a given time $\tau$. 
Also, let $\Pav$ be called \textit{plane of avoidance} in which the avoidance maneuver takes place.
This plane can be constructed  by $\vect{T}_i$ and $\vect{a}(\tau)$ where $\tau$ is the time at which the avoidance maneuver starts, and the plane's normal $\vect{n}_{\Pav} $ is obtained as follows,
\begin{equation}\label{equ:ch4:Pav}
\vect{n}_{\Pav} = \vect{T}_i \times \vect{a}(\tau).
\end{equation}
The choice of $\vect{T}_i$ can be based on different criteria to determine the best avoidance direction; for example, one can select $\vect{T}_i$ to be the tangent that makes minimal angle with $\vect{a}(\tau)$.
An illustration of these definitions is given in Fig.~\ref{fig:ch4:illustration}.

\begin{figure}[!htb]
	\centering
	\includegraphics[scale=0.55]{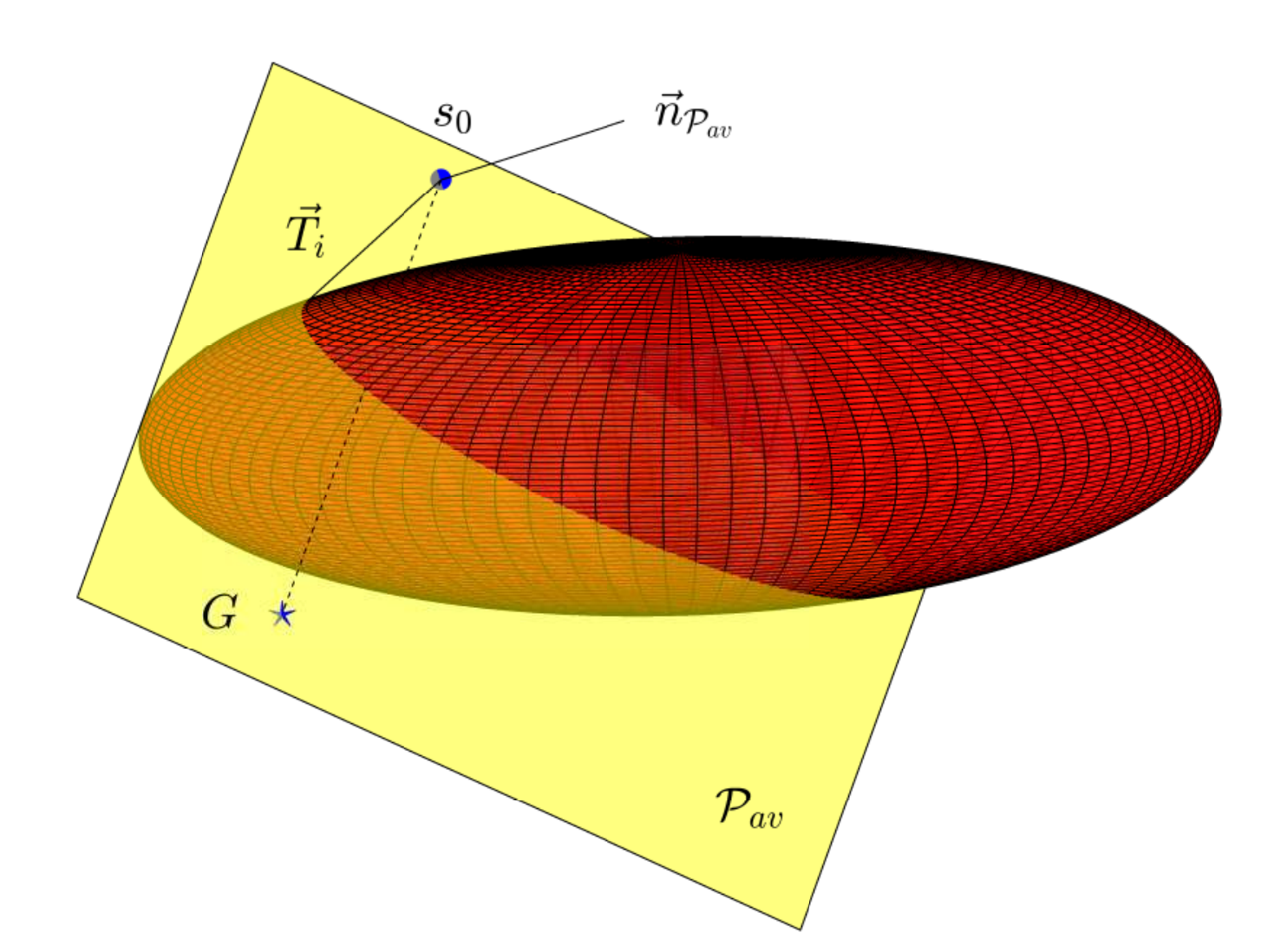} 
	\caption{Plane of avoidance illustration} \label{fig:ch4:illustration}
\end{figure}

\noindent Using these definitions, the following obstacle avoidance law can be introduced:
\begin{equation}\label{equ:ch4:navigationLaw3D}
\begin{array}{l}
V = \bar{V} \\
\vect{u}(t) = \Gamma \ \bar{u} \ \text{sgn}\Big(\dot{d}(t) + \chi (d(t) - d_0)\Big) \vect{i}_{n}(t), \\ 
\vect{i}_{n}(t)  = \vect{a}(t) \times \vect{n}_{\Pav} %
\end{array}
\end{equation}
where $\bar{V}\in (0,V_{max}]$ is a desired speed, $\vect{i}_{n} \in \Pav$ is a unit vector perpendicular to $\vect{a}(t)$ directing away from $\obs_i$, $\text{sgn}(\alpha)$ is the signum function defined as follows
\begin{equation*}
\text{sgn}(\alpha) := \left\{\begin{array}{cc}
1 & if\  \alpha>0 \\
0 & if\  \alpha=0 \\
-1 & if\  \alpha<0 \\
\end{array}\right.,
\end{equation*}
$\chi (\beta)$ is a saturation function given by
\begin{equation*}
\chi (\beta) = \left\{\begin{array}{cc}
\gamma \beta & if\  |\beta|\leq \delta \\
\delta \gamma\ \text{sgn}(\beta) & otherwise
\end{array}\right.,\ \ \gamma,\delta>0,
\end{equation*}
$d_0>d_{safe}>0$ is a desired distance with $d_{safe}$ being a safety margin, $\Gamma=1$ if $\vect{a}(\tau)$ intersects with $\obs_i$, and $\Gamma=-1$ otherwise.
These choices of $\Gamma$ are made to ensure that $\Gamma \vec{i}_n$ is always directing away from $\obs_i$.
It is worth mentioning that this navigation law is based on the sliding mode control theory.
Also, It is clear from \eqref{equ:ch4:navigationLaw3D} that $\vect{u}(t)$ is feasible since it satisfies the last two equations in \eqref{equ:ch4:model3D}.
Moreover, this law requires only access to $d_i$ (can be obtained from sensors), $\dot{d}_i$ (can be obtained numerically) and heading towards $G$.

\begin{remark}\label{rem:Ti}
	There are different possible approaches to determine $\vect{T}_i$ depending on the capabilities of the robot.
	One approach could be the use of a feature detection algorithm along with an onboard camera to detect the obstacle's edge that is closest to the current heading.
	Another possible approach is to determine a bounding shape for the obstacles based on distance measurements constructed by a perception system such as an ellipsoid.
	Then, one can determine the tangent to that ellipsoid that makes the minimal angle with $\vect{a}(\tau)$.
	The latter one is considered in the simulations done in this work.
\end{remark}

\begin{remark}
	In a specific case where the robot detects a very large obstacle for which an edge cannot be detected (a wall for example), the plane of avoidance $\Pav$ can be selected randomly to follow the obstacle's boundary until an edge is detected.
	This can be implemented more efficiently with a global planner.
\end{remark}

\begin{remark}
	The proposed algorithm can be applied to close bounded obstacles with smooth boundary given that a method for finding the best plane of avoidance $\Pav$ is available.
\end{remark}

\noindent The proposed navigation strategy consists of two modes: \\
\textbf{M1:} Obstacle avoidance given by the law \eqref{equ:ch4:navigationLaw3D} in which a plane of avoidance $\Pav$ is determined at the beginning of the maneuver\\
\textbf{M2:} Pure pursuit to destination $G$ with maximum speed

\vspace{0.5cm}
The rules for switching between the two modes are borrowed from \cite{matveev2011method}; these rules are as follows, \\
\textbf{R1:}  Mode \textbf{M1} is activated at a given time $\tau$ when the distance from the vehicle to an obstacle $\obs_i$ is reduced to $C$.
That is, $d_i(\tau) = C$ and $\dot{d}_i(\tau) < 0$.\\
\textbf{R2:} Switching from \textbf{M1} to \textbf{M2} occurs when $d_i(t) \leq d_0 + \epsilon$ ($\epsilon>0$) and the vehicle's heading is directed towards the destination $G$.

\subsection{Determining $\vect{T}_i$ for ellipsoidal obstacles (Special Case)} \label{ssec:ch4:Ti}

As highlighted in Remark~\ref{rem:Ti}, it is possible to determine $\vect{T}_i$ using different methods.
In this subsection, a possible approach is proposed considering that obstacles are enclosed in ellipsoids constructed by the vehicle's perception system.

Let $f(\sigma)$ be a desired objective function where $\sigma = [x,\ y,\ z]$.
The surface points of an ellipsoid $E$ must satisfy the ellipsoid's equation,
\begin{equation}\label{equ:ch4:ellipsoid}
h_1(\sigma) := \dfrac{x^2} {a^2} + \dfrac{y^2} {b^2} + \dfrac{z^2} {c^2} - 1 = 0
\end{equation}
where $[x,\ y,\ z]$ are the coordinates of a point in the ellipsoid's coordinate system, and $2a,\ 2b$ and $2c$ are the lengths of its principal axes.
Furthermore, if a tangent line $\vect{T}_i$ starting at point $P_0=[x_0,\ y_0,\ z_0]$ touches $E$ at a point $P=[x,\ y,\ z]$, it must satisfy the following,
\begin{equation}
h_2(\sigma) := \dfrac{2x}{a^2} (x-x_0)+ \dfrac{2y}{b^2} (y-y_0) + \dfrac{2z}{c^2} (z-z_0) = 0
\end{equation}

The proposed method of finding $\vect{T}_i$ is solving the following optimization problem,
\begin{equation}\label{equ:ch4:optProblem}
\begin{array}{lccc}
\min_{\sigma} \ f(\sigma) & & & \\
s.t. & h_1(\sigma) = 0 &\&	&h_2(\sigma) = 0 \\ 
\end{array}
\end{equation}

The solution of \eqref{equ:ch4:optProblem} yields the point at which $\vect{T}_i$ touches $\obs_i$.
For example, $\vect{T}_i$ can be chosen to be the vector that makes the minimal angle with $\vect{a}(\tau)$. 
In this case, $f(\sigma)$ can be written as follows,
\begin{equation}
f(\sigma) = -\dfrac{\vect{T}_i \cdot \vect{a}(\tau)}{||\vect{T}_i||\ ||\vect{a}(\tau)||}
\end{equation}

\subsection{Analysis}\label{subsec:ch4:Analysis}

The following assumptions are made,
\begin{assumption}\label{assmp:1}
	When mode \textbf{M1} starts at time $\tau$, the vehicle is initially oriented towards the destination, i.e,
	\begin{equation}
	\vect{a}(\tau) = \dfrac{G - s(\tau)}{||G - s(\tau)||}
	\end{equation}
\end{assumption}

As in \cite{matveev2011method}, it is assumed that \textbf{M1} cannot be activated for multiple obstacles simultaneously.
This implies that the obstacles should be disjoint and far enough from each other.
Hence, the following assumption is made.

\begin{definition}
	Consider two points $\mathbf{r},\mathbf{r}_*\in \R^3$ such that $\mathbf{r} \not\in \obs_i$ and $\mathbf{r}_*\in\partial\obs_i$.
	The C-neighborhood of an obstacle $\obs_i \subset \R^3$ is defined as follows: $\mathcal{N}[C,\obs_i]:= \{\mathbf{r}\in \R^3:\ ||\mathbf{r} - \mathbf{r}_*|| \leq C\}$.
	This gives a set of all the points at a distance $\leq\ C$ from $\obs_i$.
\end{definition}

\begin{assumption}\label{assmp:2}
	$\mathcal{N}[C+2R_{min},\obs_i] \cap \mathcal{N}[C,\obs_j] = \emptyset $ for all $i \neq j$.
\end{assumption}

However, if multiple close obstacles are detected, the strategy will still be valid by constructing a bounding object around these obstacles to satisfy Assumption~\ref{assmp:2}.

\begin{theorem}\label{thm:navLaw}
	For a given $\vect{T}_i$, the navigation law given by \eqref{equ:ch4:navigationLaw3D} forces a vehicle governed by \eqref{equ:ch4:model3D} to follow the boundary of an obstacle $\obs_i$ within the plane $\Pav$ while keeping a safe distance from the obstacle.
\end{theorem}

\begin{proof}
	Clearly, the proposed choice of $\vect{i}_n$ results in $\vect{u}(t)$ being perpendicular to both $\vect{a}(t)$ and $\vect{n}_{\Pav}$ which implies that $\vect{u}(t)\in \Pav\ \forall t$, and $\vect{a} \cdot \vect{u} = 0$.
	The first condition ensures that the proposed law restricts the motion within $\Pav$ while the latter one maintains the feasibility of $\vect{u}(t)$.
	Moreover, the law \eqref{equ:ch4:navigationLaw3D} is a class of the sliding mode control where the sliding surface is $S := \dot{d}(t) + \chi (d(t) - d_0) = 0$.
	This choice of sliding surface guarantees that the vehicle can keep a distance of $d(t)\in[d_+,d_-]$ from $\obs_i$, where $d_+>d_->0$ represent limits of a regular interval and $d_->d_{safe}$, until $d(t)$ converges to $d_0$ (see \cite{matveev2011method} for a detailed proof).
	Based on these points, the proposed design for $\vect{u}(t)$ forces the vehicle to follow the boundary of $\obs_i$ at a distance $d_0$ within the plane $\Pav$.
	This completes the proof.
\end{proof}

\begin{definition}\cite{matveev2011method}
	A navigation strategy is said to be \textit{target reaching with obstacle avoidance} if there exists a time $t_f>0$ such that $s(t_f) = G$ and $d_i(t)>d_{safe}>0$ $\forall i$ and $t \leq t_f$. 
\end{definition}

\begin{theorem}
	Consider a certain criteria is chosen for determining $\vect{T}_i$ in \eqref{equ:ch4:Pav}, let $d_0> d_{safe}$ and let Assumptions~\ref{assmp:1}-\ref{assmp:2} hold true.
	Then, the proposed navigation strategy given by the rules \textbf{R1} and \textbf{R2} and associated with the modes \textbf{M1} and \textbf{M2} is a target reaching strategy with obstacle avoidance.
\end{theorem}

\begin{proof}
	Once mode \textbf{M1} is activated when the vehicle reaches a distance $C$ close to an obstacle $\obs_i$ according to rule \textbf{R1}, it starts to follow the boundary of $\obs_i$ based on Theorem~\ref{thm:navLaw} within $\Pav$ while maintaining a distance of $d_0\in[d_+,d_-]$ from $\obs_i$.
	Since $d_->d_{safe}$, it is ensured that $d_i(t)>d_{safe}\ \forall t$.
	
	Furthermore, it follows from \eqref{equ:ch4:Pav} and Assumption~\ref{assmp:1} that $G \in \Pav$ for any choice of $\vect{T}_i$.
	Therefore, there exists an instant $t_\star$ at which the vehicle's heading $\vect{a}(t_\star)$ is oriented at $G$ during the avoidance maneuver \textbf{M1}.
	This implies that whenever mode \textbf{M1} is activated, it is eventually terminated by rule \textbf{R2}, and mode \textbf{M2} continues to navigate the vehicle towards $G$.
	Hence, the strategy  is a target reaching strategy with obstacle avoidance.
\end{proof}

For a rigorous mathematical analysis of the proposed method, the reader is referred to \cite{matveev2011method} where a planar analysis is made.
The applicability of that analysis to the proposed strategy holds within the plane of avoidance $\Pav$.

\section{Simulation Results}\label{sec:ch4:simulation}

Simulations are performed to validate the performance of the proposed navigation strategy using the MATLAB software; the obtained results are given in this section.
The developed strategy is applied considering two scenarios.
The first one considers a single static obstacle to show the possibility of executing different safe paths depending on the choice of $\Pav$.
The second scenario deals with a more complex situation considering multiple static obstacles.

In all simulations, the design parameters are chosen as follows: $\bar{V} = V_{max} = 1\ m/s$, $\bar{u}=1.5\ rad/s$, $d_0 = 1m$, $\epsilon=0.5$, $C = 2.5$, $\delta = 0.5$ and $\gamma = 1$.

In Fig.~\ref{fig:ch4:simSingle}, the navigation strategy is applied for a single obstacle case.
The initial position is set to be $s_0 = [5,\ 12,\ 2]^T$ (marked as o), and the goal location is $G = [5,\ -6,\ 1]^T$ (marked as x).
The obstacle is represented by an ellipsoid whose center is at $[5,\ 3,\ -3]^T$, and its principal axes has lengths of $a=5,\ b=7$, and $c=2$.
Simulations are performed using random choices of $\vect{T}_i$ in each one.
The results show that the vehicle can reach its destination safely using the proposed strategy.
Furthermore, it is clear from the results that the strategy can produce different safe paths in 3D which makes it capable of performing more efficient avoiding maneuvers.
Depending on the application, the choice of $\Pav$ could be simply the one that produces the shortest path, or a better direction based on a certain criteria if incorporated with a global planner.

\begin{figure}[h]
	\centering
	\begin{adjustbox}{minipage=\linewidth,scale=0.5}
		\begin{subfigure}[t]{0.5\textwidth}
			\centering
			\includegraphics[width=\linewidth,trim={5cm 0 4cm 0},clip]{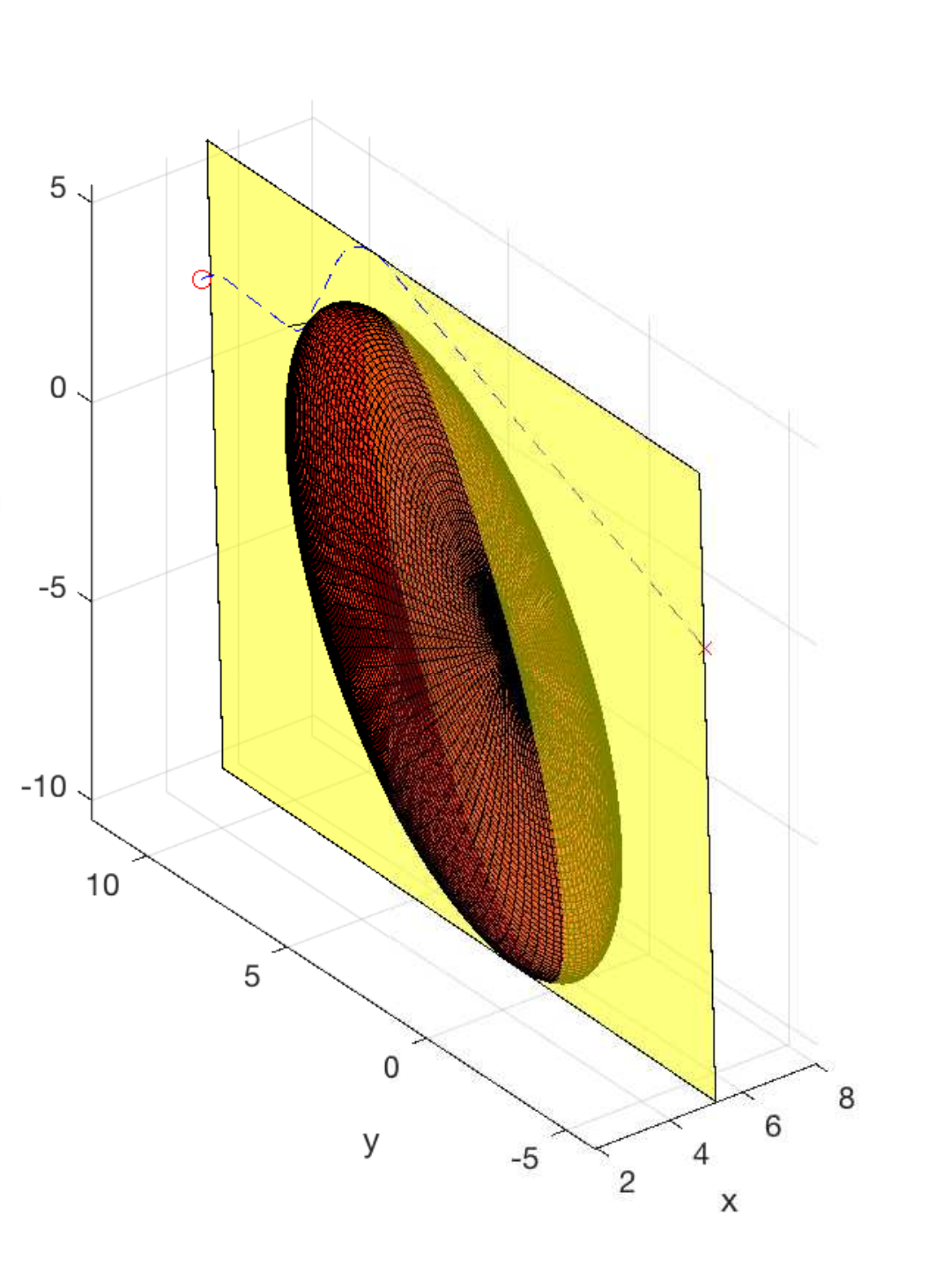} 
			\caption{}
		\end{subfigure}
		\hfill
		\begin{subfigure}[t]{0.5\textwidth}
			\centering
			\includegraphics[width=\linewidth,trim={5cm 0 4cm 0},clip]{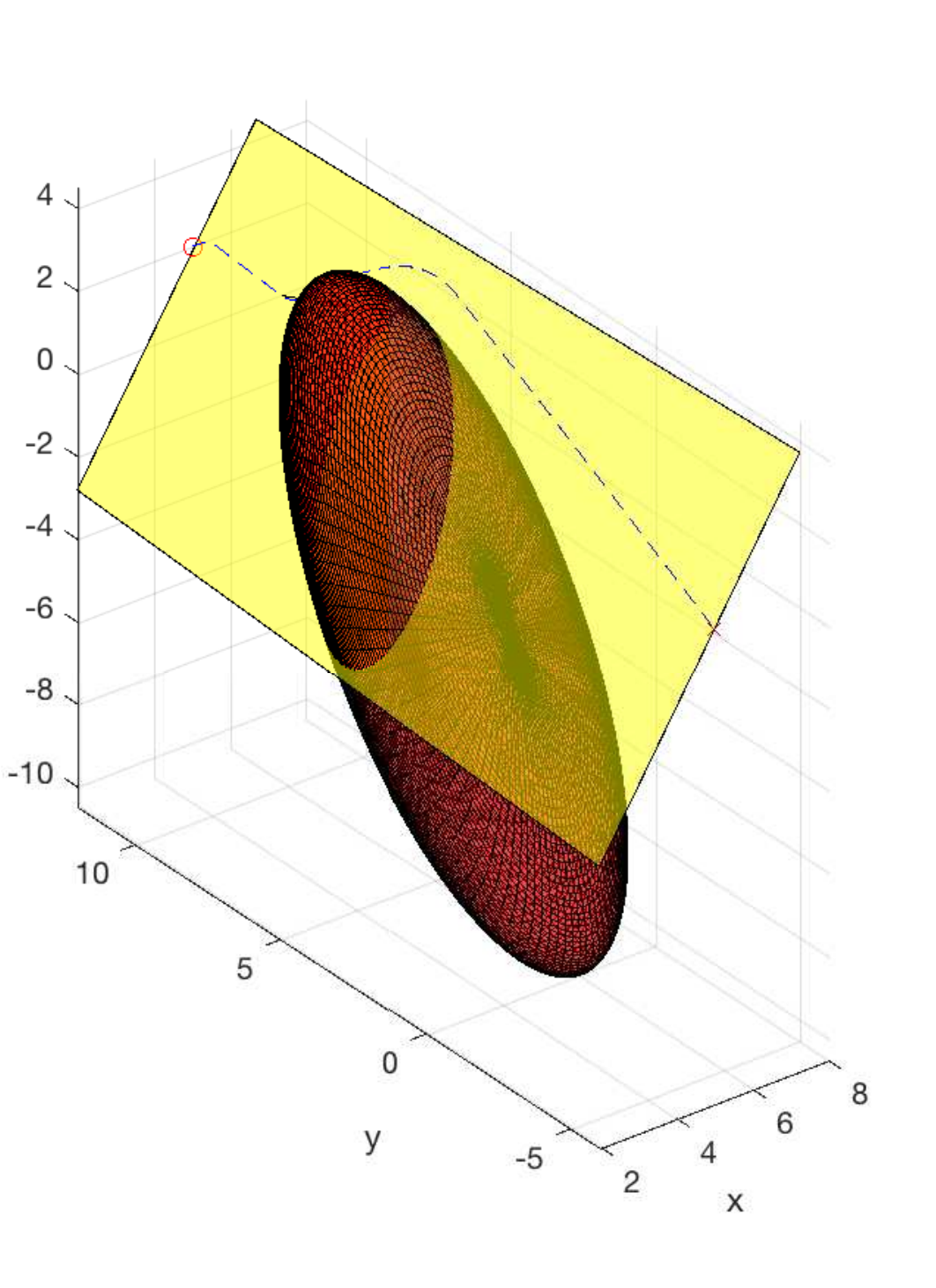} 
			\caption{}
		\end{subfigure}
		
		\begin{subfigure}[t]{0.5\textwidth}
			\centering
			\includegraphics[width=\linewidth,trim={5cm 0 4cm 0},clip]{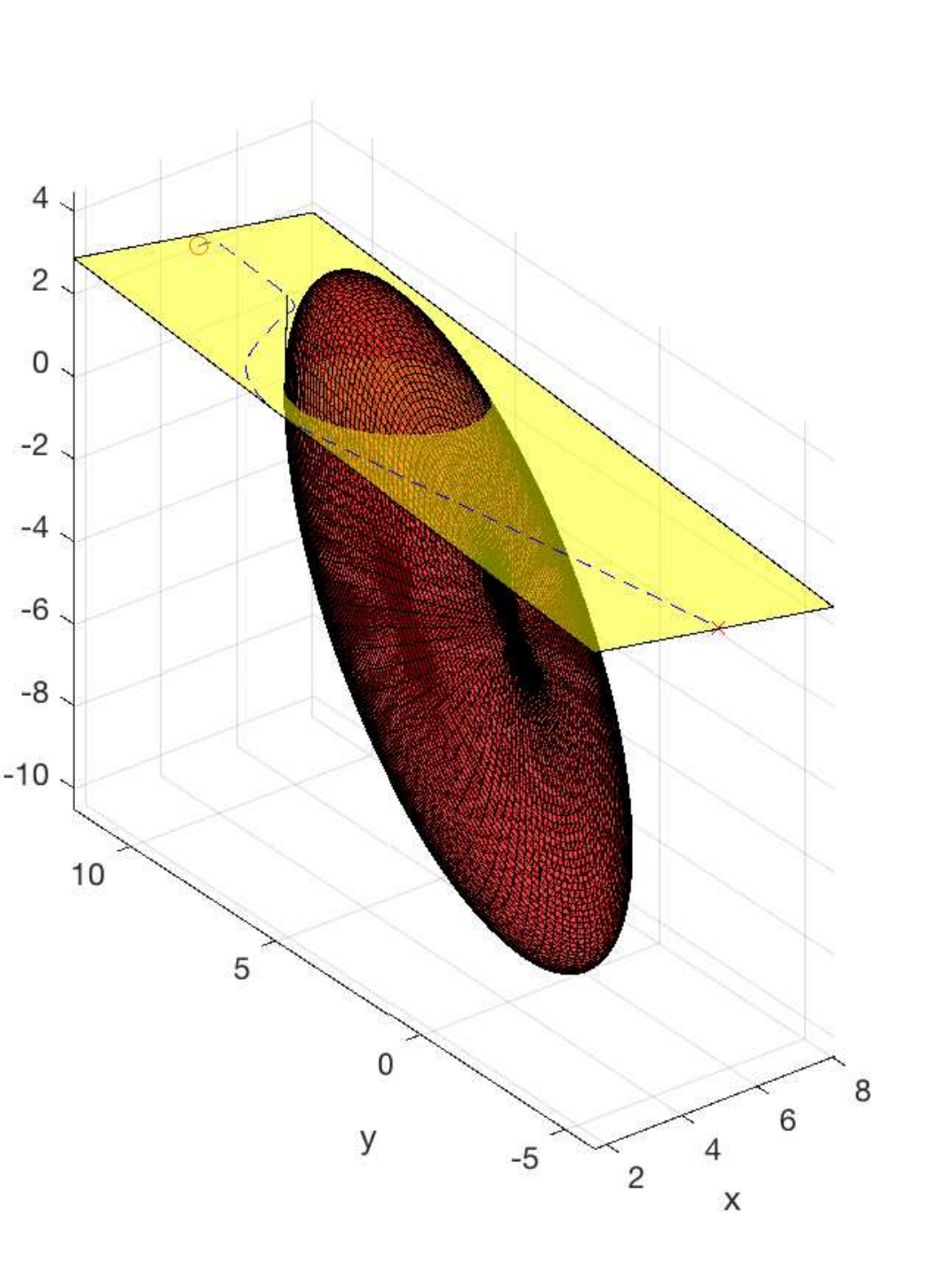} 
			\caption{}
		\end{subfigure}
		\hfill
		\begin{subfigure}[t]{0.5\textwidth}
			\centering
			\includegraphics[width=\linewidth,trim={5cm 0 4cm 0},clip]{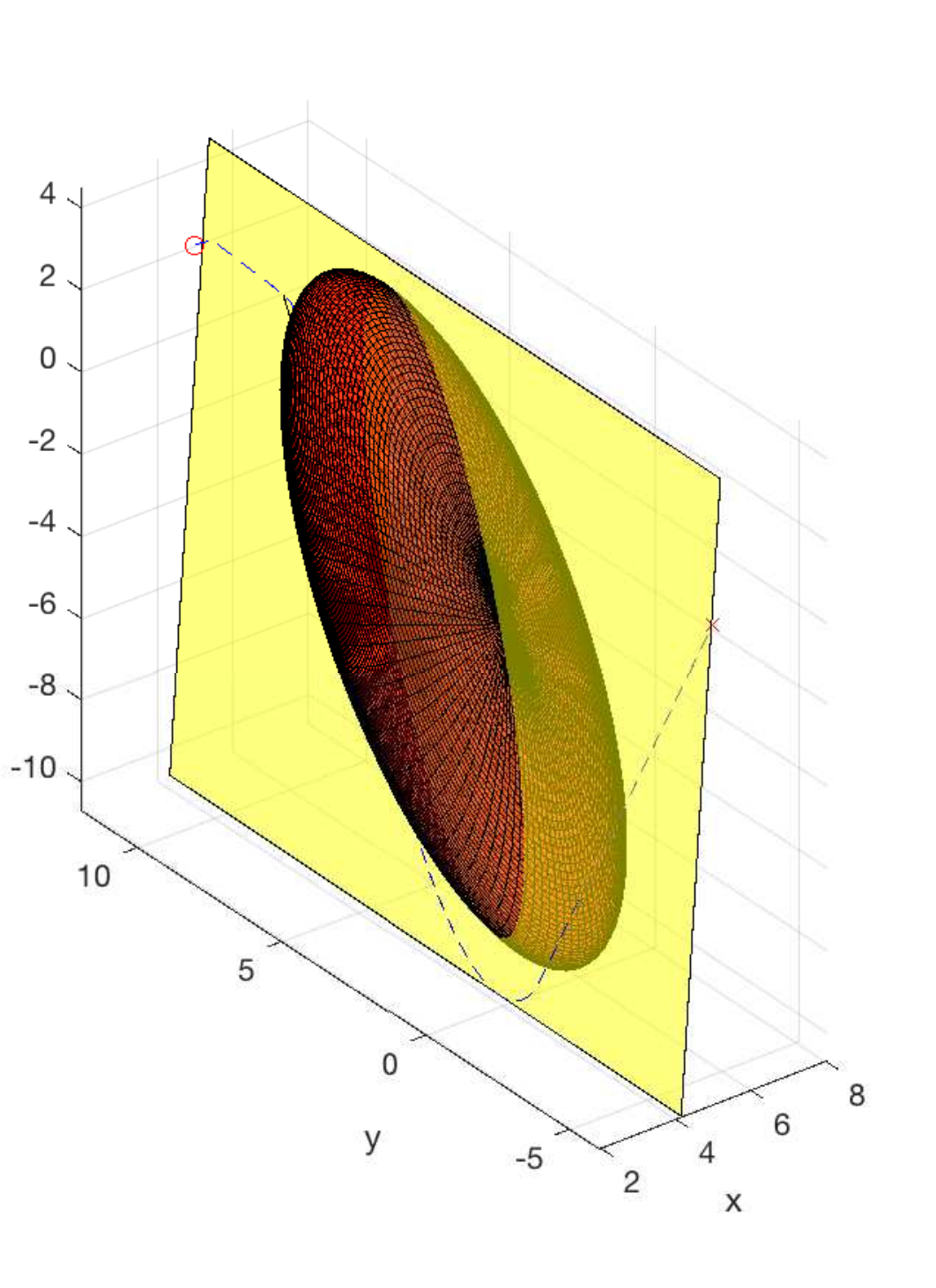} 
			\caption{}
		\end{subfigure}
		
		\begin{subfigure}[t]{0.5\textwidth}
			\centering
			\includegraphics[width=\linewidth,trim={5cm 0 4cm 0},clip]{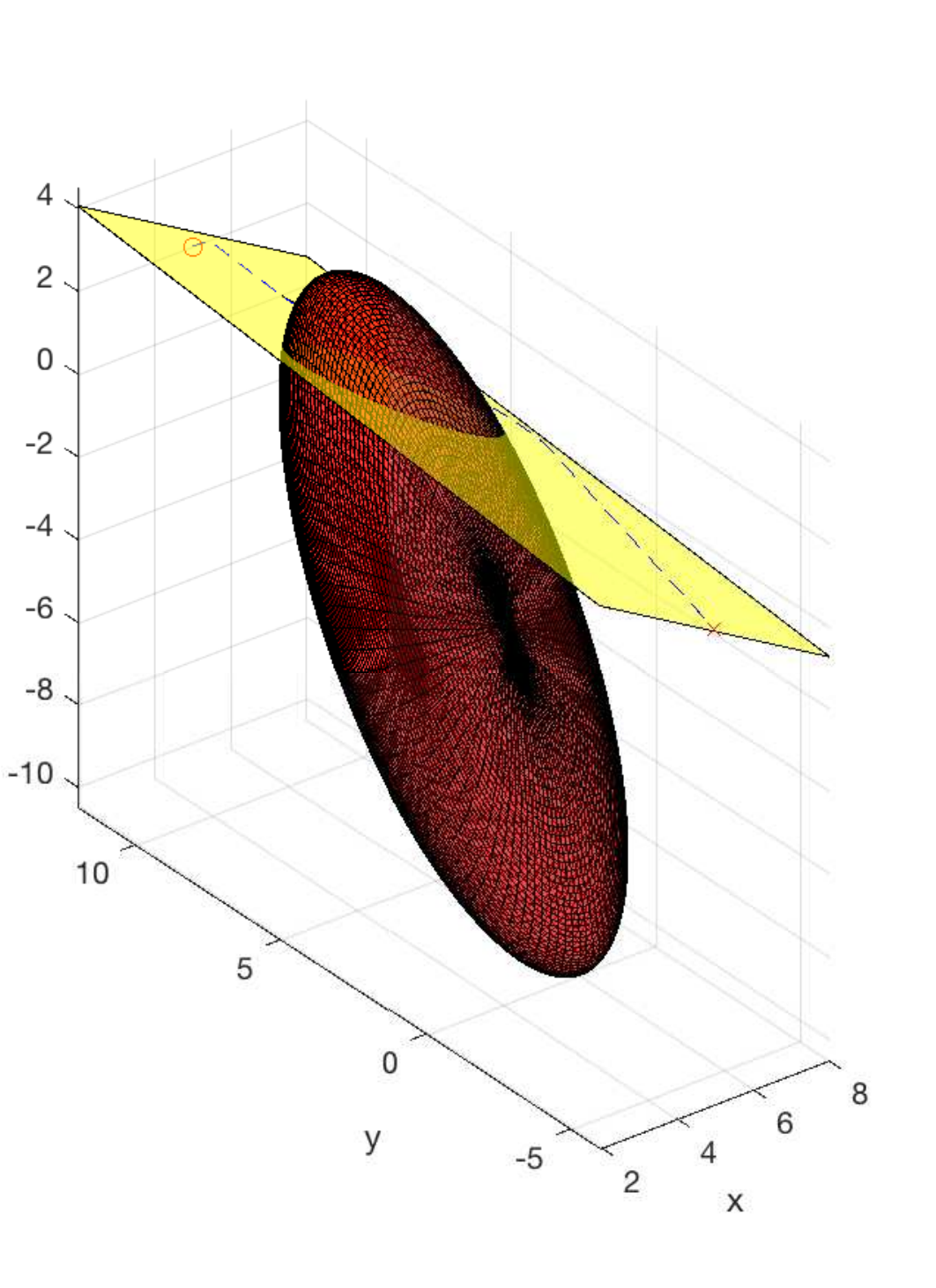} 
			\caption{}
		\end{subfigure}
		\hfill
		\begin{subfigure}[t]{0.5\textwidth}
			\centering
			\includegraphics[width=\linewidth,trim={5cm 0 4cm 0},clip]{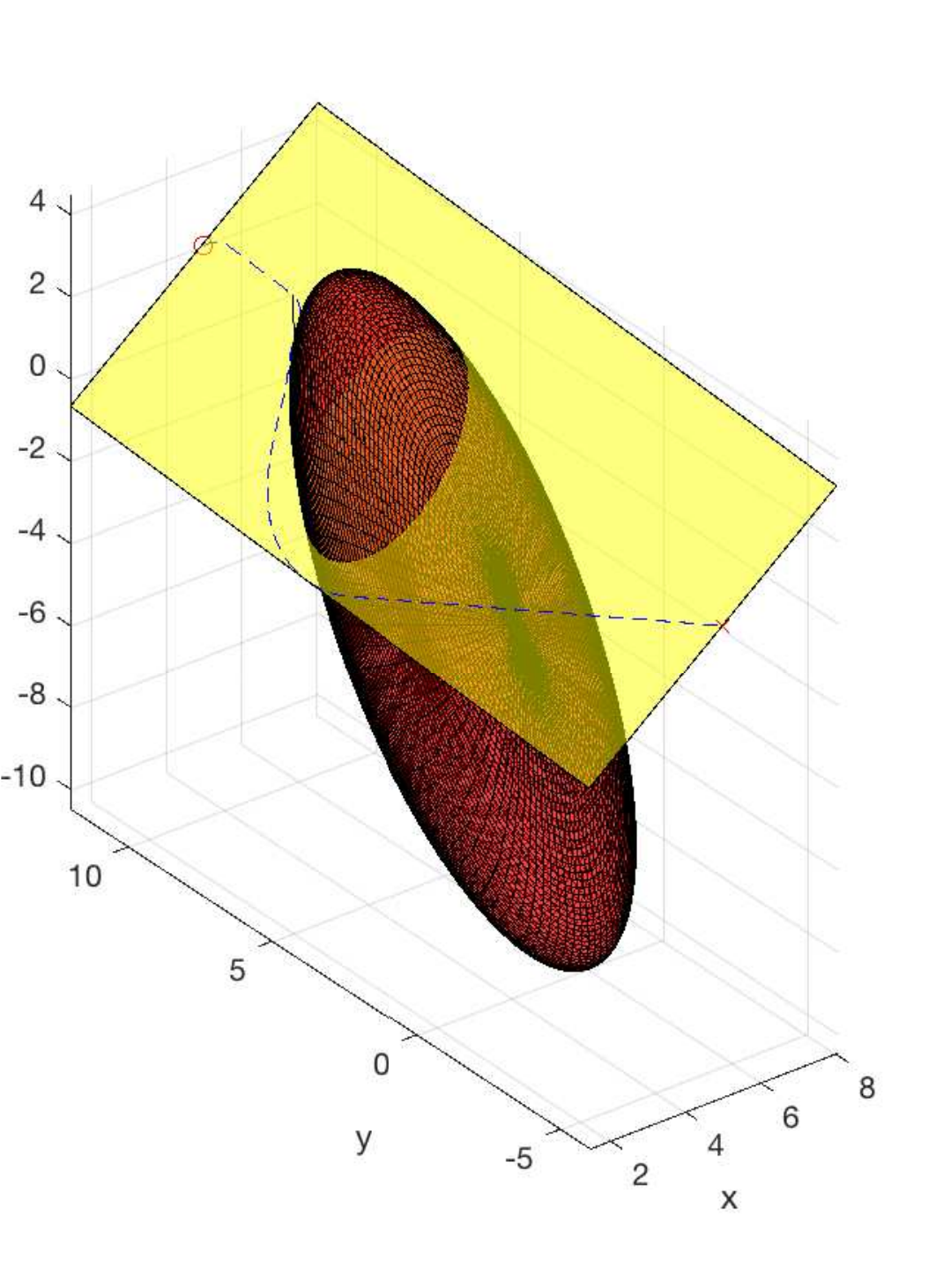} 
			\caption{}
		\end{subfigure}
	\end{adjustbox}
		\caption{Simulation results with single obstacle showing random choices of $\vect{T}_i$ with their corresponding $\Pav$ and executed paths}\label{fig:ch4:simSingle}
	
\end{figure}

The proposed strategy is then evaluated considering a 3D environment with multiple obstacles.
The environment is filled with five ellipsoidal obstacles of different sizes.
The robot's initial coordinates are taken to be $s_0 = [15,\ 5,\ 7]^T$ (marked as a triangle), and the goal location is $G = [-10,\ 35,\ 14]^T$ (marked as a star).
Furthermore, $\vect{T}_i$ is determined using the proposed method in the subsection~\ref{ssec:ch4:Ti}.
Figures~\ref{fig:ch4:simMultiple1}-\ref{fig:ch4:simMultiple3} show the simulations results of this case.
The executed path is given in Fig.~\ref{fig:ch4:simMultiple1} at different time instances when the vehicle (marked as o) reaches a distance $C$ from an obstacle (at which the switching from \textbf{M2} to \textbf{M1} occurs).
The complete path taken by the vehicle is presented in Fig.~\ref{fig:ch4:simMultiple2}.
It should be noted that the figures are shown from different viewing angles for illustration purposes.
These results clearly verify that the proposed strategy can safely navigate the vehicle in 3D environments among multiple obstacles.
Moreover, Fig.~\ref{fig:ch4:simMultiple3} depicts the distance between the vehicle and the encountered obstacles versus time $d_i(t)$.
It is obvious from this figure that the distances between the vehicle and the obstacles are lower bounded by $d_0$.
This proves that the strategy can successfully maintain a safety margin between the vehicle and the obstacles during its motion.

\begin{figure}[!htb]
	\centering
	\begin{adjustbox}{minipage=\linewidth,scale=1.0}
		\begin{subfigure}[t]{0.48\textwidth}
			\centering
			\includegraphics[width=\linewidth]{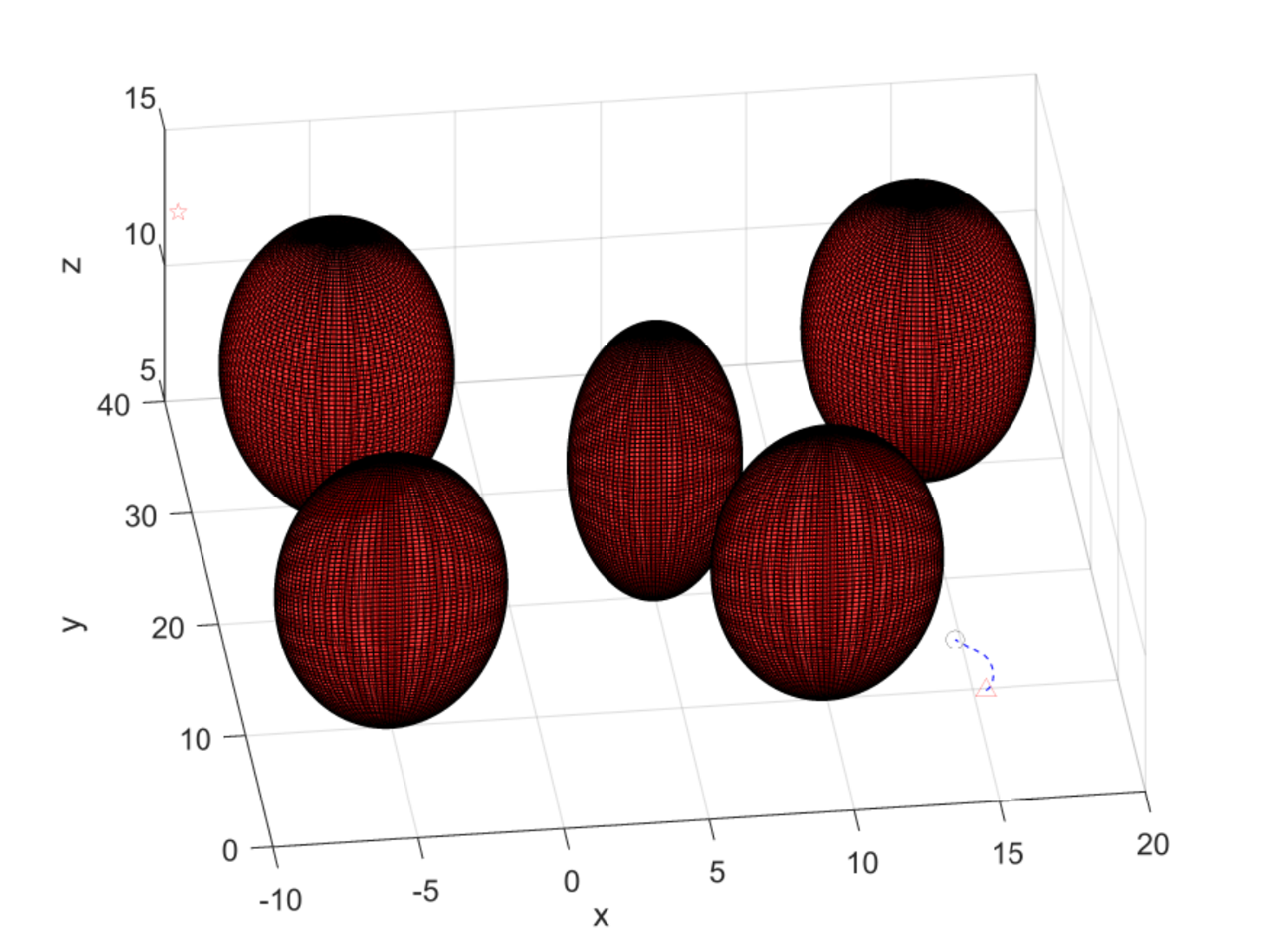} 
			\caption{at $t=2.88s$}
		\end{subfigure}
		\hfil
		\begin{subfigure}[t]{0.48\textwidth}
			\centering
			\includegraphics[width=\linewidth]{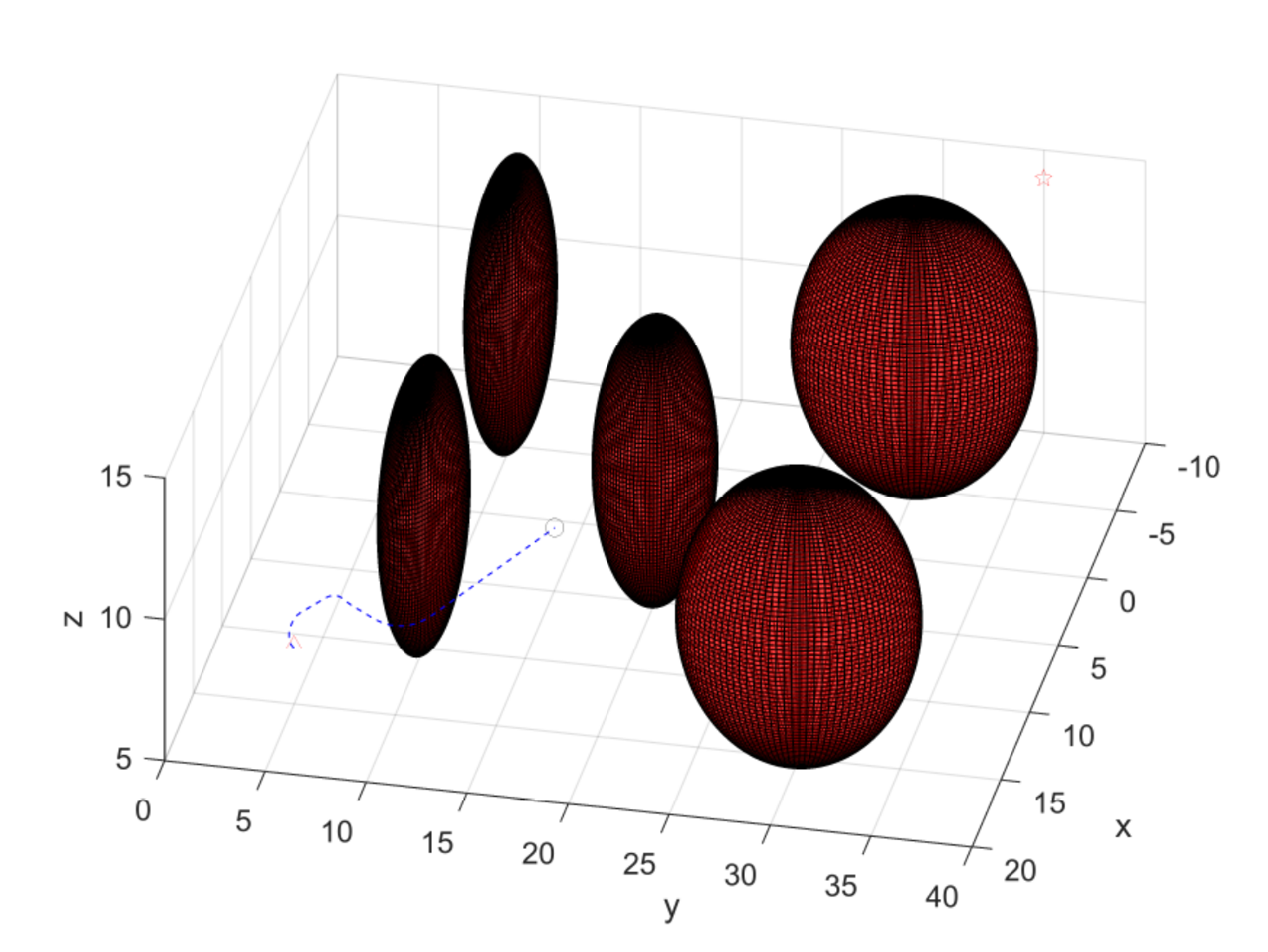} 
			\caption{at $t=14.34s$}
		\end{subfigure}
		
		\begin{subfigure}[t]{0.48\textwidth}
			\centering
			\includegraphics[width=\linewidth]{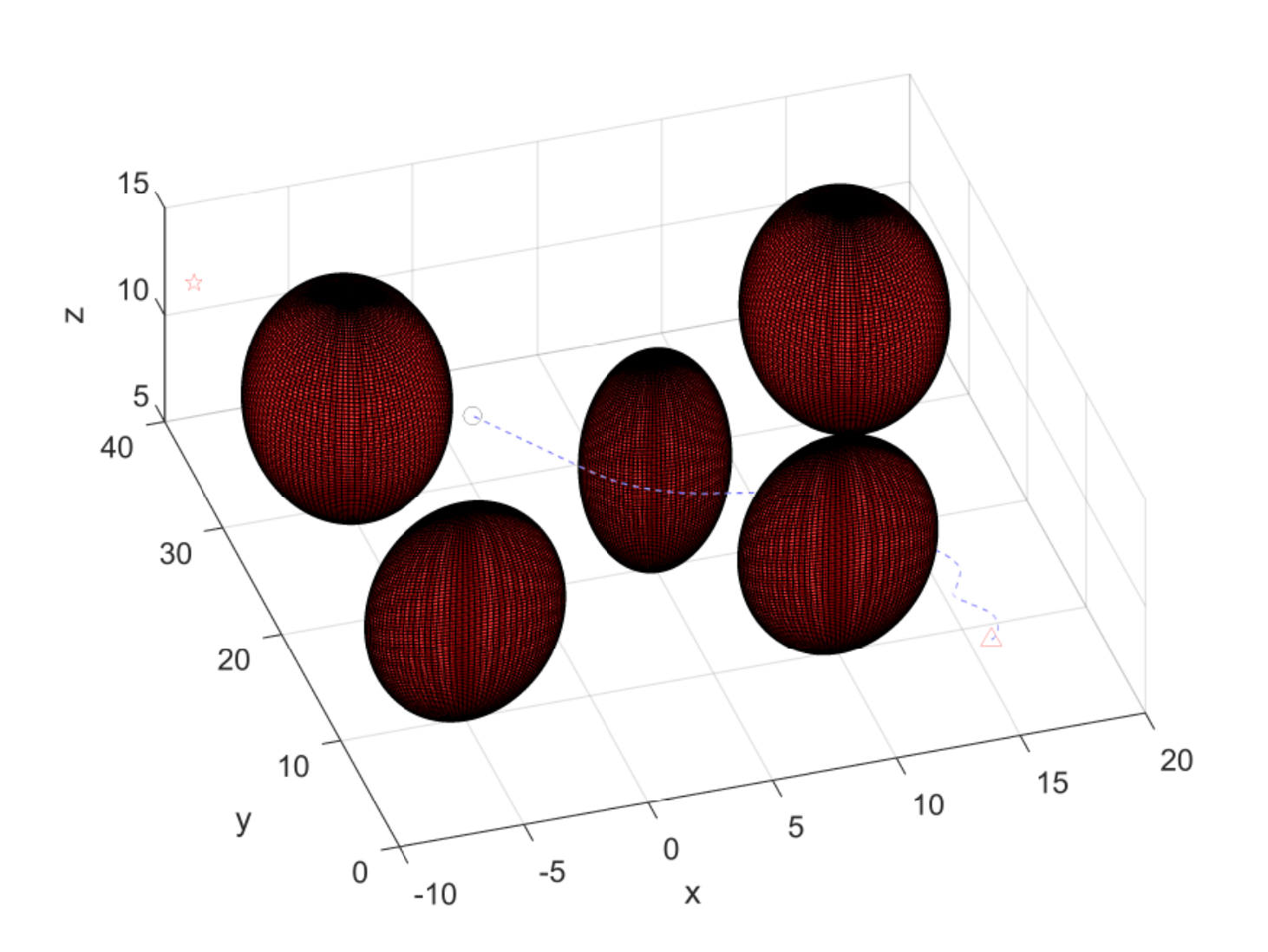} 
			\caption{at $t=26.43s$}
		\end{subfigure}
		\end{adjustbox}
		\caption{Simulation results with multiple obstacles} \label{fig:ch4:simMultiple1}
\end{figure}

\begin{figure}[!htb]
	\centering
	\begin{adjustbox}{minipage=\linewidth,scale=1.0}
		\begin{subfigure}[t]{0.45\textwidth}
			\centering
			\includegraphics[width=\linewidth]{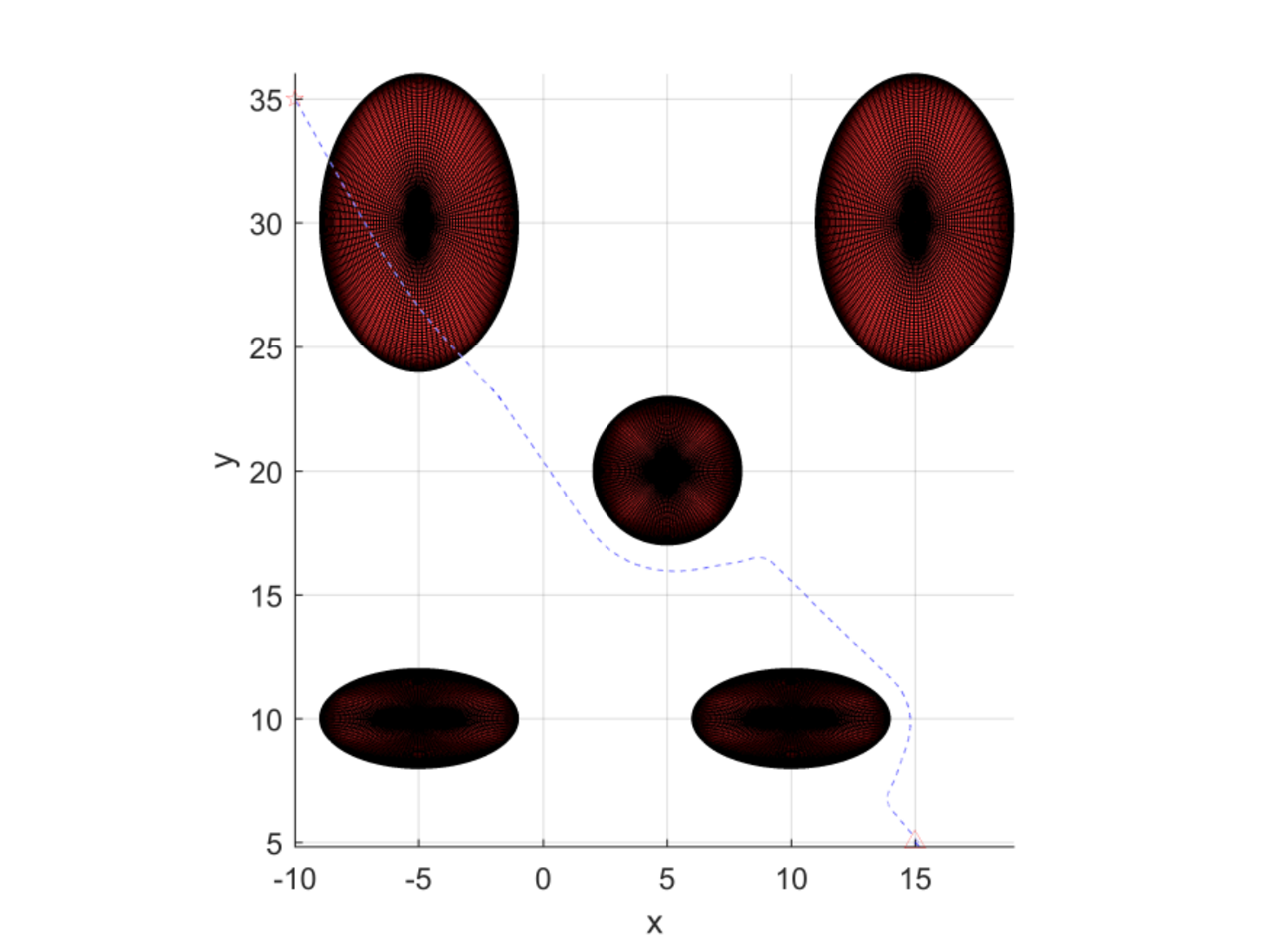} 
			\caption{XY view}
		\end{subfigure}
		\hfil
		\begin{subfigure}[t]{0.45\textwidth}
			\centering
			\includegraphics[width=\linewidth,trim={0 2.5cm 0 2.5cm},clip]{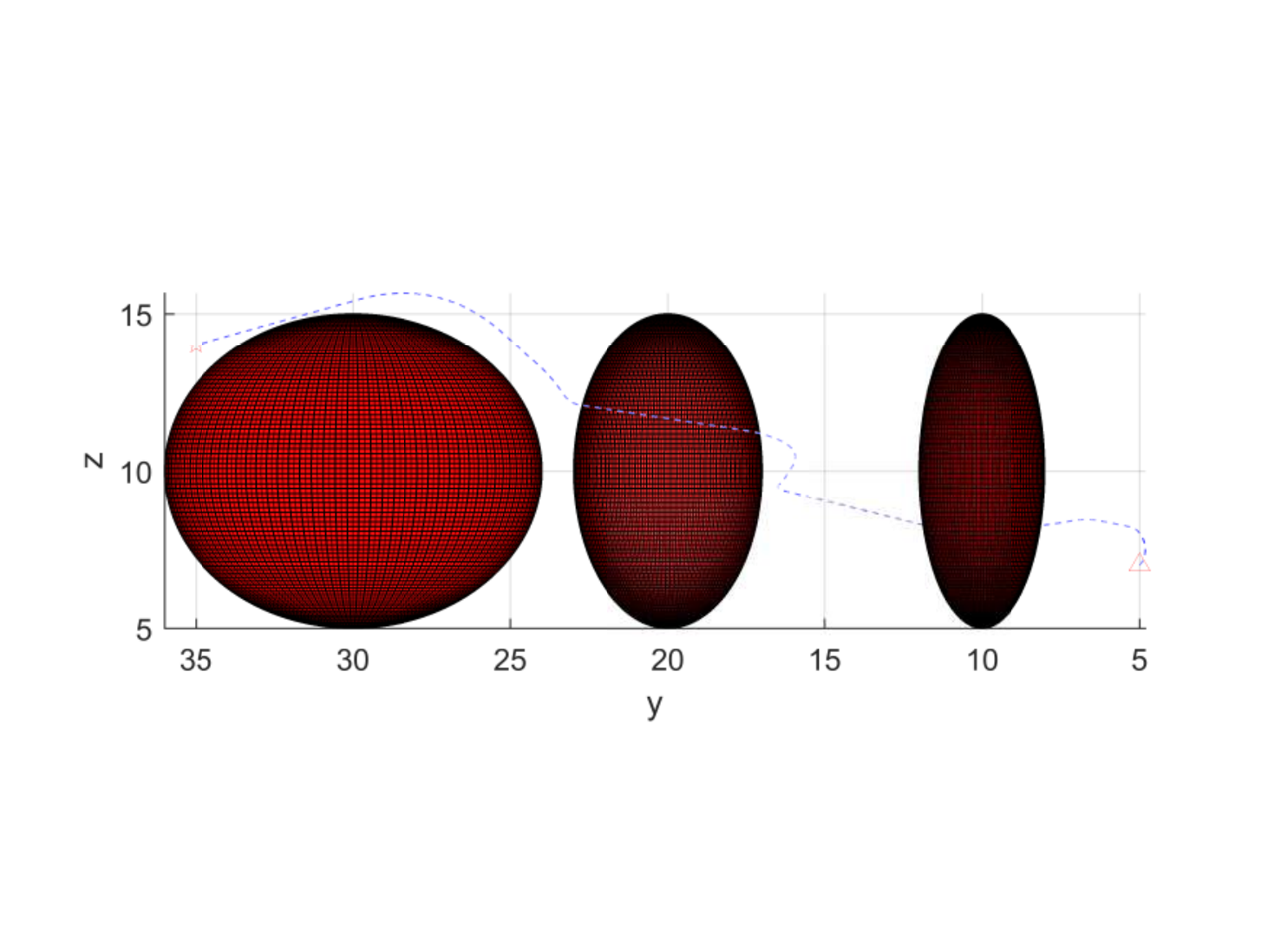} 
			\caption{YZ view}
		\end{subfigure}
		
		\begin{subfigure}[t]{0.45\textwidth}
			\centering
			\includegraphics[width=\linewidth]{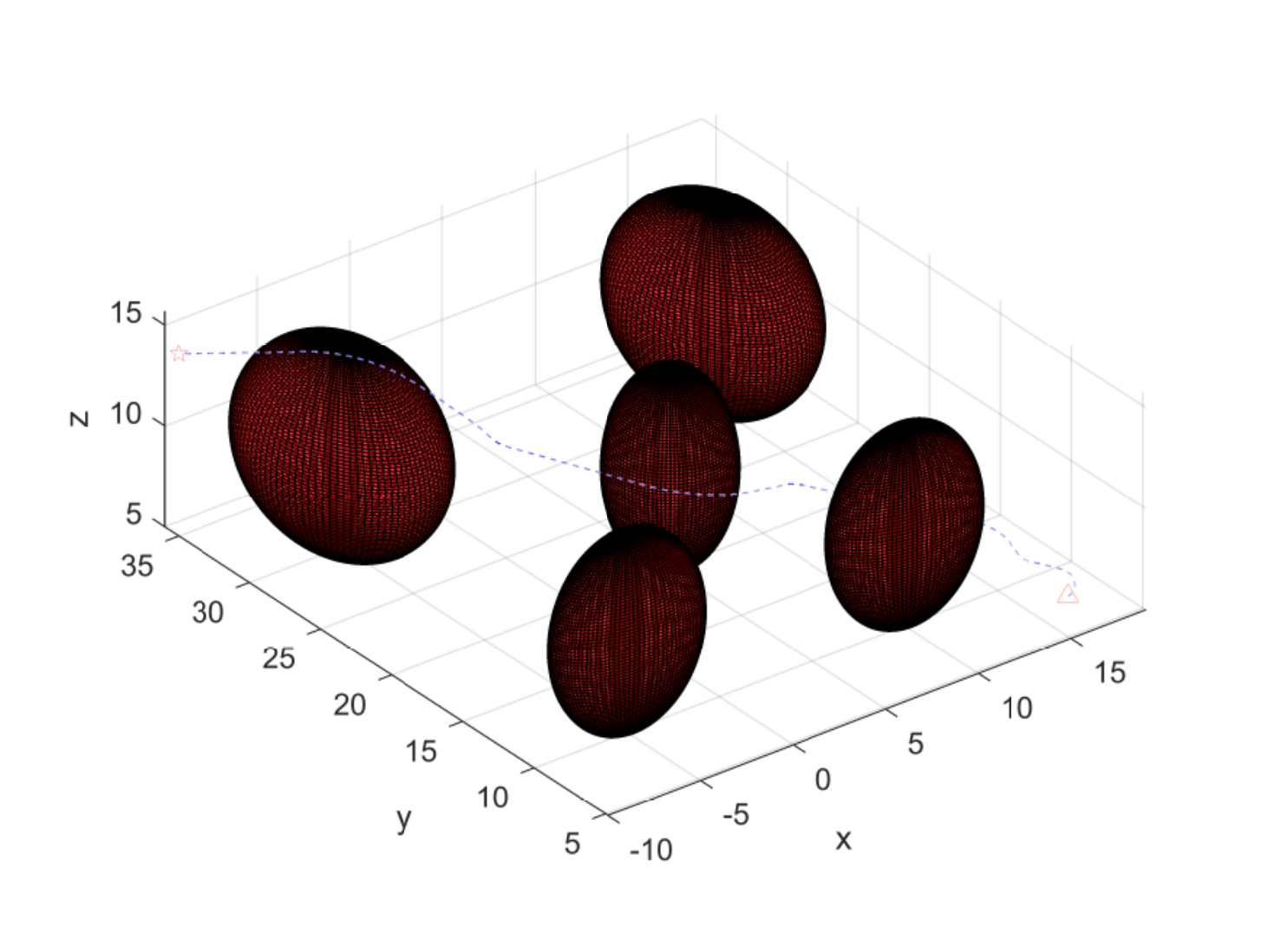} 
			\caption{3D view}
		\end{subfigure}
	\end{adjustbox}
	\caption{Simulation results with multiple obstacles (complete executed Path with different views)}\label{fig:ch4:simMultiple2}
\end{figure}

\begin{figure}[!htb]
	\centering
	\includegraphics[scale=0.55]{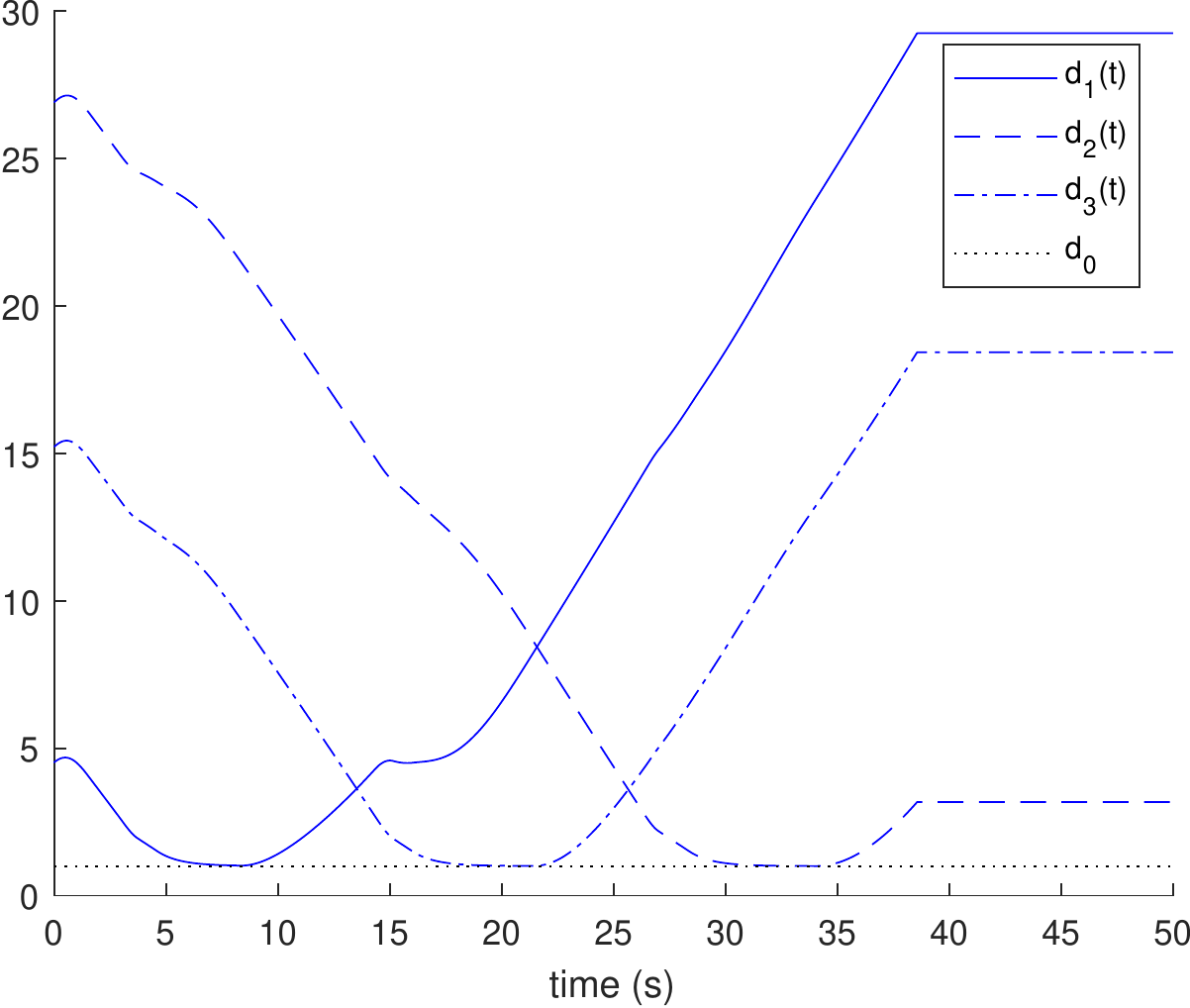} 
	\caption{Distances to obstacles $d_i(t)$ during the simulation} \label{fig:ch4:simMultiple3}
\end{figure}

\section{Conclusion \& Future Work}\label{sec:ch4:conclusion}

This work proposed a 3D navigation strategy for nonholonomic mobile robots.
Computer simulations were performed, and the results verified that the proposed strategy can successfully navigate a vehicle safely in 3D environments while keeping a safe distance from obstacles. 
An extension of this technique to handle moving obstacles is currently under investigation.
Furthermore, possible techniques of determining the plane of avoidance $\Pav$ and a practical implementation on flying robots are considered for future work.

    \renewcommand{\vect}[1]{\bm{#1}}

\chapter{A Reactive Navigation Method of Quadrotor UAVs in Unknown Environments with Obstacles based on Differential-Flatness\label{cha:reactive_impl}}

The 3D reactive navigation method suggested in \cref{cha:methods_reactive3D} was developed at a high-level considering a general 3D nonholonomic kinematic model.
This chapter further extends that approach showing a possible implementation with quadrotor UAVs with experimental validation using a real quadrotor.
Control laws are developed in this chapter based on the sliding mode control technique and the differential-flatness property of quadrotor dynamics.
This work was presented in \cite{elmokadem2019reactive}.

\section{Introduction}

Unmanned aerial vehicles (UAVs) have become very important for many applications.
Several research works have been conducted contributing towards the development of fully autonomous aerial vehicles.
Some of the major challenges for UAVs are the safe operation in unknown environments and autonomously detecting and reacting to obstacles (i.e. sense and avoid).
Hence, planning of collision-free trajectories for quadrotors is an active field of research.

Generally, navigation methods can be classified as deliberative (global planning) or sensor-based (local planning) according to \cite{hoy2015algorithms}.
Global planning methods can find optimal solutions to reach goal positions.
However, these methods require a prior knowledge about the environment (i.e. a map).
On contrary, local planners can generate motions utilizing limited local knowledge by observing a fraction of the environment with onboard sensors.

Computation performance of navigation approaches is one of the important factors especially when dealing with small fast UAVs \cite{hoy2015algorithms}.
Generally, approaches based on local planning are faster than those relying on global planning.
When the planning horizon becomes infinitesimally small, the approach acts as a reactive feedback controller \cite{hoy2015algorithms}.
One of the drawbacks of reactive methods is that they do not generate optimal trajectories.
However, they have lower computational complexity compared with optimization-based and search-based techniques, and they do not depend on convergence of optimization algorithms to find feasible solutions \cite{cole2018reactive}.

Many of the available reactive approaches consider only planar motions, and they are mostly implemented on ground vehicles or UAVs with a fixed flying altitude.
Recently, three-dimensional (3D) reactive approaches have attracted much interest.
However, many of them were only tested in simulations, and very few were implemented on UAVs.
Some of the available reactive methods considering planar motion of mobile robots can be found in \cite{matveev2011method,matveev2012real,savkin2014seeking,matveev2015safe,matveev2015globally} while examples of 3D reactive strategies include \cite{yang20133d,hebecker2015model,thanh2018simple,elmokadem20183d,wang2018strategy,cole2018reactive}.

A common practice when planning trajectories for quadrotors with enough computation power is to use a global planner to generate goal waypoints within the environment based on the current knowledge (map) and use a local planner running at a higher rate to generate collision-free trajectories between the waypoints.
In case of small quadrotors with limited computation capabilities, relying only on a local planner would be better for safe motion.
Several methods for generation of collision-free trajectories for quadrotors have been proposed such as \cite{mellinger2011minimum,mueller2015computationally,bry2015aggressive,loianno2016estimation,richter2016polynomial,chen2016online,allen2016real,liu2017planning,spedicato2017minimum,tordesillas2019fastrap}.

Most approaches use optimization-based techniques to generate optimal smooth trajectories with respect to higher order derivatives such as jerk or snap.
For example, in \cite{mellinger2011minimum}, a generation method was proposed that formulates the problem as a constrained quadratic program (QP) to find minimum snap trajectories confined in a decomposed convex region of free space expressed as corridor-like constraints.
This method was further extended in \cite{richter2016polynomial} where it was shown that it can be solved for long-range trajectories using unconstrained QP.
Another optimization-based approach using a less conservative Mixed Integer QP formulation as local planner was proposed in \cite{tordesillas2019fastrap} where it solves for two trajectories at every planning step to ensure safety.
A trajectory generation method was also proposed in \cite{mueller2015computationally} which is based on rapidly generating motion primitives with minimum input aggressiveness using closed-form solutions in addition to a recursive feasibility verification step.

This paper presents a reactive local planner which can generate collision-free trajectories so that quadrotors can safely reach a goal position.
The proposed strategy is based on concepts from guidance laws and equiangular navigation to ensure safe motion.
Our method can be implemented either using a closed-form expression where sensors observations can be mapped directly to control actions or by generating a trajectory for a short horizon.
Nevertheless, the current implementation adopts the latter approach as using the closed form expression of the current approach might be sensitive to initial conditions according to \cite{manjunath2016application} which requires further investigation.
A development of the suggested reactive planner was initially proposed in the previous chapter considering only the kinematic model.
Extending the method for quadrotor UAVs and experimental evaluation of the reactive navigation pipeline are considered the main focus of this chapter. 
Furthermore, a trajectory tracking control design for quadrotor UAVs based on differential-flatness and sliding mode control technique is also developed.
The control design utilizes the differential-flatness property of quadrotor dynamics which guarantees that it can follow any smooth trajectory in the space of flat outputs that is feasible with bounded derivatives \cite{mellinger2011minimum}.
The proposed method is implemented and tested on a quadrotor, and the results of the experiments are used to evaluate the performance of the overall control architecture.

The organization of this chapter is as follows.
In section~\ref{sec:ch5:model}, a brief description of the used quadrotor's model is given.
Then, a trajectory tracking control design is given in section~\ref{sec:ch5:control} followed by the proposal of a reactive trajectory generation method in section~\ref{sec:ch5:trajectoryGeneration}.
Experimental setup and results are then given in section~\ref{sec:ch5:results} to evaluate the performance of the proposed method\footnote{Video: https://youtu.be/ByZklzqjMW0}.
Finally, this work is concluded in section~\ref{sec:ch5:conclusion} with suggestions of possible future work.

\section{Model}\label{sec:ch5:model} %

The dynamical model of a quadrotor is considered here based on \cite{hamel2002dynamic,faessler2017differential} by neglecting wind and rotor drag effects.
Two coordinate frames are used including a world frame $W=\{\vect{x}_{W},\ \vect{y}_{W},\ \vect{z}_{W}\}$ where $\vect{z}_{W}$ is pointing upward and a body-fixed frame $\mathcal{B}=\{\vect{x}_{\mathcal{B}},\ \vect{y}_{\mathcal{B}},\ \vect{z}_{\mathcal{B}}\}$ with an origin that coincides with the quadrotor's center of mass.
Note that the vectors spanning the orthonormal basis for both frames are of unit length. 
An illustration of these coordinate frames is shown in \cref{fig:ch5:quadFrames}.
The position, velocity and acceleration of the quadrotor's center of mass are expressed in the frame $W$ as $\vect{p}=[x,\ y,\ z]^T$ relative to an arbitrary fixed origin $\bm{0} \in W$, $\vect{v}\in \R^3$ and $\vect{a}\in \R^3$ respectively.
The quadrotor's orientation is represented as a rotation matrix $\vect{R}=[\vect{x}_{\mathcal{B}},\ \vect{y}_{\mathcal{B}},\ \vect{z}_{\mathcal{B}}]\in SO(3):\mathcal{B}\to W$, and its angular velocity vector is denoted as $\omega=[p,\ q,\ r]^T$.
Both orientation and angular velocities are expressed in the body-fixed frame. 
Additionally, the orientation can also be parametrized locally by Euler angles (i.e. roll $\phi$, pitch $\theta$ and yaw $\psi$).
Therefore, the rotation matrix $\vect{R}$ can be alternatively written as:
\begin{equation}\label{equ:ch5:RotR}
\vect{R} = \left[
\begin{array}{ccc}
c_{\theta} c_{\psi} & 
s_{\phi} s_{\theta} c_{\psi} - c_{\phi} s_{\psi} & 
c_{\phi} s_{\theta} c_{\psi} + s_{\phi} s_{\psi} \\
c_{\theta} s_{\psi} & 
s_{\phi} s_{\theta} s_{\psi} + c_{\phi} c_{\psi} & 
c_{\phi} s_{\theta} s_{\psi} - s_{\phi} c_{\psi} \\
-s_{\theta} & 
s_{\phi} c_{\theta} & 
c_{\phi} c_{\theta}
\end{array}
\right]
\end{equation}
where the notation $c_{\alpha}:=\cos\alpha$ and $s_{\alpha}:=\sin\alpha$ is used.

Using the above definitions, the quadrotor's dynamical model can be written as follows:
\begin{equation}\label{equ:ch5:quadrotorModel}
\begin{aligned}
\vect{\dot{p}} &= \vect{v} \\
\vect{\dot{v}} &= -g \bm{e}_3 + T \bm{R} \bm{e}_3 \\
\vect{\dot{R}} &= \vect{R} \vect{\hat{\omega}} \\
\vect{\dot{\omega}} &= \vect{I}^{-1} \Big(-\vect{\omega} \times \vect{I} \vect{\omega} + \vect{\tau}\Big)
\end{aligned}
\end{equation}
where $g$ is the gravitational acceleration, $\bm{e}_3 = [0,0,1]^T$, $T \in \R$ is the mass-normalized collective thrust, $\vect{\hat{\omega}}$ is a skew-symmetric matrix defined according to $\vect{\hat{\omega}} \vect{r} = \vect{\omega} \times \vect{r}$ for any vector $\vect{r}\in\R^3$, $\vect{I}$ is the inertia matrix corresponding to the quadrotor's center of mass along the axes $\vect{x}_\mathcal{B},\ \vect{y}_\mathcal{B}$ and $\vect{z}_\mathcal{B}$, and $\vect{\tau} \in \R^3$ is the torques input vector.
Model \eqref{equ:ch5:quadrotorModel} is clearly an underactuated system where its states are given by $\vect{x} = [x,y,z,\dot{x},\dot{y},\dot{z},\phi,\theta,\psi,p,q,r]^T$, and its four inputs are given by $\vect{u} = \left[\begin{array}{c}
T \\ \vect{\tau}
\end{array}\right]$.

\begin{figure}[!th]
	\centering
	\includegraphics[width=0.3\linewidth]{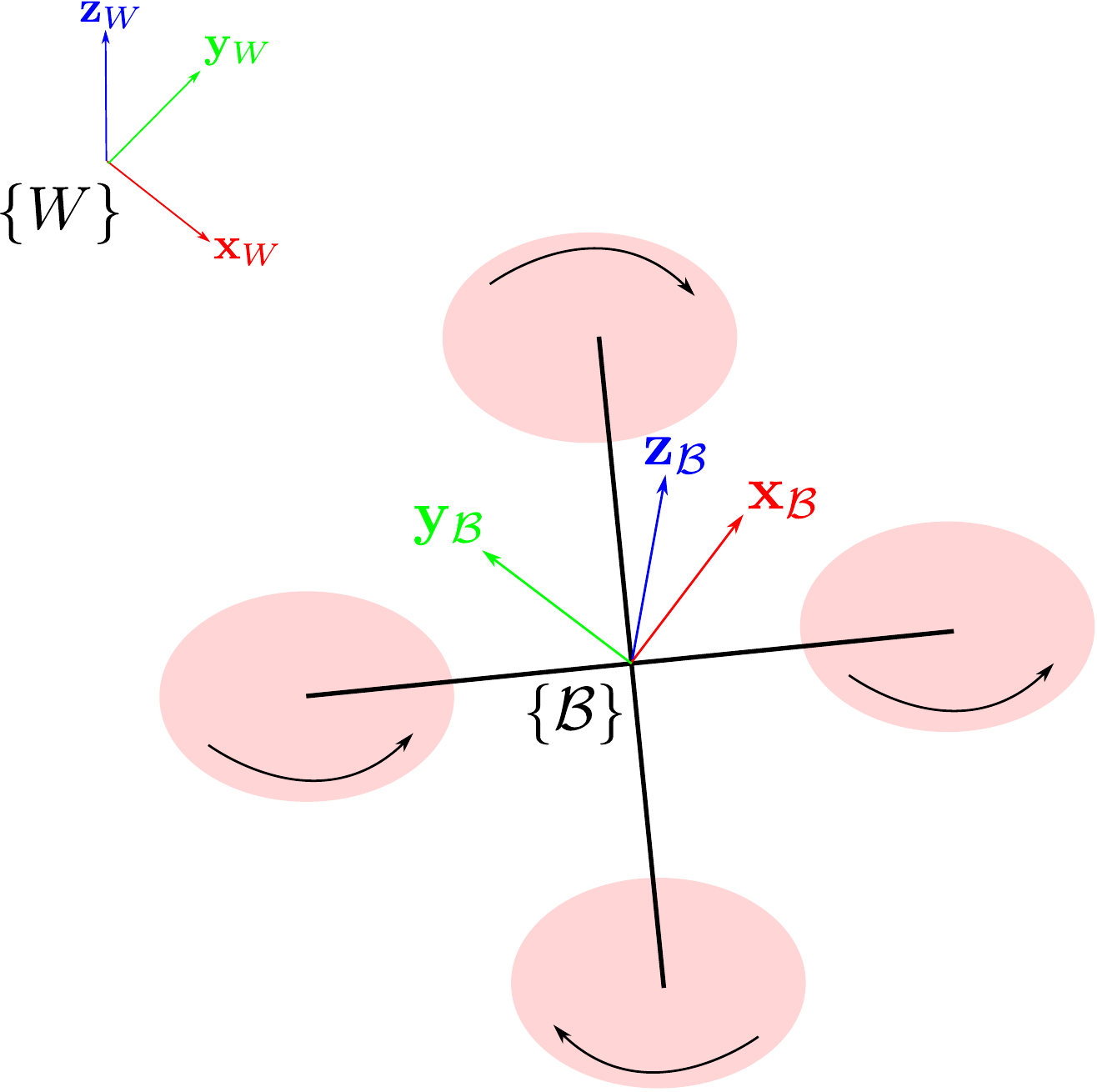}
	\caption{Schematic of the quadrotor coordinate frames} \label{fig:ch5:quadFrames}
\end{figure}

\section{Control}\label{sec:ch5:control} %

In this section, we propose a trajectory tracking control design utilizing the differential-flatness property as was done in \cite{mellinger2011minimum,faessler2017differential}; however, the adopted control technique is based on the sliding mode control theory.
It has been shown in \cite{mellinger2011minimum} that the quadrotor dynamics is differentially flat.
That is, we can express the system states $\vect{x}$ and inputs $\vect{u}$ using algebraic functions of four flat outputs and their derivatives.
We consider the following choice of flat outputs $(x,y,z,\psi)$ which is very common.

To track a reference trajectory, we define the following position and velocity tracking errors:
\begin{equation}\label{equ:ch5:errors}
\vect{e}_{\vect{p}} = \vect{p}_r - \vect{p},\ \vect{e}_{\vect{v}} = \vect{v}_r - \vect{v}
\end{equation}
where $\vect{p}_r$ and $\vect{v}_r$ are reference position and velocity respectively.
A sliding surface $\bm{\sigma}$ is then chosen to be:
\begin{equation}\label{equ:ch5:slidingsurface}
\bm{\sigma} = \vect{e}_{\vect{v}} + \vect{C}_1 \vect{e}_{\vect{p}}
\end{equation}
where $\vect{C}_1 \in \R^{3\times 3}$ is a positive-definite diagonal matrix.
It is obvious that the choice \eqref{equ:ch5:slidingsurface} guarantees that both $\vect{e}_{\vect{p}}$ and $\vect{e}_{\vect{v}}$ converge to zero when the system states reach $\bm{\sigma}=0$.
Hence, a feedback-control term for the desired acceleration of the quadrotor is designed to force the system states to reach $\vect{\sigma}=0$ as:
\begin{equation}
\vect{a}_{fb} = \vect{C}_2 \tanh\left(\mu \bm{\sigma}\right)
\end{equation} 
where $\vect{C}_2 \in \R^{3\times 3}$ is a positive-definite diagonal matrix, $\tanh(\cdot)$ is the element-wise hyperbolic tangent function, and $\mu>0$ is a parameter that controls how steep the $\tanh$ function is around 0. 
The overall desired acceleration can then be written as:
\begin{equation}
\vect{a}_{des} = \vect{a}_{fb} + \vect{a}_{r} + g \bm{e}_3
\end{equation}
where $\vect{a}_r$ is a feedforward term representing the reference acceleration.

The next step is to determine the input thrust $T$ and the desired orientation $\vect{R}_{des}=[\vect{x}_{\mathcal{B},des},\ \vect{y}_{\mathcal{B},des},\ \vect{z}_{\mathcal{B},des}]$ that satisfy the constraints imposed by the desired acceleration $\vect{a}_{des}$ and a reference yaw angle $\psi_r$.
This is done similar to \cite{faessler2017differential} using the following:
\begin{eqnarray}
\vect{z}_{\mathcal{B},des} &=& \frac{\vect{a}_{des}}{||\vect{a}_{des}||} \\
\vect{y}_{\mathcal{B},des} &=& \frac{ \vect{z}_{\mathcal{B},des} \times \vect{x}_{C} }{||\vect{z}_{\mathcal{B},des} \times \vect{x}_{C}||} \\
\vect{x}_{\mathcal{B},des} &=& \vect{y}_{\mathcal{B},des} \times \vect{z}_{\mathcal{B},des}
\end{eqnarray}
where $||\cdot||$ is the Euclidean norm in $\R^3$, and $\vect{x}_{C}$ is given by:
\begin{equation}\label{equ:ch5:yCref}
\vect{x}_{C} = [\cos(\psi_{r}),\ \sin(\psi_{r}),\ 0]^T
\end{equation}
The input mass-normalized collective thrust can then be computed by projecting the desired acceleration $\vect{a}_{des}$ onto the body-fixed frame $z-$axis (i.e. $z_{\mathcal{B}}$) as follows:
\begin{equation}\label{equ:ch5:thrustInput}
T = \vect{a}_{des} \cdot \vect{z}_{\mathcal{B}} = \vect{a}_{des}^T \vect{R} \vect{e}_3
\end{equation}

Finally, a low-level attitude controller with high bandwidth is typically used to provide the required body moments to achieve the desired orientation $\vect{R}_{des}$ and angular velocity $\vect{\omega}_{des}$.
One possible approach is as done in \cite{mellinger2011minimum}.
Let an orientation and angular velocity error vectors be defined as:
\begin{eqnarray}
\vect{e}_{\vect{R}} &=& \frac{1}{2}(\vect{R}_{des}^T \vect{R} - \vect{R}^T \vect{R}_{des})^{\vee} \\
\vect{e}_{\vect{\omega}} &=& \vect{\omega} - \vect{\omega}_{des}
\end{eqnarray}
where $\vee$ is the vee map from a skew-symmetric matrix in $SO(3)$ to a vector in $\R^3$.
The input torques can then be computed using:
\begin{equation}
\vect{\tau} = - \bm{K}_{\vect{R}} \vect{e}_{\vect{R}} - \vect{K}_{\vect{\omega}} \vect{e}_{\vect{\omega}}
\end{equation}
where $\vect{K}_{\vect{R}}$ and $\vect{K}_{\vect{\omega}}$ are positive definite gain matrices.

\section{Reactive Trajectory Generation} \label{sec:ch5:trajectoryGeneration}

Consider the problem of navigating to a goal position $\vect{p}_{goal} \in W$ in space where unknown obstacles may be detected by onboard sensors as the quadrotor moves through the environment.
A safe trajectory can be generated online depending on light processing of sensors information providing quick responses (i.e. a reactive approach).
The considered local trajectory planning here adopts the strategy described in the previous chapter which was based on concepts from guidance laws and equiangular navigation as was developed in \cite{matveev2011method} for planar motions.

A description of the overall trajectory generation is as follows.
The following kinematic model is considered to generate the reference trajectory:
\begin{equation} \label{equ:ch5:model3D}
\begin{aligned}
&\vect{\dot{p}}_r = V \vect{s}_r \\
&\vect{\dot{s}}_r = \vect{\Omega} \\
&\dot{V} = c_{v} \tanh\Big(\mu (V_0 - V)\Big) \\
\end{aligned}
\end{equation}
where $V\in \R$ is the linear speed, $\vect{s}_r \in \R^3$ is a heading vector of unit length (i.e. direction of the velocity vector), $\vect{\Omega} \in \R^3$ is an angular velocity vector, and $V_0 \in \R$ is a desired constant linear speed.
Note that a different notation is used here than the one used in the previous chapter to avoid confusions with the defined quantities from the dynamical model.

Moreover, the following conditions must hold:
\begin{eqnarray}
\vect{s}_r \cdot \vect{\Omega} &=& 0 \label{equ:ch5:constraints1}\\
||\vect{\Omega}|| &\leq& \Omega_{max} \label{equ:ch5:constraints2}
\end{eqnarray}
The reference position $\vect{p}_r$ is obtained from the solution of \eqref{equ:ch5:model3D} over a small period of time $[0,\Delta t_f]$.
Furthermore, by comparing \eqref{equ:ch5:quadrotorModel} and \eqref{equ:ch5:model3D}, the reference velocity and acceleration can be computed according to:
\begin{eqnarray}
\vect{v}_r &=& V \vect{s}_r \\
\vect{a}_r &=& c_{v} \tanh\Big(\mu (V_0 - V)\Big) \vect{s}_r + V \vect{\Omega}
\end{eqnarray}
where $c_v,\mu>0$.
The model parameters $c_v$ and $\Omega_{max}$ can be selected properly based on the choice of $V_0$ to respect physical maximum limits on velocity $v_{max}$ and acceleration $a_{max}$ of the quadrotor such that $||\vect{v}||\leq v_{max}$ and $||\vect{a}||\leq a_{max}$.

The proposed navigation law to compute $\vect{\Omega}$ needed in \eqref{equ:ch5:model3D} is based on two modes:
\begin{itemize}
	\item[] \textbf{M1:} pure pursuit mode corresponding to $\vect{\Omega} = \vect{\Omega}_{PP}$
	\item[] \textbf{M2:} obstacle avoidance mode corresponding to $\vect{\Omega} = \vect{\Omega}_{OA}$
\end{itemize}
Note that $V_0$ will be the same for both modes as will be shown later.

Let $d(t)$ be the distance to nearest obstacle from a quadrotor's position $\vect{p}_r$ at time instant $t \in [0,\ \Delta t_f]$ whose time derivative $\dot{d}(t)$ can be obtained numerically.
Also, let $d_0$ be a desired distance to be respected when moving around obstacles, and let $d_{safe}$ be a safety margin such that $d_0 > d_{safe}$.
Now, the strategy can be described as follows.
Initially, the pure pursuit mode is activated where the quadrotor moves with a velocity $\bar{V} \leq v_{max}$ towards $\vect{p}_{goal}$.
The following rules are then used to switch between the two navigation modes $\textbf{M1}$ and $\textbf{M2}$:
\begin{itemize}
	\item[] \textbf{R1:} The activation of mode \textbf{M2} occurs at a time $t_{s_1}$ when the distance $d(t_{s_1})$ reduces to a threshold value $C$ (i.e. $d(t_{s_1})=C$ and $\dot{d}(t_{s_1})<0$).
	\item[] \textbf{R2:} The switching from $\textbf{M2}$ to $\textbf{M1}$ occurs at a time $t_{s_2} > t_{s_1}$ when $d(t_{s_2}) \leq d_0 + \epsilon$ ($\epsilon>0$) and the vehicle's heading vector $\vect{s}_r(t_{s_2})$ is directed towards $\vect{p}_{goal}$.
\end{itemize}
Note that the condition $d(t_{s_1})=C$ in \textbf{R1} can be checked with some tolerance around $C$ for robustness against numerical errors.

The overall architecture including the proposed trajectory generation and control is shown in \cref{fig:ch5:blockDiag}, and the trajectory generation strategy is illustrated in \cref{fig:ch5:blockDiagTraj}. 
The following subsections present the navigation laws for both modes (i.e. $\vect{\Omega}_{PP}$ and $\vect{\Omega}_{OA}$).

\begin{remark}
	The heading vector $\vect{s}_r$ in \eqref{equ:ch5:model3D} can be represented as:
	\begin{equation}
	\vect{s}_r = \left[\begin{array}{c}
	\cos\alpha\cos\beta \\
	\sin\alpha\cos\beta \\
	\sin\beta
	\end{array}\right]
	\end{equation}
	where $\alpha$ and $\beta$ are the heading and flight path angles.
	Even though quadrotors are holonomic, the used nonholonomic model is still applicable when considering moving wit constant speed.
	Generally, $\alpha$ can differ from the quadrotor's yaw angle depending on the direction of motion and the orientation of the quadrotor but it is common to have the quadrotor's orientation aligned with the direction of motion. 
\end{remark}

\begin{remark}
	Initial conditions used to solve \eqref{equ:ch5:model3D} can be computed from system states as $V(0) = ||\vect{v}(0)||$ and $\vect{s}_r(0)=\frac{\vect{v}(0)}{V(0)}$.
	However, it can be numerically more stable to use final solutions of previous trajectory generation cycles (i.e. using $\vect{p}_r(\Delta t_f)$, $V(\Delta t_f)$ and $\vect{s}_r(\Delta t_f)$ as initial states for next cycle) assuming that a good tracking performance can be achieved by the controller.
	Also, if the quadrotor's initial velocity $V(0)=0$ (i.e. hovering), the initial heading vector can be selected as $\vect{s}_r(0)=\vect{p}_{goal}-\vect{p}(0)$ to make sure that the condition $\|\vect{s}(t)\|=1$ is satisfied.
\end{remark}

\begin{figure}[!th]
	\centering
	\includegraphics[width=0.75\linewidth]{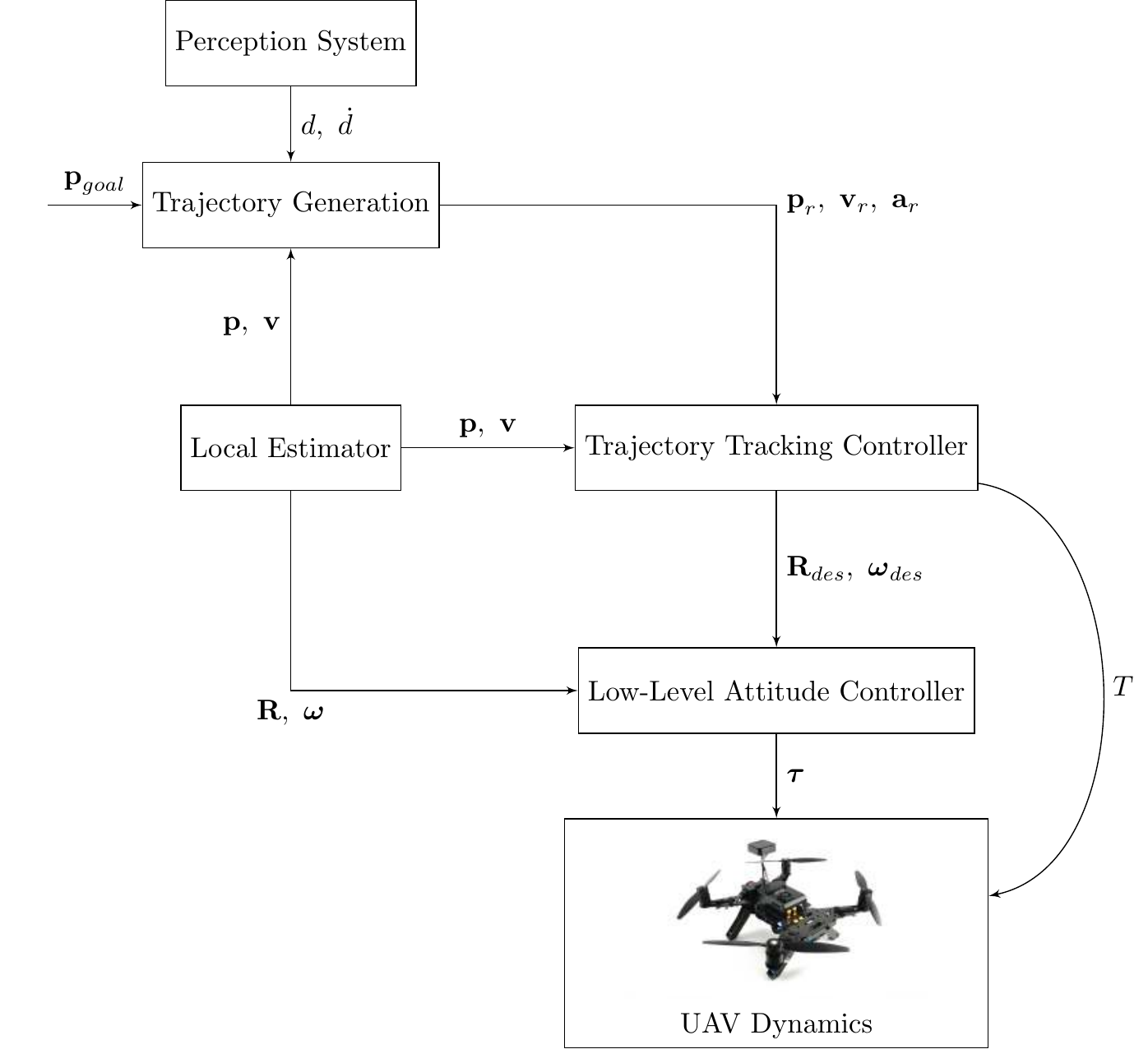}
	\caption{Overall architecture including trajectory generation and control for quadrotors} \label{fig:ch5:blockDiag}
\end{figure}

\begin{figure}[!th]
	\centering
	\includegraphics[width=0.75\linewidth]{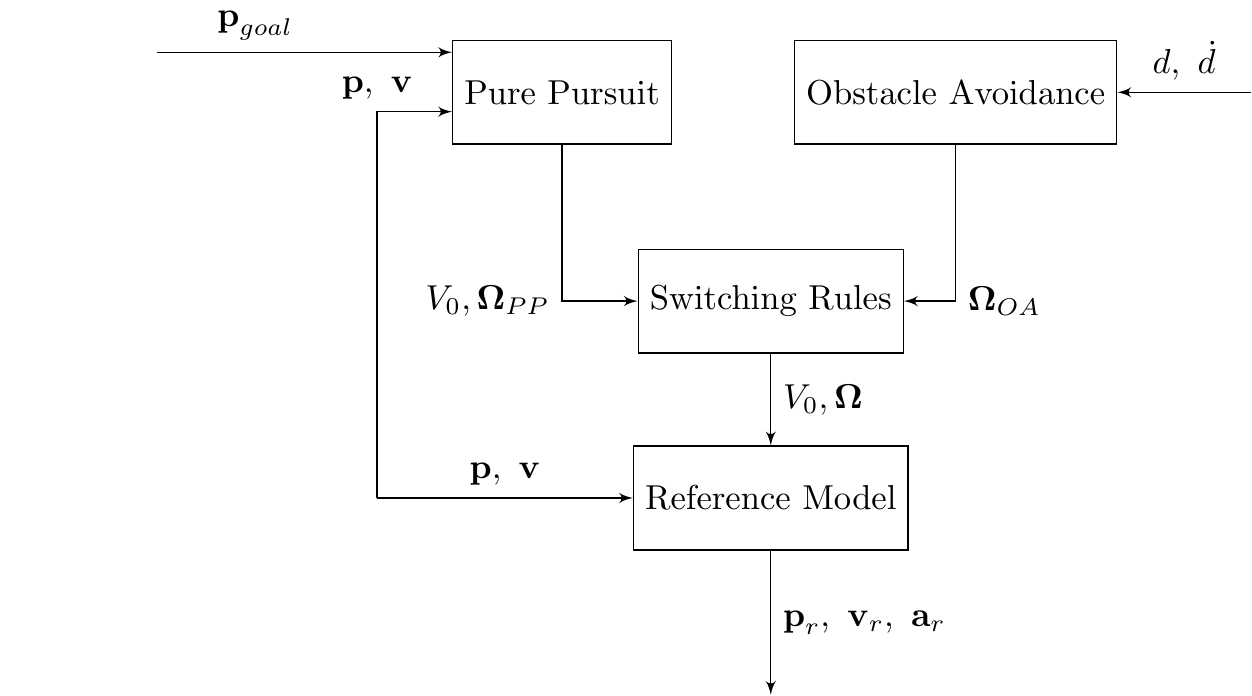}
	\caption{Illustration of the trajectory generation process} \label{fig:ch5:blockDiagTraj}
\end{figure}

\subsection{Pure Pursuit Guidance Law}

Define an error vector $\vect{p}_e$ towards the goal position as follows:
\begin{equation}\label{equ:ch5:err_goal}
\vect{p}_e = \vect{p}_{goal} - \vect{p}_r
\end{equation}
In pure pursuit mode, the quadrotor is moving with a constant velocity $\bar{V}$ towards the goal.
Hence, $V_0$ and $\vect{\Omega}_{PP}$ are designed according to the following:
\begin{eqnarray}
V_0 &=& \bar{V} \tanh(\mu ||\vect{p}_e||) \label{equ:ch5:V0} \\
\vect{\Omega}_{PP} &=& \Omega_{max} F(\vect{s}_r,\ \vect{p}_e) \label{equ:ch5:u_pp}
\end{eqnarray}
where the map function $F(\vect{w}_1,\ \vect{w}_2):\R^3\times \R^3 \to \R^3$ is defined by \cite{wang2018strategy}:
\begin{eqnarray}
&F(\vect{w}_1,\ \vect{w}_2) &:= \left\{
\begin{array}{ll}
0, & f(\vect{w}_1,\ \vect{w}_2)=0 \\
\frac{f(\vect{w}_1,\ \vect{w}_2)}{||f(\vect{w}_1,\ \vect{w}_2)||}, & otherwise
\end{array}
\right. \\
&f(\vect{w}_1,\ \vect{w}_2) &:= \vect{w}_2 - (\vect{w}_1 \cdot \vect{w}_2) \vect{w}_1
\end{eqnarray}
The map $F(\vect{w}_1,\ \vect{w}_2)$ is simply a vector in the plane spanned by both $\vect{w}_1$ and $\vect{w}_2$ that is perpendicular to $\vect{w}_1$ and directing towards $\vect{w}_2$ which acts as a steering function.
Hence, it is clear that \eqref{equ:ch5:u_pp} satisfies the constraints in \eqref{equ:ch5:constraints1}-\eqref{equ:ch5:constraints2}.

\subsection{Obstacle Avoidance Law}

Let $\vect{i}_{T}\in \R^3$ be a vector pointing towards a safe direction away from the nearest detected obstacle at time $t_{s_1}$.
A \textit{plane of avoidance} $\mathcal{P}_{av}$ is defined at the beginning of the avoiding maneuver (i.e. at $t=t_{s1}$) around each obstacle as the plane spanned by $\vect{p}_e(t_{s_1})$ (defined in \eqref{equ:ch5:err_goal}) and $\vect{i}_{T}$.
It is assumed that the heading vector $\vect{s}_r(t_{s_1})$ coincides with $\vect{p}_e(t_{s_1})$ when the obstacle avoidance mode gets activated according to the rule \textbf{R1} as it switches from the pure pursuit mode.
Thus, the plane's normal can be determined by:
\begin{equation}
\vect{i}_{\mathcal{P}_{av}} = \vect{s}_r(t_{s_1}) \times  \vect{i}_{T} 
\end{equation}
Using the above, the navigation law for obstacle avoidance mode is given by:
\begin{equation}\label{equ:ch5:navigationLaw3D}
\begin{aligned}
\vect{\Omega}_{OA} &= \Gamma \Omega_{max}\ \tanh(\mu\sigma_d) \vect{i}_{n} \\ 
\sigma_d &= \dot{d}(t) + \chi (d(t) - d_0) \\
\vect{i}_{n}  &= \vect{i}_{\Pav} \times\vect{s}_r
\end{aligned}
\end{equation}
where $\vect{i}_{n} \in \Pav$ is a unit vector orthogonal to $\vect{s}_r$ and directing away from nearest obstacle, $\Gamma=1$ if the obstacle is in the direction of motion, $\Gamma=-1$ otherwise, and $\chi(\cdot)$ is a saturation function defined as:
\begin{equation}
\chi(c) = \left\{\begin{array}{cc}
\gamma c, & \text{if}\ |c|\leq \delta \\
\delta \gamma, & \text{if}\ c > \delta \\
-\delta\gamma, & \text{if}\ c < -\delta
\end{array}\right.
\end{equation}
where $\gamma,\delta>0$.
It is evident that \eqref{equ:ch5:navigationLaw3D} is feasible since it respects the conditions in \eqref{equ:ch5:constraints1}-\eqref{equ:ch5:constraints2}.
More details and assumptions about this strategy are provided in the previous chapter.

\begin{remark}
	The switching condition in \textbf{R2} was implemented with some tolerance $\bar{\epsilon}>0$ to handle numerical and tracking errors where $\vect{s}_r(t_{s_2})$ is considered to be directed towards $\vect{p}_{goal}$ when
	\begin{equation*}
	\cos^{-1}\left(\frac{\vect{s}_r(t_{s_2}) \cdot \vect{p}_e (t_{s_2})}{||\vect{p}_e (t_{s_2})||}\right) < \bar{\epsilon}
	\end{equation*}
\end{remark}

\section{Experiments} \label{sec:ch5:results}

\subsection{Experimental Setup}

An Intel Aero Ready to Fly (RTF)\footnote{https://github.com/intel-aero/meta-intel-aero/wiki} quadrotor was used in the experiments which is shown in \cref{fig:ch5:aero}.
It is equipped with an onboard computer with Intel® Atom™ x7-Z8700 processor running the Ubuntu operating system, and it has a flight controller unit running the PX4\footnote{https://px4.io/} flight stack which is a collection of guidance, navigation and control algorithms for UAVs.
An Intel Realsense D435 camera is also attached to the quadrotor which can provide RGB images and depth information.
However, the camera was not considered in the current experiments.
We also use a motion capture system (OptiTrack) to provide the ground truth  position and orientation of the vehicle at 125Hz.
An extended Kalman filter is used within the PX4 stack to provide estimates of the quadrotor's states (i.e. position, orientation and velocity) which are used for control.

Our strategy runs completely on the onboard computer utilizing the open-source Robot Operating System (ROS) which makes it easier to build the complete navigation stack.
The proposed trajectory tracking controller in \eqref{equ:ch5:errors}-\eqref{equ:ch5:thrustInput} was implemented to generate thrust $T$ and attitude commands $(\phi_{des},\theta_{des},\psi_{des})$ extracted from $\vect{R}_{des}$ in accordance with \eqref{equ:ch5:RotR}.
These commands are sent to the flight controller at 100 Hz where a low-level attitude controller is used to generate the required body moments at a higher rate.
It should be mentioned that PX4 accepts normalized collective thrust inputs within $[0,\ 1]$ which was done by using an estimated scaling factor for $T$.
Also, the trajectory generation method presented in \cref{sec:ch5:trajectoryGeneration} was implemented to run in parallel to provide $\vect{p}_r$, $\vect{v}_r$ and $\vect{a}_r$ with a resolution of $0.01s$. 

\begin{figure}[!t]
	\centering
	\includegraphics[width=0.75\linewidth,trim=0 500 0 400, clip]{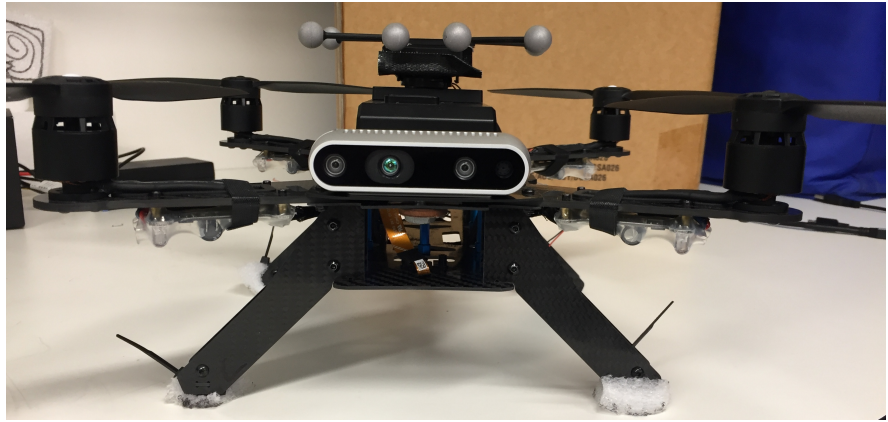}
	\caption{Intel Aero RTF quadrotor used in the experiments equipped with an Intel Realsense D435 camera} \label{fig:ch5:aero}
\end{figure}

\subsection{Experiments Description}

The purpose of the conducted experiments is to show the capability of the proposed strategy to generate safe 3D avoidance maneuvers around an obstacle to reach a desired goal position along with evaluating the performance of the trajectory tracking control where all computations were done onboard.
A single obstacle with a known location (enclosed by a sphere with some safety margin around the flying altitude) was introduced in the way of the quadrotor.
Four cases were considered for the online trajectory generation with different choices of safe directions $\vect{i}_{T}$ resulting in different $\mathcal{P}_{av}$ for each case, and they are given by:
\begin{itemize}
	\item Case 1: $\vect{\bar{i}}_{T}=[0,\ 2,\ -1]^T$
	\item Case 2: $\vect{\bar{i}}_{T}=[0,\ 0,\ -1]^T$
	\item Case 3: $\vect{\bar{i}}_{T}=[-1,\ 0,\ 0]^T$
	\item Case 4: $\vect{\bar{i}}_{T}=[0,\ 3,\ 1]^T$ 
\end{itemize}
where $\vect{i}_{T} = \frac{\vect{\bar{i}}_{T}}{||\vect{\bar{i}}_{T}||}$ since $\vect{i}_{T}$ should be of unit length.
Furthermore, the parameters used in the experiments are as follows: $g = 9.81\ m/s^2$, $\vect{C}_1 = diag\{1.3, 1.3, 3.5\}$, $\vect{C}_2 = diag\{2.0, 2.0, 4.0\}$, $c_v=4.0$, $\bar{V}=0.5\ m/s$, $\Omega_{max}=5.0$, $d_0 = 0.5 \ m$, $C = 0.75 \ m$, $\delta=0.35$ and $\gamma=0.5$.
Also, the reference yaw angle was chosen to be $\psi_r=0$ during the whole flight.
However, it is possible to align $\psi_r$ with the direction of motion (i.e. by aligning $\vect{x}_{\mathcal{B}}$ with the heading vector $\vect{s}_r$) when using the onboard camera.
Another possible approach is to choose an orientation that can maximize the camera's field of view (FOV) to get good information about nearest obstacle.

\Cref{fig:ch5:env} shows the indoor environment used for the flights, and a video of the experiments can be found at \href{https://youtu.be/ByZklzqjMW0}{https://youtu.be/ByZklzqjMW0}.

\begin{figure}[!t]
	\centering
	\includegraphics[width=\linewidth]{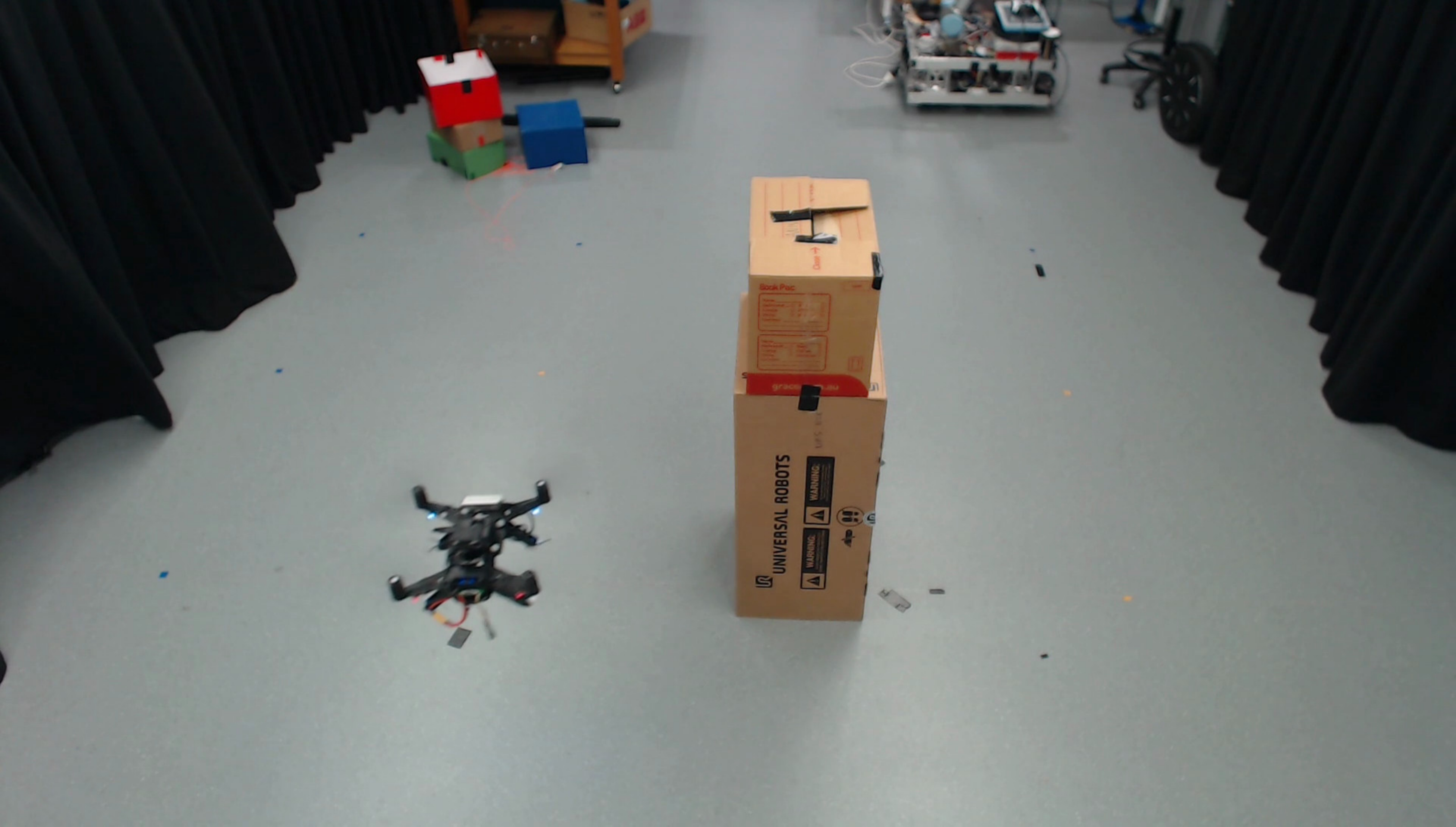}
	\caption{Indoor setup used for the experiments} \label{fig:ch5:env}
\end{figure}

\subsection{Results}

Each experiment follows a sequence of flight modes which is: takeoff $\to$ hover $\to$ reactive navigation to goal position $\to$ hover $\to$ land.
Flight data was recorded for all cases, and the results were analyzed using MATLAB.
These results are shown in \cref{fig:ch5:exp1,fig:ch5:exp2,fig:ch5:exp3,fig:ch5:exp4,fig:ch5:exps_vel,fig:ch5:exps_perr,fig:ch5:dobs}.
The actual executed paths in each case can be seen in \cref{fig:ch5:exp1,fig:ch5:exp2,fig:ch5:exp3,fig:ch5:exp4} where an enclosing sphere near the flying altitude is used to represent the obstacle, and the plane of avoidance is shown as well.
The velocity profile $||\vect{v}||$ for all cases is given in \cref{fig:ch5:exps_vel}.
For takeoff and land modes, the quadrotor's velocity was around 1.5 $m/s$ and 0.7 $m/s$ respectively.
The reactive navigation took place between $35s$ and $51s$ approximately during which the velocity was around 0.5-0.6 $m/s$ due to tracking errors.
The norms of position tracking errors $||\vect{e}_{\vect{p}}||$ for the four cases are shown in \cref{fig:ch5:exps_perr} showing good performance.
It was noticed that better tracking performance can be achieved by obtaining a better estimate of the thrust scaling factor.
When selecting the design parameters $C$ and $d_0$, good margin needs to be considered to maintain safety even with some small tracking error.
The distance between the quadrotor and the obstacle during the reactive navigation mode period is also shown in \cref{fig:ch5:dobs} which verifies that the proposed reactive strategy can maintain a good safety margin around $d_0$. 

\begin{figure}[!th]
	\centering
	\includegraphics[width=\linewidth]{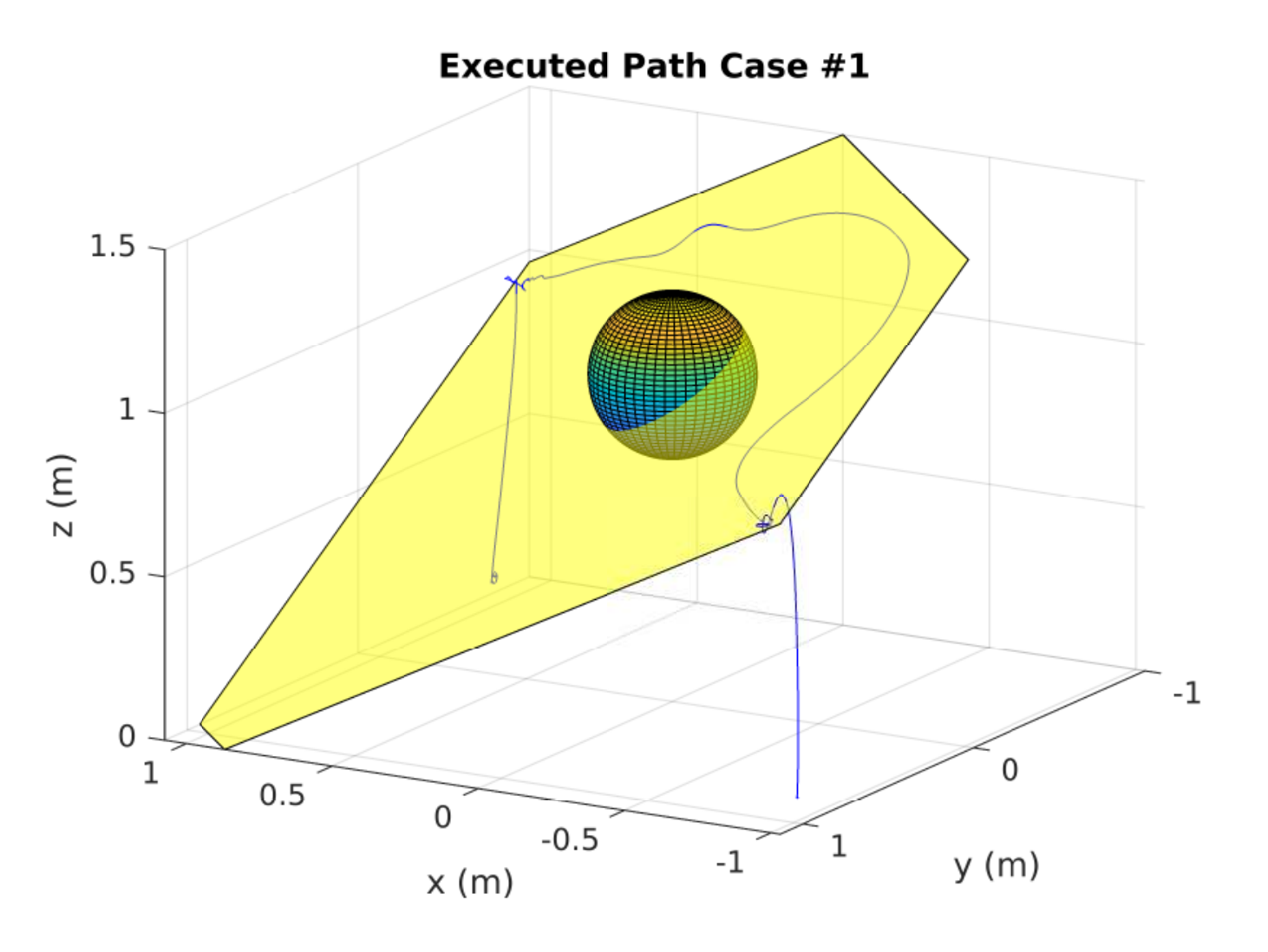}
	\caption{Actual executed path for case 1 along with the assigned plane of avoidance} \label{fig:ch5:exp1}
\end{figure}

\begin{figure}[!th]
	\centering
	\includegraphics[width=\linewidth]{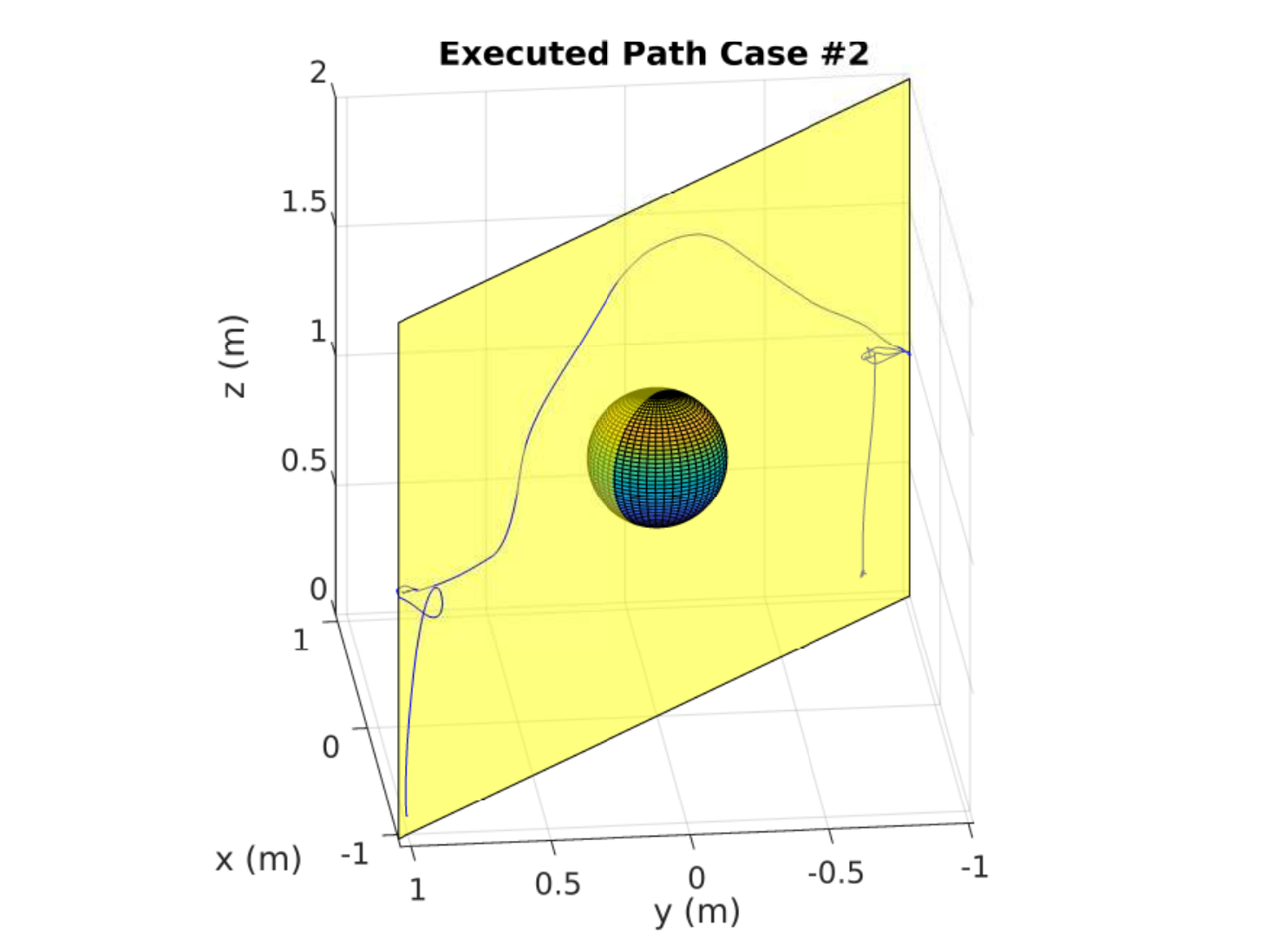}
	\caption{Actual executed path for case 2 along with the assigned plane of avoidance} \label{fig:ch5:exp2}
\end{figure}

\begin{figure}[!th]
	\centering
	\includegraphics[width=\linewidth]{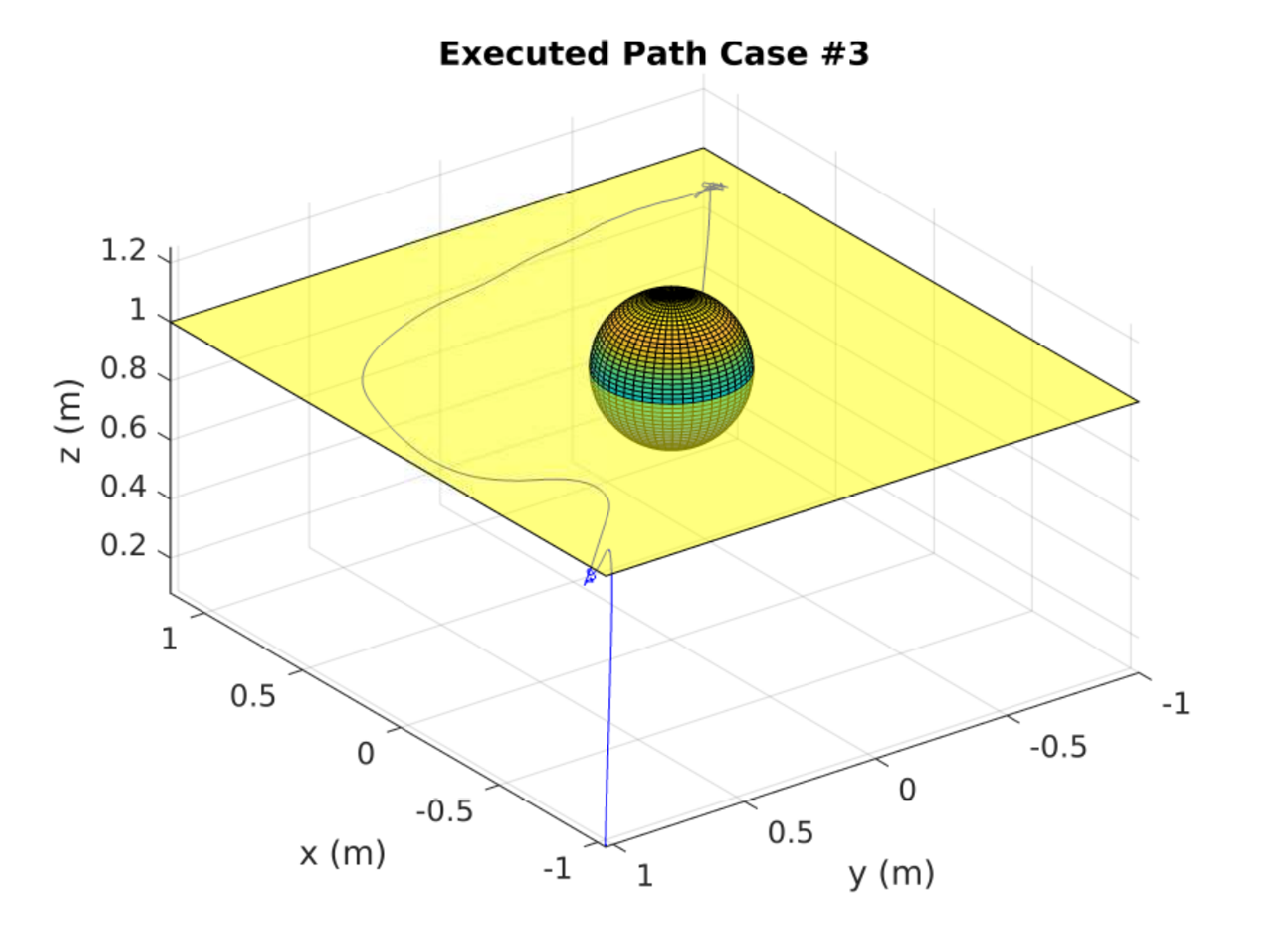}
	\caption{Actual executed path for case 3 along with the assigned plane of avoidance} \label{fig:ch5:exp3}
\end{figure}

\begin{figure}[!th]
	\centering
	\includegraphics[width=\linewidth]{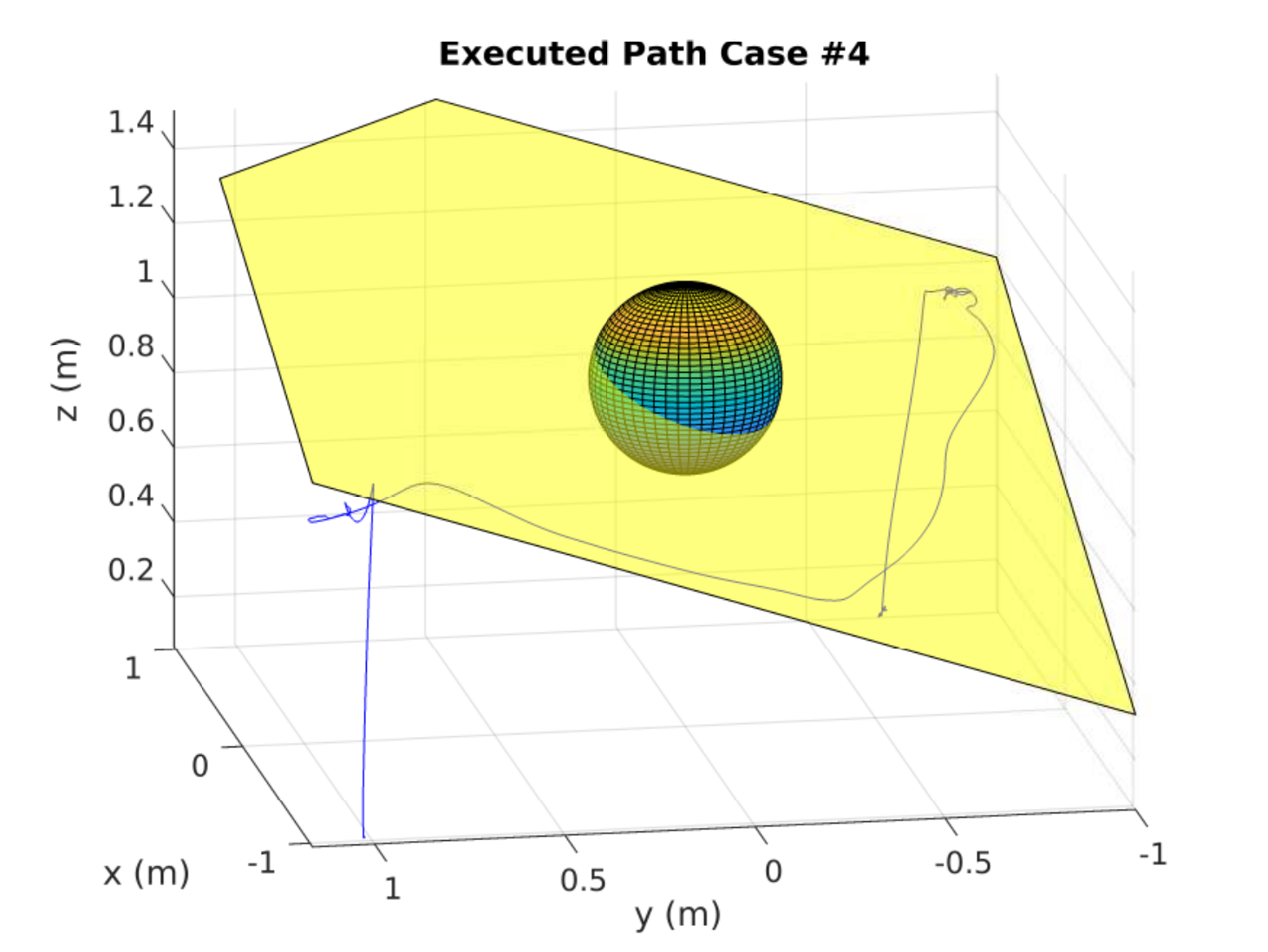}
	\caption{Actual executed path for case 4 along with the assigned plane of avoidance} \label{fig:ch5:exp4}
\end{figure}

\begin{figure}[!th]
	\centering
	\includegraphics[width=\linewidth]{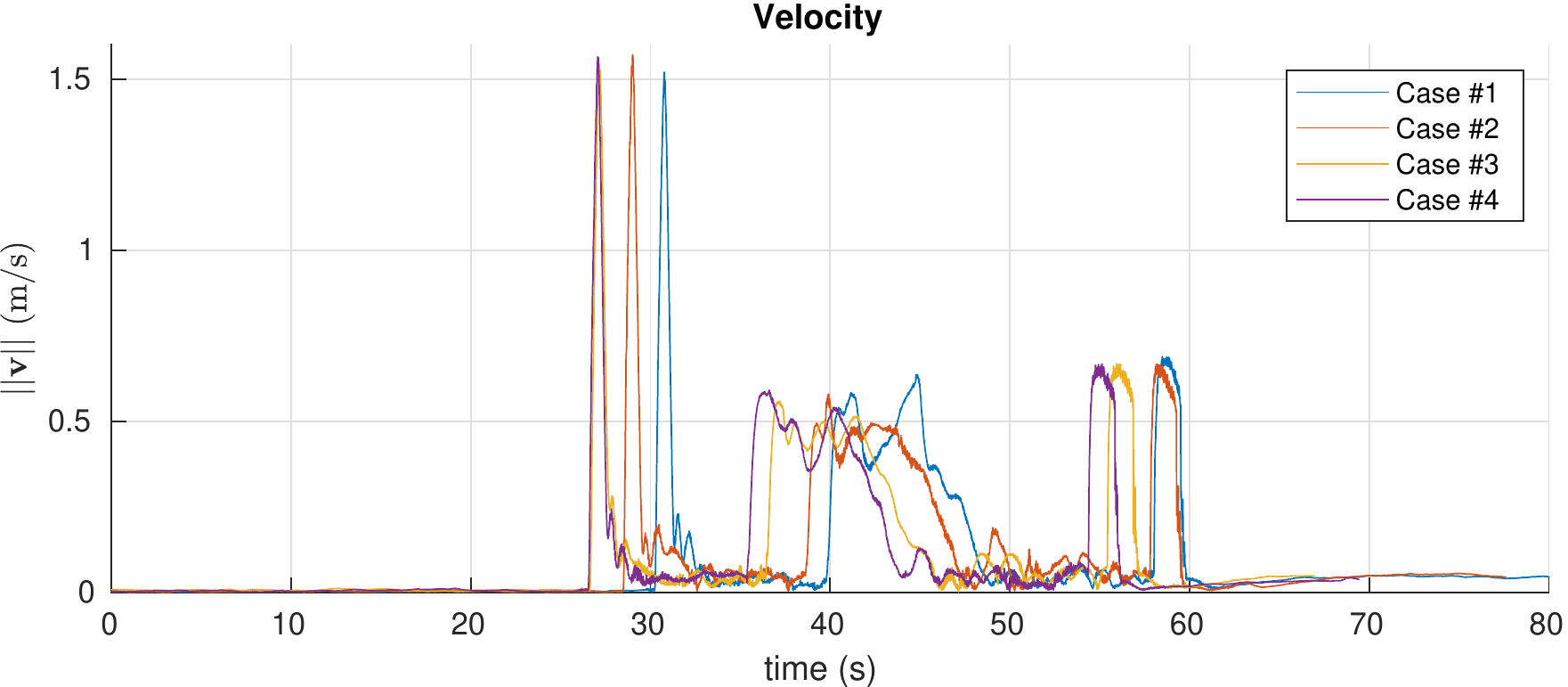}
	\caption{Recorded quadrotor velocity $||\vect{v}||$ over time for all cases} \label{fig:ch5:exps_vel}
\end{figure}

\begin{figure}[!th]
	\centering
	\includegraphics[width=\linewidth]{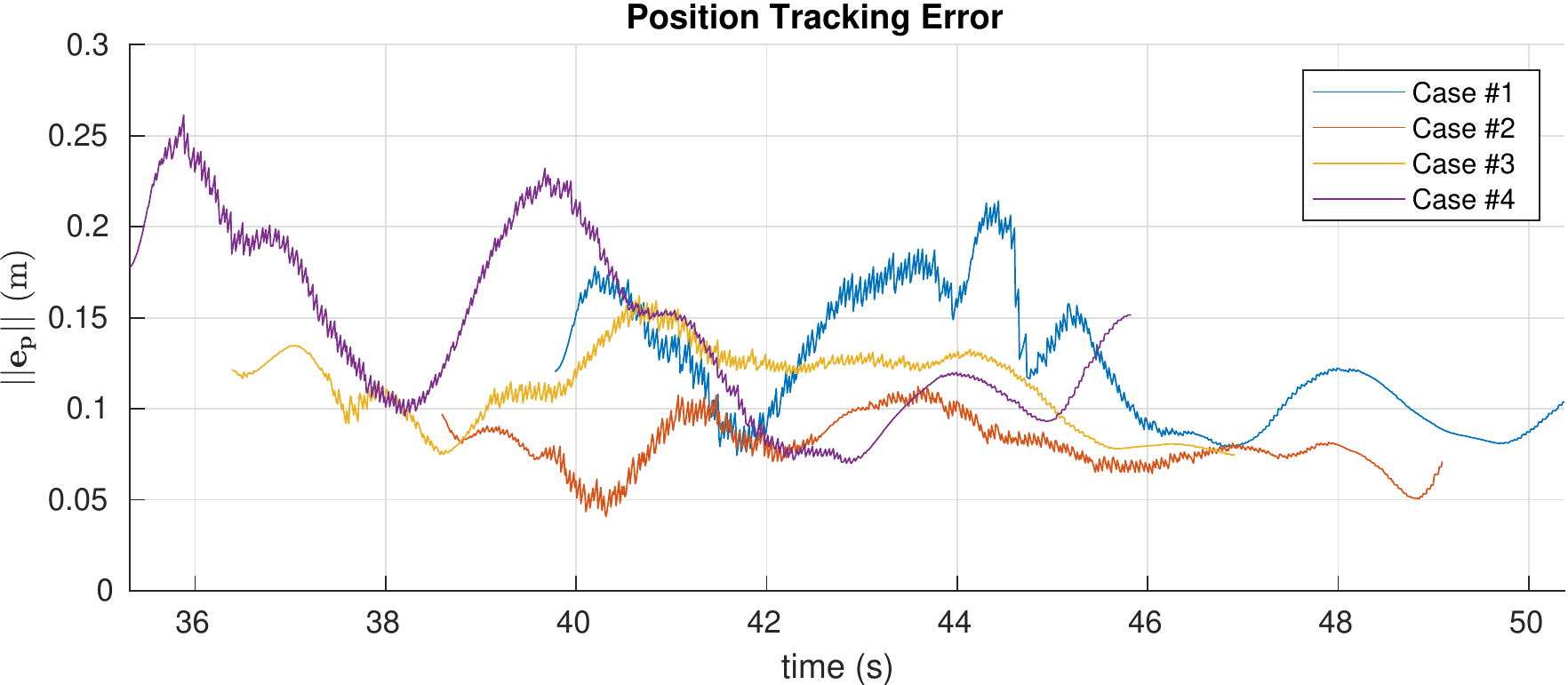}
	\caption{Recorded position tracking errors $||\vect{e}_{\vect{p}}||$ during the reactive navigation mode period for all cases} \label{fig:ch5:exps_perr}
\end{figure}

\begin{figure}[!th]
	\centering
	\includegraphics[width=\linewidth]{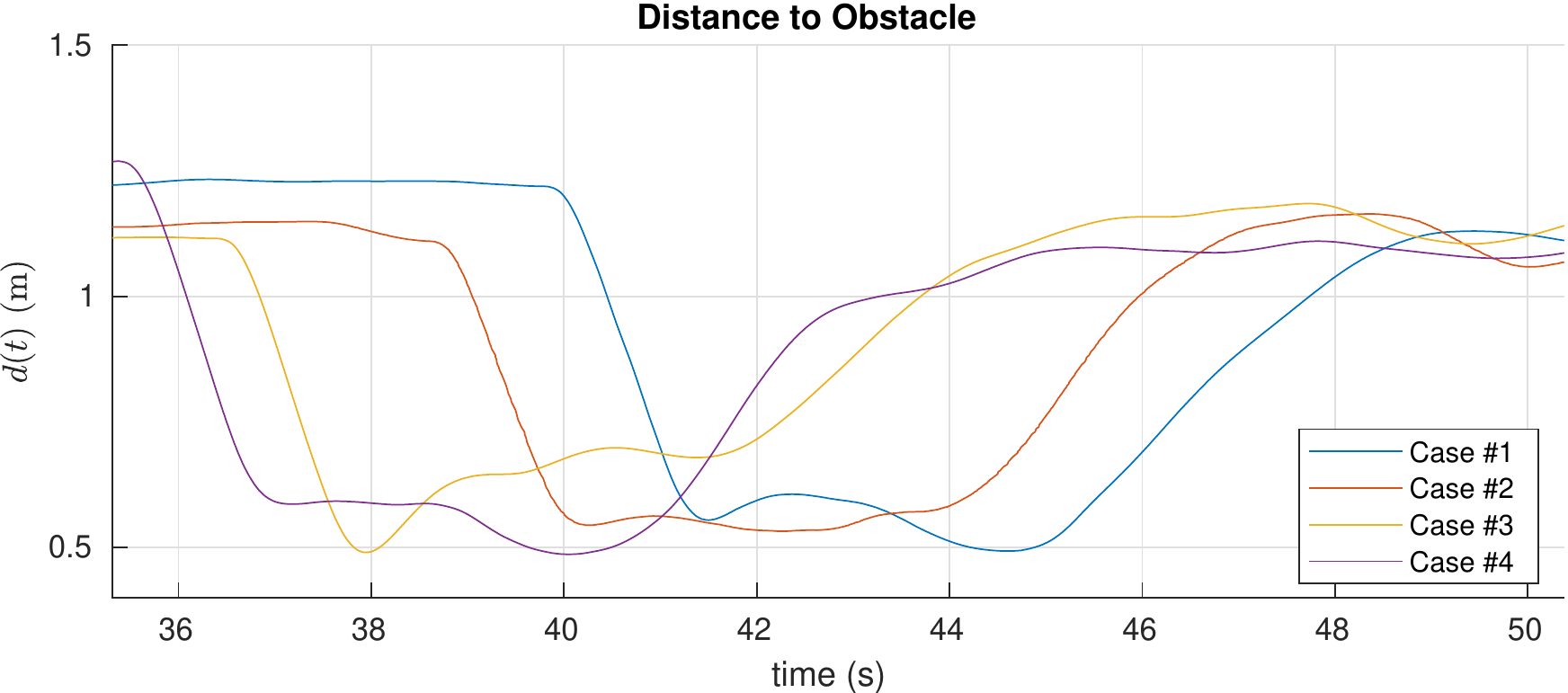}
	\caption{Distance to nearest obstacle $d(t)$ during the reactive navigation mode period for all cases} \label{fig:ch5:dobs}
\end{figure}

\section{Conclusion \& Future Work} \label{sec:ch5:conclusion}

A reactive strategy was proposed to generate collision-free trajectories for quadrotors navigating in unknown environments.
Concepts from guidance laws and equiangular navigation are used to allow for fast computational performance when avoiding obstacles.
Additionally, a trajectory tracking control law was developed based on sliding mode control and differential-flatness property of quadrotor dynamics.
Experiments were conducted running the proposed trajectory generation and control methods online on a quadrotor.
Future work include implementing the developed strategy using measurements from onboard sensors and considering navigation in cluttered environments.
    \chapter{A 3D Collision Avoidance Method for UAVs using Deformable Paths \label{cha:deforming_approach}}

The developed 3D reactive method in \cref{cha:methods_reactive3D,cha:reactive_impl} can generate 3D boundary following avoidance maneuvers based only on relative distance to obstacles which offers a good solution for vehicles with very limited resources.
Another 3D collision-free navigation approach for UAVs is presented in this chapter.
It can offer better maneuverability around dynamic obstacles by applying real-time 3D deformations to local paths.
The method has low computational cost compared to search-based and optimization-based local planning methods as it determines the deformations directly based on sensors measurements.
An improved trajectory tracking controller for quadrotors is also proposed in this chapter based on the one developed in the previous chapter.
Some of the results shown in this chapter were presented in \cite{elmokadem2019real,elmokadem2020control}.

\section{Introduction}\label{sec:ch6:Intro}

Unmanned aerial vehicles (UAVs) have rapidly emerged in many civilian applications such as search \& rescue \cite{goodrich2008supporting}, wireless sensor networks \cite{li2018wireless}, 3D mapping \cite{nex2014uav}, objects grasping and aerial  manipulation \cite{korpela2012mm,ruggiero2018aerial}, underground Mines exploration \cite{li2020autonomous}, etc.
Achieving a fully autonomous behavior with least human interaction is highly desirable in many applications.
However, it is very challenging especially when sharing flight space with other aerial vehicles or navigating in highly dynamic areas such as indoor and urban environments.
Thus, reliable navigation strategies with collision avoidance abilities are required to maintain the safety of the vehicle and its surroundings.

As highlighted earlier, navigation approaches can generally be classified as global path planning, sensor-based or hybrid \cite{hoy2015algorithms}.
Global path planning requires full knowledge about the environment which makes it inefficient in dynamic scenarios in terms of memory and computational requirements especially in three-dimensional (3D) spaces.
On contrary, sensor-based and hybrid methods can handle dynamic environments by planning safe motions in real-time.
Reactive or \textit{Sense and Avoid (S\&A)} techniques offer solutions with the best computational cost among other sensor-based methods. 

There exist a number of collision avoidance methods dealing with dynamic obstacles such as velocity obstacle \cite{fiorini1998motion}, collision cones \cite{chakravarthy1998obstacle, chakravarthy2012generalization}, boundary following \cite{matveev2011method,elmokadem20183d}, artificial potential field (APF) \cite{khatib1986real,zhu20163d}, optimization-based \cite{gao2017quadrotor,lindqvist2020nonlinear}, and other methods like \cite{shim2007evasive, belkhouche2009reactive,van2011reciprocal,yang20133d,kamel2017robust,wang2018strategy,falanga2020dynamic}.
Some of these methods achieve avoidance of dynamic obstacles by moving relative to its boundary which can sometimes produce jerky and/or non-optimal motions.
Another major drawback is that most of the existing methods focused only on the 2D problem without utilizing the full maneuvering capabilities of vehicles which can move freely in 3D such as UAVs and autonomous underwater vehicles (AUVs).
Some recent works have attempted to address the more complex 3D problem such as \cite{thanh2018simple,wang2018strategy,yang20133d,Choi2017,Wiig2018,lindqvist2020nonlinear}.

As the complexity of UAV applications increases, more advanced 3D navigation strategies are needed.
Hence, the main contribution of this chapter is the development of a computationally-light 3D navigation strategy for UAVs moving in unknown and dynamic environments.
A sense-and-avoid based approach is adopted inspired by Elastic Bands \cite{quinlan1993elastic} to close the gap between path planning and control.
Quintic Bezier splines are used to generate smooth paths, and real-time smooth deformations are applied based on information interpreted from sensors measurements to provide quick reactions to obstacles.
As only light processing of sensory data are needed, the approach can be classified as reactive \cite{tobaruela2017reactive}.
The method is developed at a high-level first considering a general 3D kinematic model applicable to different UAV types and AUVs.
The design is further extended to quadrotor UAVs including their dynamics.
Several simulations were carried out based on the general kinematic model considering different scenarios of unknown and dynamic obstacles.
Software-in-the-loop (SITL) simulations were also performed using the full quadrotor dynamical model running in real-time using the Gazebo simulator.
The used hardware in these SITL simulations is similar to that of our real quadrotors which help evaluating the computational performance of our algorithms.

The structure of this chapter is as follows.
First, 3D navigation problems are formulated in \cref{sec:ch6:problem} based on both a general 3D kinematic model and quadrotor dynamical model.
Next, the suggested sense-and-avoid navigation strategy is described in \cref{sec:ch6:strategy} with control laws designed for both models.
After that, several simulation cases are presented in \cref{sec:ch6:simulation} to evaluate the overall approach and its computational performance including software-in-the-loop simulations in Gazebo.
Finally, this work is concluded in \cref{sec:ch6:conclusion}.

\section{Problem Formulation}\label{sec:ch6:problem}

Consider a UAV navigating in a three-dimensional (3D) environment $\mathcal{W}\subset \R^3$ filled with obstacles $\mathcal{O} \subset \mathcal{W}$ which can be static or dynamic.
The environment is assumed unknown which means information about obstacles are not known a priori.
We first consider a general 3D kinematic model which is applicable to different types of vehicles moving in 3D spaces.
A special case is then considered for quadrotor UAVs where a full dynamical model is used capturing quadrotor dynamics.
The description of these two models are given next.

\subsection{General 3D Kinematic Model}

Let the position of the UAV be defined with respect to some inertial coordinate frame $\mathcal{I}$ as $\bm{p} = [x, y, z]^T$.
The linear speed of the vehicle is denoted as $V$ where the direction of motion (i.e. the normalized velocity vector) is expressed using two angles, namely the heading angle $\psi$ and the flight path angle $\alpha$.
The following nonholonomic kinematic model is considered:
\begin{equation}\label{equ:ch6:model}
	\begin{aligned}
		\dot{x}(t) &= V(t) \cos\beta(t) \cos\alpha(t) \\
		\dot{y}(t) &= V(t) \sin\beta(t) \cos\alpha(t) \\
		\dot{z}(t) &= V(t) \sin\alpha(t) \\
		\dot{\beta}(t) &= u_{\beta}(t) \\
		\dot{\alpha}(t) &= u_{\alpha}(t)
	\end{aligned}
\end{equation}
where the control inputs to this model are the linear speed $V$, the heading angular speed $u_{\beta}$ and the flight path angular speed $u_{\alpha}$.
Due to physical limitations on any mechanical system, it is also considered that the control inputs need to be bounded according to the following:
\begin{equation}
	\begin{aligned}
		0 \leq V(t) &\leq V_{max} \\ |u_{\beta}(t)| \leq u_{\beta,max}&,\ \ |u_{\alpha}(t)| \leq u_{\alpha,max}
	\end{aligned}
\end{equation}
The model \eqref{equ:ch6:model} is a general kinematic model which is applicable to different vehicles moving in 3D spaces not just UAVs such as missiles, and autonomous underwater vehicles. 

\subsection{Quadrotor Dynamical Model}

The general kinematic model \eqref{equ:ch6:model} can be mapped into a quadrotor full model which include both kinematics and dynamics.
In order to describe the quadrotor dynamics, one need to define an additional coordinate frame attached to the vheicle, namely a body-fixed frame $\mathcal{B}$.
The UAV position, linear velocity and linear acceleration ($\bm{p},\bm{v},\bm{a}\in \R^3$) are expressed in the inertial frame.
A rotation matrix $\bm{R}$ between the inertial and body-fixed frames can be used to describe the vehicle orientation.
Commonly, quaternions or Euler's angles, namely roll $\phi$, pitch $\theta$ and yaw $\psi$, are also used to describe the UAV orientation.
Furthermore, the angular velocity of the UAV is expressed in th inertial frame as $\bm{\omega}=[\omega_x,\ \omega_y,\ \omega_z]^T$.
Using the above definitions, a quadrotor full UAV can be written as \cite{hamel2002dynamic,faessler2017differential}:
\begin{align}
	\vect{\dot{p}}(t) &= \vect{v}(t) \label{equ:ch6:model1}\\
	\vect{\dot{v}}(t) &= -g \bm{e}_3 + T(t) \bm{R}(t)\bm{e}_3 := \vect{a}(t)\label{equ:ch6:model2}\\
	\vect{\dot{R}}(t) &= \vect{R}(t) \vect{\hat{\omega}}(t)\label{equ:ch6:model3} \\
	\vect{\dot{\omega}}(t) &= \vect{J}^{-1} \Big(-\vect{\omega}(t) \times \vect{J} \vect{\omega}(t) + \vect{\tau}(t)\Big)\label{equ:ch6:model4}
\end{align}
where $g$ is the gravitational acceleration, $\vect{J}$ is the UAV inertia matrix, and $\bm{e}_3 = [0,\ 0,\ 1]^T$.
Also, $\bm{\hat{\omega}}$ is a skew-symmetric matrix defined in terms of $\bm{\omega}$ as follows:
\begin{equation}
	\bm{\hat{\omega}} = \left[\begin{array}{ccc}
		0 & -\omega_z & \omega_y \\
		\omega_z & 0 & -\omega_x \\
		-\omega_y & \omega_x & 0
	\end{array}\right].
\end{equation}
The control inputs are the mass-normalized collective thrust $T(t) \in \R$, and the body torques $\bm{\tau}(t)\in \R^3$ which can then be mapped into appropriate motors speeds.

\subsection{Problem Definition}

The first considered problem is how to allow a vehicle to safely reach some desired feasible goal position $\bm{p}_{goal} \in \mathcal{W}\backslash\{\mathcal{O}\}$ starting from any initial position $\bm{p}_{initial}\in \mathcal{W}\backslash\{\mathcal{O}\}$ while avoiding collisions with obstacles given that it is feasible to do so.
Thus, it is required that $\lim\limits_{t \to \infty}\bm{p}(t) = \bm{p}_{goal}$ and $d(t)\geq d_{safe}$ for $t \geq 0$ where $d(t)=\min_{\bm{p}_*(t) \in \mathcal{O}}{\|\bm{p}_*(t) - \bm{p}(t)\|}$ is the shortest distance to the closest obstacle at time $t$, and $d_{safe}>0$ is some required minimum safety distance.
This problem is dealt with in a general setup where it is required to develop a control strategy which can be applicable to any vehicle moving in 3D spaces whose motion is governed by the general kinematic model \eqref{equ:ch6:model}.
The second problem is targeted towards how to implement such a strategy with quadrotor UAVs where the model \eqref{equ:ch6:model1}-\eqref{equ:ch6:model4} is adopted for the control system design.

Generally, the following assumptions are made:
\begin{assumption}
	The vehicle can estimate its position $\bm{p}$ and orientation (i.e. $\beta$ and $\alpha$) at any given time. 
\end{assumption}
\begin{assumption}\label{asm:ch6:sensing}
	The vehicle can determine the distance to the closest obstacle using onboard sensors.
\end{assumption}
\begin{assumption}\label{asm:ch6:obs_speed}
	Dynamic obstacles are not chasing the UAV, and their linear speeds are lower than the maximum possible speed of the UAV for guaranteed safety. 
\end{assumption}

\begin{remark}
	\Cref{asm:ch6:obs_speed} is a necessary technical assumption to ensure the motion safety.
	However, the suggested approach can still work in many cases where the obstacles are moving with higher speeds in a non-aggressive direction.
\end{remark}

\section{Sense \& Avoid Control Strategy}\label{sec:ch6:strategy}

The overall suggested control structure is decomposed into local path planning and path tracking control design.
The local path planning adopts the idea of \textit{deformable paths} which is basically modifying path segments around obstacles in a way similar to stretching elastic materials such as rubber bands.
Hence, it can be categorized as a sense and avoid strategy.
This is motivated by the method of Elastic Bands proposed by Quinlan and Khatib in \cite{quinlan1993elastic}.
The main goal behind this idea is to close the gap between path planning and control providing reactions directly based on sensors information.
Hence, the deformation process is mainly based on a closed-form solution in terms of the distance to the closest obstacle to avoid typical higher latency that comes when using search-based path planning methods.
Thus, the proposed method has a better computational cost compared to search-based path planning methods making it suitable for vehicles with limited computing power as well as increased capability in avoiding obstacles with higher speeds.

The proposed strategy can be described as follows.
On a lower-level, a path tracking controller is used to track some reference geometric path $\Gamma$ which is generated by a high-level path planning component. 
Starting from any initial position $\bm{p}_{initial}$ and given some initial knowledge about the obstacles $\mathcal{O}_0$, an initial path $\Gamma_0 \subset \mathcal{W}\backslash\{\mathcal{O}_0\}$ is determined to reach $\bm{p}_{goal}$ which is then assigned as the current reference path to track such that $\Gamma = \Gamma_0$.
The goal position can typically be a location of interest in a global map, a visually obtained target, a promising exploration location or a command from a remote operator.
In many cases, it can be assumed that there is no knowledge about the obstacles initially (i.e. $\mathcal{O}_0=\emptyset$) which makes it easier to quickly compute a path that satisfies some given boundary conditions.
Otherwise, it is possible to consider initial paths obtained from a global path planner if one exists.
The high-level layer checks a segment of the reference path $\Gamma$ ahead of the vehicle's position whenever new measurements arrive from onboard sensors.
Once a detected obstacle is found to be closer to the current reference path than the safety margin, the path segment is deformed in a certain way such that the new deformed path becomes safe.
The deformation behavior looks more like stretching part of the path away from the obstacle in some safe direction as will be explained in the next subsections.
An illustration of such deformation is shown in \cref{fig:ch6:illustration}.
Note that the suggested method does not make assumptions on the obstacles shapes and considers only a general sensing model where a fraction of nearby obstacles can be detected.

\subsection{Path Parametrization using quintic Bezier splines}

An ideal choice to represent geometric paths is to use parametric curves or splines whose shapes can be manipulated through some control points (or knots).
This allows us to apply real-time deformations to existing path with very low computational cost.
Examples of parametric functions commonly used for path representation include interpolating polynomials, Bezier curves/splines, B-splines, etc.
In order to satisfy $C2$ continuity constraints, quintic Bezier splines were chosen in this work to represent paths and to generate smooth path deformations.
Another driving factor for this choice is the localism property of Bezier splines where it is possible to apply changes to a segment of the overall path without significantly affecting the whole path \cite{sprunk2008planning}.
This helps providing more efficient path deformations around obstacles.
Also, the selected Bezier splines order mainly relies on the number of boundary conditions need to be satisfied as will be shown later.
It is possible to consider Bezie splines with lower orders but it was observed that applied deformations result sometimes in longer paths away to satisfy the continuity constraints which is less efficient.

In general, Bezier splines can be described using the following equation:
\begin{equation}\label{equ:ch6:bezier}
f(s) = \sum_{k=0}^{n} \left(\begin{array}{c}
n \\ k
\end{array}\right) (1-s)^{n-k} s^k P_k, \ \ \ 0 \leq s \leq 1
\end{equation}
where $n$ represents the Bezier curve order, $P_k,\ k={0,\cdots,n}$ are $n+1$ control points (knots), and
\begin{equation*}
\left(\begin{array}{c}
n \\ k
\end{array}\right) = \frac{n!}{k!(n-k)!}
\end{equation*}

Quintic Bezier curves requires 6 control points $\{P_0,\cdots,P_5\}$ where $n = 5$.
Thus, one can rewrite \eqref{equ:ch6:bezier} for quintic Bezier splines in a matrix multiplication form such as:
\begin{equation}\label{equ:ch6:quinticBezier}
\begin{array}{lr}
	f(s) = T(s) \cdot G \cdot CP, & s \in[0,1]
\end{array}
\end{equation}
where
\begin{equation*}
\begin{aligned}
T(s) &= [s^5,\ s^4,\ s^3,\ s^2,\ s,\ 1]_{1\times 6} \\
G &= \left[\begin{array}{cccccc}
-1 & 5 & -10 & 10 & -5 & 1 \\
5 & -20 & 30 & -20 & 5 & 0 \\
-10 & 30 & -30 & 10 & 0 & 0 \\
10 & -20 & 10 & 0 & 0 & 0 \\
-5 & 5 & 0 & 0 & 0 & 0 \\
1 & 0 & 0 & 0 & 0 & 0
\end{array}\right]_{6\times 6} \\
CP &= \left[\begin{array}{c}
P_0 \\
P_1 \\
\vdots \\
P_5
\end{array}\right]_{6\times 1}
\end{aligned}
\end{equation*}

Given a number of $L$ waypoints $\{p_1,p_2,\cdots,p_{L}\}$, $p_i \in \R^3$, a smooth path connecting all points can be constructed as piecewise continuous Bezier splines according to the following:
\begin{equation}\label{equ:ch6:fullPath}
\Gamma(\bar{s}) = \left\{\begin{array}{c}
f_1(s_1) \\
f_2(s_2) \\
\vdots \\
f_{L-1}(s_{L-1})
\end{array}\right.
\end{equation}
where the splines $f_i(s_i)$ are given by \eqref{equ:ch6:quinticBezier}.
It is important to ensure that the path $\Gamma(\bar{s})$ interpolating the $L$ waypoints has $C2$ continuity.
To that end, the following constraints must be satisfied: 
\begin{itemize}
	\setlength\itemsep{0.2em}
	\item[\textbf{D1.}] $f_i(0) = p_i$ for $i={1,\cdots,L-1}$ and $f_{L-1}(1) = p_{L}$
	\item[\textbf{D2.}] $f_i^{'}(1) = f_{i+1}^{'}(0)$ for $i={1,\cdots,L-2}$
	\item[\textbf{D3.}] $f_i^{''}(1) = f_{i+1}^{''}(0)$ for $i={1,\cdots,L-2}$
	\item[\textbf{D4.}] $f_1^{'}(0) = \sigma^{'}_s$ and
	$f_1^{''}(0) = \sigma^{''}_s$
	\item[\textbf{D5.}] $f_{L-1}^{'}(1) = \sigma^{'}_e$ and
	$f_{L-1}^{''}(1) = \sigma^{''}_e$
\end{itemize}
where $\sigma^{'}_s$, $\sigma^{''}_s$, $\sigma^{'}_e$ and $\sigma^{''}_e$ are the first and second derivatives at the endpoints (i.e. $p_1$ and $p_L$).

Typically, two endpoints are needed two represent the path segment where the deformation will be applied.
In order to be able to deform that segment, at least one more intermediate waypoint is need which can be moved arbitrary to properly deform the considered path segment.
The more intermediate waypoints considered the more complex deformations could be applied.
For simplicity, we will consider the case with a single intermediate waypoint in addition to the endpoints such that $L=3$.
Therefore, only two Bezier splines are needed to connect these three waypoints.
The coefficients of these splines can be obtained based on the conditions \textbf{D1}-\textbf{D5}.
Let $p_i = (x_i,y_i,z_i)^T$, where $i={1,2,3}$, be the 3 waypoints representing the selected path segment.
Also, let the conditions on the endpoints $p_1$ and $p_3$ be defined as follows:
\begin{equation}
p_1^{'} = \sigma^{'}_s, \ \ p_1^{''} = \sigma^{''}_s, \ \ p_3^{'} = \sigma^{'}_e, \ \ p_3^{''} = \sigma^{''}_e
\end{equation} 
Now, the Bezier splines coefficients (i.e. control points) can be obtained in the following manner.
\begin{itemize}
	\item Conditions on the start position $p_1$ gives:
	 \begin{equation}\label{equ:ch6:CPs1}
	 \begin{aligned}
	 P_{0,1} &= p_1 \\
	 P_{1,1} &= \frac{1}{5} \sigma^{'}_s + P_{0,1} \\
	 P_{2,1} &= \frac{1}{20} \sigma^{''}_s + 2 P_{1,1} - P_{0,1}
	 \end{aligned} 
	 \end{equation}
 	\item Conditions on the final position $p_3$ yields:
 	\begin{equation}\label{equ:ch6:CPs2}
 	\begin{aligned}
 	P_{5,2} &= p_3 \\
 	P_{4,2} &= P_{5,2} - \frac{1}{5} \sigma^{'}_e \\
 	P_{3,2} &= \frac{1}{20} \sigma^{''}_e + 2 P_{4,2} - P_{5,2}
 	\end{aligned} 
 	\end{equation}
 	\item Satisfying continuity constraints at the intermediate point $p_2$ gives:
 	\begin{equation}\label{equ:ch6:CPs3}
 	\begin{aligned}
 	&P_{5,1} = P_{0,2} = p_2 \\[0.25cm]
 	&\left[\begin{array}{c}
 	P_{3,1} \\ P_{4,1} \\ P_{1,2} \\ P_{2,2}
 	\end{array}\right] = \left[\begin{array}{cccc}
 	0 & 1 & 1 & 0 \\
 	1 & -2 & 2 & -1 \\
 	2 & -2 & -2 & 2 \\
 	6 & -4 & 4 & -6
 	\end{array}\right]^{-1} D, \\[0.25cm]
 	&D := \left[\begin{array}{c}
 	P_{5,1} + P_{0,2} \\
 	-P_{5,1} + P_{0,2} \\
 	P_{2,1} - P_{5,1} - P_{0,2} + P_{3,2} \\
 	-P_{1,1} + 4 P_{2,1} - P_{5,1} + P_{0,2} - 4 P_{3,2} + P_{4,2} \\
 	\end{array}\right]
 	\end{aligned} 
 	\end{equation}
\end{itemize}

To sum up, equations \eqref{equ:ch6:fullPath}-\eqref{equ:ch6:CPs3} are used to compute the new deformed path segment where the two endpoints are stitched to the original path, and an intermediate midpoint is moved freely to produce safer paths.
A more detailed description about this deformation process is provided next.

\subsection{Path Deformation} \label{sec:ch6:deformation}

The suggested deformation process, which can produce deformations in arbitrary 3D directions around obstacles, is presented here.
Let $\Gamma_* \subset \mathcal{W}$ be a parametrized geometric reference path to be followed.
It can be obtained either from a global path planner, a straight path generator, or given by \eqref{equ:ch6:fullPath} after a deformation process.
During the motion, deformations are applied to $\Gamma_*$ at different segments whenever needed.
This deformation occurs in a "sense and avoid" manner whenever a sensed obstacle is found to make the current reference path $\Gamma_*$ unsafe.
That is, a deformation is triggered whenever the distance $d$ from obstacles to closest point on $\Gamma_*$ violates the safety constraint (i.e. $d < d_{safe}$).
Only a segment of $\Gamma_*$ starting from the vehicle's current position $p_v\ \in \R^3$ when it was triggered is deformed as shown in \cref{fig:ch6:illustration}. %
In practice, $p_v\ \in \R^3$ can be chosen as some point ahead of the vehicle's position to allow for any sensing and computational latencies.
Generally, the considered path segment can be defined by two endpoints $\bm{p}_s=\bm{\Gamma}_*(s_0),\ \bm{p}_e =\bm{\Gamma}_*(s_2)$ and some other point $\bm{p}_c=\bm{\Gamma}_*(s_1)$ in between such that $s_0 < s_1 < s_2$.
The point $p_c$ is used to control the deformation process.
In other words, this point acts like a handle that can be used to stretch the path in a certain direction changing its shape around the obstacle.
Manipulating $p_c$ depends on interpreted information from onboard sensors.
Thus, the deformation is performed by moving $\bm{p}_c$ in a safe direction $\bm{v}_{safe}\in\R^3$ resulting in a new smooth segment $\Gamma^{new}$ connecting $\bm{p}_s$, $\bm{p}_c^{new}$ and $\bm{p}_e$ such that:
\begin{equation}
	\bm{p}_c^{new} = \bm{p}_c + \bm{v}_{safe}
\end{equation}
Using advanced perception algorithms such as tracking nearby dynamic obstacles may help in making better deformations resulting in overall shorter paths which would be more optimal in many cases.
This can be seen from \cref{fig:ch6:illustration} which shows two possible deformed paths ($\Gamma_{*,1} = \tau_1$ and $\Gamma_{*,2} = \bar{\tau}_1$) generated by moving $p_c$ into some safe directions ($v_{safe}\ \in \R^3$ and $\bar{v}_{safe}\ \in \R^3$ respectively).

Different approaches could be adopted to determine a proper deformation.
In this work, we consider one way based only on the distance to obstacles edges.
Let $\vect{E}\ \in \R^3$ be a unit vector from $p_c$ in the direction towards the nearest edge, and let $\vect{T}\ \in \R^3$ be the unit tangent vector of $\Gamma_*$ at $p_c$.
A plane $\mathcal{P} \subset \R^3$ spanning both $\vect{T}$ and $\vect{E}$ vectors can be described by $p_c$ and its normal which is given by:
\begin{equation}
	\vect{N} = \vect{T} \times \vect{E}
\end{equation}
It is possible to find some vector $\vect{v}_{safe}$ by rotating $\vect{T}$ with an angle $\alpha$ around $\vect{N}$ towards or away from $\vect{E}$.
The decided direction of rotation relies on whether the obstacle is overlapping with $\Gamma_*$ or not.
Since $\vect{N}$ is used as a rotation axis, the following rotation matrix $R_{N}(\alpha)$ can be obtained based on the Rodrigues' rotation formula:
\begin{equation}\label{equ:ch6:rotation}
R_{N}(\alpha) = \cos(\alpha)\ \mathbf{I}_{3\times3} + \sin(\alpha) [\vect{N}]_{\times} + (1-\cos(\alpha)) (\vect{N} \otimes \vect{N})
\end{equation}
where $\mathbf{I}_{3\times3}$ is an identity matrix, and $[\vect{N}]_{\times}$ is a matrix related to the cross product which is given by:
\begin{equation*}
[\vect{N}]_{\times} = \left[\begin{array}{ccc}
0 & -N_z & N_y \\ N_z & 0 & -N_x \\ -N_y & N_x & 0
\end{array}\right]
\end{equation*}
Also, $\otimes$ represents the tensor product where $\vect{N} \otimes \vect{N} = \mathbf{N}\ \mathbf{N}^T$.
Thus, one can now use \eqref{equ:ch6:rotation} to compute $\vect{v}_{safe}$ according to the following:
\begin{equation}\label{equ:ch6:vsafe}
\vect{v}_{safe} = \gamma R_{N}(\alpha_0) \vect{T}
\end{equation}
where $\gamma>0$ is a safety factor which determines how far from the obstacle the deformed segment is at $p_c$.
Additionally, $\alpha_0$ is some desired rotation angle that affects the deformed segment shape, and its sign is determined by the right-hand rule convention with respect to $\vect{N}$.
To sum up, equation \eqref{equ:ch6:vsafe} corresponds to a vector of length $\gamma$ obtained through rotating $\vect{T}$ by $\alpha_0$ around the axis $\vect{N}$.

Typically, multiple real-time deformations may occur whenever it is become unsafe based on the sensors measurements.
A scenario with deformations at two different times during the motion is shown in \cref{fig:ch6:illustrationMultiple}.
It can be seen that the first time the path was found unsafe was when the vehicle reached $p_{v1}$.
The deformation direction was determined then based on sensed fraction of the obstacle to generate $\tau_2$.
Once the vehicle reached $p_{v2}$, a second deformation was triggered as a reaction to the newly sensed part of the obstacle to maintain the motion safety.

\begin{remark}
	It is assumed that vectors $\vect{T}$ and $\vect{E}$ are not collinear.
	However, if that case occurs, $\vect{N}$ can be chosen randomly.
	Another possible way is to chose it similar to the previously obtained $\vect{N}$ from last deformation which results in stretching the path further in the same direction increasing clearance from the obstacle.
\end{remark}

\begin{remark}
	More complex deformations can be applied by considering using more than one control point along the selected segment which gives more control over the deformed segment shape. 
	However, a proper way needs to be used in this case to determine how to manipulate these points in a way that guarantees increased clearance from obstacles.
\end{remark}

\begin{figure}[!htb]
	\centering
	\includegraphics[scale=0.45]{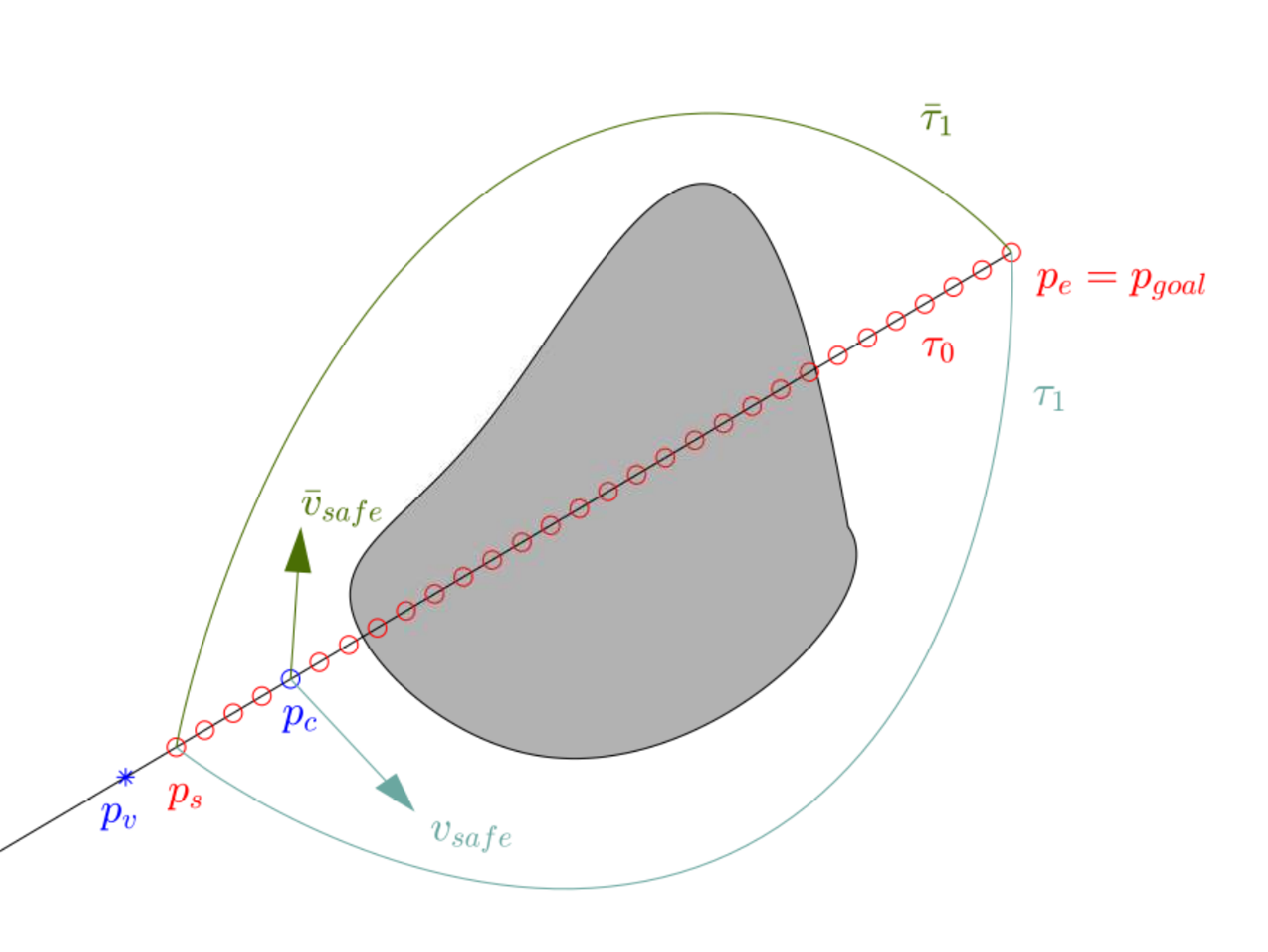} 
	\caption{Planar illustration of the suggested strategy} \label{fig:ch6:illustration}
\end{figure}

\begin{figure}[!htb]
	\centering
	\includegraphics[scale=0.45]{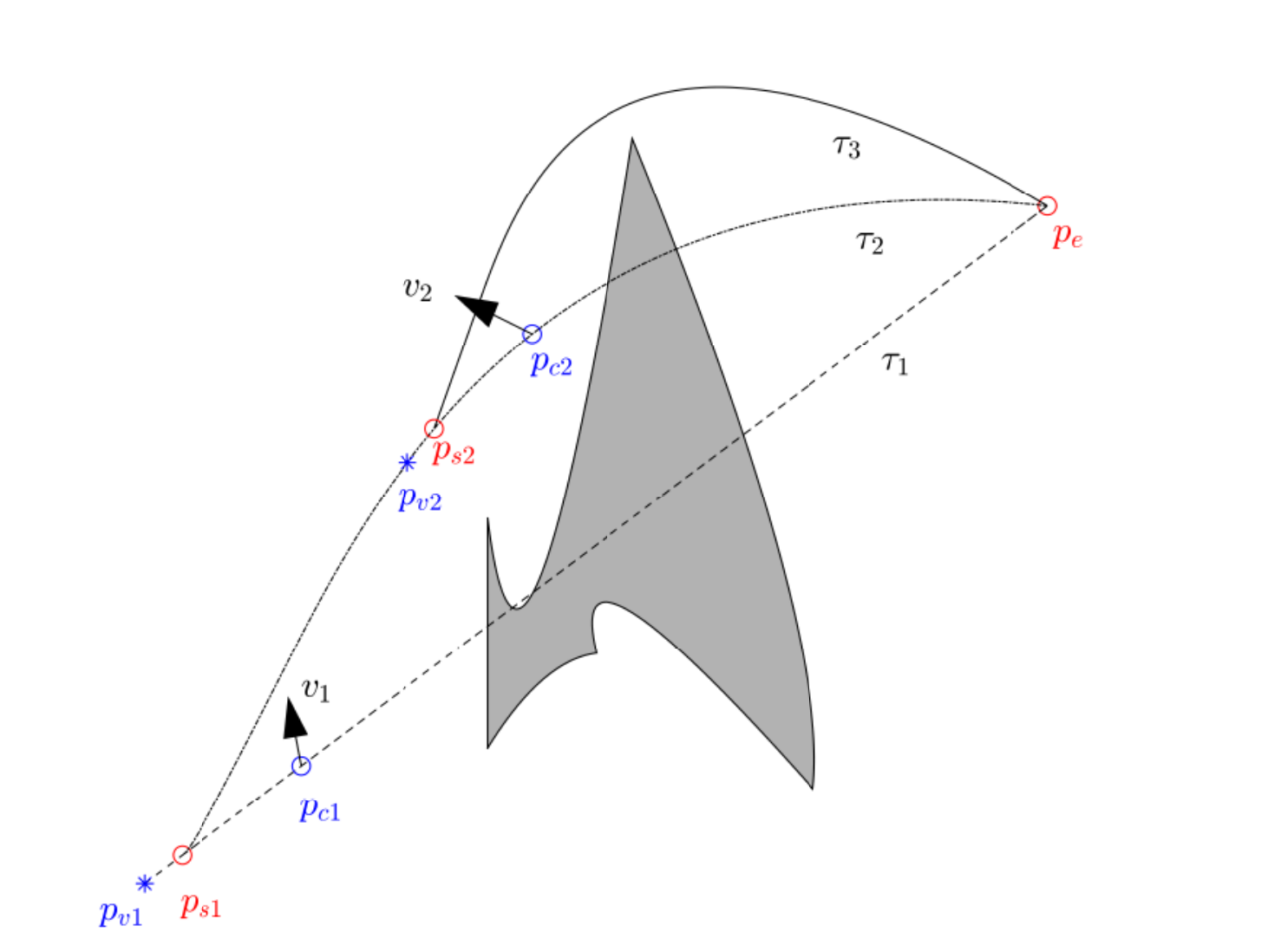} 
	\caption{Strategy illustration with multiple deformations} \label{fig:ch6:illustrationMultiple}
\end{figure}

\subsection{Low-level Control}

The low-level control component is responsible for generating control inputs (i.e. velocities) to track the deformable reference path $\Gamma$ at all times.
Different path-tracking control methods exists in the literature based on the considered motion models.
We provide control designs considering both the general kinematic model \eqref{equ:ch6:model} and the special case of quadrotor UAVs with the model \eqref{equ:ch6:model1}-\eqref{equ:ch6:model4}.

\subsubsection{General Kinematic Model Control}\label{sec:ch6:generalKinControl}

The adopted algorithm here is based on the pure pursuit guidance laws which can be summarized as follows.
At any given time, a virtual target moving on the reference path $\bm{p}^* \in \Gamma$ can be determined $L$ distance ahead of the vehicle's current position.
A desired velocity vector to move towards $\bm{p}^*$ is defined as $\bm{v}_d=[v_{x,d},\ v_{y,d},\ v_{z,d}]$.
The corresponding desired orientation angles following the model in \eqref{equ:ch6:model} are given by:
\begin{equation}
\begin{aligned}
\beta_d &= \tan^{-1}\left(\frac{v_{y,d}}{v_{x,d}}\right) \\
\alpha_d &= \tan^{-1}\left(\frac{v_{z,d}}{\sqrt{v_{x,d}^2+v_{y,d}^2}}\right)
\end{aligned}
\end{equation} 
To enforce the vehicle to track this moving target, we consider the following guidance law based on the sliding mode control method:
\begin{equation}\label{equ:ch6:guidanceLaw}
u_\mu = k_\mu \tanh\Big(c(\mu - \mu_d)\Big),\ \ \ \ \mu=\{\alpha,\beta\}
\end{equation}
where $k_\mu$ and $c$ are positive design parameters, and $V$ is set as constant.
In theory, the sliding mode control law would use a signum function.
However, the suggested control law in \eqref{equ:ch6:guidanceLaw} uses the hyperbolic tangent function $\tanh(\cdot)$ as a smooth approximation to avoid the well-known chattering problem in practice. 
The finite-time convergence of the system trajectories to achieve $v_d$ under the application of \eqref{equ:ch6:guidanceLaw} can be shown using Lyapunov stability analysis.

\subsubsection{Quadrotor Control}\label{sec:ch6:quadrotor_control}

There are different possible ways to extend the suggested approach to implemented with quadrotors.
One possible way is to directly map the guidance laws \eqref{equ:ch6:guidanceLaw} into velocity and acceleration commands which then can be used to compute the control inputs utilizing the differential-flatness property of quadrotor dynamics.
Another potential approach is to generate a feasible trajectory based on the geometric reference path satisfying the quadrator dynamic constraints.
Then, a controller can be designed to ensure that the trajectory can be tracked.
This approach is considered here, and it described next. 

\textbf{Trajectory Generation}\label{sec:ch6:traj_generation}

Whenever a deformation occurs, a new trajectory needs to be computed in a computationally-efficient way to ensure that the UAV can track $\Gamma^{new}$ while satisfying some constraints such as bounds on velocity and acceleration.
To that end, an reference model is used to generate smooth trajectories by extending the model \eqref{equ:ch6:model} as follows:
\begin{align}
	\dot{x}_r &= V_r \cos\alpha_r\cos\beta_r \label{equ:ch6:nhModel1} \\
	\dot{y}_r &= V_r \cos\alpha_R\sin\beta_r \label{equ:ch6:nhModel2}\\
	\dot{z}_r &= V_r \sin\alpha_r \label{equ:ch6:nhModel3} \\
	\dot{\alpha}_r &= \omega_{\alpha,r},\ \ \ \ \dot{\beta}_r = \omega_{\beta,r} \label{equ:ch6:nhModel4} \\
	\dot{\omega}_{\alpha,r} &= u_{\alpha,r},\ \ \  \dot{\omega}_{\beta,r} = u_{\beta,r} \label{equ:ch6:nhModel5}  \\
	\dot{V}_r & = a_r,\ \ \ \  \dot{a}_r = j_r \label{equ:ch6:nhModel6}
\end{align}
where $V_r \in [0,V_{max}]$ is the linear speed, $a_r$ is the linear acceleration, $\alpha_r$, $\beta_r$, $\omega_{\alpha,r}$ and $\omega_{\beta,r}$ are as defined in \eqref{equ:ch6:model}, and the control inputs are the jerk $j_r$ and angular accelerations $u_{\alpha,r}$ and $u_{\beta,r}$.
The model \eqref{equ:ch6:nhModel1}-\eqref{equ:ch6:nhModel6} is integrated forward in time every time a new trajectory is needed.
This can also be done in a receding horizon fashion where the model is only integrated for a short period ahead in time compensating for state-estimations errors. 
For a lower computational cost, the model \eqref{equ:ch6:model} can be used instead with a similar control design as proposed earlier for the trajectory generation.

We adopt a similar approach to \cref{sec:ch6:generalKinControl} for the reference input design suggested here.
Thus, we determine the closest position $\bm{p}_*=\bm{\Gamma^{new}}(s^*)$ on the path to $\bm{p}_r(t) = [x_r,y_r,z_r]^T$.
Then we compute a lookahead position as explained earlier such that $\|\bm{p}_L - \bm{p}_*\| > L$.
This can be used to obtain the following desired velocity vector:
\begin{equation}\label{equ:ch6:lookaheadError}
	\bm{v}_d(t) = \bm{p}_L(t) - \bm{p}_r(t) := [v_{x,d},\ v_{y,d},\ v_{z,d}]^T
\end{equation}
which corresponds to the following desired orientation:
\begin{equation}
	\begin{array}{cc}
		\alpha_{r,d} = \tan^{-1}\frac{v_{z,d}}{\sqrt{v_{x,d}^2 + v_{y,d}^2}}, & \beta_{r,d} = \tan^{-1}\frac{v_{y,d}}{v_{x,d}}
	\end{array}
\end{equation}
Finally, the following laws are used to compute the reference inputs for the model \eqref{equ:ch6:nhModel1}-\eqref{equ:ch6:nhModel6} to ensure that \eqref{equ:ch6:lookaheadError} can be achieved:
\begin{equation}
	\begin{aligned}
		\sigma_\mu(t) &= \dot{\mu}(t) + c_{\mu,1} \tanh(\gamma_1(\mu(t) - \mu_d(t))) \\
		u_\mu (t)     &= - c_{\mu,2} \tanh(\gamma_2\sigma_\mu(t))
	\end{aligned}
\end{equation}
where $\mu = \{V_r, \alpha_r, \beta_r\}$, $c_{\mu,1}$, $c_{\mu,2}$ and $\gamma_1,\gamma_2$ are positive design parameters, $\mu_d = \{V_d, \alpha_{r,d}, \beta_{r,d}\}$ are the desired values for some desired speed $V_d>0$, and $u_\mu = \{j_r, u_\alpha, u_\beta\}$ are the reference control inputs.
This design provides an easy way to tune the gains while satisfy limits on velocities and accelerations.

\textbf{Trajectory Tracking Control}

For the quadrotor dynamical model \eqref{equ:ch6:model1}-\eqref{equ:ch6:model4}, we design control laws based on the sliding mode control method and the differential-flatness property of the quadrotor dynamics.
The control objective is to track a smooth position trajectory $\bm{p}_r(t) = [x_r(t),\ y_r(t),\ z_r(t)]^T$ with bounded derivatives and a reference yaw trajectory $\psi_r(t)$.
Let the position and velocity tracking errors be defined as:
\begin{equation}\label{equ:ch6:errors}
	\vect{e}_{\vect{p}}(t) = \vect{p}_r(t) - \vect{p}(t),\ \vect{e}_{\vect{v}}(t) = \vect{\dot{p}}_r(t) - \vect{v}(t)
\end{equation}
By regarding the acceleration a virtual input, we adopt the sliding mode control technique to obtain the following acceleration command:
\begin{align}
	\bm{\sigma}(t) &= \vect{e}_{\vect{v}}(t) + \vect{C}_1 \tanh(\gamma_1\vect{e}_{\vect{p}}(t)) \label{equ:ch6:acc_cmd1} \\
	\vect{a}_{cmd}(t) &= \vect{C}_2 \tanh\left(\gamma_2 \bm{\sigma}(t)\right) + \vect{\ddot{p}}_{r}(t) + g \bm{e}_3 \label{equ:ch6:acc_cmd2}
\end{align}
where $\bm{\sigma}(t)\in\R^3$ is the sliding surface, $\vect{C}_1$ and $\vect{C}_2$ are positive-definite $3\times 3$ diagonal matrices, $\gamma_1,\gamma_2\in\R^{+}$, and $\tanh(\bm{b})$ is the element-wise tangent hyperbolic function of a vector $\bm{b}\in\R^3$.
Note that we replaced the signum function commonly used in the sliding mode reaching law with a smooth approximation (i.e. the hyperbolic tangent function).
Thus, using the control law \eqref{equ:ch6:acc_cmd2} will ensure that the system states converge to the sliding surface $\bm{\sigma}(t) = \bm{0}$.
Furthermore, once the sliding surface is reached, the proposed design in \eqref{equ:ch6:acc_cmd1} guarantees that the position and velocity tracking errors will asymptotically converge to zero (i.e. $\lim\limits_{t \to \infty} \bm{e}_p = \bm{0}$ and $\lim\limits_{t \to \infty} \bm{e}_v = \bm{0}$).

The next step is to determine the proper input thrust $T(t)$ and desired UAV orientation $\bm{R}_{des}(t)$ to achieve the acceleration in \eqref{equ:ch6:acc_cmd2}.
Ought to the differential flatness property of quadrotor dynamics, this can be done using the following equations \cite{mellinger2011minimum}:
\begin{equation}\label{equ:ch6:thrustInput}
	T(t) = \vect{a}_{cmd}^T(t) \vect{R}(t) \vect{e}_3
\end{equation}
and, 
\begin{align}
	\vect{R}_{des}(t) &= [\vect{x}_{\mathcal{B},des}(t),\ \vect{y}_{\mathcal{B},des}(t),\ \vect{z}_{\mathcal{B},des}(t)] \\
	\vect{z}_{\mathcal{B},des}(t) &= \frac{\vect{a}_{cmd}(t)}{\|\vect{a}_{cmd}(t)\|} \\
	\vect{y}_{\mathcal{B},des}(t) &= \frac{ \vect{z}_{\mathcal{B},des}(t) \times \vect{x}_{C}(t) }{\|\vect{z}_{\mathcal{B},des}(t) \times \vect{x}_{C}(t)\|} \\
	\vect{x}_{\mathcal{B},des}(t) &= \vect{y}_{\mathcal{B},des}(t) \times \vect{z}_{\mathcal{B},des}(t)
\end{align}
where $\vect{x}_{C}(t)$ is given by:
\begin{equation}\label{equ:ch6:yCref}
	\vect{x}_{C}(t) = [\cos\psi_{r}(t),\ \sin\psi_{r}(t),\ 0]^T.
\end{equation}
Moreover, the body torques control inputs $\bm{\tau}(t)$ need to be designed to ensure that the vehicle's attitude can track the desired attitude $\vect{R}_{des}(t)$ (for example, as in \cite{mellinger2011minimum}).

\section{Simulation Results}\label{sec:ch6:simulation}

\subsection{General Model Simulation Cases}

Simulations were carried out to validate the performance of the suggested approach considering different static and dynamic cases.
This section shows the results obtained when applying the proposed control laws \eqref{equ:ch6:guidanceLaw} for the general 3D kinematic model \eqref{equ:ch6:model} using MATLAB. 
Simulations step time was taken to be 0.1s, and the constant forward speed was set as $V(t) = 1\ m/s, \ t > 0$.

The first simulation scenario considers a static environment, and the results are shown in \cref{fig:ch6:simStatic1,fig:ch6:simStatic2,fig:ch6:simStatic3}.
The environment consists of a set of cylindrical obstacles can represent bounding objects for trees, pillars, people, etc.
Using bounding shapes to represent obstacles reduces the complexity of collision checking algorithms.
However, the proposed approach does not really add restriction on obstacles shapes as long as a computationally-efficient algorithm can be used to detect collisions with obstacles in real-time.
The initial location of the vehicle was chosen as $p_i = [1,1,3]^T$, and the vehicle was required to safely navigate to a goal location at $p_{goal}=[20,20,20]^T$.
A safety factor of $\gamma = 0.6$ was considered in this simulation which determines the length of $\vect{v}_{safe}$.
A planar view of the vehicle's motion at different time instants is shown in \cref{fig:ch6:simStatic1}.
Starting from a straight path towards the goal, multiple deformations are applied to the path whenever an obstacle is detected (i.e. when the vehicle becomes close enough to sense part of it).
A small value for $\vect{v}_{safe}$ was used to make sure that deformations do not produce longer paths.
However, you can see that multiple deformations are also needed around a single obstacle until the deformed path becomes safe as in \cref{fig:ch6:simStatic1}(a)-(d).
You can also notice that the direction of deformation could be different depending on the closest obstacle edge as you can see from \cref{fig:ch6:simStatic1}(e)-(f).
The complete trajectory is shown \cref{fig:ch6:simStatic2} showing the 3D motion.
These results clearly shows how well the proposed method works in avoiding static obstacles.
Furthermore, the minimum distance to obstacles during the motion is given in \cref{fig:ch6:simStatic3} which shows how the vehicle can keeps a good clearance from the obstacles above the safety margin $d_{safe}=0.5m$.

The second scenario deals with dynamic obstacles where two cases were considered.
In these cases, moving spherical obstacles were deployed such that they intercept the vehicle's motion when going towards the goal from different directions as in \cref{fig:ch6:simDynamic1,fig:ch6:simDynamic2}.
Note that the obstacles velocities $\bm{v}_{o,i}(t),\ i=\{1,2\}$ where chosen such that $\|\bm{v}_{o,i}(t)\| < V$.
Also, the safety factor $\gamma$ was chosen differently in the two cases as 1.5 and 2.5 respectively. 
This was done to show that smaller values of $\gamma$ produce obstacle avoidance manoeuvrers closer to the obstacle while avoidance with better clearance can be obtained using larger values at the expense of having longer deformed paths.
An adaptive approach can also be considered to determine a good value for $\gamma$ depending on the distance between the current path and the obstacle rather than using a constant value if found more efficient.
The smaller value of $\gamma$ in the first case caused the need for several deformations to maintain safety as the obstacle was approaching the vehicle which can be seen from ~\cref{fig:ch6:simDynamic1}(c) and \cref{fig:ch6:simDynamic1}(d).
Choosing a larger value for $\gamma$ or performing a deformation away from the motion direction of the obstacle would have been better in that case to determine $\vect{v}_{safe}$ in a way that can result in better manoeuvrers.
In this case, an estimate of the obstacle's velocity needs to be computed.
However, the proposed approach still managed to guarantee the collision avoidance while maintaining a distance larger than the safety margin.
The second case with a larger safety factor is shown in \cref{fig:ch6:simDynamic2} which clearly verifies that the vehicle can safely avoid the dynamic obstacle.
\Cref{fig:ch6:simDynamicDt} shows how close the distance to the moving obstacle from the safety margin when using a smaller value of $\gamma$.
Overall, the obtained results confirms the performance of our method in static and dynamic unknown environments.

\begin{figure}[!htb]
	\centering
	\begin{adjustbox}{minipage=\linewidth,scale=0.9}
		\begin{subfigure}[t]{0.45\textwidth}
			\centering
			\fbox{\includegraphics[width=\linewidth]{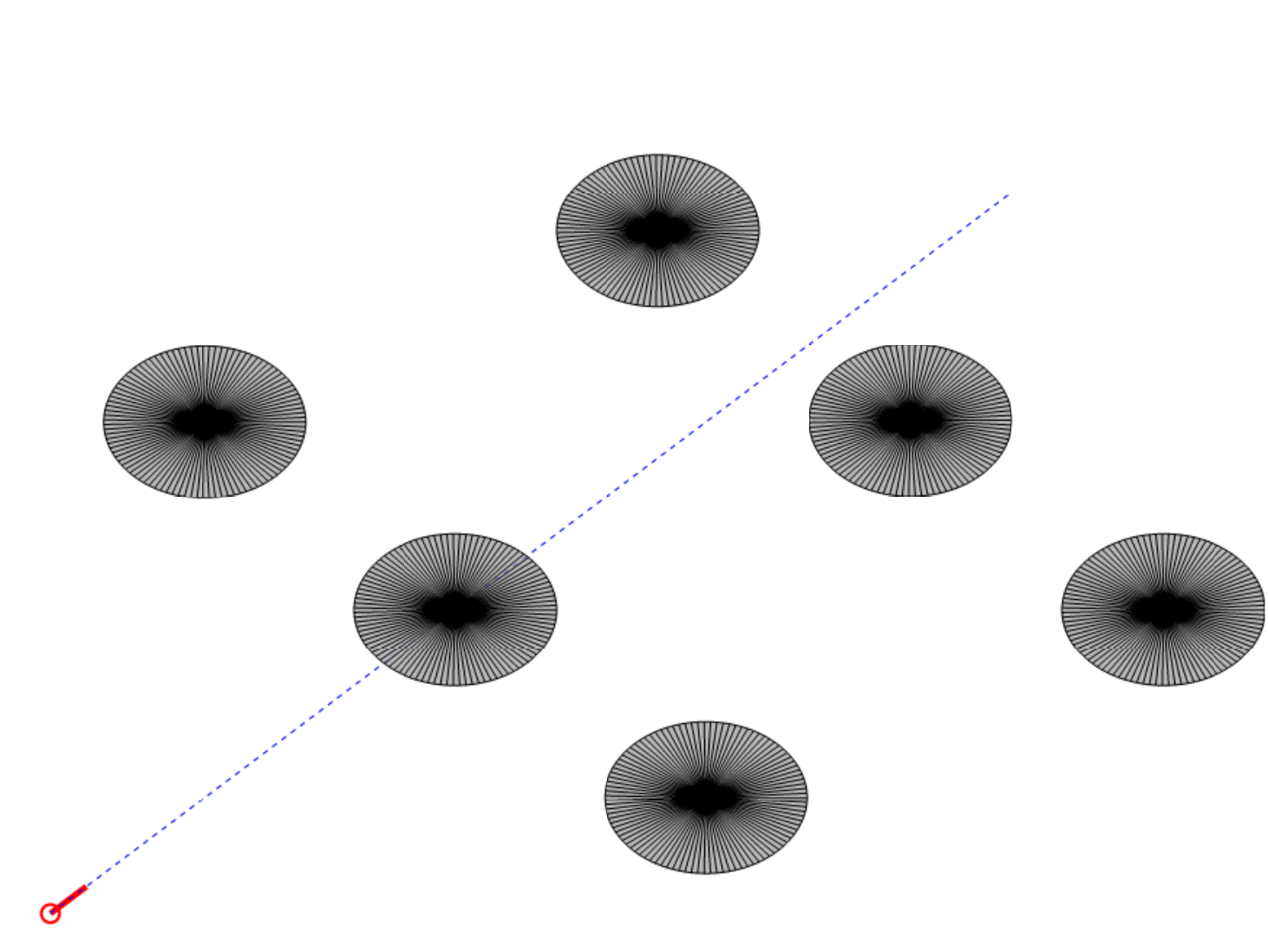}}
			\caption{}
		\end{subfigure}
		\hfill
		\begin{subfigure}[t]{0.45\textwidth}
			\centering
			\fbox{\includegraphics[width=\linewidth]{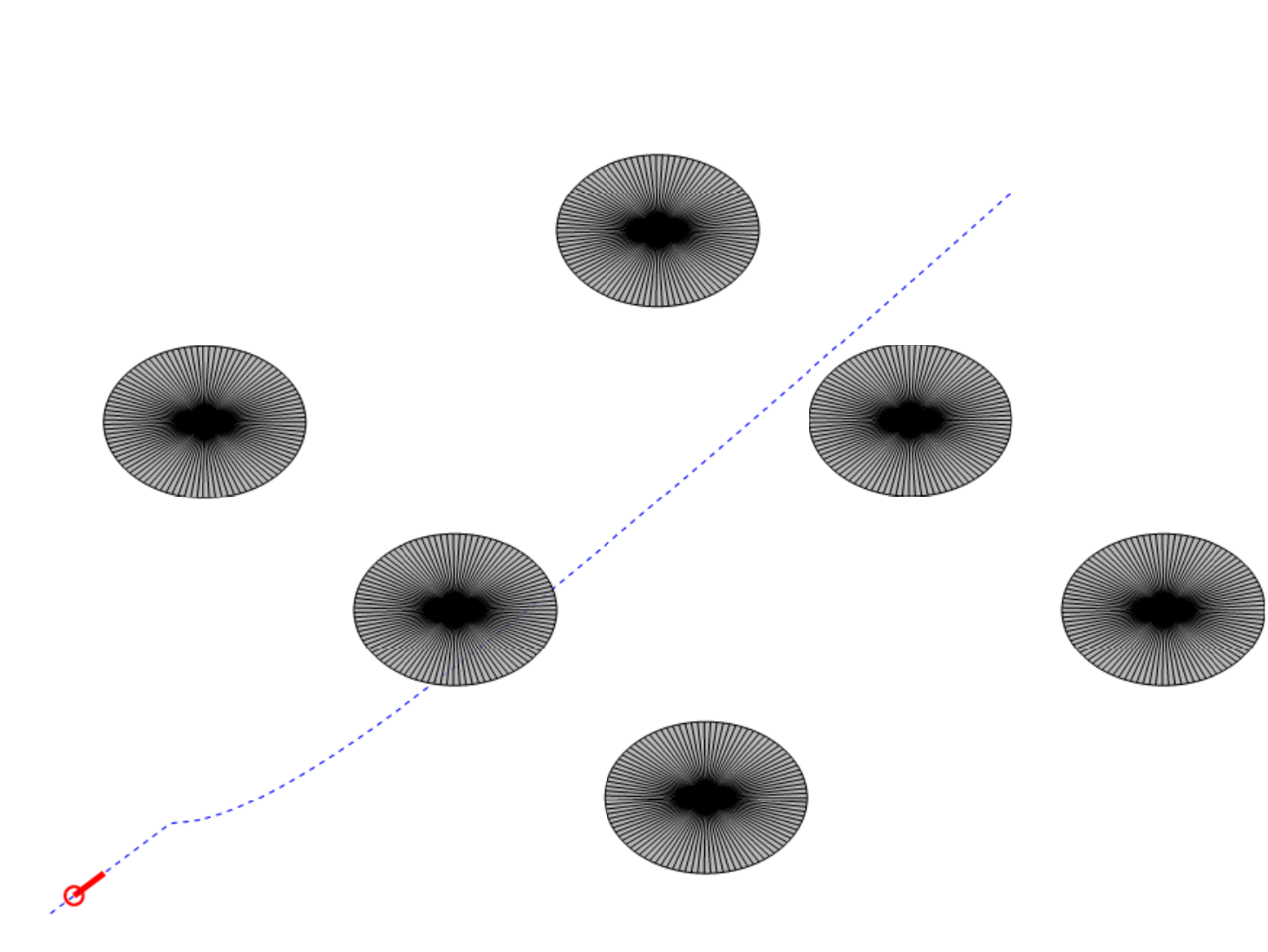}}
			\caption{}
		\end{subfigure}
		
		\begin{subfigure}[t]{0.45\textwidth}
			\centering
			\fbox{\includegraphics[width=\linewidth]{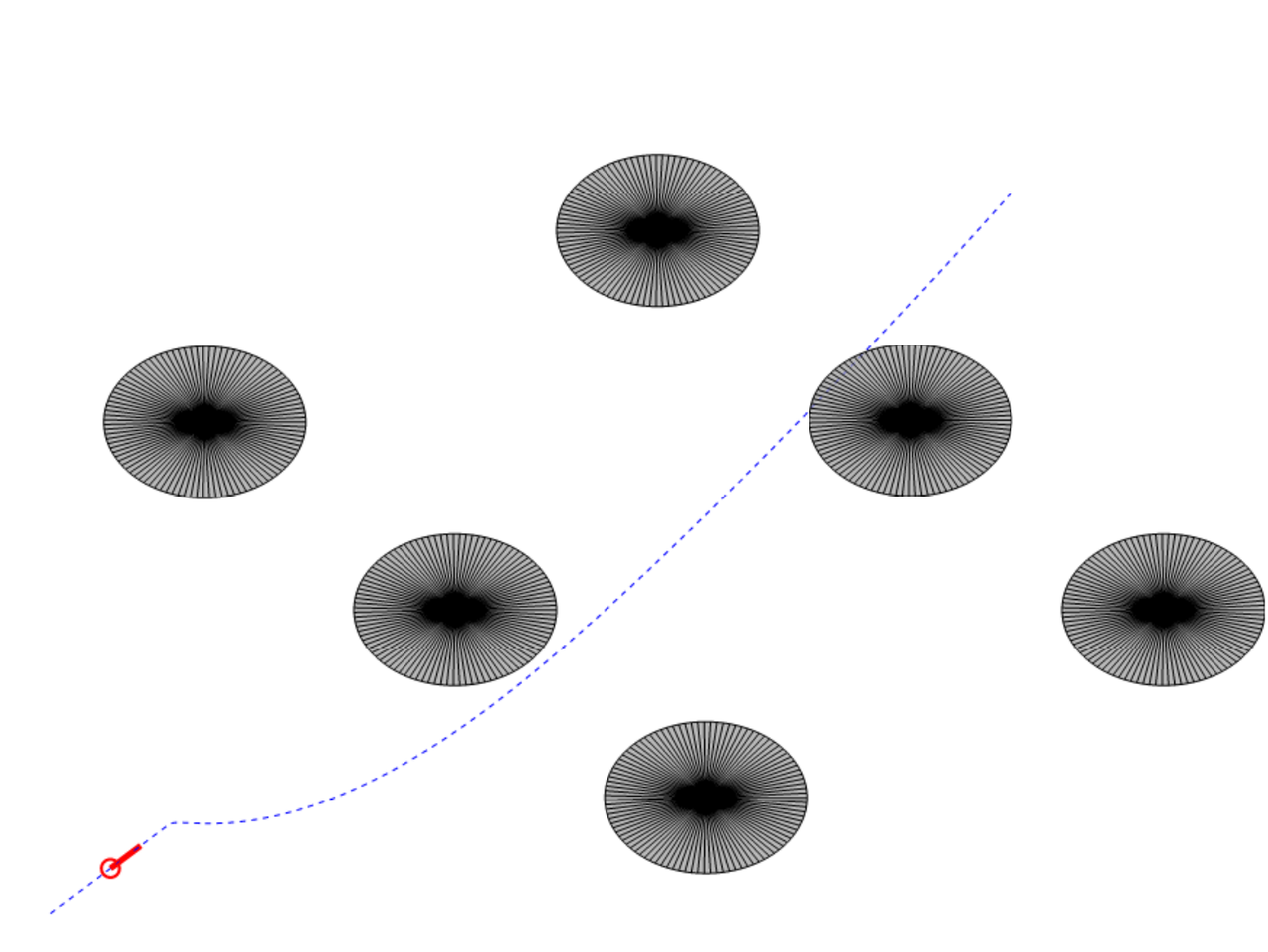}}
			\caption{}
		\end{subfigure}
		\hfill
		\begin{subfigure}[t]{0.45\textwidth}
			\centering
			\fbox{\includegraphics[width=\linewidth]{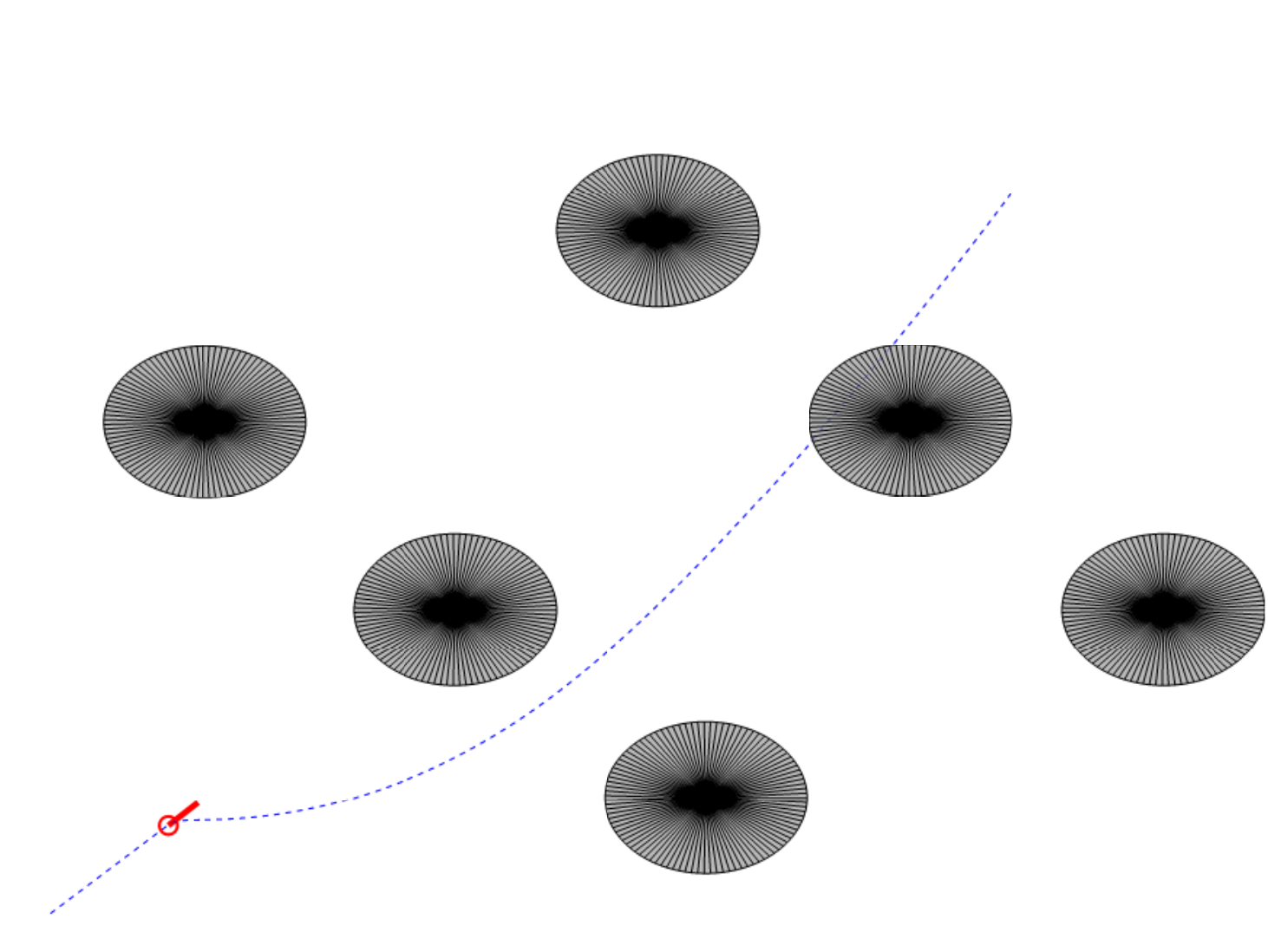}}
			\caption{}
		\end{subfigure}
		
		\begin{subfigure}[t]{0.45\textwidth}
			\centering
			\fbox{\includegraphics[width=\linewidth]{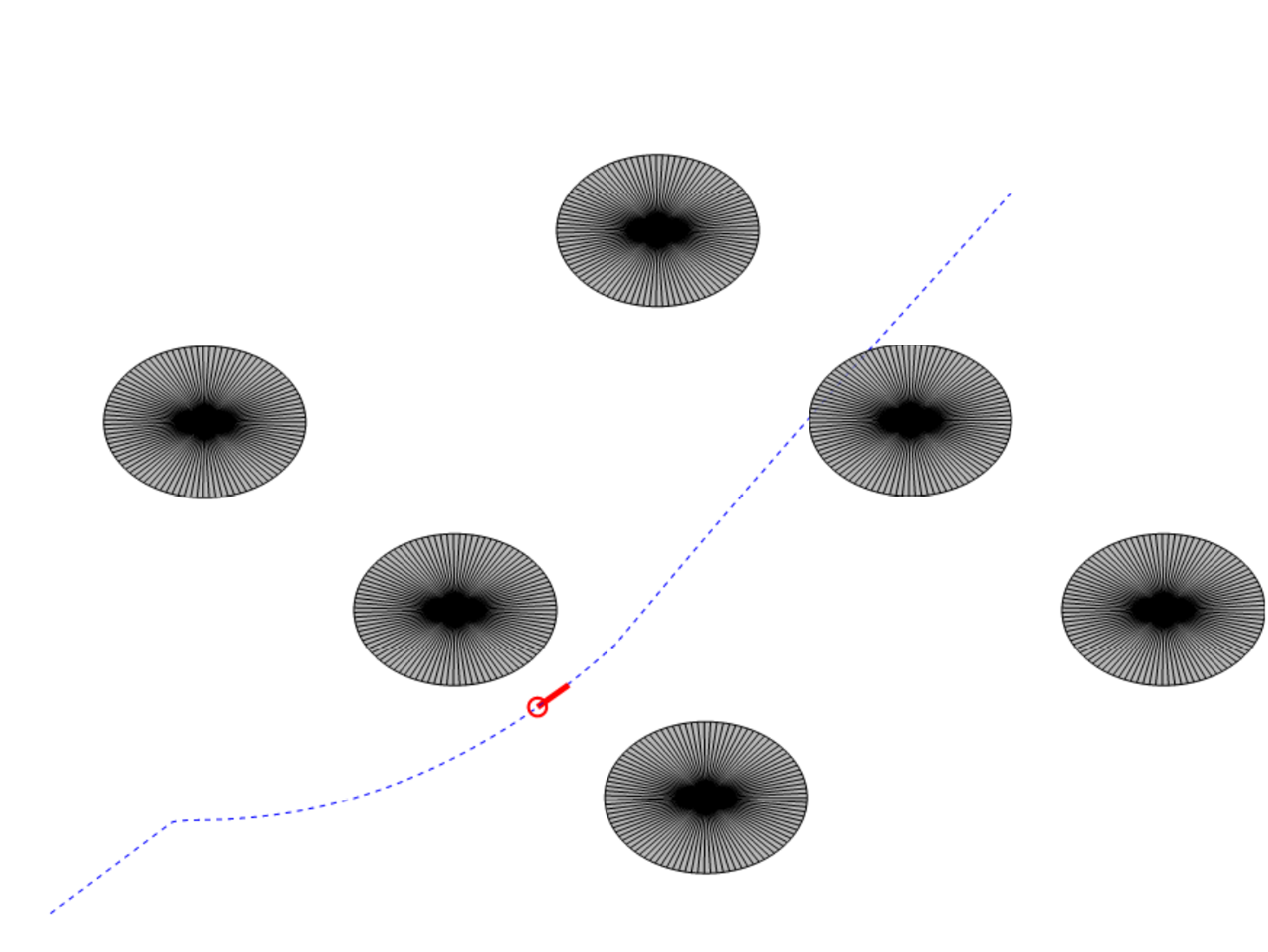}}
			\caption{}
		\end{subfigure}
		\hfill
		\begin{subfigure}[t]{0.45\textwidth}
			\centering
			\fbox{\includegraphics[width=\linewidth]{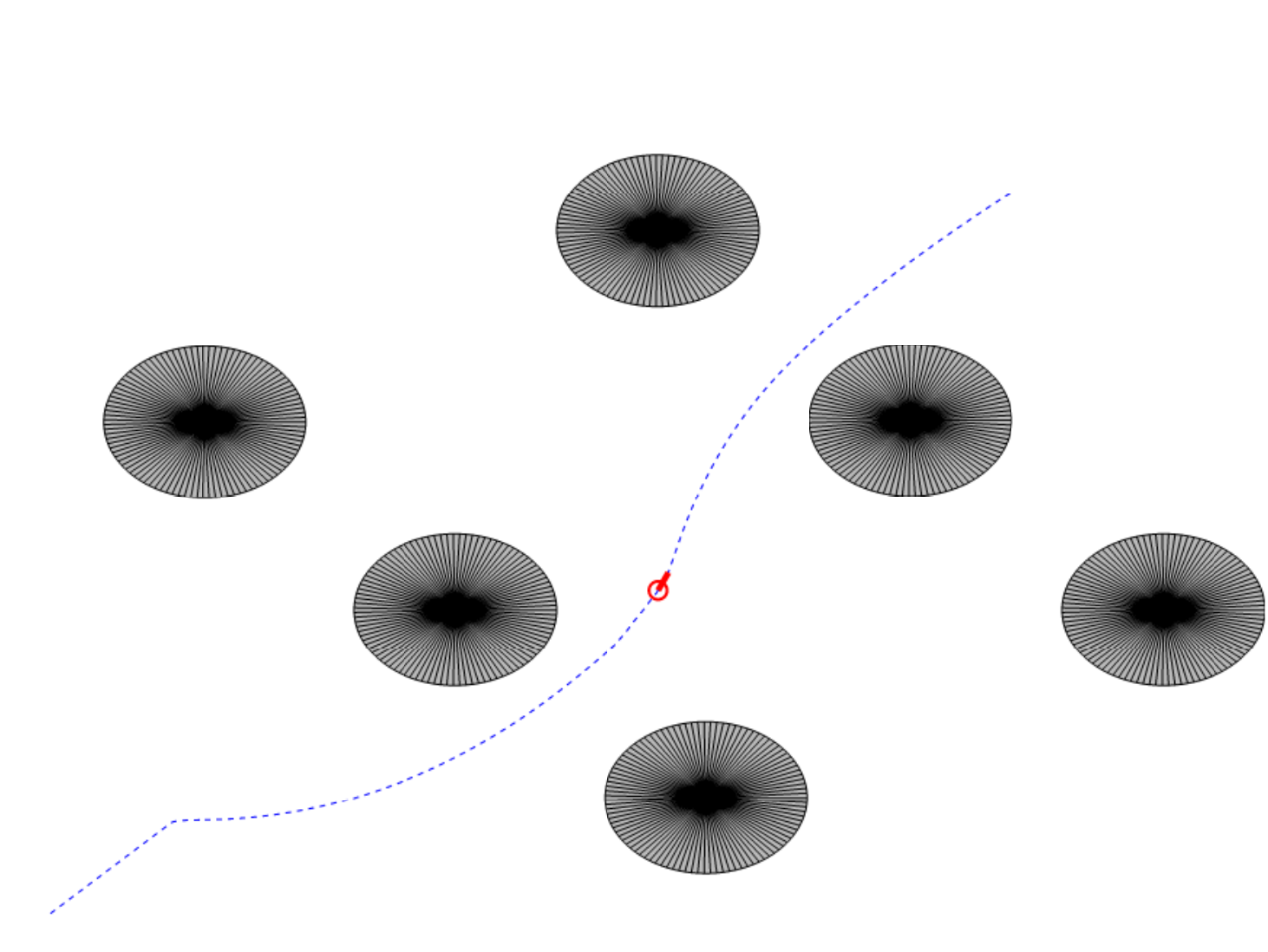}}
			\caption{}
		\end{subfigure}
		
		\caption{Different instances during motion at which segments of the path are being deformed (shown from top view)}\label{fig:ch6:simStatic1}
	\end{adjustbox}
\end{figure}

\begin{figure}[!htb]
	\centering
	\begin{adjustbox}{minipage=\linewidth,scale=0.9}
		\begin{subfigure}[t]{\textwidth}
			\centering
			\includegraphics[width=\linewidth]{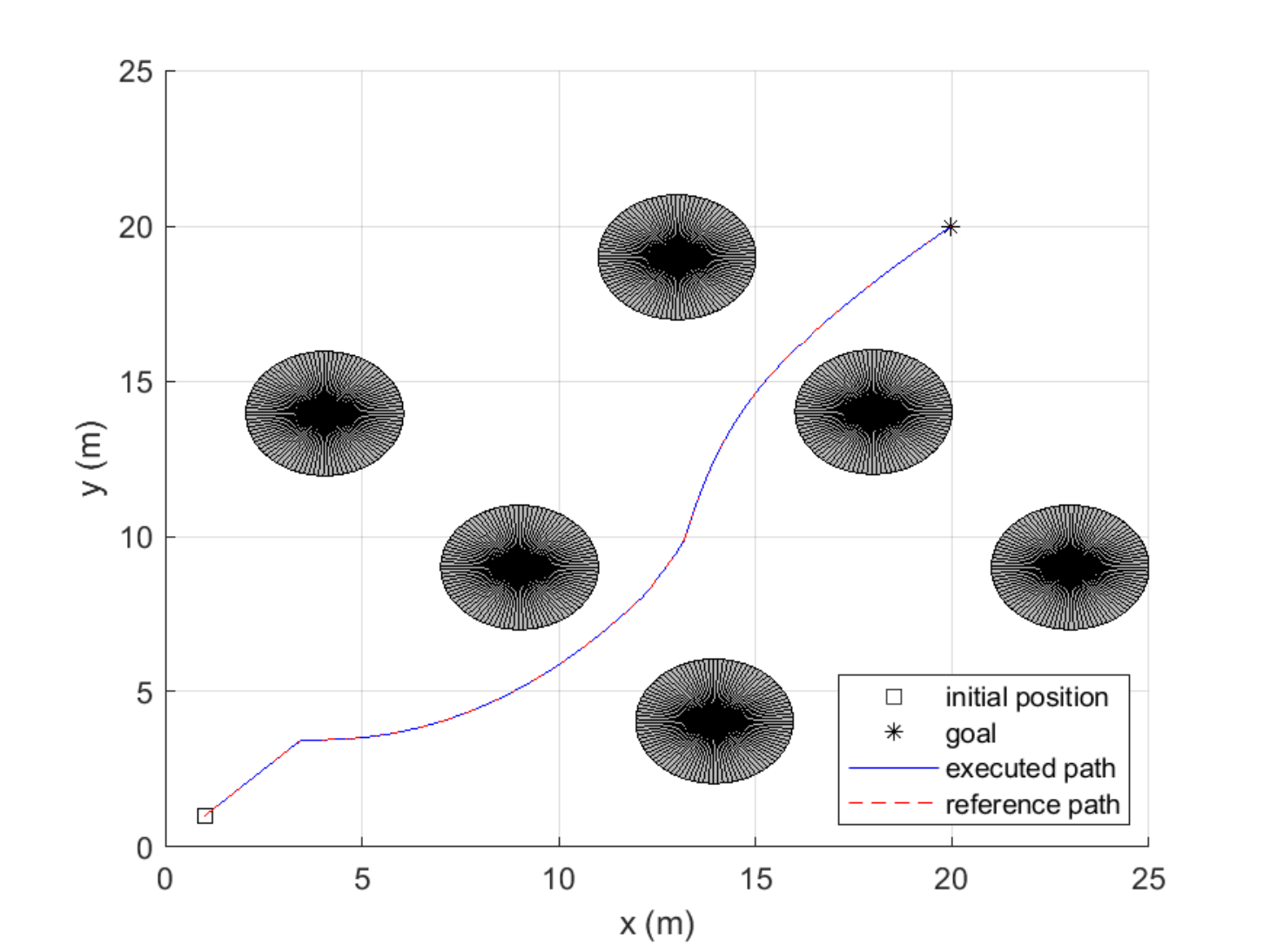} 
			\caption{Top view}
		\end{subfigure}
		
		\begin{subfigure}[t]{\textwidth}
			\centering
			\includegraphics[width=\linewidth]{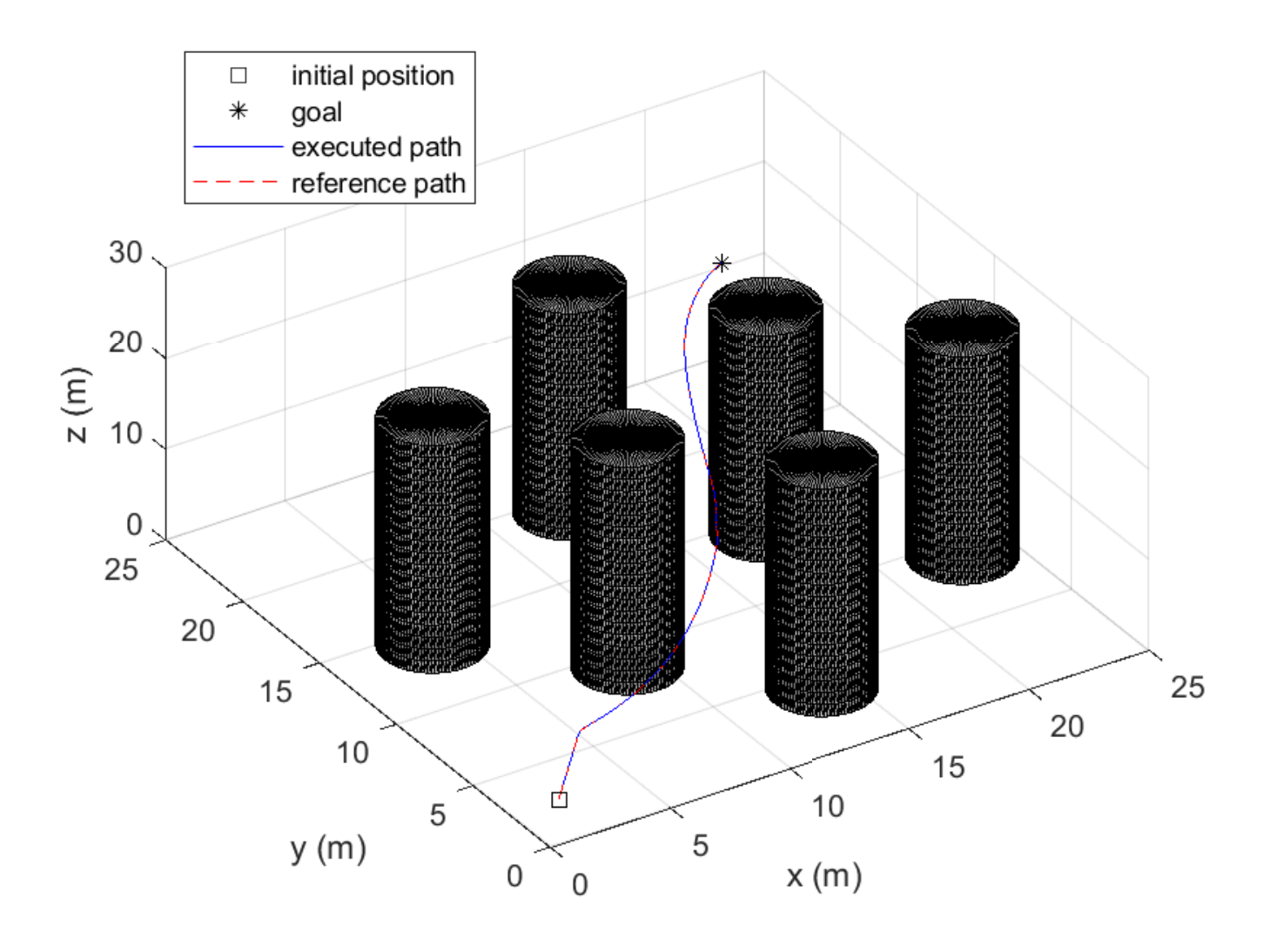} 
			\caption{3D view}
		\end{subfigure}
		
		\caption{Simulation results: executed Path with different views}\label{fig:ch6:simStatic2}
	\end{adjustbox}
\end{figure}

\begin{figure}[!htb]
	\centering
	\includegraphics[width=0.9\linewidth]{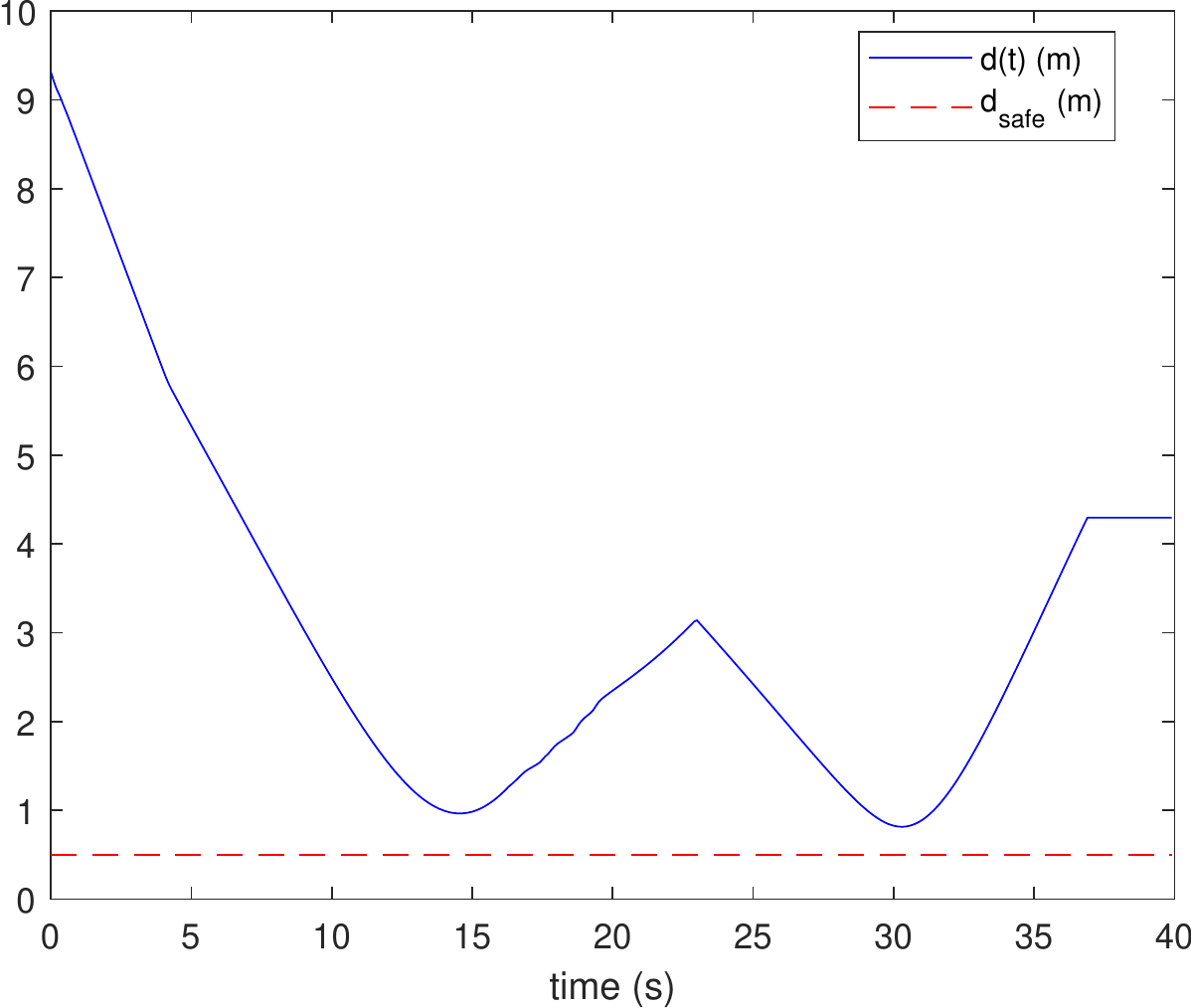} 
	\caption{Simulation results: vehicle's relative distance to nearest obstacle during the motion ($d(t)$ versus time)} \label{fig:ch6:simStatic3}
\end{figure}

\begin{figure}[!htb]
	\centering
	\begin{adjustbox}{minipage=\linewidth,scale=0.9}
		\begin{subfigure}[t]{0.48\textwidth}
			\centering
			\includegraphics[width=\linewidth]{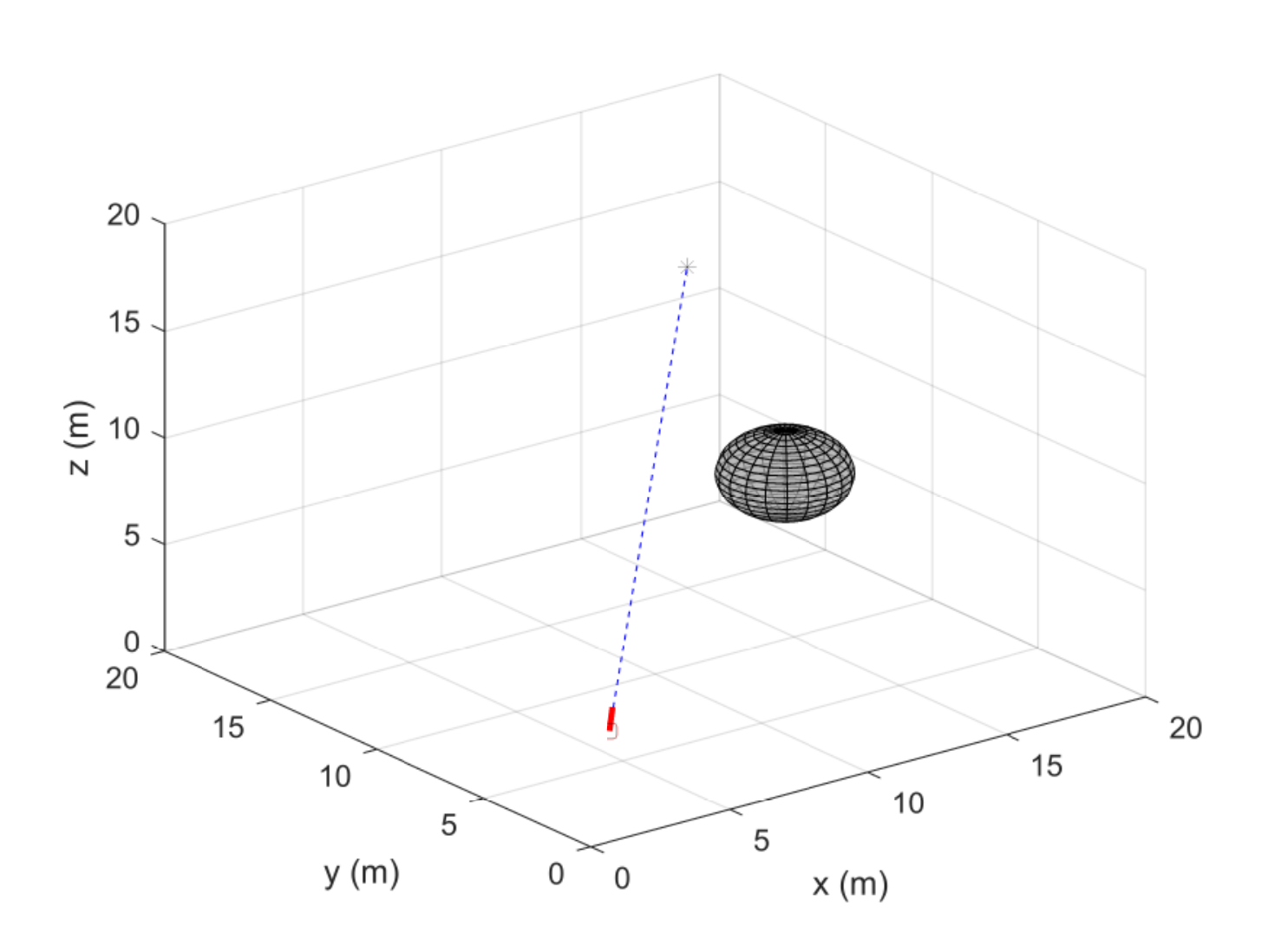} 
			\caption{}
		\end{subfigure}
		\hfill
		\begin{subfigure}[t]{0.48\textwidth}
			\centering
			\includegraphics[width=\linewidth]{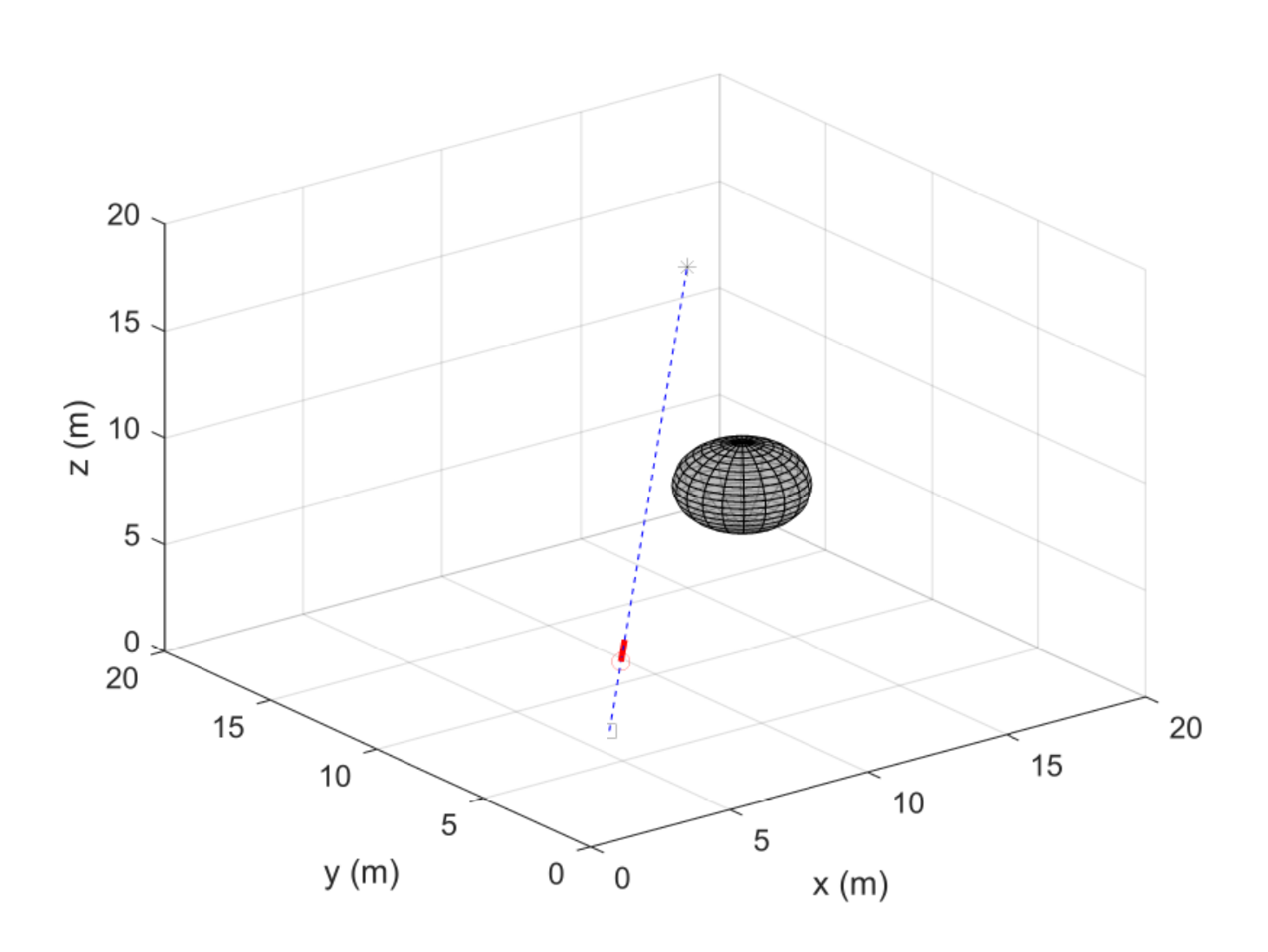} 
			\caption{}
		\end{subfigure}
		
		\begin{subfigure}[t]{0.48\textwidth}
			\centering
			\includegraphics[width=\linewidth]{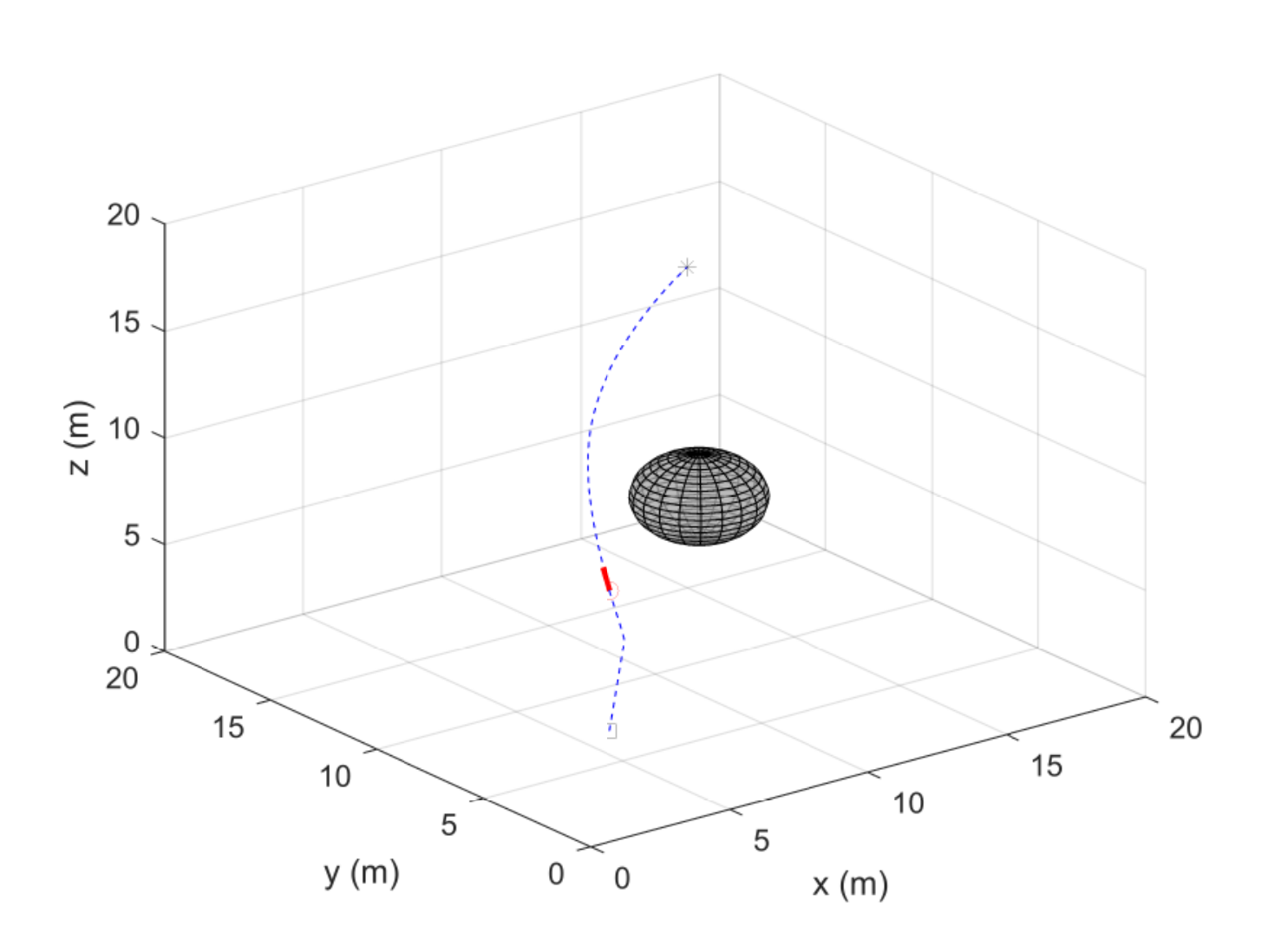} 
			\caption{}
		\end{subfigure}
		\hfill
		\begin{subfigure}[t]{0.48\textwidth}
			\centering
			\includegraphics[width=\linewidth]{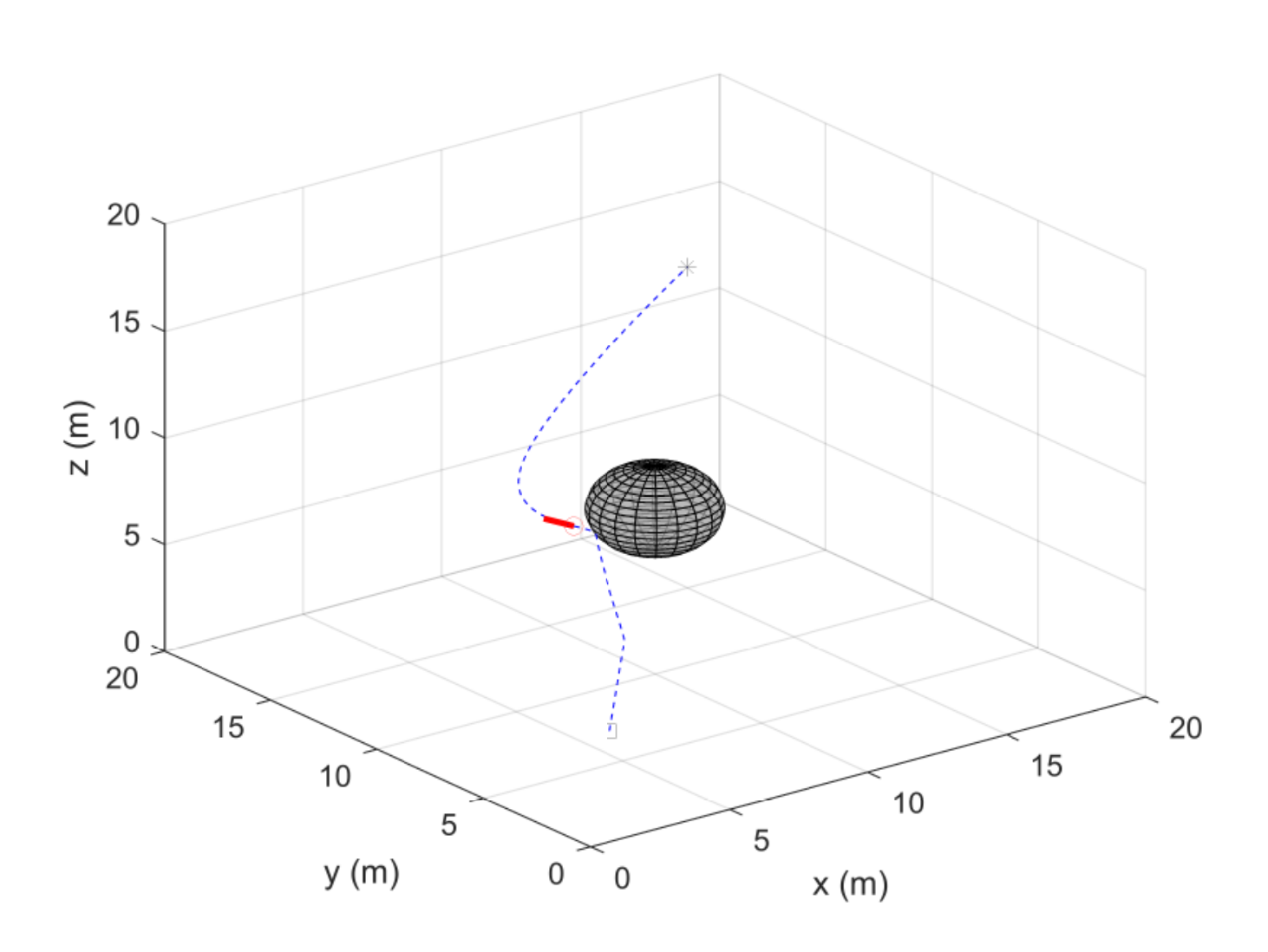} 
			\caption{}
		\end{subfigure}
		
		\begin{subfigure}[t]{0.48\textwidth}
			\centering
			\includegraphics[width=\linewidth]{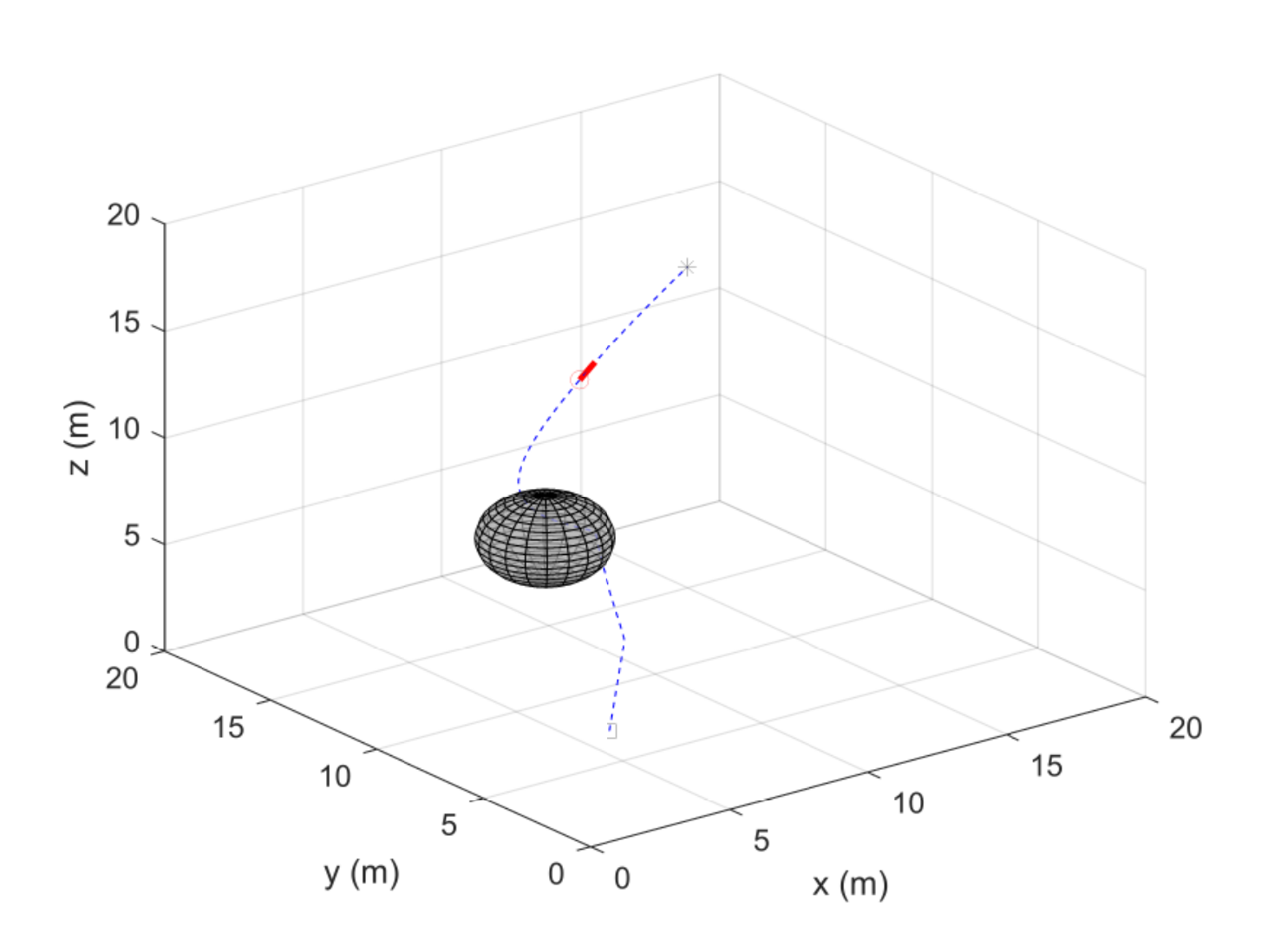} 
			\caption{}
		\end{subfigure}
		\hfill
		\begin{subfigure}[t]{0.48\textwidth}
			\centering
			\includegraphics[width=\linewidth]{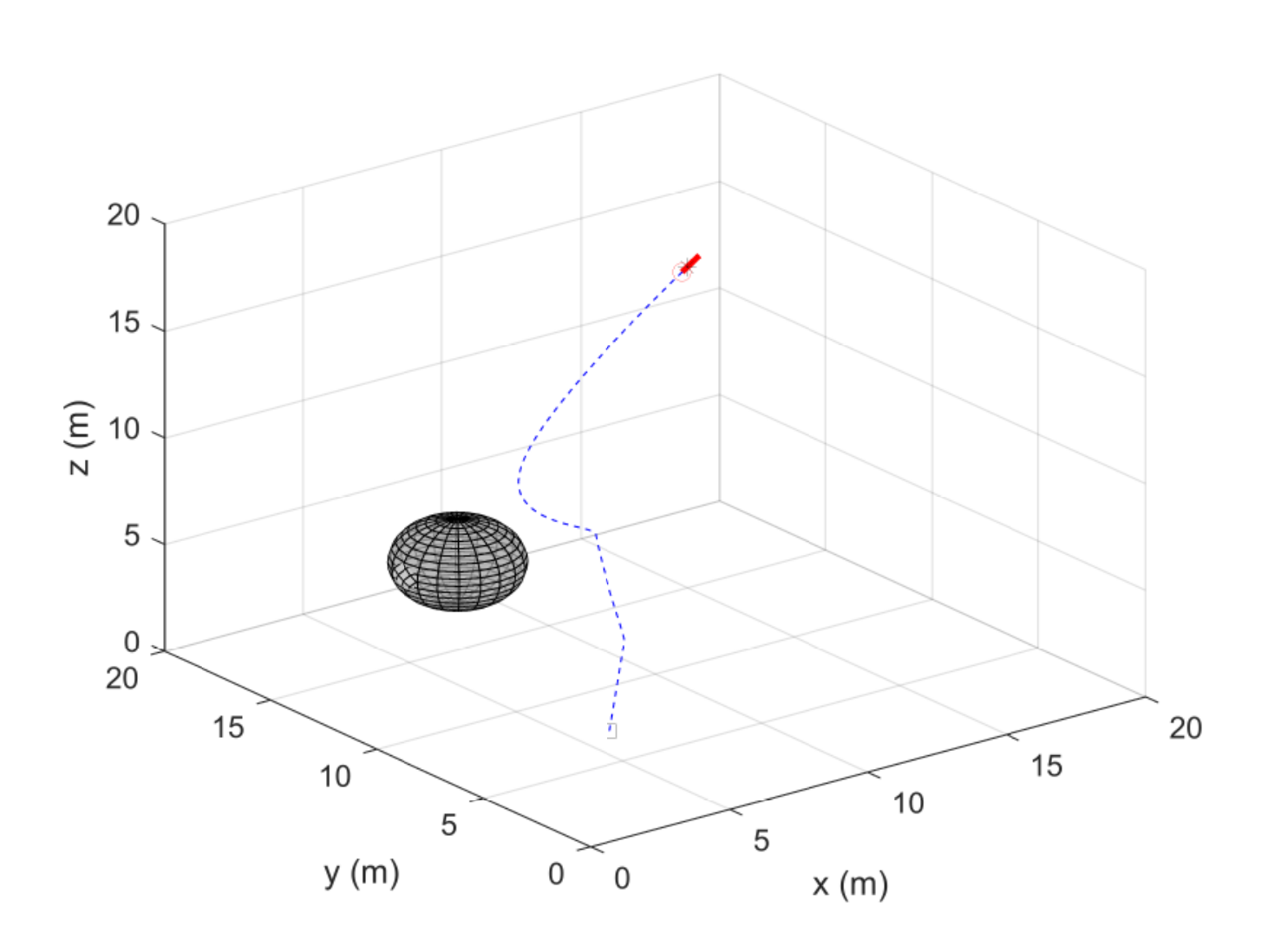} 
			\caption{}
		\end{subfigure}
		
		\caption{Simulation results: strategy validation with dynamic obstacles (Case 1)}\label{fig:ch6:simDynamic1}
	\end{adjustbox}
\end{figure}

\begin{figure}[!htb]
	\centering
	\begin{adjustbox}{minipage=\linewidth,scale=0.9}
		\begin{subfigure}[t]{0.48\textwidth}
			\centering
			\includegraphics[width=\linewidth]{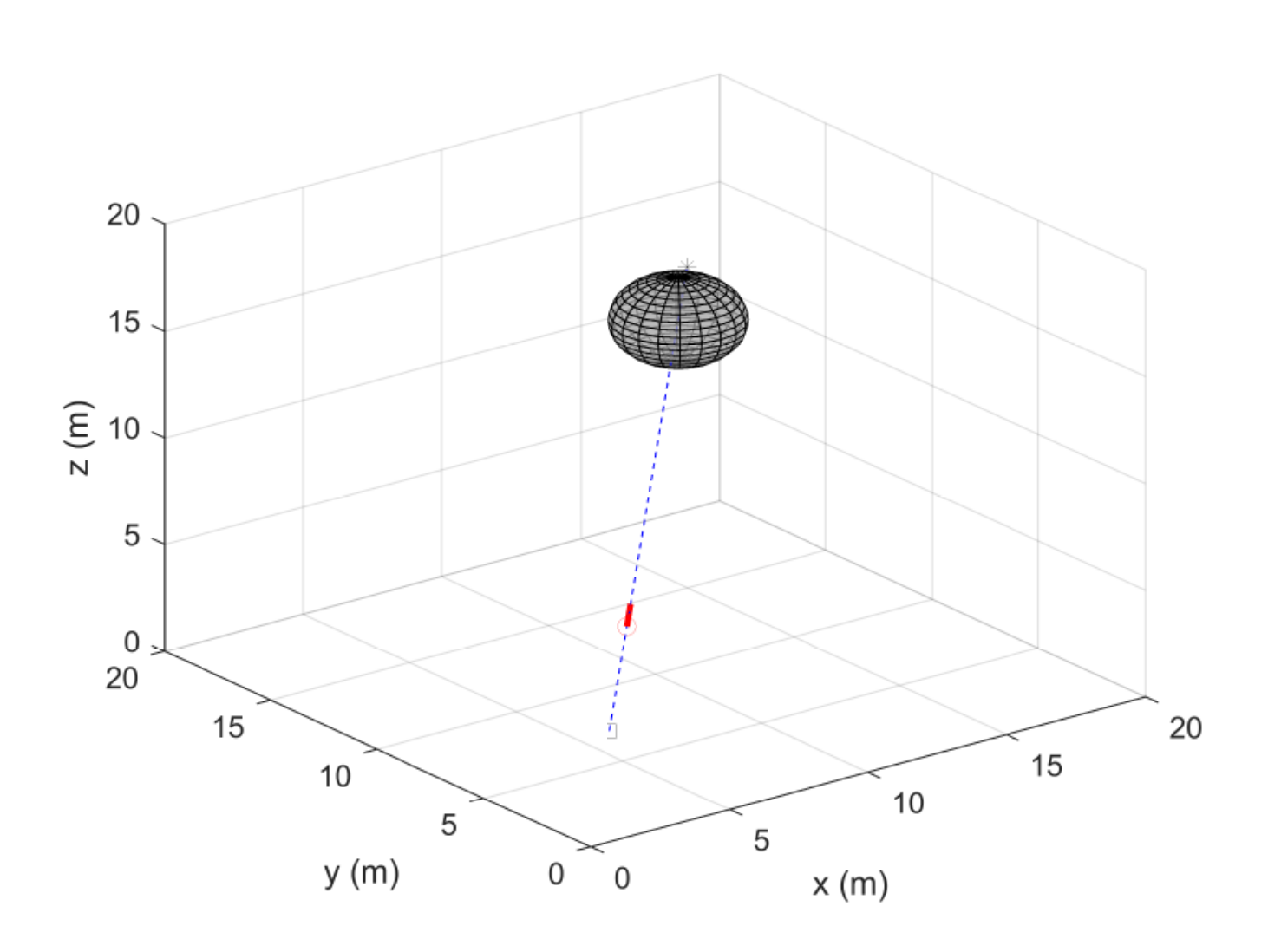} 
			\caption{}
		\end{subfigure}
		\hfill
		\begin{subfigure}[t]{0.48\textwidth}
			\centering
			\includegraphics[width=\linewidth]{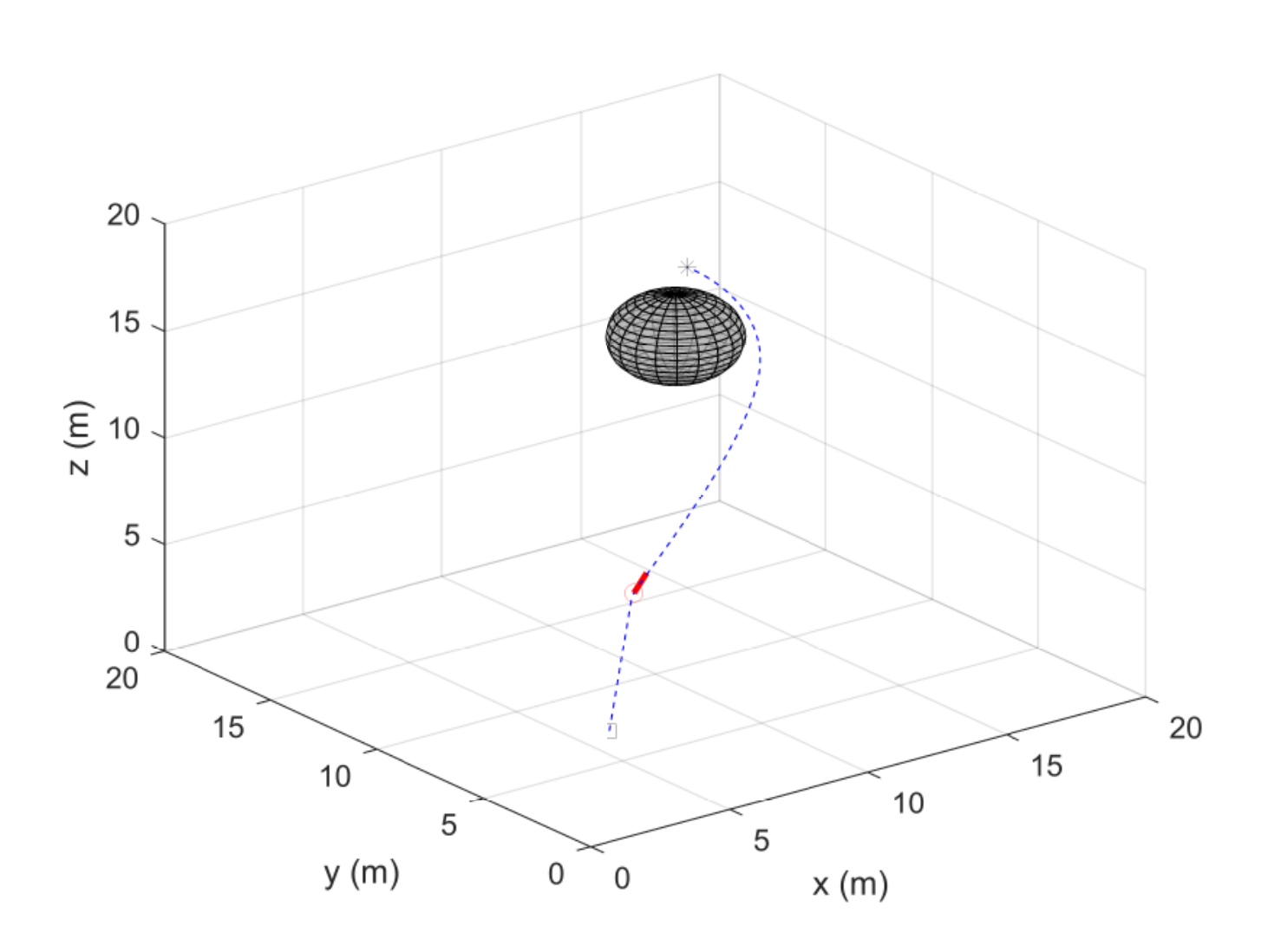} 
			\caption{}
		\end{subfigure}
		
		\begin{subfigure}[t]{0.48\textwidth}
			\centering
			\includegraphics[width=\linewidth]{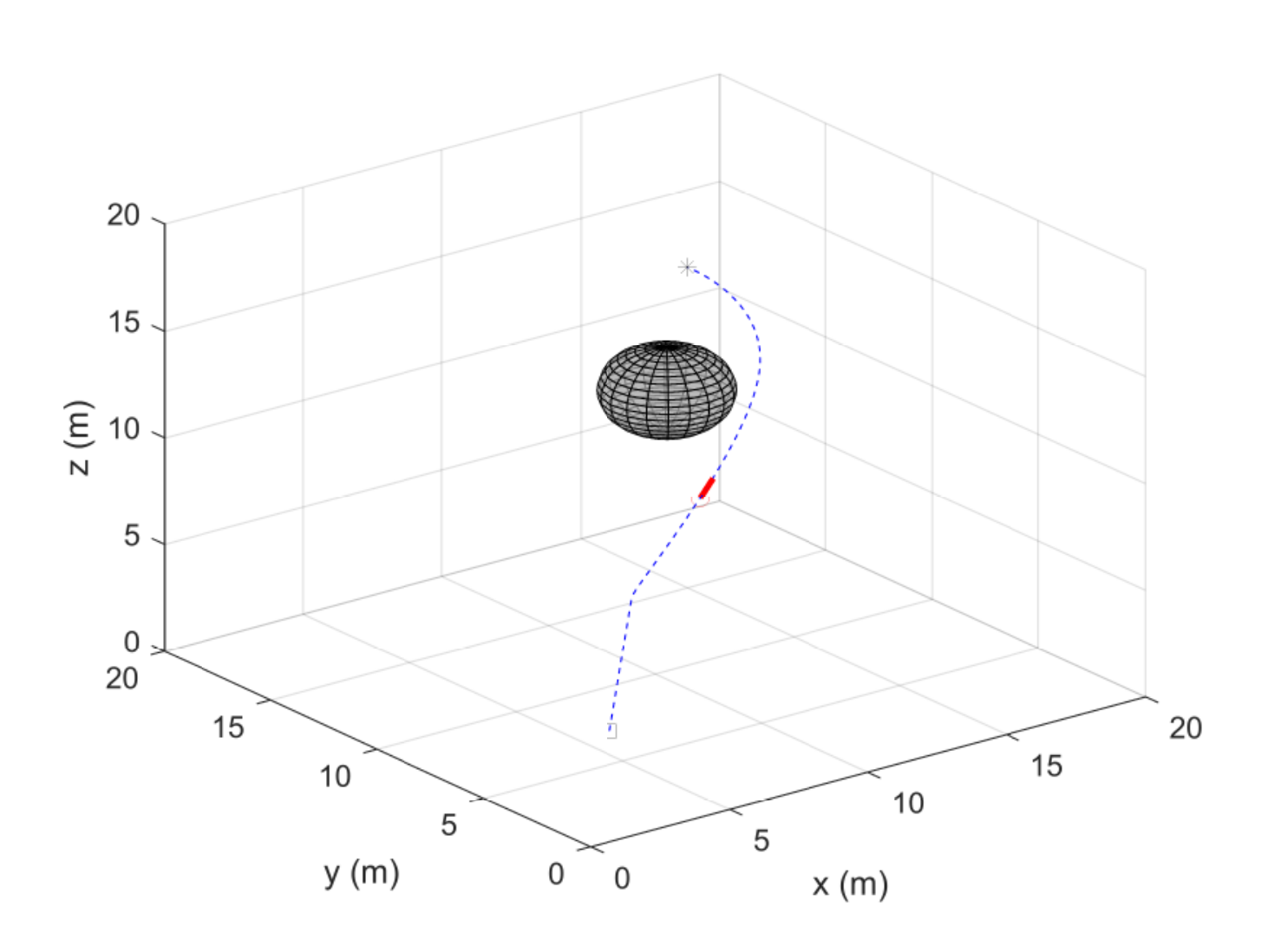} 
			\caption{}
		\end{subfigure}
		\hfill
		\begin{subfigure}[t]{0.48\textwidth}
			\centering
			\includegraphics[width=\linewidth]{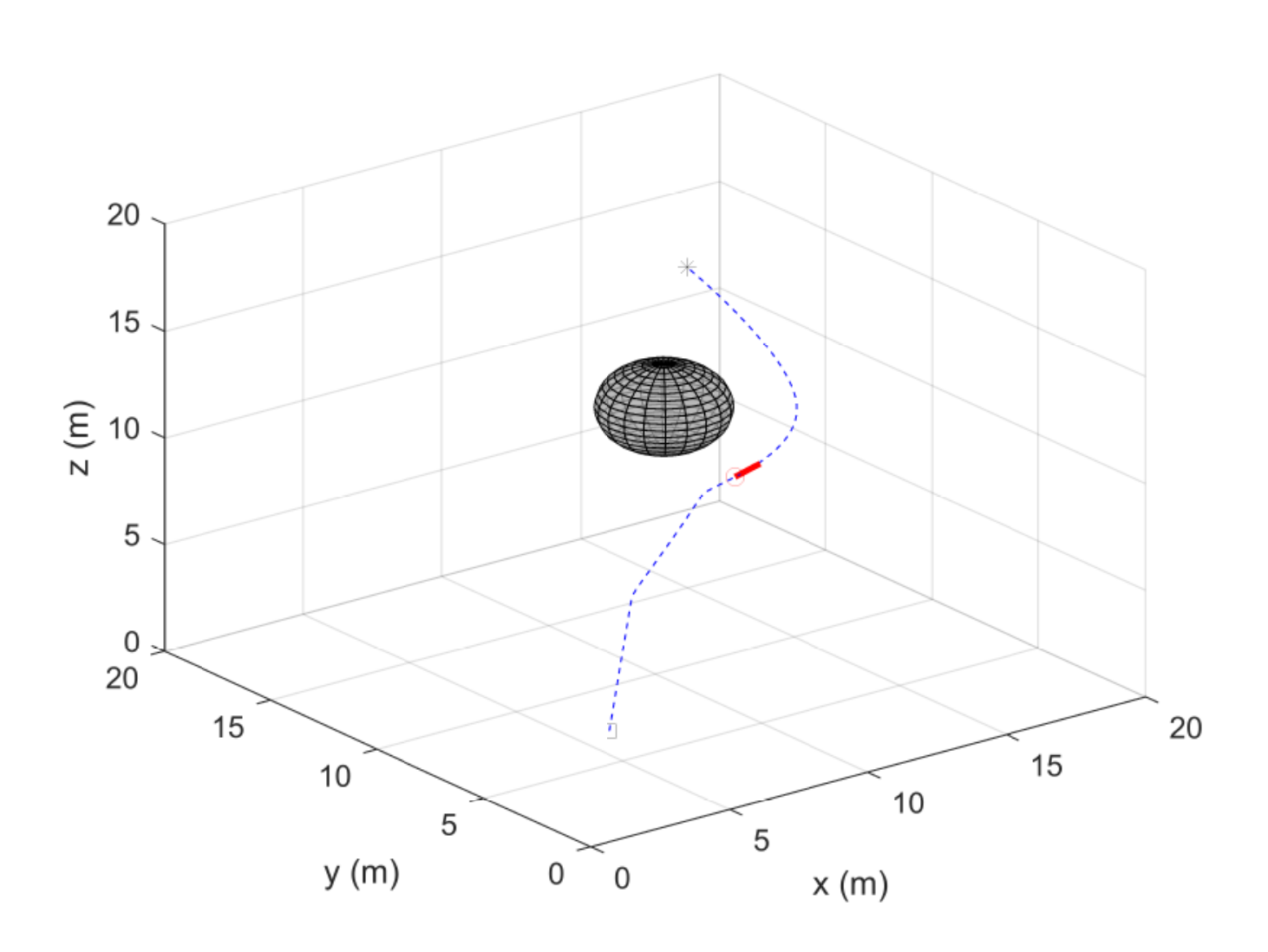} 
			\caption{}
		\end{subfigure}
		
		\begin{subfigure}[t]{0.48\textwidth}
			\centering
			\includegraphics[width=\linewidth]{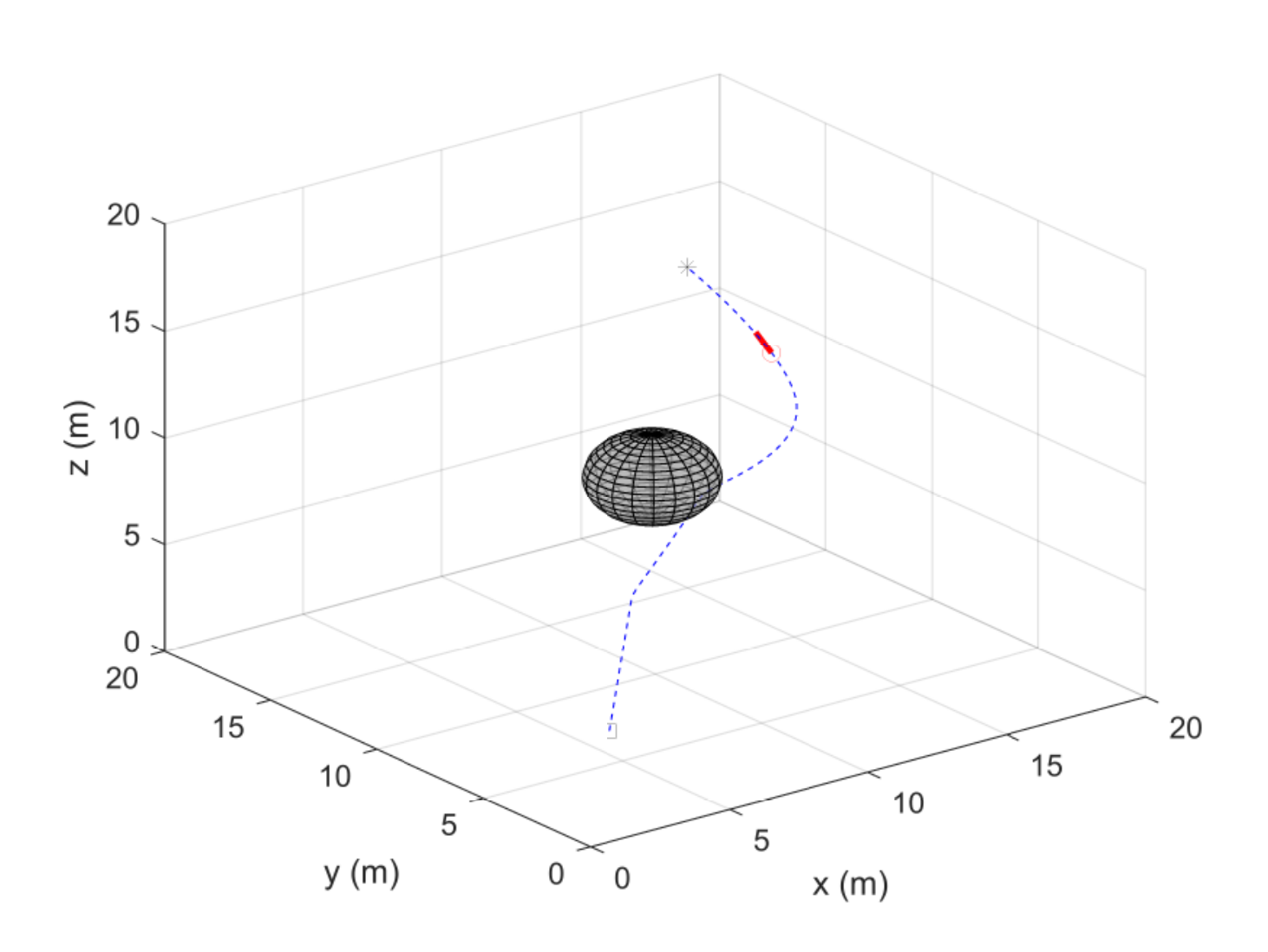} 
			\caption{}
		\end{subfigure}
		\hfill
		\begin{subfigure}[t]{0.48\textwidth}
			\centering
			\includegraphics[width=\linewidth]{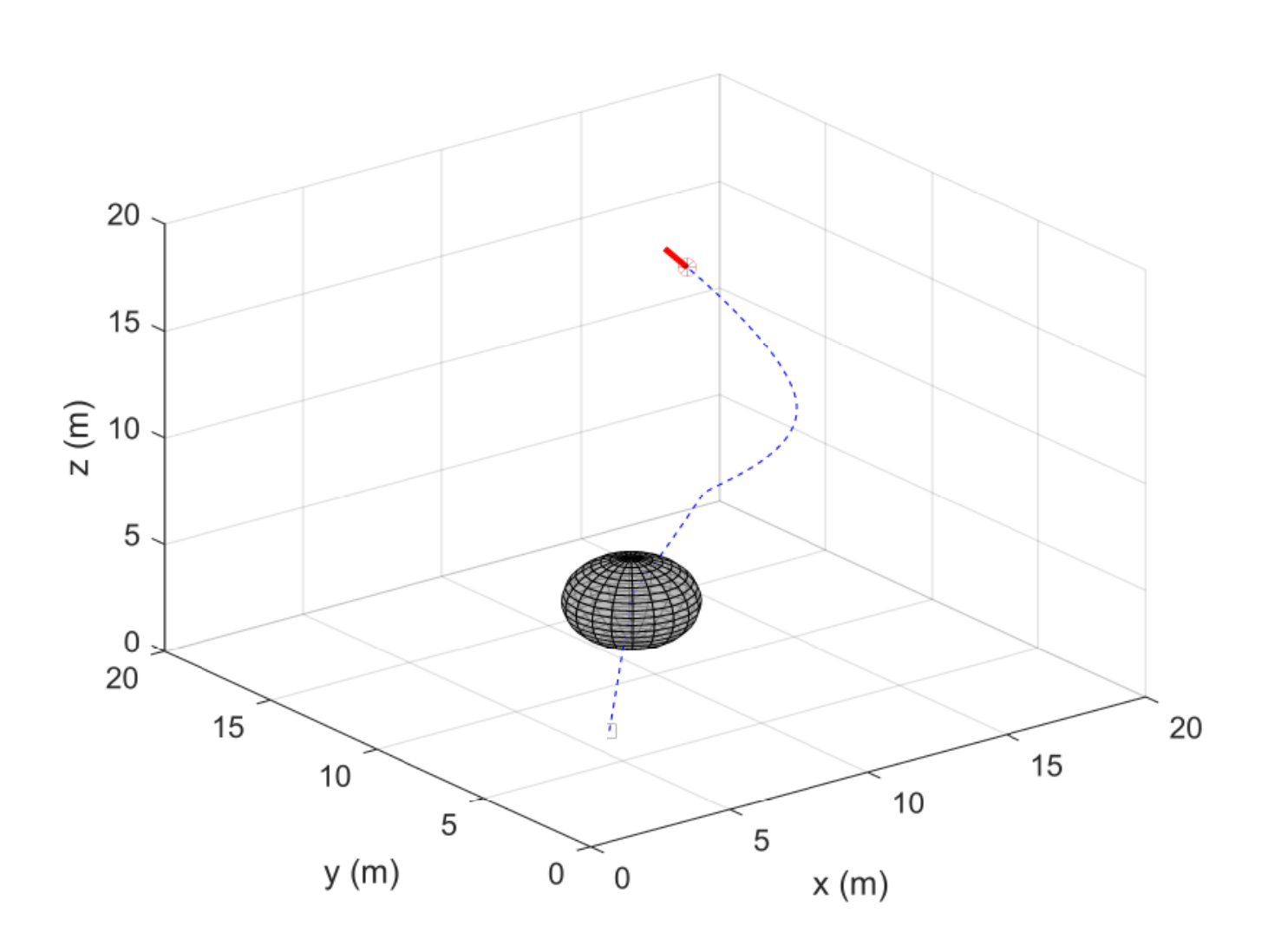} 
			\caption{}
		\end{subfigure}
		
		\caption{Simulation results: strategy validation with dynamic obstacles (Case 2)}\label{fig:ch6:simDynamic2}
	\end{adjustbox}
\end{figure}

\begin{figure}[!htb]
	\centering
	\begin{adjustbox}{minipage=\linewidth,scale=0.9}
		\begin{subfigure}[t]{0.48\textwidth}
			\centering
			\includegraphics[width=\linewidth]{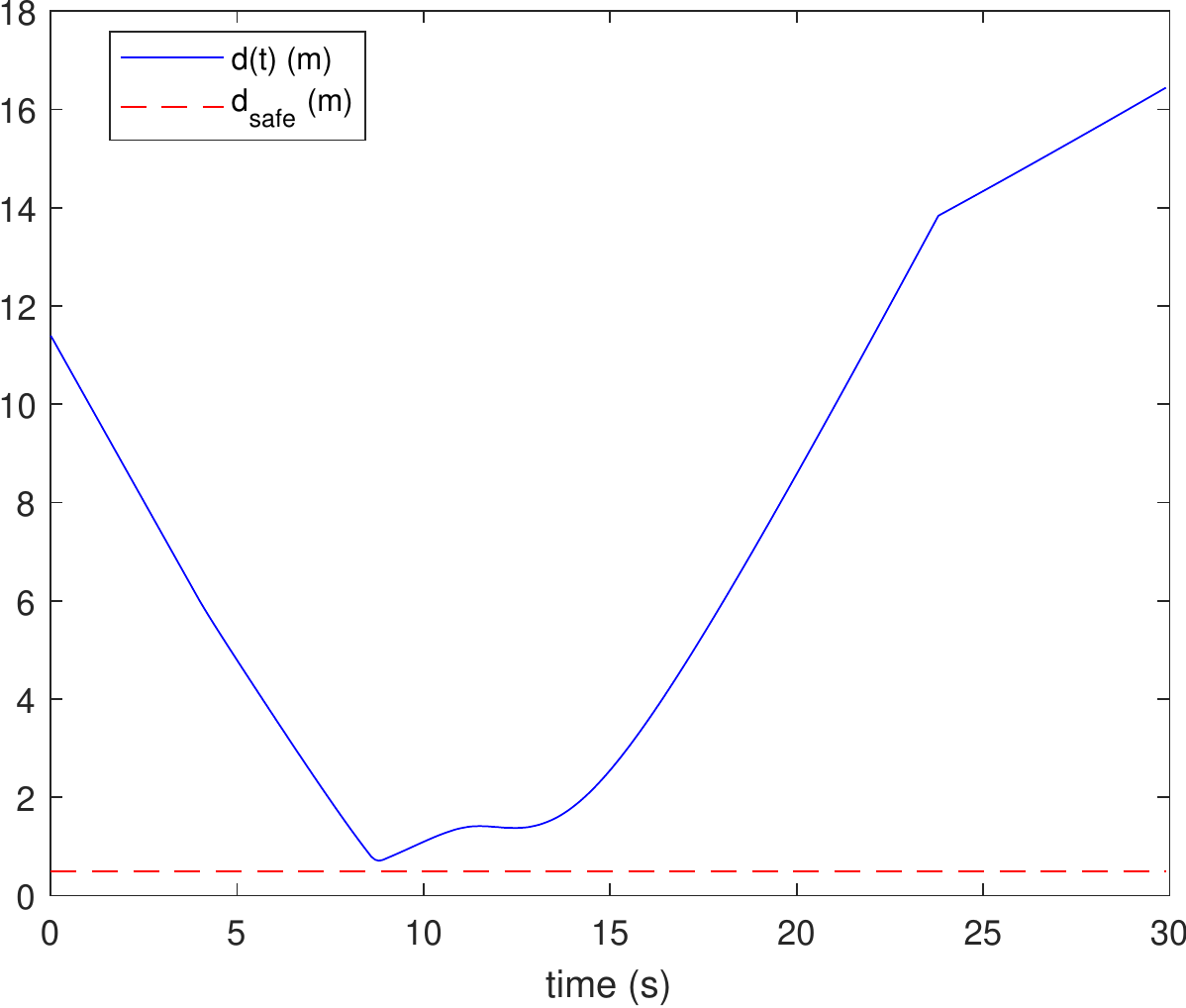} 
			\caption{Case 1}
		\end{subfigure}
		\hfill
		\begin{subfigure}[t]{0.48\textwidth}
			\centering
			\includegraphics[width=\linewidth]{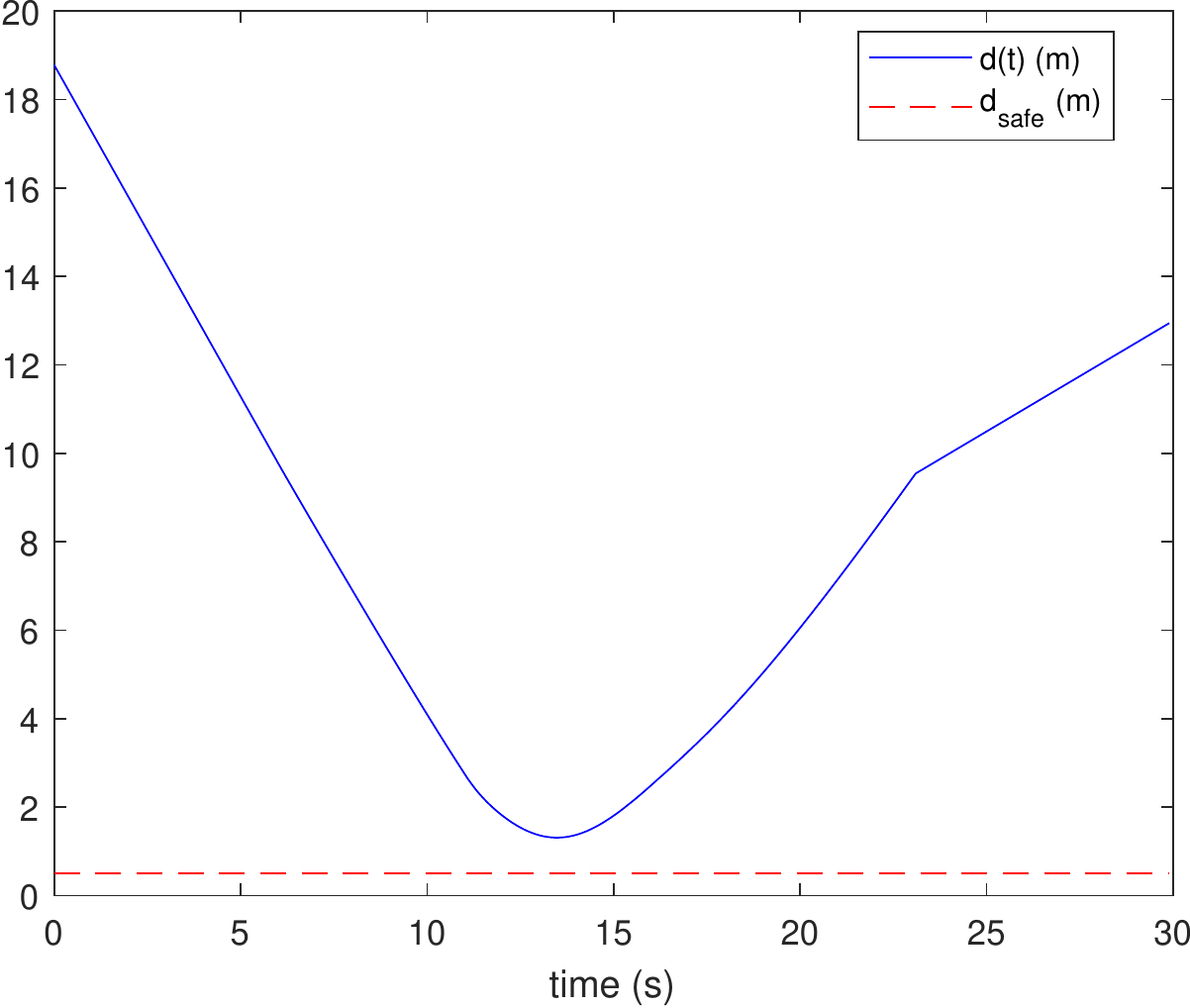} 
			\caption{Case 2}
		\end{subfigure}
		
		\caption{Simulation results: vehicle's relative distance to nearest obstacle ($d(t)$ versus time) for dynamic cases}\label{fig:ch6:simDynamicDt}
	\end{adjustbox}
\end{figure}

\subsection{Quadrotor Software-in-the-Loop Simulations}

The suggested implementation in \cref{sec:ch6:quadrotor_control} for quadrotor UAVs have also been evaluated in simulations to further investigate the computation performance of the overall strategy.
To that end, Software-in-the-Loop (SITL) simulations were performed where it's possible to integrate the production source code into a robotic simulator.
Thus, such implementation can be applied to physical UAVs without modifications other than just tuning some parameters.
We use the Gazebo simulator, and the overall software structure is implemented using the Robot Operating System (ROS) framework.
Furthermore, the presented control scheme is built on top of the open-source PX4 flight stack which is simply a collection of tools for low-level control, states estimation through extended Kalman filter (EKF) and interfacing with onboard sensors.
This is similar to our hardware setup to allow for quick deployment, and the simulations were performed on a computer with similar capabilities to the mini computers used with our UAVs (typically, an Intel NUC with good processing power).

We considered two simulation cases with dynamic environments.
In these cases, the quadrotor performs a takeoff to reach some position which is considered as the initial position $\bm{p}_{initial}$.
A goal location $\bm{p}_{goal}$ gets assigned next to the UAV with a signal to start the mission.
The vehicle directly computes and executes a trajectory with a smooth trapezoidal velocity profile as explained in Appendix~A. %
Along the motion, a collision checker algorithm running at 10Hz checks whether the active path is safe or not based on sensors measurements and some safety margin $d_{safe}$.
It simply computes the closest distance $d_*(t)=\min\limits_{\bm{p}_o(t) \in \mathcal{O}} \|\bm{p}_*(t) - \bm{p}_o(t)\|$ to the obstacles set $\mathcal{O}$ from some point on the active reference path $\bm{p}_* \in \Gamma$.
Once it is found that $d_*(t) < d_{safe}$, a segment of the currently active path is selected for deformation where the starting point can be at some location ahead of the vehicle's current position to allow for computation latency $\triangle t_{c}$ while maintaining the motion continuity.
The control point $\bm{p}_c$ is chosen as the path point which is the closest to obstacles (i.e. $\bm{p}_c = \bm{p}_*$); however, different methods could be used to determine a good control point.
The deformation process is applied in real-time as described in \cref{sec:ch6:traj_generation,sec:ch6:traj_generation}.

The key factors affecting the computational complexity of the overall strategy were found to be the the reference path resolution used for collision checking and the adopted algorithms to determine $\bm{p}_*$ and $\bm{p_L}$ (used in \eqref{equ:ch6:lookaheadError}).
Based on the implemented algorithms, $\triangle t_{c}$ was found to be relatively small, few milliseconds, compared to the control and collision checking update rates.

In these simulations, a different and simpler approach was used to determine $\bm{v}_{safe}$ rather than the one suggested in \cref{sec:ch6:deformation}.
Let $\bm{\vec{D}} \in \R^3$ be some unit vector perpendicular to the reference path tangent $\bm{T}$ at point $\bm{p}_{c}$.
Then, $\bm{v}_{safe}$ can be computed as follows:
\begin{equation}
	 \bm{v}_{safe} = \gamma \bm{\vec{D}}
\end{equation}
where $\|\bm{\vec{D}}\| = 1$, and $\gamma$ is the safety factor defined in \eqref{equ:ch6:vsafe}.
Similar to our observations from the previous simulation cases, better quality deformations in terms of the path length were achieved using smaller values of $\gamma$ at the expense of requiring more deformations at each computation cycle to increase the path safety.

\begin{figure}[!htb]
	\centering
	\includegraphics[clip,width=0.9\columnwidth]{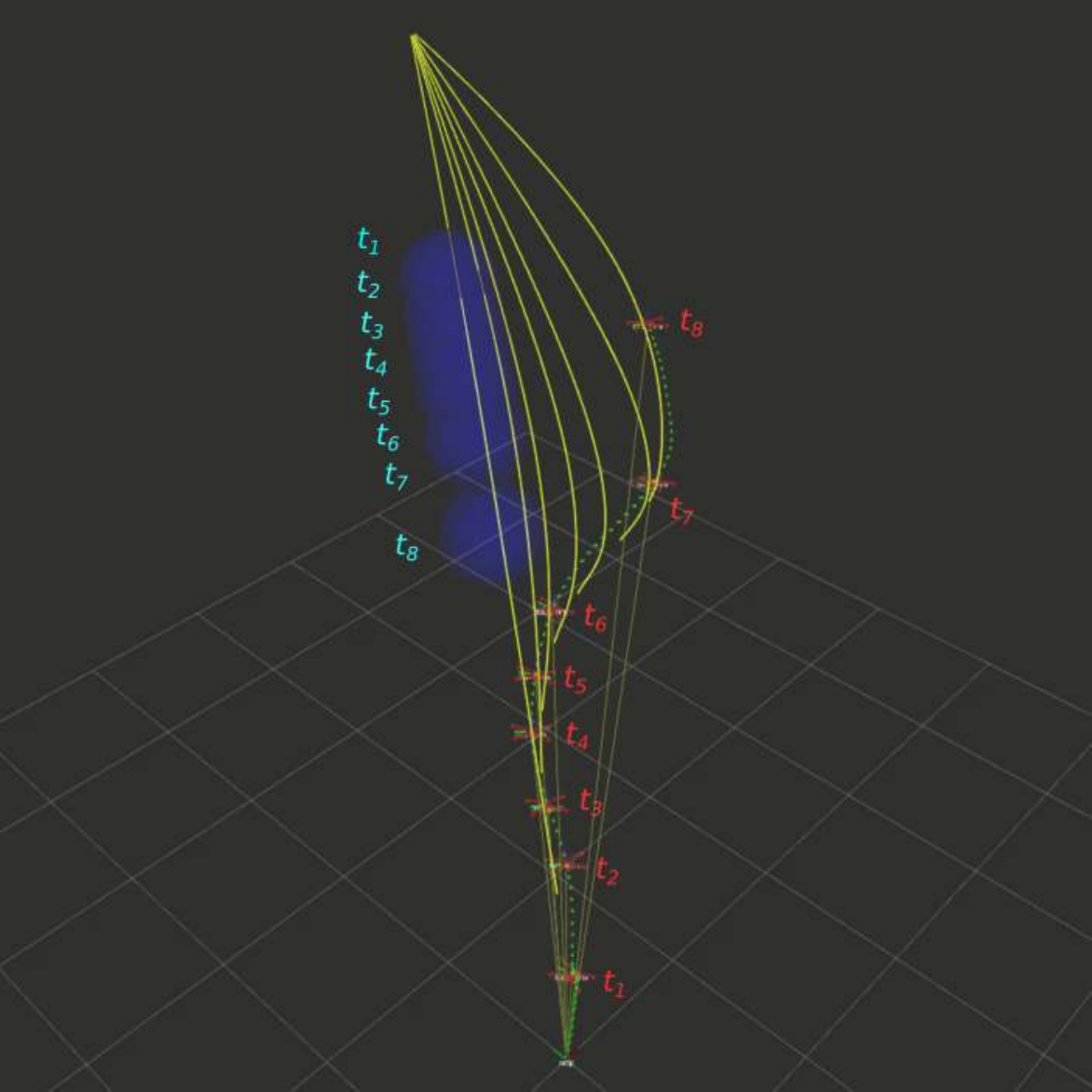}%
	\caption{Snapshots of the UAV position during movement for scenario 1}
	\label{fig:motion1}
\end{figure}

\begin{figure}[!htb]
	\centering
	\includegraphics[clip,width=0.75\columnwidth]{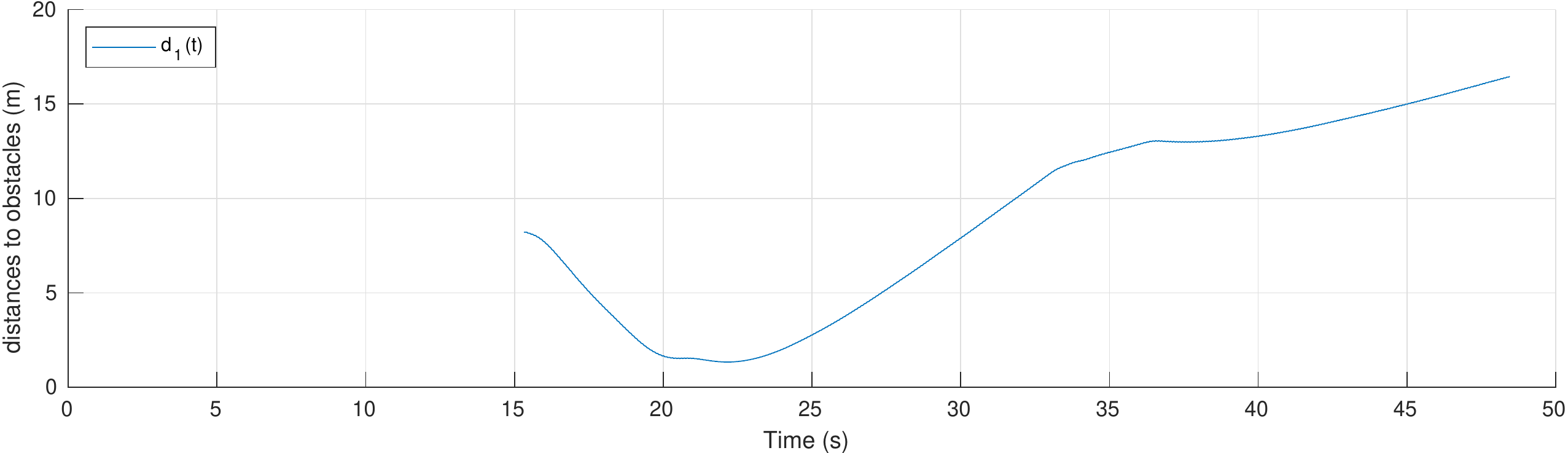} \\
	\includegraphics[clip,width=0.75\columnwidth]{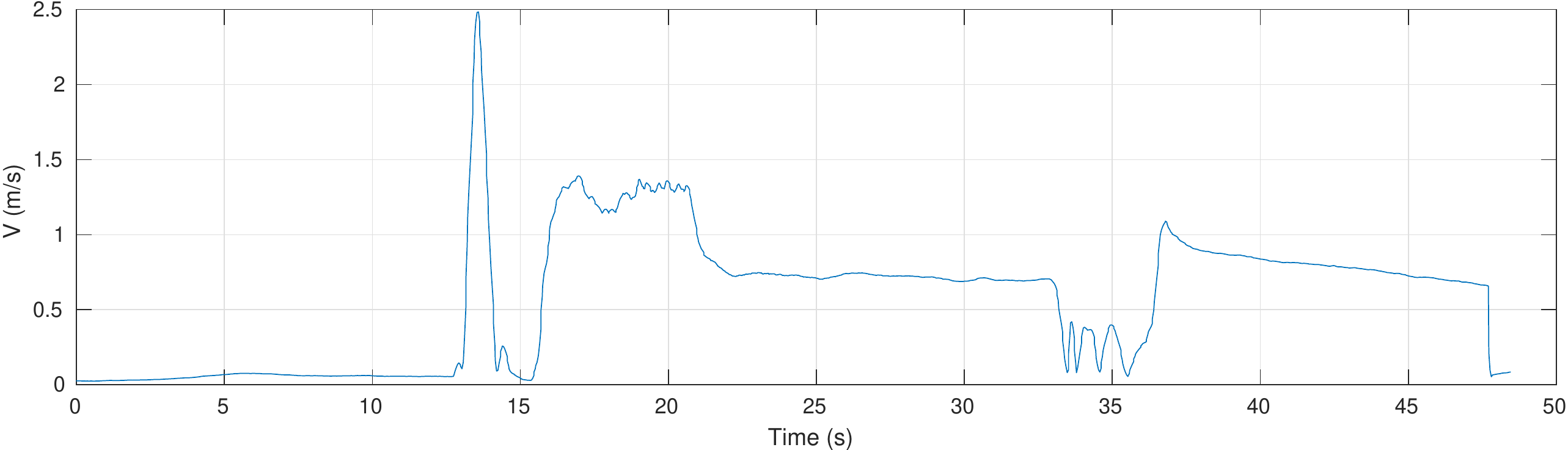}%
	\caption{Distance to obstacle and UAV velocity for scenario 1}
	\label{fig:results1}
\end{figure}

\begin{figure}[!htb]
	\centering
	\includegraphics[clip,width=0.75\columnwidth]{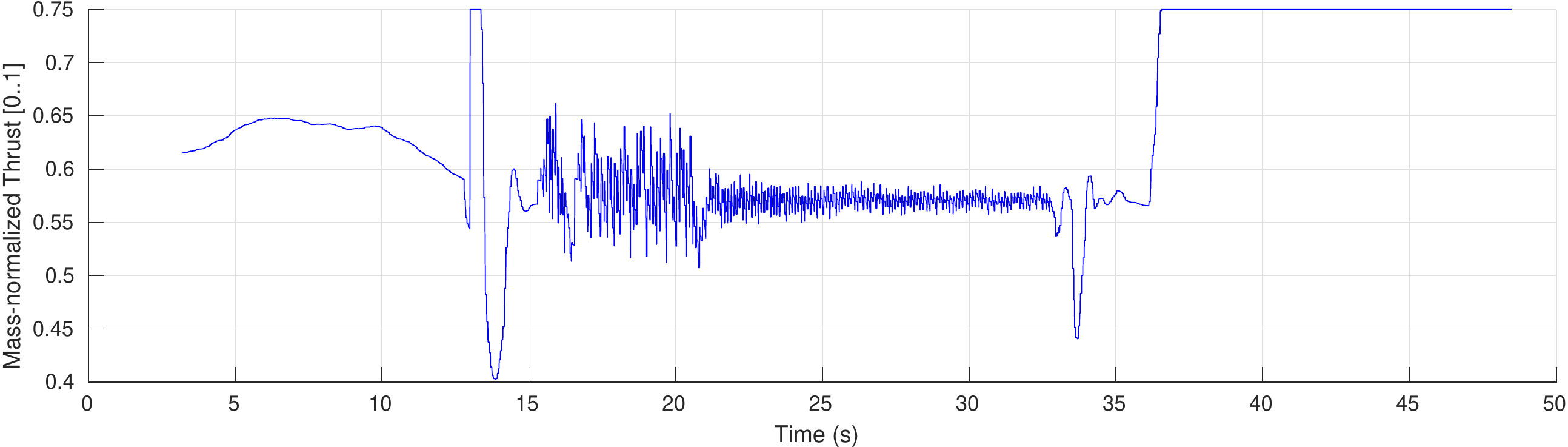} \\
	\includegraphics[clip,width=0.75\columnwidth]{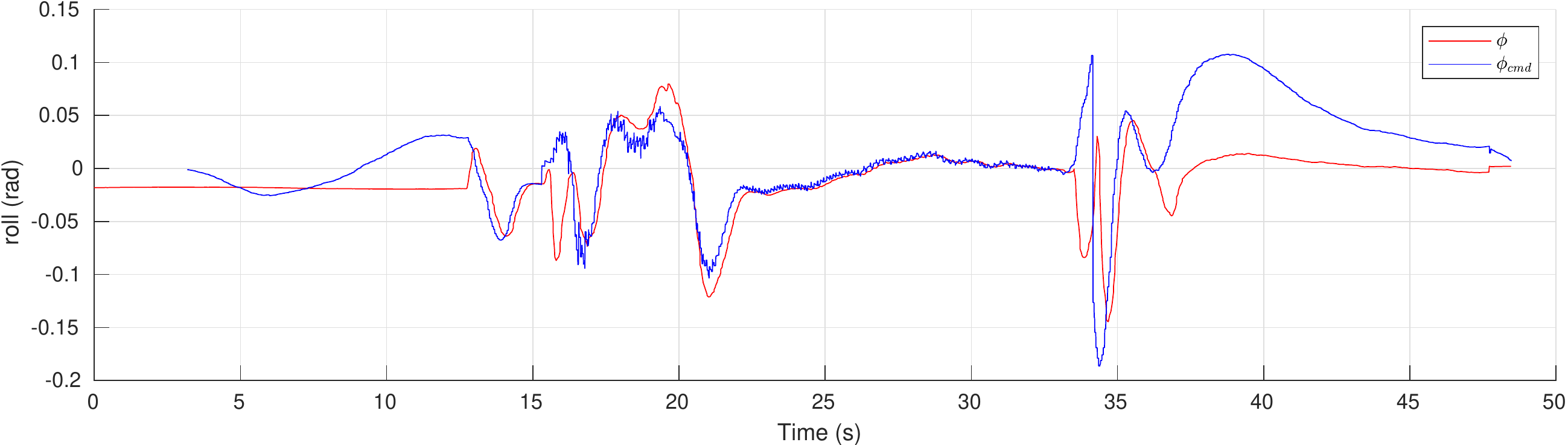} \\
	\includegraphics[clip,width=0.75\columnwidth]{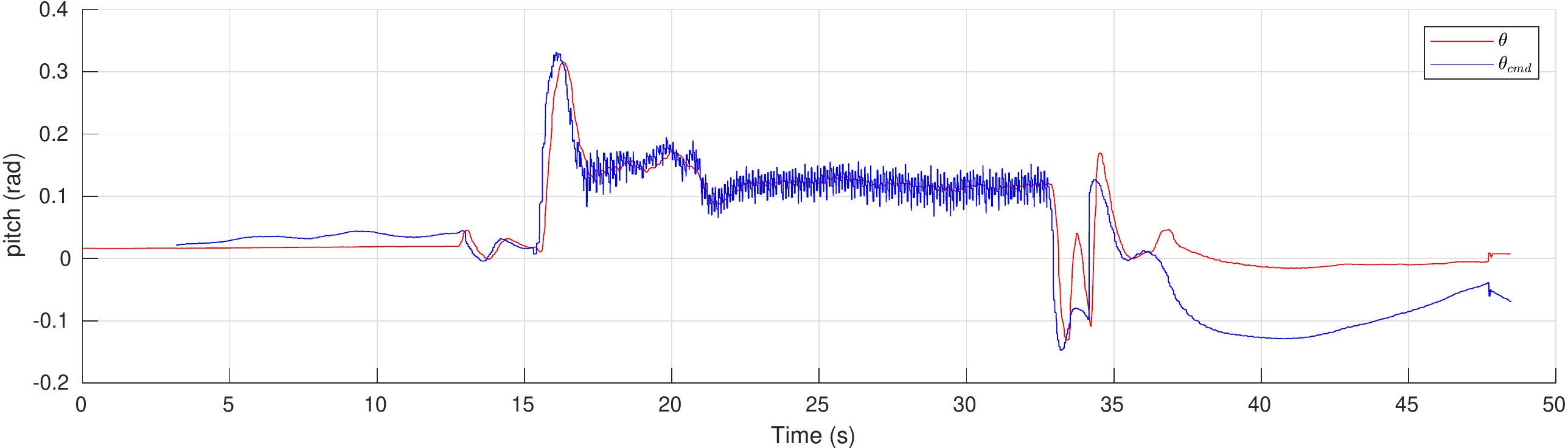} \\
	\includegraphics[clip,width=0.75\columnwidth]{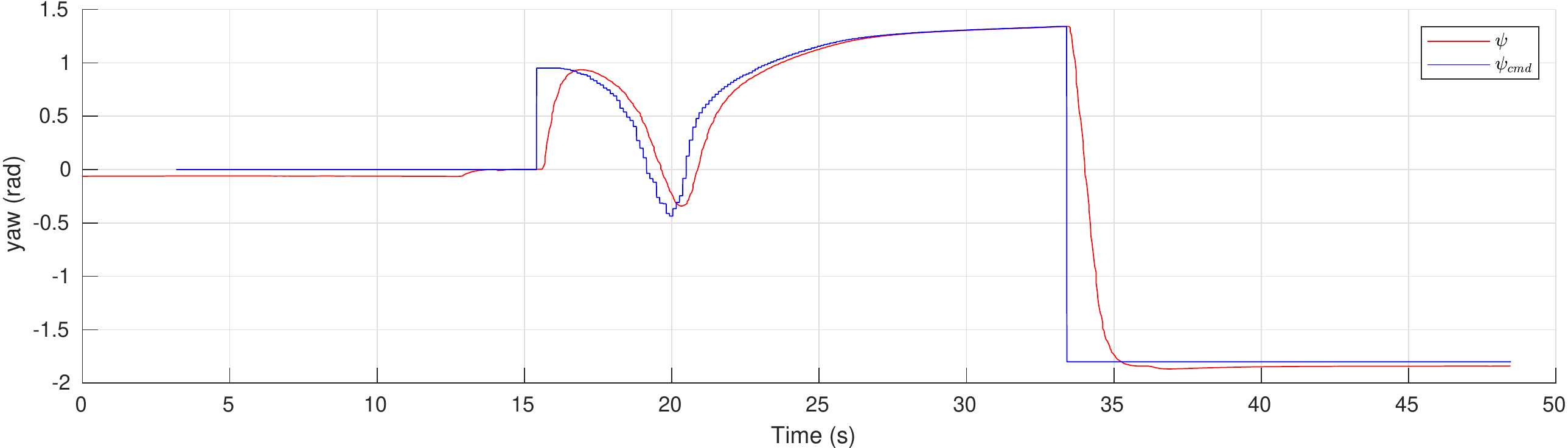}
	\caption{Input mass-normalized thrust and attitude commands vs time for scenario 1}
	\label{fig:results1b}
\end{figure}

\begin{figure}[!htb]
	\centering
	\includegraphics[clip,width=\columnwidth]{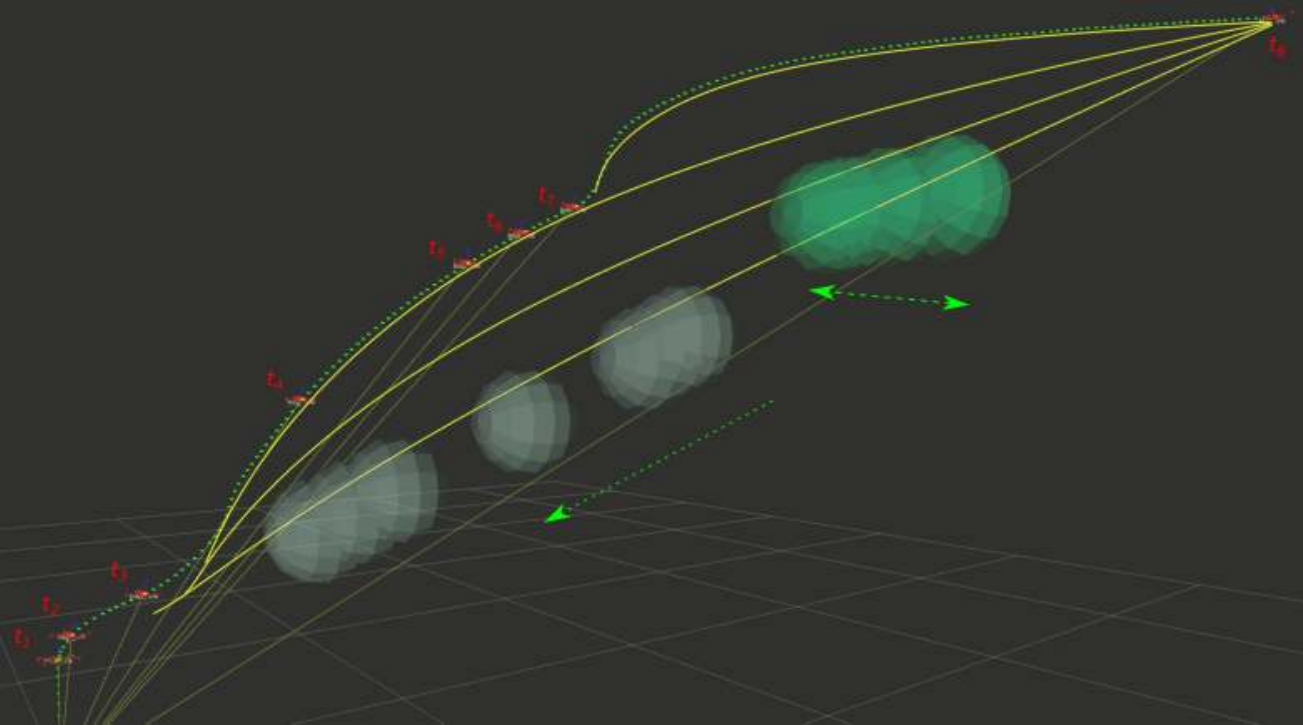}%
	\caption{Snapshots of the UAV position during movement for scenario 2}
	\label{fig:motion3}
\end{figure}

\begin{figure}[!htb]
	\centering
	\includegraphics[clip,width=0.75\columnwidth]{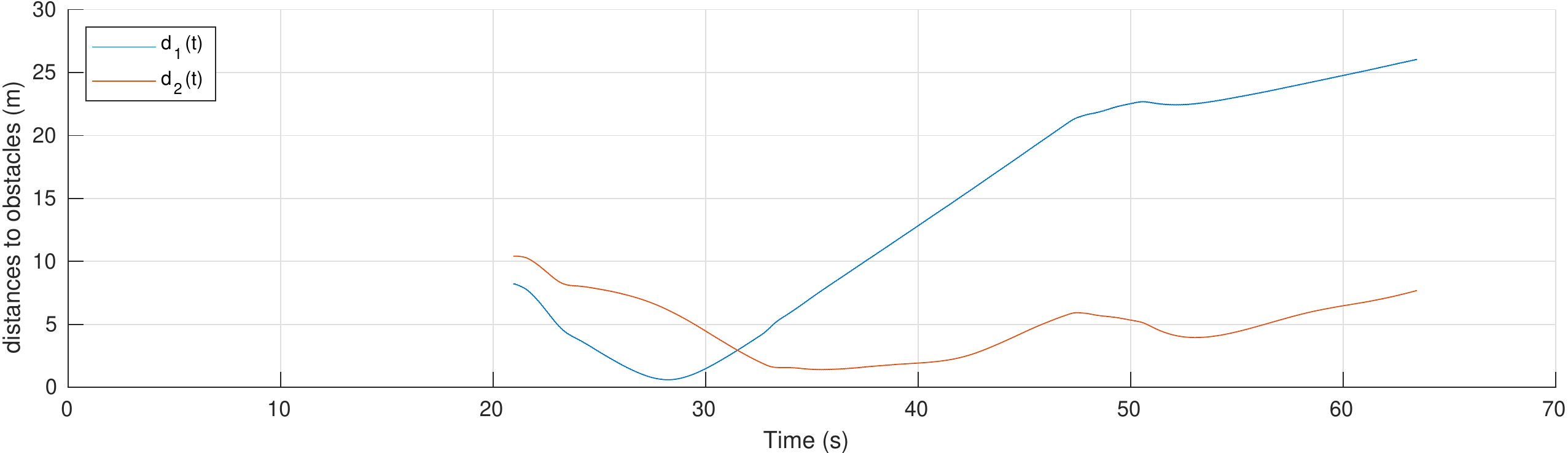} \\
	\includegraphics[clip,width=0.75\columnwidth]{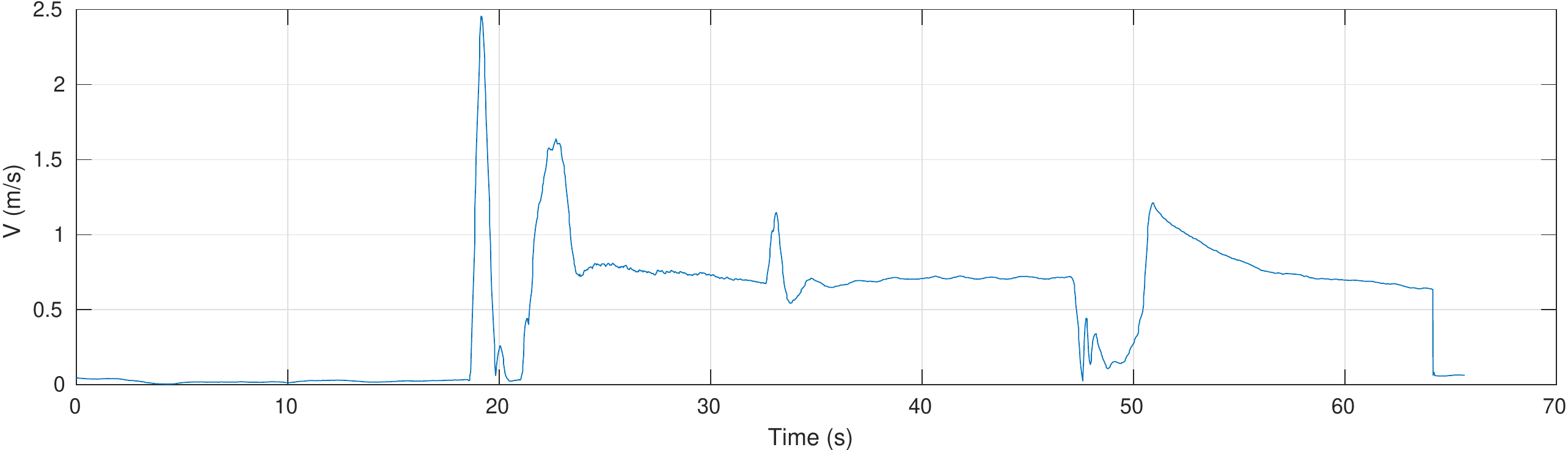}%
	\caption{Distances to obstacles and UAV velocity for scenario 2}
	\label{fig:results3}
\end{figure}

\begin{figure}[!htb]
	\centering
	\includegraphics[clip,width=0.75\columnwidth]{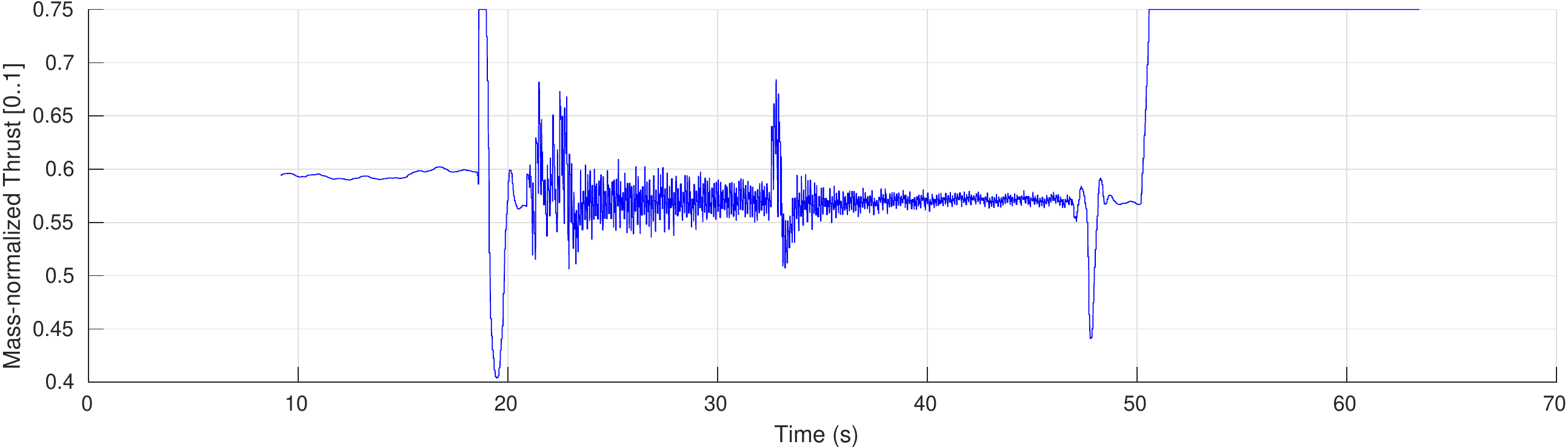} \\
	\includegraphics[clip,width=0.75\columnwidth]{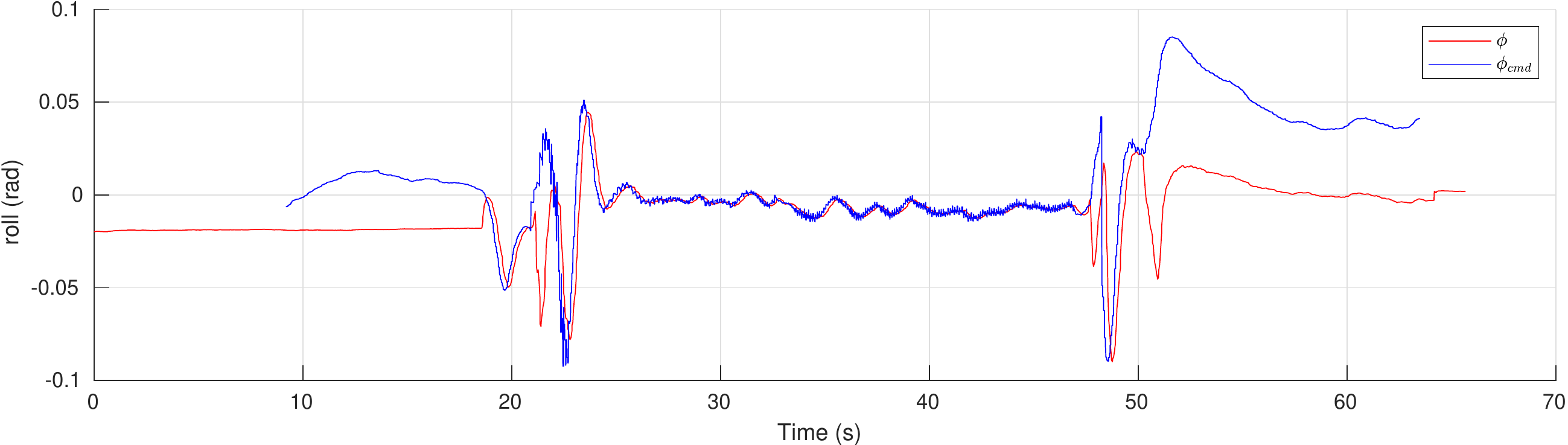} \\
	\includegraphics[clip,width=0.75\columnwidth]{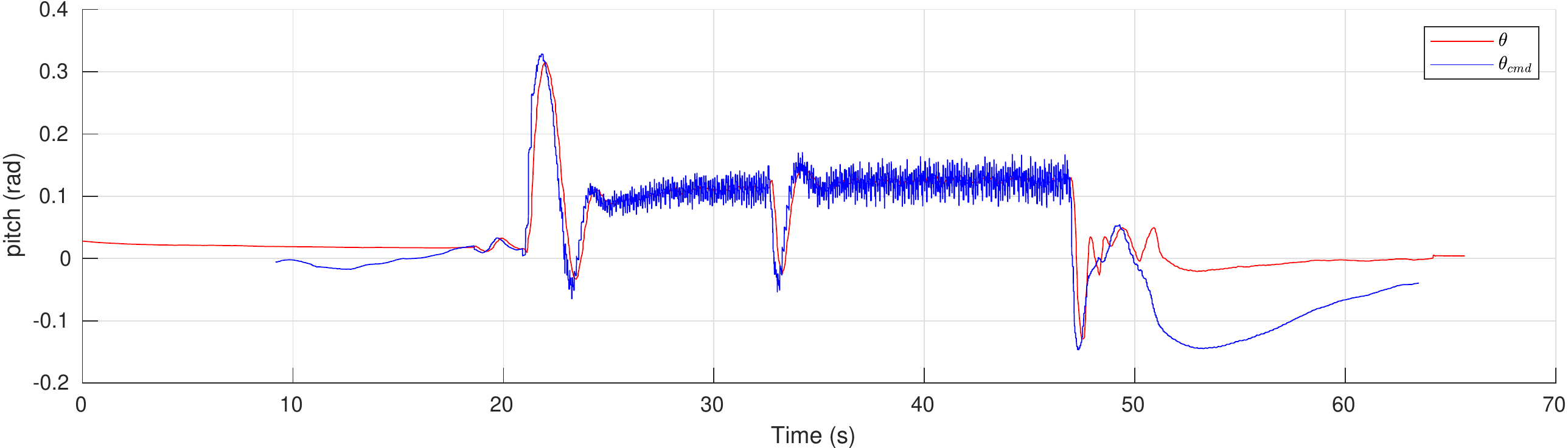} \\
	\includegraphics[clip,width=0.75\columnwidth]{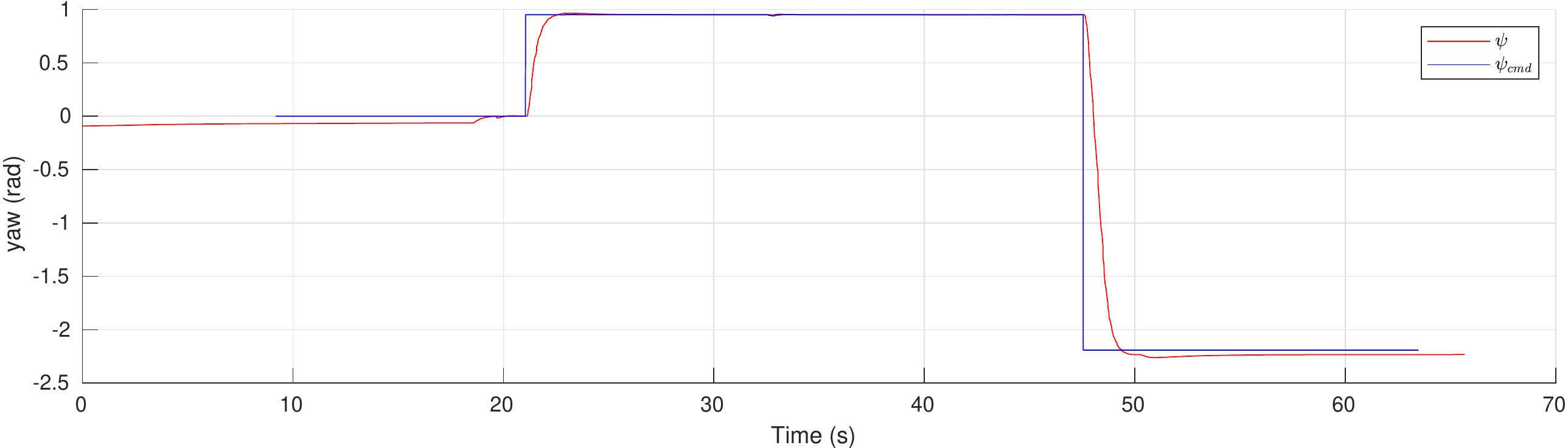}
	\caption{Input mass-normalized thrust and attitude commands vs time for scenario 2}
	\label{fig:results3b}
\end{figure}

\section{Conclusion}\label{sec:ch6:conclusion}

This chapter presented a 3D navigation strategy for collision avoidance in unknown and dynamic environments suitable for different UAV types and AUVs.
It adopts a sense and avoid paradigm with good computational cost to provide quick reflex-like reactions to obstacles.
Quintic Bezier splines are used to generate real-time smooth deformations around obstacles.
Implementation details with quadrotor UAVs were also provided.
Simulations were carried out in MATLAB and Gazebo to show the effectiveness of the suggested method.
Potential directions for future work include investigating different approaches to determine deformation direction based on onboard cameras and/or range sensors in addition to real implementation with a quadrotor in a dynamic environment.

\newcommand{\RotMatrix}{{\bf R}}%

\chapter[Autonomous Navigation in Tunnel-Like Environments]{A Method for Autonomous Collision-Free Navigation in Unknown Tunnel-Like Environments\label{cha:tunnel_navigation}}

Unmanned aerial vehicles (UAVs) have become essential tools for exploring, mapping and inspection of unknown tunnel-like environments which can be sometimes dangerous and/or unreachable by humans.  
The harsh conditions in such environments urge the need for more reliable safe navigation methods that can allow a fully autonomous operation.
Thus, a computationally-light navigation method is developed in this chapter for UAVs considering this special case of the navigation problem tackled in this report to autonomously guide the vehicle through unknown three-dimensional (3D) confined environments (research question 2 in \cref{ch1:problem}).
It can use depth information from onboard sensors to estimate points along the tunnel axis which can then be used to direct UAV motion without the need for accurate localization.
The development of this method is based on a general kinematic model which makes it applicable to different UAV types and autonomous underwater vehicles.
Considering the confined space structure of such environments in the control design can provide more tailored methods for this particular application to autonomously determine the direction of progressive advancement through the environment without colliding with its boundaries.
One can further combine this approach with one of the previously designed reactive obstacle avoidance methods in \cref{cha:methods_reactive3D,cha:reactive_impl,cha:deforming_approach} to handle both progressive advancement through the environment while avoiding collisions with its boundaries in addition to other obstructing unknown/dynamic obstacles.
Several Computer simulations were carried out to validate the proposed method considering different 3D tunnel-like environments and realistic sensing models. %
Furthermore, implementation details are provided for quadrotor-type UAVs in addition to a control design based on the differential-flatness property of quadrotor dynamics and sliding mode control similar to the previously developed controller.
Experiments were carried out to autonomously fly a quadrotor using the proposed method through tunnel-like structures where all computations needed for navigation were done onboard.
The work presented in this chapter is published in \cite{elmokadem2021method}.

\section{Introduction}

Recent developments in technologies related to unmanned aerial vehicles (UAVs) have made them very popular in many applications as agile mobile platforms with low operational costs.
It has become possible with UAVs to perform hard tasks in unreachable harsh environments that are risky to human lives.
One important problem in this area is the safe navigation of unmanned aerial vehicles through unknown tunnel-like environments which is the main focus of this study.
This problem arises in many industrial applications such as navigation of flying robots through 
underground mines and connected tunnels, navigating small aerial vehicles in cluttered indoor 
environments, 3D mapping of cave networks, interior inspection of pipeline networks, search \& rescue missions during disaster events in underground rail networks, etc.
For example, some variants of these applications that have gained a great interest by researchers recently are dam penstocks inspection and/or mapping  \cite{ozaslan2015inspection,ozaslan2016towards,ozaslan2017autonomous,ozaslan2018spatio}, chimney inspection \cite{quenzel2019autonomous}, hazardous deep tunnels inspection \cite{tan2018smart,tan2019design}, mapping and navigation in underground mines/tunnels \cite{mansouri2018towards,mascarich2018multi,kanellakis2019open,li2020visual,kominiak2020mav,sharif2020mav, li2020autonomous,mansouri2020deploying}, search \& rescue in underground mines \cite{dang2020autonomous,petrlik2020robust}, inspection of ventilation systems \cite{petrlik2020robust}, inspection of narrow sewer tunnels \cite{chataigner2020arsi} and inspection tasks in the oil industry \cite{shukla2016application}.
In all these applications, a UAV should navigate through a tunnel-like unknown environment while avoiding collisions with the tunnel walls.
A more favorable behavior is to have a fully autonomous operation with least human interaction.
This problem is highly challenging due to several factors that may vary from one environment to another such as poor lighting conditions, narrow flying space, absence of GPS signals and featureless structures. 
Some other challenging factors are vehicle-specific such as sensing capability, payload capacity, onboard computing power and maximum flight time.
All these factors have a great effect on the overall system design and navigation algorithm development.

The problem under consideration is also of great importance to many marine applications with autonomous underwater vehicles (AUVs) where it is required to navigate through underwater tunnel-like environments.
This include applications in underwater geology and archeology, inspecting different kinds of underwater structures, military operations, inspecting flooded spring tunnels, bypass tunnels for dams, storm runoff networks and freshwater delivery tunnels, etc (for example, see \cite{mallios2016toward,fairfield2007real,vidal2018online,am20013d,martins2018uxnexmin,vidal2020multisensor,nocerino20193d,jacobi2015autonomous,weidner2017underwater,white2010malta,gary20083d} and references therein).

In general, existing solutions to the navigation problem in unknown environments may be classified into \textit{planning-based} or \textit{reactive} methods.
\textit{planning-based} approaches require a map representation of the environment and localization information to find safe paths locally.
Local path planning normally adopts an optimization-based or sampling-based search approach.
As the search space size increases, the computational cost of such algorithms becomes more expensive \cite{li2021efficient}.
These approaches mostly adopt simultaneous localization and mapping (SLAM) techniques to allow for autonomous operation in unknown environments.
On contrary, \textit{reactive} approaches may directly generate motion decisions based on light processing of current sensors observations to produce reflex-like reactions \cite{hoy2015algorithms}.
These methods can provide a better computational cost compared with planning-based methods without the need for accurate localization.

The available methods addressing the navigation problem of interest suggest different approaches in terms of the overall system design, the level of autonomy and the algorithm adopted to traverse the tunnel.
The choices made for UAV system design are mostly made specifically to serve a specific application. 
The use of redundant sensors may be found common among different systems to attain a fully autonomous operation in some harsh environments by combining range and vision-based sensors.
Depending only on one kind of sensors may cause the system to fail at some situations.
For example, range sensors can suffer from wet structures causing them to fail sometimes, and vision-based solutions may be useless against textureless environments \cite{ozaslan2016towards}.
An evaluation of localization solutions in underground mines based solely on cameras can be found in \cite{kanellakis2016evaluation}.
Therefore, it is common in the available solutions to use multi-modal sensors which can improve localization and/or reactive responses to cope with the harsh conditions in tunnel environments.
The following subsection summarizes some of the recent solutions.

\subsection{Related Work}

Many of the available navigation methods tackling the same problem fall under the planning-based category where the main focus of the development is shifted towards the localization system design.
For example, the approach presented in \cite{ozaslan2015inspection} suggested a combination of Unscented Kalman Filter (UKF) and a particle filter to process IMU and range measurements for UAV localization in dam penstocks where a map was available a priori.
Then, the navigation was achieved in a semi-autonomous fashion to perform an inspection task where a remote operator was sending goal position commands to guide the UAV through the tunnel environment.
Extensions were then proposed in \cite{ozaslan2016towards} and \cite{ozaslan2017autonomous} in an effort towards a more autonomous solution for penstocks inspection with UAVs.
In \cite{ozaslan2016towards}, UKF was used to provide 6-DOF estimation of the UAV pose by fusing data from IMU, two range sensors and four cameras against a 3D occupancy grid map known in advance.
However, a remote operator was still needed to provide waypoints to guide the UAV.
A SLAM-based approach was then suggested in \cite{ozaslan2017autonomous} combining range and vision-based estimators.
An algorithm was proposed to perform local mapping where fitting a cylindrical model was applied to point clouds obtained from the heterogeneous sensors.
Then, the tunnel axis is estimated from the local map, and the UAV position is determined along the tunnel axis which was then used to guide the UAV.

Another SLAM-based method was presented in \cite{mascarich2018multi} to address the problem of autonomous exploration and mapping in underground tunnels using a UAV equipped with two IMUs, four cameras and three depth sensors.
The open-source Robust Visual Inertial Odometry (ROVIO) framework \cite{bloesch2015robust} was adopted in that work to perform SLAM.
A local planner based on rapidly-exploring random tree (RRT) algorithm was used in a receding horizon manner to generate motion commands towards a direction that maximizes some exploration gain.
The same SLAM framework and local planning algorithm was also used in \cite{papachristos2019autonomous}.
Similarly, a local planner based on the Rapidly-exploring Random Graph (RRG) algorithm is used in \cite{dang2019field} and \cite{dang2020autonomous} to guide the UAV maximizing volumetric exploration gain in underground tunnels with multiple branching locations. 
In these works, data from range, thermal, vision and inertial sensors are fused as a part of their SLAM implementation (sometimes using only a subset of these sensors).
In \cite{petrlik2020robust}, a different approach was presented for operations in extremely narrow tunnels to find safe paths using a modified A* algorithm in 2D occupancy maps generated by onboard SLAM.
A low-level model predictive controller was used to track generated local trajectories based on the planned paths.

Inspection of deep tunnels (i.e. vertical) was also addressed in \cite{tan2018smart,tan2019design} where the authors suggested the use of a UAV with a rotating camera for minimal field-of-view (FOV) obstructions when collecting images for inspection. 
Localization was performed using measurements from an array of laser range sensors to estimate the UAV position and heading in the tunnel with a prior knowledge about its geometry.
They also proposed a method to estimate the tunnel axis using measurements from the sensors array.
Their navigation method was based on maintaining the localized position of the UAV at the center of the tunnel. 
Additionally, an optical-flow sensor was used with a time-of-flight range sensor to estimate the distance traveled along the tunnel axis.

A rather different approach based on deep learning was presented in \cite{mansouri2018towards,mansouri2020deploying,sharif2020mav} for navigation in underground mines.
These works suggested a low-cost UAV system design which relies only on a single camera with LED light bar.
Convolution Neural Network (CNN) was used to classify images into three categories (left, center and right) which was then used to correct the UAV heading to avoid collisions with tunnel walls without relying on localization information.
The UAV motion was controlled in the horizontal plane with a fixed altitude provided by a remote operator.
The performance of such methods relies on how good the training dataset is which can be challenging when deployed in new environments.

On the other hand, some reactive methods have been proposed to address the general problem of navigation in tunnel-like environments such as \cite{savkin2017method,matveev2018proofs,matveev2020method} which rely only on local sensory depth information of the surrounding tunnel walls.
In these works, a 3D nonholonomic kinematic model is considered for the motion, and rigorous mathematical proofs of the methods' performance are provided.
In \cite{savkin2017method}, a control law was developed to maintain a movement in a direction parallel to the tunnel axis while keeping a safe distance from tunnel walls. 
Alternatively, \cite{matveev2018proofs,matveev2020method} presented a method based on estimating a direction parallel to a nearby sensed patch of the tunnel surface in the direction of progressive motion through the tunnel.

\subsection{Aims \& Contributions}

The main contributions of this chapter can be highlighted as follows:
\begin{itemize}
	\item A novel collision-free autonomous navigation method is proposed in this work for UAVs flying in unknown 3D tunnel-like environments.
	\item Rigorous mathematical proof is provided, in contrast to many of the existing methods, to show that our method can safely guide the UAV to progressively advance through the environment.
	\item Detailed implementation approach is suggested for quadrotor UAVs considering the system dynamics and suggesting a low-level sliding mode controller design.
	\item Perception pipelines and algorithms with different computational cost based on the suggested method are proposed for simple and robust implementations with narrow field-of-view sensors.
	\item Experimental results with a quadrotor are given to further validate the overall approach and discuss some of the practical aspects to consider.
\end{itemize}
The novelty of the approach is that it can handle movements in tunnel-like environments that changes shape and direction in 3D in a reactive manner where it is not suitable to use any of the existing 2D reactive approaches as they constraint UAV movement to some fixed altitude.
On contrary to available 3D planning-based approaches, our method can provide a computationally-light solution for the autonomous navigation problem which can be suitable for vehicles with limited resources.
Motion decisions are mainly based on available measurements from onboard sensors to guide the UAV with no need for accurate localization.
The suggested method can also benefit from available local maps of the surroundings if one exists.

The general idea adopted here is to move the UAV towards estimated three-dimensional (3D) points on the tunnel curvy axis in the direction of progressive movement through the tunnel. 
These points are interpreted from available depth measurements of the tunnel walls which can be for example in the form of 3D point clouds represented in a sensor-fixed coordinate frame.
Note that we do not consider environments filled with obstacles.
However, it is possible to extend our approach to consider those environments  by combining it with a reactive obstacle avoidance law using a behavior-based control approach; for example, see \cite{hoy2015algorithms,matveev2011method,wang2018strategy,elmokadem2019reactive} and references therein.
The use of a general kinematic model for the development makes our approach applicable to vehicles moving in 3D such as UAVs of different types (multi-rotors and fixed-wing) and autonomous underwater vehicles.
Slight modifications could be done to take into consideration nonholonomic constraints related to some vehicles such as fixed-wing UAVs.
Moreover, implementation details for quadrotor-type UAVs are also presented in this work considering the system dynamics and suggesting a low-level sliding mode controller design.
Computer simulations and practical experiments were carried out to evaluate the performance of our approach in several environments.
A realistic sensing model was used in all simulations in addition to considering noisy measurements to investigate the robustness of the suggested method, and different sensors configurations and perception algorithms were used in the real experiments.

\subsection{Chapter Outline}

This chapter is structured as follows.
The UAV navigation problem in tunnel-like environments is defined in \cref{sec:problem} considering a general kinematic model.
The proposed navigation algorithm is then presented in \cref{sec:nav}.
Our navigation algorithm is first validated through several simulation scenarios considering different environments, the details and results of these simulations are given in \cref{sec:sim}.
After that, implementation details with quadrotor UAVs are presented in \cref{sec:UAVimpl}. 
Proof-of-concept experiments were carried out to validate the performance of our navigation method.
The used quadrotor UAV system and the experiment setup are described in \cref{sec:exp} along with the results.
Finally, this work is concluded in \cref{sec:conc}.

\section{Kinematic Model and Navigation Problem}\label{sec:problem}
We consider an autonomous UAV whose motion is
described by the kinematic model:
let $c(t):=[x(t),y(t),z(t)]$
be the three-dimensional vector of the UAV's Cartesian coordinates defined in a world (inertial) reference frame.
Then, the motion of the UAV is described by the equation:
\begin{equation}
	\label{1}
	\dot{c}(t) = V(t).
\end{equation}
Here, $V(t)\in \R^3$ is the linear velocity vector, $\|V(t)\|=v$ for all $t$, where $v>0$ is some given constant, and  $\|\cdot\|$ denotes the standard Euclidean vector norm.
The vector variable
$V(t)$ is the  control input,  $v$ is the speed or linear velocity of the UAV, hence, the UAV is moving with a constant speed.
We assume that the control input $V(t)$ is updated at discrete times $0,\delta,2\delta,3\delta, \ldots$:
\begin{equation}
	\label{2}
	V(t):=V_k~~~\forall t\in [k\delta,(k+1)\delta),~~~\forall k=0,1,2,\ldots
\end{equation}
where $\delta>0$ is the sampling period.
The kinematics of many unmanned aerial and underwater vehicles can be described by the  model 
(\ref{1}) or its slight modifications. %

We consider a quite general three-dimensional problem of autonomous UAV navigation in unknown tunnel-like environments with
collision  avoidance. 

\begin{definition}
	\label{DT}
	Let $I$ be a straight line in $\R^3$, and $P$ be a closed, bounded, connected and linearly connected planar set. Then, the three dimensional set ${\cal T}:=I\times P$ is called a perfect
	cylindrical tunnel, and the straight line $I$ is called the axis of the perfect cylindrical tunnel ${\cal T}$ (see \cref{fig:ch6:tunnIllustration1}).  Furthermore, the set ${\cal W}$ of all the boundary points of ${\cal T}$
	is called the wall of the  perfect tunnel ${\cal T}$. 
	Furthermore, let $C$ be a circle in $\R^3$, and $P$ be a closed, bounded, connected and linearly connected planar set. Then, the three dimensional set ${\cal T}:=C\times P$ is called a perfect torus-shaped tunnel, and the circle $C$ is called the axis of the perfect torus-shaped tunnel ${\cal T}$ (see \cref{fig:ch6:tunnIllustration2}).
\end{definition}

Now we can introduce the following definition generalizing Definition \ref{DT}.
\begin{definition}
	\label{DT1}
	Let $C$ be a smooth non-self-intersecting infinite  (or closed) curve in $\R^3$. Assume that for any point $q\in C$ there exists a closed, bounded, connected and linearly connected planar set $P(q)$ intersecting $C$ at the only point $q$ and such that the plane of $P(q)$ is orthogonal to $C$ at the point $q$. Also, we assume that $P(q_1)$ and $P(q_2)$ do not overlap for any $q_1\neq q_2$. Then, the three dimensional set ${\cal T}:=\cup_{q\in C} P(q)$ is called a deformed cylindrical (or torus-shaped) tunnel, and the curve $C$ is called the curvy axis of the deformed tunnel ${\cal T}$, see \cref{fig:ch6:tunnIllustration3,fig:ch6:tunnIllustration4}.
	Furthermore, the set ${\cal W}$ of all the boundary points of ${\cal T}$ is called the wall of the  deformed cylindrical (torus-shaped) tunnel ${\cal T}$.
\end{definition}

It is obvious that perfect tunnels are  special cases of  deformed tunnels where the axis is either a straight line or a circle and all sets $P(q)$ are identical.

{\bf Notations:} We introduce some curvilinear coordinate along the curvy axes $C$ so that the difference of the coordinates of any two points of $C$ is the length of the segment of $C$ between them.
In the case of a deformed cylindrical tunnel, the curvilinear coordinate takes values in $(-\infty, +\infty)$, and in the case of deformed torus-like tunnel, the curvilinear coordinate takes values in $[0,L)$ where $L>0$ is the length of the closed axis curve $C$.
By Definition \ref{DT1}, for any point $a$ in the deformed tunnel $\mathcal{T}$,  there exists a unique
$q(a)\in C$ such that $a\in P(q(a))$. Let $\tilde{q}(a)$ denote the curvilinear coordinate of $q(a)$. Also, let $r(a)$ denote the distance between the points $a$ and $q(a)$. Moreover, $w(a)$ will denote
the tangent vector to the curve $C$ at the point $q(a)$, see \cref{fig:ch6:tunnIllustration}. Furthermore, let $a$ be some point in the deformed tunnel,  $F$ be some vector, and
$D_2>D_1\geq 0$ be given numbers. We introduce the points $O_1(a,F)$ and $O_2(a,F)$ ahead of the point $a$ at the distances $D_1$ and $D_2$, respectively, in the direction of the vector $F$.
Let $\mathcal{P}_1(a,F)$ and $\mathcal{P}_2(a,F)$ be the planes that are
orthogonal to $F$ and contain the points $O_1(a,F)$ and $O_2(a,F)$, correspondingly; see \cref{fig:ch6:illustration} for $a=c(k\delta)$.
Then, let $G_1(a,F)\in \mathcal{P}_1(a,F)$ and $G_2(a,F)\in \mathcal{P}_2(a,F)$ be the gravity centers of the sets of the tunnel wall points belonging to 
the planes $\mathcal{P}_1(a,F)$ and $\mathcal{P}_2(a,F)$, respectively. Furthermore, we introduce the vector $A(a,F)$ departing from $G_1(a,F)$ to $G_2(a,F)$. Moreover, let $h(a,F)$ denote the distance from the point $a$ to the straight line connecting $G_1(a,F)$ and $G_2(a,F)$.

{\bf Available Measurements:} Let $D_2>D_1\geq 0 $ 
be  given.   
We assume that for any time $t\geq 0$, the UAV can measure the coordinates
of all the points of the tunnel wall lying in the planes $\mathcal{P}_1(c(t),V(t))$ and $\mathcal{P}_2(c(t),V(t))$; \cref{fig:ch6:illustration}. Hence, the UAV can calculate
the vector $A(c(t),V(t))$ and the number $h(c(t),V(t))$.

\begin{definition}
	\label{D}
	Let $d_{safe}>0$ be a given constant, and 
	let $d(t)$ denote the distance between the robots' coordinates $c(t)$ and the wall of the deformed cylindrical or torus-shaped tunnel $\mathcal{T}$.
	A UAV navigation law is said to be  safely navigating through the deformed tunnel $\mathcal{T}$ if 
	\begin{eqnarray}
		\label{dist}
		d(t)>d_{safe}~~~\forall t\geq 0,
	\end{eqnarray}
	\begin{eqnarray}
		\label{tend}
		\tilde{q}(c(t))\rightarrow\infty~~~as~~~t\rightarrow\infty.
	\end{eqnarray}
	Moreover, a UAV navigation law is said to be  safely navigating through the deformed torus-shaped tunnel ${\cal T}$ if (\ref{dist}) holds and for any $Q\in [0,L)$ there exists a sequence $t_k\rightarrow+\infty$ such that
	\begin{eqnarray}
		\label{tend1}
		\tilde{q}(c(t_k))=Q.
	\end{eqnarray}
\end{definition}

The requirement (\ref{tend}) means that in the case of deformed cylindrical tunnel, the UAV will go to infinity inside the tunnel, and the requirement (\ref{tend1}) means that in the case of deformed torus-shaped tunnel, the UAV will do infinitely many loops  inside the tunnel.

{\bf Problem Statement:} Our objective is  to design a navigation law for quadrotor UAVs to safely navigate through the deformed cylindrical or torus-shaped tunnel $\mathcal{T}$.

\begin{figure}[!htb]
	\centering
	\begin{adjustbox}{minipage=\linewidth,scale=1.0}
		\begin{subfigure}[t]{0.4\textwidth}
			\centering
			\includegraphics[width=\linewidth]{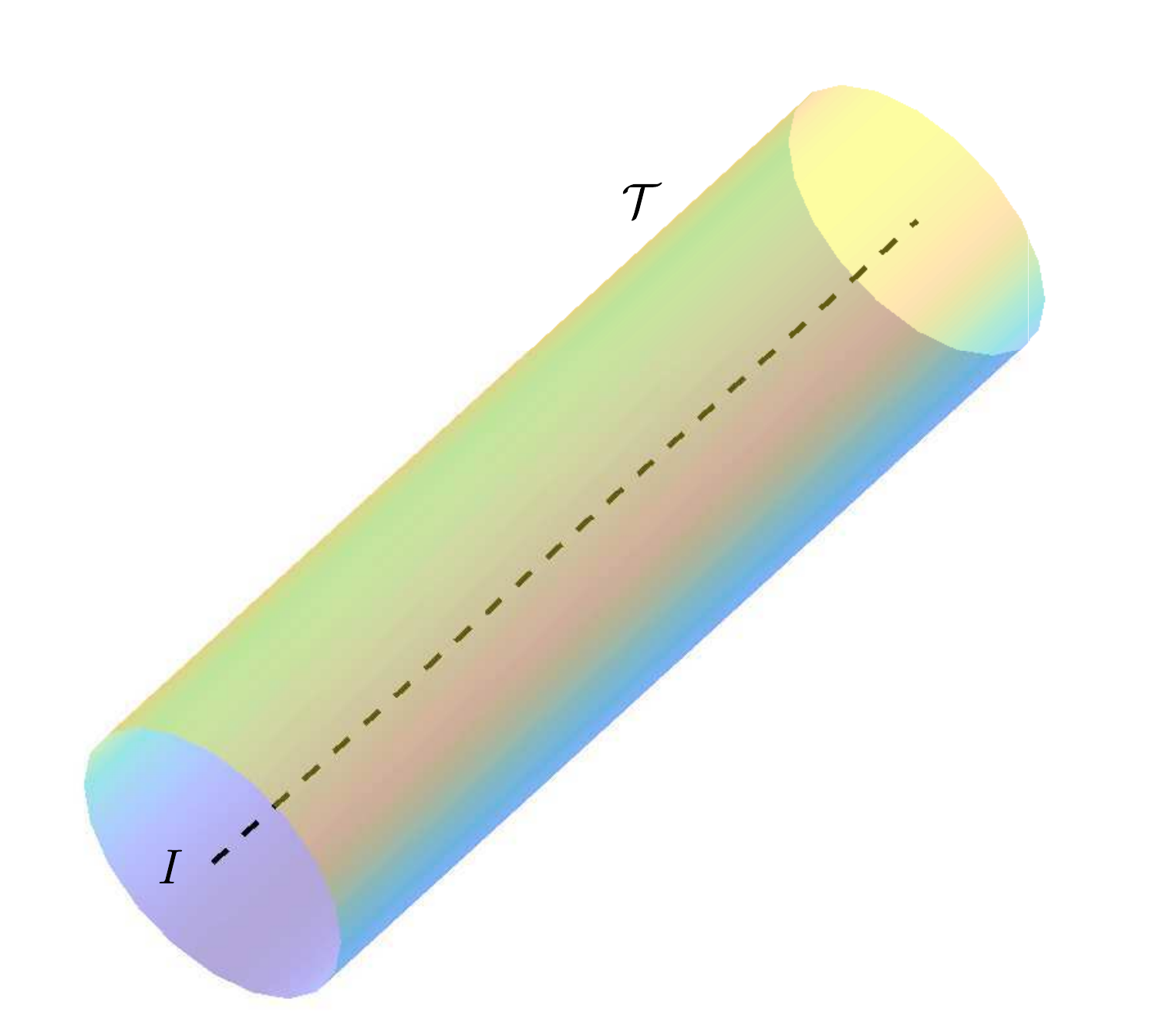} 
			\caption{Perfect cylindrical Tunnel}
			\label{fig:ch6:tunnIllustration1}
		\end{subfigure}
		\hfill
		\begin{subfigure}[t]{0.4\textwidth}
			\centering
			\includegraphics[trim={10cm 2cm 12cm 6cm},clip,width=\linewidth]{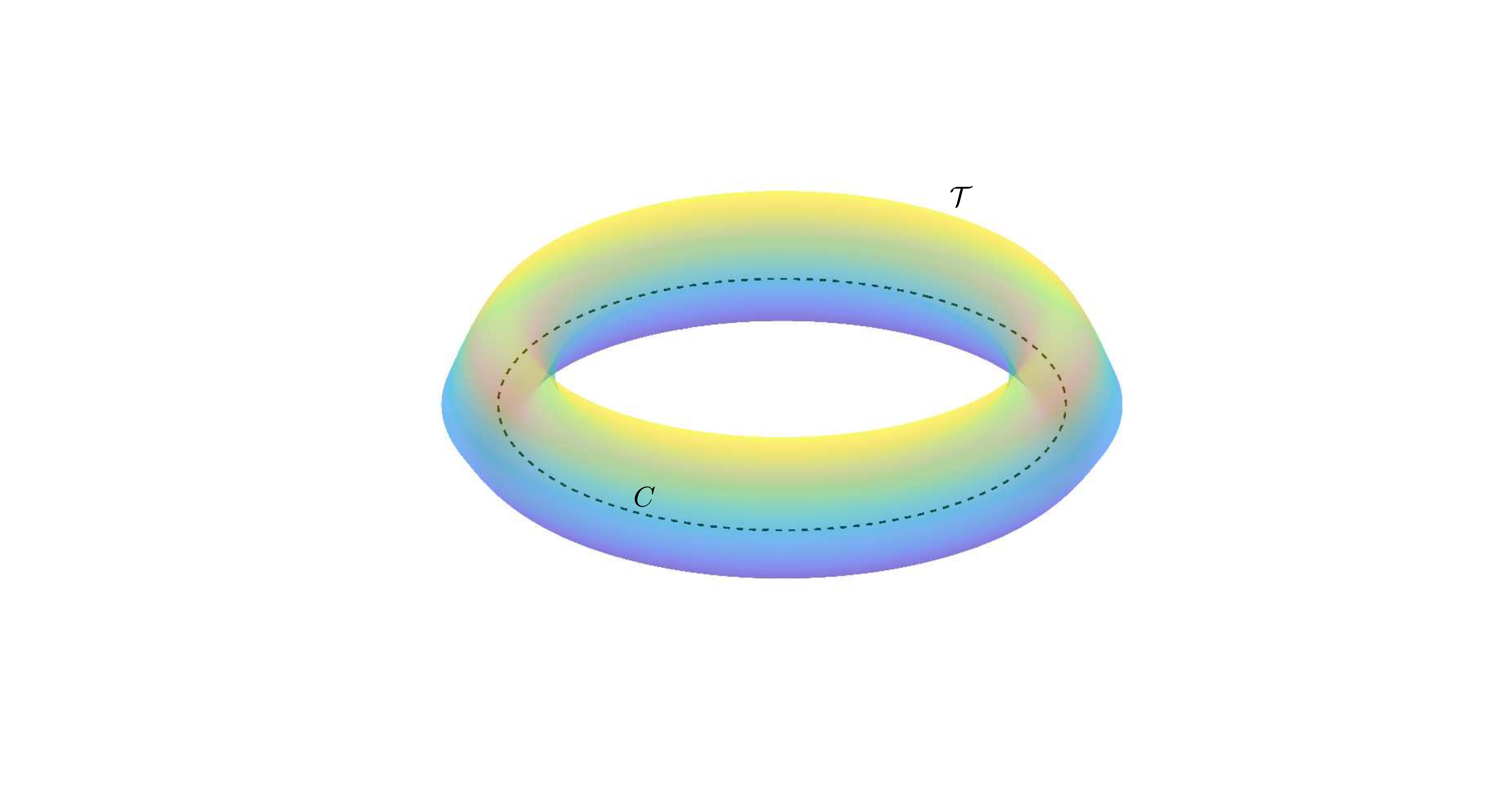} 
			\caption{Perfect torus-shaped tunnel}
			\label{fig:ch6:tunnIllustration2}
		\end{subfigure}
	
		\begin{subfigure}[t]{0.4\textwidth}
			\centering
			\includegraphics[width=\linewidth]{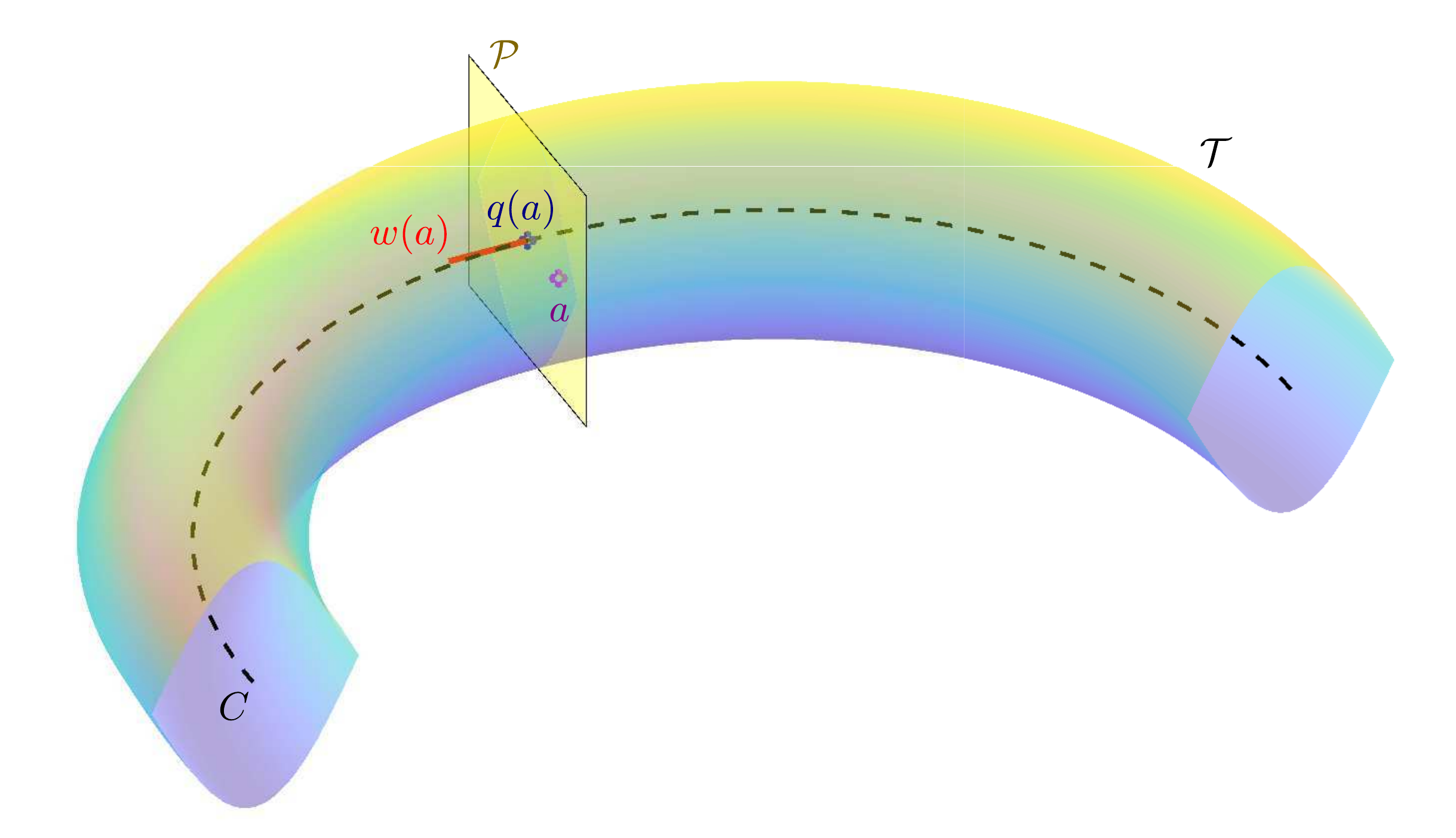} 
			\caption{Deformed cylindrical tunnel}
			\label{fig:ch6:tunnIllustration3}
		\end{subfigure}
		\hfill
		\begin{subfigure}[t]{0.4\textwidth}
			\centering
			\includegraphics[trim={12cm 0cm 12cm 0cm},clip,width=\linewidth]{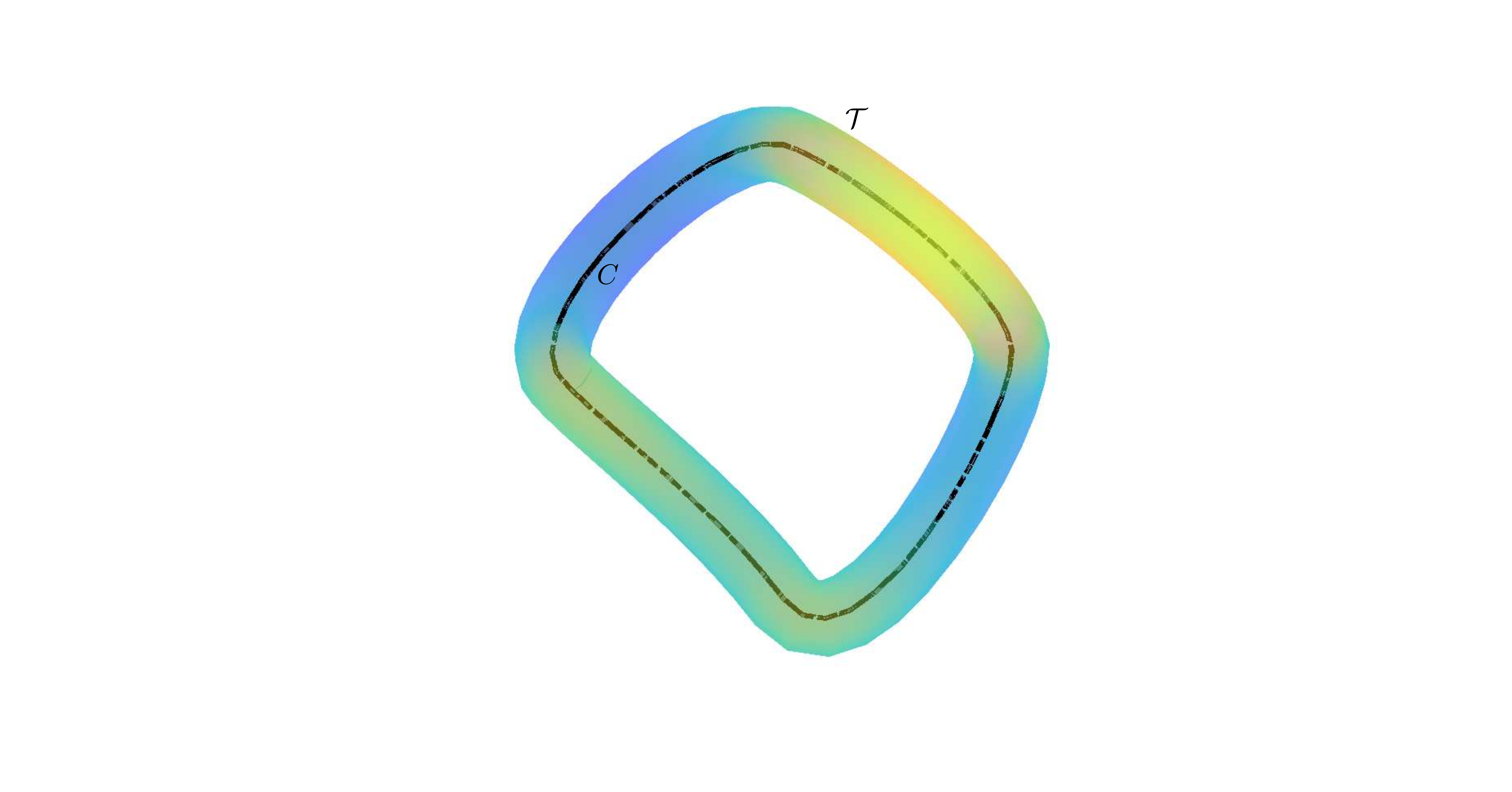} 
			\caption{Deformed torus-shaped tunnel}
			\label{fig:ch6:tunnIllustration4}
		\end{subfigure}
	\end{adjustbox}
		\caption{An illustration of tunnels definitions}
		\label{fig:ch6:tunnIllustration}
\end{figure}

\section{Navigation Algorithm}\label{sec:nav}

In the following assumptions the deformed tunnel can be either cylindrical or torus-shaped.
Suppose that there exist constants $0< \alpha <\frac{\pi}{2}, \beta>0, \beta_0>0, R>0$ such that $\beta+\beta_0<\alpha$, and the following assumptions hold.

\begin{assumption}
	\label{A1}
	At time $0$ the UAV is inside the deformed tunnel $\mathcal{T}$, i.e. $c(0)\in \mathcal{T}$, and $r(c(0))\leq R-\epsilon_0$ where $\epsilon_0:=v\delta$. Moreover, the UAV knows some
	estimate $V_0$ of the tangent vector $w(c(0))$ such that the  angle between the vectors $V_0$ and $w(c(0))$ is less than $\alpha$.
	This $V_0$ is used as the first input in the controller (\ref{2}).
\end{assumption}

\begin{assumption}
	\label{A2}
	Any set $P(q)$ of the deformed tunnel contains the disc $W_R$ consisting of the points $H$ such that $r(H)\leq R$. Moreover, for all such points $H$, the safety constraint (\ref{dist}) holds.
\end{assumption}

\begin{assumption}
	\label{A3}
	For any points $a_1,a_2\in W_R$ and any vector $F_1$ such that $\|a_1-a_2\|<\epsilon_0$ and the angle between the vectors $F_1$ and $w(a_1)$ is less than $\alpha$, the angle between
	the vectors $A(a_1,F_1)$ and $w(a_2)$ is less than $\beta$.
\end{assumption}

\begin{assumption}
	\label{A4}
	For any points $a_1,a_2\in W_R$ and any vector $F_1$ such that $\|a_1-a_2\|<\epsilon_0$, and the angle between the vectors $F_1$ and $w(a_1)$ is less than $\alpha$, the inequality  $|h(a_1,F_1)-r(a_2)|<\epsilon_1$ holds where $\epsilon_1:=\frac{v\delta\sin{\beta_0}}{2}$.
\end{assumption}

\begin{remark}
	\label{R1}
	In the case of a deformed cylindrical tunnel, Assumptions \ref{A3} and \ref{A4} describe how close the deformed tunnel is from a perfect tunnel, as it is obvious that for any perfect tunnel,
	these assumptions hold with $\beta=\epsilon_1=0$. In the case of a deformed torus-shaped tunnel, Assumptions \ref{A3} and \ref{A4} hold as the minimum curvature of the axis $C$ is small enough.
\end{remark}

We introduce the vector $B(c(t),V(t))$ such that the angle between the vectors $B(c(t),V(t))$ and $A(c(t),V(t))$ equals $\beta_0$,  and 
$\|B(c(t),V(t))\|=v$. %
It is clear from the construction that $A(c(t),V(t))\neq 0$.
Now,
introduce the following navigation law defined by (\ref{2}) and the following rule:
\begin{equation}
	\label{cont}
	\begin{adjustbox}{scale=0.85}
		$
		V_k:= \left\{
		\begin{array}{lcr}
			\frac{v}{\|A(k\delta)\|}A(k\delta),\ \ \ \ \ \ \ \ &  h(c(k\delta),V(k\delta))<R-2\epsilon_0 & {\bf (M1)}\\
			B(c(k\delta),V(k\delta)),\ \ \ \ \ \ \ \ &   h(c(k\delta),V(k\delta))\geq R-2\epsilon_0 & {\bf (M2)} \\
		\end{array} \right.
		$
	\end{adjustbox}
\end{equation}
for $k=1,2,\cdots$.

Now, we are in a position to present the main theoretical result of this chapter.

\begin{theorem}
	\label{T1}
	Let a constant $d_{safe}>0$ and a deformed cylindrical or torus-shaped tunnel $\mathcal{T}$ be given. Suppose that $0< \alpha <\frac{\pi}{2}, \beta>0, \beta_0>0, R>0$,
	$\epsilon_0:=v\delta$ and $\epsilon_1:=\frac{v\delta\sin{\beta_0}}{2}$  are constants such that $\beta+\beta_0<\alpha$ and  Assumptions \ref{A1}--\ref{A4} hold.
	Then, the UAV navigation law (\ref{2}), (\ref{cont}) with $V_0$ from Assumption \ref{A1} is safely navigating through the deformed tunnel $\mathcal{T}$.
\end{theorem}

{\bf Proof of Theorem \ref{T1}:}  At any time, the navigation law (\ref{2}), (\ref{cont}) operates in either mode ${\bf (M1)}$ or ${\bf (M2)}$. 
In any case, over any time interval $(k\delta,(k+1)\delta]$, the UAV makes the distance $\epsilon_0=v\delta$. Therefore, 
in the mode ${\bf (M1)}$, due to Assumption \ref{A3}, the angle between  the vectors $V(t)= \frac{v}{\|A(k\delta)\|}A(k\delta)$ and $w(c(t))$ is less than 
$\beta$. Correspondingly, in the mode ${\bf (M2)}$, due to Assumption \ref{A3}, the angle between  the vectors $V(t)= \frac{v}{\|A(k\delta)\|}A(k\delta)$ and $w(c(t))$ is less than $\beta+\beta_0<\frac{\pi}{2}$. Since $w(c(t))$ is a tangent vector of the tunnel axis, this implies that
$\tilde{q}(c(t))\geq \tilde{q}(c(0))+t\cos(\beta+\beta_0)\rightarrow \infty$. Therefore, the condition (\ref{tend}) of Definition \ref{D} holds.
Furthermore, if over some time interval $(k\delta,(k+1)\delta]$, the UAV operates in the  mode ${\bf (M1)}$, then it follows from (\ref{cont}) that 
$h(c(k\delta),V(k\delta))<R-2\epsilon_0$, and since the UAV makes the distance $\epsilon_0=v\delta$ over this time interval, this implies that
$h(c(t),V(t))<R-\epsilon_0$ for all $t\in (k\delta,(k+1)\delta]$. If over some time interval $(k\delta,(k+1)\delta]$, the UAV operates in the  mode 
${\bf (M2)}$, then it follows from (\ref{cont}) and Assumption \ref{A4} that $h(c((k+1)\delta),V((k+1)\delta))\leq h(c((k\delta),V(k\delta))$. This and (\ref{cont}) imply that $h(c(t),V(t))<R$ in the mode ${\bf (M2)}$. Therefore, $h(c(t),V(t))<R$ for all $t$, hence, according to Assumption \ref{A2}, 
the requirement  (\ref{dist}) of Definition \ref{D} holds.
This completes the proof of Theorem \ref{T1}.

\begin{remark}
	Note that we do not consider tunnels whose axes branch off at some points according to the problem definition in \cref{sec:problem}.
	However, it is possible to extend our navigation algorithm defined by the control law \eqref{cont} to address such cases by defining a third mode ${\bf (M3)}$.
	This mode could be responsible for guiding the UAV through one of the branches selected arbitrary or based on some heuristics.
	The switching mechanism from and to this mode can be based mainly on interpreting the tunnel axis branching off scenario from sensors measurements. 
\end{remark}

\begin{figure}[!htb]
	\centering
	\includegraphics[clip,width=0.5\linewidth]{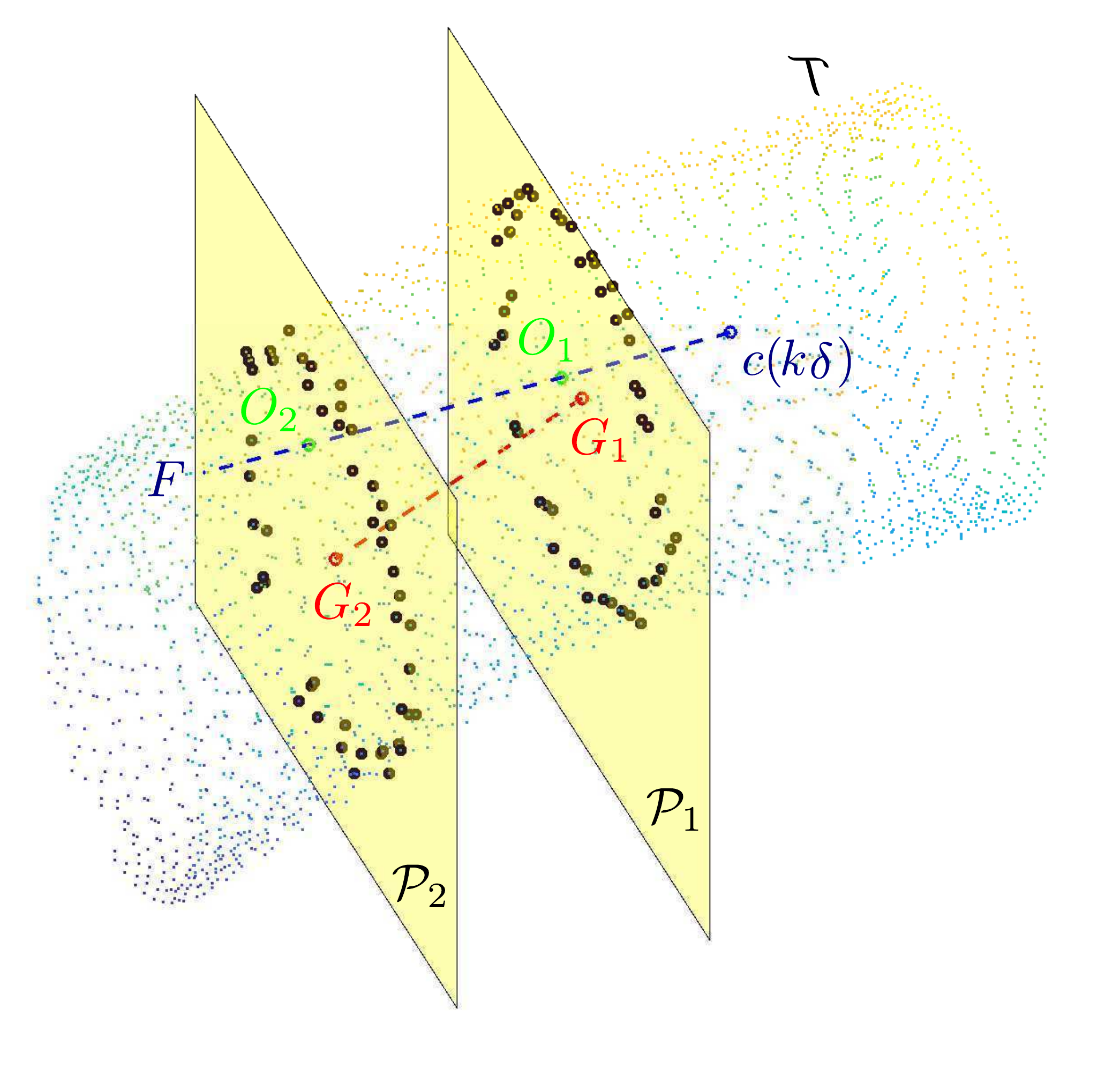}%
	\caption{An illustration of available measurements along with the corresponding gravity centers computed according to our method}
	\label{fig:ch6:illustration}
\end{figure}

\section{Computer Simulations}\label{sec:sim}

The proposed navigation strategy was validated through many simulation scenarios.
Several 3D tunnel-like environments have been considered including tunnels with nonsmooth walls and sharp turnings.
In all simulations, the environment was represented using a 3D point cloud.
The UAV sensing module has only access to a fraction of the environment limited by some sensing range $d_{sensing}$ mimicking the behavior of onboard sensors commonly used in practice.
Additionally, noisy sensor measurements were also considered in one of the simulation cases.

Our navigation algorithm provided in \eqref{cont} was implemented in these simulations as follows.
Initially, we provide the first control input $V_0$ based on some initial knowledge about the environment in accordance with \cref{A1}.
This assumption is valid in practice at the time of UAV deployment before the mission starts.
At each subsequent time step $k\delta$, a heading unit vector $F$ represents the current direction of motion is determined using:
\begin{equation}
	\label{F}
	F(k\delta) = V((k-1)\delta)/\|V((k-1)\delta)\|
\end{equation}
Then, two points ahead of $c(k\delta)$ are computed in the direction of $F(k\delta)$ using:
\begin{equation}
	\label{equ:Oi}
	\begin{array}{lll}
		O_i(c(k\delta),F(k\delta)) = c(k\delta) + D_i F(k\delta), & i=\{1,2\}
	\end{array}
\end{equation}

Let $\mathcal{W}_s := \{p \in \mathcal{W}: \|p - c(k\delta)\| \leq d_{sensing}\}$ be the fraction of tunnel wall within the sensing range (represented as a point cloud).
We then determine the two sets $\mathcal{O}_1 \subset \mathcal{W}_s$ and $\mathcal{O}_2 \subset \mathcal{W}_s$ of tunnel wall points within sensing range belonging to the planes $\mathcal{P}_1(c(k\delta),F(k\delta))$ and $\mathcal{P}_2(c(k\delta),F(k\delta))$ by a filtering process according to the following:
\begin{equation}
	\label{intersect}
	\begin{array}{lll}
		\mathcal{O}_i := \{p \in \mathcal{W}_s: \langle p - O_i,\ F(k\delta)\rangle = 0\}, & i=\{1,2\}
	\end{array}
\end{equation}
where $\langle\cdot,\cdot\rangle$ is the dot product of the two vectors.
Notice that some tolerance $\epsilon>0$ is used to pick the points within a very small proximity of $\mathcal{P}_1(c(k\delta),F(k\delta))$ and $\mathcal{P}_2(c(k\delta),F(k\delta))$ to handle point clouds discontinuities.
That is, the condition in \eqref{intersect} becomes:
\begin{equation}\label{equ:ptcloud_section}
	\begin{aligned}
		{\cal O}_i := \{p \in {\cal W}_s: |\langle p - O_i,\ F(k\delta)\rangle| \leq \epsilon\}, & i=\{1,2\}
	\end{aligned}
\end{equation}
where $\epsilon$ is some small positive constant.

After that, $G_1$ and $G_2$ are computed as the centroids of $\mathcal{O}_1$ and $\mathcal{O}_2$ respectively.
These can then be used to get $A(k\delta)$, $h(c(k\delta), V(k\delta))$ and $B(c(k\delta), V(k\delta))$ to apply our navigation law \eqref{cont}.

\Cref{fig:ch6:sim1a,fig:ch6:sim1b,fig:ch6:sim1c,fig:ch6:sim1d,fig:ch6:sim1e,fig:ch6:sim1f} present simulation scenarios for six different environments showing the executed paths by the UAV using our navigation algorithm.
Scenarios (a)-(c) considered deformed tunnels with smooth 3D deformations.
On contrary, environments with nonsmooth boundaries were handled in scenarios (d)-(f).
It was observed that the UAV managed to quickly reach and follow the curvy axis $C$ of the tunnel in cases (a)-(c) keeping a safe distance from the tunnel boundary.
In cases like (d)-(f), the UAV could sometimes diverge from moving across $C$ for a short segment when there is a sharp change in the direction of the tunnel boundary.
However, it still manages to maintain a safe distance from the wall.
These results clearly confirms the performance of our control approach.
Even though our algorithm was developed assuming that tunnel walls are smooth, it clearly shows good performance in tunnels with nonsmooth walls and sharp turnings.

An additional simulation scenario was carried out to investigate the robustness of our method against noisy sensor measurements.
The UAV was required to navigate through some pipeline structure as shown in \cref{fig:ch6:sim2}(a).
A Gaussian noise was added to the point cloud seen by the sensing module as presented in \cref{fig:ch6:sim2}(b) along with the executed motion by the UAV (different view prospectives are shown in \cref{fig:ch6:sim3a,fig:ch6:sim3b} for better visualization).
\Cref{fig:ch6:sim4} shows the time evolution of the UAV position $c(k\delta)$.
The actual distance to the tunnel wall during the motion along with the distance based on the noisy point cloud are shown in \cref{fig:ch6:sim5}.
It is clear that the motion executed by the vehicle is collision-free.
Notice that the vehicle gets close to the tunnel walls around $t= 42\ s$ because of the very sharp bend of the pipe structure at that location.
Clearly, these results shows how robust our method is against noisy measurements which is a key feature for practical implementation.
The simulations update time was selected as $\delta = 0.1 s$, and the parameters used for each scenario are provided in \cref{tab:ch6:1}.
Animations of all simulation cases with corresponding time plots showing distance to tunnel walls are available at \href{https://youtu.be/r2Add9lctEU}{https://youtu.be/r2Add9lctEU}.

\begin{table}[ht]
	\centering
	\begin{tabular}{ |c||c|c|c|c|c|c|c|}
		\hline
		\multirow{2}{4em}{Parameters} & \multicolumn{7}{|c|}{Simulation Scenario} \\ \cline{2-8}
		& a & b & c & d & e & f & g \\
		\hline
		$v\ (m/s)$           & 1.0 & 2.0 & 2.0 & 1.0 & 2.0 & 1.0 & 5.0 \\
		$D_1\ (m)$         & 1.0 & 1.5 & 1.0 & 1.0 & 1.0 & 1.0 & 2.0 \\
		$D_2\ (m)$         & 3.0 & 3.0 & 3.0 & 2.5 & 3.0 & 3.0 & 5.0 \\
		$R\ (m)$           & 1.5 & 1.0 & 1.5 & 1.5 & 1.5 & 1.5 & 1.5 \\
		$\beta_0\ (rad)$     & $\pi/4$ & $\pi/5$ & $\pi/5$ & $\pi/4$ & $\pi/5$ & $\pi/5$ & $\pi/5$ \\
		$d_{sensing}\ (m)$ & 20 & 20 & 10 & 30 & 10 & 10 & 25 \\
		\hline
	\end{tabular}
	\caption{Parameters used in simulations}
	\label{tab:ch6:1}
\end{table}

\begin{figure}[!htb]
	\centering
	\begin{adjustbox}{minipage=\linewidth,scale=1.0}
		\begin{subfigure}[t]{0.48\textwidth}
			\centering
			\includegraphics[clip, width=\linewidth]{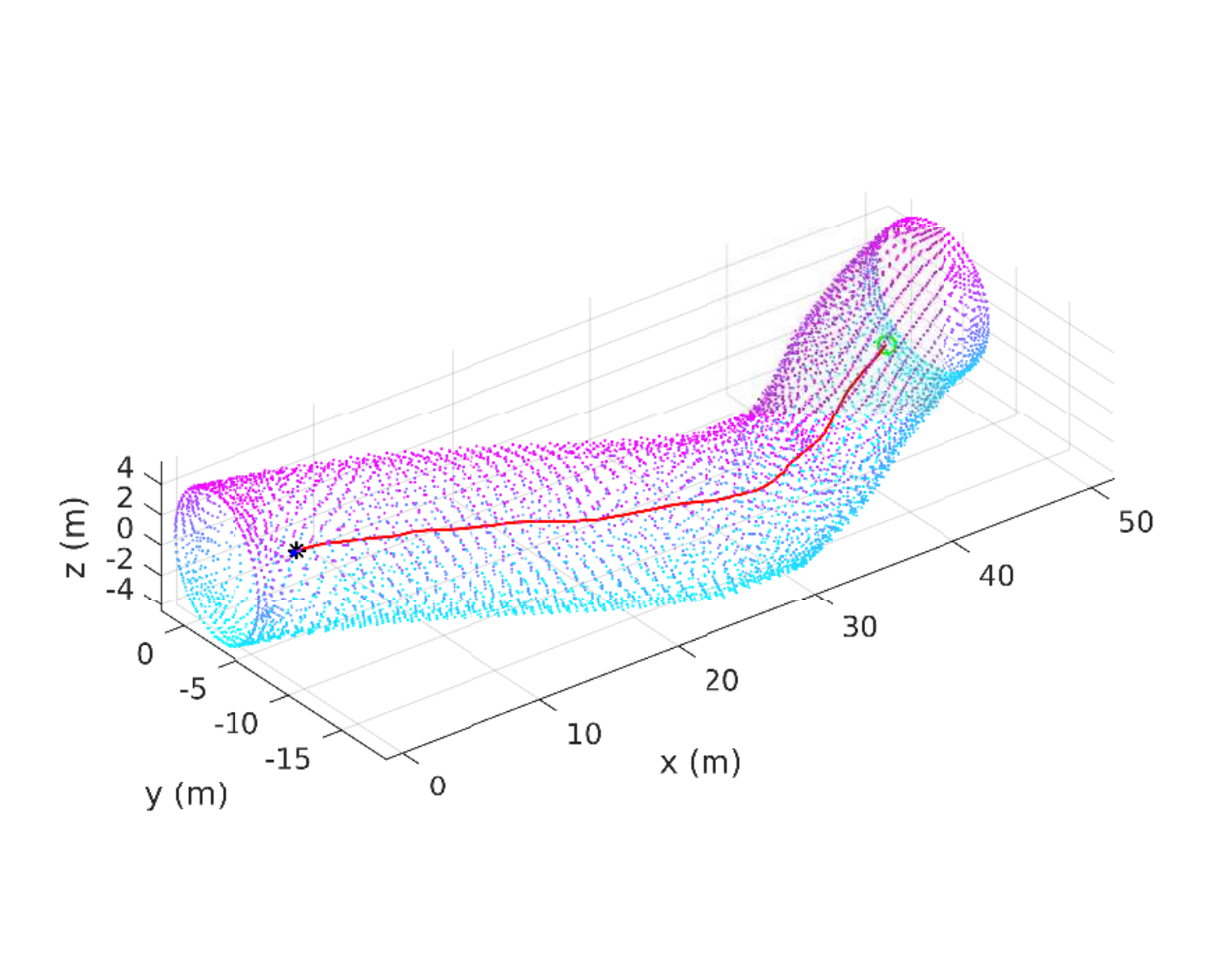} 
			\caption{A pipe with a smooth bend}
			\label{fig:ch6:sim1a}
		\end{subfigure}
		\hfill
		\begin{subfigure}[t]{0.48\textwidth}
			\centering
			\includegraphics[trim={0 0 0 0.5cm}, clip, width=\linewidth]{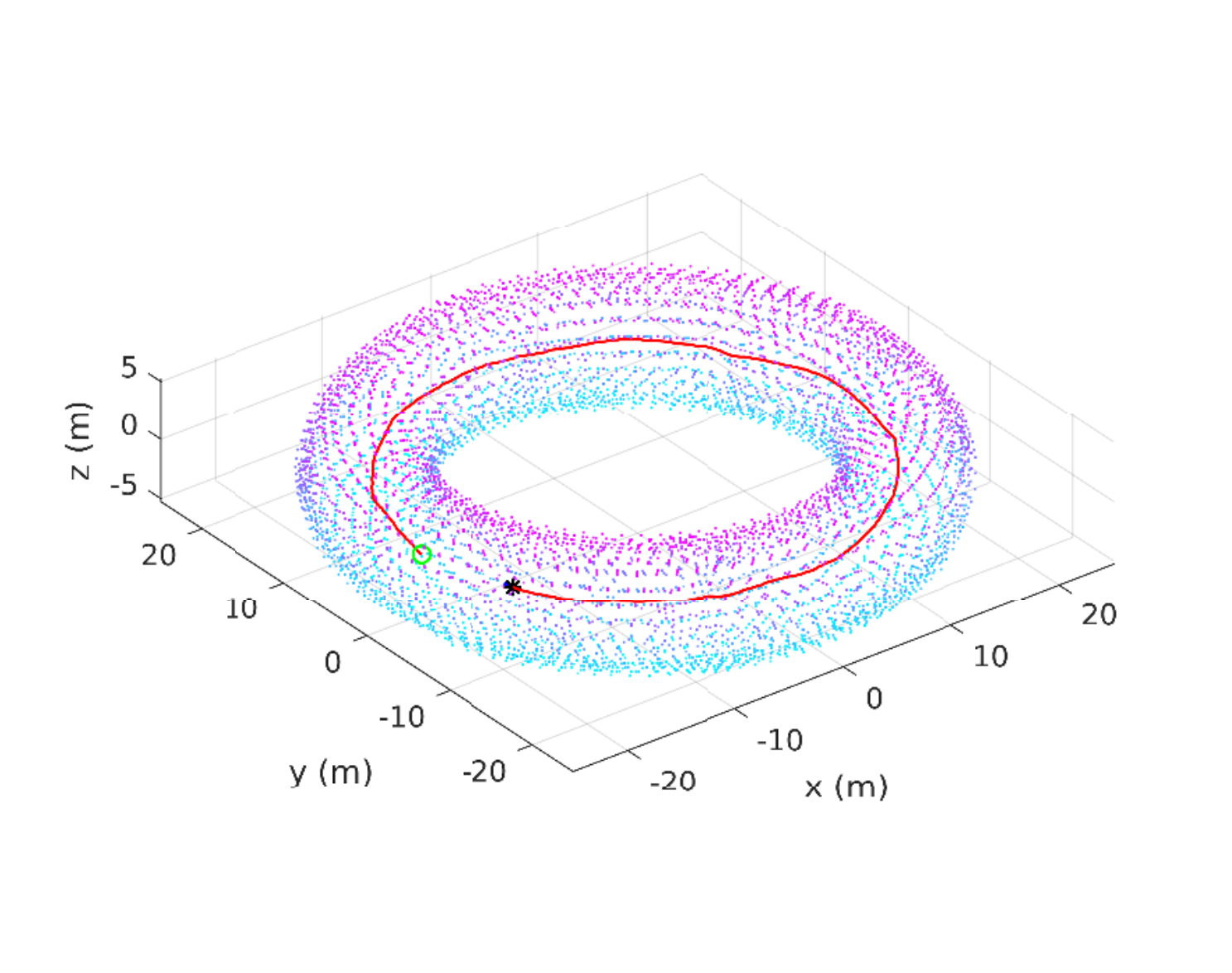} 
			\caption{A torus-shaped tunnel}
			\label{fig:ch6:sim1b}
		\end{subfigure}
		
		\begin{subfigure}[t]{0.48\textwidth}
			\centering
			\includegraphics[clip, width=\linewidth]{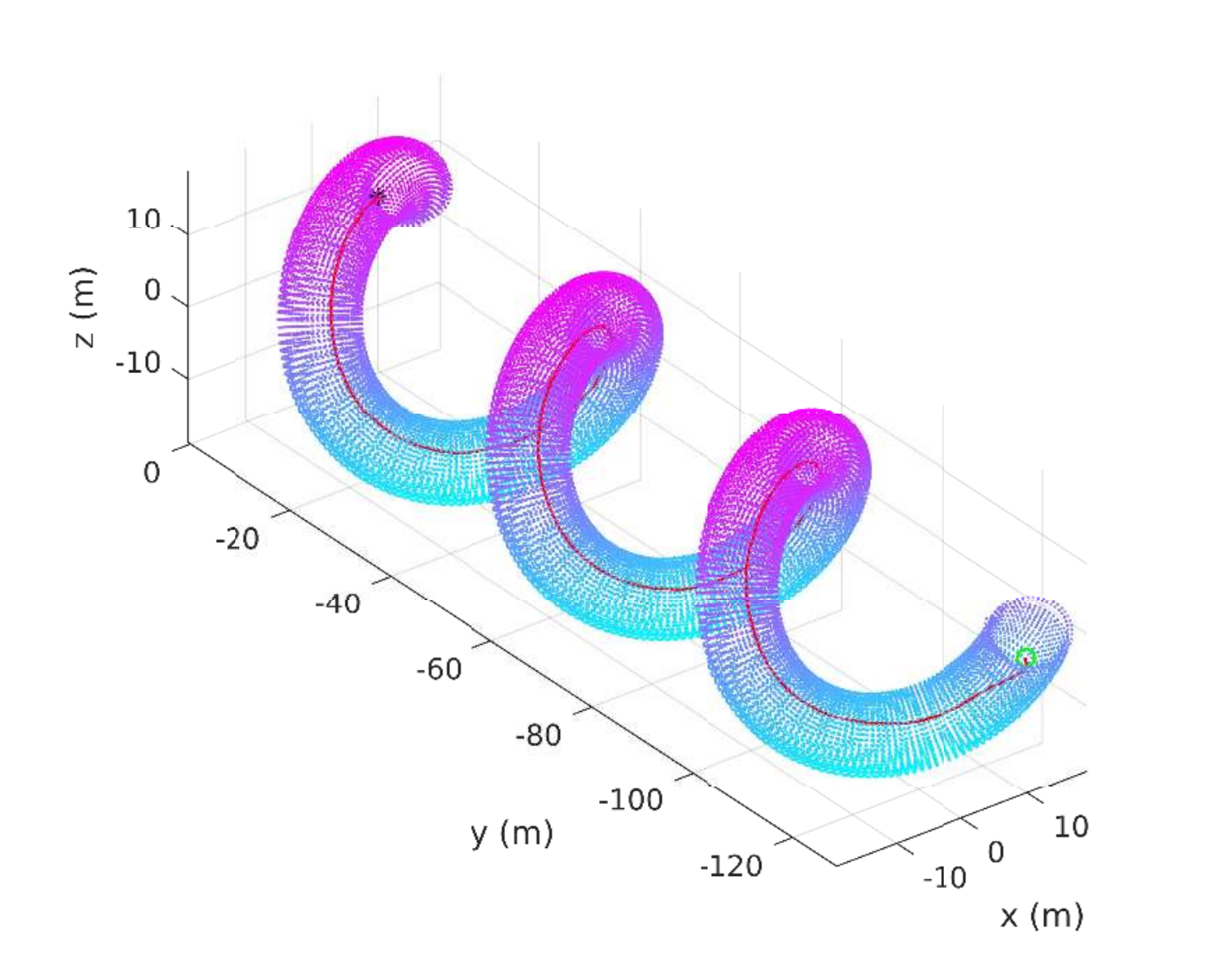} 
			\caption{A helix-shaped tunnel}
			\label{fig:ch6:sim1c}
		\end{subfigure}
		\hfill
		\begin{subfigure}[t]{0.48\textwidth}
			\centering
			\includegraphics[clip, width=\linewidth]{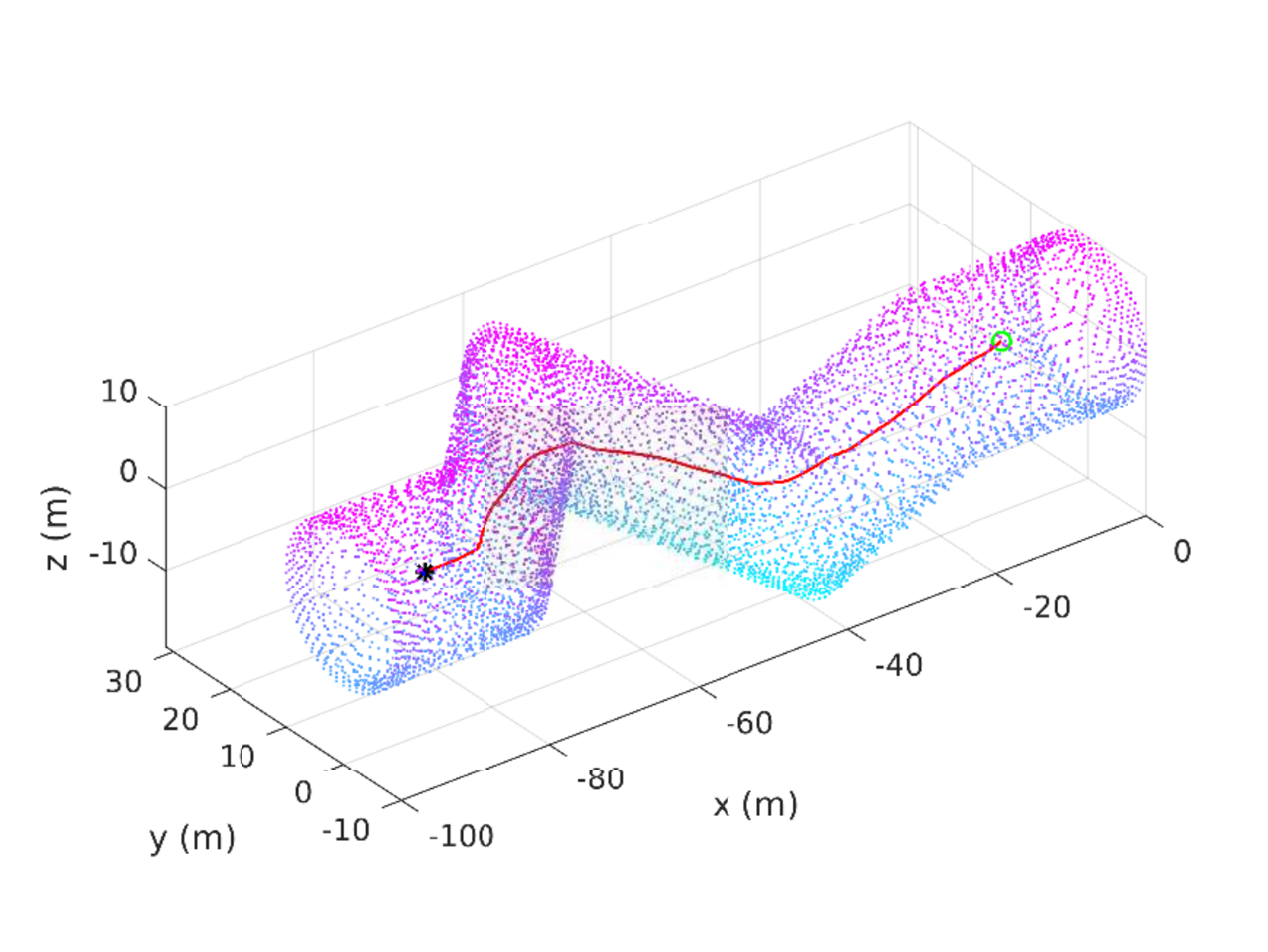} 
			\caption{A pipe with sharp bends}
			\label{fig:ch6:sim1d}
		\end{subfigure}
	
		\begin{subfigure}[t]{0.48\textwidth}
			\centering
			\includegraphics[trim={0 0 0 0.5cm}, clip, width=\linewidth]{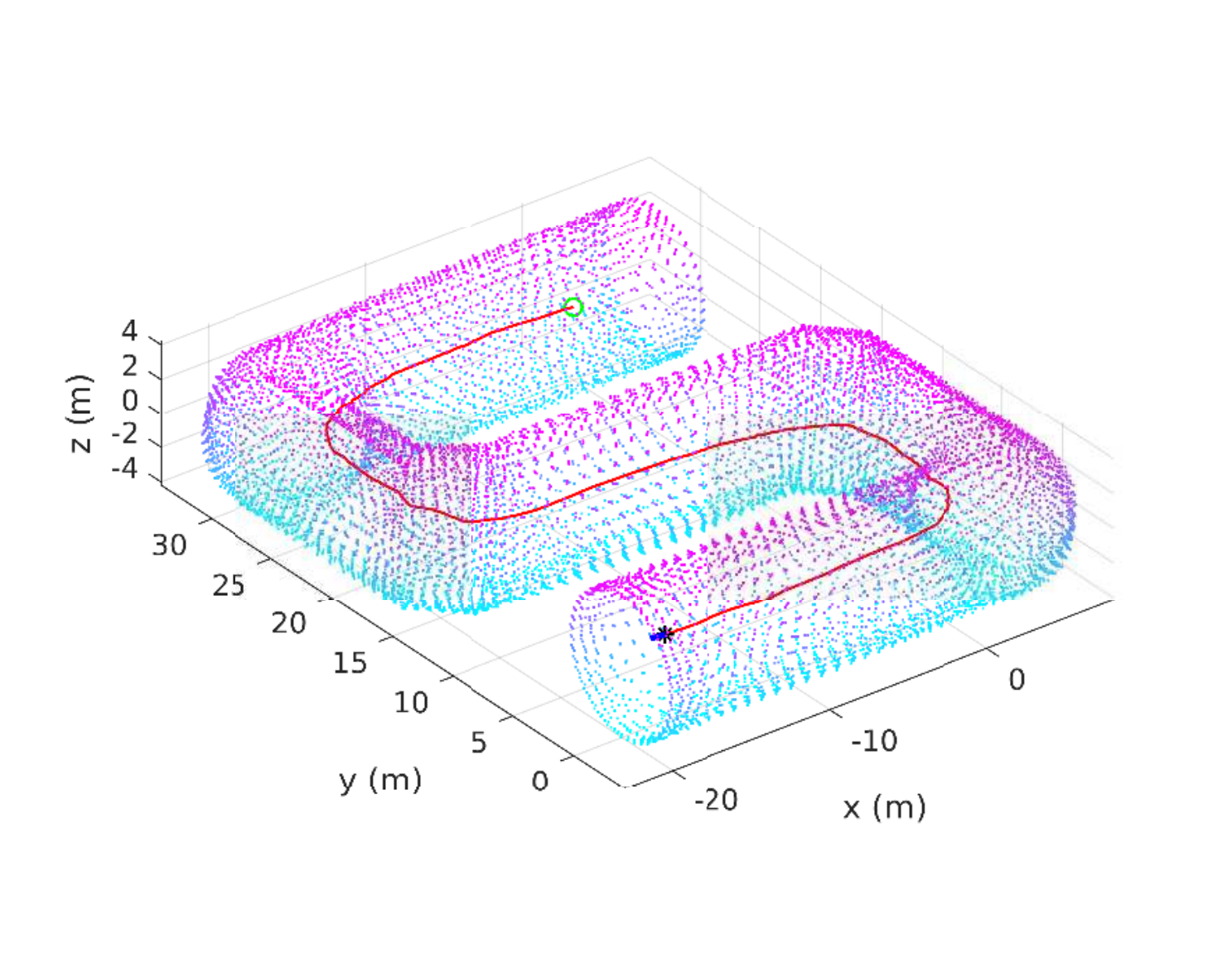} 
			\caption{S-shaped tunnel with sharp edges}
			\label{fig:ch6:sim1e}
		\end{subfigure}
		\hfill
		\begin{subfigure}[t]{0.48\textwidth}
			\centering
			\includegraphics[clip, width=\linewidth]{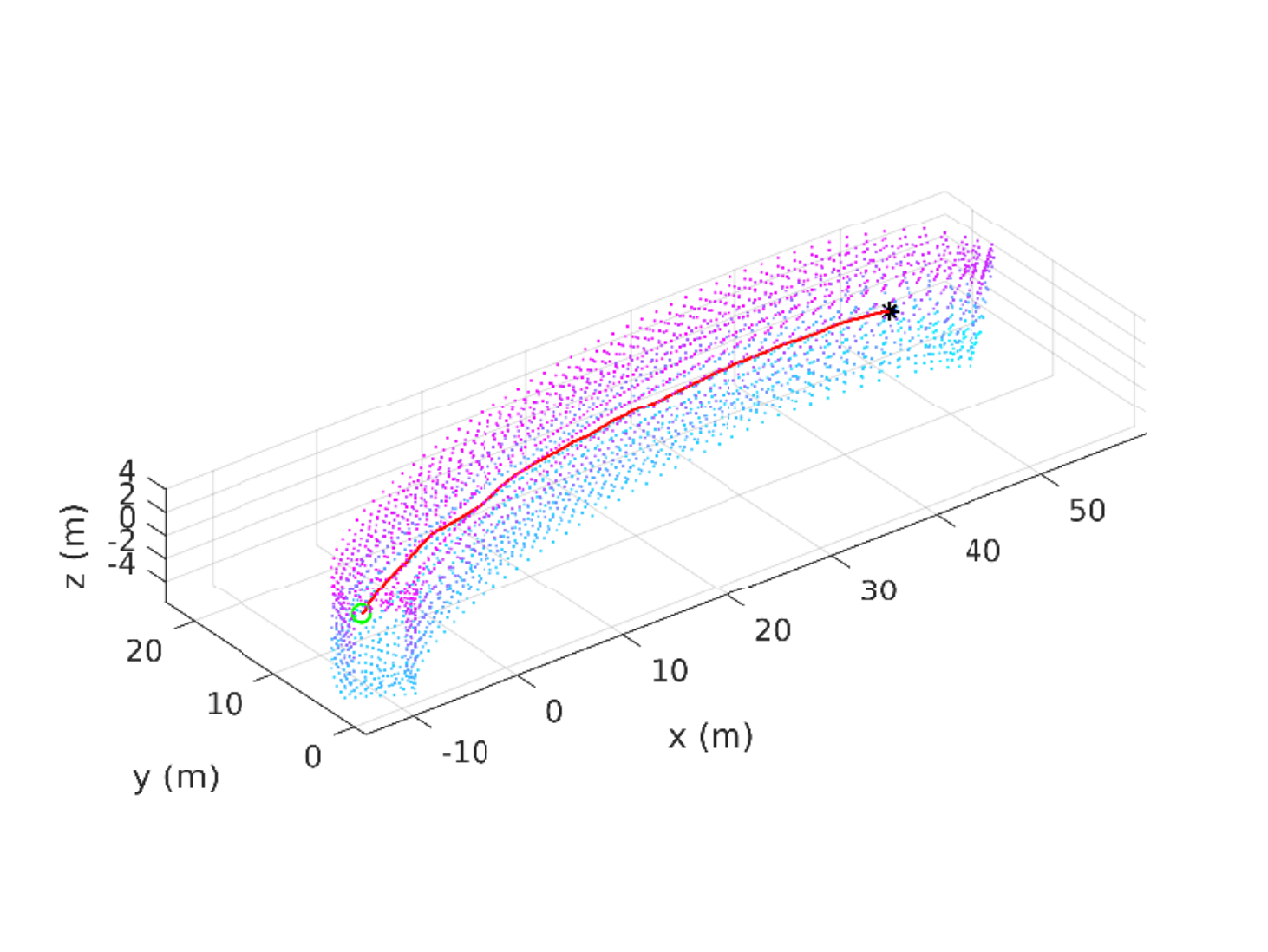} 
			\caption{A rectangular-shaped tunnel}
			\label{fig:ch6:sim1f}
		\end{subfigure}
		\caption{Simulation cases of deformed tunnel environments with different shapes considering smooth (a-c) and nonsmooth (d-f) boundaries}
		\label{fig:ch6:sim1}
	\end{adjustbox}
\end{figure}

\begin{figure}[!htb]
	\centering
	\begin{adjustbox}{minipage=\linewidth,scale=1.0}
		\begin{subfigure}[t]{0.48\textwidth}
			\centering
			\includegraphics[trim={0 0 0 0.5cm}, clip, width=\linewidth]{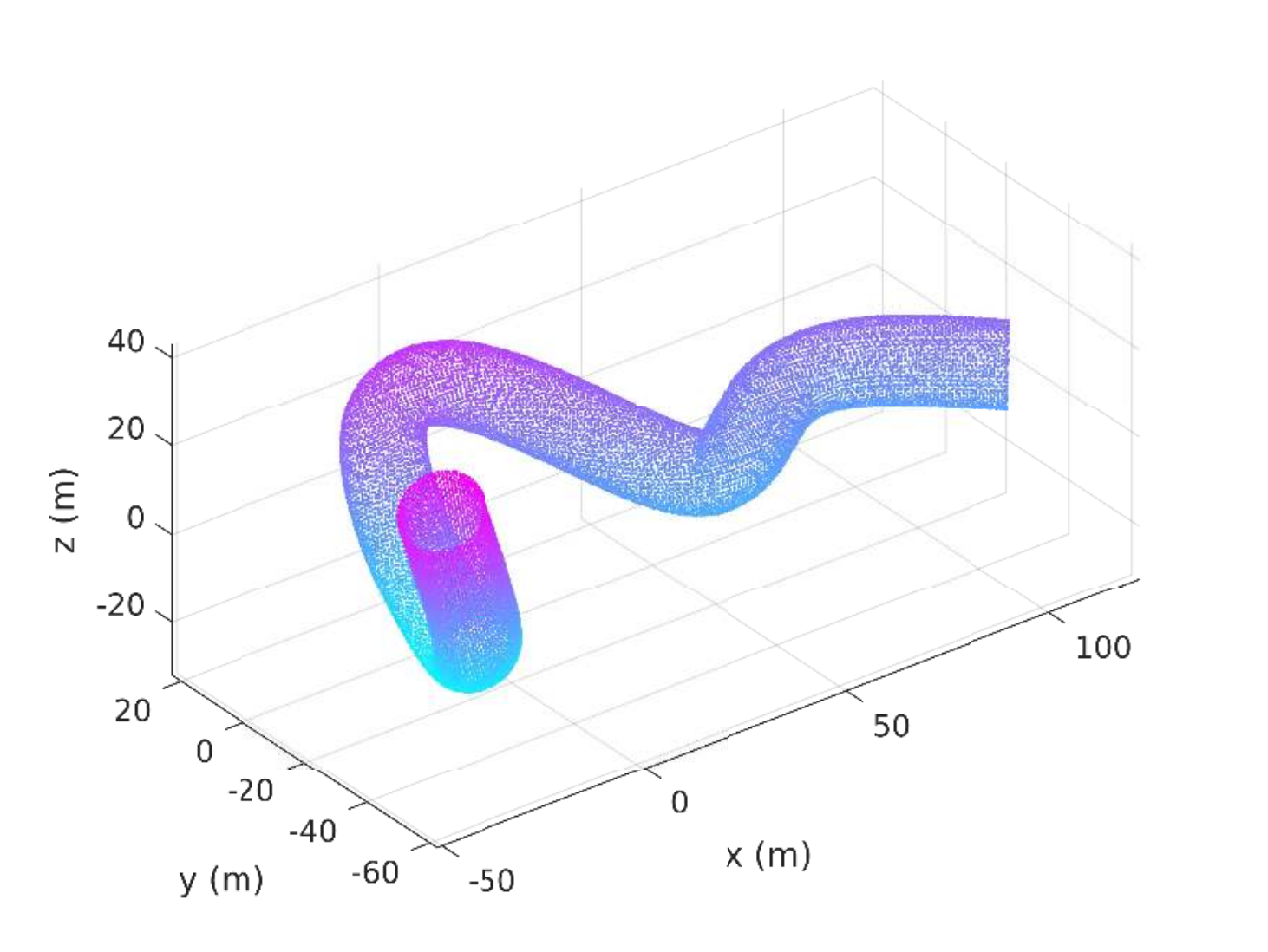} 
			\caption{A complex pipline segment}
			\label{fig:ch6:sim2a}
		\end{subfigure}
		\hfill
		\begin{subfigure}[t]{0.48\textwidth}
			\centering
			\includegraphics[clip, width=\linewidth]{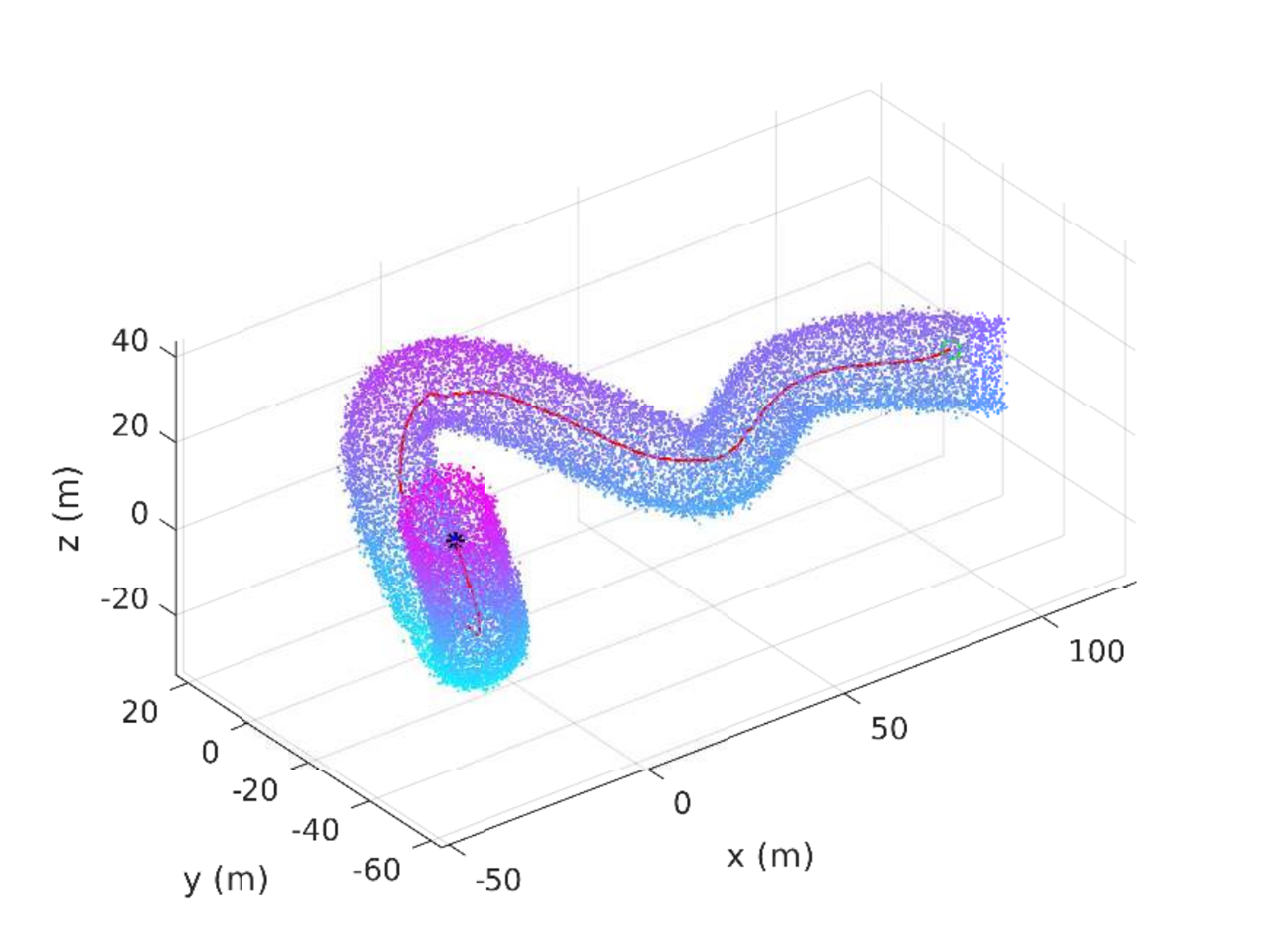} 
			\caption{The executed motion based on the noisy point cloud as seen by the UAV sensors}
			\label{fig:ch6:sim2b}
		\end{subfigure}
	
		\begin{subfigure}[t]{0.48\textwidth}
			\centering
			\includegraphics[trim={0 0 0 0.5cm}, clip, width=\linewidth]{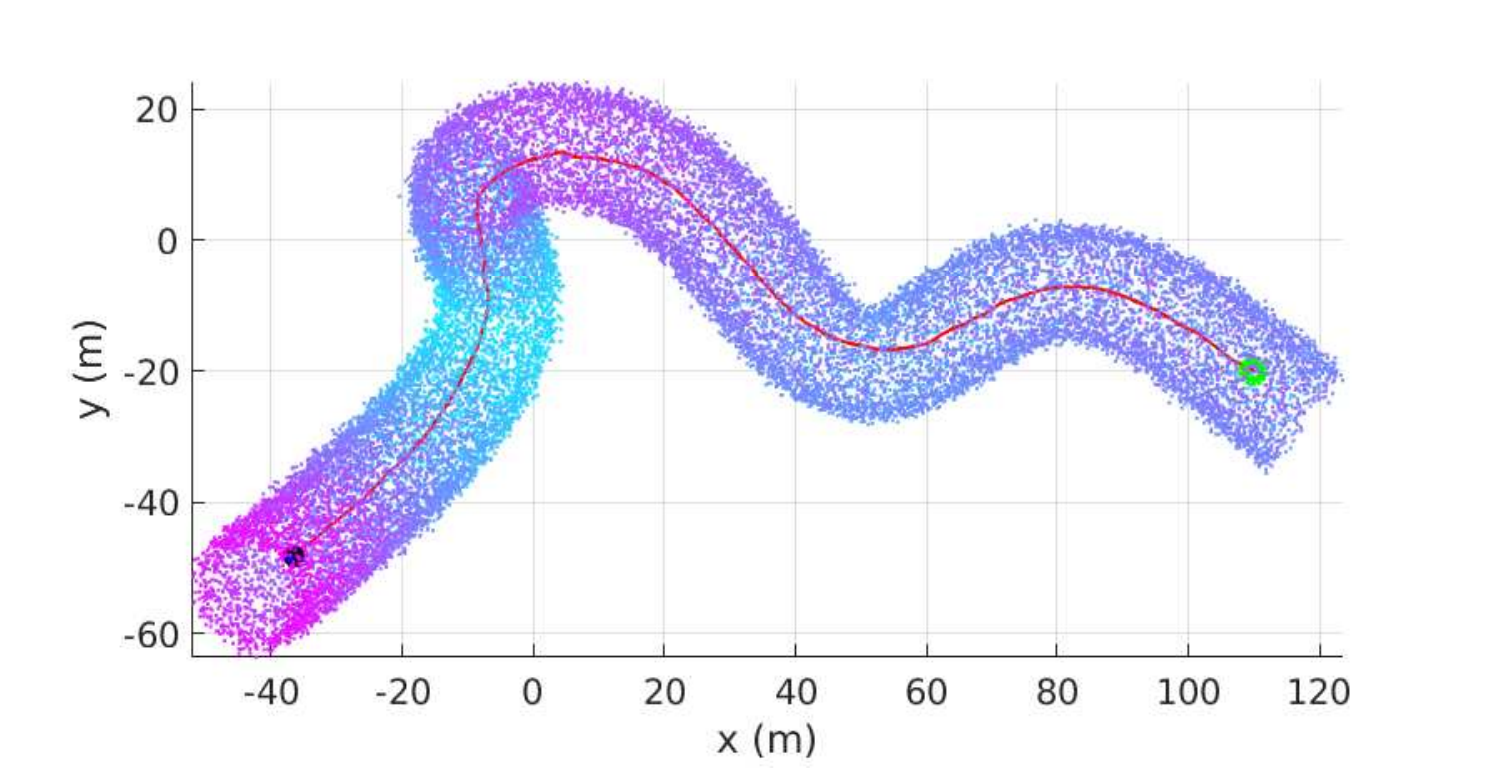} 
			\caption{XY View}
			\label{fig:ch6:sim3a}
		\end{subfigure}
		\hfill
		\begin{subfigure}[t]{0.48\textwidth}
			\centering
			\includegraphics[clip, width=\linewidth]{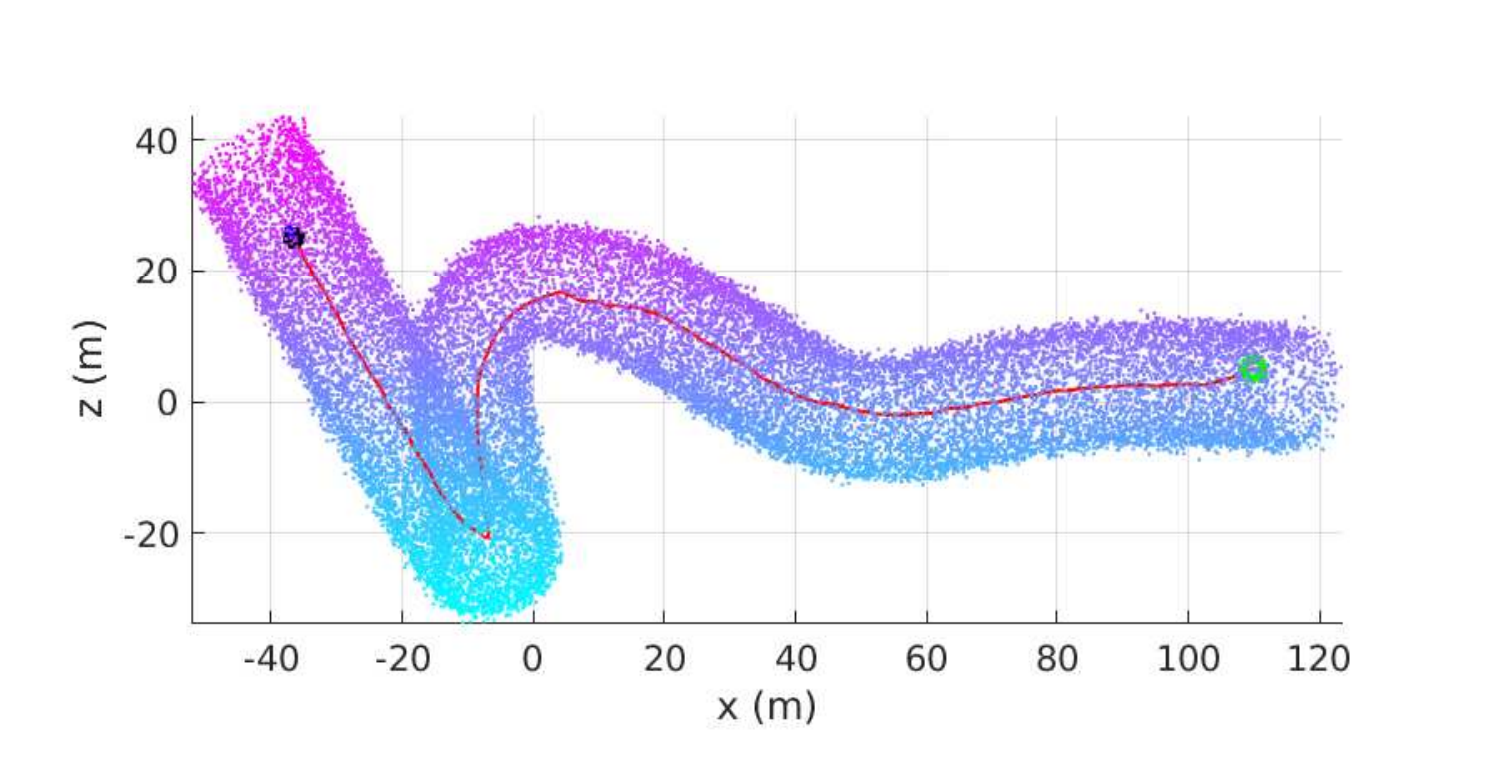} 
			\caption{XZ View}
			\label{fig:ch6:sim3b}
		\end{subfigure}

		\caption{Simulation scenario g: movement in a complex tunnel environment with noisy sensor observations}
		\label{fig:ch6:sim2}
	\end{adjustbox}
\end{figure}

\begin{figure}[!htb]
	\centering
	\includegraphics[clip,width=0.5\columnwidth]{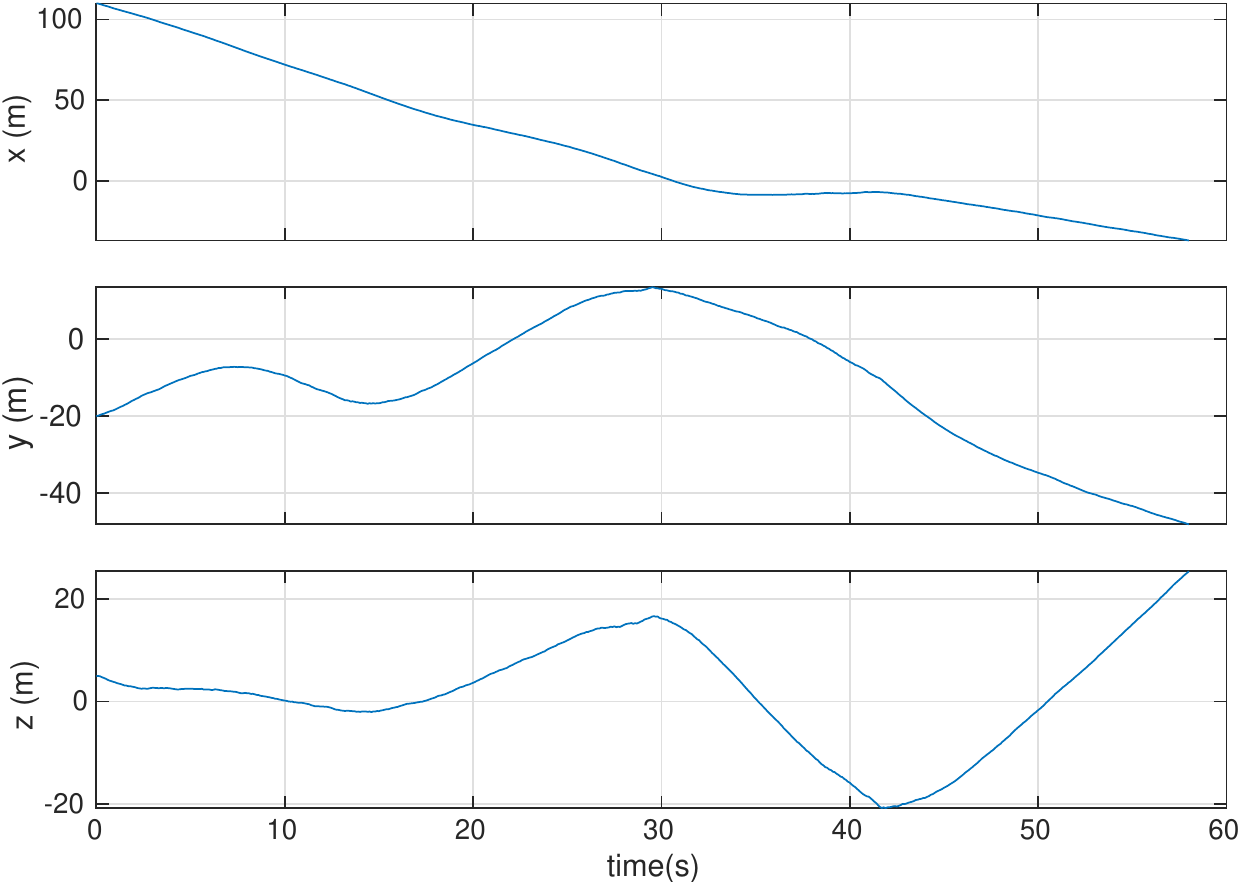}%
	\caption{Simulation scenario g: the time evolution of the UAV position (i.e $c(t)$)}
	\label{fig:ch6:sim4}
\end{figure}

\begin{figure}[!htb]
	\centering
	\includegraphics[clip,width=0.5\columnwidth]{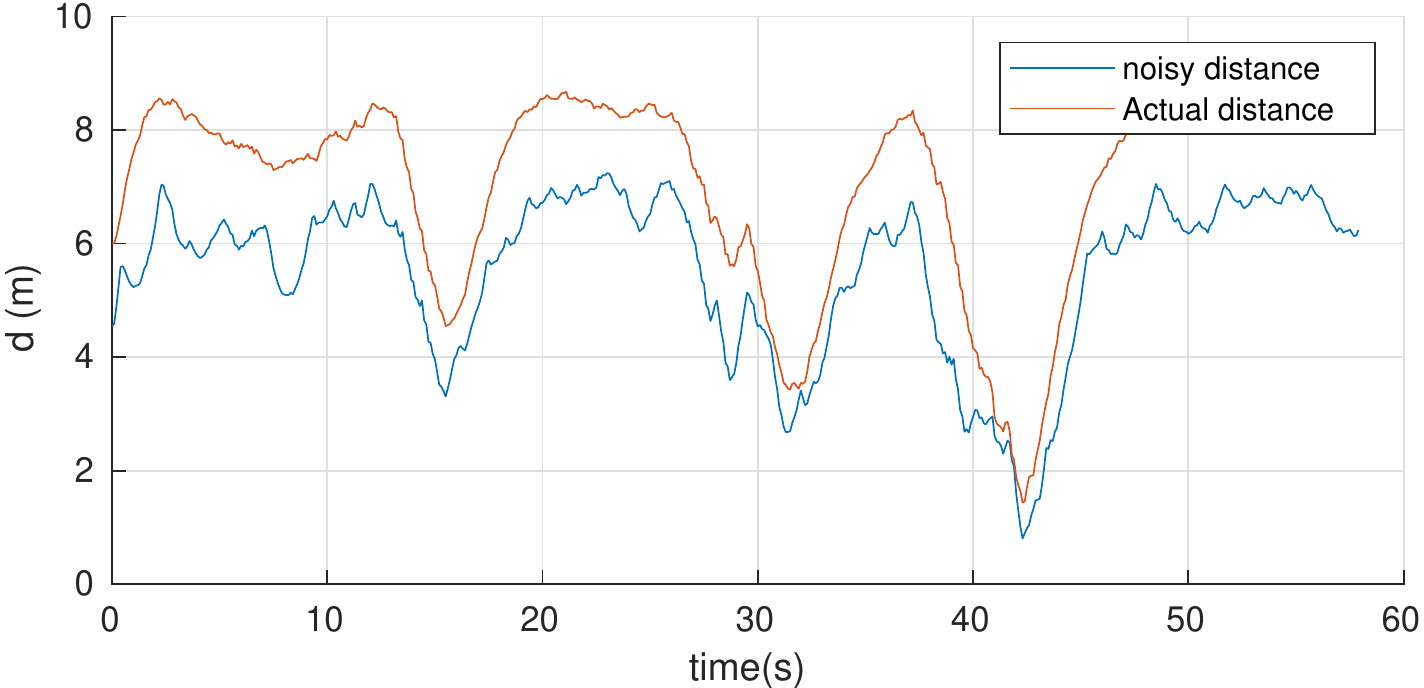}%
	\caption{Simulation scenario g: the distance between the UAV and the tunnel wall versus time}
	\label{fig:ch6:sim5}
\end{figure}

\section{Implementation with a Quadrotor UAV}\label{sec:UAVimpl}

Our navigation algorithm was developed using a general kinematic model applicable to many vehicles moving in 3D constrained environments.
Specific implementation details for quadrotor UAVs including control design and online trajectory generation method description are provided in this section.
This is the implementation used in our proof-of-concept experiment.

\subsection{Quadrotor Dynamics}

The kinematic model \eqref{1} can be extended to include quadtrotor dynamics.
To that effect, we define two coordinate frames, namely an inertial frame $\{\mathcal{I}\}$ and a body-fixed frame $\{\mathcal{B}\}$ attached to the UAV.
The origin of $\{\mathcal{I}\}$ can be chosen arbitrary in $\R^3$, and the origin of $\{\mathcal{B}\}$ coincides with the UAV's center of mass (COM).
The attitude of the UAV is expressed as a rotation matrix $\RotMatrix \subset SO(3): \{\mathcal{B}\}\to \{\mathcal{I}\}$.
An associated vector $\Omega$ is defined in $\{\mathcal{B}\}$ representing the angular velocity of the UAV relative to $\{\mathcal{I}\}$. 
Additionally, Euler angles (roll $\phi$, pitch $\theta$ and yaw $\psi$) or quaternions can also be used to describe the UAV attitude where transformations between the three representations are widely known.
Hence, the model from \cite{hamel2002dynamic,faessler2017differential} is used neglecting wind and rotor drag effects which is given by:
\begingroup
\allowdisplaybreaks
\begin{align}
	\bm{\dot{c}}(t) &= \bm{V}(t) \label{equ:model1}\\
	\bm{\dot{V}}(t) &= -g \bm{e}_3 + T(t) \RotMatrix(t) \bm{e}_3 \label{equ:model2}\\
	\dot{\RotMatrix}(t) &= \RotMatrix(t) \bm{\hat{\Omega}}(t)\label{equ:model3} \\
	\bm{\dot{\Omega}}(t) &= \bm{J}^{-1} \Big(\bm{\tau}(t) -\bm{\Omega}(t) \times \bm{J} \bm{\Omega}(t)\Big)\label{equ:model4}
\end{align}
\endgroup
where $g$ is the gravitational constant, $\bm{e}_3=[0,0,1]^T$, $T(t) \in \R^{+}$ is the mass-normalized collective thrust, $\bm{\hat{\Omega}}(t)$ is a skew-symmetric matrix defined according to $\bm{\hat{\Omega}} \bm{r} = \bm{\Omega} \times \bm{r}$ for any vector $\bm{r}\in\R^3$, $\bm{J}$ is the inertia matrix with respect to $\{\mathcal{B}\}$, and $\bm{\tau}(t) \in \R^3$ is the torques input vector defined in $\{\mathcal{B}\}$.
The above model can be modified to consider the effects of disturbances as in \cite{faessler2017differential,garcia2020robust} for a more robust control design especially when flying near to tunnels boundaries in narrow spaces.
We will assume that a low-level attitude controller exists for $\bm{\tau}(t)$ which can achieve any desired attitude $\RotMatrix_{des}(t)$.
Hence, the control design provided in the next subsection considers $T(t)$ and $\RotMatrix_{des}(t)$ as control inputs.
Note that this section adopts the notation of representing vectors and matrices using boldface letters while scalar quantities are represented using light letters.

\subsection{Control}\label{sec:uavCont}

A sliding-mode based controller design is presented here for the system \eqref{equ:model1}-\eqref{equ:model4} based on the differential-flatness property of quadrotor dynamics.
In \cite{mellinger2011minimum,faessler2017differential}, it has been shown that the model \eqref{equ:model1}-\eqref{equ:model4} is differentially flat such that it is possible to express the system states and inputs in terms of four flat outputs, namely $x$, $y$, $z$ and $\psi$, and their derivatives.

Consider a smooth reference trajectory to be tracked characterized by $\bm{r}(t)=[x_r(t), y_r(t), z_r(t), \psi_r(t)]$ with bounded time derivatives.
We define trajectory tracking errors according to (i.e. position and velocity tracking errors):
\begin{equation}\label{equ:errors}
	\bm{e}_{\bm{c}}(t) = \bm{c}_r(t) - \bm{c}(t),\ \bm{e}_{\bm{V}}(t) = \bm{\dot{c}}_r(t) - \bm{V}(t)
\end{equation}
where $\bm{c}_r(t)=[x_r(t), y_r(t), z_r(t)]^T$.
A sliding variable is then introduced as follows:
\begin{equation}\label{equ:slidingsurface}
	\bm{\sigma}(t) = \bm{e}_{\bm{V}}(t) + \bm{K}_1 \tanh(\mu\bm{e}_{\bm{c}}(t))
\end{equation}
where $\bm{K}_1 \in \R^{3\times 3}$ is a positive-definite diagonal matrix, $\tanh(\bm{v})\in\R^3$ is the element-wise hyperbolic tangent function for a vector $\bm{v}\in\R^3$, and $\mu>0$.
By applying Lyapunov's direct method, it can be easily found that this choice of a sliding variable will guarantee that both $\bm{e}_{\bm{c}}(t)$ and $\bm{e}_{\bm{V}}(t)$ asymptotically converge to $\bm{0}=[0,0,0]^T$ when the system trajectories reach the sliding surface $\bm{\sigma}(t)=\bm{0}$.

By taking the time derivative of \eqref{equ:slidingsurface}, one can get:
\begin{equation}\label{equ:Sdot}
	\bm{\dot{\sigma}}(t) = \bm{\ddot{c}}_r(t) - \bm{\dot{V}}(t) + \mu\bm{K}_1 \Big(\bm{e}_{\bm{V}}(t) \odot \text{sech}^2(\mu\bm{e}_{\bm{c}}(t))\Big)
\end{equation}
where $\text{sech}^2(\bm{v})=[\text{sech}^2(v_x),\text{sech}^2(v_y),\text{sech}^2(v_z)]^T$ for some vector $\bm{v}=[v_x,v_y,v_z]^T$, and $\bm{v}_1\odot \bm{v}_2 \in \R^3$ is defined as the element-wise product between the two vectors $\bm{v}_1,\bm{v}_2 \in \R^3$.

Let $\bm{a}_{cmd}(t) = T(t) \RotMatrix(t) \bm{e}_3$ be regarded as a virtual input (i.e. a command acceleration).
Now, we propose the following control law:
\begin{equation}\label{equ:acc_cmd}
	\begin{array}{ll}
		\bm{a}_{cmd}(t) = \ddot{\bm{c}}_{r}(t) + g \bm{e}_{3} &+ \mu\bm{K}_1 \Big(\bm{e}_{\bm{V}}(t) \odot \text{sech}^2(\mu\bm{e}_{\bm{c}}(t))\Big) \\[0.1cm]
		&+ \bm{K}_2 \tanh\left(\mu \bm{\sigma}(t)\right)
	\end{array}
\end{equation}
where $\bm{K}_2 \in \R^{3\times 3}$ is a positive-definite diagonal matrix.
By substituting \eqref{equ:model2} and \eqref{equ:acc_cmd} into \eqref{equ:Sdot}, we obtain the following:
\begin{equation}\label{equ:Sdot_final}
	\bm{\dot{\sigma}}(t) = - \bm{K}_2 \tanh\left(\mu \bm{\sigma}(t)\right)
\end{equation}
Equation \eqref{equ:Sdot_final} clearly implies that $\bm{\sigma}(t)$ is asymptotically stable.
Hence, the control law \eqref{equ:acc_cmd} will force the system trajectories to reach the sliding surface $\bm{\sigma}(t) = 0$ which leads to $\bm{e}_{\bm{c}}(t) \to \bm{0}$ and $\bm{e}_{\bm{V}}(t) \to \bm{0}$ as $t \to \infty$.

Now, the input thrust $T(t)$ and the desired attitude $\RotMatrix_{des}=[\bm{x}_{\mathcal{B},des}(t),\ \bm{y}_{\mathcal{B},des}(t),\ \bm{z}_{\mathcal{B},des}(t)]$ can be obtained to achieve \eqref{equ:acc_cmd} and $\psi_r(t)$ according to the following:
\begin{align}
	\bm{z}_{\mathcal{B},des}(t) &= \frac{\bm{a}_{cmd}(t)}{\|\bm{a}_{cmd}(t)\|} \\
	\bm{y}_{\mathcal{B},des}(t) &= \frac{ \bm{z}_{\mathcal{B},des}(t) \times \bm{x}_{C}(t) }{\|\bm{z}_{\mathcal{B},des}(t) \times \bm{x}_{C}(t)\|} \\
	\bm{x}_{\mathcal{B},des}(t) &= \bm{y}_{\mathcal{B},des}(t) \times \bm{z}_{\mathcal{B},des}(t) \\
	T(t) &= \bm{a}_{cmd}^T(t) \bm{R}(t) \bm{e}_3
\end{align}
where $\bm{x}_{C}(t)$ is defined as:
\begin{equation}\label{equ:yCref}
	\bm{x}_{C}(t) = [\cos\psi_{r}(t),\ \sin\psi_{r}(t),\ 0]^T
\end{equation}
A low-level attitude controller is then used to compute $\bm{\tau}(t)$ that can achieve the tracking $\bm{R}(t) \to \bm{R}_{des}(t)$.

\subsection{Online Trajectory Generation}\label{sec:trajGen}

In the current implementation, we use $G_1$ and $G_2$ defined in the proposed strategy to determine the direction of progressive motion through the tunnel with minimum jerk trajectories.
A computationally efficient solution proposed in \cite{mueller2015computationally} is adopted to generate minimum jerk trajectories for $(x,y,z)$ which can be done independently for each axis.
This solution treats the problem as an optimal control problem of a triple integrator system for each output with a state vector $\bm{s}(t) = [q(t),\ \dot{q}(t),\ \ddot{q}(t)]$ where $q=\{x,y,z\}$, and the jerk $\dddot{q}(t)$ is taken as input.
Furthermore, to produce minimum jerk solutions, the following cost function is used:
\begin{equation}
	J_c = \frac{1}{t_f}\int_{0}^{t_f}\dddot{q}^2(t)dt
\end{equation}
where $t_f$ is the duration of a motion segment.
The optimal solution to this problem is \cite{mueller2015computationally}:
\begin{equation}\label{equ:trajectory}
	\bm{s}^*(t) = \left[\begin{array}{c}
		\frac{k_1}{120}t^5 + \frac{k_2}{24}t^4 + \frac{k_3}{6}t^3 + \frac{\ddot{q}_0}{2}t^2 + \dot{q}_0t + q_0 \\
		\frac{k_1}{24}t^4 + \frac{k_2}{6}t^3 + \frac{k_3}{2}t^2 + \ddot{q}_0 t + \dot{q}_0 \\
		\frac{k_1}{6}t^3 + \frac{k_2}{2}t^2 + k_3 t + \ddot{q}_0
	\end{array}\right]
\end{equation}
where $(q_0,\ \dot{q}_0, \ddot{q}_0)$ are the components of the initial state vector $\bm{s}(0)$, and $(k_1,\ k_2,\ k_3)$ are solved for to satisfy the desired final state $\bm{s}(t_f)$.

So, at every computation cycle, equation \eqref{equ:trajectory} is used for each flat output to generate a trajectory segment by setting the boundary conditions as follows:
\begin{itemize}
	\item \textit{initial state:} the current state of the UAV $\bm{s}(t_0)$ where $t_0$ is the time at which computation starts or a time ahead to allow for computation latency where the states gets estimated from the trajectory currently being executed.
	\item \textit{final state:} the final position $(x(t_f),y(t_f),z(t_f))$ is set to be $G_1$, and $\psi(t_f)$ is determined such that the vehicle is oriented towards $G_2$ from $G_1$.
	Furthermore, the final velocity is set to be
	\begin{equation}
		(\dot{x}(t_f),\dot{y}(t_f),\dot{z}(t_f))=v_{avg}\frac{G_2-\bm{c}}{G_2-\bm{c}}
	\end{equation}
	where $v_{avg}$ is some desired average velocity to keep the UAV moving.
\end{itemize}
Note that a smooth trajectory for the yaw angle can be generated considering some constant yaw rate with the boundary conditions $\psi(t_0)$ and $\psi(t_f)$.

\begin{remark}
	Another possible implementation for our approach is by relying directly on velocity commands based on \eqref{cont} in the quadrotor control design without the need for localization.
	In this case, the command acceleration in \eqref{equ:acc_cmd} can be designed differently such as:
	\begin{equation}\label{equ:acc_cmd2}
		\begin{aligned}
			\bm{a}_{cmd}(t) = \bm{K}_3 \tanh\Big(\mu (\bm{V}_{cmd}(t) - \bm{V}(t))\Big) + g \bm{e}_{3}
		\end{aligned}
	\end{equation}
	where $\bm{K}_3$ is a positive definite gain matrix with some condition related to the bound of $||\bm{\dot{V}}_{cmd}(t)||$, and $\bm{V}_{cmd}(t)$ is a filtered version of \eqref{cont} obtained by applying some smoothing technique.
\end{remark}

\subsection{Perception Pipelines \& Robust Implementation}

Good interpretation of sensors measurements is a crucial component for navigation.
There are different factors that affect the design of perception systems.
Overall system cost, payload capacity, power requirements and required UAV size have great impact on deciding what kind of sensors to use.
For example, lightweight 3D LIDARs can be used to provide a sensing solution with a great field of view (FOV) but their sizes and expensive costs need to be considered.
Recently, solid-state 3D LIDARs have been developed to a state where they can even provide better solutions for UAVs in terms of size and cost.
Alternatively, the use of stereo and depth cameras tends to be popular with small sized UAVs \cite{sanchez2018survey}.
However, such depth sensors have narrow FOV, limited range, noisy depth measurements and problems with reflective or highly absorptive surfaces \cite{naudet2021constrained}.
This adds more challenges on perception algorithms development to produce reliable and robust solutions. %
In this section, we provide two possible perception pipelines based on the suggested navigation approach with different computational costs.
The goal of both algorithms is to determine an estimate of the gravity centers $G_1$ and $G_2$ described by our navigation strategy.

\subsubsection{Simple Algorithm}\label{sec:simple_perception}
The first algorithm is targeted towards vehicles with very limited computational power.
It has basic steps to allow for low-latency perception at the expense of being prone to some situations where the vehicle may need to hover and rotate to be able to continue progressing through the tunnel.

The recent available point cloud from onboard sensors are processed at certain rate according to the following.
Consider that all calculations are made in a camera-fixed frame $\{\mathcal{C}\}$ which has a known transformation relative to the body-fixed frame.
Note that we will use the notation ${}^{\mathcal{C}}\bm{p}$ to represent vectors expressed in the $\{\mathcal{C}\}$ frame.
The first step is to downsample the raw point cloud to reduce the computational cost.
Then, the nearest $k$ points to the the current UAV position ${}^{\mathcal{C}}\bm{c}$ are determined where ${}^{\mathcal{C}}\bm{c}$ is the UAV's COM expressed in the camera frame and $k$ can be chosen arbitrary. %
A geometric average ${}^{\mathcal{C}}\bar{G}_a$ is then calculated for the nearest neighbors points.
Let ${}^{\mathcal{C}}\bm{i}_{a} = {}^{\mathcal{C}}\bar{G}_a - {}^{\mathcal{C}}\bm{c}$ be the vector towards ${}^{\mathcal{C}}\bar{G}_a$.
Then, we compute $\alpha$ which is the angle between the current velocity vector and ${}^{\mathcal{C}}\bm{i}_{a}$ using:
\begin{equation}
	\alpha = \cos^{-1}\Big(\frac{{}^{\mathcal{C}}\bm{i}_{a}\cdot{}^{\mathcal{C}}\bm{V} }{\|{}^{\mathcal{C}}\bm{i}_{a}\| \|{}^{\mathcal{C}}\bm{V}\|}\Big)
\end{equation}
Another vector ${}^{\mathcal{C}}\bm{i}_{b}$ is obtained next by rotating ${}^{\mathcal{C}}\bm{V}$ by $-\alpha$ in the plane containing both ${}^{\mathcal{C}}\bm{V}$ and ${}^{\mathcal{C}}\bm{i}_{a}$.
Hence, a second point ${}^{\mathcal{C}}\bar{G}_b$ can be computed as the gravity center of the tunnel wall points in the direction of ${}^{\mathcal{C}}\bm{i}_{b}$.
Similarly, another two points ${}^{\mathcal{C}}\bar{G}_c$ and ${}^{\mathcal{C}}\bar{G}_d$ can be obtained associated with rotating the vector ${}^{\mathcal{C}}\bm{V}$ by angles $\beta$ and $-\beta$ respectively where the relation between $\alpha$ and $\beta$ is defined in our strategy.
Hence, ${}^{\mathcal{C}}G_1$ and ${}^{\mathcal{C}}G_2$ are computed according to:
\begin{equation}
	\begin{array}{ll}
		{}^{\mathcal{C}}G_1 = 0.5({}^{\mathcal{C}}\bar{G}_a + {}^{\mathcal{C}}\bar{G}_b), & {}^{\mathcal{C}}G_2 = 0.5({}^{\mathcal{C}}\bar{G}_c + {}^{\mathcal{C}}\bar{G}_d)
	\end{array}
\end{equation}
which can then be transformed to the inertial frame $\{\mathcal{I}\}$ to get $G_1$ and $G_2$.

\subsubsection{Complete \& Robust Algorithm}\label{sec:robust_perception}
The proposed strategy in this work have shown good results in simulations using sensors with wide FOV (ex. LIDAR or multiple cameras).
Based on experimental observations, additional layers can be added to the overall algorithm to deal with some practical aspects when using sensors with narrow FOV  for increased robustness.
The algorithm can be summarized using the following steps whenever new measurements arrives or at some other update rate slower than sensors measurements rate:
\begin{enumerate}
	\item Downsample the raw point cloud to obtain $\mathcal{W}_s$ for improved computational performance.
	\item Select $N$ points $({}^{\mathcal{C}}\bm{O}_1, {}^{\mathcal{C}}\bm{O}_2,\cdots,{}^{\mathcal{C}}\bm{O}_N)$ ahead of the vehicle position according to \eqref{equ:Oi} at distances $D_1, D_2, \cdots, D_N$ where $D_1 < D_2 < \cdots  D_N$ (rather than just 2 as suggested earlier).
	\item Filter the downsampled point cloud $\mathcal{W}_s$ around each point ${}^{\mathcal{C}}\bm{O}_i$ obtained from the previous step to extract the corresponding sections $\mathcal{O}_i \subset \mathcal{W}_s$ as defined in \eqref{equ:ptcloud_section} with some tolerance $\epsilon > 0$.
	\item For each filtered section $\mathcal{O}_i$, compute the geometric mean $\bm{g}_i \in \R^3$ of all the points (i.e. the centroid) and add those centroids to a list $\mathcal{L}$ such that $\mathcal{L} = \{\bm{g}_1, \bm{g}_2, \cdots, \bm{g}_N\}$.
	\item compute the minimum distance from each point in $\mathcal{L}$ to the downsampled point cloud $\mathcal{W}_s$, and flag it as valid if $d_{g,i} > d_{safe}$ where
	\begin{equation*}
		d_{g,i} = argmin_{\bm{p}_i \in \mathcal{O}_i} \|\bm{g}_i - \bm{p}_i\|
	\end{equation*}
	Otherwise, flag the point as invalid.
	\item For each invalid point in $\mathcal{L}$, compute a safer position by moving it away from the nearest neighbors in $\mathcal{W}_s$ in the direction of the average estimated surface normals at the nearest neighbors with some distance larger than $ d_{safe} - d_{g,i}$.
	\item Add the adjusted points to $\mathcal{L}$, and flag them as valid or invalid according to step 4.
	\item Iterate through $\mathcal{L}$ to obtain the closest two valid points as ${}^{\mathcal{C}}G_1$ and ${}^{\mathcal{C}}G_2$ which can then be transformed from the sensors frame to obtain $G_1$ and $G_2$.
	\item If the number of valid points in $\mathcal{L}$ is less than 2, increment some counter $i$ which was initialized with 0. Otherwise, reset $i = 0$. If the counter $i$ reaches some predefined threshold $k$, terminate.
\end{enumerate}

\section{Proof-of-Concept Experiment}\label{sec:exp}

\subsection{Implementation Details}\label{sec:impl}

We conducted different experiments to validate our navigation method using a quadrotor UAV with two different sensors configurations.
Three experimental cases are given in this section showing flights through deformed tunnel-like structures made in the lab using our suggested method.
In all experiments, the sides of the tunnel were nonsmooth and curved, and the ceiling structure was not even.
The first case deals with a deformed tunnel with approximately $2.4m$ width, $2m$ height and $5.2m$ in length.
The last two experiments were carried out using a different structure where the tunnel is more curvy in the middle.
Also, it gets narrower towards the end where it becomes more challenging to fly such that it has a $2.3m$ width and $3m$ height at the beginning which reduces to $1.5m$ width by $1.4m$ in height towards the end for a total length of approximately $6m$.
In the last case, the tunnel floor was elevated at the beginning by adding a blocking obstacle which was $0.9m$ high.

A custom made quadrotor is used in the experiments which is shown in \cref{fig:ch6:uav}.
It is equipped with a Pixhawk Flight Controller Unit (FCU) which contains a 32-bit Microcontroller Unit (MCU) running the PX4 firmware in addition to a set of sensors including gyroscopes, accelerometers, magnetometer and barometer.
The open-source PX4 software stack handles the low-level attitude stabilization and implements an Extended Kalman Filter (EKF) that fuses IMU data and visual odometry to provide an estimate of the quadrotor states (i.e. position, attitude and velocity).
To allow for a fully autonomous operation, our UAV is equipped with an onboard computer connected with two cameras for localization and sensing.
Hence, all computations needed to implement our navigation method can be done onboard.
A powerful onboard computer (Intel NUC), Intel® Core™ i5-8259U CPU @ 2.30GHz, is used to implement the overall navigation stack.
Intel RealSense tracking camera T265 is used for visual localization, and Intel RealSense D435/L515 depth cameras are used as to detect the tunnel surface.
The T265 module provides monochrome fisheye images with a great FOV, and it contains an IMU and a Vision Processing Unit (VPU) to implement onboard visual SLAM.
The D435 camera provides depth information as 3D point clouds, and it has a Depth FOV of ($87^o\pm3^o\ \text{Horizontal}\ \times\ 58^o\pm1^o\ \text{Vertical}\ \times\ 95^o\pm3^o\ \text{Diagonal}$) and a maximum range of approx. $10m$.
However, a shorter range could be used in practice as the D435 depth data are more noisy for points further than 3 meters. 
Note that it is possible to use only the D435 camera to perform both localization and tunnel surface detection on the onboard computer.
The RealSense L515 camera provides more accurate depth data with accuracy of about $5 mm\ to\ 14 mm$ for a range of $9m$ since it is based on solid-state LIDAR technology.
However, it has a narrower FOV of $70^o\ \times\ 55^o \ (\pm3^o)$.
\Cref{fig:ch6:uav} shows the two sensors configuration used in the three experiments where the one on top was used in the first case and the other configuration was used in the other two cases.
The second configuration provides a wider FOV by combining the depth data from both the D435 and L515 depth cameras, which are oriented differently, after applying proper transformations.

\begin{figure}[!htb]
	\centering
	\begin{adjustbox}{minipage=\linewidth,scale=1.0}
	\includegraphics[clip,width=0.48\textwidth]{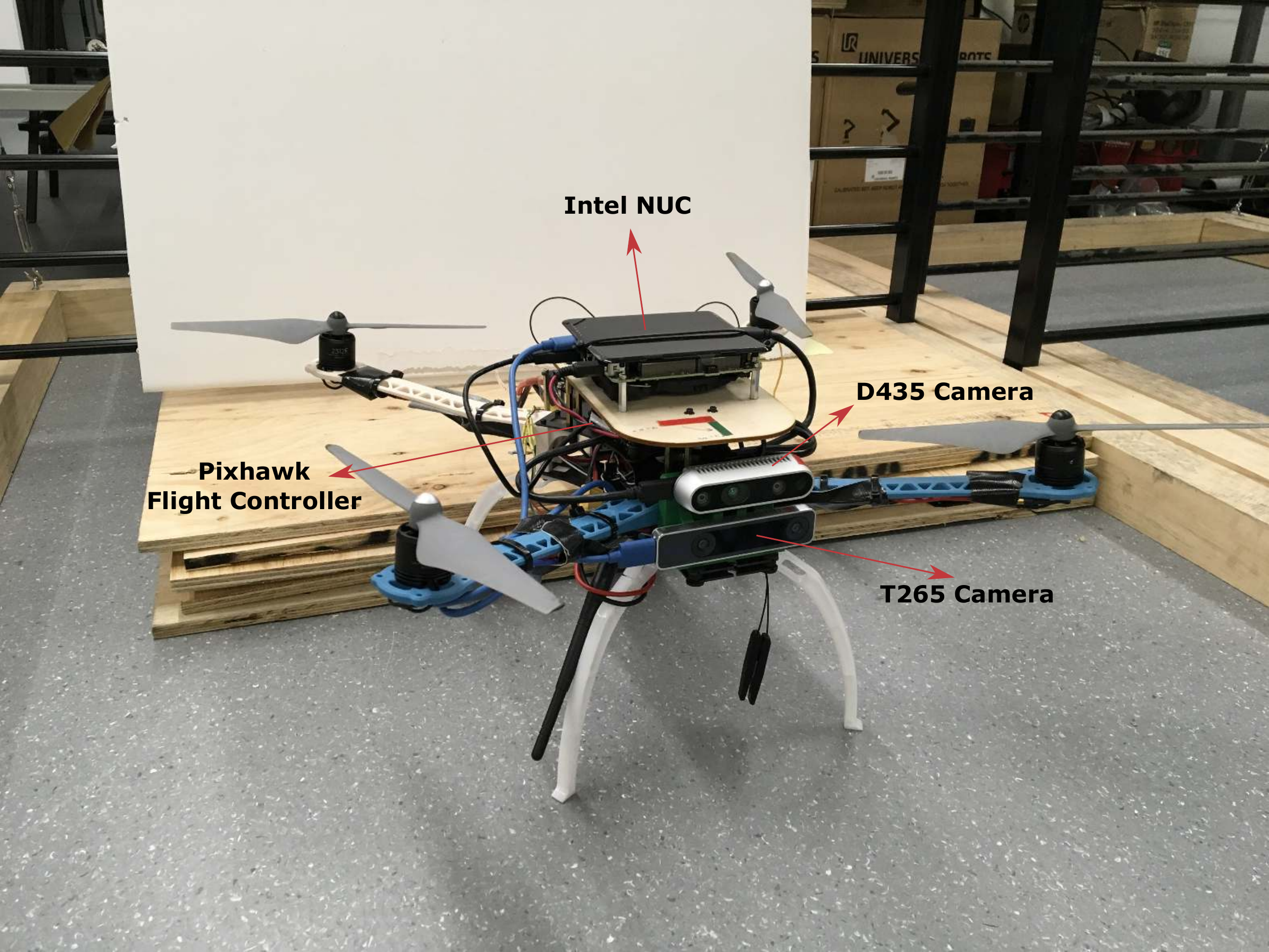}
	\hfill
	\includegraphics[clip,width=0.48\textwidth]{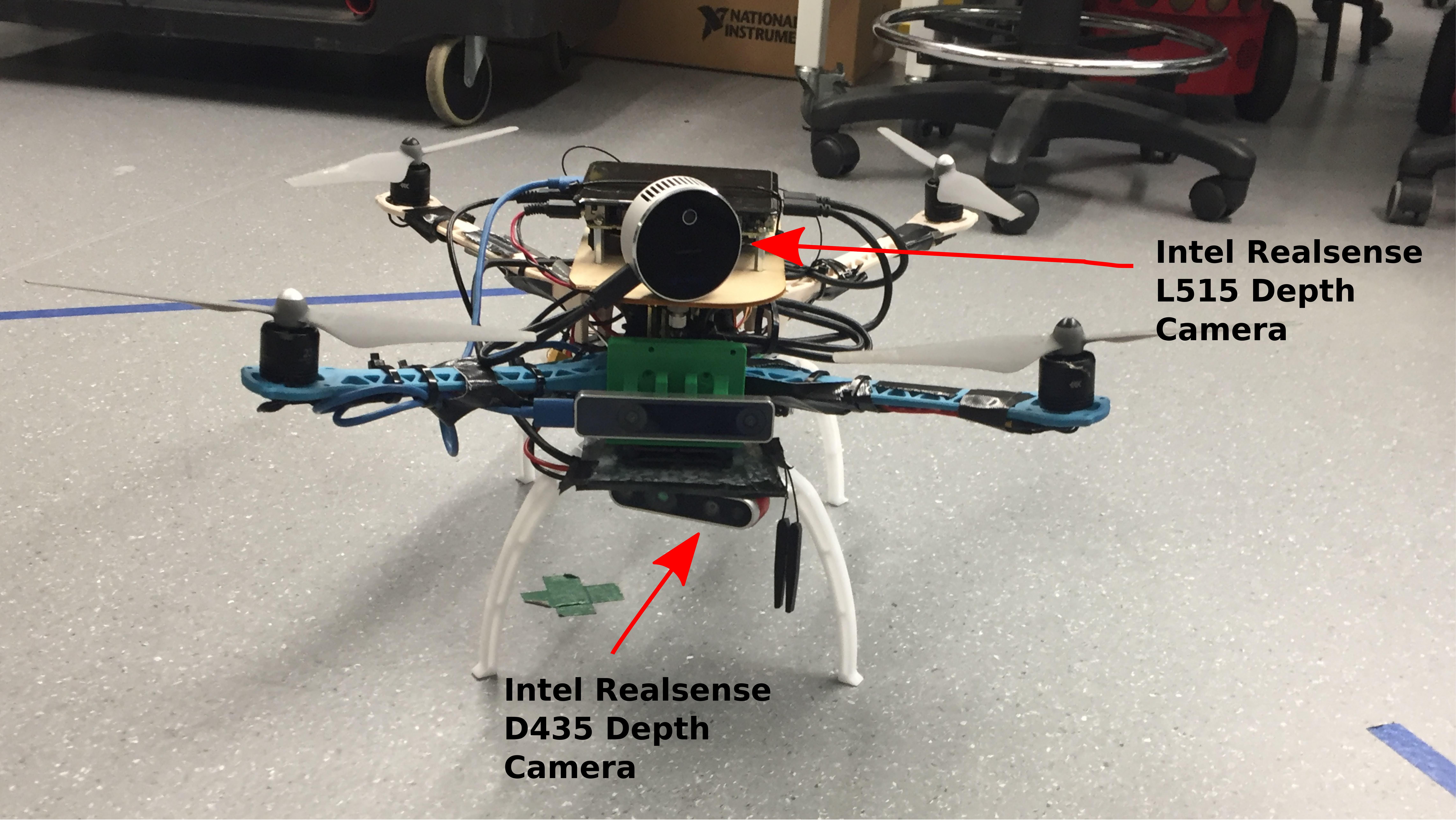}
	\end{adjustbox}
	\caption{The quadrotor used in the experiments with different sensors configurations}
	\label{fig:ch6:uav}
\end{figure}

The Robot Operating System (ROS) framework was adopted to implement the overall navigation software stack as connected nodes (i.e. simultaneously running processes) where each node handles a specific task.
A UAV control node implements the trajectory tracking controller described in \cref{sec:uavCont} to generate thrust and attitude commands for the low-level attitude controller at 100Hz.
These commands are sent to the flight controller unit through a link with the onboard computer (over USB) using the MAVLink messaging protocol through MAVROS library.
The received visual odometry from the T265 camera is also sent to the FCU to be fused with IMU data through an extended Kalman filter.
Also, camera nodes are used to process received 3D point clouds from the depth cameras to make them available for the other nodes with an update rate of 30Hz.
The proposed simpler algorithm described in \cref{sec:simple_perception} was used in the first experiment, and the more robust approach proposed in \cref{sec:robust_perception} was used in the other two cases.
These algorithms were running at 2-10Hz update rates, and they were implemented in C++ using useful tools from the Point Cloud Library (PCL) to handle point clouds processing in a computationally-efficient way.
A downsampling filter using PCL VoxelGrid is applied to the 3D point clouds to reduce the computational burden combined with some other filtering processes such as considering measurements that are within 5 meters or less.
A further processing is applied to assemble a single point cloud from all depth sensors if multiple are used by applying proper transformations from the sensors' frames to the vehicle's body-fixed frame.
The obtained points $G_1$ and $G_2$ from the previous algorithms are used to generate reference trajectories to be sent to the UAV control node where the approach described in \cref{sec:trajGen} was used in the first case.
In the last two cases, the similar idea was used but with slower straight motion trajectories based on trapezoidal velocity profile to deal with the very narrow flying space (i.e. only \eqref{equ:trajectory} was implemented differently).
Note that generating minimum jerk trajectories is recommended to produce less jerky motions; however, some corridor constraints may need to be considered to refine the result of \eqref{equ:trajectory} when flying in very narrow spaces similar to what was done in \cite{mellinger2011minimum}.

A description of the overall hardware and software architecture of our system is shown in \cref{fig:ch6:sysOverview}.

\begin{figure}[!ht]
	\centering
	\includegraphics[clip,width=\columnwidth]{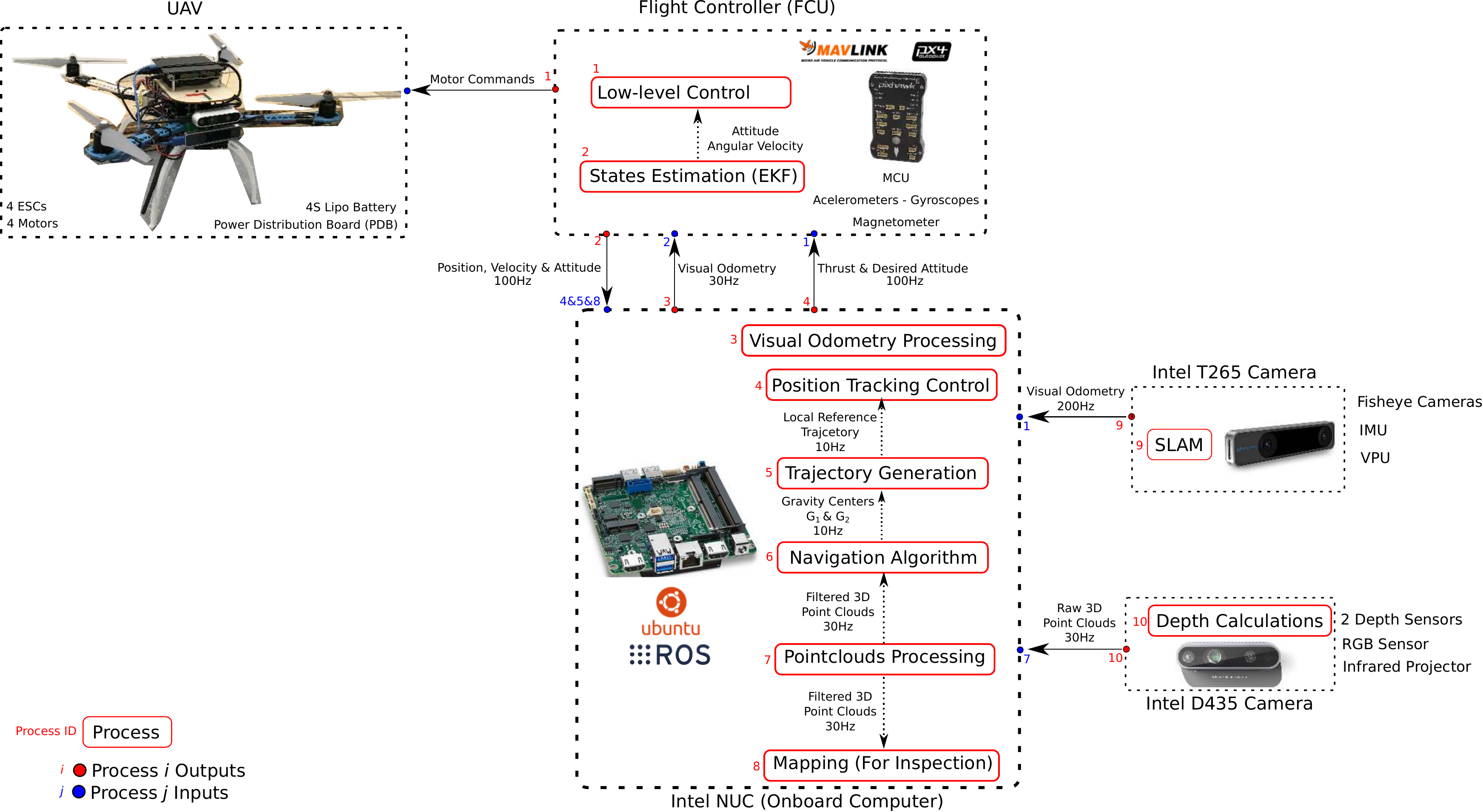}%
	\caption{Hardware and Software Architecture of our UAV system}
	\label{fig:ch6:sysOverview}
\end{figure}

\subsection{Results}

A video of the conducted experiments is available at \href{https://youtu.be/r2Add9lctEU}{https://youtu.be/r2Add9lctEU}.
Snapshots of the motion at different time instants are shown in \cref{fig:ch6:motion,fig:ch6:motion2,fig:ch6:motion3} for the three cases where a line connecting positions at each time instant was added for visualization purposes only (i.e. it is not the actual path).
Additionally, visualizations of the sensors feedback along with the results of the implemented perception pipelines at some specific moments during the flights are shown in \cref{fig:ch6:rviz,fig:ch6:perception_rviz}.

\Cref{fig:ch6:rviz} shows the detected patch of the tunnel surface, the vector directing from $G_1$ to $G_2$ and cameras feedback at the initial time for the first experiment.
In that figure, the current position of the UAV is indicated by the axes named 'base\_link' while the red arrow is at $G_1$ and directing towards $G_2$ as described in \cref{sec:trajGen}.
Notice that an online mapping algorithm was also performed onboard in this case to provide a map of the tunnel for visualization purposes only.
The velocity of the quadrotor during the flight and the distance to tunnel walls is shown in \cref{fig:ch6:results1}.
Moreover, the applied control inputs along with the vehicle's attitude is shown in \cref{fig:ch6:results1b}.
The mass-normalized collective thrust is further normalized to be within $[0,1]$ as required by PX4.
For safety purposes, the maximum value of the input thrust was limited to $0.75$.
Different regions are highlighted on the figures corresponding to the mode of operation.
Initially, sensors and safety checks are done in order to arm the drone before performing a takeoff to some predefined altitude.
Then, the vehicle switches to autonomous mode where the suggested navigation strategy is applied.
Once a terminating condition is detected, the vehicle goes out of the autonomous mode where the control commands are no longer being used in order to land.

It can be seen from the video and \cref{fig:ch6:motion} that the vehicle manages to maintain its movement along the tunnel curvy axis in the first experiment until it reaches the tunnel open end where it goes closer to one of the sides (as can be seen from \cref{fig:ch6:results1}).
It can also be observed that the nonsmooth tunnel surface results in $D_1$ and $D_2$ being dynamic during the motion when using the approach described in \cref{sec:simple_perception}.
This counts as a reaction to any bumps on the surface close to the UAV to achieve a collision-free motion.
The computational latency using this simple approach was less than $1ms$ using the mentioned mini computer.
It was observed in this experiment that using a depth sensor with small FOV which detects only a small patch of the tunnel surface can be very challenging.
This explains the behaviour near the end of the motion due to the tunnel being open and the depth measurements being filtered to only consider information within 2 meters or less.
In that case, a stopping policy was applied at the end to yaw away from the tunnel side and land immediately.
Another possible policy to apply in these situations is to hover and perform a yaw rotation to proceed the movement using that suggested simple perception approach.
One of the used methods in practice to have a wider 3D FOV is to use a mechanism to rotate the sensor at some frequency during the motion (for example, see \cite{kang2016full,kownacki2016concept}).
It is also possible to integrate more sensors by combining vision-based sensors with LIDARs depending on the application requirements and the environment conditions; however, this will reflect on the system overall cost, payload and power requirements.

The observations from the first case motivated the design of the robust algorithm given in \cref{sec:robust_perception} which was applied in the next two experiments.
Only three points were computed (i.e. $N=3$) corresponding to $D_1 = 1.25m$, $D_2 = 1.5m$ and $D_3 = 1.75m$.
Also, the sections extraction tolerance was selected as $\epsilon=0.1m$, and the safety margin was $d_{safe}=0.45m$. %

Different scenarios based on depth measurements are shown in \cref{fig:ch6:perception_rviz} where all points in the list $\mathcal{L}$ are represented with spherical markers with different colors.
Also, corresponding extracted sections $\mathcal{O}_1$, $\mathcal{O}_2$ and $\mathcal{O}_3$ from the point cloud are highlighted in different colors (green, yellow and orange respectively).
Yellow markers represent valid points obtained directly from step~4 as in \cref{fig:ch6:perception_rviz}(a); hence, steps~6-7 were not executed at that computation cycle.
\Cref{fig:ch6:perception_rviz}(b) shows a case where all computed points were invalid (red markers), and valid new points (orange markers) were obtained after performing steps~6-7.
This may happen whenever the vehicle senses only a fraction of a certain side without seeing the side in th opposite direction (the lower part of the tunnel is not detected in that case).
Similar case is shown in \cref{fig:ch6:perception_rviz}(c) where the upper part of the tunnel is not within the sensors' FOV at that point.
As a result, the geometric means will be closer to the detected portion.
However, it is clear from the experiments that the proposed approach managed to handle such cases very well.
Further case is shown in \cref{fig:ch6:perception_rviz}(d) where one of the new obtained points after applying step~6 remains invalid (blue marker).
The computational latency when steps~6-7 are not executed was less than $5ms$.
Otherwise, the latency was less than $70ms$ which can be hugely improved by estimating the surface normals for only the closest fraction of the point cloud rather than the whole downsampled cloud as was done in the experiments.
The surface normals estimation is computationally more expensive than the other steps; however, the overall computational performance is still very low compared to path planning based methods.

The quadrotor's velocity, minimum distance to tunnel walls and control inputs are shown in \cref{fig:ch6:results2,fig:ch6:results2b,fig:ch6:results3,fig:ch6:results3b} for the two cases.
The velocity of the reference trajectory was designed to be around $0.4 m/s$ which can be seen from the actual velocity plot.
The minimum distance to tunnel walls was computed based on the sensors point clouds which might not be a good estimate of the actual distance at some points.
However, these plots indicate that the vehicle maintained a safe distance during the flights.
Thus, these experiments validate the performance of the suggested tunnel navigation strategy.

\begin{figure}[!htb]
	\centering
	\includegraphics[clip,width=0.75\columnwidth]{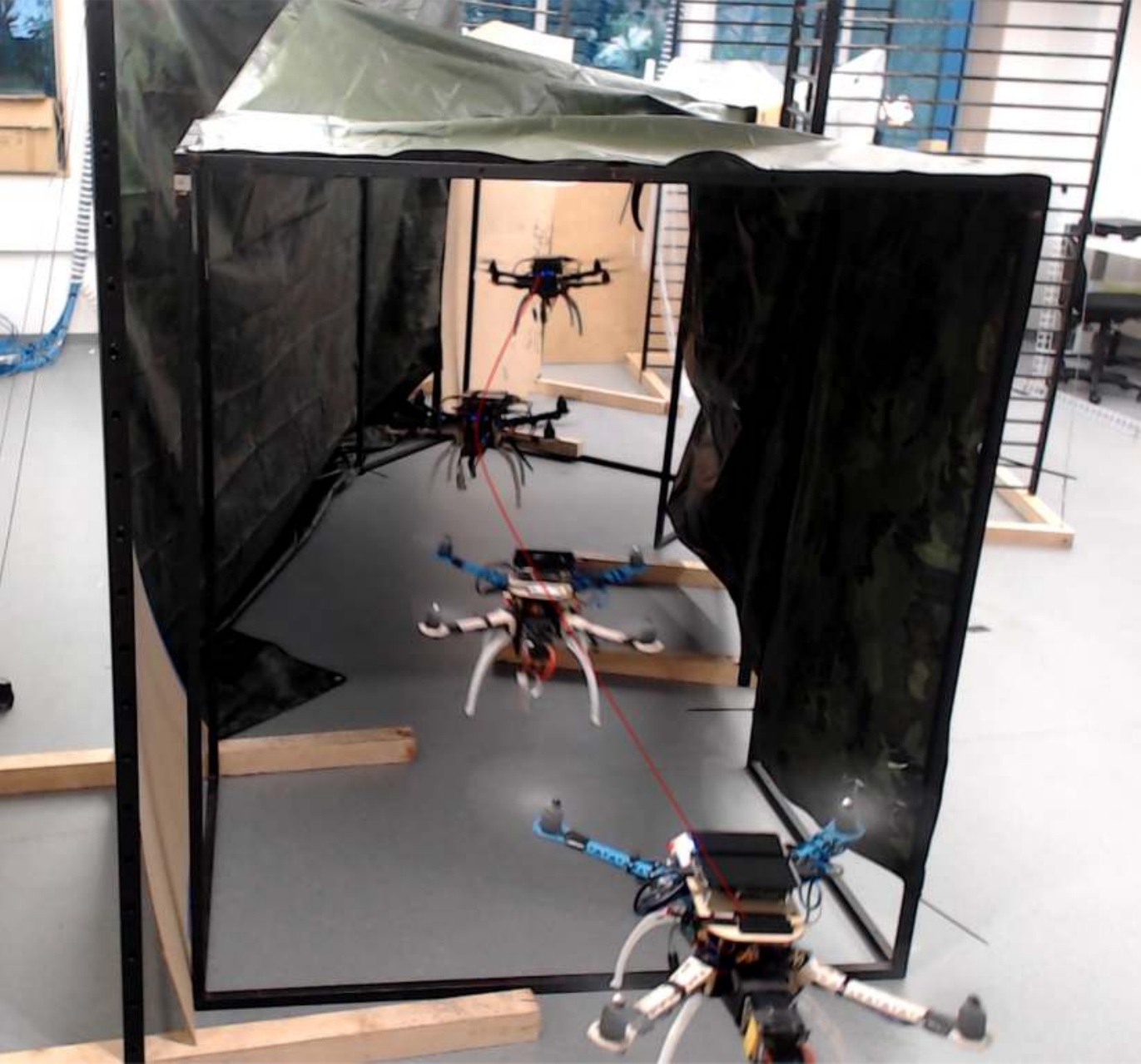}%
	\caption{Snapshots of the UAV position during movement }
	\label{fig:ch6:motion}
\end{figure}

\begin{figure}[!htb]
	\centering
	\includegraphics[clip,width=\columnwidth]{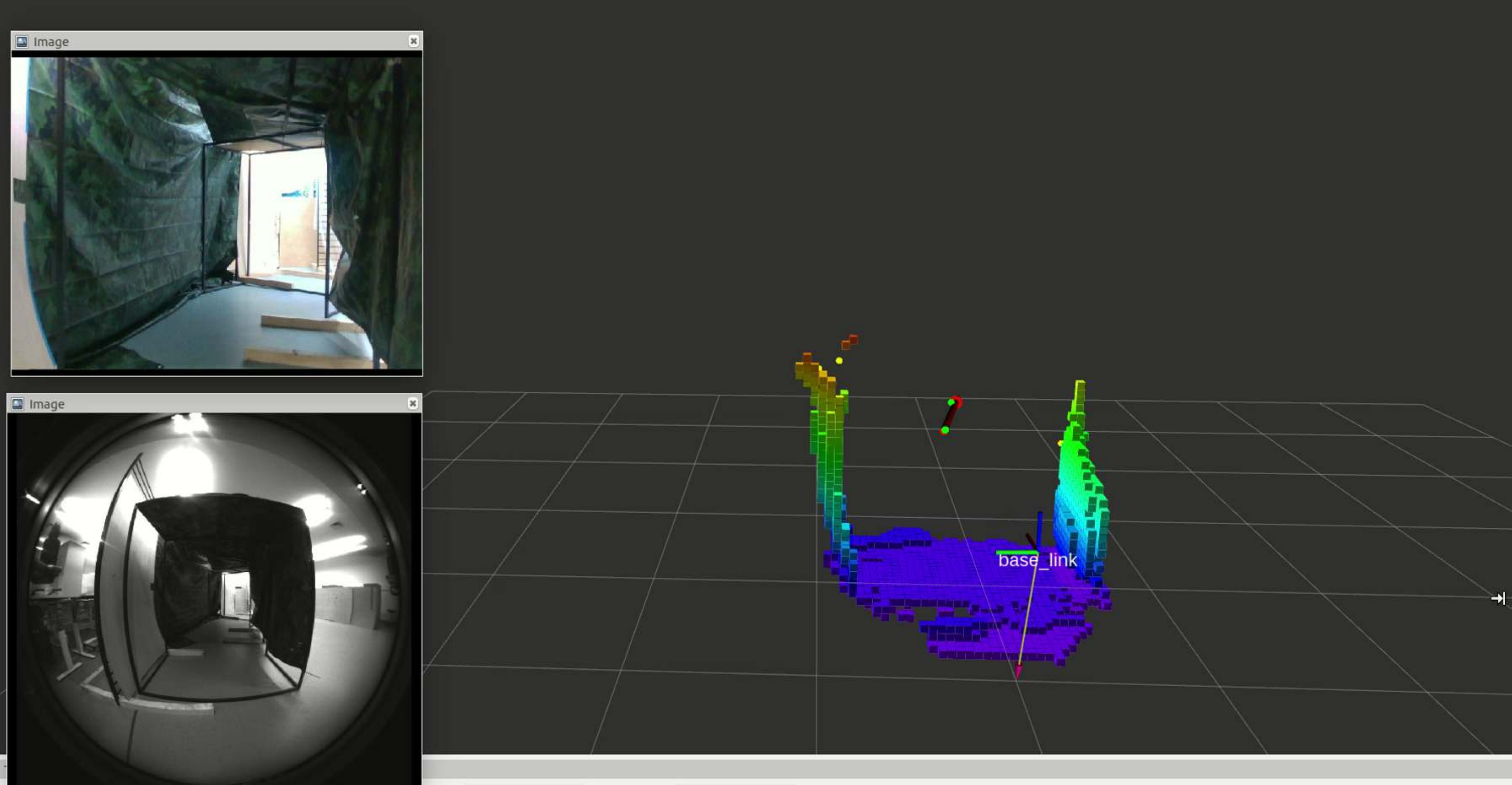}%
	\caption{A fraction of the tunnel wall sensed at motion start along with cameras feedback}
	\label{fig:ch6:rviz}
\end{figure}

\begin{figure}[!htb]
	\centering
	\includegraphics[clip,width=0.75\columnwidth]{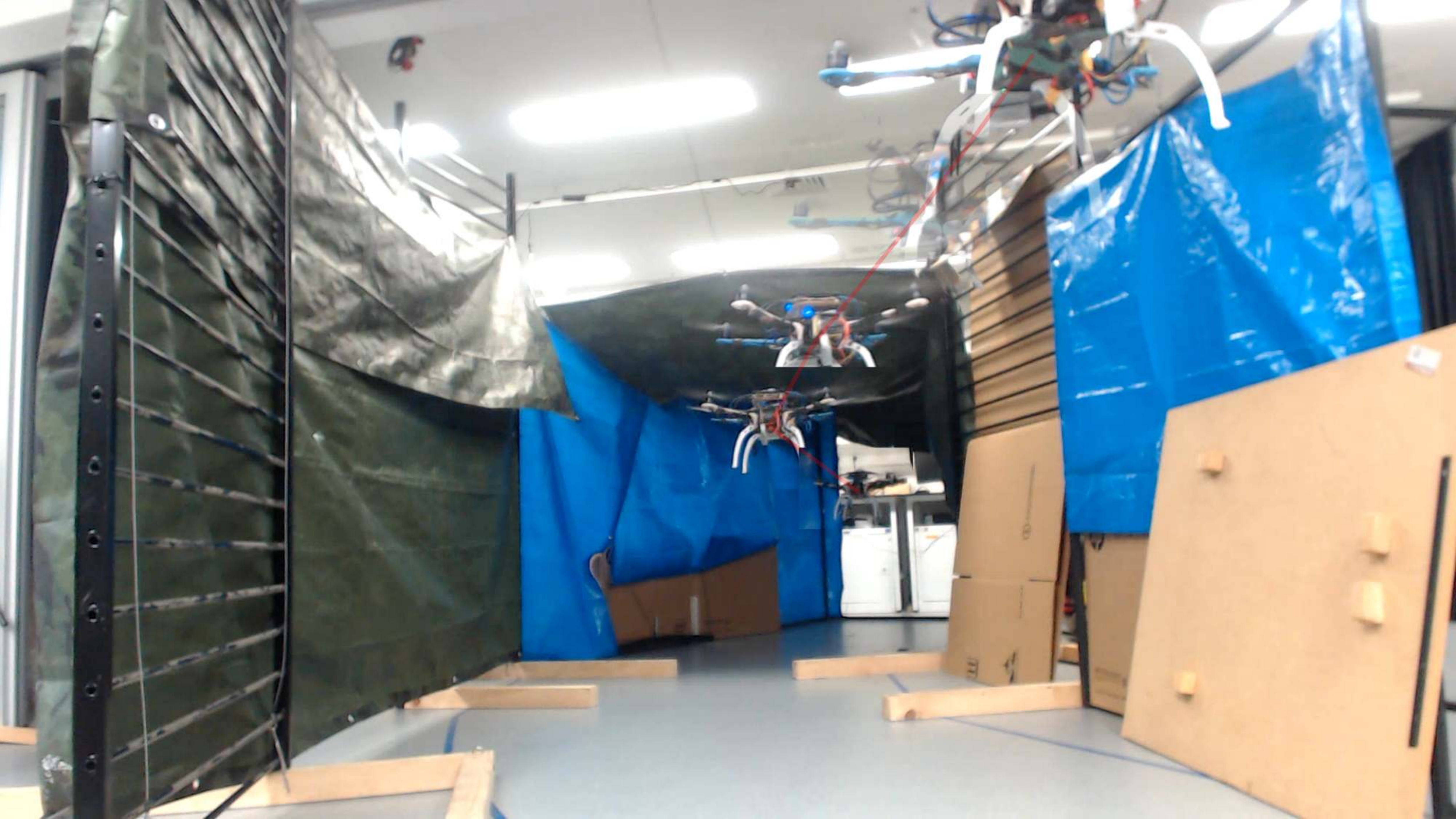}%
	\caption{Snapshots of the UAV position during movement for case 2}
	\label{fig:ch6:motion2}
\end{figure}

\begin{figure}[!htb]
	\centering
	\includegraphics[clip,width=0.75\columnwidth]{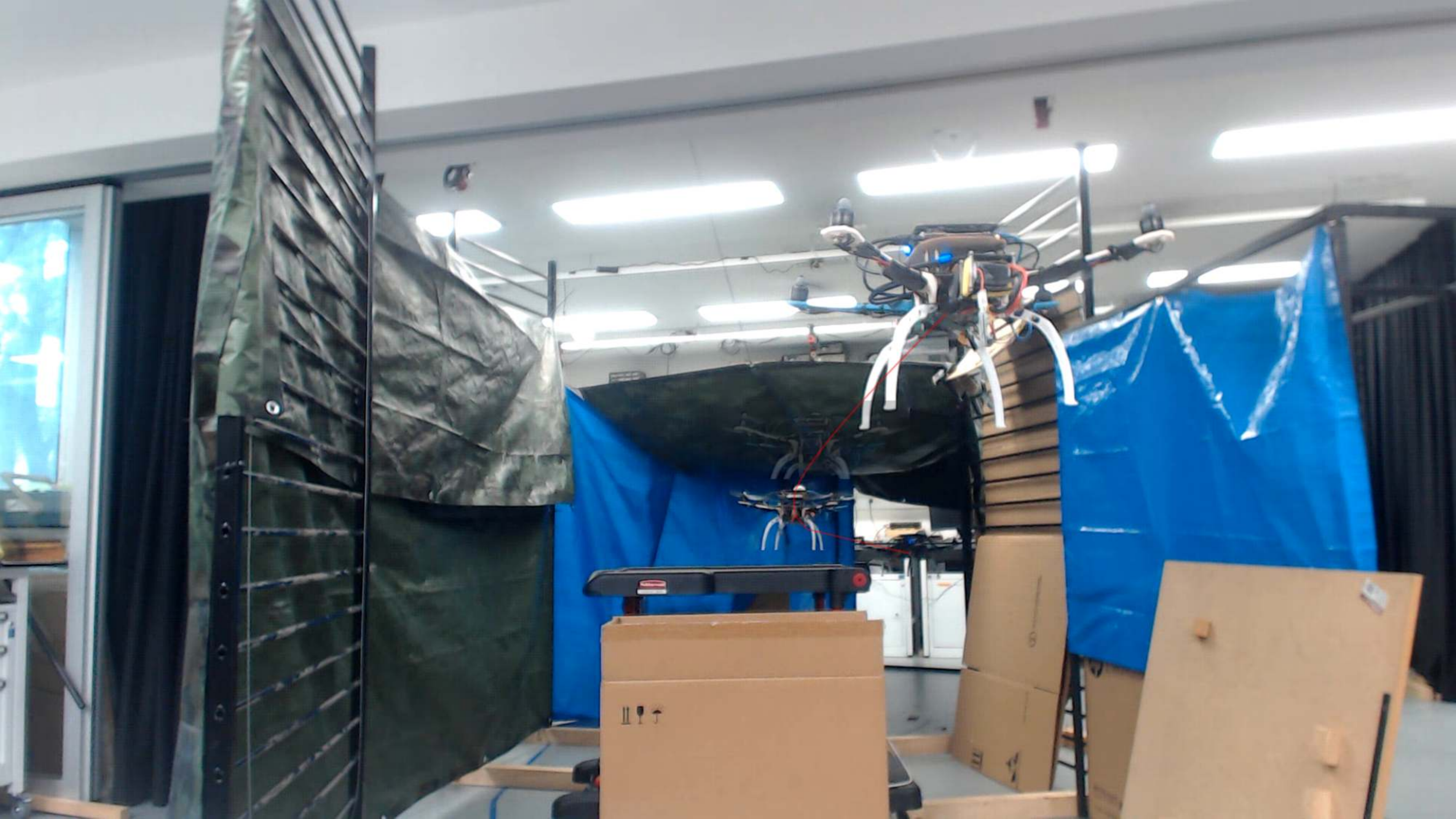}%
	\caption{Snapshots of the UAV position during movement for case 3}
	\label{fig:ch6:motion3}
\end{figure}

\begin{figure}[!htb]
	\centering
	\begin{adjustbox}{minipage=\linewidth,scale=1.0}
		\begin{subfigure}[t]{0.48\textwidth}
			\centering
			\includegraphics[clip, width=\linewidth]{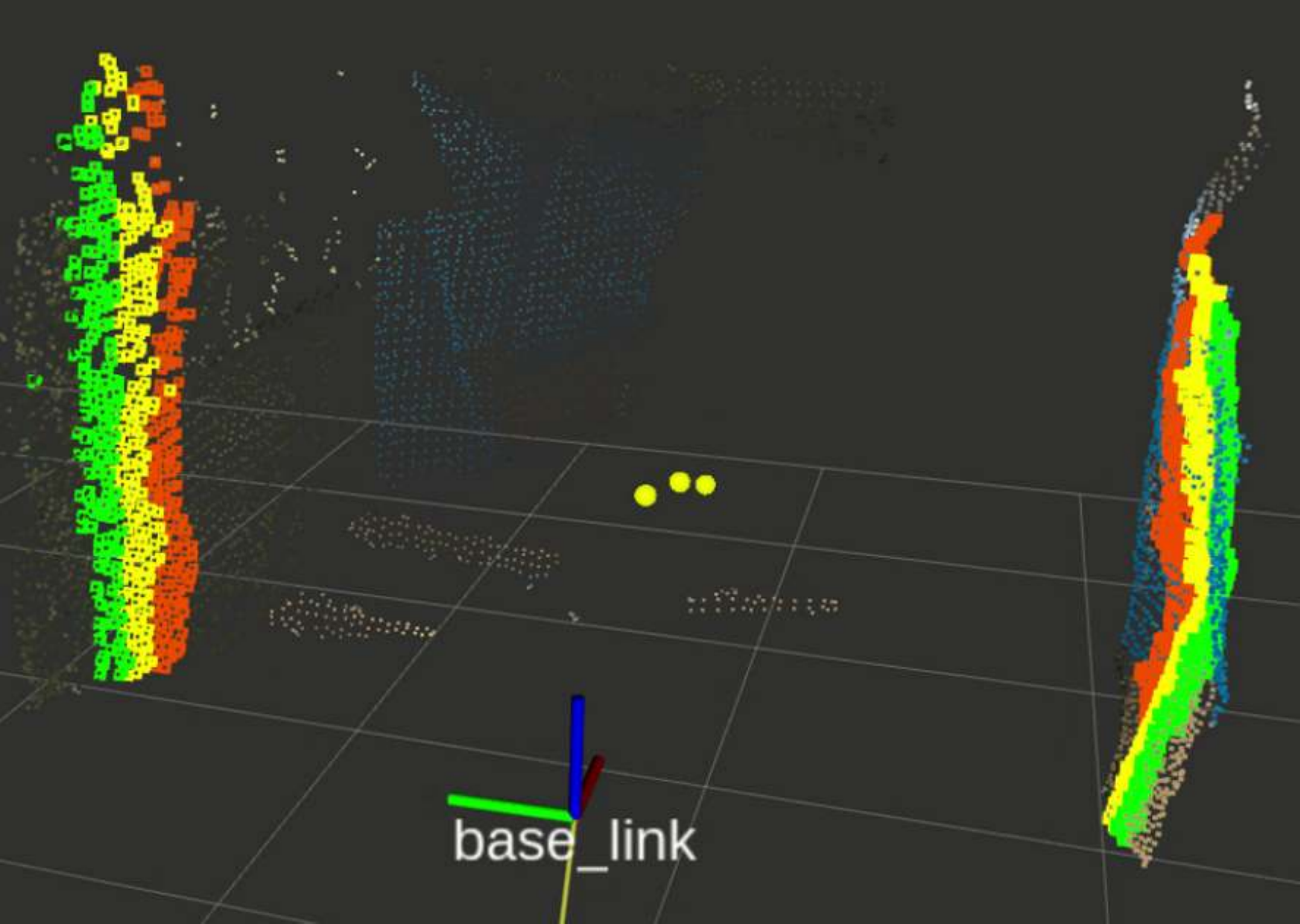} 
			\caption{}
			\label{fig:ch6:rviz2a}
		\end{subfigure}
		\hfill
		\begin{subfigure}[t]{0.48\textwidth}
			\centering
			\includegraphics[clip, width=\linewidth]{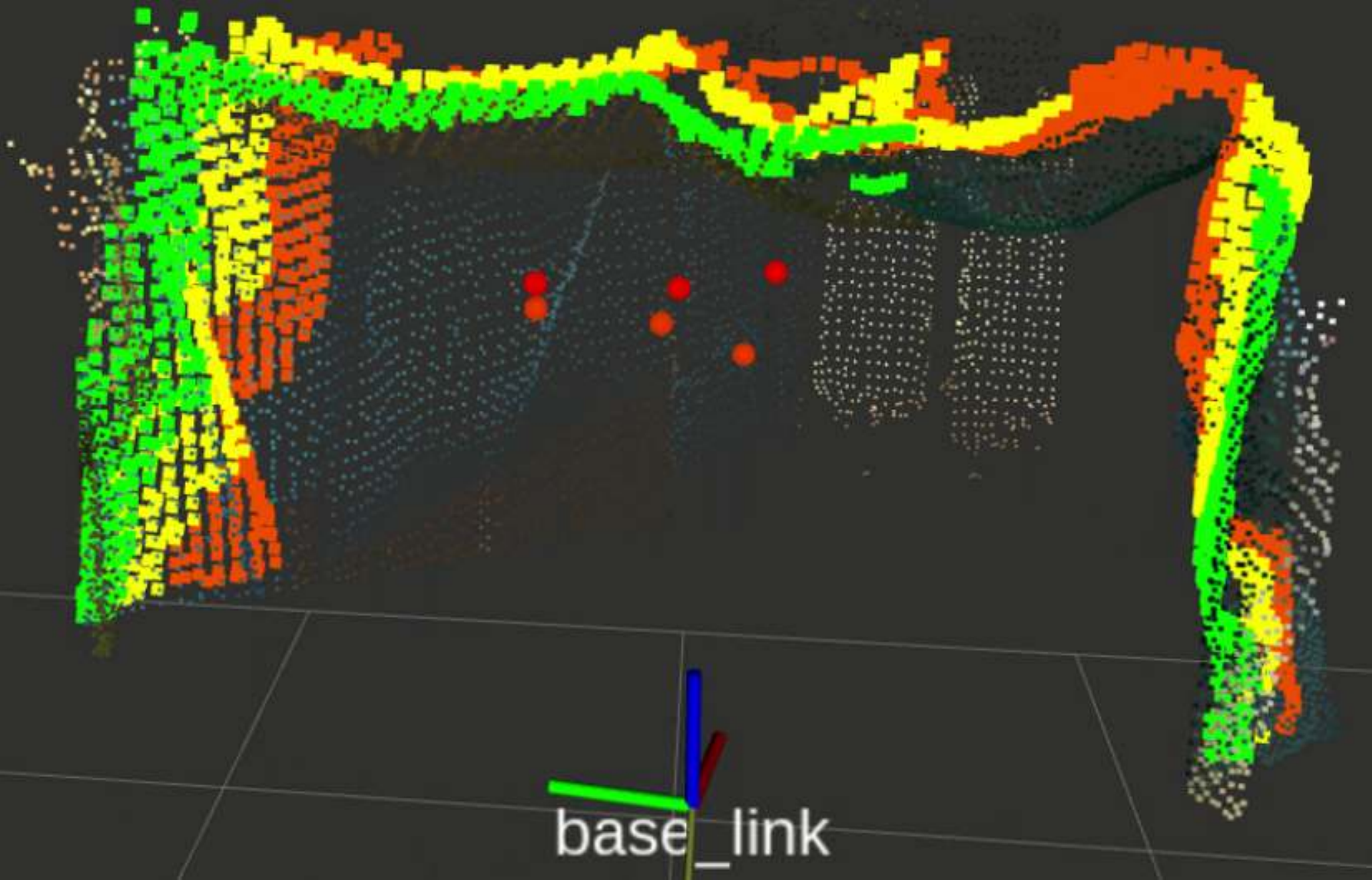} 
			\caption{}
			\label{fig:ch6:rviz2b}
		\end{subfigure}
	
		\begin{subfigure}[t]{0.48\textwidth}
			\centering
			\includegraphics[clip, width=\linewidth]{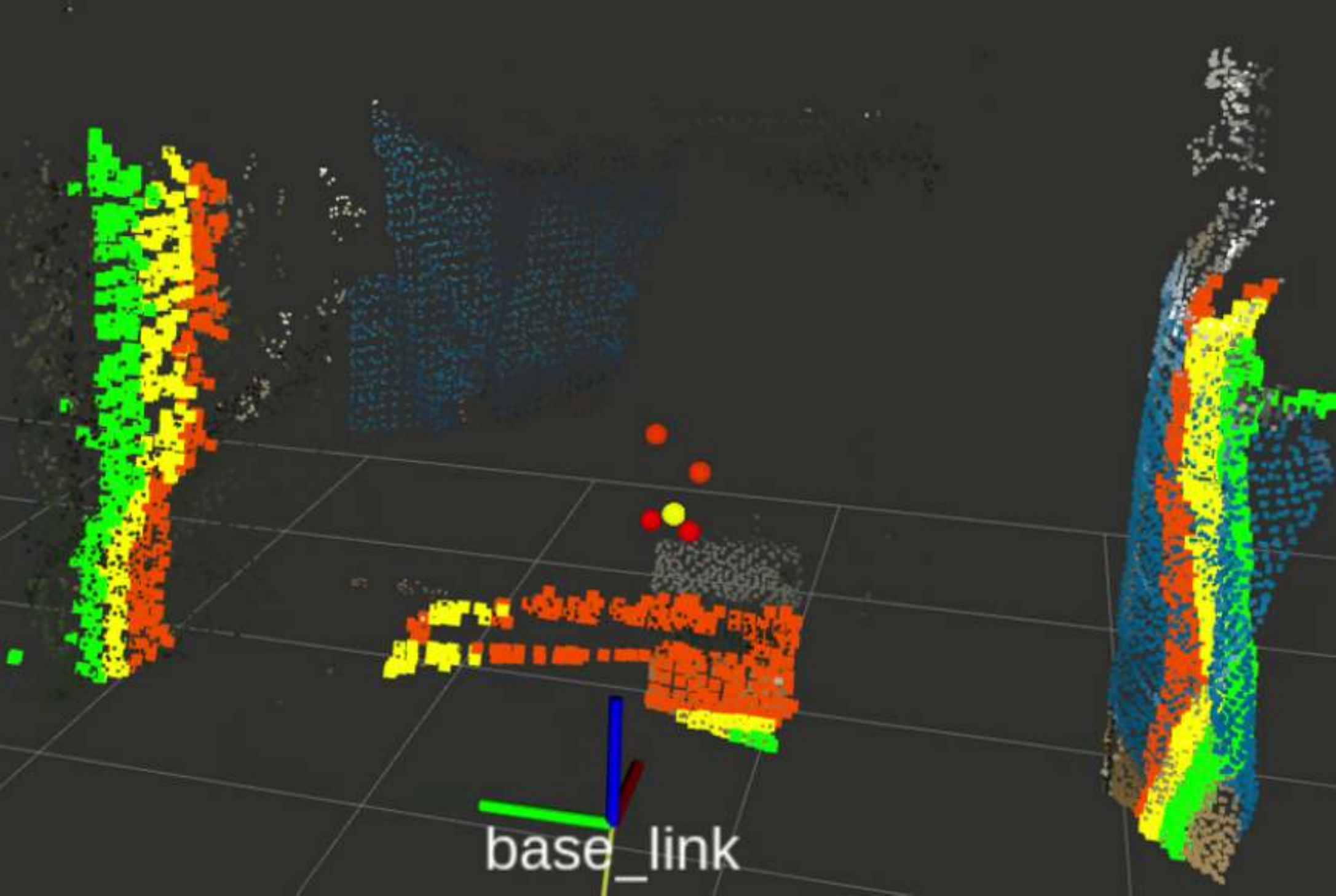} 
			\caption{}
			\label{fig:ch6:rviz3a}
		\end{subfigure}
		\hfill
		\begin{subfigure}[t]{0.48\textwidth}
			\centering
			\includegraphics[clip, width=\linewidth]{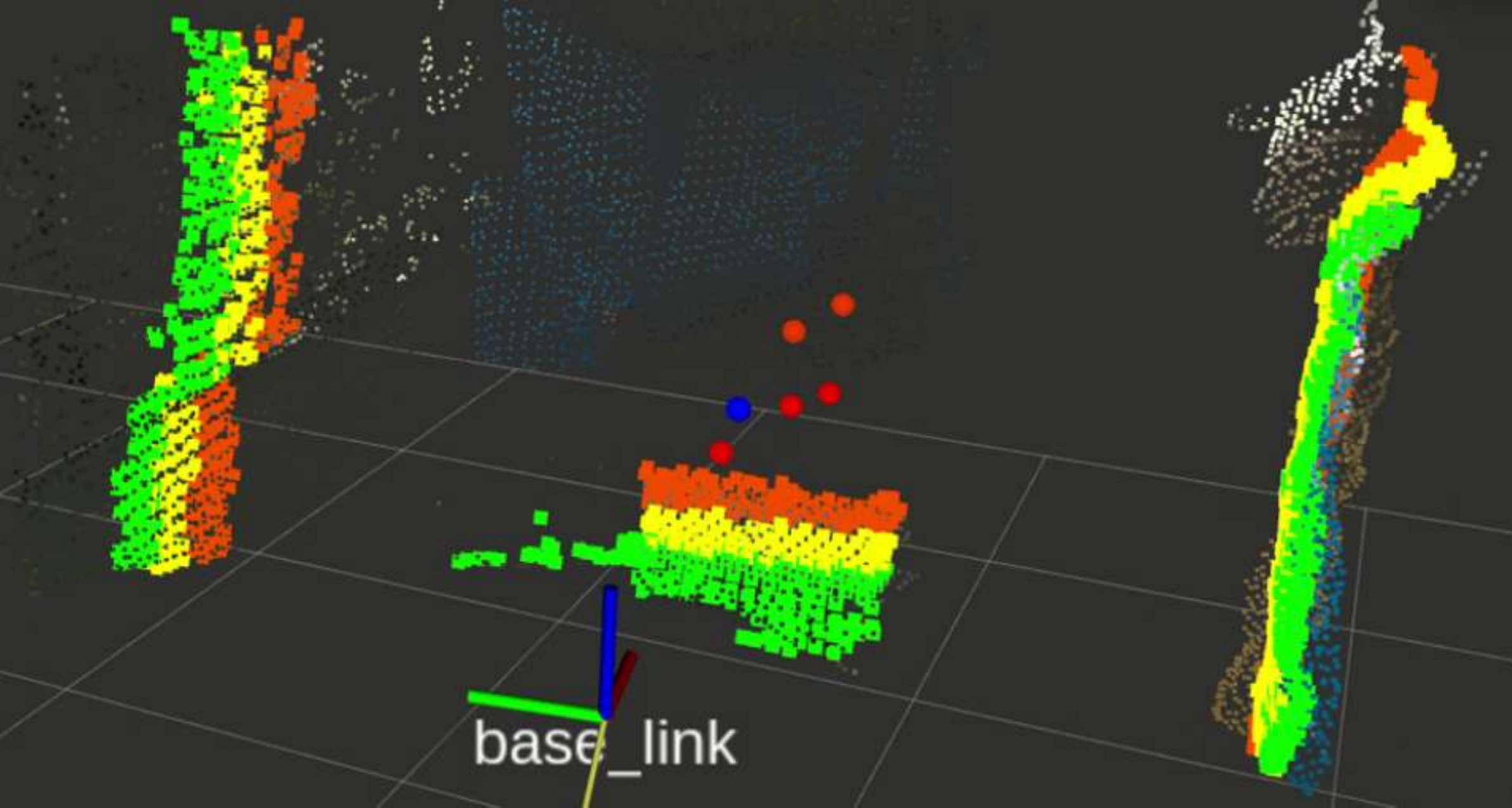} 
			\caption{}
			\label{fig:ch6:rviz3b}
		\end{subfigure}
	\end{adjustbox}
	\caption{Robust perception pipeline visualisation for cases 2 and 3}
	\label{fig:ch6:perception_rviz}
\end{figure}

\begin{figure}[!htb]
	\centering
	\begin{adjustbox}{minipage=\linewidth,scale=0.65}
	\includegraphics[clip,width=\textwidth]{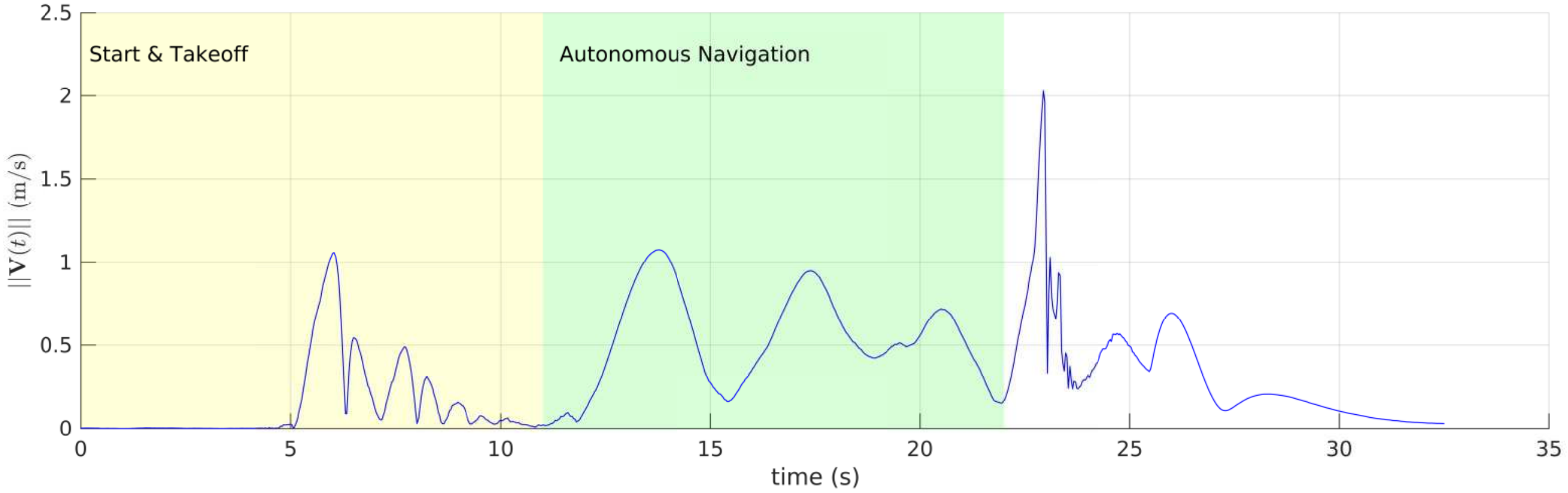} \\
	\includegraphics[clip,width=\textwidth]{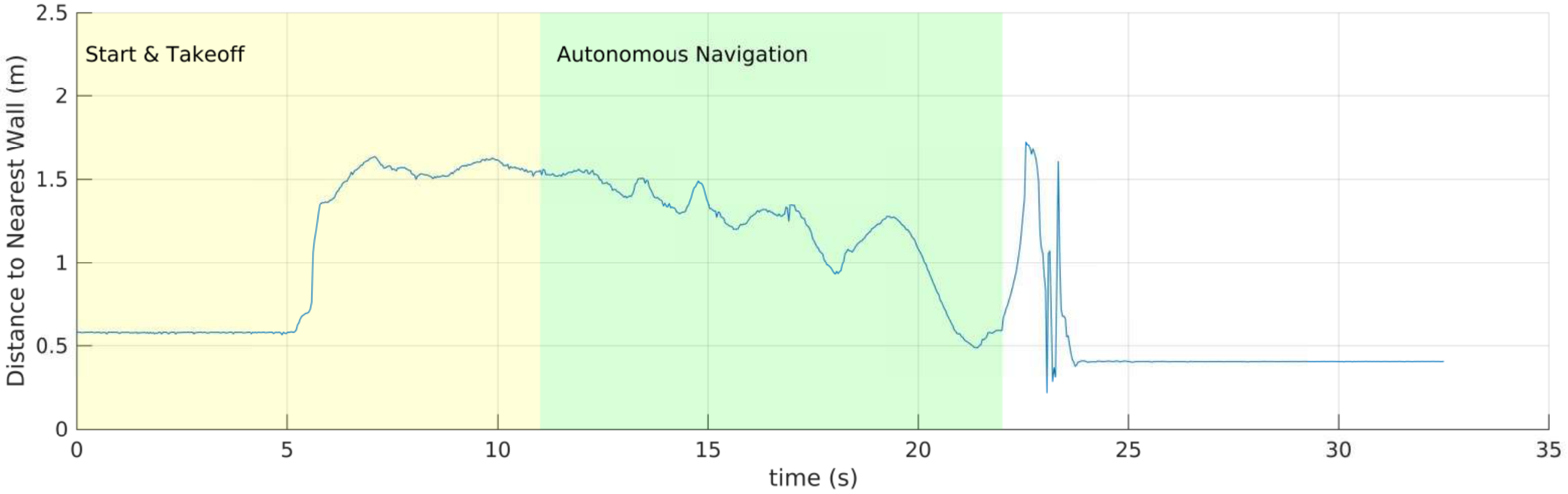}
	\end{adjustbox}
	\caption{UAV velocity \& distance to closest point versus time for case 1}
	\label{fig:ch6:results1}
\end{figure}

\begin{figure}[!htb]
	\centering
	\begin{adjustbox}{minipage=\linewidth,scale=0.65}
	\includegraphics[clip,width=\textwidth]{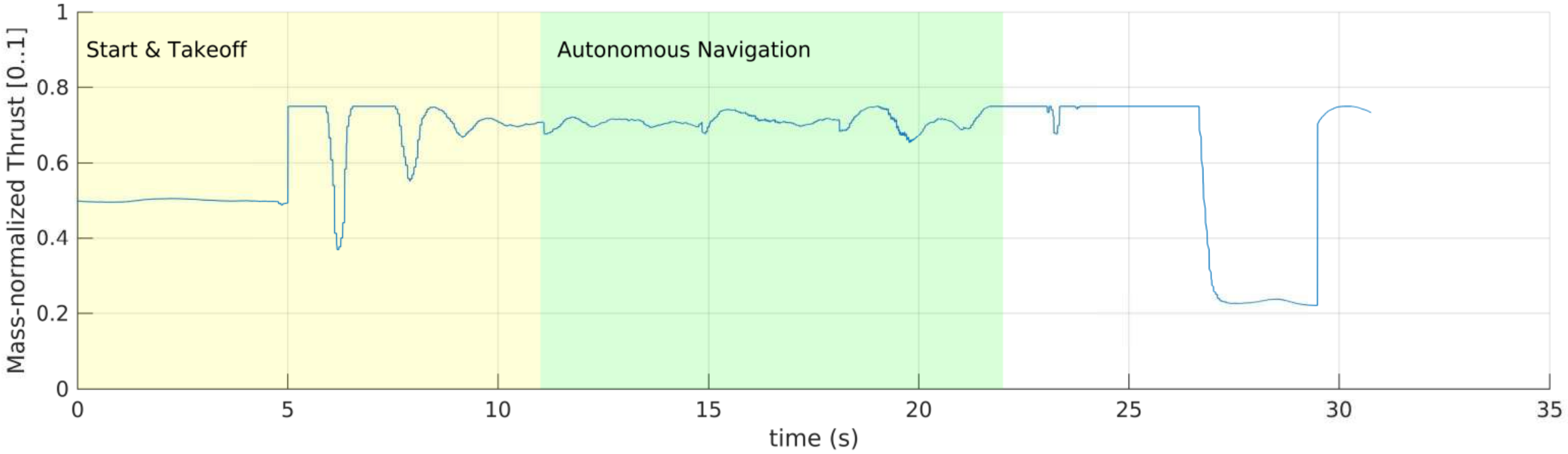} \\
	\includegraphics[clip,width=\textwidth]{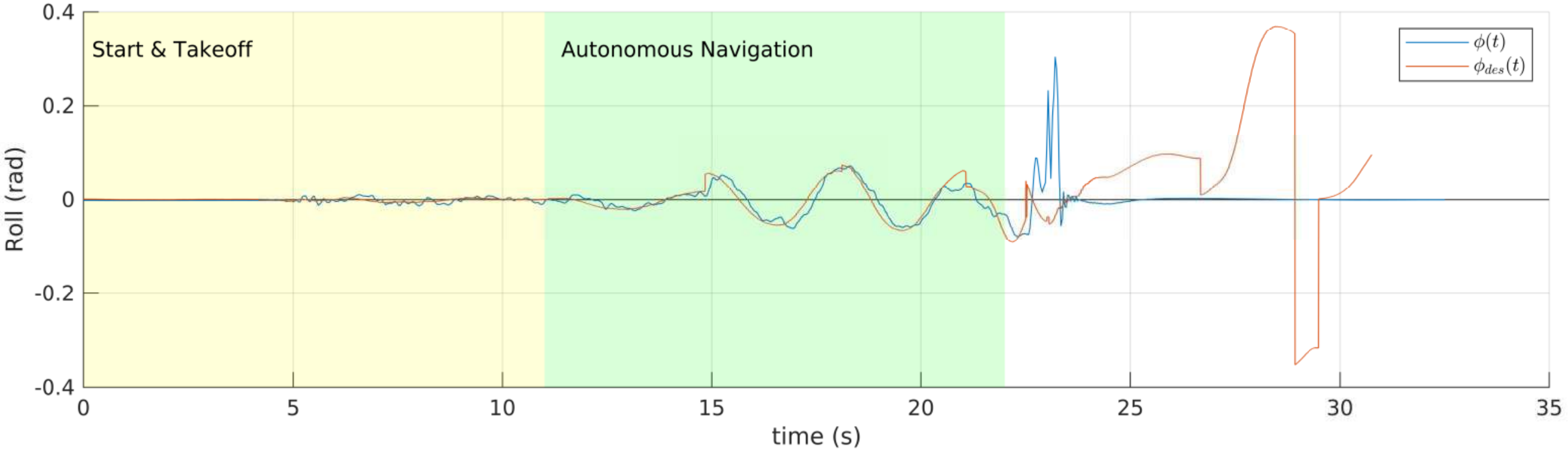} \\
	\includegraphics[clip,width=\textwidth]{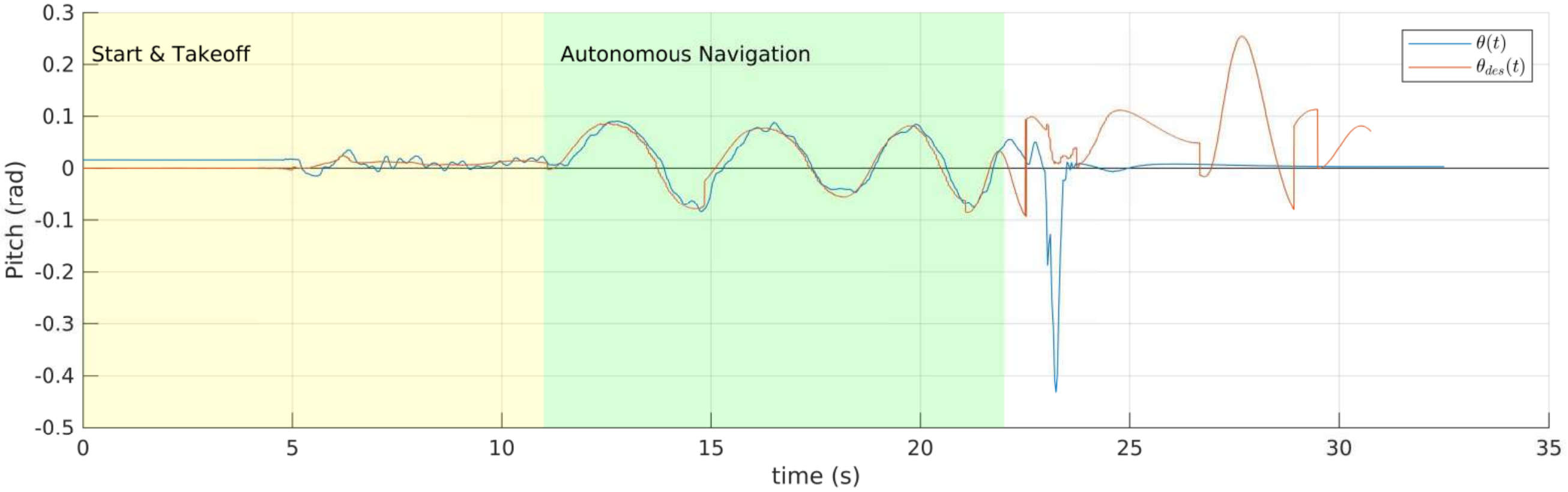} \\
	\includegraphics[clip,width=\textwidth]{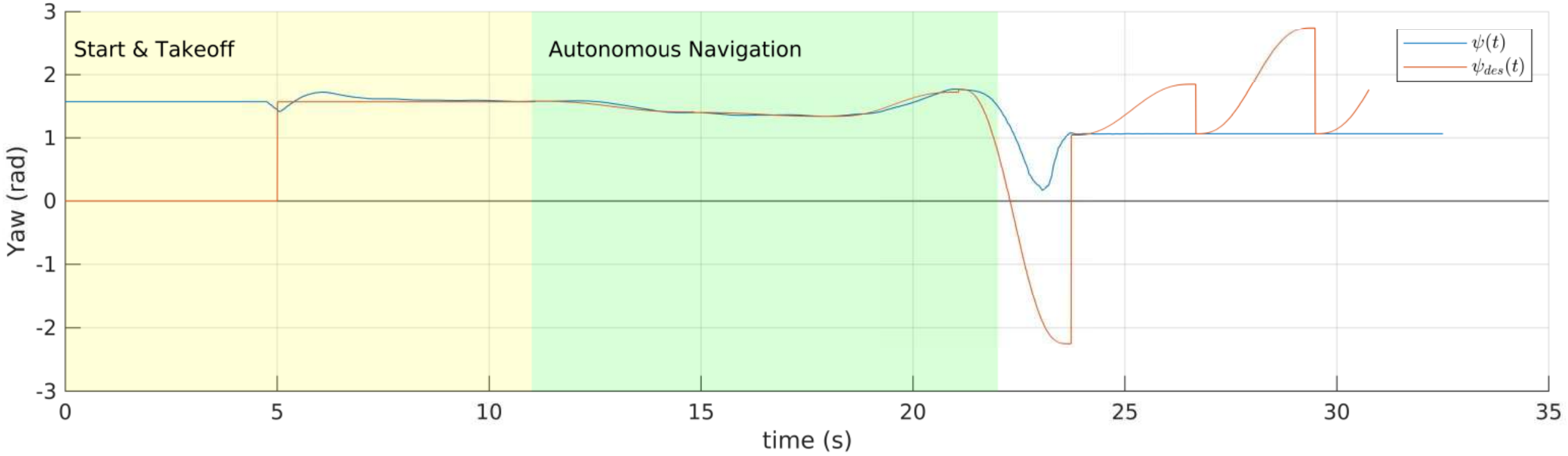}
	\end{adjustbox}
	\caption{Input mass-normalized collective thrust and attitude commands vs time for case 1}
	\label{fig:ch6:results1b}
\end{figure}

\begin{figure}[!htb]
	\centering
	\begin{adjustbox}{minipage=\linewidth,scale=0.65}
	\includegraphics[clip,width=\textwidth]{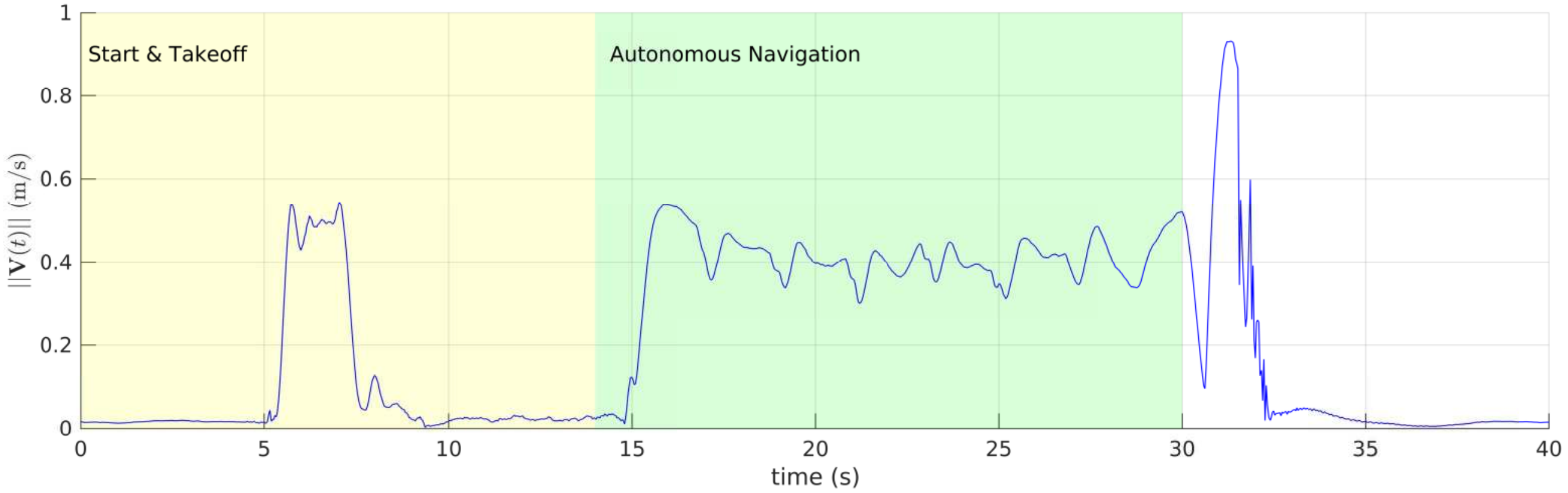} \\
	\includegraphics[clip,width=\textwidth]{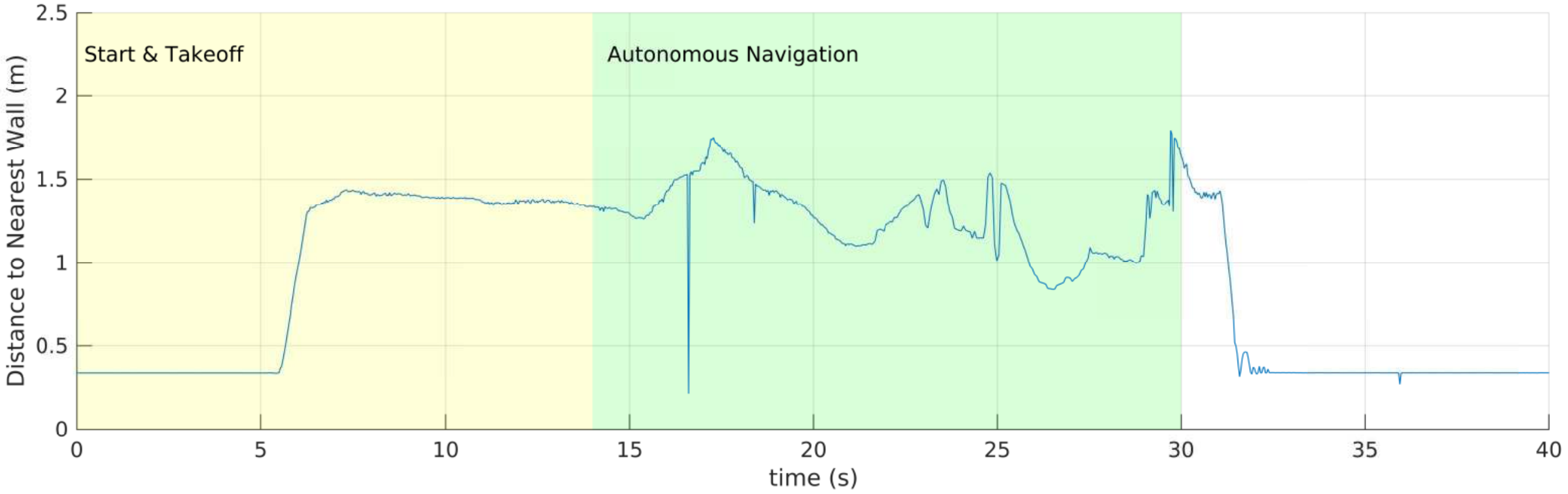}
	\end{adjustbox}
	\caption{UAV velocity \& distance to closest point versus time for case 2}
	\label{fig:ch6:results2}
\end{figure}

\begin{figure}[!htb]
	\centering
	\begin{adjustbox}{minipage=\linewidth,scale=0.65}
	\includegraphics[clip,width=\textwidth]{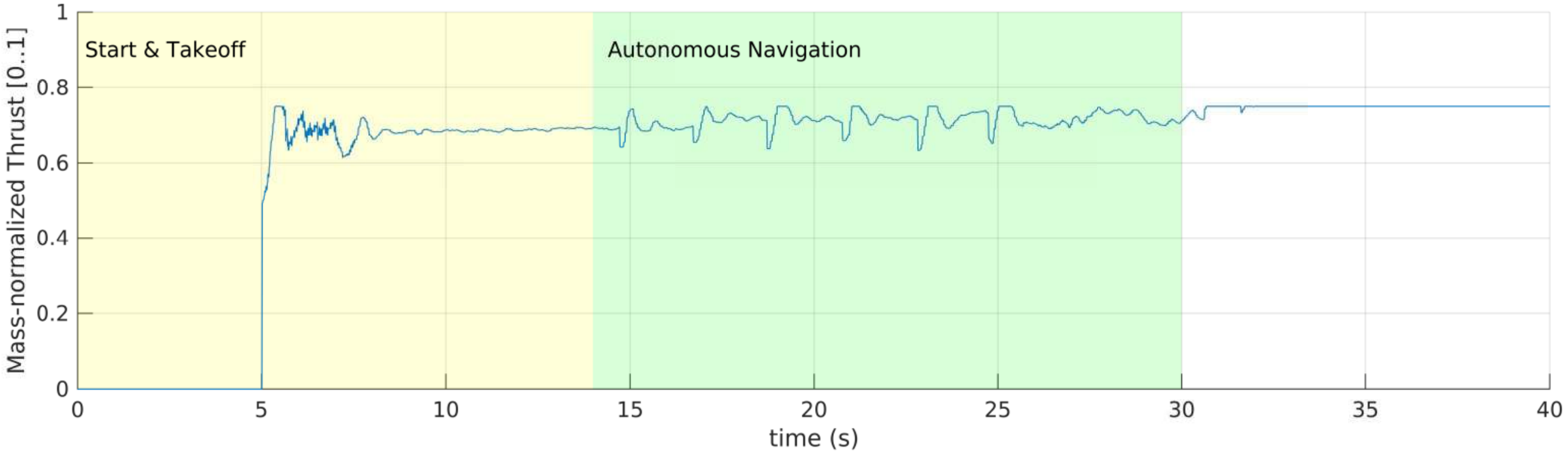} \\
	\includegraphics[clip,width=\textwidth]{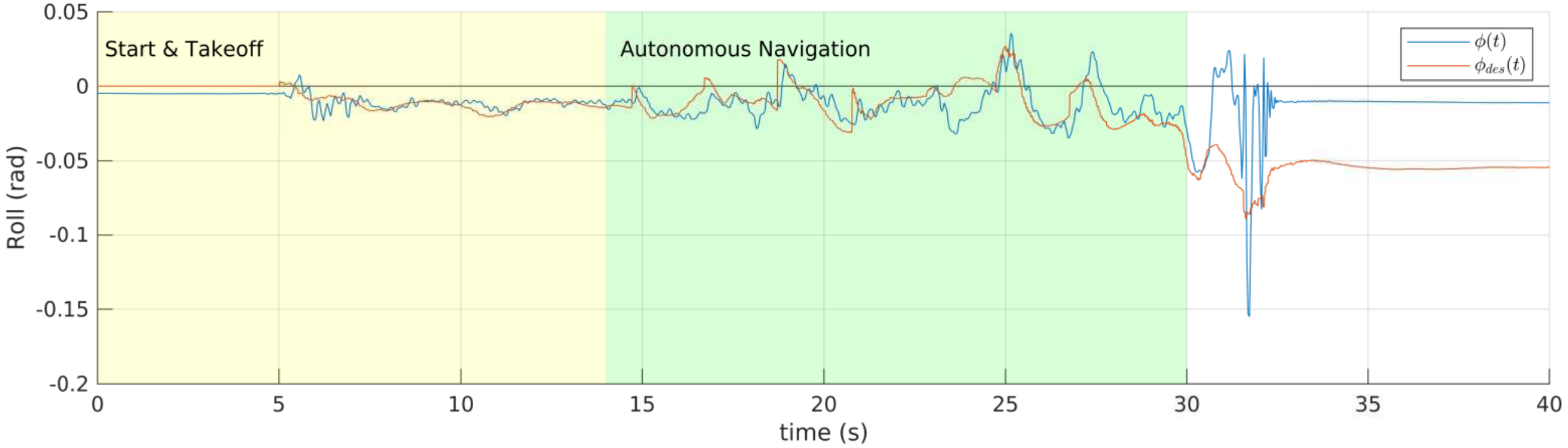} \\
	\includegraphics[clip,width=\textwidth]{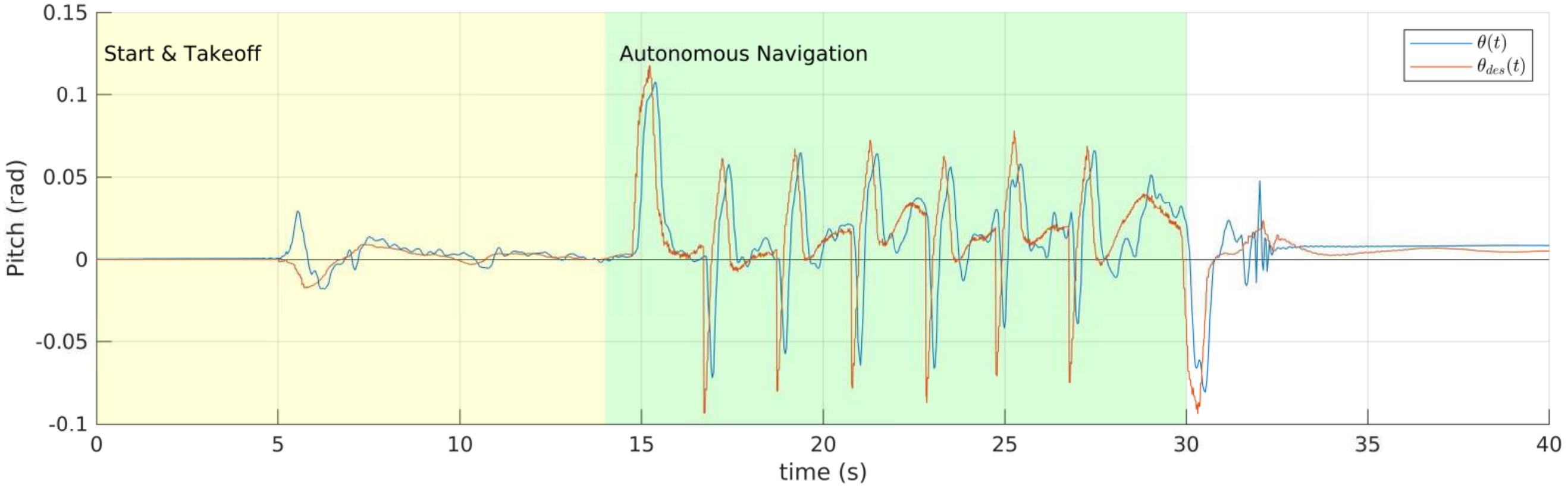} \\
	\includegraphics[clip,width=\textwidth]{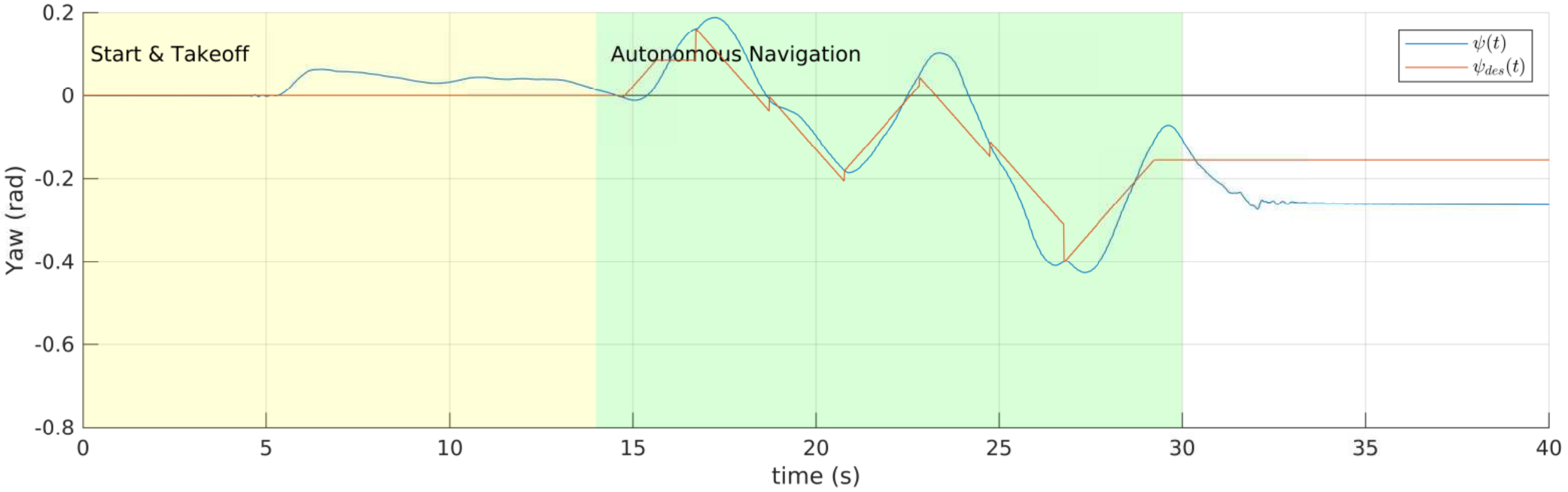}
	\end{adjustbox}
	\caption{Input mass-normalized collective thrust and attitude commands vs time for case 2}
	\label{fig:ch6:results2b}
\end{figure}

\begin{figure}[!htb]
	\centering
	\begin{adjustbox}{minipage=\linewidth,scale=0.65}
	\includegraphics[clip,width=\textwidth]{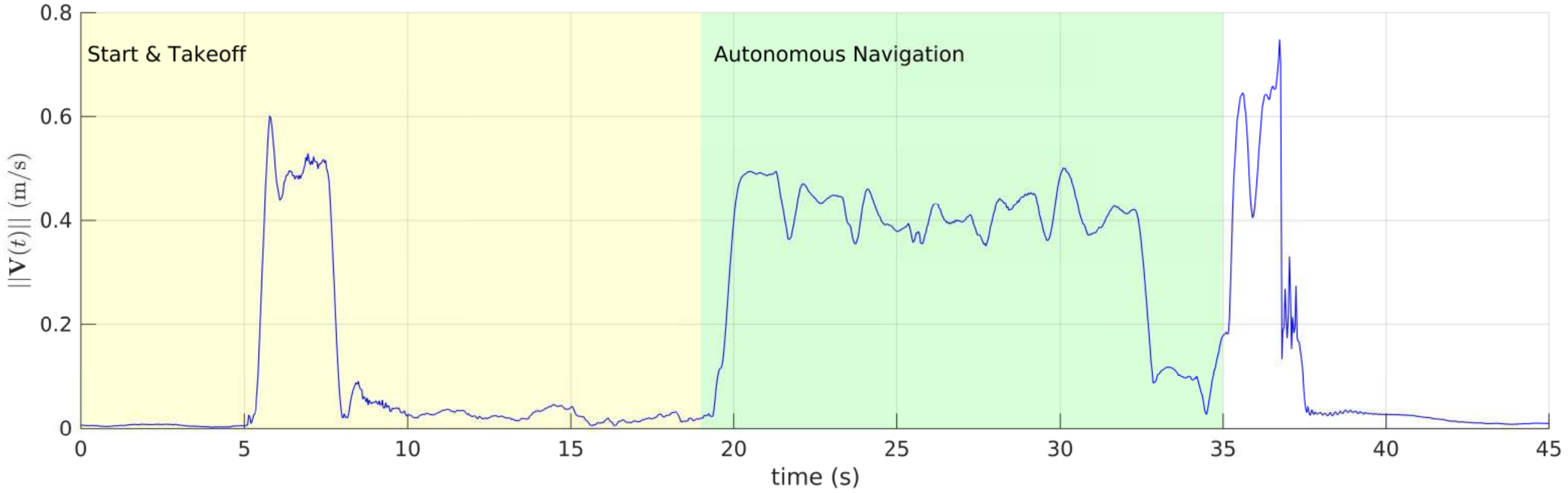} \\
	\includegraphics[clip,width=\textwidth]{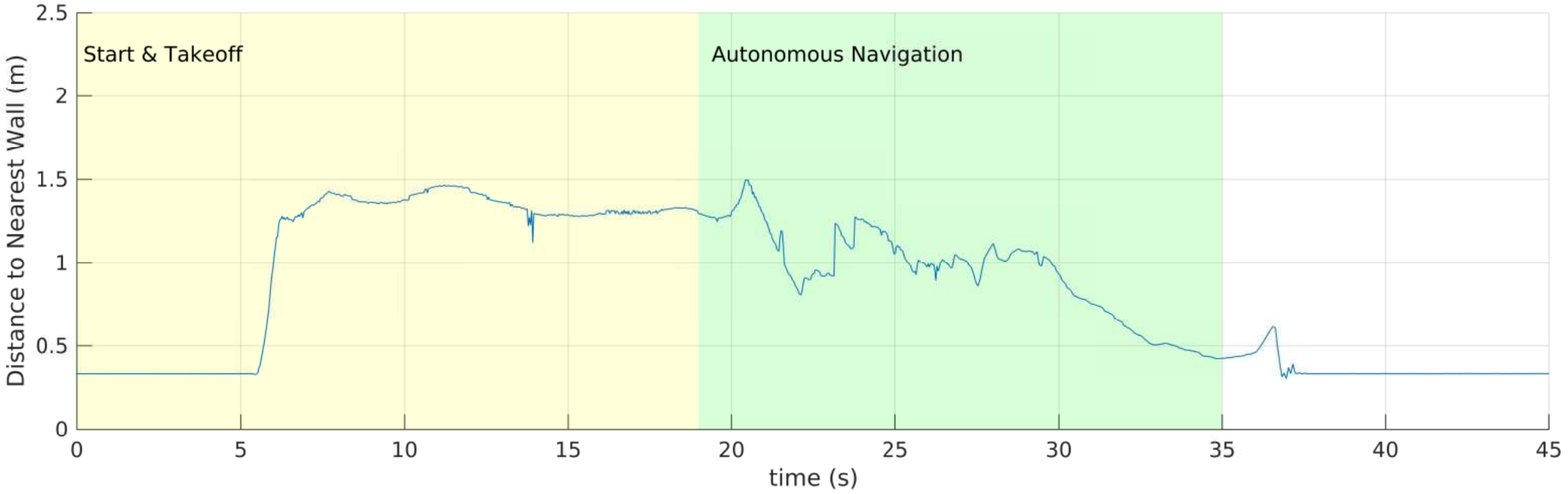}
	\end{adjustbox}
	\caption{UAV velocity \& distance to closest point versus time for case 3}
	\label{fig:ch6:results3}
\end{figure}

\begin{figure}[!htb]
	\centering
	\begin{adjustbox}{minipage=\linewidth,scale=0.65}
	\includegraphics[clip,width=\textwidth]{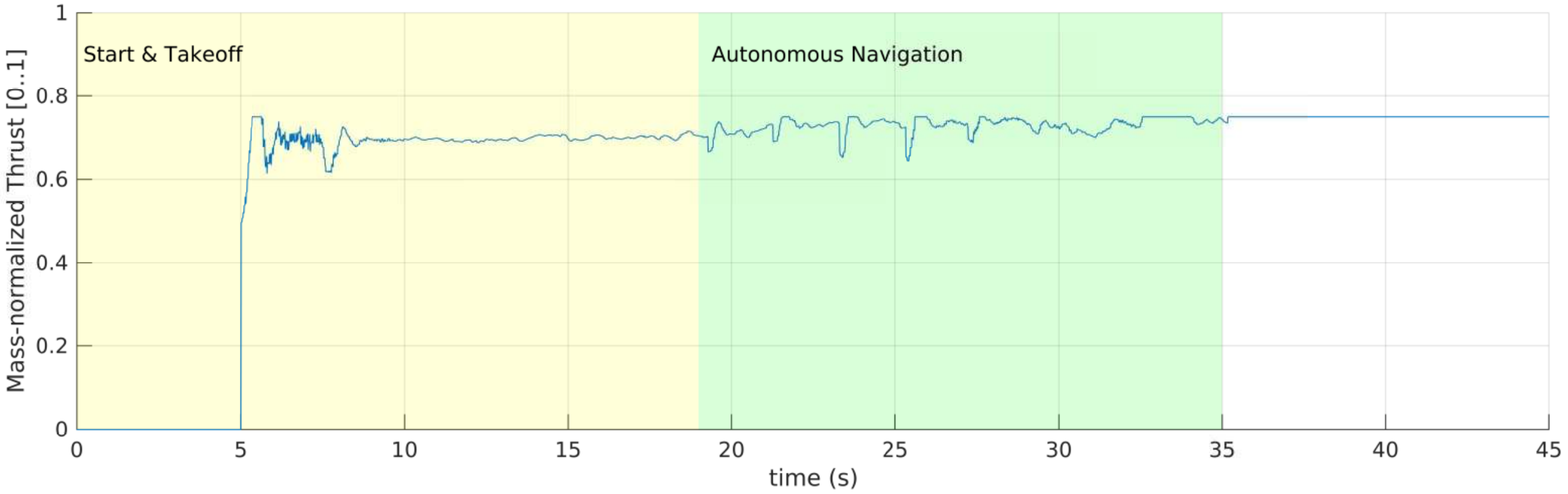} \\
	\includegraphics[clip,width=\textwidth]{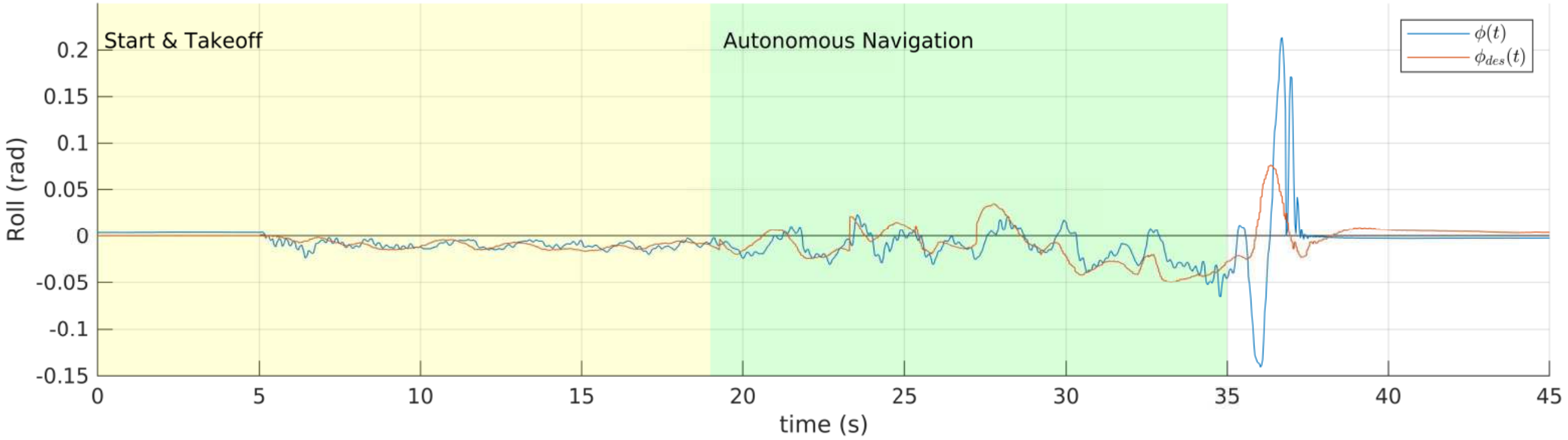} \\
	\includegraphics[clip,width=\textwidth]{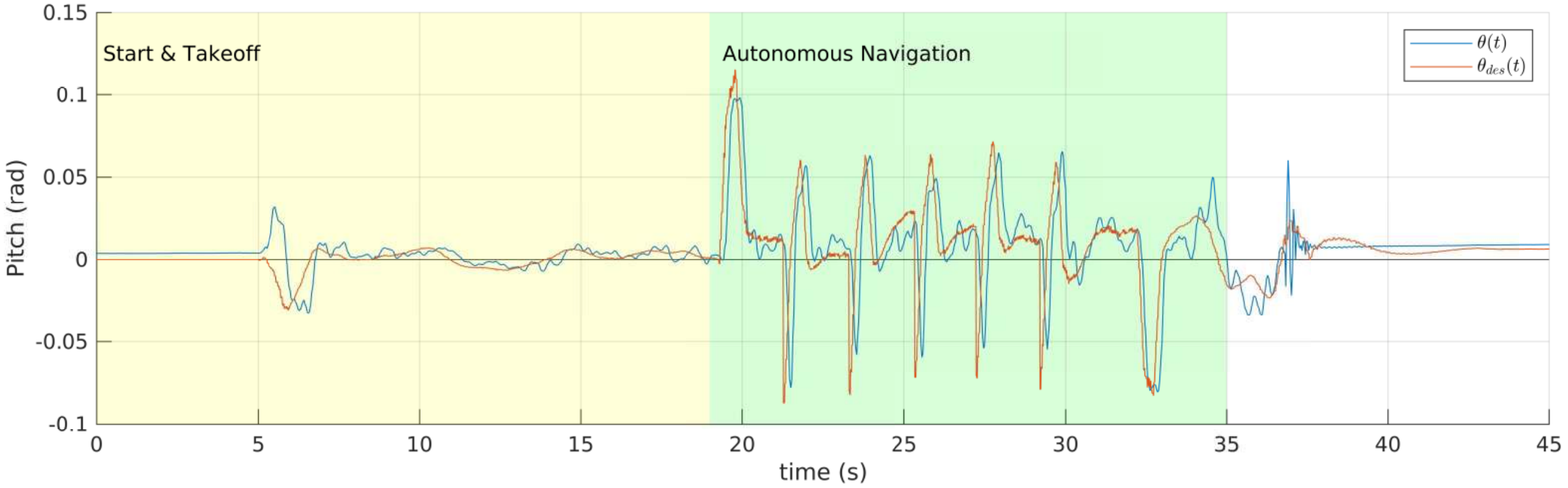} \\
	\includegraphics[clip,width=\textwidth]{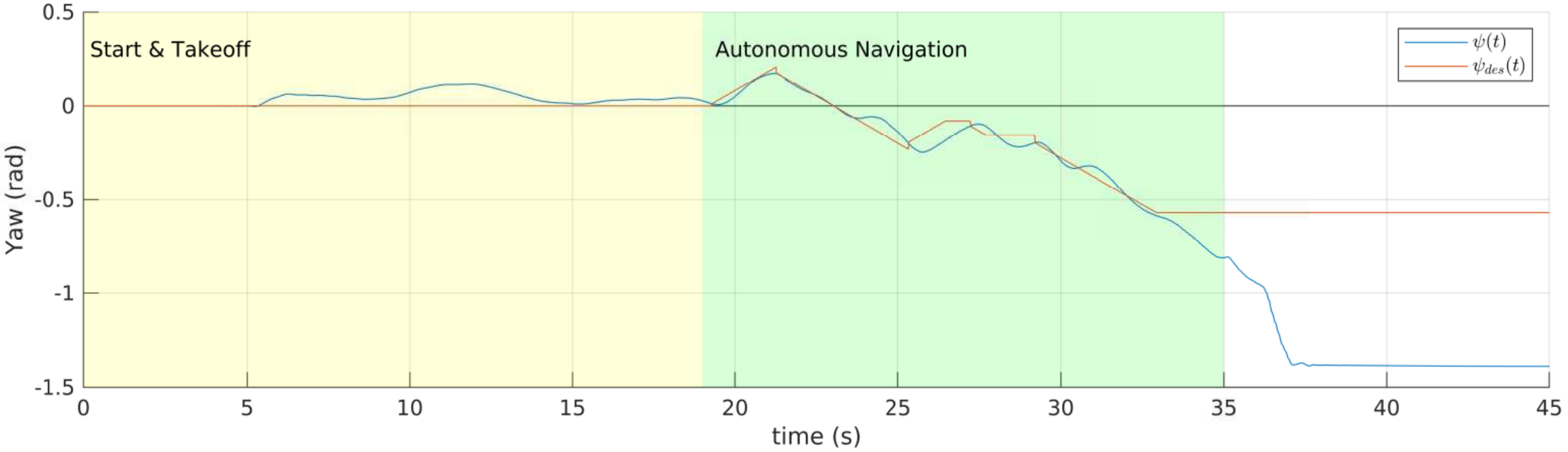}
	\end{adjustbox}
	\caption{Input mass-normalized collective thrust and attitude commands vs time for case 3}
	\label{fig:ch6:results3b}
\end{figure}

\section{Conclusion}\label{sec:conc}

This work presented a computationally-light method for UAVs to allow autonomous collision-free navigation in unknown tunnel-like environments.
It relies on light processing of sensors' measurements to guide the UAV along the tunnel axis.
A general 3D kinematic model is used for the development which extends the applicability of our method to different UAV types and autonomous underwater vehicles navigating through 3D tunnel-like environments. 
Several simulations were performed to validate our method considering tunnels with different structures using a realistic sensing model.
Robustness against noisy sensors measurements was also investigated in simulation.
Moreover, we provided implementation details for quadrotor UAVs including control design based on sliding mode control technique and differential-flatness property of quadrotor dynamics.
Experimental validation was done in a tunnel-like structure built in the laboratory where all computations needed by our navigation stack were done onboard.
Overall, the obtained results from simulations and the practical implementation show how well our navigation method works in unknown tunnel-like environments.

    \part{Motion Coordination for Multi-UAV Systems}

\chapter{Bounded Distributed Control of Multi-Vehicle Systems for Flocking Behaviour\label{cha:flocking_control}}

The previously developed safe navigation strategies can be adopted by vehicles within a multi-vehicle system in a decentralized manner.
However, it is more efficient to utilize information shared among the networked vehicles to generate more advanced motions.
Thus, this part of the report will focus more on developing distributed control methods for multi-vehicle systems.
The problem of motion coordination for multi-vehicle systems with collision avoidance is addressed in this chapter, specifically the flocking problem.
Control laws are developed to tackle this problem with a global objective of moving the whole system as a group to reach a goal region.
Collision and obstacle avoidance are considered as motion requirements within the control design, and the overall design is based on general kinematic models for 2D and 3D motions applicable to UAVs and AUVs.
The control laws are bounded to account for physical limits, and it is distributed in nature to ensure that the solution is scalable.
Stability analysis is done to show that the vehicles can safely navigate while avoiding collisions with other vehicles and respecting some safety margin, and related conditions on design parameters are provided to ensure that.
Several simulations with different number of vehicles moving in 3D are carried out to validate the performance of the developed control laws.
This chapter extends the results proposed in \cite{elmokadem2019flocking}.

\section{Introduction}\label{sec:ch8:Intro}

A great inspiration for collective behavior of multi-vehicle systems comes from different fields of biology.
Local interactions can contribute towards collective global outcomes without the need for central unit for coordination.
For example, many animal species can move coherently as a group without a leader such as bird flocks, fish schools, ants and bees \cite{gordon2014ecology}.
No doubt that multi-vehicle systems can provide more efficient solutions for many applications in terms of robustness, flexibility, cost and fault tolerance in comparison with single-vehicle systems.
This comes at the expense of increased system complexity which motivates lots of research in different areas related to the development of multi-vehicle systems including cooperative control and motion coordination strategies.

When tackling cooperative control problems, adopted methods may be either \textit{centralized}, \textit{decentralized} or \textit{distributed}.
\textit{Centralized} approaches use a central unit that compute control commands for all agents within the system which requires measurements from all agents to be available.
Obviously, such approaches are not robust nor scalable where it becomes more demanding computationally as the number of vehicles increases.
On contrary, \textit{decentralized} and \textit{distributed} approaches offer more robustness and scalability where each agent/vehicle can compute their own control actions either by only relying on its measurements (\textit{decentralized}) or by using its own measurements in addition to information communicated by neighbor agents (\textit{distributed}).
A great deal of existing cooperative control methods are either decentralized or distributed due to limited sensing and communication capabilities of mobile robots.
In the literature, some works may use the the terms decentralized and distributed interchangeably as one can consider the latter as a subset of the former.
It is clear that using distributed approaches would offer the best form of motion coordination inspired by biological systems where each agent in the group rely on local interactions to determine how to move contributing towards the collective behavior of the swarm.

Cooperative control of multi-vehicle systems is related to Networked Control Systems (NCSs) \cite{tipsuwan2003control,hespanha2007survey,wang2008networked,matveev2009estimation,bemporad2010networked,ge2017distributed}.
Thus, it is very important to keep in mind communication challenges in NCSs when designing control methods for large-scale networked multi-vehicle systems.
For example, we may assume that required information shared among the vehicles are available to the controller to facilitate the control design.
However, more practical aspects of NCSs can be further considered in the overall system design such as delays introduced in communication channels \cite{tipsuwan2003control,matveev2003problem,onat2010control}, noises \cite{matveev2007analogue,goodwin2010analysis}, loss/corruption of data \cite{matveev2003problem,onat2010control} and bandwidth constraints \cite{savkin2003set,matveev2004problem,savkin2006analysis,savkin2007detectability}.

Formation control, as a form of cooperative control, has become a very active field of research for multi-agent systems where the focus is to move multiple agents with some constraints on their states \cite{oh2015survey} to achieve some global objective(s).
The common formation control structures in the literature are \textit{leader-follower}, \textit{virtual} and \textit{behavioral-based structures}\cite{saif2019distributed}.
One of the agents is assigned as leader to be followed by the other agents within the group in \textit{leader-follower} structures such as \cite{mercado2013quadrotors,hou2015distributed,dehghani2016communication,xuan2019robust,walter2019uvdar,wang2019coordinated,tagliabue2019robust}.
Methods based on \textit{virtual structure} achieves motion formation through forcing each agent to follow a corresponding virtual target (or reference trajectory) such that the selection of these virtual references results in a desired formation.
Examples of such methods include \cite{ren2004decentralized,li2008formation,yoshioka2008formation,bayezit2012distributed,kushleyev2013towards,zhihao2020virtual}.
In a \textit{behavioral-based} structure, agents follow a set of rules contributing towards the collective behavior to achieve a certain formation or global objective(s).
Flocking is a collective behavior where a group of interacting agents needs to move together to achieve some global objectives; thus, flocking control can be classified as a subset of behavioral-based methods.
The local interactions between agents under a flocking behavior can be defined according to Reynolds' three rules based on his early model for the aggregate motion of flocks which are: flock centering (cohesion), collision avoidance (separation) and velocity matching (alignment) \cite{reynolds1987flocks,olfati2006flocking}.
A similar motion model yet simpler was proposed by Vicsek et al. \cite{vicsek1995novel} where each agent relies on local interactions in the form of information about its state and its neighbors to make motion decisions.
Several research works have addressed the flocking problem such as \cite{olfati2006flocking,tanner2007flocking,dimarogonas2008connection,savkin2010decentralized,reyes2014flocking,khaledyan2019flocking,do2011flocking,antonelli2010flocking,viragh2014flocking,ghapani2016fully,jafari2019biologically}.
Generally, flocking and other formation control laws design may differ based on the adopted model of the vehicles' motion.
Examples of considered models in the literature are single integrator \cite{ren2007distributed,ji2007distributed,antonelli2010flocking,saulnier2017resilient}, double integrator \cite{olfati2006flocking,tanner2007flocking,cao2011distributed,jafari2019biologically}, nonholonomic models \cite{liang2006decentralized,dimarogonas2007rendezvous,dimarogonas2008connection,savkin2010decentralized,do2011flocking,reyes2014flocking,khaledyan2019flocking} and Euler-Lagrangian systems \cite{yang2014fully,ghapani2016fully}.

In \cite{olfati2006flocking}, a theoretical framework for distributed flocking control of multi-agent systems was proposed along with flocking algorithms based on potential fields.
The algorithms were verified using systems of up to 100 agents with simulations in 2D and 3D environments; obstacle avoidance was only considered in the 2D cases.
This work motivated many research works in this field to rely on potential functions as a way to handle local interactions between agents according to Reynolds' rules.
Another potential-based flocking control law was analyzed in \cite{tanner2007flocking} based on a double integrator system showing robustness against arbitrary changes in sensing and communication networks as long as the network topology remains connected.
Obstacle avoidance was not considered in this work.

Control laws for single integrator and 2D nonholonomic models were suggested in \cite{dimarogonas2008connection} to regulate the inter-agent distances to achieve some desired formations without considering obstacle avoidance.
Since it is hard to set a desired formation in systems with large number of vehicles, the authors studied the connection between formation infeasibility and flocking behavior, and they provided an analytic expression for the common velocity vector all agents converge to in such situations.
A simple bio-inspired flocking control law based also on 2D nonholonomic models for wheeled vehicles was proposed in \cite{savkin2010decentralized} to move a group of vehicles in the same direction with equal speeds in an obstacle-free environment.
The work \cite{reyes2014flocking} presented a control law based also on a 2D nonholonomic model to address flocking, formation control and path following problems concurrently.
The authors also provided convergence analysis for the considered nonsmooth potential functions in their control law.
Similarly, the authors of \cite{khaledyan2019flocking} tackled the flocking problem of agents with 2D nonholonomic kinematics in combination with target interception as a global objective.
In \cite{do2011flocking}, a new pairwise potential function between two neighbor agents was considered in the derived flocking algorithm which adopted a 2D nonholonomic model as well; however, it further considered elliptical-shaped agents with limited communication ranges.

A Null-Space-based Behavioral (NSB) approach based on \cite{antonelli2008null} was proposed in \cite{antonelli2010flocking} to address the flocking problem considering single integrator kinematic models.
In this approach, each agent independently implements Reynolds' rules which were defined as behaviors with different priorities. 
A flocking control framework was presented in \cite{viragh2014flocking} taking into account realistic factors such as inertial effects, time delay, communication locality, sensors inaccuracy and refresh rates which was demonstrated by implementing two different flocking algorithms.
This framework was further evaluated experimentally in \cite{vasarhelyi2014outdoor} using a group of 10 multirotor UAVs in outdoor flights.
In \cite{ghapani2016fully}, flocking control laws using the leader-follower structure were suggested for networked Lagrange systems with parametric uncertainties.
The development of these laws was based on adaptive control theory, and they were validated in simulation using a system of four spacecrafts modeled with a 3D Lagrangian dynamical model.
Similarly, model uncertainties and unknown disturbances was considered in \cite{jafari2019biologically} where a neurologically-motivated distributed resilient flocking controller was proposed based on a double integrator model.
This approach aimed towards enabling the agents to track a virtual leader with collision avoidance while satisfying multiple control objectives such as control effort minimization and robustness against disturbances and model uncertainties.
Optimization-based methods were also utilized to address the flocking problem such as \cite{wang2017safety,ibuki2020optimization}.
For a more detailed literature review about formation control and coordination of multi-agent systems in general, the reader is referred to the surveys \cite{cao2013overview,dong2014time,oh2015survey,chung2018survey,hadi2021review}.

Many of the existing approaches were validated using systems of small sizes which may not show problems related to local minimums of adopted potential functions.
Also, the available methods based on nonholonomic kinematic models focused mostly on the 2D case, and obstacle avoidance was not considered in many approaches.
To address these limitations, this chapter proposes a new flocking control law for multi-vehicle systems by building on some of the concepts in \cite{olfati2006flocking,antonelli2010flocking,liang2006decentralized}.
A general 2D/3D kinematic model with nonholonomic constraints is adopted in the control design which make it applicable to different types of unmanned aerial vehicles (UAVs) and autonomous underwater vehicles (AUVs).

Addressing the flocking problem requires setting a global group objective; otherwise, the flocking behavior cannot be achieved according to \cite{olfati2006flocking}.
This is due to the fact that absence of global objective leads to breaking the flock into several disjoint flocks commonly known as the \textit{fragmentation} problem \cite{olfati2006flocking} which was also observed in \cite{savkin2004coordinated} using a simulated Vicsek model.
Therefore, a group objective of reaching or tracking a desired region is considered in this work with collision and obstacle avoidance as local objectives.

A bounded feedback control law is developed which is important to satisfy constraints on the vehicle's velocity and acceleration in a good computational way without the need to solve an optimization problem.
One of the main contributions in this new flocking control design is the proposal of null-space-based modified potential functions for goal reaching and collision/obstacle avoidance to avoid falling in a local minimum related to the inter-agent separation distances.
It is also designed particularly for 3D navigation which differ from many of the available flocking methods that adopt only 2D nonholonomic models.
Existing methods adopting 3D models are mostly based on single/double integrator which is one of main differences between our approach and \cite{antonelli2010flocking}. %

Furthermore, the designed control is analyzed through a stability analysis of the multi-vehicle system in addition to providing conditions on the design parameters for a guaranteed collision avoidance with respect to some safety margin.
Obstacle avoidance is also considered in the designed approach using some simple implementation for validation purposes.
However, different existing 3D reactive obstacle avoidance laws can be used with our flocking control law using the null-space-based modification to handle more complex environments.
Examples of such 3D methods can be seen in \cite{ren2008modified,matveev2015safe,yang20133d,savkin2013simple,hoy2015algorithms,elmokadem20183d} and references therein. 

This chapter is organized as follows.
In \cref{sec:ch8:problem}, we provide some preliminary information, the multi-vehicle system modeling and formulation of the tackled flocking problem.
After that, we propose a distributed flocking control law in \cref{sec:ch8:control} with proper stability analysis.
The performance of the suggested controller is then validated in \cref{sec:ch8:simulation} through simulations with systems of different sizes moving in 3D environments.
Finally, concluding remarks are made in \cref{sec:ch8:conclusion}.

\section{Preliminaries and Problem Statement}\label{sec:ch8:problem}

\subsection{Notation}\label{sec:ch8:notation}
Throughout the chapter, we represent scalar quantities using non-bold typeface letters/symbols while vectors and matrices are represented using bold typeface.
Consider the following definitions which are needed for the mathematical analysis:
$\bm{1}_k = [1\ 1\ \cdots\ 1]^T \in \R^k$, $\bm{0} = [0\ 0\ \cdots\ 0]^T\in \R^k$, and $\bm{I}_k$ is an identity matrix of size $k \times k$.
Moreover, whenever a well-known function is written in terms of a vector/matrix, the operation is intended to be element-wise unless otherwise stated.
For example, applying $\tanh(\bm{v})$ to a vector $\bm{v} = [v_1\ \cdots\ v_k]\in \R^k$ means the following $\tanh(\bm{v})= [\tanh(v_1)\ \tanh(v_2)\ \cdots\ \tanh(v_k)]^T$.
The following diagonal matrix is also defined:
\begin{equation}\label{equ:ch8:Tanh}
	\text{Tanh}(\bm{v}) = \left[\begin{array}{ccccc}
		\tanh(v_1) & 0 & 0 & \cdots & 0 \\
		0 & \tanh(v_2) & 0 & \cdots & 0 \\
		0 & 0 & \tanh(v_3) & \cdots & 0 \\
		\vdots & \vdots & \vdots & \ddots & \vdots \\
		0 & 0 & 0 & \cdots & \tanh(v_k)
	\end{array}\right] \in \R^{k\times k}
\end{equation}
such that $\text{Tanh}(\bm{v}) \bm{1}_k = \tanh(\bm{v})$.
We also use the notation $||\bm{v}||$ to express the Euclidean norm of $\bm{v}$ in $\R^k$.

\subsection{Multi-Vehicle System Modelling}\label{sec:ch8:model}
Consider a multi-vehicle system of size $n$ where each vehicle's motion in a space of dimension $m=\{2,3\}$ can be described using the following nonholonomic kinematic model:
\begin{equation}\label{equ:ch8:model}
	\begin{aligned}
		\dot{\bm{q}}_i &= v_i \bm{r}_i(\bm{\Theta}_i) \\
		\dot{\bm{\Theta}}_i &= \bm{\Omega}_i \\
		\dot{\bm{\nu}}_i &= \bm{\tau}_i
	\end{aligned}
\end{equation}
where $\bm{q}_i\in\R^m$ denotes the $i$-th vehicle position with respect to some inertial frame, $v_i\in\R$ is its linear speed, and $\bm{r}_i(\bm{\Theta}_i)\in \R^{m}$ is a unit vector expressing the vehicle's orientation which can be described in terms of orientation angles given by the vector $\bm{\Theta}_i \in \R^{m-1}$.
Changes in orientation can be described using the angular velocity vector $\bm{\Omega}_i \in \R^{m-1}$.
Also, a stacked vector of both the linear and angular velocities is denoted by $\bm{\nu}_i = [v_i,\ \bm{\Omega}_i]^T \ \in \R^m$.
The control inputs to this model are the linear and angular accelerations which are denoted by $\bm{\tau}_i = \left[\begin{array}{c}
a_i \\ \bm{\alpha}_i
\end{array}\right] \in \R^m$ (i.e $\dot{v}_i = a_i$ and $\dot{\bm{\Omega}}_i = \bm{\alpha}_i$).
The above model can describe the motion of different UAV types, autonomous underwater vehicles and unmanned ground vehicles.

Each vehicle can estimate its position, orientation and linear/angular velocities.
Any two vehicles within the system can exchange information as long as they are within a certain communication range $r_c>0$ from each other.
Thus, information exchange within the multi-UAV system can be modelled using concepts from graph theory which is summarized next based on \cite{olfati2006flocking}.

The networked multi-vehicle system can be described using a graph $G$ which consists of the pair $(\V,\E)$ where $\V$ is a set of vertices (ex. vehicles) and $\E$ is a set of edges $\E \subseteq \{(i,j): i,j\in \V, j \neq i \}$.
A graph is called \textit{undirected} if $(i,j)\in \E \leftrightarrow (j,i)\in \E$; otherwise, it is called \textit{directed}.
A \textit{path} between two vertices $i$ and $i$ is the sequence of edges connecting vertex $i$ to vertex $j$ through some intermediate vertices.  
The connectivity between vehicles can be better represented using an \textit{adjacency matrix} $A=[a_{ij}]$ of $G$ which contains non-zero elements such that $a_{ij} \neq 0$ for $(i,j) \in \E$, and $a_{ij} = 0$ otherwise.
If the adjacency matrix has full rank, the graph is called \textit{connected}.
In other words, there exists a path connecting every two vertices in $\V$.

Furthermore, a \textit{neighbourhood} of vertex $i$ is defined as follows
\begin{equation}
\mathpzc{N}_{i} = \{j\in \V \backslash \{i\}: (i,j)\in \E \}
\end{equation}
In practice, such neighbourhood can be defined in terms of some communication range $r_c>0$ as:
\begin{equation}\label{equ:ch8:comm_neighborhood}
\mathpzc{N}_{i} = \{j\in \V \backslash \{i\}: ||\bm{q}_i - \bm{q}_j|| \leq r_c \}
\end{equation}
This means that vehicle $i$ can communicate only with vehicles in $\mathpzc{N}_{i}$.
If this neighbourhood remains fixed as the vehicles move, the communication topology is called fixed; otherwise, it is called a dynamic topology (i.e. $\mathpzc{N}_{i}(t)$ changes over time).
Additionally, one can stack the position vectors $\bm{q}_i$ of all vehicles belonging to the graph (i.e. $i \in \mathcal{V}$) to form the graph configuration vector $\bm{q} = [\bm{q}_1^T,\ \cdots,\ \bm{q}_n^T]^T \in \R^{mn}$.

\subsection{Problem Statement}

\begin{problem}\label{prob:ch8:flocking_control}
	Design a distributed control law for multi-vehicle systems which can be modelled according to \cref{sec:ch8:model} to ensure flocking behaviour by achieving the following objectives:
	\begin{itemize}
	 	\item \textit{G1}: The multi-vehicle system can move towards a goal region $\goal \subset\R^m$ which is defined as a closed circle/ball with a radius $R_G$ whose center is at some location $\bm{q}_G \in \R^m$ (i.e. $\goal=\{\bm{q}_i:\ ||\bm{q}_i - \bm{q}_G|| \leq R_G\ \forall i \in \mathcal{V}\}$).
	 	\item \textit{G2}: Distances between a vehicle and its neighbours should be kept at some desired value $d_{ij}$ such that $||\bm{q}_i - \bm{q}_j|| = d_{ij}\ \forall j \in \mathpzc{N}_{i}$ for $t \geq 0$.
	 	\item \textit{G3}: Vehicles must avoid collisions with each others (i.e. $||\bm{q}_i - \bm{q}_j||> d_{s}\ \forall t$ where $d_s < d_{ij}$ is some safety margin).
	 	\item \textit{G4}: Vehicles must avoid collisions with any obstacles within the environment.
	\end{itemize}
\end{problem}

\begin{assumption}
	The radius $R_G$ is large enough such that there exist a formation where all the vehicles can maintain a separation distance of $d_{ij}$ while remaining within $\goal$
	(i.e. it is possible to satisfy both \textit{G1} and \textit{G2} simultaneously).
\end{assumption}

\begin{remark}
	The control objective \textit{G2} will result in a formation of a geometric structure referred to as $\alpha-$lattice where vehicles are separated by a certain distance (see \cite{olfati2006flocking} for a detailed description of $\alpha-$lattices).
\end{remark}                                                               

\section{Distributed Control Design and Stability Analysis}\label{sec:ch8:control}

A flocking control method was developed to address \cref{prob:ch8:flocking_control} based on some ideas from \cite{olfati2006flocking,antonelli2010flocking,liang2006decentralized}.
Details about the designed control laws and stability analysis are provided in this section.

\subsection{Aggregate Potential Function}

Consider the goal region as was defined according to \textit{G1} and the desired separation distance $d_{ij} \in \R_{+}$ defined in the objective \textit{G2}.
We also define a critical region $\mathcal{R}_{\obs} \subset \R^m$ around a nearby obstacle $\obs_k\subset \R^m$ by:
\begin{equation}
\mathcal{R}_{\obs} = \{\bm{q}^*\in \obs_k: ||\bm{q}_i - \bm{q}^*|| \leq C \}
\end{equation}
where $C > 0$, and $\bm{q}^* = argmin_{\bm{x} \in \obs_k}||\bm{x} - \bm{q}_i||$ (i.e. the closest point on the nearest obstacle's boundary).
One can then define the following scalar errors:
\begin{align}
q_{ij} &= \|\bm{q}_i - \bm{q}_j\| - d_{ij} \label{equ:ch8:error_formation} \\
q_{iG} &= \|\bm{q}_i - \bm{q}_G\| \label{equ:ch8:error_goal} \\
q_{i \obs} &= \|\bm{q}_{i} - \bm{q}^*\| \label{equ:ch8:error_obs_avoid}
\end{align}
Using the above definitions, we define an aggregate potential function for each vehicle using:
\begin{equation}\label{equ:ch8:potential_function}
	U_i(\bm{q}_i) = U_{i,\alpha}(q_{ij}) + U_{i,G}(q_{iG}) + U_{i,\obs}(q_{i \obs})
\end{equation}
where $U_{i,\alpha}(q_{ij})$ represents an inter-vehicle potential function to maintain the group's formation associated with repulsive/attractive forces, $U_{i,G}(q_{iG})$ corresponds to attractive forces to move the vehicle towards the goal region, and $U_{i,\obs}(q_{i \obs})$ corresponds to repulsive forces from nearby obstacles.

Let the inter-vehicle potential function $U_{i,\alpha}(q_{ij})$ be defined according to the following \cite{liang2006decentralized}:
\begin{equation} \label{equ:ch8:U_i_ca}
	U_{i,\alpha} = \frac{1}{2}\sum_{j \in \mathpzc{N}_{i}} k_{ij} \ln\left(\cosh\left(q_{ij}\right)\right) %
\end{equation}
where $k_{ij}>0$ is a design parameter.
The gradient-based force of the above function can be obtained as follows:
\begin{equation}\label{equ:ch8:sep_gradient}
  \bm{f}_{i,\alpha} = -\nabla U_{i,\alpha} = \sum_{j \in \mathpzc{N}_{i}} k_{ij} \tanh(q_{ij})\ \bm{n}_{ij} 
\end{equation}
where $\bm{n}_{ij} = \dfrac{\bm{q}_{j} - \bm{q}_{i}}{||\bm{q}_{j} - \bm{q}_{i}||}$ is a unit vector directing from vehicle $i$ to $j$.
The global minimum of the smooth function $U_{i,\alpha}$ occurs at $q_{ij} = 0$ which implies $||\bm{q}_i - \bm{q}_j|| = d_{ij}$.

The goal potential function $U_{i,G}(q_{iG})$ is designed to have a global minimum whenever the vehicle reaches the goal region $\bm{q}_i \in \goal$ (i.e. $q_{iG} \leq R_G$).
This is possible by choosing the gradient of $U_{i,G}(q_{iG})$ according to the follows:
\begin{equation}\label{equ:ch8:goal_gradient}
	\bm{f}_{i,G} = -\nabla U_{i,G} = k_{i,G} \Gamma_{G}(q_{iG},R_G)\ \bm{n}_{iG}
\end{equation}
where $k_{i,G}$ is a positive design parameter, and $\bm{n}_{iG} = \dfrac{\bm{q}_{G} - \bm{q}_{i}}{||\bm{q}_{G} - \bm{q}_{i}||}$ is a unit vector in the direction of $\bm{q}_{G}$ (i.e. towards the goal region).
Also, $\Gamma_{G}(q_{iG},R_G)$ is a smooth sigmoid function that vanishes whenever the vehicle reaches a region of interest which is the goal region $\goal$ in this case.
It should satisfy the following:
\begin{equation}
	\Gamma_{G}(q_{iG},R_G) = \left\{\begin{array}{lr}
	0 & q_{iG} \leq R_G \\
	\mu & otherwise
	\end{array}\right.
\end{equation}
where $\mu \in (0,1]$. 
In a similar manner, $U_{i,\obs}(q_{i \obs})$ is designed to have a global minimum whenever the vehicle is outside $\mathcal{R}_{\obs}$ in accordance with the following:
\begin{equation}\label{equ:ch8:obs_gradient}
	\bm{f}_{i,\obs} = -\nabla U_{i,\obs} = k_{i,\obs}\ \Gamma_{\obs}(q_{i \obs},d_{s})\ \bm{n}_{i\obs}
\end{equation}
where $k_{i,\obs}>0$, $\bm{n}_{i\obs}$ is a unit vector directing towards a safe direction away from a nearby obstacle, and $\Gamma_{\obs}(q_{i \obs},d_{s})$ is a smooth sigmoid function that vanishes whenever the vehicle is at a distance from the obstacle larger than some safety margin $d_s>0$.
That is, it should satisfy the following:
\begin{equation}
\Gamma_{\obs}(q_{i \obs},d_{s})= \left\{\begin{array}{lr}
0 & \bm{q}_i \notin \mathcal{R}_{\obs} \\
\mu & otherwise
\end{array}\right.
\end{equation}
where $\mu \in (0,1]$.
One of the possible designs for both $\Gamma_{G}(q_{iG},R_G)$ and $\Gamma_{\obs}(q_{i \obs},d_{s})$, considered in this work, is as follows:
\begin{equation}\label{equ:regionFunction}
\Gamma(z) = 0.5 + 0.5 \tanh(\gamma z)
\end{equation}
where $\gamma>0$, and $z$ defines the distance to some region of interest.
Thus, the parameter $z$ is defined as $z = ||\bm{q}_i - \bm{q}_G|| - R_G$ for $\Gamma_G$ and $z = C - ||\bm{q}_i - \bm{q}^*||$ for $\Gamma_{\obs}$.

Using \eqref{equ:ch8:sep_gradient}, \eqref{equ:ch8:goal_gradient} and \eqref{equ:ch8:obs_gradient}, the gradient of the aggregate potential function in \eqref{equ:ch8:U_i_ca} can be defined as follows:
\begin{equation}\label{equ:ch8:fi_total}
	-\nabla U_i := \bm{f}_i = \bm{f}_{i,\alpha} +\bm{f}_{i,G} + \bm{f}_{i,\obs}
\end{equation}

\begin{remark}
	Another possible way of designing \eqref{equ:ch8:goal_gradient} is by choosing $\bm{n}_{iG}$ to be in the direction of the whole goal region using a projection from the current position rather than having it directing towards just the center $\bm{q}_{G}$ of the goal region.
\end{remark}

\subsection{Null-space-based Modified Potential Function}

Using the aggregate potential energy function \eqref{equ:ch8:potential_function} in a potential-based approach may result in getting stuck at local minimas.
Therefore, a modified potential function is introduced in this section based on the null-space-based behavioural approach presented in \cite{antonelli2010flocking} to escape local minimum situations.
To that end, the vector field \eqref{equ:ch8:fi_total} is modified as follows:
\begin{equation}\label{equ:ch8:nsbEq}
	\bm{\tilde{f}}_i = \bm{f}_{i,1} + \bm{N}_{i,1} \bm{f}_{i,2} + \bm{N}_{i,2} \bm{f}_{i,3}
\end{equation}
where $\bm{f}_{i,1}$, $\bm{f}_{i,2}$ and $\bm{f}_{i,3}$ are some force vectors ordered by importance from highest to lowest (for example, \eqref{equ:ch8:sep_gradient}, \eqref{equ:ch8:goal_gradient} or \eqref{equ:ch8:obs_gradient}).
Also, $\bm{N}_{i,1}$ and $\bm{N}_{i,2}$ are projection matrices which are defined as follows:

\begin{align}
	\bm{N}_{i,1} &= \bm{I}_m - \bm{\bar{f}}_{i,1}\ \bm{\bar{f}}_{i,1}^T \label{equ:ch8:nsbN1} \\
	\bm{N}_{i,2} &= \bm{N}_{i,1} (\bm{I}_m - \bm{\bar{f}}_{i,2}\ \bm{\bar{f}}_{i,2}^T) \label{equ:ch8:nsbN2}
\end{align}
where $\bm{\bar{f}}_{i,k},\ k=\{1,2,3\}$ is a normalized vector such that $\bm{\bar{f}}_{i,k} = \bm{f}_{i,k}/||\bm{f}_{i,k}||$.
The second term of \eqref{equ:ch8:nsbEq} represents the projection of $\bm{f}_{i,2}$ into the null-space of vector $\bm{f}_{i,1}$.
Similarly, the last term corresponds to the projection of $\bm{f}_{i,3}$ into the null-space of $\bm{N}_{i,1} \bm{f}_{i,2}$ (i.e. the null-space of both $\bm{f}_{i,1}$ and $\bm{f}_{i,2}$).

Currently, we consider the following sequence forces ordered from the most critical to the least critical: $\{\bm{f}^{OA},\ \bm{f}^{\alpha},\ \bm{f}^{G}\}$ where $\bm{f}^{OA}$ is only considered whenever the vehicle enters the critical region around any obstacle (i.e. $\bm{q}_i \notin \mathcal{R}_{\obs}$).
Thus, the proposed modified force vector field is written based on \eqref{equ:ch8:fi_total} and \eqref{equ:ch8:nsbEq} as:
\begin{equation}\label{equ:ch8:nsbEqFinal}
	\bm{\tilde{f}}_i = \bm{f}_i + \bm{\xi}(\bm{f}_{i,1},\bm{f}_{i,2},\bm{f}_{i,3})
\end{equation}
where $\bm{\xi}:=\bm{\xi}(\bm{f}_{i,1},\bm{f}_{i,2},\bm{f}_{i,3})$ is defined by:
\begin{equation}\label{equ:ch8:xiBound}
\bm{\xi} = - \bm{\bar{f}}_1 \bm{\bar{f}}_1^T \bm{f}_{i,2} - (\bm{\bar{f}}_1 \bm{\bar{f}}_1^T + \bm{\bar{f}}_2 \bm{\bar{f}}_2^T - \bm{\bar{f}}_1 \bm{\bar{f}}_1^T \bm{\bar{f}}_2 \bm{\bar{f}}_2^T) \bm{f}_{i,3}
\end{equation}
One can see from \eqref{equ:ch8:xiBound} that the following property is true:
\begin{equation}
	||\bm{\xi}|| \leq ||\bm{f}_{i,2}|| + ||\bm{f}_{i,3}||
\end{equation}
In \eqref{equ:ch8:xiBound}, a general form was used to consider different possible ordered sequences of potential forces based on importance.
However, the considered sequence here as mentioned before is such that:
\begin{equation}\label{equ:ch8:ordered_seq}
	\{\bm{f}_{i,1},\ \bm{f}_{i,2},\ \bm{f}_{i,3}\}=\{\bm{f}_{i,\obs},\ \bm{f}_{i,\alpha},\ \bm{f}_{i,G}\}
\end{equation}

\begin{remark}
	For a two-dimensional workspace where $m=2$, only two forces will have effect on \eqref{equ:ch8:nsbEq} because the null-space of both $\bm{f}_{i,1}$ and $\bm{f}_{i,2}$ in $\R^2$ is an empty set (i.e. $\bm{N}_{i,2}=\bm{0}$).
	Hence, the considered ordered sequence can be dynamically changing based on the importance of $\bm{f}^{OA},\ \bm{f}^{\alpha}$ and $\bm{f}^{G}$ which can vary based on the situation (see \cite{antonelli2010flocking}).
\end{remark}

\subsection{Control Design}

The proposed distributed flocking control design can now be presented in this section.
To that end, consider the error vectors defined in \eqref{equ:ch8:error_formation}-\eqref{equ:ch8:error_obs_avoid} and the modified vector field \eqref{equ:ch8:nsbEqFinal}.
Note that the orientation of the vector $\bm{\tilde{f}}_i$ in $\R^m$ can be defined as:
\begin{equation}\label{equ:ch8:f_i_new}
	\bm{\tilde{f}}_i (\bm{q}_{i}) = ||\bm{\tilde{f}}_i (\bm{q}_{i})|| \bm{r}(\bm{\Theta}_{fi})
\end{equation}
where $\bm{r}(\cdot)$ and $\bm{\Theta}_{fi}$ are obtained similar to the way $\bm{r}(\cdot)$ and $\bm{\Theta}_i$ are constructed in \eqref{equ:ch8:model}.
In order to align the vehicles' orientation with $\bm{\tilde{f}}_i$, we define the following orientation error vector:
\begin{equation}
	\bm{e}_{\Theta,i} = \bm{\Theta}_i - \bm{\Theta}_{fi}
\end{equation}
Now, one can define states vectors to express the error dynamics of the multi-vehicle system as follows:
\begin{equation}
\begin{aligned}
\bm{E}_i &= [U_i,\ \bm{e}_{\Theta,i}^T,\ \dot{\bm{\bm{e}}}_{\Theta,i}^T,\ v_i]^T \in \R^{2m} \\
\bm{E} &= [\bm{E}_1^T,\ \cdots,\ \bm{E}_n^T]^T \in \R^{2mn} \\
\end{aligned}
\end{equation}

\begin{proposition}\label{prop:ch8:convergence}
	When the system trajectories converge to $\bm{\tilde{f}}_i = \bm{0}$ such that $\bm{\tilde{f}}_i$ is defined as in \eqref{equ:ch8:nsbEq}, it is guaranteed that  ${\bm{f}_{i,1},\bm{f}_{i,2},\bm{f}_{i,3}} \to \bm{0}$ escaping the case where $\bm{f}_{i,1} + \bm{f}_{i,2} + \bm{f}_{i,3} = 0$ while ${\bm{f}_{i,1},\bm{f}_{i,2},\bm{f}_{i,3}} \neq \bm{0}$.
	Furthermore, the potential functions $U_i^{OA}$, $U_{i,\alpha}$ and $U_{i,G}$ will also converge to zero reaching their local/global minimums considering any ordered sequence such as \eqref{equ:ch8:ordered_seq}.
\end{proposition}

\begin{proof}
	According to the definition in \eqref{equ:ch8:nsbEq}, the vectors $\bm{N}_{i,1} \bm{f}_{i,2}$ and $\bm{N}_{i,2} \bm{f}_{i,3}$ are in the null-space of $\bm{f}_{i,1}$.
	Therefore, the sum of these two vectors cannot cancel out $\bm{f}_{i,1}$.
	In a similar manner, $\bm{N}_{i,2} \bm{f}_{i,3}$ lies in the null-space of $\bm{N}_{i,1} \bm{f}_{i,2}$, and they cannot cancel out each others.  
	This means that as $\bm{\tilde{f}}_i \to \bm{0}$, the following is true:
	\begin{equation}\label{equ:ch8:prop_proof1}
		\bm{f}_{i,1},\bm{N}_{i,1} \bm{f}_{i,2},\bm{N}_{i,2} \bm{f}_{i,3} \to \bm{0}
	\end{equation}
	Furthermore, \eqref{equ:ch8:nsbN1} indicates that $\bm{N}_{i,1} \to \bm{I}_m$ as $\bm{f}_{i,1} \to \bm{0}$.
	This clearly implies that $\bm{f}_{i,2}\to \bm{0}$ based on \eqref{equ:ch8:prop_proof1}.
	Similarly, $\bm{N}_{i,2} \to \bm{I}_m$ according to \eqref{equ:ch8:nsbN2} which indicates that $\bm{f}_{i,3} \to \bm{0}$ as well.
	Thus, it is ensured that ${\bm{f}_{i,1},\bm{f}_{i,2},\bm{f}_{i,3}} \to \bm{0}$ when $\bm{\tilde{f}}_i$ converges to $\bm{0}$ where the case $\bm{f}_{i,1} + \bm{f}_{i,2} + \bm{f}_{i,3} = 0$ while ${\bm{f}_{i,1},\bm{f}_{i,2},\bm{f}_{i,3}} \neq \bm{0}$ is never reached.
	Moreover, considering any ordered sequence of vector force fields such as the one in \eqref{equ:ch8:ordered_seq} and the definitions in \eqref{equ:ch8:potential_function},\eqref{equ:ch8:sep_gradient}, \eqref{equ:ch8:goal_gradient} and \eqref{equ:ch8:obs_gradient}, it is evident that the corresponding potential energy functions will converge to their local/global minimums since $\nabla U_{i,\alpha},\ \nabla U_{i,G},\ \nabla U_{i,\obs} \to \bm{0}$.
\end{proof}

Now, we can present the proposed flocking control laws as follows:
\begin{equation}\label{equ:ch8:control}
\renewcommand*{\arraystretch}{1.5}
\left[\begin{array}{c}
a_i \\ \bm{\alpha}_i
\end{array}\right] = \left[\begin{array}{l}
\bm{\tilde{f}}_i^T(\bm{q}_{i})  \bm{r}_i(\bm{\Theta}_i)- k_{v,i} \text{sgn}(v_i) \\
\ddot{\bm{\Theta}}_{fi} - \bm{K}_{1,i} \tanh(\bm{\Theta}_i - \bm{\Theta}_{fi} ) - \bm{K}_{2,i} \tanh(\dot{\bm{\Theta}}_i  - \dot{\bm{\Theta}}_{fi} )
\end{array}\right]
\end{equation}
where $\bm{K}_{1,i}$ and $\bm{K}_{2,i}$ are positive definite diagonal matrices of appropriate sizes, and $k_{v,i}>0$ is a design parameter.
Furthermore, $\text{sgn}(\cdot)$ is the signum function which is defined according to:
\begin{equation}\label{equ:ch8:sign_func}
	\text{sgn}(\alpha) = \left\{\begin{array}{lr}
		-1, & \alpha < 0 \\
		0, & \alpha = 0 \\
		1, & \alpha > 0
	\end{array}\right.
\end{equation}
Also, the following condition must be satisfied:
\begin{equation}\label{equ:ch8:controlCondition}
	k_{v,i} > k_{i,G} + \sum_{j \in \mathpzc{N}_{i}} k_{ij}
\end{equation}
For a different ordered sequence of vector force fields than the one considered in \eqref{equ:ch8:ordered_seq}, the condition in \eqref{equ:ch8:controlCondition} can be written in a more general form as:
\begin{equation}\label{equ:ch8:controlCondition_general_form}
	k_{v,i} > ||\bm{\xi}||
\end{equation}

\begin{remark}\label{rem:ch8:sign_approx}
	The signum function defined in \eqref{equ:ch8:sign_func} can be approximated in practice using some smooth saturation function (ex. the hyperbolic tangent function) to avoid the well-known chattering effect.
\end{remark}

The main results of this section are presented next.
\begin{theorem}\label{thm:ch8:main_result}
	Consider a networked multi-vehicle system of size $n$ following the model described in \cref{sec:ch8:model}.
	Under the assumption that the network graph $\graph$ is connected and the network topology is fixed, the control law \eqref{equ:ch8:control} solves the flocking control problem defined in \cref{prob:ch8:flocking_control} given that the condition \eqref{equ:ch8:controlCondition} is satisfied.
	In other words, it is guaranteed to achieve the following ($\forall i,j \in \V,\ i \neq j$):
	\begin{itemize}
		\item[(i)] \textit{Convergence to an $\alpha$-lattice formation:} $U_{i,\alpha}$ reaches a minimum as $t \to \infty$
		\item[(ii)] \textit{Velocity consensus:} $\lim\limits_{t \to \infty} ||v_i \bm{r}_i - v_j \bm{r}_j|| = 0$
		\item[(iii)] \textit{Convergence to a goal region:} $\lim\limits_{t \to \infty} \bm{q}_{i} \in \goal$
		\item[(iv)] \textit{Obstacle avoidance:} obstacles are avoided for all $t\geq 0$ (i.e. $||\bm{q}_{i} - \bm{q}^*|| > d_s$)
	\end{itemize}
\end{theorem}

\begin{proof}
	Define a Lyapunov function $V(\bm{E}): \R^{2mn} \to \R$ as follows:
	\begin{equation}\label{equ:ch8:lyap}
		\begin{aligned}
		V(\bm{E}) =& \sum_{i=1}^{n} U_i + \sum_{i=1}^{n} \bm{1}_{m-1}^T\bm{K}_{1,i}\ln\Big(\cosh(\bm{e}_{\Theta,i}) \Big) \\ 
		&+ \frac{1}{2} \sum_{i=1}^{n}\dot{\bm{e}}_{\Theta,i}^T\dot{\bm{e}}_{\Theta,i} + \frac{1}{2} \sum_{i=1}^{n} v_i^2
		\end{aligned}
	\end{equation}
	The above definition clearly satisfies $V(\bm{E})>0\ \forall \bm{E} \in \R^{2mn} \textbackslash \{\bm{0}\}$ and $V(\bm{0}) = 0$.
	
	Using \eqref{equ:ch8:model} and \eqref{equ:ch8:fi_total}, the time derivative of \eqref{equ:ch8:lyap} can be obtained as follows:
	\begin{align}
		\nonumber \dot{V}(\bm{E}) =&  \sum_{i=1}^{n} \left[(\nabla_{\bm{q}_i} U_i)^T \dot{\bm{q}}_i + \bm{1}_{m-1}^T\bm{K}_{1,i} \text{Tanh}\Big(\bm{e}_{\Theta,i}\Big) \dot{\bm{e}}_{\Theta,i} \right.\\
		\nonumber &\left.+ \dot{\bm{e}}_{\Theta,i}^T\ddot{\bm{e}}_{\Theta,i} + v_i \dot{v}_i \right] \\
		=& \sum_{i=1}^{n} \left[-v_i \bm{f}_i^T \bm{r}_i + \dot{\bm{e}}_{\Theta,i}^T\bm{K}_{1,i}\tanh(\bm{e}_{\Theta,i}) + \dot{\bm{e}}_{\Theta,i}^T (\bm{\alpha}_i - \ddot{\bm{\Theta}}_{fi})+ v_i a_i \right] \label{equ:ch8:Wdot}
	\end{align}
	where the notation and properties defined in \cref{sec:ch8:notation} for $\tanh(\cdot)$ and $\text{Tanh}(\cdot)$ are used, and $\bm{r}_i = \bm{r}_i(\bm{\Theta}_i)$ and $\bm{f}_i = \bm{f}_i(\bm{q}_{i})$ are considered for brevity.
	
	Substituting the control laws \eqref{equ:ch8:control} into \eqref{equ:ch8:Wdot} yields the following:
	\begin{equation}\label{equ:ch8:Wdot_nd}
		\dot{V}(\bm{E}) = \sum_{i=1}^{n} \left[ v_i \bm{\xi}_i^T \bm{r}_i - k_{v,i} v_i \text{sgn}(v_i) - \dot{\bm{e}}_{\Theta,i}^T\bm{K}_{2,i} \tanh(\dot{\bm{e}}_{\Theta,i})  \right]
	\end{equation}
	Moreover, \eqref{equ:ch8:Wdot_nd} can be simplified to reach the following:
	\begin{align}
		\dot{V}(\bm{E}) &= \sum_{i=1}^{n} \left[ v_i \bm{\xi}_i^T \bm{r}_i - k_{v,i} |v_i| - \dot{\bm{e}}_{\Theta,i}^T\bm{K}_{2,i} \tanh(\dot{\bm{e}}_{\Theta,i})  \right] \label{equ:ch8:Wdot_nd1} \\
		&\leq \sum_{i=1}^{n} \left[ \Big(||\bm{\xi}_i|| - k_{v,i}\Big) |v_i| - \dot{\bm{e}}_{\Theta,i}^T\bm{K}_{2,i} \tanh(\dot{\bm{e}}_{\Theta,i})  \right] \label{equ:ch8:Wdot_nd2}
	\end{align}
	According to \eqref{equ:ch8:xiBound}, it is clear that $||\bm{\xi}_i||\leq k_{i,G} + \sum_{j \in \mathpzc{N}_{i}} k_{ij}$.
	Thus, under the condition \eqref{equ:ch8:controlCondition}, it is guaranteed that  $\dot{V}(\bm{E}) \leq 0 \ \forall \bm{E} \in \R^{2mn} \textbackslash \{\bm{0}\} \to V(\bm{E}(t)) \leq V(\bm{E}(0))$ since $\bm{K}_{2,i}$ is positive definite and $\tanh(\cdot)$ is an odd function.
	Therefore, there exists a compact set $\mathcal{W}\subset \R^{2mn}$ such that the system trajectory $\bm{E}(t)$ remains in it starting from any initial condition $\bm{E}(0) \in \mathcal{W}$ (i.e. $\mathcal{W}$ is invariant).
	Now, consider the set $\Omega_1=\{\bm{E}\in \mathcal{W}: \dot{V}(\bm{E}) = 0\}$.
	From \eqref{equ:ch8:Wdot_nd1}, the following is implied:
	\begin{equation}\label{equ:ch8:zero_vel}
		\dot{V}(\bm{E}) = 0 \to v_i = 0\ \ and \ \ \dot{\bm{e}}_{\Theta,i} = \bm{0}
	\end{equation}
	since $v_i\bm{\xi}_i^T \bm{r}_i < k_{v,i} |v_i|$.
	Moreover, $\dot{v}_i = a_i = 0$ and $\ddot{\bm{e}}_{\Theta,i} = (\bm{\alpha}_i - \ddot{\bm{\Theta}}_{fi}) = \bm{0}$ are also implied.
	As a result, $\bm{\tilde{f}}_i^T \bm{r}_i(\bm{\Theta}_i) = 0$ and $\tanh(\bm{e}_{\bm{\Theta},i}) = \bm{0}$ based on \eqref{equ:ch8:control}.
	This indicates that $\bm{e}_{\bm{\Theta},i} \to 0$ (i.e. $\bm{\Theta}_i \to \bm{\Theta}_{fi}$ and $\bm{r}(\bm{\Theta}_i) \to \bm{r}(\bm{\Theta}_{fi})$).
	Also, one can find out the following using \eqref{equ:ch8:f_i_new}:
	\begin{equation}
		\bm{\tilde{f}}_i^T \bm{r}_i(\bm{\Theta}_i) = ||\bm{\tilde{f}}_i|| \bm{r}^T(\bm{\Theta}_{fi})\bm{r}_i(\bm{\Theta}_i) \to 0,
	\end{equation}
	which ensures that $\bm{\tilde{f}}_i \to 0$ since $\bm{\Theta}_i \to \bm{\Theta}_{fi}$ and $||\bm{r(\cdot)}||=1$.
	Based on that, it is clear that the system trajectory will converge to the largest invariant set in $\Omega_1$ which only contains the origin.
	Hence, the control law \eqref{equ:ch8:control} guarantees the asymptotic convergence of $\bm{E}$ according to LaSalle's invariance principle.
	Moreover, using the result of \cref{prop:ch8:convergence}, it is guaranteed that $U_{i,\alpha}$, $U_{i,G}$ and $U_i^{OA}$ will reach their minimums which proves (i), (iii) and (iv).

	Additionally, the velocity consensus is achieved during motion once the formation of an $\alpha-$lattice is achieved.
	Consider the time interval $t\geq t_k$ where $t_k$ is the time when $U_{i,\alpha}$ reaches a minimum, and the formation is kept unchanged afterwards (i.e. $||\bm{q}_i - \bm{q}_j|| = d_{ij}$ for $t \geq t_k$).
	Based on that, the relative inter-agent dynamics after $t_k$ reduces to $\dot{\bm{q}}_i - \dot{\bm{q}}_j = \bm{0}$ showing that a velocity consensus is achieved during the motion.
	Also, \eqref{equ:ch8:zero_vel} indicates that all vehicles will reach zero velocity when the goal region is reached which proves (ii).
	This completes the proof.
\end{proof}

\begin{remark}
	The bounds of the control laws \eqref{equ:ch8:control} are as follows:
	\begin{align}
		|a_i| &\leq \sum_{j \in \mathpzc{N}_{i}}k_{ij} + k_{i,G} + k_{i,\obs} + k_{v,i} \\
		||\bm{\alpha}_i|| &\leq ||\ddot{\bm{\Theta}}_{fi}|| + \Big(\lambda_{max}(\bm{K}_{1,i}) + \lambda_{max}(\bm{K}_{2,i})\Big) \sqrt{m}
	\end{align}
This could make it easier to properly tune the control parameters while ensuring that any physical limits are satisfied.
\end{remark}

Further results are now presented following \cref{thm:ch8:main_result} to show that collision avoidance is also guaranteed using the proposed control laws.
Also, a proper choice of control parameters can ensure that the distance between vehicles can remain larger than some required safety margin.%
\begin{corollary}
	(Collision Avoidance) Consider a networked multi-vehicle system of size $n$ following the model described in \cref{sec:ch8:model}.
	The control laws \eqref{equ:ch8:control} can ensure that the vehicles' relative distances satisfy a safety margin $d_s < d_{ij}$ (i.e. $||\bm{q}_i - \bm{q}_j||\geq d_s,\ \ \forall i,j\in\mathcal{V},\ i \ne j,\ $) for all $t \geq 0$ for any solution starting in the set $\Omega_2 = \{\bm{E}\in\R^{2mn}: V(\bm{E}) \leq c \}$ under the condition
	\begin{equation}
		c  < \min\{k_{ij}\} \ln\left(\cosh(d_{ij} - d_{s})\right):= c^*
	\end{equation}
	 where $min\{\cdot\}$ is the minimum value.
\end{corollary}

\begin{proof}
	This can be proved by contradiction.
	Assume that there exist two vehicles whose relative distance is less than the safety margin at some time $t = t_1$.
	That is, $||\bm{q}_k(t_1) - \bm{q}_l(t_1)||< d_s$.
	Let's rewrite the potential energy function \eqref{equ:ch8:lyap} as follows:
	\begin{equation}\label{equ:ch8:WCor}
	 	V(\bm{E}) = \frac{1}{2}\sum_{i\in \mathcal{V}}\sum_{j \in \mathpzc{N}_{i}} k_{ij} \ln\left(\cosh\left(q_{ij}\right)\right) + V_1(\bm{E})
	\end{equation}
	The first term in the above equation corresponds to $U_{\alpha}$, and $V_1(\bm{E})>0$ corresponds to the remaining terms in \eqref{equ:ch8:lyap}. 
	At $t = t_1$, $U_{\alpha}$ is such that:
	\begin{equation}\label{equ:ch8:WCor2}
	\begin{aligned}
	U_{\alpha}(t_1) =& k_{kl} \ln\left(\cosh\Big(||\bm{q}_k(t_1) - \bm{q}_l(t_1)|| - d_{ij}\Big)\right) \\
	&+ \frac{1}{2}\sum_{i\in \mathcal{V}\backslash\{k,l\}}\sum_{j \in \mathpzc{N}_{i}\backslash\{k,l\}} k_{ij} \ln\left(\cosh\left(q_{ij}\right)\right) \\
	&\geq k_{kl} \ln\left(\cosh\Big(||\bm{q}_k(t_1) - \bm{q}_l(t_1)|| - d_{ij}\Big)\right)
	\end{aligned}
	\end{equation}
	Under the assumption that $||\bm{q}_k(t_1) - \bm{q}_l(t_1)||< d_s$, the following is implied:
	\begin{equation}\label{equ:ch8:UcaCond}
	U_{\alpha}(t_1) \geq k_{kl} \ln\left(\cosh(d_{ij} - d_{s})\right) \geq c^*
	\end{equation}
	where $d_{ij}>d_s$ is a necessary condition for any feasible formation.
	On the other hand, any solution starting in $\Omega_2$ should satisfy $U_{\alpha} \leq c < c*$ based on \eqref{equ:ch8:WCor} since $V_1(\bm{E})>0$.
	This clearly is in contradiction with \eqref{equ:ch8:UcaCond} which proves that the situation $||\bm{q}_k(t_1) - \bm{q}_l(t_1)||< d_s$ never occurs.
	Hence, collision avoidance is guaranteed as well as satisfying the safety margin $d_s$ such that $||\bm{q}_i - \bm{q}_j||\geq d_s,\ \ \forall i\in\mathcal{V},\ \forall j \in \mathpzc{N}_{i}$ for $t \geq 0$.
\end{proof}

\begin{remark}
	It is possible to extend the analysis to consider switching network topology.
	However, in this case, the error dynamics may exhibit discontinuities whenever $\mathpzc{N}_{i}(t)$ changes due to edges being added to or removed from $\E$.
	Thus, non-smooth analysis methods \cite{clarke1990optimization} and differential inclusions methods \cite{filippov2013differential} can be applied to analyze the stability under the application of the control laws \eqref{equ:ch8:control}.
\end{remark}

\begin{remark}\label{rem:ch8:local_min}
	In some cases, the potential function $U_{i,\alpha}$ may stuck in a local minimum corresponding to $\sum_{j \in \mathpzc{N}_{i}} k_{ij} \tanh(q_{ij})\ \bm{n}_{ij}=0$ where $q_{ij} \neq 0$. 
	Thus, a deviation from the $\alpha$-lattice formation may occur resulting in what can be called a quasi $\alpha$-lattice formation which is defined as follows \cite{olfati2006flocking}: 
	\begin{equation}
	-\epsilon \leq || \bm{q}_i - \bm{q}_j|| - d_{ij} \leq \epsilon,\ \ \forall j \in \mathpzc{N}_i 
	\end{equation}
	A possible solution to escape this local minimum is by further considering an associated modified vector field based on \eqref{equ:ch8:nsbEq} for the inter-agent forces.
\end{remark}

\section{Simulation Results}\label{sec:ch8:simulation}

Simulations have been carried to validate the performance of the proposed flocking control laws using MATLAB.
Different cases were considered for multi-vehicle systems of sizes $n=4$, $n=20$ and $n=100$ to show the scalability of the approach.
In all simulations, a three-dimensional workspace was considered, and the vehicles' kinematic model \eqref{equ:ch8:model} was used for $m=3$.
That is, the vehicle's position is represented using $\bm{q}_i = [x_i,\ y_i,\ z_i]^T$, and its orientation is expressed using two angles $\bm{\Theta}_i = [\theta_i,\ \psi_i]^T$ where $\theta_i$ and $\psi_i$ are the flight path and heading angles of the $i^{th}$ vehicle respectively.
The orientaion vector is constructed in terms of the orientation angles as follows:
\begin{equation}
	\bm{r}_i(\bm{\Theta}_i) = \left[\begin{array}{c}
	\cos\theta_i \cos \psi_i \\
	\cos\theta_i \sin \psi_i \\
	\sin \theta_i
	\end{array}\right]
\end{equation}
The control laws \eqref{equ:ch8:control} were applied in all cases, and the signum function $\text{sgn}(v_i)$ was approximated by the smooth hyperbolic tangent function $\tanh(v_i)$ as was suggested in \cref{rem:ch8:sign_approx}.

All vehicles were initially distributed at random in some initial region $\mathcal{I}$ which is slightly different for each case depending on the number of vehicles.
Similarly, different spherical goal regions $\goal$ was assigned to the group to reach as one of the mission objectives.
For simplicity, the desired inter-agent separation distance for all cases was chosen the same as $d_{ij}=5m$. 
In the first simulation case, an obstacle was placed between the vehicles' initial position and the goal region.
A vortex field around the obstacle was used as a way to compute $\bm{n}_{i\obs}$.
However, different advanced approaches can be applied to determine $\bm{n}_{i\obs}$ based on some of the available obstacle avoidance methods in the literature.

The first simulation case considers a system of 4 vehicles initially deployed at random within $\mathcal{I}=\{(x_i,y_i,z_i): x_i,y_i,z_i \in [0,20] \}$.
The goal region was also set such that $\bm{q}_G = [50,\ 50,\ 80]^T$ and $R_G=10$.
The control parameters for this case was selected as: $k_{ij}=0.6$, $k_{i,G}=0.5$, $k_{i,\obs}=1$, $k_{v,i}=2.5$, $\bm{K}_{1,i}= 0.25 \bm{I}_2$, and $\bm{K}_{2,i}=2 \bm{I}_2$.
Furthermore, the safety margin was chosen to be $d_{s}=1$, and the safety distance around the obstacle was chosen as $C = 5.5$ which was needed in \eqref{equ:regionFunction} to compute $\Gamma_{\obs}(q_{i \obs},d_{s})$.
Also, the communication range was chosen as $r_c = 20m$ which determines the neighbourhood of each vehicle according to \eqref{equ:ch8:comm_neighborhood}.
The obtained results for this case are presented in \cref{fig:ch8:simMotion,fig:ch8:simdObs,fig:ch8:simDij,fig:ch8:simLinVelAcc,fig:ch8:simAngVel,fig:ch8:simAngAcc}.
\Cref{fig:ch8:simMotion} shows the overall paths taken by the vehicles which clearly verifies that the vehicles can reach the goal region $\goal$ with obstacle avoidance capability (distances to obstacle are shown in \cref{fig:ch8:simdObs}).
The relative distances between the vehicles are shown in \cref{fig:ch8:simDij} which confirms that the collision avoidance objective is achieved (i.e. $||\bm{q}_i - \bm{q_j}|| \geq d_s\ \ \forall t$).
The vehicles also reach the desired formation after about $100s$ such that the desired separation distance $d_{ij}=5m$ is achieved (i.e. $||\bm{q}_i - \bm{q_j}|| \to d_{ij}\ \forall j \in \mathpzc{N}_{i}$).
After about $200s$ when the vehicles enters the critical region around the detected obstacle, the input component relative to obstacle avoidance becomes more important causing a deviation from the desired formation as long as nearby vehicles are still at a safe distance.
Once the obstacle is avoided, the groups goes back into the desired formation around $310s$.
Similar behaviour can be seen when the vehicles reach the goal region where vehicles start to slow down.
However, they quickly adjust their positions within the goal region to maintain the desired formation.
The linear/angular velocities and accelerations are given in \cref{fig:ch8:simLinVelAcc,fig:ch8:simAngVel,fig:ch8:simAngAcc}.
This shows that the vehicles come to a complete stop after reaching the goal region and forming an $\alpha-$lattice.
The velocity consensus can also be seen from \cref{fig:ch8:simMotion,fig:ch8:simLinVelAcc} every time an $\alpha-$lattice is formed except when the vehicles are avoiding the obstacle.

\begin{figure}[!htb]
	\centering
	\includegraphics[width=0.7\linewidth]{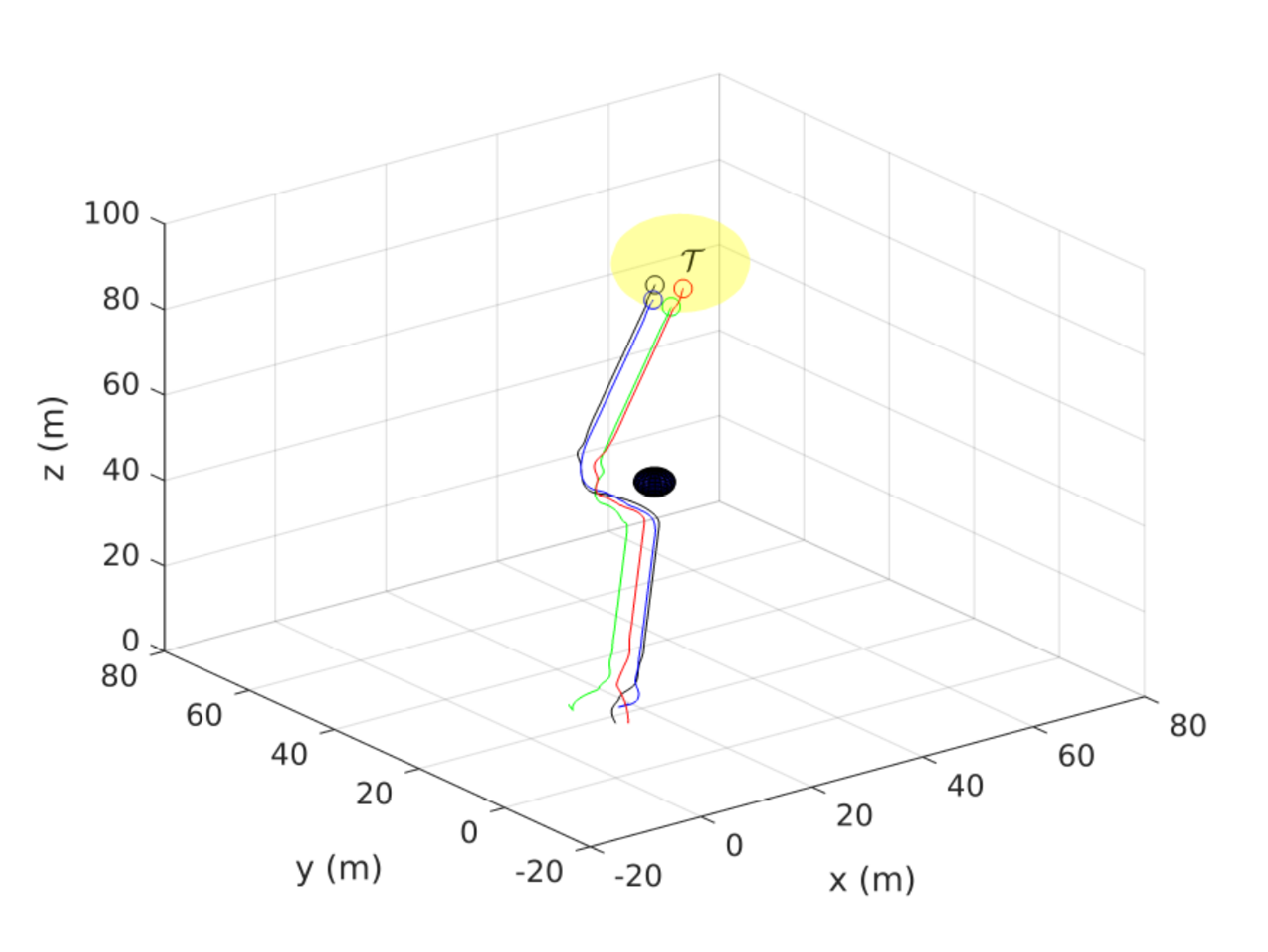} 
	\caption{Simulation case ($n = 4$): motion trajectories of vehicles} \label{fig:ch8:simMotion}
\end{figure}

\begin{figure}[!htb]
	\centering
	\includegraphics[width=0.7\linewidth]{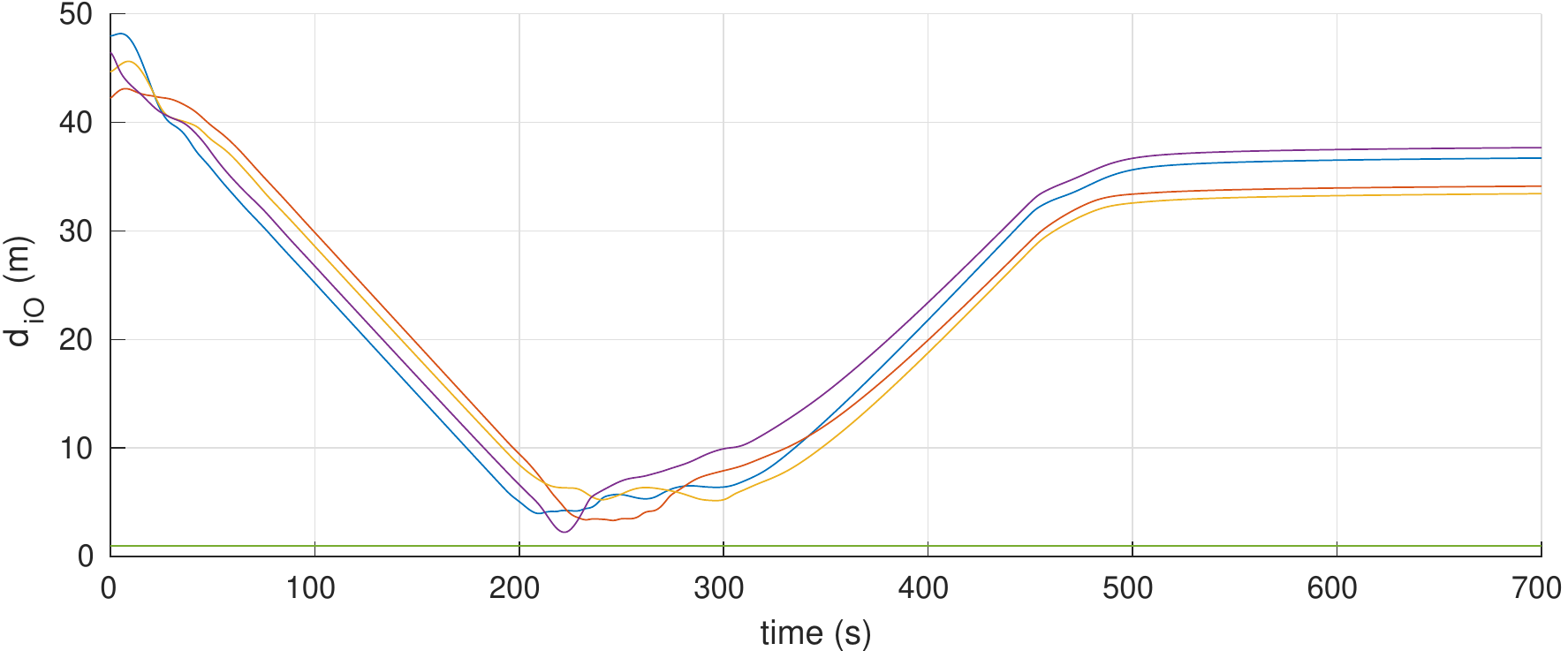} 
	\caption{Simulation case ($n = 4$): vehicles' distances to obstacle versus time} \label{fig:ch8:simdObs}
\end{figure}

\begin{figure}[!htb]
	\centering
	\includegraphics[width=0.7\linewidth]{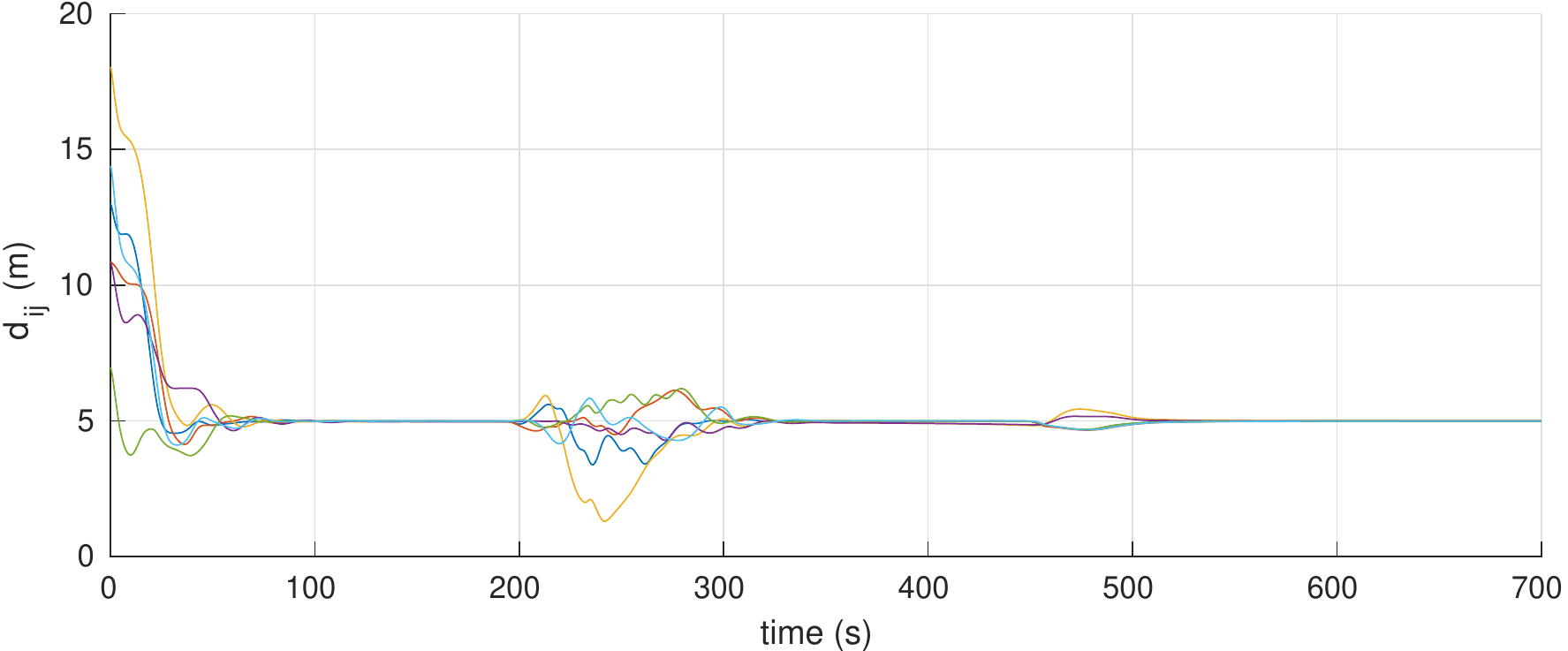} 
	\caption{Simulation case ($n = 4$): relative distances between vehicles versus time} \label{fig:ch8:simDij}
\end{figure}

\begin{figure}[!htb]
	\centering
	\includegraphics[width=0.7\linewidth]{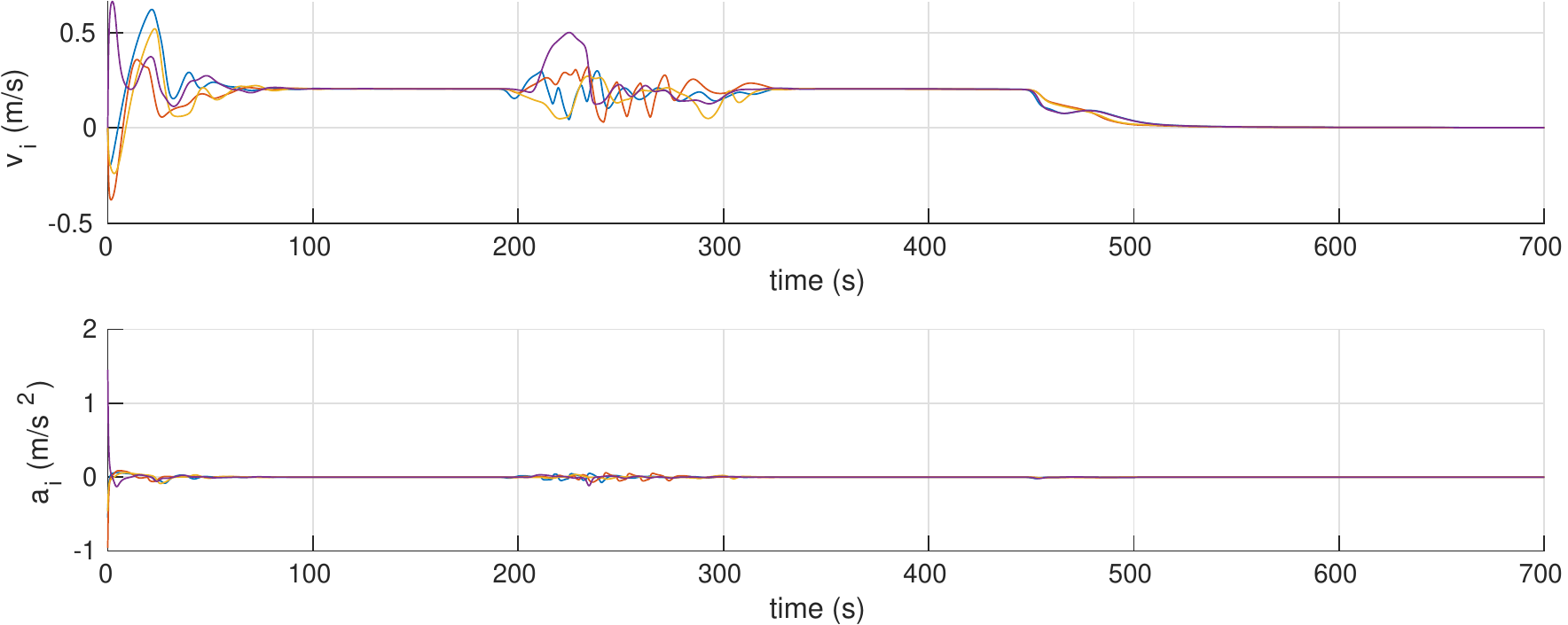} 
	\caption{Simulation case ($n = 4$): linear velocities $v_i$ and accelerations $a_i$ of vehicles versus time} \label{fig:ch8:simLinVelAcc}
\end{figure}

\begin{figure}[!htb]
	\centering
	\includegraphics[width=0.7\linewidth]{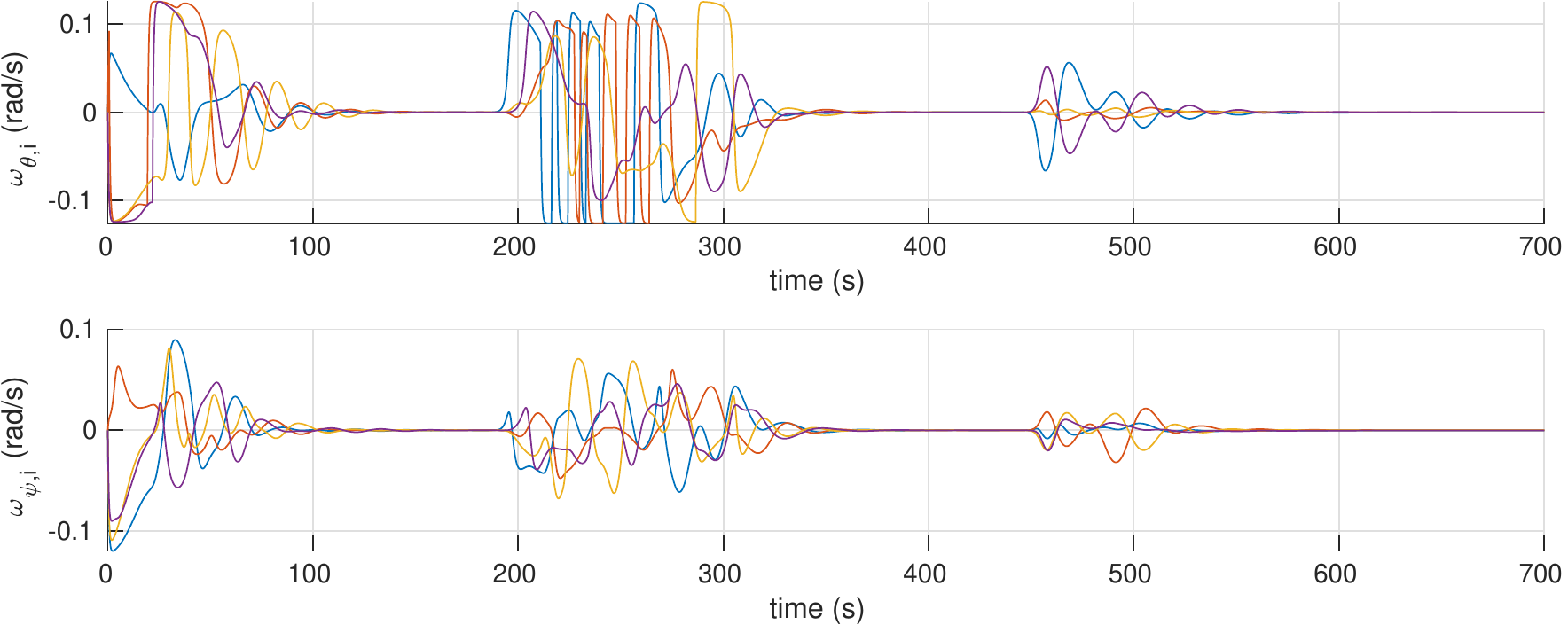} 
	\caption{Simulation case ($n = 4$): angular velocities $\Omega_{\theta,i}$ and $\Omega_{\psi,i}$ of vehicles versus time} \label{fig:ch8:simAngVel}
\end{figure}

\begin{figure}[!htb]
	\centering
	\includegraphics[width=0.7\linewidth]{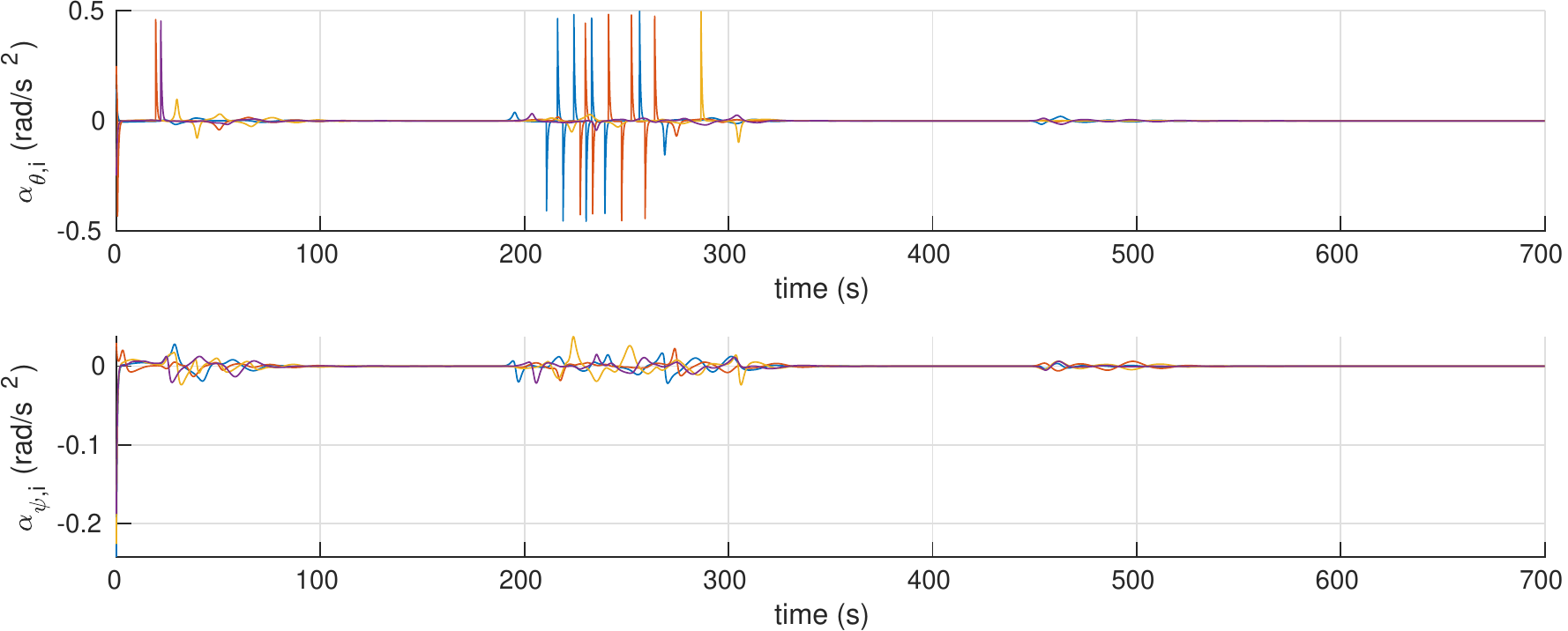} 
	\caption{Simulation case ($n = 4$): angular accelerations $\alpha_{\theta,i}$ and $\alpha_{\psi,i}$ of vehicles versus time} \label{fig:ch8:simAngAcc}
\end{figure}

The other two simulation cases were carried out using 20 and 100 vehicles respectively to show the scalability of the approach.
In these cases, no obstacles were considered; however, it would be handled the same way as it was done in the first case.
The control parameters for both cases were chosen as follows: $k_{ij}=1.0$, $k_{i,G}=0.75$, $k_{v,i}=2.0$, $\bm{K}_{1,i}= diag\{1, 2.5\}$, and $\bm{K}_{2,i}=2 \bm{I}_2$.
Also, the communications range was selected as $r_c = 12m$.
The suggestion made in \cref{rem:ch8:local_min} was used in these two simulation cases such that $\bm{f}_{i,\alpha}$ will have attractive/repulsive forces due to the nearest two vehicles.

The vehicles positions at different time instances as well as the overall executed trajectories are shown in \cref{fig:ch8:sim2MotionInstances,fig:ch8:sim2Motion,fig:ch8:sim3MotionInstances,fig:ch8:sim3Motion} for both cases.
The vehicles managed to successfully reach the goal region in both scenarios.
The relative distances are shown in \cref{fig:ch8:sim2Dij,fig:ch8:sim3Dij} which clearly confirms that the motion is collision-free.
Moreover, the vehicles keep a distance of $d_{ij}$ to its closest Neighbors forming a quasi $\alpha-$lattice. 
It was also observed that the adjacency matrix of the network graph maintained a full-rank during the motion which means that the proposed control laws managed to also persevere the connectivity of the multi-agent system.
The vehicles' linear velocities and acceleration are given in \cref{fig:ch8:sim2LinVelAcc,fig:ch8:sim3LinVelAcc} showing that velocity consensus is successfully achieved, and the vehicles come to a complete stop when they reach the goal region.

Overall, the provided results show how well our proposed distributed flocking control method work in achieving the considered control objectives.

\begin{figure}[!htb]
	\centering
	\begin{adjustbox}{minipage=\linewidth,scale=1.0}
		\begin{subfigure}[t]{0.48\textwidth}
			\centering
			\includegraphics[clip, width=\linewidth]{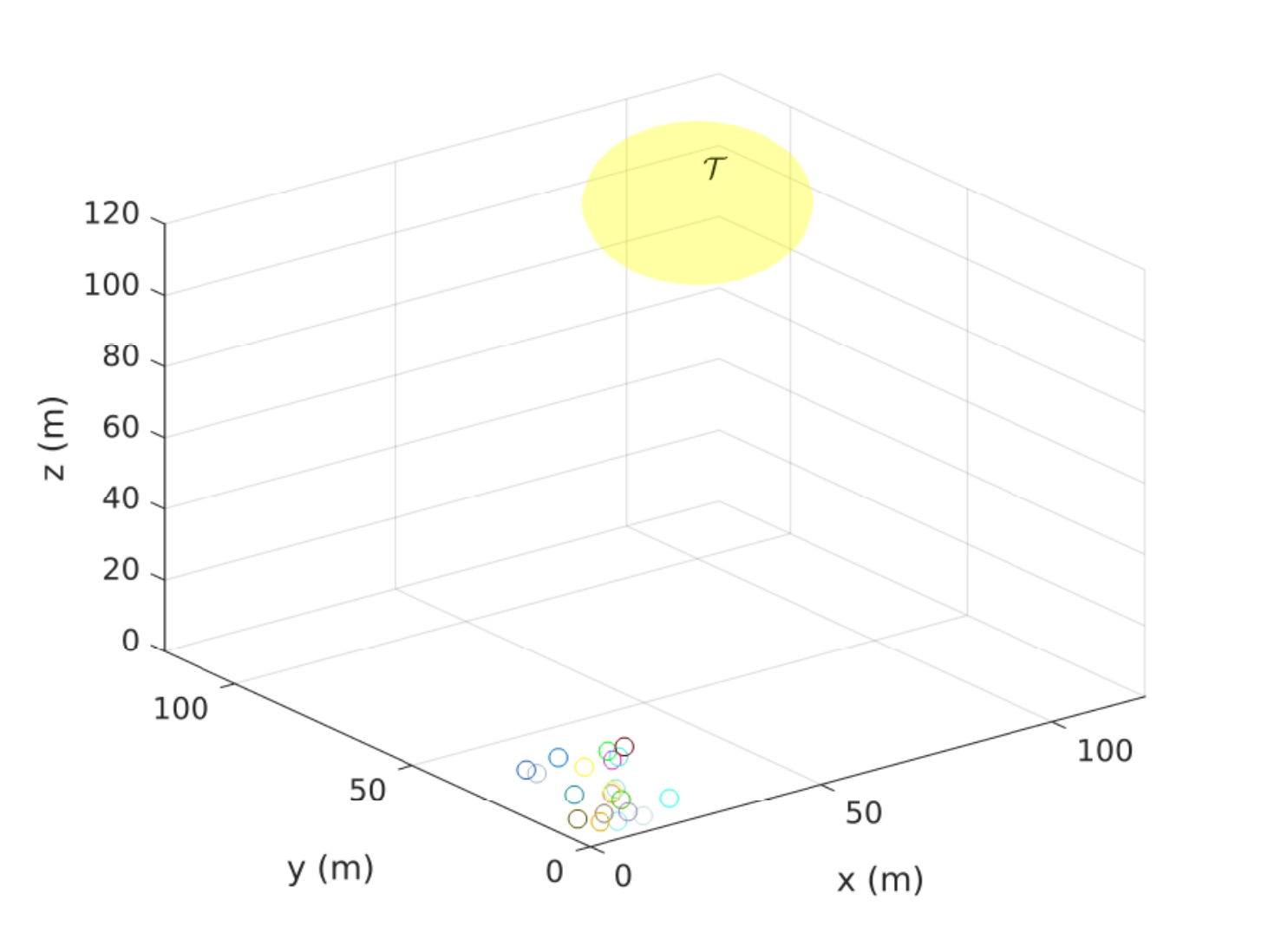} 
			\caption{}
		\end{subfigure}
		\hfill
		\begin{subfigure}[t]{0.48\textwidth}
			\centering
			\includegraphics[clip, width=\linewidth]{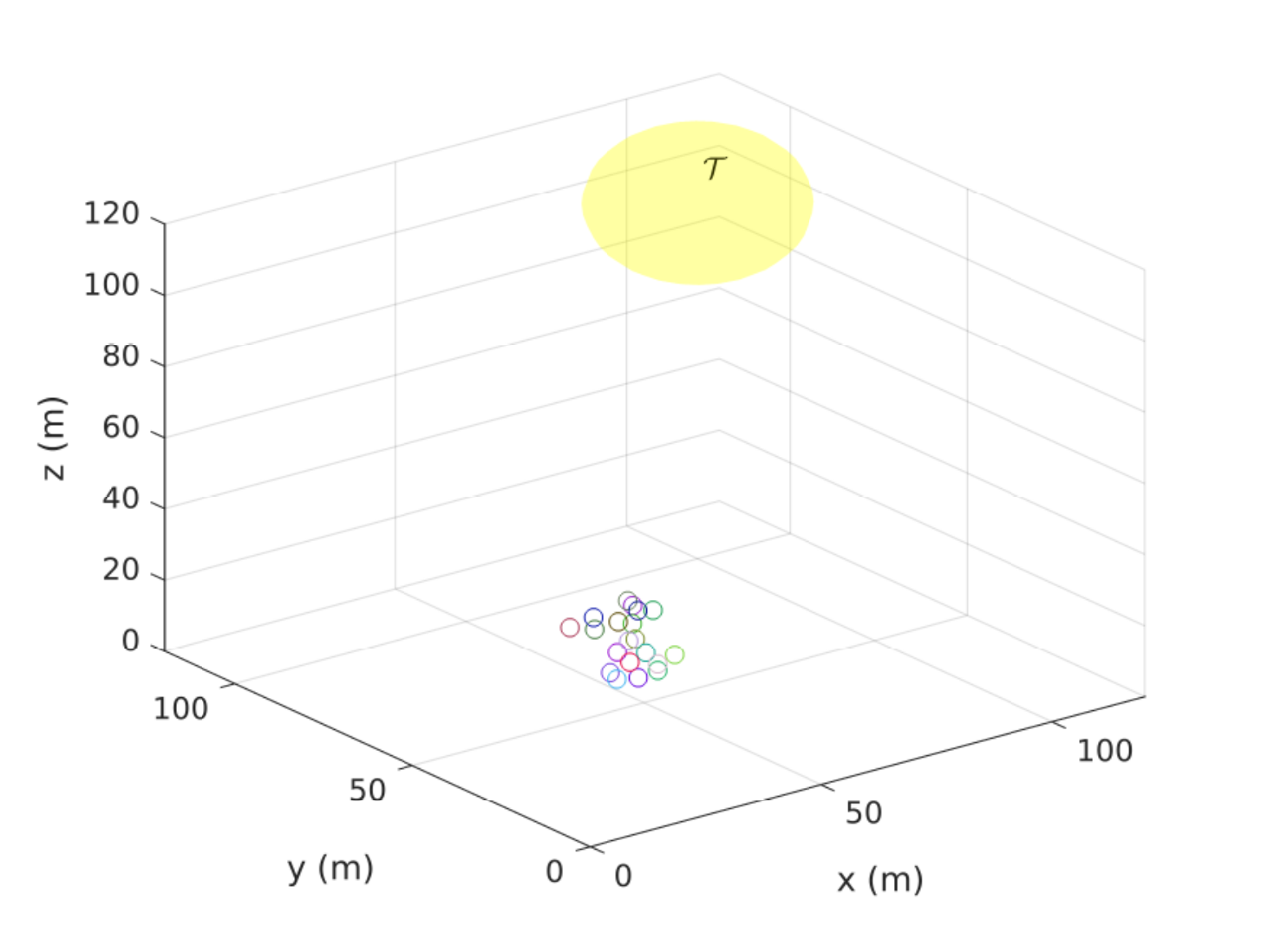}
			\caption{}
		\end{subfigure}
	
		\begin{subfigure}[t]{0.48\textwidth}
			\centering
			\includegraphics[clip, width=\linewidth]{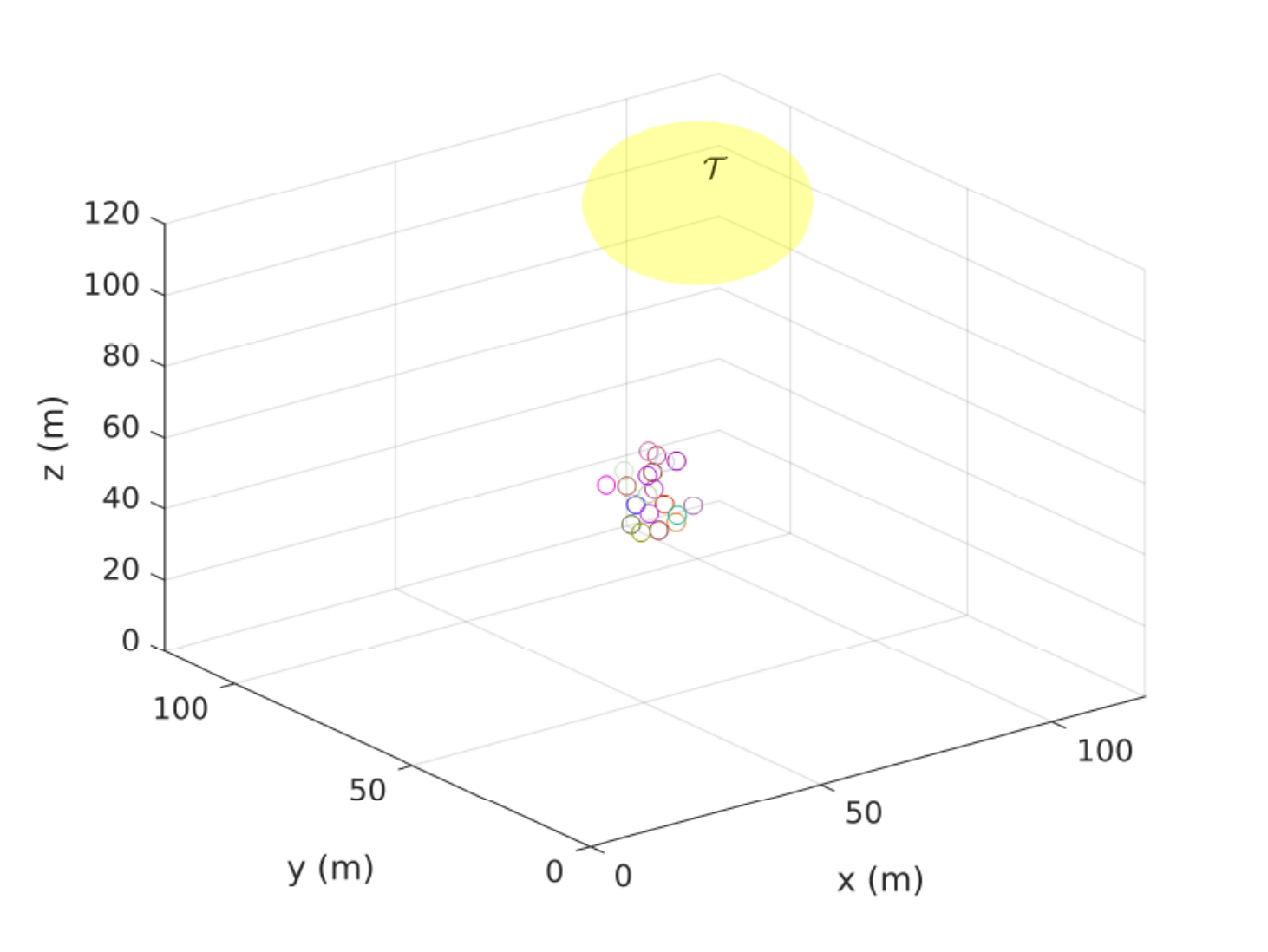} 
			\caption{}
		\end{subfigure}
		\hfill
		\begin{subfigure}[t]{0.48\textwidth}
			\centering
			\includegraphics[clip, width=\linewidth]{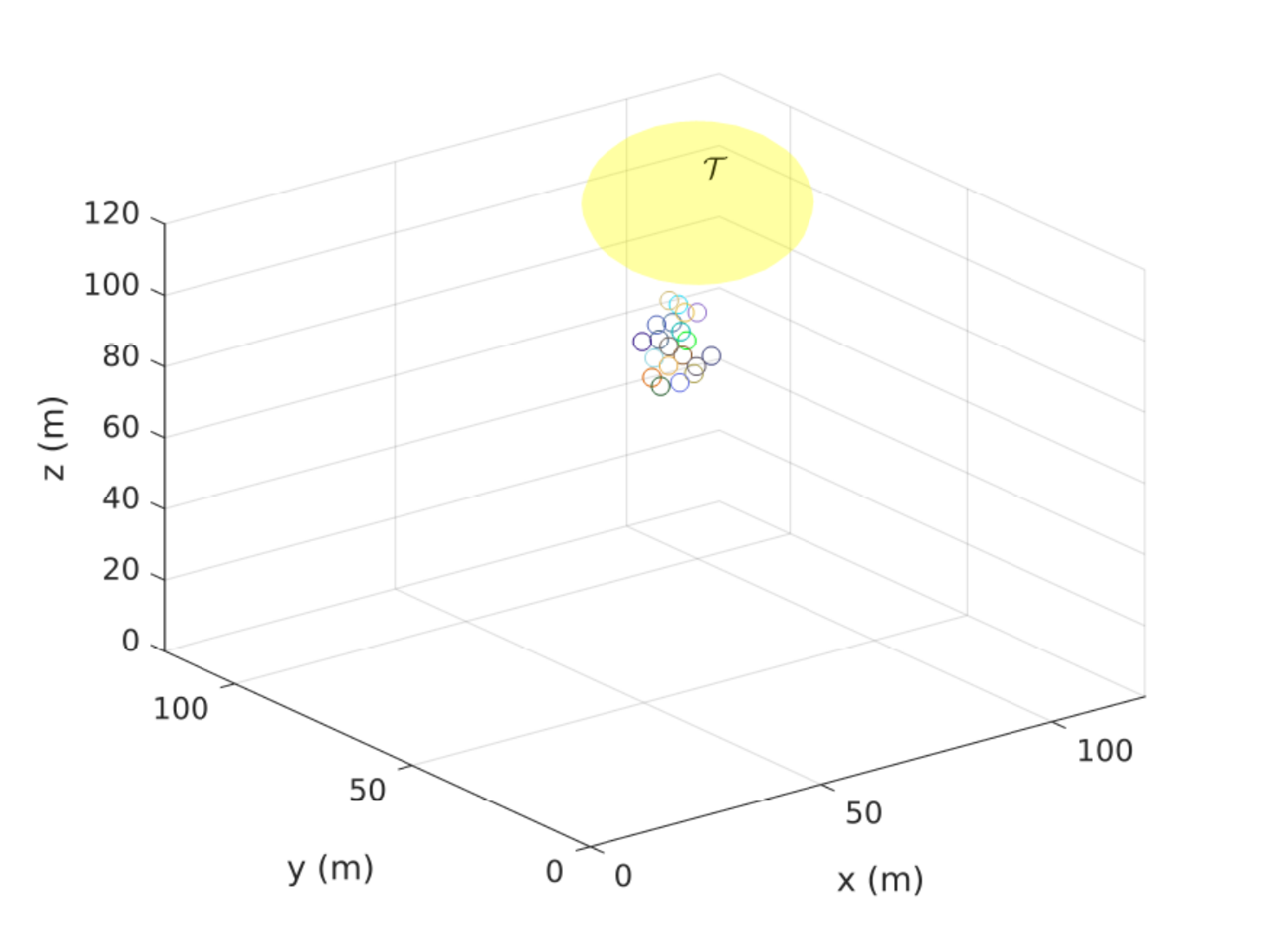}
			\caption{}
		\end{subfigure}
	
		\begin{subfigure}[t]{0.48\textwidth}
			\centering
			\includegraphics[clip, width=\linewidth]{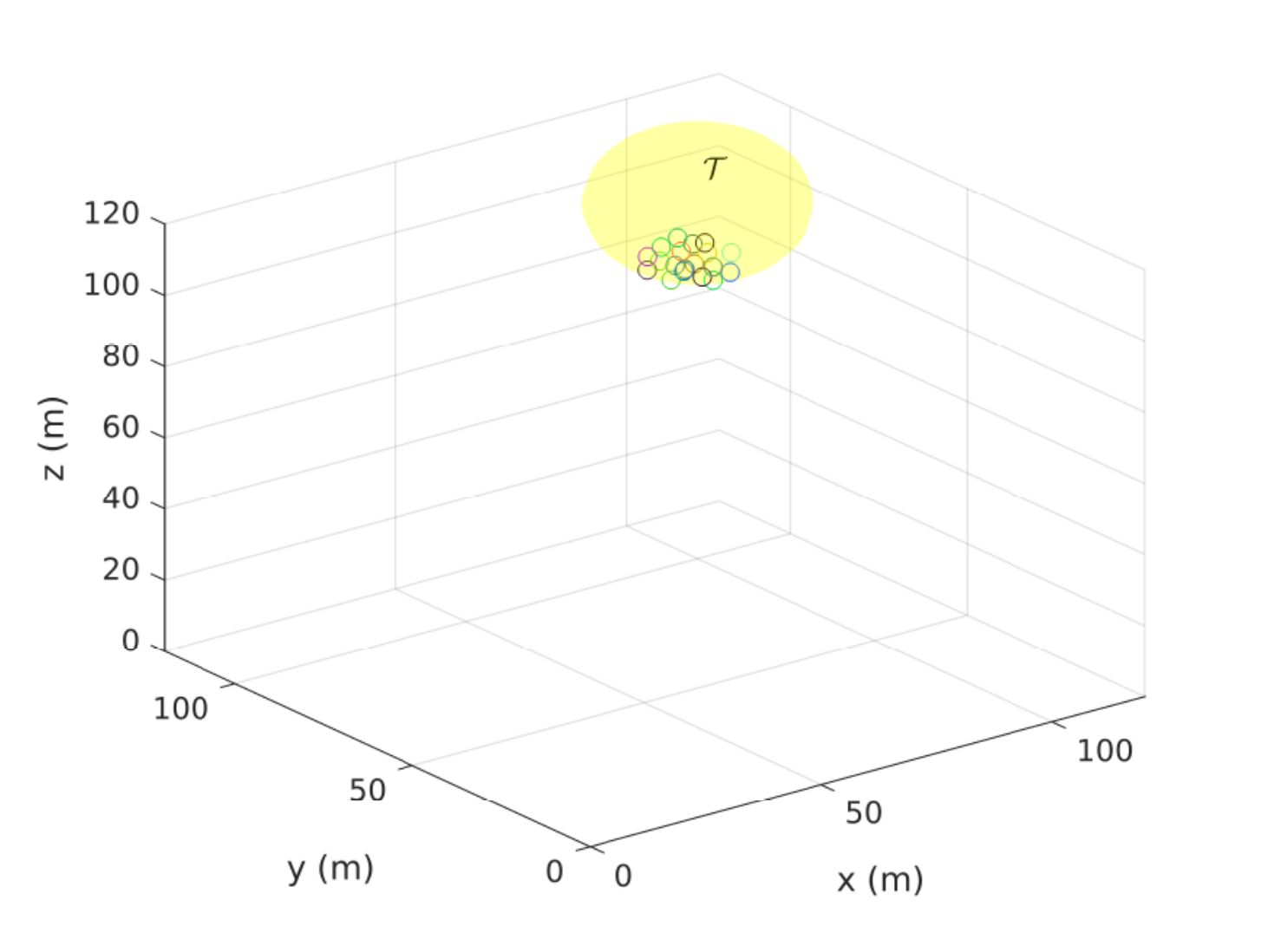} 
			\caption{}
		\end{subfigure}

		\caption{Simulation case ($n = 20$): vehicles position at different time instances during the motion}
		\label{fig:ch8:sim2MotionInstances}
	\end{adjustbox}
\end{figure}

\begin{figure}[!htb]
	\centering
	\includegraphics[width=0.7\linewidth]{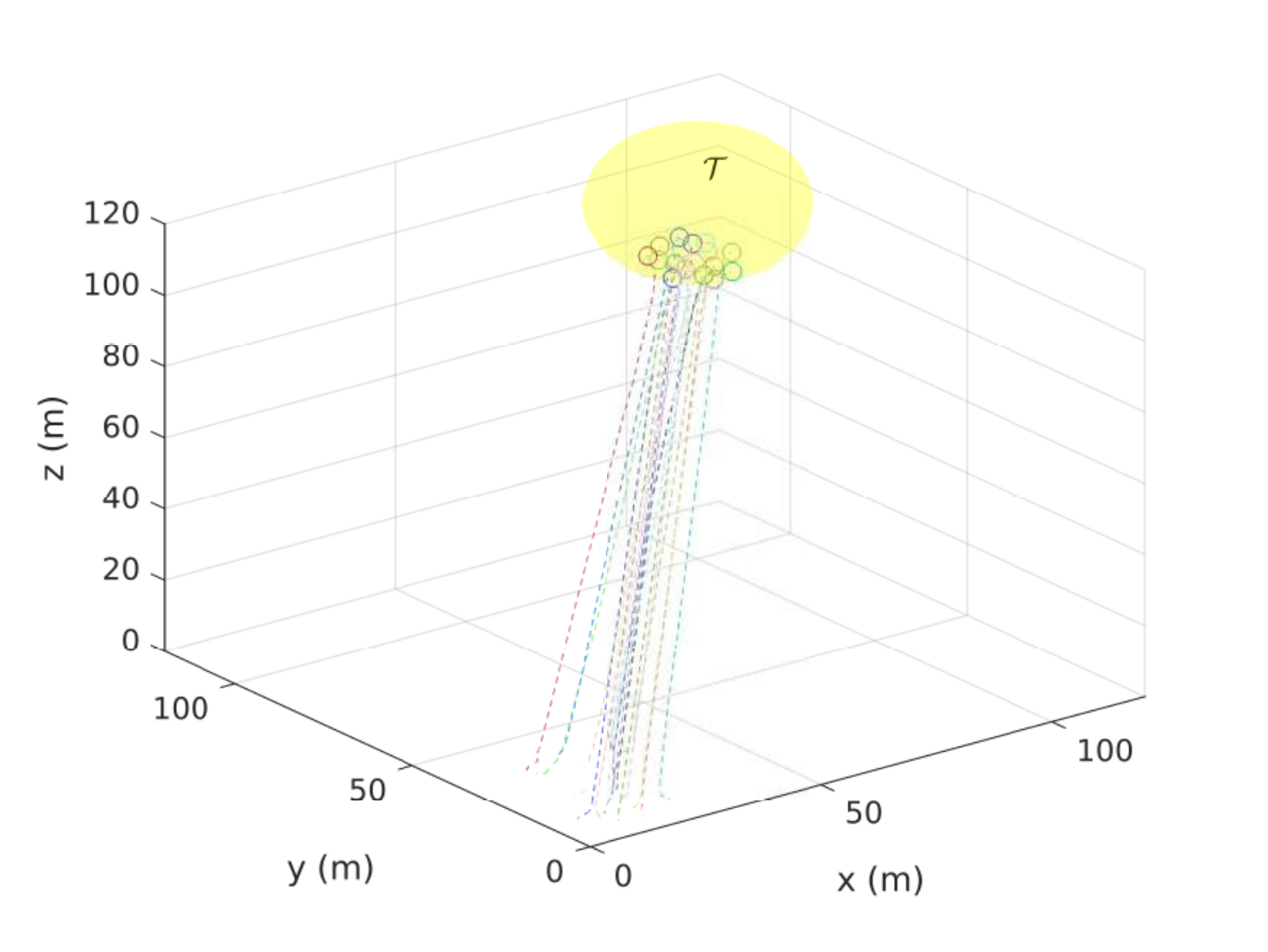} 
	\caption{Simulation case ($n = 20$): motion trajectories of vehicles} \label{fig:ch8:sim2Motion}
\end{figure}

\begin{figure}[!htb]
	\centering
	\includegraphics[width=0.7\linewidth]{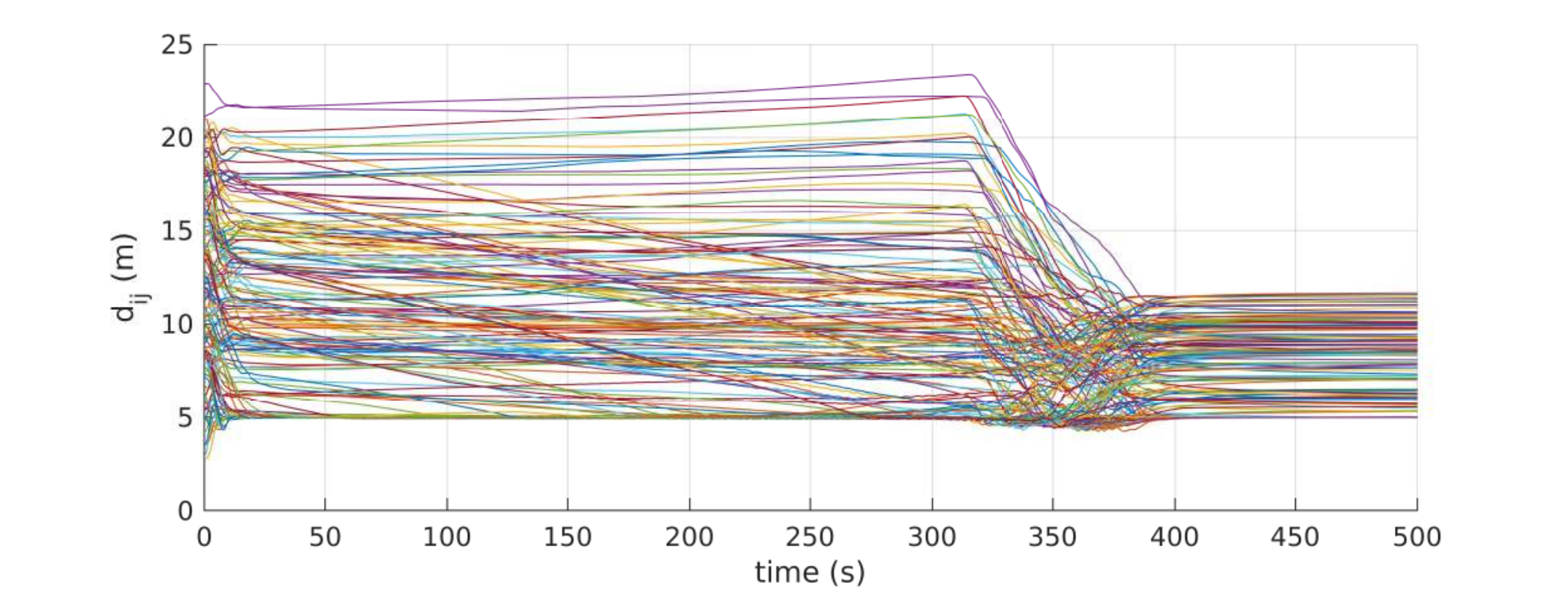} 
	\caption{Simulation case ($n = 20$): relative distances between vehicles versus time} \label{fig:ch8:sim2Dij}
\end{figure}

\begin{figure}[!htb]
	\centering
	\includegraphics[width=0.7\linewidth]{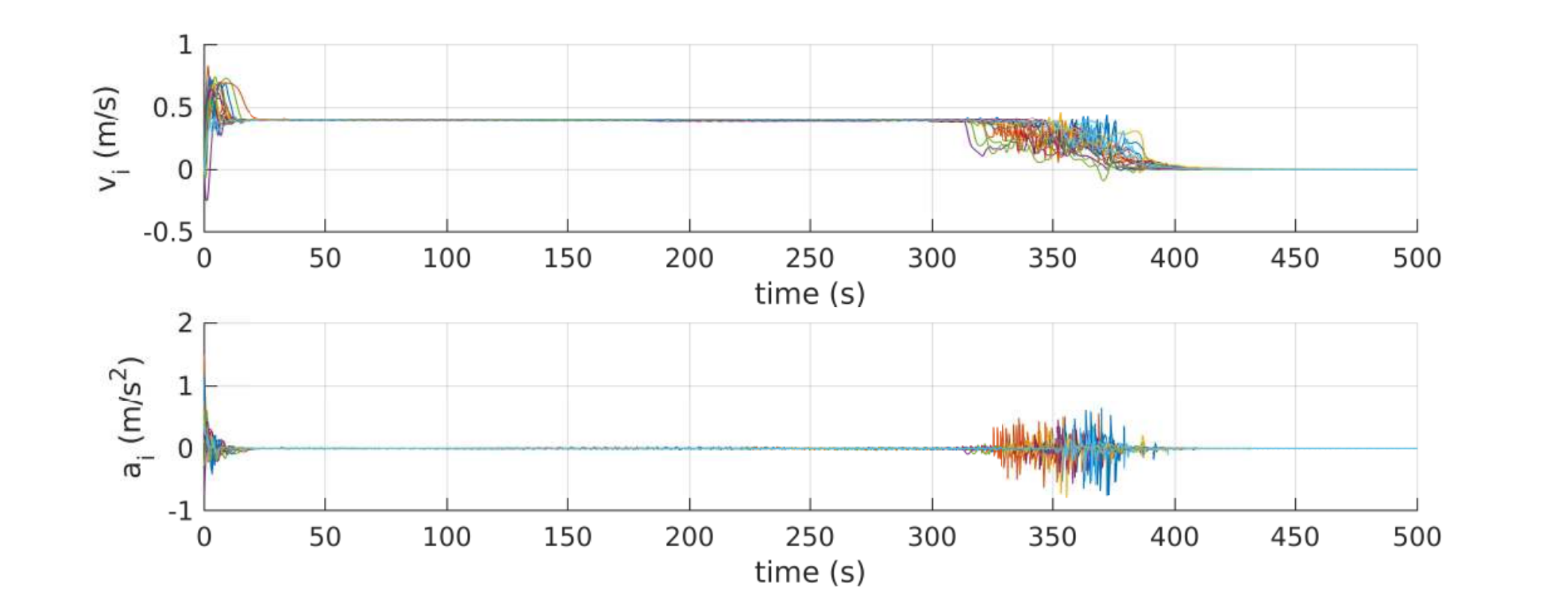} 
	\caption{Simulation case ($n = 20$): linear velocities $v_i$ and accelerations $a_i$ of vehicles versus time} \label{fig:ch8:sim2LinVelAcc}
\end{figure}

\begin{figure}[!htb]
	\centering
	\begin{adjustbox}{minipage=\linewidth,scale=1.0}
		\begin{subfigure}[t]{0.48\textwidth}
			\centering
			\includegraphics[clip, width=\linewidth]{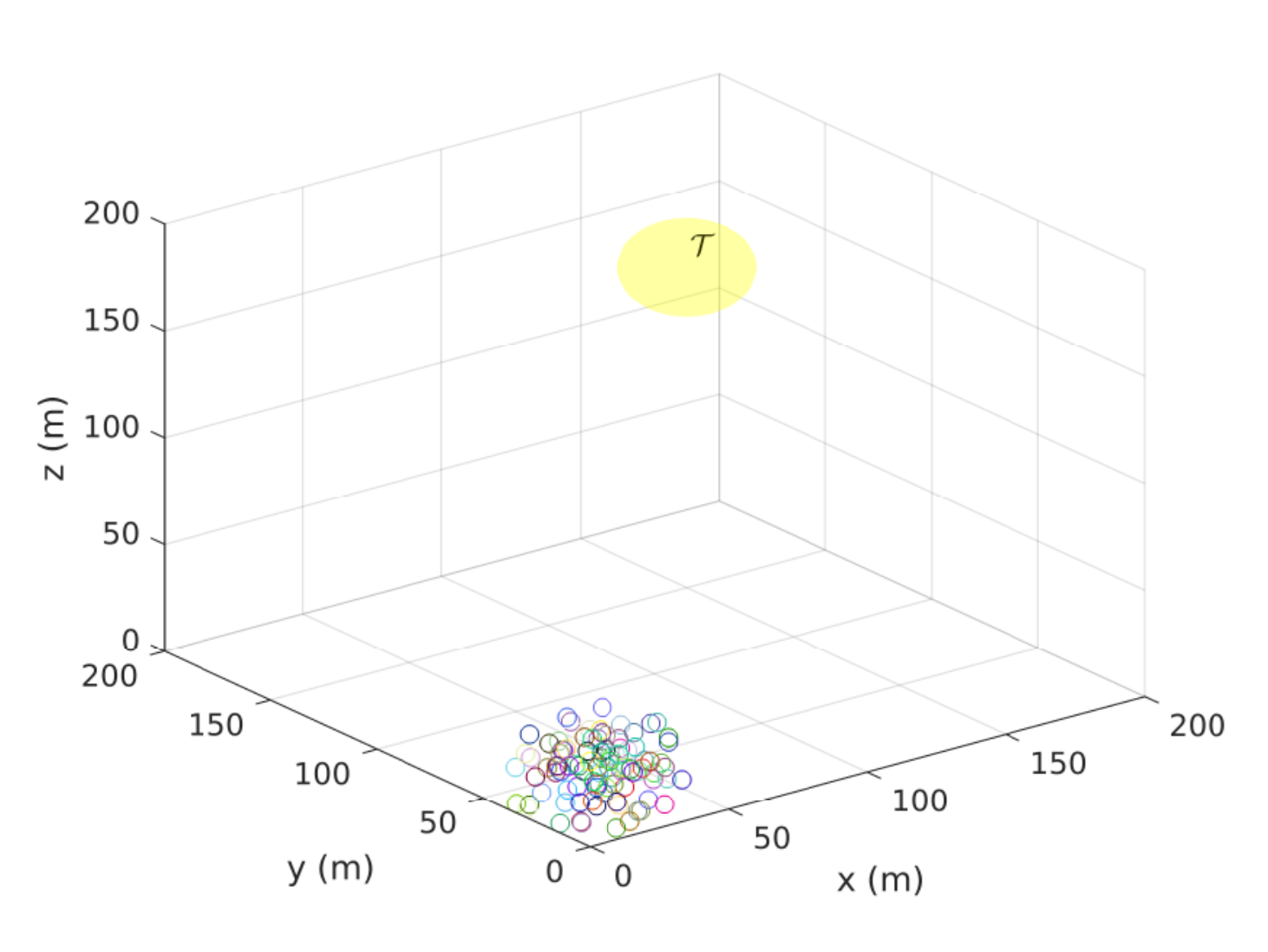} 
			\caption{}
		\end{subfigure}
		\hfill
		\begin{subfigure}[t]{0.48\textwidth}
			\centering
			\includegraphics[clip, width=\linewidth]{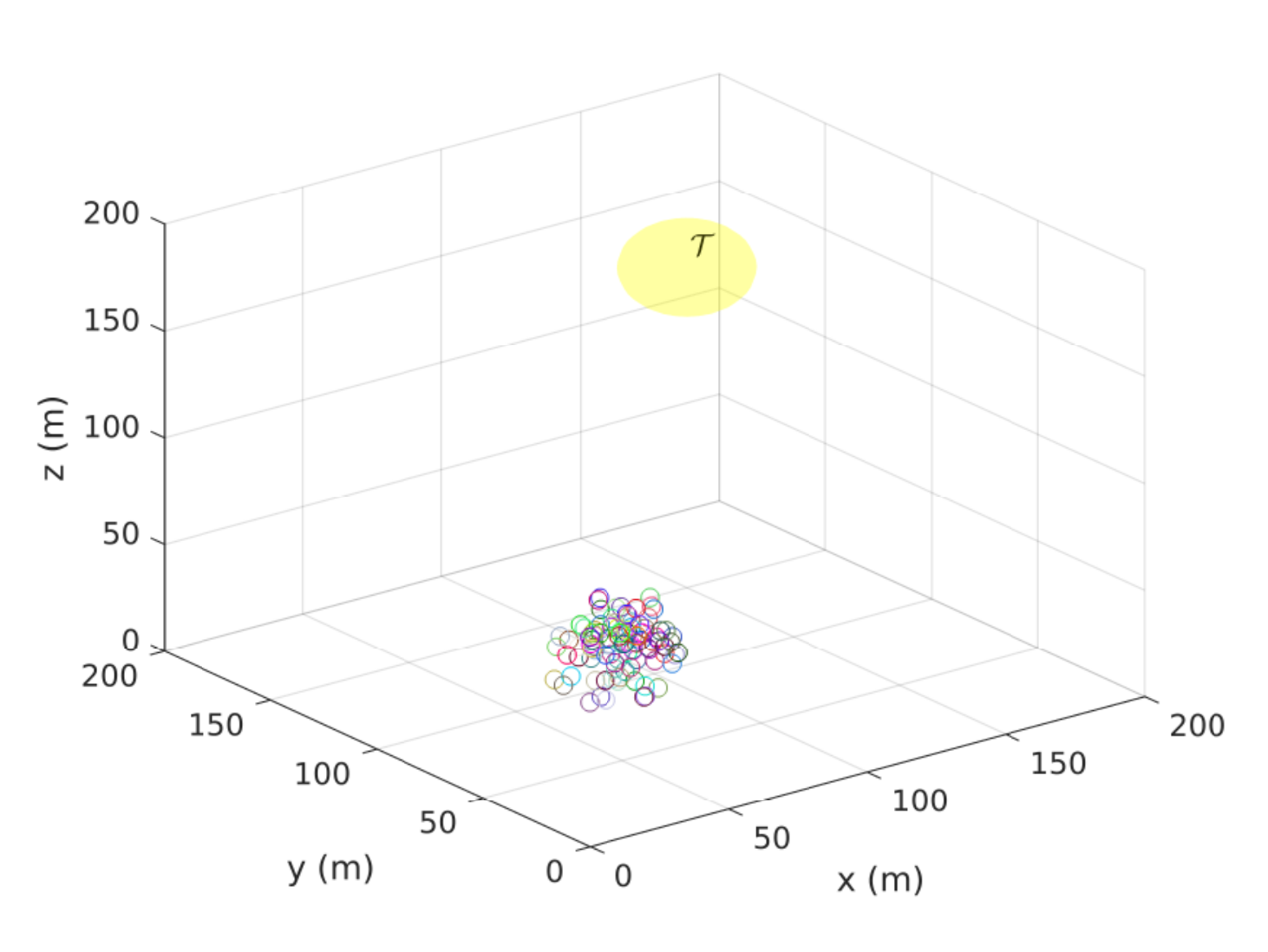}
			\caption{}
		\end{subfigure}
		
		\begin{subfigure}[t]{0.48\textwidth}
			\centering
			\includegraphics[clip, width=\linewidth]{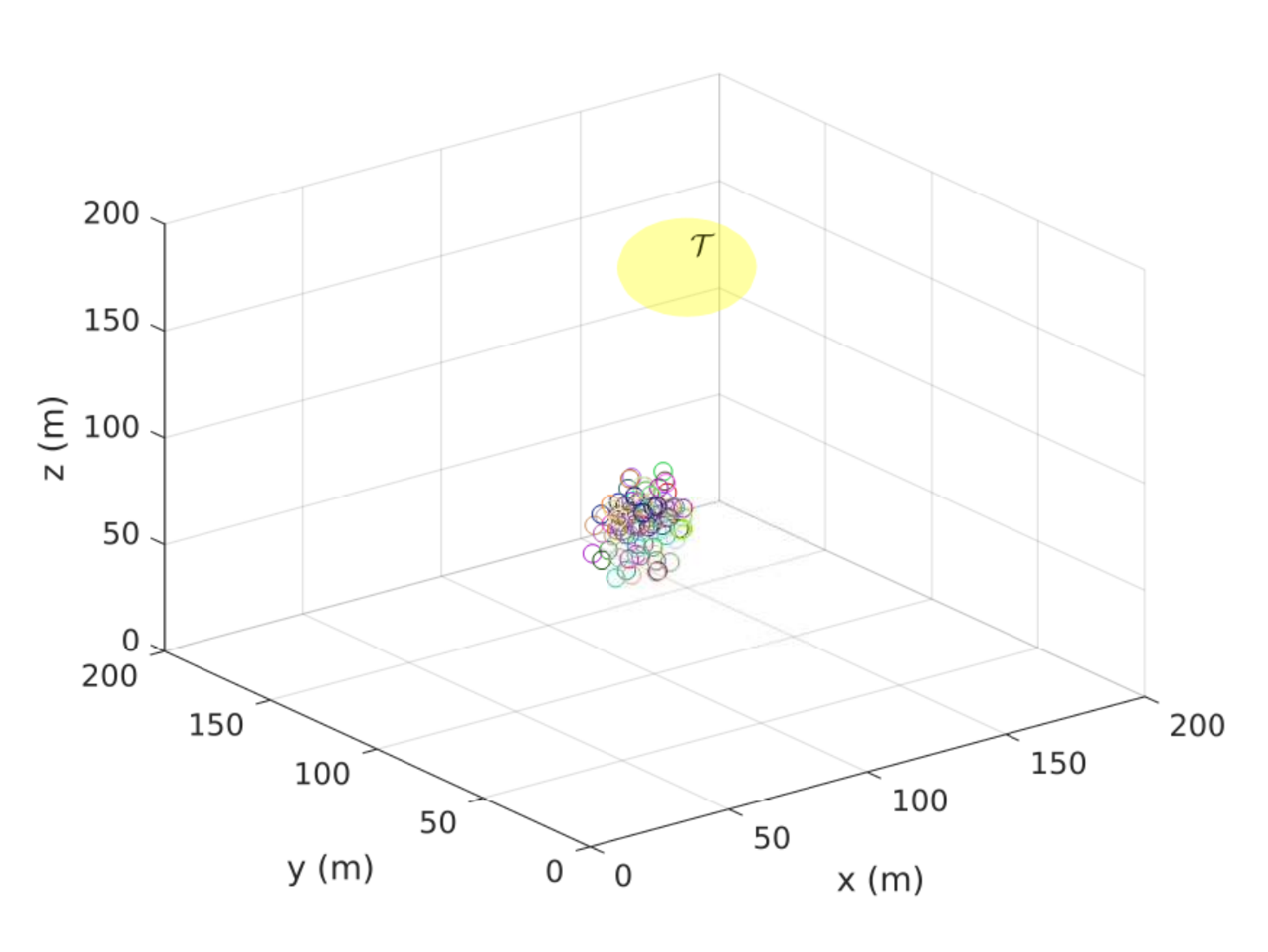} 
			\caption{}
		\end{subfigure}
		\hfill
		\begin{subfigure}[t]{0.48\textwidth}
			\centering
			\includegraphics[clip, width=\linewidth]{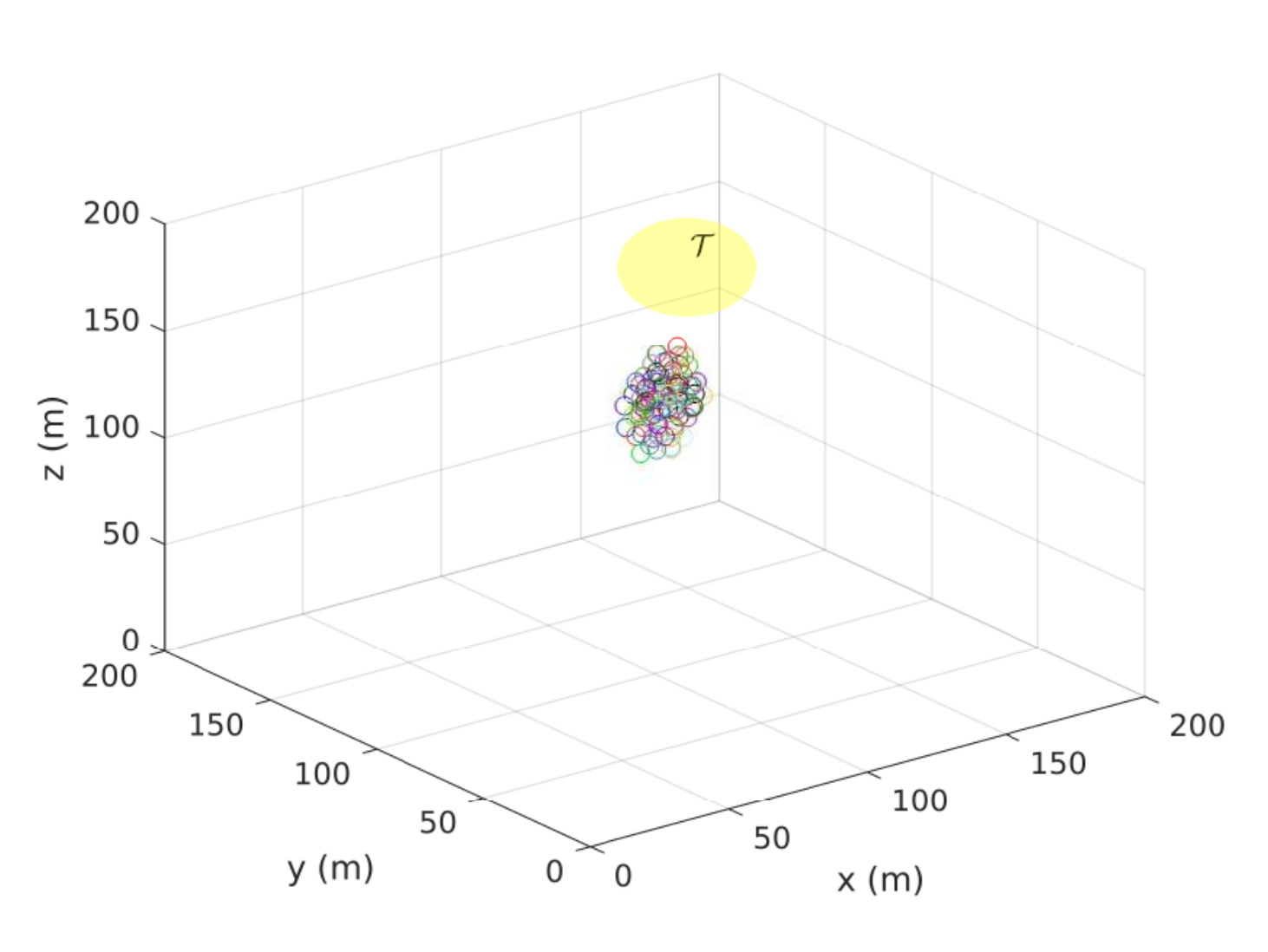}
			\caption{}
		\end{subfigure}
		
		\begin{subfigure}[t]{0.48\textwidth}
			\centering
			\includegraphics[clip, width=\linewidth]{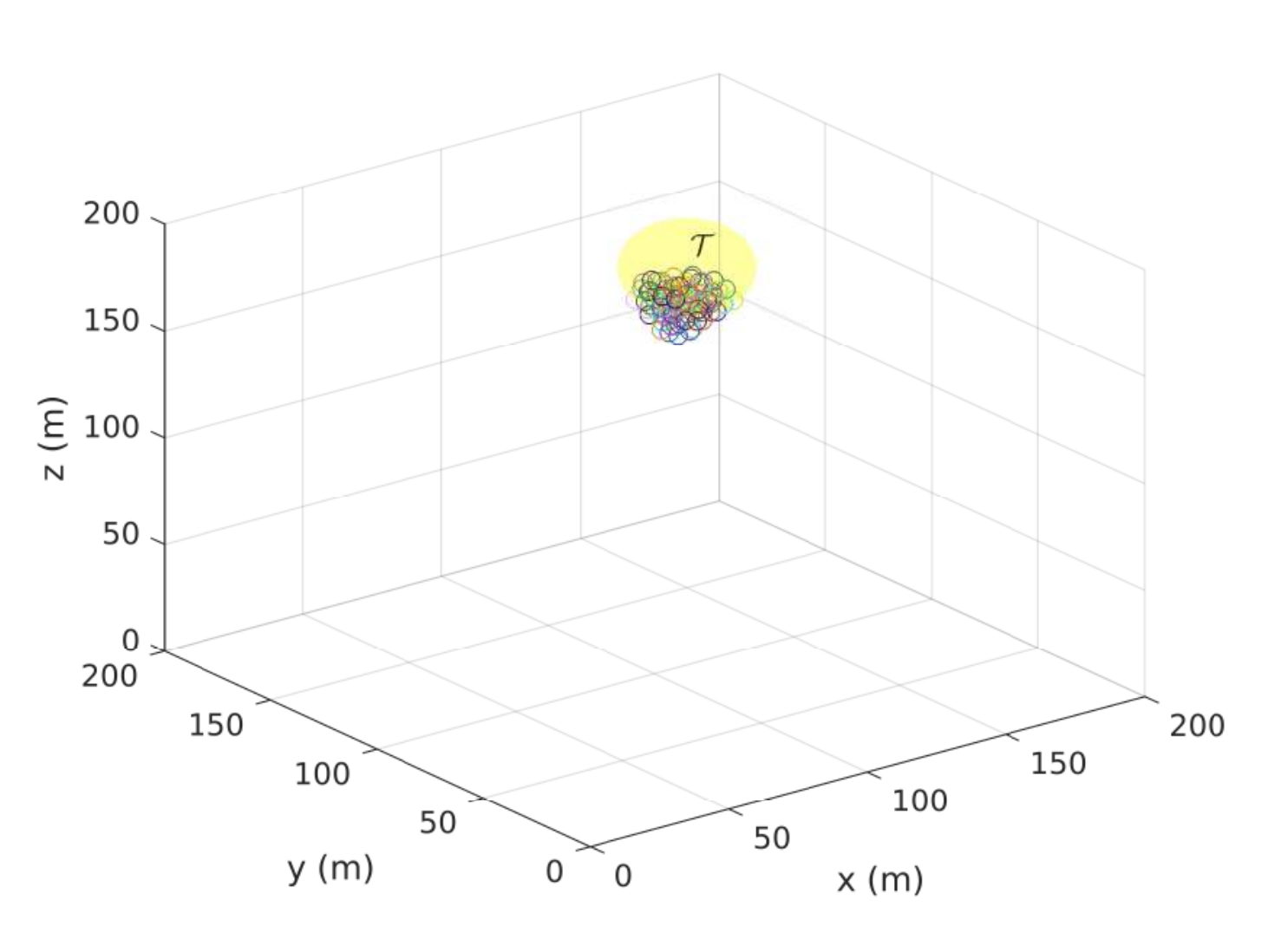} 
			\caption{}
		\end{subfigure}
		
		\caption{Simulation case ($n = 100$): vehicles position at different time instances during the motion}
		\label{fig:ch8:sim3MotionInstances}
	\end{adjustbox}
\end{figure}

\begin{figure}[!htb]
	\centering
	\includegraphics[width=0.7\linewidth]{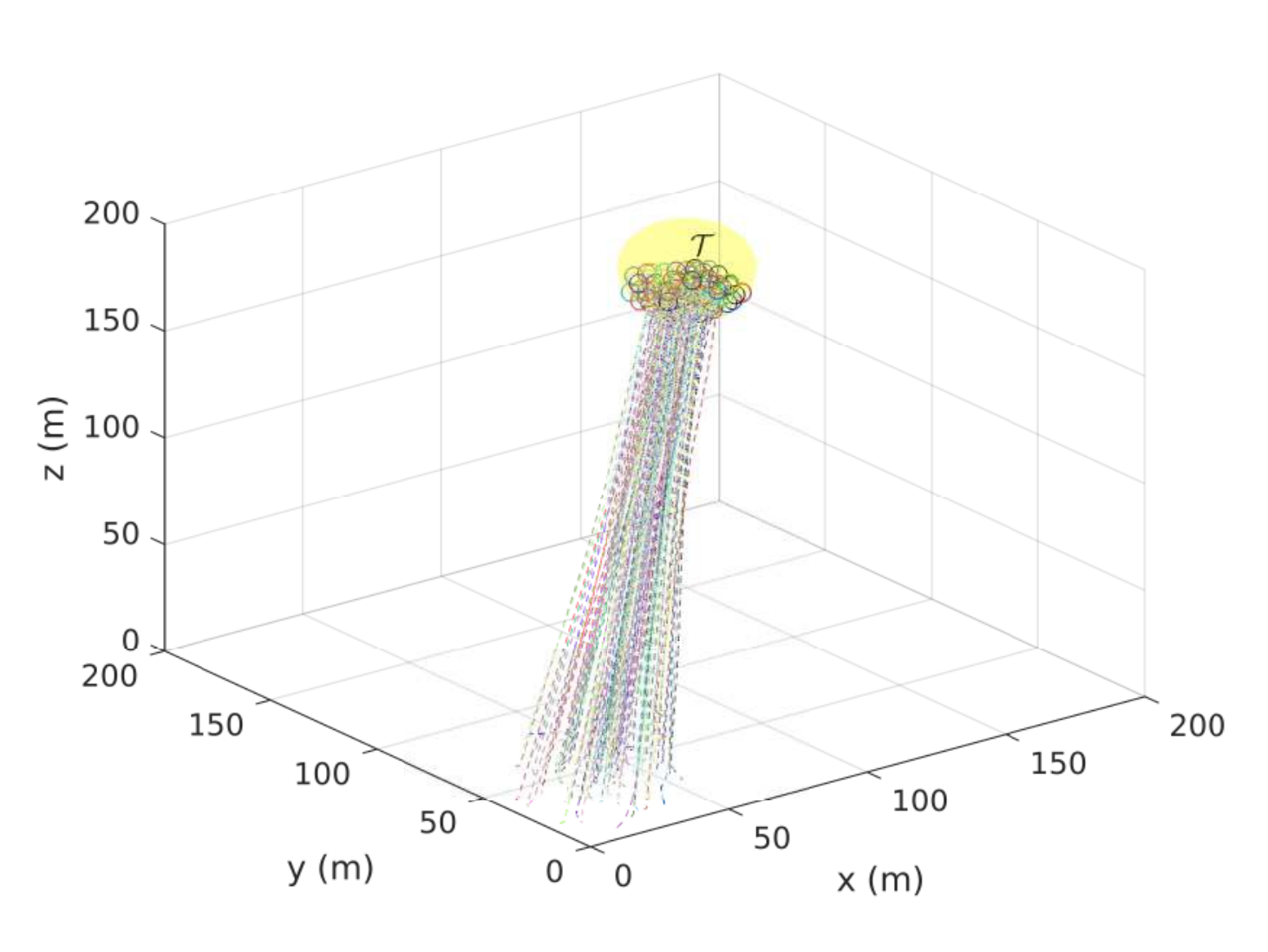} 
	\caption{Simulation case ($n = 100$): motion trajectories of vehicles} \label{fig:ch8:sim3Motion}
\end{figure}

\begin{figure}[!htb]
	\centering
	\includegraphics[width=0.7\linewidth]{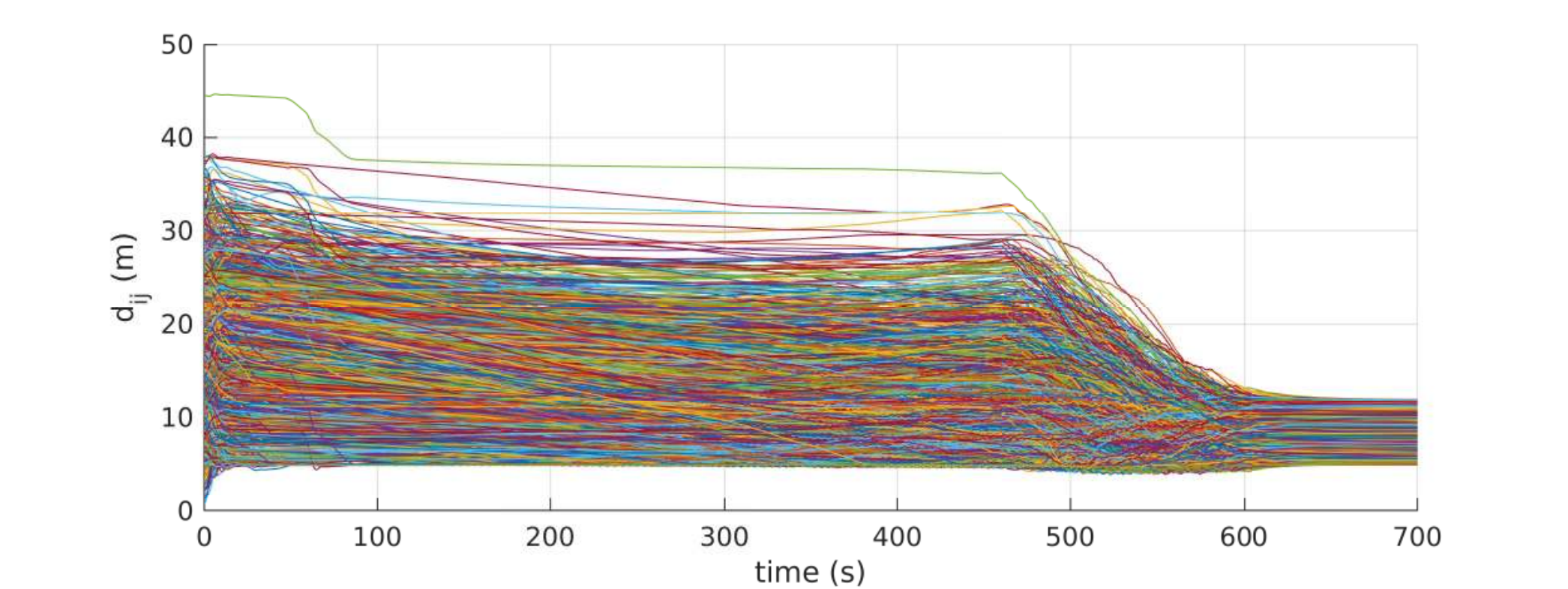} 
	\caption{Simulation case ($n = 100$): relative distances between vehicles versus time} \label{fig:ch8:sim3Dij}
\end{figure}

\begin{figure}[!htb]
	\centering
	\includegraphics[width=0.7\linewidth]{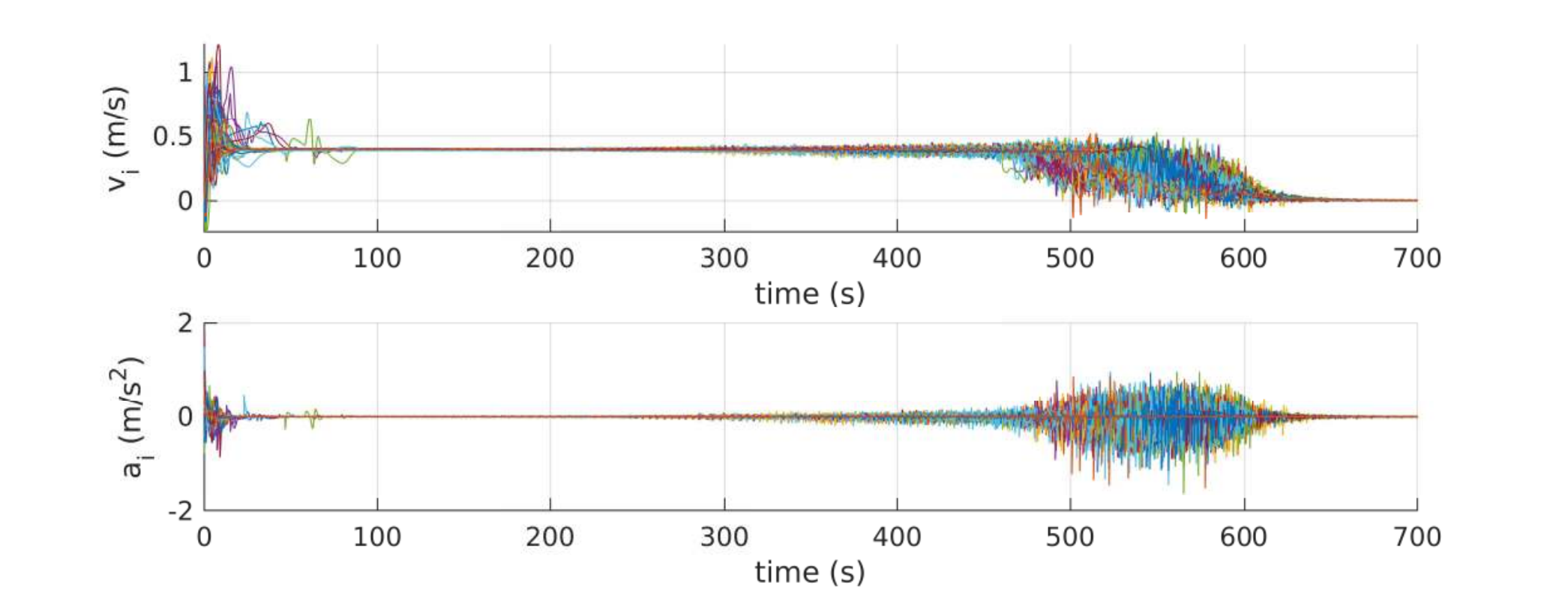} 
	\caption{Simulation case ($n = 100$): linear velocities $v_i$ and accelerations $a_i$ of vehicles versus time} \label{fig:ch8:sim3LinVelAcc}
\end{figure}

\section{Conclusion}\label{sec:ch8:conclusion}

A distributed flocking control method was suggested in this chapter for multi-vehicle systems considering a general 2D/3D kinematic model which is applicable to various unmanned vehicle types such as ground, aerial and underwater vehicles.
The main control objectives which can be achieved by the proposed control laws are maintaining a fixed formation, navigate to a goal region, collisions and obstacles avoidance.
The stability analysis of the multi-agent system have been studied under the application of the suggested control laws considering fixed communication topologies.
Also, conditions were provided on the control parameters to ensure collision-free motions while respecting some required safety margin.
Furthermore, simulations were carried out using systems with different number of vehicles moving in three-dimensional workspaces to validate our design and to show the scalability of the approach.

    \chapter{Distributed 3D Coverage Control Methods for Multi-UAV Systems\label{cha:coverage_control}}

This chapter proposes novel distributed control strategies to address coverage problems in three-dimensional (3D) sensing fields using multiple unmanned aerial vehicles (UAVs) which is another form of cooperative control with different global objective than flocking control which was addressed in \cref{cha:flocking_control}.
Two classes of coverage problems are considered here, namely barrier and sweep problems.
3D barrier coverage is defined as forming a static 3D arrangement (i.e. a barrier) of the multi-vehicle system for detecting objects/intruders going through the barrier. 
Contrarily, 3D sweeping problems require the multi-vehicle system to achieve maximal volumetric dynamic coverage with its sensors collecting data by moving across the 3D region.
The proposed control strategies adopt a region-based control approach based on Voronoi partitions to ensure collision-free self-deployment and coordinated movement of all vehicles within a 3D region.
The problem formulation is rather general considering mobile robots navigating in 3D spaces which make the proposed approach applicable to autonomous underwater vehicles (AUVs) as well.
However, further implementation details have also been investigated considering quadrotor-type UAVs with particular interest in precision agriculture applications.
Validation of the proposed methods have been performed using several simulations considering different simulation platforms such as MATLAB and Gazebo.
Software-in-the-loop simulations help to asses the real-time computational performance of the methods showing the actual implementation with quadrotors using C++ and the Robot Operating System (ROS) framework.
The work presented in this chapter is published in \cite{elmokadem2021coverage}.

\section{Introduction}\label{sec:ch9Intro}

In recent years, there has been an increasing interest in mobile wireless sensor networks (MWSNs) where a number of networked autonomous vehicles can be deployed in different environments to achieve sensing tasks.
Advances in communication made MWSNs more appealing where vehicles (sensors) can share information to perform cooperative monitoring, sensing, detection and exploration.
This have given rise to new challenges to traditional cooperative control in the field of \textit{coverage control} \cite{cortes2004coverage,hussein2007effective,pimenta2008sensing,cheng2009distributed,schwager2009decentralized,cheng2011decentralized,stergiopoulos2015distributed,savkin2015decentralized}.
Unmanned aerial vehicles (UAVs) have become a popular choice to form MWSNs especially in places inaccessible by ground vehicles.
Multi-UAV systems have been emerging in various applications such as precision agriculture \cite{chao2008band,hu2018application,ju2018multiple,maes2019perspectives,hegde2020multi}, aerial manipulation and transportation \cite{bernard2011autonomous,michael2011cooperative,fink2011planning,sreenath2013dynamics,ruggiero2018aerial}, surveillance and monitoring \cite{li2021networked}, search and rescue \cite{bernard2011autonomous,arnold2018search,hayat2020multi}, mapping and exploration \cite{cole2010system,hu2013cooperative,mahdoui2019communicating}, etc.

Coverage control problems can be classified as either static or dynamic.
Another classification is based on \cite{gage1992command} where coverage problems are categorized into \textit{Blanket coverage}, \textit{Barrier coverage} and \textit{Sweeping coverage} which are defined as follows:
\begin{itemize}
	\item \textit{Blanket coverage} is forming a static arrangement to maximize the detection rate of events through an area of interest.
	\item \textit{Barrier coverage} is a static formation over some region (i.e. a barrier) to minimize intrusions or maximizing detections of objects going through it.
	\item \textit{Sweeping coverage} is the formation of dynamic arrangements moving across a region of interest for maximal detection/exploration along the whole region.
\end{itemize}
Clearly, blanket and barrier coverage problems belong to the static class while sweeping is a dynamic coverage problem.

According to \cite{huang2018coverage}, some of the common techniques used to address static coverage control problems are resource-aware \cite{kwok2007energy,dieber2011resource,wang2012coverage}, search space-based \cite{morsly2011particle,abo2015rearrangement}, potential-based \cite{wang2008decentralized,howard2002mobile}, Voronoi partition-based \cite{cortes2004coverage,cortes2005spatially,schwager2009decentralized,bhattacharya2014multi,schwager2017robust,stergiopoulos2012autonomous,stergiopoulos2014cooperative,papatheodorou2017collaborative,thanou2014distributed,kantaros2016distributed} and angle view \cite{hexsel2011coverage,mohapatra2016big,saeed2017argus}.
There also exist recent methods addressing dynamic coverage problems such as \cite{atincc2020swarm,panagou2014vision,panagou2016distributed,li2017dynamic,bentz20173d,bentz2018complete,zuo2017dynamic,song2013persistent,bhattacharya2014multi,bhattacharya2013distributed}.

Many of the existing static and dynamic coverage control approaches consider only two-dimensional sensing fields, and the literature lacks a proper analysis of sensor networks deployed in three-dimensional (3D) sensing fields \cite{wang2012three}.
Even those proposed for multi-UAV and multi-AUV systems assume that the vehicles will be moving at a fixed altitude/depth without utilizing the full capabilities of such vehicles.
It is hence more motivating to  work towards addressing 3D coverage problems exploiting the rich geometric properties of 3D MWSNs \cite{wang2012three}.
Some efforts have been made in that area such as \cite{pompili2006deployment,stirling2010energy,barr2011efficient,wang2012three,boufares2015three,nazarzehi2018distributed}.
Thus, the main contribution of this work is to develop novel distributed control strategies to address the 3D barrier and sweep coverage problems motivated by some of the ideas in \cite{cortes2004coverage,cortes2005spatially,cheng2011decentralized}. 
In a 3D environment, a barrier can be defined as a static arrangement of sensors with overlapping sensing zones \cite{barr2011efficient} forming a surface or a 3D region.
The suggested control strategies rely on estimated centroidal Voronoi configurations over a virtual barrier generated by the sensors locations in a distributed manner depending only on shared information from neighbor vehicles.

The designed control laws require relative distances with neighbor vehicles to be shared over communication channels which makes the overall problem related to Networked Control Systems (NCSs) \cite{tipsuwan2003control,hespanha2007survey,wang2008networked,matveev2009estimation,bemporad2010networked,ge2017distributed}.
The current work assumes that such information is available to the control system. However, several challenges related to communication channels imperfections needs to be considered when evaluating the performance of the overall networked control system.
Example of such issues include delays introduced in communication channels \cite{tipsuwan2003control,matveev2003problem,onat2010control}, noises \cite{matveev2007analogue,goodwin2010analysis}, loss/corruption of data \cite{matveev2003problem,onat2010control} and bandwidth constraints \cite{savkin2003set,matveev2004problem,savkin2006analysis,savkin2007detectability}.

Overall, the vehicles' collective motion becomes constrained within a specific region (i.e. the virtual barrier) under the application of the suggested control methods similar to region-based shape control methods \cite{cheah2009region}.
Furthermore, one can control the dynamics of the barrier to generate 3D sweeping behavior which is the key idea used in the developed 3D sweeping coverage strategy.
This is also considered to handle obstacle avoidance where vehicles can collaboratively control the dynamics of the virtual barrier and even apply deformations to its shape in real-time which is then communicated through the networked multi-vehicle system. 
Also, bounded control laws are proposed which is important in practice to satisfy limits on the vehicles' velocities and accelerations.
The main advantages of the suggested approaches can be highlighted as follows:
\begin{itemize}
	\item collision avoidance among vehicles and connectivity is ensured by the adopted Voronoi-based approach
	\item the approach is highly scalable and robust against vehicles' failure
	\item obstacle avoidance can be managed in a decomposed and distributed manner
\end{itemize}
A general 3D kinematic model is adopted in the design which is applicable to different UAV types and AUVs.
A 6DOF dynamical model for quadrotors is further considered to show a possible way of implementation with low-level control design.
Several simulations were carried out to validate the performance of the suggested methods in addition to showing its scalability and robustness.
Moreover, software-in-the-loop (SITL) simulations were also performed in Gazebo based on the quadrotor full dynamical model to evaluate the computational complexity of the implemented algorithms with particular interest in applications related to precision agriculture.

The organization of this chapter is as follows.
\Cref{sec:ch9pre} introduces some essential concepts related to graph theory, locational optimization and Voronoi Partitions which are used in our control strategy, and the tackled 3D coverage problems are defined in \cref{sec:ch9problem1}.
After that, distributed barrier and sweeping coverage control strategies are proposed in \cref{sec:ch9control} considering a general 3D kinematic model.
These approaches are validated through several simulation cases in \cref{sec:ch9simulation}.
Further implementation details considering quadrotors dynamics with low-level control design are presented in \cref{sec:ch9:impl} which is evaluated using software-in-the-loop simulations.
Finally, this work is concluded in \cref{sec:ch9conclusion} with a suggestion for a potential direction of future work.

\section{Preliminaries}\label{sec:ch9pre}

The proposed methods in chapter relies on concepts from graph theory, locational optimisation and Voronoi partitions.
A summary of these concepts is provided in this section based on \cite{olfati2006flocking,cortes2004coverage,cortes2005spatially,bullo2009distributed}.
Note that when considering a multi-UAV system as a mobile wireless sensor network, UAVs may interchangeably referred to throughout the chapter as sensors, nodes, agents or vehicles.

\subsection{Graph Theory}

A multi-UAV sensor network consisting of $n$ UAVs can generally be characterised using a set of nodes/vertices $\U=\{1,2,\cdots,n\}$ and a set of edges (paired vertices) $\E \subseteq \{(i,j): i,j\in \U,\ j \neq i \}$.
Each vertex corresponds to a single UAV/sensor, and edges represent interaction between UAVs which are within communication or detection range from each others.
The overall network topology is then described using a graph $G=(\U, \E)$ which can be \textit{directed} or \textit{undirected}.
In an \textit{undirected graph}, an edge exists from vertex $i$ to vertex $j$ if and only if an edge exists from $j$ to $i$ (i.e. $(i,j)\in \E \leftrightarrow (j,i)\in \E$).
Otherwise, the graph is called \textit{directed}.
Generally, homogeneous multi-UAV systems can be described using undirected graphs since all UAVs have same communication and sensing capabilities.
Moreover, a \textit{path} between two vertices $i_0$ and $i_k$ is defined as a sequence of vertices $\{i_0,\ i_1,\ \cdots,\ i_k\} \subset \U$ where an edge exists between each subsequent vertices in the sequence such that $(i_l,i_{l+1}) \in \E,\ \forall l \in \{0,\ 1,\ \cdots,\ k-1\}$.
If every pair of vertices in $\U$ is connected by a path, the graph $G(\U,\E)$ is then called \textit{connected}.
Clearly, a crucial part for MWSNs is to maintain network connectivity all the time.

Furthermore, define a \textit{neighbourhood} $\mathpzc{N}_{i}$ around a vertex $i$ as the set of all vertices which have edges with vertex $i$ such that:
\begin{equation}
	\mathpzc{N}_{i} = \{j\in \U \backslash \{i\}: (i,j)\in \E \}
\end{equation}
For a homogeneous system, let $r>0$ denote the communication range for all UAVs.
Hence, all UAVs within a spherical region of radius $r$ around UAV $i$ belong to its neighborhood such that
\begin{equation}
\mathpzc{N}_{i} = \{j\in \U \backslash \{i\}: ||\boldsymbol{p}_i - \boldsymbol{p}_j|| \leq r \}
\end{equation}
where $\boldsymbol{p}_i \in \R^3$ is the position of UAV $i$, and $||\cdot||$ is the Euclidean norm in $\R^3$.

\subsection{Locational Optimization}

Deployment of mobile sensors in an environment to achieve optimal sensor coverage is regarded as a multicenter problem from locational optimization  (i.e. spatial resource-allocation problem).
A brief description about some of the facts related to this class of problems is summarised next based on \cite{cortes2004coverage,bullo2009distributed}.

Consider a bounded region of interest $\mathcal{Q}$, including its interior, defined in a space $\R^d$ with dimension $d$.
A \textit{partition} of $\mathcal{Q}$ consists of a group of $n$ non-overlapping polytopes $\mathcal{W} = \{W_1,\ \cdots,\ W_n\}$ such that $W_1\cup \cdots\cup W_n = \mathcal{Q}$.
Also, let $\phi: \mathcal{Q} \to \R_{+},\ (\R_{+}=\{a\in \R:a>0\})$ be defined as a \textit{distribution density function} representing a measure of information or the likelihood of an event to take place over $\mathcal{Q}$.
The \textit{sensing performance} of a sensor located at some position $\boldsymbol{p}_i$ as seen from any point $\boldsymbol{q} \in \mathcal{Q}$ depends mostly on the distance 
$||\boldsymbol{q} - \boldsymbol{p}_i||$.
Clearly, as this distance increases, the sensing performance degrades.
Hence, one can describe the sensing performance at location $\boldsymbol{q}$ of the sensor $\boldsymbol{p}_i$ using a non-increasing piecewise continuously differentiable function $f(||\boldsymbol{q} - \boldsymbol{p}_i||):\R_{+} \to \R$.
Thus, the larger the value of $f$, the better the sensing performance at $\boldsymbol{q}$ is.

Using the above definitions, one can define a multicenter cost function characterizing the average coverage provided by a set of $n$ sensors at $\boldsymbol{p}_1,\ \cdots,\ \boldsymbol{p}_n$  over an point in $\mathcal{Q}$ as follows:
\begin{equation}\label{equ:ch9locationalOpt}
\mathcal{H}(\boldsymbol{p}_1,\ \cdots,\ \boldsymbol{p}_n) = \int_{\mathcal{Q}} \max_{i \in \{1,\cdots,n\}} f(||\boldsymbol{q} - \boldsymbol{p}_i||)\phi(\boldsymbol{q})d\boldsymbol{q}
\end{equation}
The above function provides a measure of the sensing performance expected value  provided by all sensors at any point $\boldsymbol{q} \in \mathcal{Q}$ \cite{cortes2005spatially}.
Now, in order to find the optimal placement for all sensors, an optimization problem needs to be solved to maximize the value of $\mathcal{H}(\boldsymbol{p}_1,\ \cdots,\ \boldsymbol{p}_n)$.

\begin{remark}
	Note that there are slightly different definitions for $f$ in the references  \cite{cortes2004coverage,cortes2005spatially,bullo2009distributed} where it can be either considered as a representation of sensing degradation or sensing performance over $\mathcal{Q}$ (as considered here).
	This does not affect the overall analysis done here except that the considered optimization problem will either be minimization (of sensing degradation) or maximization (of sensing performance).
\end{remark}

\subsection{Voronoi Partitions}

This subsection highlights some key points about Voronoi partitions needed for our problem formulation.
A \textit{Voronoi partition/diagram} is the subdivision of a space into a number of regions generated by a set of points (see \cref{fig:ch9Voronoi} for a 2D example).
Consider that we have $n$ sensors located at fixed locations $\boldsymbol{P}=\{\boldsymbol{p}_1,\ \cdots,\ \boldsymbol{p}_n\} \in \mathcal{Q}^n$.
A voronoi partition of $\mathcal{Q}$ consists of a set of disjoint Voronoi regions/cells $\mathcal{V}(\boldsymbol{P}) = \{V_1,\ \cdots,\ V_n\}$ generated by these sensors where
\begin{equation}
V_i = \{\boldsymbol{q}\in \mathcal{Q}: ||\boldsymbol{q} - \boldsymbol{p}_i|| \leq ||\boldsymbol{q} - \boldsymbol{p}_j||,\ \forall j \neq i \}
\end{equation}
and $V_1 \cup V_2 \cup \cdots \cup V_n = \mathcal{Q}$.

It has been established that this Voronoi partition is the optimal partition of $\mathcal{Q}$ among all other partitions \cite{bullo2009distributed}.
For any sensor located at a position $\boldsymbol{p}_i$, its \textit{Voronoi neighbors} $\mathcal{N}_{\mathcal{V},i} \subset \mathcal{P}$ are defined as the sensors corresponding to adjacent Voronoi cells such that:
\begin{equation}
	\mathcal{N}_{\mathcal{V},i} = \{j\in \{1,\ \cdots,\ n\}: V_i \cap V_j \neq \emptyset,\ j \neq i \}
\end{equation}
 
\begin{figure}[!htb]
	\centering
	\includegraphics[width=0.5\linewidth]{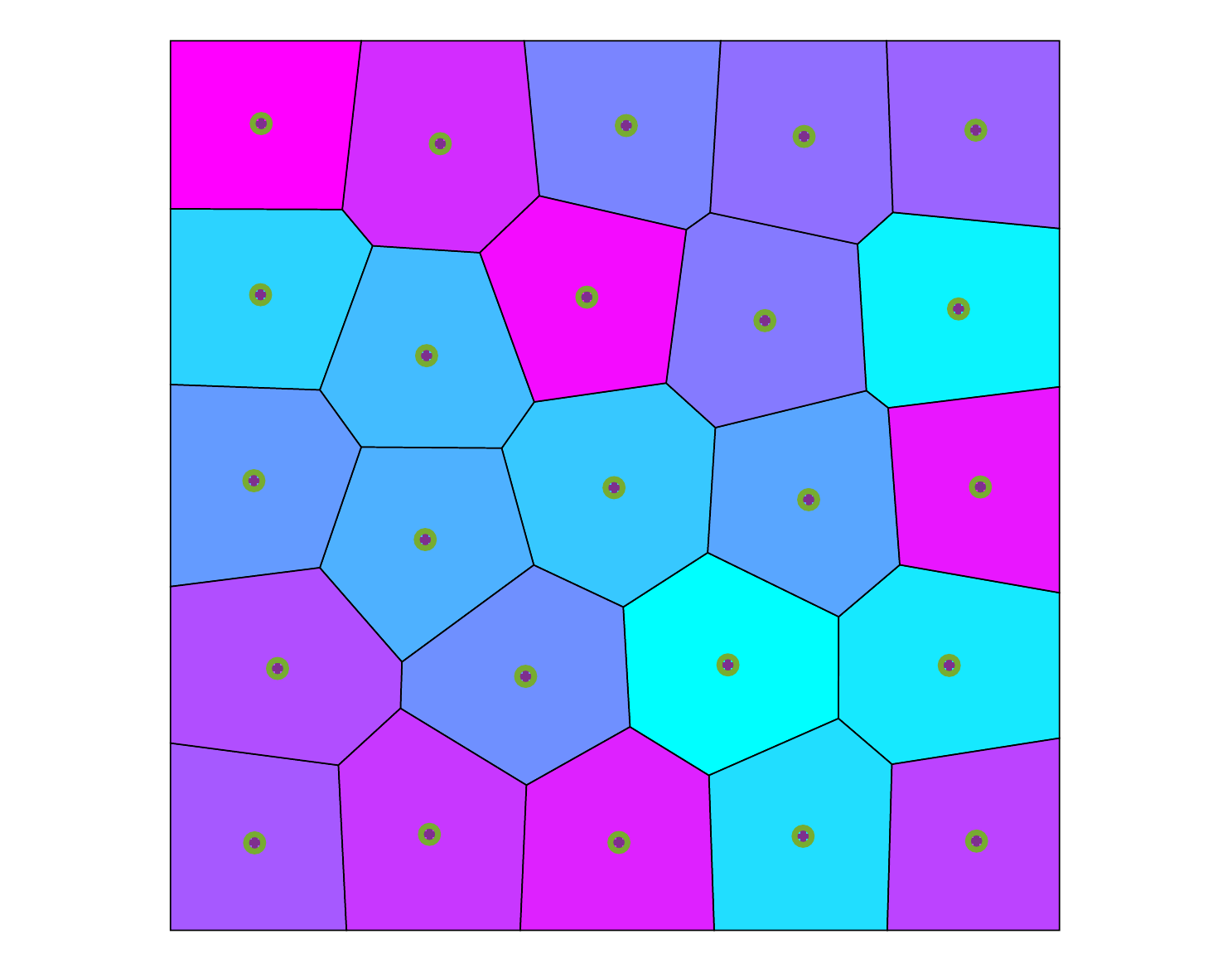} 
	\caption{Example Voronoi partition of a 2D plane divided into number of regions with their associated points} \label{fig:ch9Voronoi}
\end{figure}

Considering the above definition, one can rewrite \eqref{equ:ch9locationalOpt} as:
\begin{equation}\label{equ:ch9HV}
\mathcal{H}_{\mathcal{V}}(\boldsymbol{P},\mathcal{V}(\boldsymbol{P})) =  \sum_{i=1}^{n}\int_{V_i(\boldsymbol{P})} f(||\boldsymbol{q} - \boldsymbol{p}_i||)\phi(\boldsymbol{q})d\boldsymbol{q}
\end{equation}
By taking the partial derivative of \eqref{equ:ch9HV} with respect to $\boldsymbol{p}_i$, the following is obtained:
\begin{equation}\label{equ:ch9H_derivative}
\dfrac{\partial \mathcal{H}_{\mathcal{V}}}{\partial \boldsymbol{p}_i} (\boldsymbol{P},\mathcal{V}(\boldsymbol{P})) = \int_{V_i} \dfrac{\partial}{\partial \boldsymbol{p}_i} f(||\boldsymbol{q} - \boldsymbol{p}_i||)\phi(\boldsymbol{q})d\boldsymbol{q}
\end{equation}
where it is assumed that $f$ does not have any discontinuities. 
Furthermore, considering $f(x) = -x^2$, the multicenter cost function in \eqref{equ:ch9HV} becomes:
\begin{equation}\label{equ:ch9HV2}
\begin{aligned}
\mathcal{H}_{\mathcal{V}}(\boldsymbol{P},\mathcal{V}(\boldsymbol{P})) &=
-\sum_{i=1}^{n}\int_{V_i(\boldsymbol{P})} ||\boldsymbol{q} - \boldsymbol{p}_i||^2 \phi(\boldsymbol{q})d\boldsymbol{q} := -\sum_{i=1}^{n} J_{V_i,\boldsymbol{p}_i}
\end{aligned}
\end{equation}
where $J_{V_i,\boldsymbol{p}_i}$ is the polar moment of inertia of $V_i$ about $\boldsymbol{p}_i$.
Consequently, \eqref{equ:ch9HV2} reduces to:
\begin{equation}\label{equ:ch9dHdp}
\dfrac{\partial \mathcal{H}_{\mathcal{V}}}{\partial \boldsymbol{p}_i} = 2 M_{V_i} (\cvi - \boldsymbol{p}_i)
\end{equation}
where $M_{V_i}\in \R_{+}$ and $\cvi \in \R^3$ are the mass and center of mass (centroid) of the corresponding Voronoi partition $V_i$ with respect to the density function $\phi(\boldsymbol{q})$.
It is clear from \eqref{equ:ch9dHdp} that the critical points of $\mathcal{H}_{\mathcal{V}}(\boldsymbol{P},\mathcal{V}(\boldsymbol{P}))$ are the configurations $\boldsymbol{P}\in \mathcal{Q}^n$ where $\boldsymbol{p}_i = \cvi\ \forall i$  which are referred to as \textit{centroidal Voronoi configurations}.

\section{3D Coverage Problems}\label{sec:ch9problem1}

Consider a 3D bounded region of interest $\mathcal{S} \subset \R^3$.
A multi-vehicle system can perform coverage tasks over $\mathcal{S}$ where coverage objective may vary according to the problem in hand.
Definitions of the considered \textit{barrier} and \textit{sweep} coverage problems in 3D are defined next.

\begin{problem}\label{prob:ch9barrier} (3D Barrier Coverage) 
	Deploy a network of vehicles/sensors to form a static arrangement over some region $\mathcal{B} \subset \mathcal{S}$ (i.e. a barrier) maximizing the sensing performance of the overall network to detect any intruder going through the barrier.
\end{problem}
A special case of the above problem is when deploying the sensors over a planar region within $ \mathcal{S}$ defined by $\mathcal{B}=\{(x,y,z):\sigma(x,y,z)=0\}$ where $\sigma(x,y,z)$ is the plane's equation.
Note that for this problem to be solvable, the number of vehicles/sensors needed depends the size of $\mathcal{S}$ and the sensing range of all vehicles (assuming a homogenous system). 

\begin{problem}\label{prob:ch9sweep} (3D Sweep Coverage) 
	Consider a group of vehicles whose overall sensing range is not large enough to achieve complete coverage over $\mathcal{S}$.
	It is required to scan the region $\mathcal{S}$ by moving the whole group across $\mathcal{S}$ as a dynamic formation over some region $\mathcal{F}(t)\subset\mathcal{S}$.
	This task can be done once or contentiously.
\end{problem}
Note that $\mathcal{F}(t)$ can be of any 3D shape in which its motion along $\mathcal{S}$ following a certain pattern can achieve complete coverage.
A special case is when $\mathcal{F}(t)$ is a plane which is considered in this work.
In this case, $\mathcal{F}(t)$ will referred to as the \textit{sweeping plane}.
The dynamics of the sweeping plane $\dot{\mathcal{F}}(t)$ can be determined in a way to achieve complete coverage over $\mathcal{S}$.
It is also assumed that $\mathcal{F}(t)$ can change size and shape over time which can be utilized for other motion objectives such as obstacle avoidance as will be shown later.

\subsection{Problem Formulation}\label{sec:ch9problem2}

The aim of this work is to develop distributed control laws for multi-UAV systems to address problems~\ref{prob:ch9barrier} and \ref{prob:ch9sweep}.
We consider a system of $n$ homogenous vehicles (UAVs/AUVs) with a single integrator motion model given by:
\begin{equation}\label{equ:ch9model}
\dot{\boldsymbol{p}}_i(t) = \boldsymbol{u}_i(t)
\end{equation}
where $\boldsymbol{p}_i \in \R^3$ is the $i$-th vehicle position defined in some inertial frame $\{\mathcal{I}\}$, and $\boldsymbol{u}_i \in \R^3$ is its control input (velocity) where
\begin{equation}\label{equ:ch9:maxInput}
\|\boldsymbol{u}_i\| \leq u_{max}
\end{equation}
All vehicles can sense events in the environment within a sensing range $r_s>0$.
Also, any vehicle can exchange information with nearby vehicles within some communication range $r_c>0$.
Obviously, it is assumed that $r_c>r_s$ so that it is possible to design control laws which can maintain the connectivity of the network with minimal sensors overlapping.

\section{Distributed Coverage Control Strategies}\label{sec:ch9control}

The proposed control schemes to address problems~\ref{prob:ch9barrier} and \ref{prob:ch9sweep} are based on a self-deployment method for the multi-vehicle system over a planar region $\mathcal{B}\subset\mathcal{S}$ which is static for barrier coverage problems and dynamic for sweep coverage problems.
We consider a set of $l$ vertices $\boldsymbol{E}$ to describe the boundary of $\mathcal{B}$ as a polygon such that $\boldsymbol{E}=\{\boldsymbol{e}_1,\ \boldsymbol{e}_2,\ \cdots,\ \boldsymbol{e}_l\} \subset \mathcal{S}$; clearly, $l>2$.
Lloyd's algorithm is adopted in the designed controllers to guide the vehicles to reach the instantaneous centroids of their associated Voronoi regions over $\mathcal{B}$ (i.e. reaching the centroidal Voronoi configuration).
Once this is reached, an optimal coverage over $\mathcal{B}$ is achieved.
Furthermore, for sweeping problems, the designed dynamics of $\mathcal{B}$ will achieve coverage over $\mathcal{S}$.

\subsection{Online Computation of Centroidal Voronoi Configurations}\label{sec:ch9:centroidsVor}

The developed control law requires each vehicle to be able to compute the centroid of its Voronoi region in a distributed fashion based only on information exchanged with vehicles within its neighbourhood.
We extend the approach proposed in \cite{cortes2004coverage} to compute Voronoi cells for planar regions in 3D in a distributed fashion.

To simplify the mathematical development, a new 3D coordinate frame $\{\mathcal{B}\}$ attached to $\mathcal{B}$ is needed.
The origin of $\{\mathcal{B}\}$ can be selected to be one of the barrier vertices defined as $\mathcal{O}_{\mathcal{B}}=\boldsymbol{e}_1$.
Furthermore, the axes of $\{\mathcal{B}\}$ are defined using the orthonormal basis $\{\vect{a}_1, \vect{a}_2, \vect{a}_3\}$ where:
\begin{equation}\label{equ:ch3:barrierAxes}
\begin{aligned}
\vect{a}_1 &=\frac{ \boldsymbol{e}_2 - \boldsymbol{e}_1}{||\boldsymbol{e}_2 - \boldsymbol{e}_1||} \\
\vect{a}_2 &= \frac{h(\vect{a}_1,\vect{b})}{||h(\vect{a}_1,\vect{b})||},\ \ \vect{b} = \frac{\boldsymbol{e}_l - \boldsymbol{e}_1}{||\boldsymbol{e}_l - \boldsymbol{e}_1||} \\
\vect{a}_3 &= \vect{a}_1 \times \vect{a}_2
\end{aligned}
\end{equation}
where $h(\vect{w}_1,\vect{w}_2) = \vect{w}_2 - (\vect{w}_1 \cdot \vect{w}_2)\vect{w}_1$ is a mapping function which gives an orthogonal vector to $\vect{w}_1 \in \R^3$ directed towards $\vect{w}_2 \in \R^3$.

All computations needed to find Voronoi centroids are carried out in the $\{\mathcal{B}\}$ frame through a transformation between $\{\mathcal{B}\}$ and the inertial frame $\{\mathcal{I}\}$.
Let $\boldsymbol{v}^{\mathcal{I}} \in \R^3$ be a vector defined in the $\{\mathcal{I}\}$ frame.
This vector can be transformed to the $\{\mathcal{B}\}$ coordinate frame using a transformation matrix $T_{\mathcal{I}}^{\mathcal{B}}$ as follows:
\begin{equation}\label{equ:ch9transformation}
	\left[\begin{array}{c}
	\boldsymbol{v}^{\mathcal{B}} \\ 1
	\end{array}\right] = T_{\mathcal{I}}^{\mathcal{B}} \left[\begin{array}{c}
	\boldsymbol{v}^{\mathcal{I}} \\ 1
	\end{array}\right]
\end{equation}
where $T_{\mathcal{I}}^{\mathcal{B}}$ is a $4\times 4$ affine transformation matrix given by:
\begin{equation*}
T^{\mathcal{B}}_{\mathcal{I}} = \left(T^{\mathcal{I}}_{\mathcal{B}}\right)^{-1} = \left[\begin{array}{cccc}
\vect{a}_1 & \vect{a}_2 & \vect{a}_3 & \mathcal{O}_{\mathcal{B}} \\
0 & 0 & 0 & 1
\end{array}\right]^{-1}
\end{equation*}
Note that we represent the transformation using an augmented matrix to consider both rotation and translation in a single matrix multiplication.
Also, $\vect{a}_1, \vect{a}_2, \vect{a}_3,$ and $\mathcal{O}_{\mathcal{B}}$ are column vectors.

From now on, vectors represented in the inertial frame will be represented without the $\mathcal{I}$ superscript for simplicity.
Given a UAV at position $\bm{p}_i$, it is required to compute instantaneous Voronoi centroid $\cvi$ of its projection onto $\mathcal{B}$.
First, the following assumption is made.
\begin{assumption}\label{assm:voronoiCells}
	Each Voronoi cell $V_i \subset \mathcal{B}$, generated by the projection of UAV $i$ onto $\mathcal{B}$, is a convex polygon defined by $m$ vertices $\{v_1^{\mathcal{B}}, v_2^{\mathcal{B}}, \cdots, v_m^{\mathcal{B}}\}$ where $v_j^{\mathcal{B}} = (\bar{x}_j,\bar{y}_j,0),\ j=\{1,\cdots,m\}$.
\end{assumption}
The proposed approach can now be described in these steps:
\begin{enumerate}
	\item[\textbf{S}1:] Transform the position $\bm{p}_i$ into the $\{\mathcal{B}\}$ frame to obtain $\bm{p}_i^{\mathcal{B}}=[\bar{x}_i, \bar{y}_i,\bar{z}_i]^T$ by applying \eqref{equ:ch9transformation}.
	\item[\textbf{S}2:] Compute the projection of $\bm{p}_i^{\mathcal{B}}$ onto $\mathcal{B}$ defined in the $\{\mathcal{B}\}$ frame by setting $\bar{z}_i = 0$ as $\tilde{\bm{p}}_i^{\mathcal{B}}=[\bar{x}_i, \bar{y}_i,0]^T$.
	\item[\textbf{S}3:] Compute the Voronoi cell centroid $\tilde{\boldsymbol{C}}_{V_i}^{\mathcal{B}}=[\bar{c}_{x,V_i}, \bar{c}_{y,V_i}, 0]^T$ associated with $\tilde{\bm{p}}_i^{\mathcal{B}}$ using the following \cite{cortes2004coverage}:
	\begin{equation}\label{equ:ch9centroid}
		\begin{aligned}
			M_{V_i} &= \frac{1}{2} \sum_{k=1}^{m} (\bar{x}_k \bar{y}_{k+1} - \bar{x}_{k+1} \bar{y}_k) \\
			\bar{c}_{x,V_i} &= \frac{1}{6 M_{V_i}} \sum_{k=1}^{m} (\bar{x}_k + \bar{x}_{k+1}) (\bar{x}_k \bar{y}_{k+1} - \bar{x}_{k+1} \bar{y}_k) \\
			\bar{c}_{y,V_i} &= \frac{1}{6 M_{V_i}} \sum_{k=1}^{m} (\bar{y}_k + \bar{y}_{k+1}) (\bar{x}_k \bar{y}_{k+1} - \bar{x}_{k+1} \bar{y}_k)
		\end{aligned}
	\end{equation}
	where $v_{m+1}^{\mathcal{B}} = (\bar{x}_{m+1},\bar{y}_{m+1},0) = v_{1}^{\mathcal{B}}$.
	These equations are obtained considering assumption~\ref{assm:voronoiCells} and a constant distribution density function $\phi(\bm{q}) = 1$.
	Voronoi cell vertices $\{v_1^{\mathcal{B}}, v_2^{\mathcal{B}}, \cdots, v_m^{\mathcal{B}}\}$ can be determined based on locations of Voronoi neighbours in a distributed fashion (see remark~\ref{rem:ch9VoronoiVertices}).
	\item[\textbf{S}4:] Transform $\tilde{\boldsymbol{C}}_{V_i}^{\mathcal{B}}$ to the inertial frame $\{\mathcal{I}\}$ to get $\tilde{\boldsymbol{C}}_{V_i}=[c_{x,V_i}, c_{y,V_i}, c_{z,V_i}]^T$ using \eqref{equ:ch9transformation}.
\end{enumerate}

\begin{remark}\label{rem:ch9VoronoiVertices}
	The vertices of a Voronoi cell $V_i$ can be found as the circumcenters of triangles formed by $\tilde{\bm{p}}_i^{\mathcal{B}}$ and any two of its Voronoi neighbours.
	A triangle made by three points $\bm{p}_1$, $\bm{p}_2$ and $\bm{p}_3$ with an area of $A$ has a circumcenter at \cite{cortes2004coverage}:
	\begin{equation}\label{equ:ch9circumcenter}
	\begin{aligned}
	&circumcenter = \frac{1}{4A} \Big(||\boldsymbol{\alpha}_{32}||^2 (\boldsymbol{\alpha}_{21} \cdot \boldsymbol{\alpha}_{13}) \boldsymbol{p}_1 \\
	&+ ||\boldsymbol{\alpha}_{13}||^2 (\boldsymbol{\alpha}_{32} \cdot \boldsymbol{\alpha}_{21}) \boldsymbol{p}_2 + ||\boldsymbol{\alpha}_{21}||^2 (\boldsymbol{\alpha}_{13} \cdot \boldsymbol{\alpha}_{32}) \boldsymbol{p}_3\Big)
	\end{aligned}
	\end{equation}
	where $\boldsymbol{\alpha}_{ls} = \boldsymbol{p}_l - \boldsymbol{p}_s$.
\end{remark}

\subsection{Barrier Coverage Control Design}

In order to present our control design, some technical assumptions need to be made as follows.

\begin{assumption}\label{assmp:ch9:Gconnected}
	The communication graph $G(\U,\E)$ remains connected for $t \geq 0$.
\end{assumption}

\begin{assumption}\label{assm:ch9:VoronoiNg}
	Each UAV is capable of estimating the Voronoi cell centroid of its projected position onto $\mathcal{B}$ at any time using only information from UAVs within its communication range (i.e UAVs in its neighbourhood $\mathcal{N}_i$).
\end{assumption}

\begin{assumption}\label{assm:ch9:VoronoiProjection}
	The Voronoi neighbours of $\tilde{\bm{p}}_i^{\mathcal{B}}$ correspond to UAVs within the neighbourhood $\mathcal{N}_i$.
\end{assumption}

\begin{assumption}\label{assmp:ch9:initialP}
	The initial configuration of the multi-UAV system satisfies the following condition: $\frac{\boldsymbol{p}_i(0) - \boldsymbol{p}_j(0)}{||\boldsymbol{p}_i(0) - \boldsymbol{p}_j(0)||}\cdot \bm{a}_3 \neq \pm 1,\ \forall  i,j\in\{1,\ \cdots,\ n\}, i\neq j$ where $\bm{a}_3$ is the barriers normal as defined in \eqref{equ:ch3:barrierAxes}.
\end{assumption}

Assumption~\ref{assmp:ch9:Gconnected} is made to make sure that updated information about the barrier $\mathcal{B}$ is available to all vehicles at any time during the motion.
This is essential in cases where $\mathcal{B}$ is dynamic.
For example, a decision could be made by one of the UAVs to apply changes to the shape of $\mathcal{B}$ based on some detected obstacles.
Such information needs to be shared among all vehicles so that they can compute their Voronoi regions accordingly.
Assumptions~\ref{assm:ch9:VoronoiNg} and \ref{assm:ch9:VoronoiProjection} ensures that UAVs can determine Voronoi centroids in a distributed fashion. 
Finally, assumption~\ref{assmp:ch9:initialP} ensures that all UAVs are initially located at positions with unique projections onto $\mathcal{B}$.

Now, the main results of this section can be presented.
Consider the following control law $\mathbf{u}_i:\R^3 \to \R^3$ based on Lloyd's algorithm:
\begin{equation}\label{equ:ch9u_unbounded}
\mathbf{u}_i = \bar{\boldsymbol{K}}_i (\tilde{\boldsymbol{C}}_{V_i} - \boldsymbol{p}_i)
\end{equation}
where the parameter $\bar{\boldsymbol{K}}_i$ is a diagonal positive definite  gain matrix.
We also propose a more practical bounded control law $\mathbf{u}_i:\R^3 \to \boldsymbol{\Gamma}$ which can satisfy bounds $u_{max}$ on the control input such that $\boldsymbol{\Gamma}=\{\boldsymbol{u}\in \R^3: ||\boldsymbol{u}|| \leq u_{max}\}$.
It is given as follows:
\begin{equation}\label{equ:ch9u_bounded}
\mathbf{u}_i = \boldsymbol{K}_i \tanh\Big(\gamma_i(\tilde{\boldsymbol{C}}_{V_i} - \boldsymbol{p}_i)\Big)
\end{equation}
where $\boldsymbol{K}_i=diag\{k_{i,x},\ k_{i,y},\ k_{i,z}\}$ is a positive definite diagonal matrix, $\gamma_i>0$, and $\tanh(\boldsymbol{v})$ is the hyperbolic tangent function defined element-wise for any vector $\boldsymbol{v} \in \R^3$.
Clear, the bound of this control law depends on the gain matrix $\boldsymbol{K}_i$ as follows:
\begin{equation} \label{equ:ch9u_bounds}
||\mathbf{u}_i|| \leq \sqrt{k_{i,x}^2 + k_{i,y}^2 + k_{i,z}^2} := u_{i,max}
\end{equation}

\begin{theorem}\label{thm:ch9:barrier}
	Consider a multi-UAV system with $n$ vehicles whose motion are modelled as \eqref{equ:ch9model}.
	Under assumptions~\ref{assmp:ch9:Gconnected}-\ref{assmp:ch9:initialP},
	the distributed control law \eqref{equ:ch9u_unbounded} (or \eqref{equ:ch9u_bounded} for bounded inputs) along with the algorithm in \textbf{S1}-\textbf{S4} solves the 3D barrier coverage problem defined in problem~\ref{prob:ch9barrier}.
\end{theorem}

\begin{proof}
	Let $\bm{P}_e$ be a centroidal Voronoi configuration.
	We define a Lyapunov canidate function as
	\begin{equation}\label{equ:ch9:barrierLyap}
		V_1(\bm{P}(t))=\frac{1}{2}\bar{\mathcal{H}} - \frac{1}{2}\mathcal{H}_{\mathcal{V}}(\bm{P}(t),\mathcal{V}(\bm{P}(t)))
	\end{equation}
	where $\mathcal{H}_{\mathcal{V}}$ is defined in \eqref{equ:ch9HV2} with $\phi(\bm{q}) = 1$, and $\bar{\mathcal{H}}=\mathcal{H}_{\mathcal{V}}(\bm{P}_e,\mathcal{V}(\bm{P}_e))$ which is constant.
	Hence, $V_1(\bm{P}_e) = 0$
	Furthermore, $\bar{\mathcal{H}} \geq \mathcal{H}_{\mathcal{V}}(\bm{P}(t),\mathcal{V}(\bm{P}(t)))\ \ \forall \bm{P}\in \mathcal{S}^n\backslash \{\bm{P}_e\}$ since a centroidal Voronoi configuration is optimal for $\mathcal{H}_{\mathcal{V}}$ among all other configurations (see Proposition 2.13 in \cite{bullo2009distributed}).
	This indicates that $V_1(\bm{P}(t))>0$ which makes it a valid Lyapunov function.
	
	The time derivative of \eqref{equ:ch9:barrierLyap} can be obtained using \eqref{equ:ch9dHdp}, \eqref{equ:ch9model} and the control law \eqref{equ:ch9u_unbounded} as:
	\begin{eqnarray}
			\dot{V}_1 &=& \frac{-1}{2}\sum_{i=1}^n \left(\dfrac{\partial \mathcal{H}_{\mathcal{V}}}{\partial \boldsymbol{p}_i}\right)^T \boldsymbol{u}_i \label{equ:ch8:barrierProof1} \\
			&=& - \sum_{i=1}^n (\tilde{\mathcal{C}}_{V_i} - \boldsymbol{p}_i)^T \bar{\boldsymbol{L}}_i (\tilde{\mathcal{C}}_{V_i} - \boldsymbol{p}_i)  \\
			&\leq& - \sum_{i=1}^n \lambda_{min}(\bar{\boldsymbol{L}}_i) \|\tilde{\mathcal{C}}_{V_i} - \boldsymbol{p}_i\|^2 \label{equ:ch8:barrierProof3}
	\end{eqnarray}
	where $\bar{\boldsymbol{L}}_i = M_{V_i}\bar{\boldsymbol{K}}_i$, and $\lambda_{min}(\bar{\boldsymbol{L}}_i)$ is the smallest eigenvalue of $\bar{\boldsymbol{L}}_i$.
	It is clear from \eqref{equ:ch8:barrierProof3} that $\dot{V}_1 < 0 \ \ \forall \bm{P}\in \mathcal{S}^n\backslash \{\bm{P}_e\}$ since $\bar{\boldsymbol{K}}_i$ is positive definite (i.e. $\lambda_{min}(\bar{\boldsymbol{L}}_i)>0$).
	Hence, the set of centroidal Voronoi configurations $\bm{P}_e$ is locally asymptotically stable, and $\lim\limits_{t\to \infty}\bm{p}_i=\tilde{\boldsymbol{C}}_{V_i},\ \forall i$.
	
	Similarly, for the bounded control law, the time derivative of $V_1(\bm{P}(t))$ can be obtained by substituting \eqref{equ:ch9u_bounded} into \eqref{equ:ch8:barrierProof1} as follows:
	\begin{eqnarray}
	\dot{V}_1 &=& - \sum_{i=1}^n (\tilde{\mathcal{C}}_{V_i} - \boldsymbol{p}_i)^T \boldsymbol{L}_i \tanh\Big(\gamma_i(\tilde{\boldsymbol{C}}_{V_i} - \boldsymbol{p}_i)\Big) \\
	&:=& - \sum_{i=1}^n \bm{e}_i^T \bm{L}_i \tanh(\gamma_i\bm{e}_i) \label{equ:ch8:barrierProofB2}
	\end{eqnarray}
	It is evident from \eqref{equ:ch8:barrierProofB2} that $\dot{V}_1 < 0 \ \ \forall \bm{P}\in \mathcal{S}^n\backslash \{\bm{P}_e\}$ since the hyperbolic tangent function is an odd function and $\gamma_i > 0$.
	This implies that $\bm{P}_e$ is also locally asymptotically stable under the application of the bounded control law \eqref{equ:ch9u_bounded}.
	
	Therefore, the control law \eqref{equ:ch9u_unbounded} (or \eqref{equ:ch9u_bounded}) guarantees that the vehicles will converge to centroidal Voronoi configurations which maximizes the sensing performance over the barrier $\mathcal{B}$ according to Proposition~2.13 in \cite{bullo2009distributed}.
	Furthermore, assumptions~\ref{assmp:ch9:Gconnected}-\ref{assmp:ch9:initialP} ensures that all vehicles can compute the centroids of their Voronoi cells in a distributed fashion at all times with no overlapping following the algorithm in \textbf{S1}-\textbf{S4}.
	This completes the proof.
\end{proof}

Note that the proposed control law ensures that the vehicles will converge to centroidal Voronoi configurations generated by their projections onto $\mathcal{B}$.
Thus, all the vehicles will eventually reach $\mathcal{B}$ such that $\lim\limits_{t\to\infty}\bm{p}_i=\tilde{\boldsymbol{C}}_{V_i}\in \mathcal{B}$ even if they are initially deployed at some positions $\bm{p}_i(0) \notin \mathcal{B}$.
Additionally, the trajectory of each vehicle remains within its Voronoi region which does not intersect with any other regions by definition.
This guarantees that vehicles motions are collision-free using the proposed control laws.

\subsection{Sweep Coverage Control Design}

\Cref{thm:ch9:barrier} shows that the proposed control laws can force the vehicles to reach a specified region within the 3D space (i.e. the barrier) and constrain their motion within that region.
This is extended in this section to address the sweep coverage control problems.
It can be achieved by enforcing vehicles to deploy over some dynamical "virtual" region whose motion is determined by the group.

In general, motion coordination control laws for coverage problems needs to satisfy the following objectives:
\begin{itemize}
	\item[$\bm{O1}$] Avoid collisions with other vehicles while maintaining a certain formation as a group 
	\item[$\bm{O2}$] Avoid collisions with obstacles within the environment
	\item[$\bm{O3}$] Achieve optimal coverage of the targeted environment $\mathcal{S}$ collaboratively 
\end{itemize}
The proposed sweeping algorithm targets these objective on two levels.
At a lower level, the vehicles motions are constrained within a dynamic "sweeping "region $\mathcal{F}(t)\subset\mathcal{S}$, and they maintain an optimal formation over that region for maximal sensing.
This achieves objective $\bm{O1}$.
At a higher level, decisions can be made in real-time collaboratively by the vehicles to decide the dynamics of $\mathcal{F}(t)$ (i.e $\mathcal{\dot{F}}(t)$).
Note that it is also possible to apply deformations to $\mathcal{F}(t)$ (will be shown in simulations) as long as the deformed region is large enough for the vehicles to distribute over with safe spacing. 
This provides a good way in addressing objectives $\bm{O2}$ and $\bm{O3}$.
In particular, the trajectory of $\mathcal{F}(t)$ will result in sweeping the whole environment $\mathcal{S}$ providing optimal coverage.
Moreover, obstacle avoidance can be achieved by only changing the dynamics of $\mathcal{F}(t)$ rather than having each vehicle reacting independently to obstacles.

For $\boldsymbol{r}\in \mathcal{F}(t)$, the motion dynamics of the sweeping region can be described as follows:
\begin{equation}\label{equ:ch9Sdynamics}
	\dot{\boldsymbol{r}} = \boldsymbol{g}(t,\boldsymbol{r}), \ \ \|\boldsymbol{g}(t,\boldsymbol{r})\| \leq \bar{g}, \ \ \|\dot{\bm{g}}(t,\boldsymbol{r})\| \leq \bar{\bar{g}}
\end{equation}
where $\boldsymbol{g}(t,\boldsymbol{r})\in\R^3$ is a desired velocity profile, and the following conditions must hold:
\begin{equation}\label{equ:ch9Vcondition}
	\bar{g} \leq \bar{u}
\end{equation}
In other words, the speed of $\mathcal{F}(t)$ should not be larger than the maximum physical speed that can be achieved by the vehicles.
Note that $\boldsymbol{g}(t,\boldsymbol{r})\in\R^3$ can have any direction.
However, for simplicity in achieving sweeping coverage, it assumed that $\mathcal{F}(t)$ is a planar region, and it is defined similar to $\mathcal{B}$ (see section~\ref{sec:ch9:centroidsVor}) at different time instants.
A simple example is moving $\mathcal{F}(t)$ in the direction of its normal (i.e. $\vect{a}_3$) with a constant \textit{sweeping speed} $g_0>0$ such that:
\begin{equation}\label{equ:ch9S_vel}
\boldsymbol{g}(t,\boldsymbol{r}) = g_0 \vect{a}_3
\end{equation}
More complex movements can be achieved depending on the considered environment shape and nearby obstacles.
One can also adopt a 3D holonomic or non-holonomic model for \eqref{equ:ch9Sdynamics} utilizing the available literature in obstacle for these models.
Generally, obstacle avoidance can be achieved either by varying the dynamics in \eqref{equ:ch9Sdynamics} or by dynamically deforming $\mathcal{F}(t)$.
Simulation cases showing both approaches will be shown later.
Note that the distributed behaviour of the proposed algorithm is maintained in all these cases since information about $\boldsymbol{g}(t,\boldsymbol{r})$ are exchanged over the connected network following assumption~\ref{assmp:ch9:Gconnected}.
At this point, we will leave out the design of \eqref{equ:ch9Sdynamics} to a high-level controller shared among the vehicles while assuming the following:
\begin{assumption}\label{assm:ch9:sweep}
	The sweeping plane $\mathcal{F}(t)$ remains all the time within the sensing environment (i.e. $\mathcal{F}(t)\subset \mathcal{S},\ \ \forall t>0$), and its movement governed by \eqref{equ:ch9Sdynamics} will completely span the volume of the sensing environment $\mathcal{S}$.
\end{assumption}

The main results of this section can now be presented.
\begin{theorem}\label{thm:ch9:sweeping_unb}
	Consider a multi-UAV system of size $n$ where each vehicle's motion model is represented by \eqref{equ:ch9model}.
	The control law \eqref{equ:ch9u_unbounded} along with the algorithm in \textbf{S1}-\textbf{S4} defined for a sweeping region $\mathcal{F}(t)$ whose dynamics is governed by \eqref{equ:ch9Sdynamics} solves the 3D sweeping coverage problem defined in problem~\ref{prob:ch9sweep} under assumptions~\ref{assmp:ch9:Gconnected}-\ref{assm:ch9:sweep} and the condition \eqref{equ:ch9Vcondition}.
\end{theorem}

\begin{proof}
	Let
	\begin{equation}\label{equ:ch9:errors}
		\bm{e}_{C_i} = \tilde{\bm{C}}_{V_i} - \bm{p}_i \to \dot{\bm{e}}_{C_i} = \dot{\tilde{\bm{C}}}_{V_i} - \bm{u}_i
	\end{equation}
	where $\bm{u}_i(t)$ is given by \eqref{equ:ch9u_unbounded}.
	Now, define a Lyapunov candidate function as follows:
	\begin{equation}\label{equ:ch9Lyap_thm2}
	\begin{aligned}
	V_2(t)&:= V_2(\bm{e}_{C_1}(t),\cdots, \bm{e}_{C_n}(t),\dot{\bm{e}}_{C_1}(t),\cdots, \dot{\bm{e}}_{C_n}(t)) \\
	&= V_1 + \frac{1}{2} \sum_{i=1}^n \Big(\bm{e}_{C_i}^T \bm{e}_{C_i} + \dot{\bm{e}}_{C_i}^T \dot{\bm{e}}_{C_i} \Big)
	\end{aligned}
	\end{equation}
	where $V_1$ is defined in \eqref{equ:ch9:barrierLyap}.
	The choice \eqref{equ:ch9Lyap_thm2} guarantees that $V_2(t)>0$ for
	\begin{equation*}
		(\bm{e}_{C_1}(t),\cdots, \bm{e}_{C_n}(t),\dot{\bm{e}}_{C_1}(t),\cdots, \dot{\bm{e}}_{C_n}(t)) \neq (0,\cdots,0),
	\end{equation*}
	and $V_2(0,\cdots,0) = 0$ is also true.
	
	The time derivative of \eqref{equ:ch9Lyap_thm2} is obtained using \eqref{equ:ch9model} and \eqref{equ:ch9u_unbounded} as follows:
	\begin{align}
	\dot{V}_2 &= \dot{V}_1 + \sum_{i=1}^n \Big(\bm{e}_{C_i}^T \dot{\bm{e}}_{C_i} + \dot{\bm{e}}_{C_i}^T \ddot{\bm{e}}_{C_i} \Big) \nonumber\\
	&= \dot{V}_1 + \sum_{i=1}^n \Big(\bm{e}_{C_i}^T (\dot{\tilde{\bm{C}}}_{V_i} - \bar{\bm{K}}_i\bm{e}_{C_i}) + \dot{\bm{e}}_{C_i}^T (\ddot{\tilde{\bm{C}}}_{V_i} - \bar{\bm{K}}_i\dot{\bm{e}}_{C_i}) \Big) \nonumber\\ 
	&\leq \dot{V}_1 + \sum_{i=1}^n \Big(\|\dot{\tilde{\bm{C}}}_{V_i}\|\|\bm{e}_{C_i}\| - \lambda_{min}(\bar{\bm{K}}_i)\|\bm{e}_{C_i}\|^2 + \bar{\bar{g}} \|\dot{\bm{e}}_{C_i}\| - \lambda_{min}(\bar{\bm{K}}_i) \|\dot{\bm{e}}_{C_i}\|^2 \Big) \label{equ:ch9H1_dot}
	\end{align}
	Recall that $\dot{V}_1<0$ which was established in \eqref{equ:ch8:barrierProofB2}.
	Therefore, the time derivative of $V_2$ is negative outside the compact set $B_{\Gamma} = B_{\Gamma_1}\cup \cdots \cup B_{\Gamma_n}$ where $B_{\Gamma_i}=\{\bm{e}_{C_i}\in\R^3, \dot{\bm{e}}_{C_i}\in\R^3 : \|\bm{e}_{C_i}\| > \Gamma_{i,1}, \|\dot{\bm{e}}_{C_i}\| > \Gamma_{i,2} \}$ and $\Gamma_{i,1}$ and $\Gamma_{i,2}$ are given by:
	\begin{equation}
		\begin{array}{lr}
		\Gamma_{i,1} = \frac{\|\dot{\tilde{\bm{C}}}_{V_i}\|}{\lambda_{min}(\bar{\boldsymbol{K}}_i)}, & \Gamma_{i,2} = \frac{\bar{\bar{g}}}{\lambda_{min}(\bar{\boldsymbol{K}}_i)}
		\end{array}
	\end{equation}
	which represent closed balls with radius $\Gamma_{i,1}$ and $\Gamma_{i,2}$ respectively.
	Thus, starting from any initial condition outside the set $B_{\Gamma}$, the errors will converge to the closed set $B_{\Gamma}$ in finite time and stay there forever.
	That is,
	\begin{equation}
	\begin{array}{lr}
	\limsup\limits_{t\to \infty} \|\tilde{\bm{C}}_{V_i} - \bm{p}_i\| \leq \Gamma_{i,1}, & \limsup\limits_{t\to \infty} \|\dot{\tilde{\boldsymbol{C}}}_{V_i} - \boldsymbol{u}_i\| \leq \Gamma_{i,2}
	\end{array}
	\end{equation}
	which means that the errors will be uniformly ultimately bounded with respect to $B_{\Gamma_i}$. 
	Moreover, the radius of $B_{\Gamma_i}$ can be made arbitrary small by increasing $\bar{\boldsymbol{K}}_i$ such that $\lambda_{min}(\bar{\bm{K}}_i) \gg \bar{g}$ and $\lambda_{min}(\bar{\bm{K}}_i) \gg \bar{\bar{g}}$.
	Furthermore, since $\tilde{\bm{C}}_{V_i} \in \mathcal{F}(t)$, its derivative follows \eqref{equ:ch9Sdynamics} (i.e. $\limsup\limits_{t\to \infty} \| \bm{g}(t,\tilde{\boldsymbol{C}}_{V_i}) - \bm{u}_i \| \leq \Gamma_{i,2}$).
	Hence, all the vehicles will converge to their centroidal Voronoi configurations over $\mathcal{F}(t)$ and follow its trajectory associated with \eqref{equ:ch9Sdynamics}.
	According to \cref{assm:ch9:sweep}, the movement of the multi-UAV along the trajectory of $\mathcal{F}(t)$ solves the sweeping coverage problem.
	This completes the proof.
\end{proof}

\begin{theorem}\label{thm:ch9:sweeping_bounded}
	Consider a multi-UAV system of size $n$ where each vehicle's motion model is represented by \eqref{equ:ch9model}.
	Also, consider using the algorithm in \textbf{S1}-\textbf{S4} defined for a sweeping region $\mathcal{F}(t)$ whose dynamics is governed by \eqref{equ:ch9Sdynamics} so that the vehicles can compute their centroidal Voronoi configurations.
	The bounded control law \eqref{equ:ch9u_bounded} solves the 3D sweeping coverage problem defined in problem~\ref{prob:ch9sweep} under assumptions~\ref{assmp:ch9:Gconnected}-\ref{assm:ch9:sweep} and the condition \eqref{equ:ch9Vcondition}.
\end{theorem}

\begin{proof}
	Again, consider the errors definitions in \eqref{equ:ch9:errors}.
	We now define a Lyapunov candidate function as follows:
	\begin{equation}\label{equ:ch9Lyap_thm3}
		\begin{aligned}
		V_3(t)&:= V_3(\bm{e}_{C_1}(t),\cdots, \bm{e}_{C_n}(t),\dot{\bm{e}}_{C_1}(t),\cdots, \dot{\bm{e}}_{C_n}(t)) \\
		&= V_1 + \sum_{i=1}^n \Big(\frac{1}{\gamma_i} [1\ 1\ 1]^T \log(\cosh(\gamma_i\bm{e}_{C_i})) + \frac{1}{2} \dot{\bm{e}}_{C_i}^T \dot{\bm{e}}_{C_i} \Big)
		\end{aligned}
	\end{equation}
	where $\log(\bm{v})\in\R^3$ and $\cosh(\bm{v})\in\R^3$ are defined element-wise for any vector $\bm{v}\in\R^3$, and $\bm{u}_i$ is defined using \eqref{equ:ch9u_bounded}.
	Also, let $\text{Sech}^2(\bm{v}): \R^3 \to \R^{3\times3}$ be a mapping function, based on the hyperbolic secant function, which maps the vector $\bm{v}=[v_1, v_2, v_3]^T$ into a diagonal matrix as follows:
	\begin{equation}
	\text{Sech}^2(\bm{v}) = \left[\begin{array}{ccc}
	\text{sech}^2(v_1) & 0 & 0 \\
	0 & \text{sech}^2(v_2) & 0 \\
	0 & 0 & \text{sech}^2(v_3)
	\end{array}\right]
	\end{equation}
	The time derivative of $V_3$ can then be obtained as follows:
	\begin{align}
		\dot{V}_3 &= \dot{V}_1 + \sum_{i=1}^n \Big(\tanh(\gamma_i\bm{e}_{C_i})^T (\dot{\tilde{\bm{C}}}_{V_i} - \bm{K}_i\tanh(\gamma_i\bm{e}_{C_i})) + \dot{\bm{e}}_{C_i}^T (\ddot{\tilde{\bm{C}}}_{V_i} - \gamma_i\bm{K}_i \text{Sech}^2(\gamma_i\bm{e}_{C_i})\dot{\bm{e}}_{C_i} )  \Big) \nonumber\\ 
		&\leq \dot{V}_1 + \sum_{i=1}^n \Big(\|\dot{\tilde{\bm{C}}}_{V_i}\|\|\tanh(\gamma_i\bm{e}_{C_i})\| - \lambda_{min}(\bm{K}_i)\|\tanh(\gamma_i\bm{e}_{C_i})\|^2 + \|\ddot{\tilde{\bm{C}}}_{V_i}\| \|\dot{\bm{e}}_{C_i}^T\| - \gamma_i \lambda_{min}(\Lambda_i) \|\dot{\bm{e}}_{C_i}\|^2 \Big) \label{equ:ch9V3_dot}
	\end{align}
	where $\Lambda_i = \bm{K}_i \text{Sech}^2(\gamma_i\bm{e}_{C_i})$.
	
	Similar to the previous analysis, $\dot{V}_3$ is negative outside the compact set $D_{\Upsilon} = D_{\Upsilon_1}\cup \cdots \cup D_{\Upsilon_n}$ where $D_{\Upsilon_i}=\{\bm{e}_{C_i}\in\R^3, \dot{\bm{e}}_{C_i}\in\R^3 : \|\tanh(\bm{e}_{C_i})\| > \Upsilon_{i,1}, \|\dot{\bm{e}}_{C_i}\| > \Upsilon_{i,2} \}$ where $\Upsilon_{i,1}$ and $\Upsilon_{i,2}$ are given by:
	\begin{equation}
		\begin{array}{lr}
		\Upsilon_{i,1} = \frac{\|\dot{\tilde{\bm{C}}}_{V_i}\|}{\lambda_{min}(\boldsymbol{K}_i)}, & \Upsilon_{i,2} = \frac{\bar{\bar{g}}}{\gamma_i\lambda_{min}(\bar{\boldsymbol{K}}_i)}
		\end{array}
	\end{equation}
	Thus, the system trajectories will converge to the set $D_{\Upsilon}$ starting from any initial condition, and stay there forever leading to the following:
	\begin{equation}
	\begin{array}{lr}
	\limsup\limits_{t\to \infty} \|\tanh(\tilde{\bm{C}}_{V_i} - \bm{p}_i)\| \leq \Upsilon_{i,1}, & \limsup\limits_{t\to \infty} \|\dot{\tilde{\boldsymbol{C}}}_{V_i} - \boldsymbol{u}_i\| \leq \Upsilon_{i,2}
	\end{array}
	\end{equation}
	Therefore, the tracking errors are uniformly ultimately bounded with respect to $D_{\Upsilon_i}$.
	By choosing, $\lambda_{min}(\bm{K}_i) \gg \bar{g}$ and $\gamma_i\lambda_{min}(\bm{K}_i) \gg \bar{\bar{g}}$, the errors can be made arbitrarily small.
	
	Moreover, $\tilde{\bm{C}}_{V_i} \in \mathcal{F}(t) \to \dot{\tilde{\bm{C}}}_{V_i}=g(t, \tilde{\bm{C}}_{V_i})$ according to \eqref{equ:ch9Sdynamics} and \textbf{S1}-\textbf{S4}.
	By assumption~\ref{assm:ch9:sweep}, the multi-UAV system solves the sweeping coverage problem which completes the proof.
\end{proof}

\begin{remark}
	Note that the proposed methods can be extended to consider 3D barriers $\mathcal{B}$ and sweeping regions $\mathcal{F}(t)$ (i.e. non planar). 
	However, a computationally efficient way to compute 3D Voronoi centroids needs to be considered instead of the suggested approach in section~\ref{sec:ch9:centroidsVor}. 
\end{remark}

\section{Validation \& Discussion}\label{sec:ch9simulation}

Simulations were carried out to validate the performance of the proposed 3D coverage control laws in \eqref{equ:ch9u_unbounded} and \eqref{equ:ch9u_bounded} using the algorithm in \textbf{S1}-\textbf{S4} to compute the centroidal Voronoi configurations.
Additionally, more simulation cases were performed to demonstrate the robustness of the proposed method and how obstacle avoidance can be incorporated within the overall framework.
The following subsections provide details of these simulations and the obtained results.
\subsection{Simulation Cases 1-4: Performance Validation}

In the first set of simulations, 
a multi-UAV system of size $n=20$ was used.
All vehicles have been initially deployed to random locations in some predefined region $\mathcal{I}=\Big\{(x_i,y_i,z_i):0\leq x_i,y_i \leq 4, 0 \leq z_i \leq 2,\ \forall i=\{1,\cdots,20\}\Big\}$.
The control design parameters were chosen to be $\boldsymbol{K}_i=diag\{2.5,0.5,0.5\}$ and $\bar{\boldsymbol{K}}_i=diag\{0.3,0.3,0.3\}$ for all vehicles.

The 3D barrier coverage problem was considered in the first two simulation cases where the unbounded control law \eqref{equ:ch9u_unbounded} and the bounded control law \eqref{equ:ch9u_bounded} are used.
The goal was to achieve optimal coverage over a barrier region defined according to $\mathcal{B}=\{(x,y,z):x=20,\ 0\leq(y,z)\leq 10\}$.
The obtained results for this case are presented in \cref{fig:ch9sim1,fig:ch9sim2,fig:ch9sim3,fig:ch9sim4}.
\Cref{fig:ch9sim1,fig:ch9sim3} show the complete trajectories taken by all UAVs and their final locations which are optimally distributed over the barrier $\mathcal{B}$ (rectangular area highlighted in yellow).
The multi-UAV system reaches the centroidal Voronoi configurations and all vehicles form a static arrangement over $\mathcal{B}$.
Note that the shape of $\mathcal{B}$ and the number of UAVs were chosen arbitrarily just as a proof of concept.
However, in practical applications, the number of UAVs and the design of the barrier can be considered as a design problem which depends on the UAVs sensing and communications capabilities.
Having a larger number of vehicles may result in some overlapping between sensors field-of-view (FOV).
On the other hand, It may not be possible to completely cover $\mathcal{B}$ using lower number of UAVs than what is needed (i.e. the combined FOV of all sensors is less than the size (area/volume) of $\mathcal{B}$.
Overall, the sensing performance will be maximized using the developed strategy.
It can also be seen that the resultant trajectories are collision-free.

The time evolution of control inputs for all vehicles are shown in \cref{fig:ch9sim2,fig:ch9sim4} for both the bounded and unbounded control laws respectively.
Using \eqref{equ:ch9u_unbounded} would require choosing the controller gain matrix $\bar{\bm{K}}_i$ properly in order to satisfy the constraint \eqref{equ:ch9:maxInput}.
This depends mostly on how far the vehicle is from its centroidal Voronoi configuration initially which indicates that tuning $\bar{\bm{K}}_i$ could not be ideal in practice.
Alternatively, using \eqref{equ:ch9u_bounded} provides an easier way of choosing $\bm{K}_i$ to ensure that the physical limit on the vehicles velocity is respected (i.e. \eqref{equ:ch9:maxInput} is satisfied).
This can be clearly seen in \cref{fig:ch9sim4} where the vehicles speed remains constant for the first 8 seconds until the barrier is reached at which the velocities drop down to zero, and the vehicles become statically distributed over $\mathcal{B}$.
The upper bounds on $||\boldsymbol{u}_i||$ in this case was $u_{i,max}= 2.6\ m/s\ \forall i$ which is in accordance with \eqref{equ:ch9u_bounds}.

In the next two simulation cases, the sweeping coverage problem was considered.
The task was to completely scan a 3D sensing region $\mathcal{S}$ which was defined as $\mathcal{S} = \mathcal{S}_1 \cup \mathcal{S}_2 \cup \mathcal{S}_2$ where
\begin{align*}
\mathcal{S}_1 &= \{(x,y,z):\ 10 \leq x \leq 60,\ 0 \leq y \leq 10,\ 0 \leq z \leq 5\} \\
\mathcal{S}_2 &= \{(x,y,z):\ 10 \leq x \leq 60,\ 0 \leq y \leq 2.5,\ 5 \leq z \leq 10\} \\
\mathcal{S}_3 &= \{(x,y,z):\ 10 \leq x \leq 60,\ 7.5 \leq y \leq 10,\ 5 \leq z \leq 10\}
\end{align*}
Based on the environment shape, an initial sweeping plane $\mathcal{F}(t=0)$ was determined by the multi-UAV system as $\mathcal{F}(0)=\{(x,y,z)\in \mathcal{S}:x=10\}$.
The dynamics of $\mathcal{F}(t)$ was also considered to be the simplest case as in \eqref{equ:ch9S_vel} where the sweeping plane is moving with a constant speed of $g_0=1.5\ m/s$ in a progressive direction that can result in a sweeping behaviour as will be shown.
In more complex cases, some other patterns could be adopted for moving $\mathcal{F}(t)$ as a higher level decision making which can still be done in a distributed manner as the vehicles can exchange information over the connected network.
For example, the literature on coverage path planning for single-vehicle systems can be utilized in this case.

\Cref{fig:ch9sim5,fig:ch9sim6} shows the results for the sweeping coverage problem when using \eqref{equ:ch9u_unbounded}, and the results obtained when applying \eqref{equ:ch9u_bounded} are shown in \cref{fig:ch9sim7,fig:ch9sim8}.
For both cases, the vehicles move from their initial positions to quickly deploy over $\mathcal{F}(t)$ reaching their centroidal Voronoi configurations.
As the position of $\mathcal{F}(t)$ evolves over time according to \eqref{equ:ch9S_vel}, the vehicles corresponding centroidal Voronoi configurations evolve accordingly since $\tilde{\bm{C}}_{V_i} \in \mathcal{F}(t)$.
Hence, the vehicles will start to move with the same velocity as of $\mathcal{F}(t)$.
This can be clearly seen from \cref{fig:ch9sim6,fig:ch9sim8}.
You can see that the vehicles starts with a higher speed to reach the moving sweeping plane and achieve optimal distribution.
Once this is achieved (around $t=10s$), the vehicles are no longer moving within $\mathcal{F}(t)$ at which all velocities converge to $g_0=1.5 m/s$ in the direction of the sweeping plane's movement.
It is important to notice that the speed of $\mathcal{F}(t)$ (i.e. $g_0$) should be slower than the maximum velocity achievable by any vehicle ($\bar{u}\geq g_0$).
After scanning the desired region $\mathcal{S}$, the sweeping plane $\mathcal{F}(t)$ becomes static which reflects on all the vehicles as can be seen from the results where all velocities converge to 0. 
The complete trajectories of the vehicles along with the scanned 3D volume (highlighted in yellow) are shown in \cref{fig:ch9sim5,fig:ch9sim7} which confirms that the sweeping coverage problem is achieved over $\mathcal{S}$.
These results clearly validate the performance of the proposed 3D coverage control strategy.

\begin{figure}[!htb]
	\centering
	\includegraphics[width=0.75\linewidth]{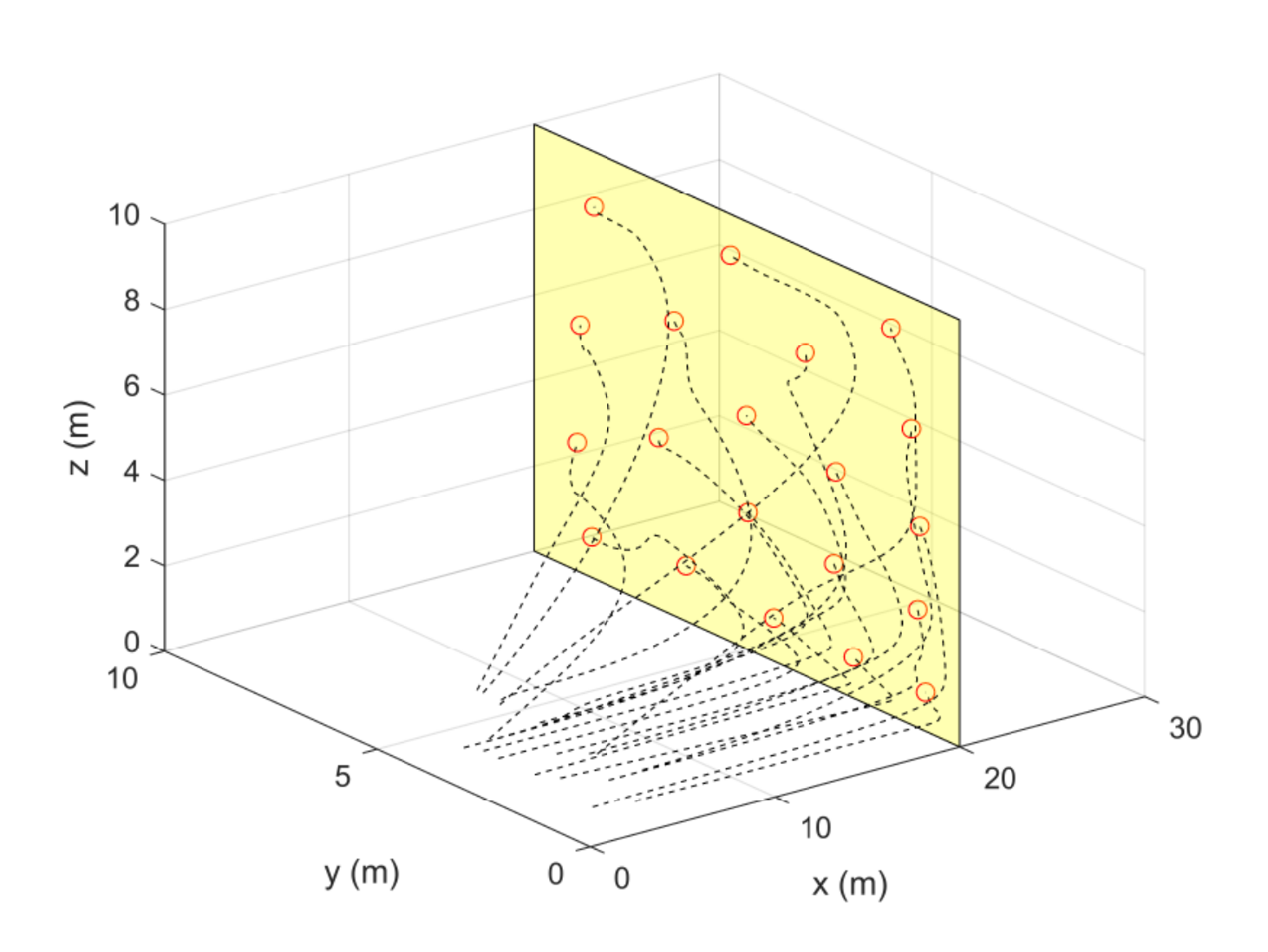} 
	\caption{UAVs trajectories for the barrier coverage case using \eqref{equ:ch9u_unbounded} (Simulation Case 1)} \label{fig:ch9sim1}
\end{figure}

\begin{figure}[!htb]
	\centering
	\includegraphics[width=0.65\linewidth]{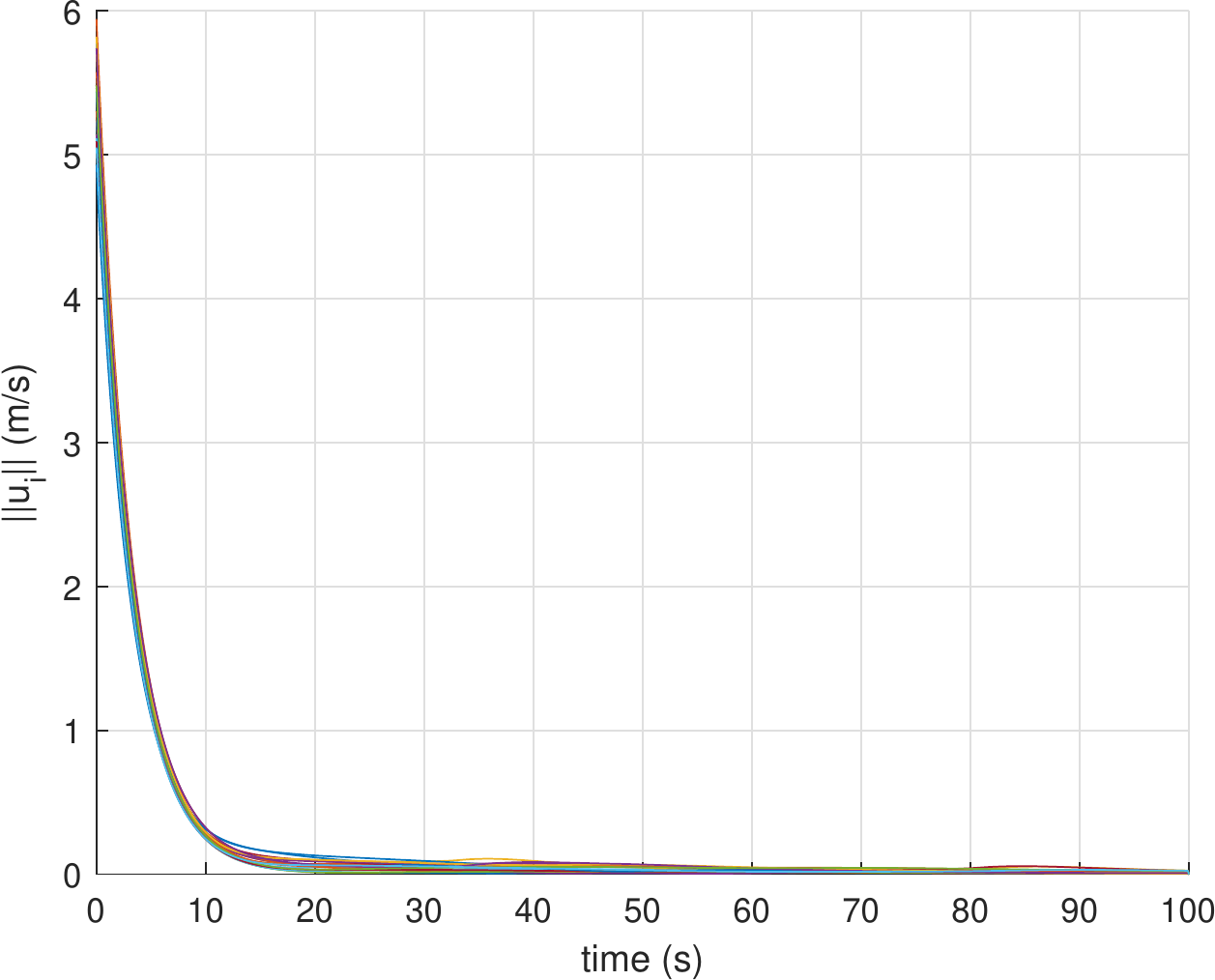} 
	\caption{Norms of the control inputs applied to the UAVs for the barrier coverage case using \eqref{equ:ch9u_unbounded} (Simulation Case 1)} \label{fig:ch9sim2}
\end{figure}

\begin{figure}[!htb]
	\centering
	\includegraphics[width=0.75\linewidth]{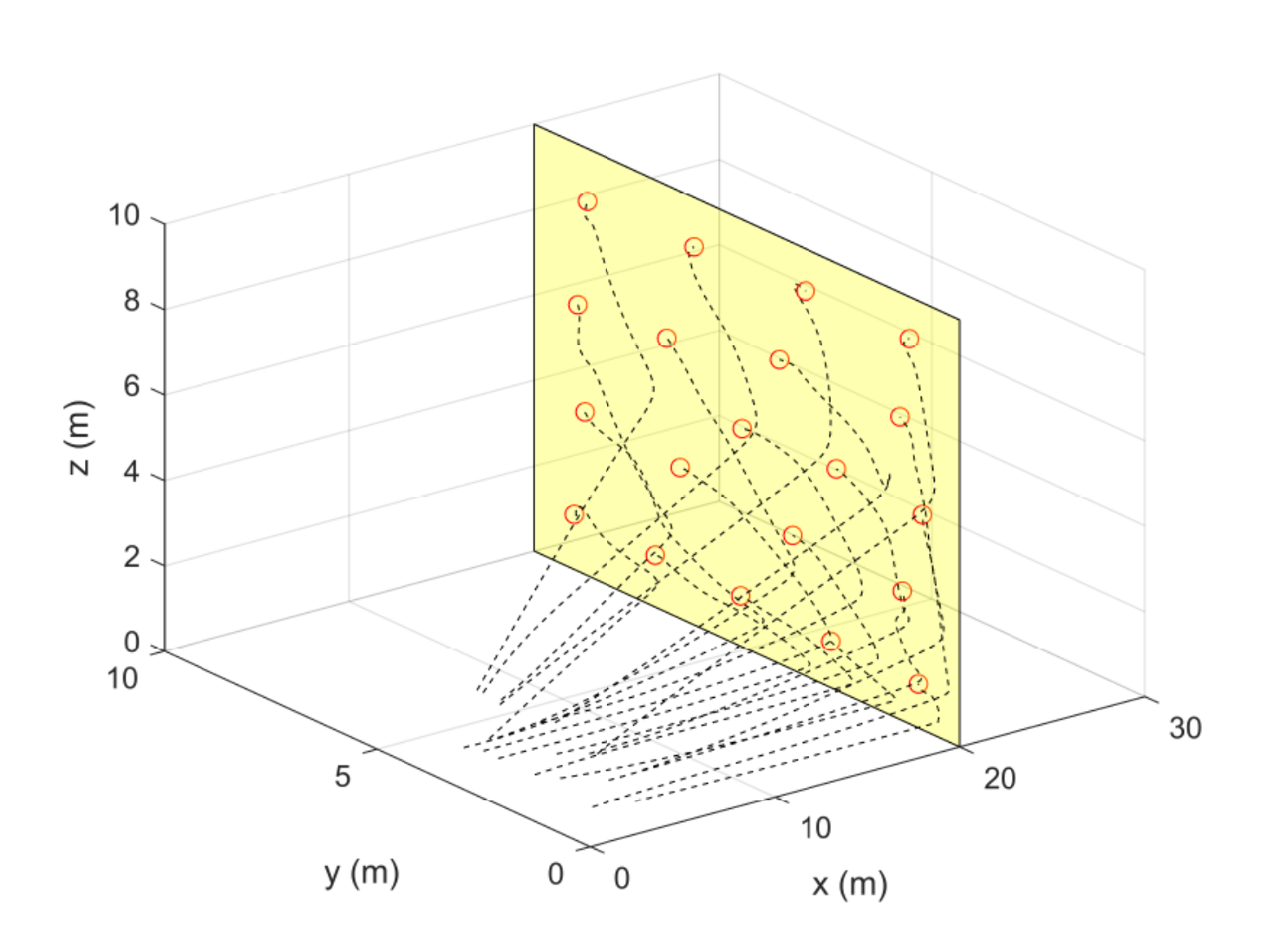} 
	\caption{Sensors trajectories for a barrier coverage case using bounded control laws \eqref{equ:ch9u_bounded} (Simulation Case 2)} \label{fig:ch9sim3}
\end{figure}

\begin{figure}[!htb]
	\centering
	\includegraphics[width=0.65\linewidth]{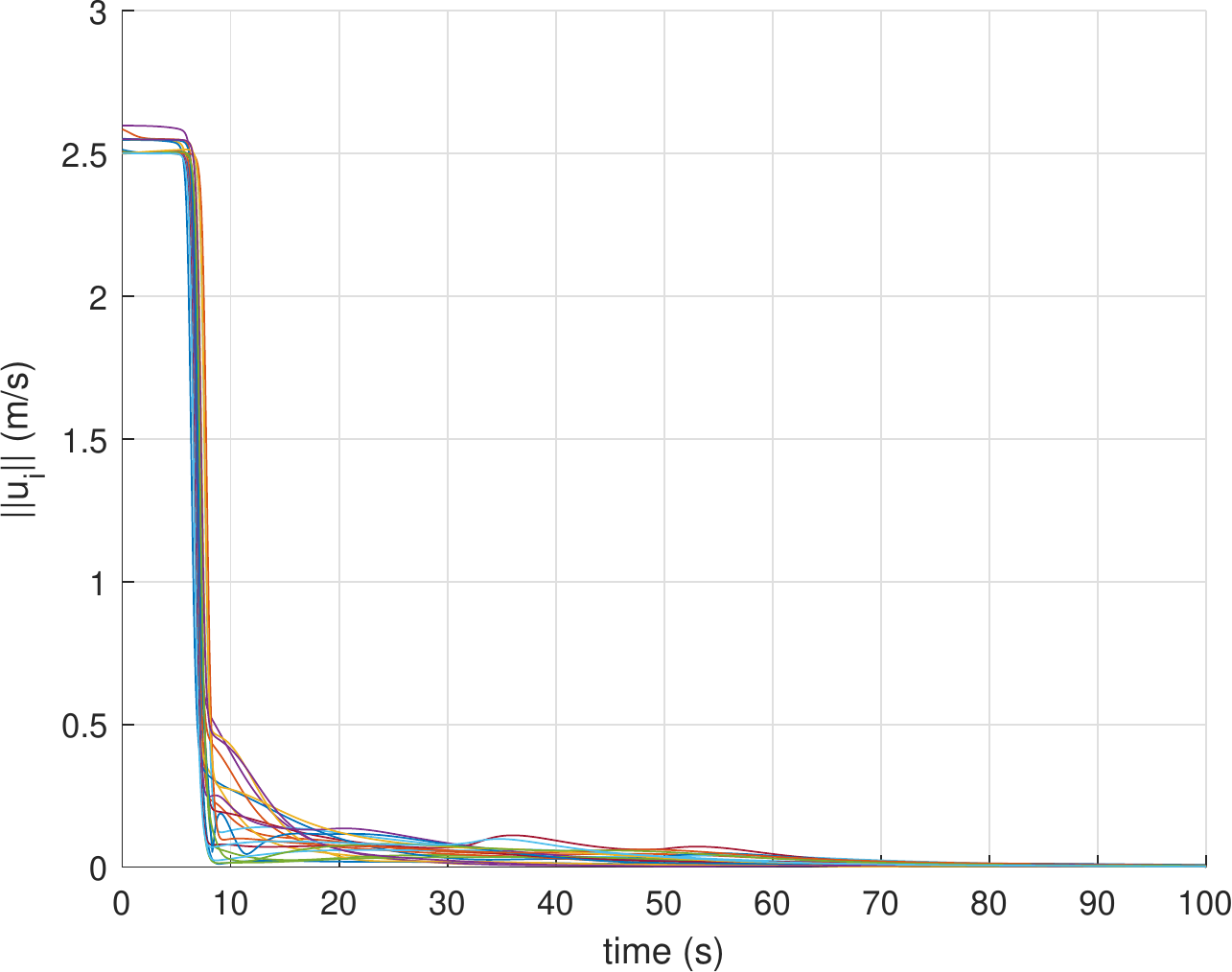} 
	\caption{Sensors control inputs norms for a barrier coverage case using bounded control laws \eqref{equ:ch9u_bounded} (Simulation Case 2)} \label{fig:ch9sim4}
\end{figure}

\begin{figure}[!htb]
	\centering
	\includegraphics[width=0.75\linewidth]{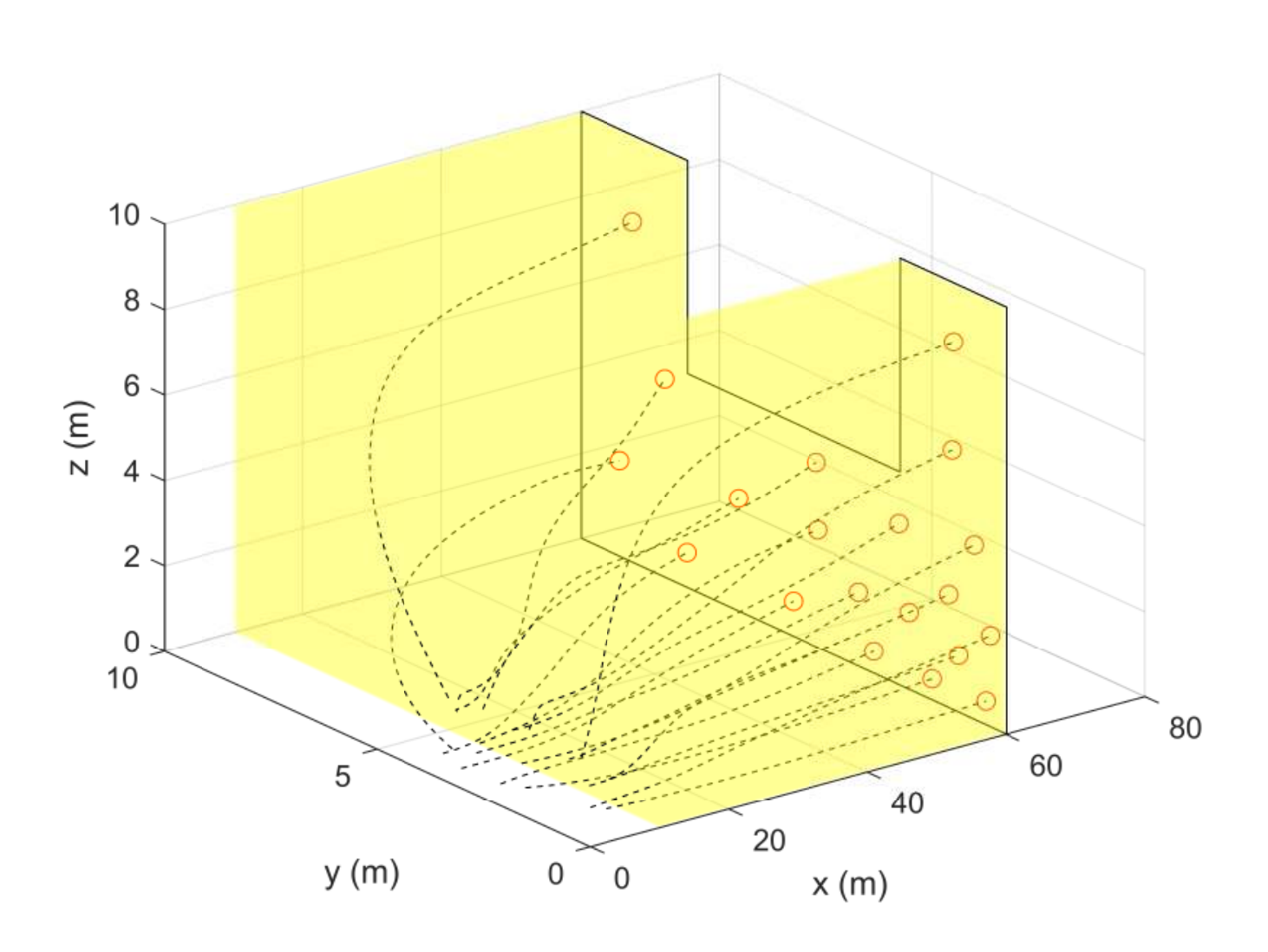} 
	\caption{Sensors trajectories for a sweeping coverage case (Simulation Case 3)} \label{fig:ch9sim5}
\end{figure}

\begin{figure}[!htb]
	\centering
	\includegraphics[width=0.65\linewidth]{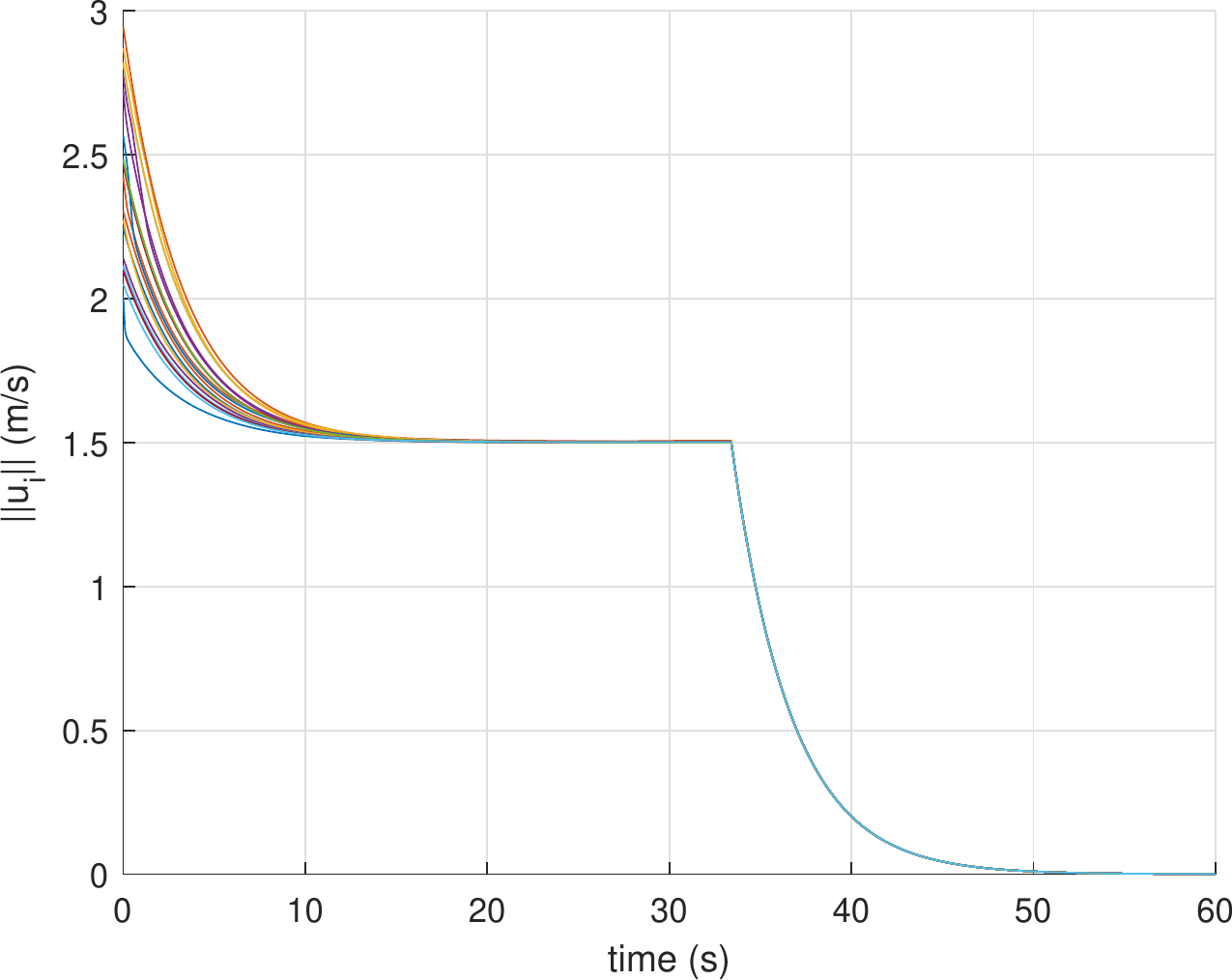} 
	\caption{Sensors control inputs norms for a sweeping coverage case using \eqref{equ:ch9u_unbounded} (Simulation Case 3)} \label{fig:ch9sim6}
\end{figure}

\begin{figure}[!htb]
	\centering
	\includegraphics[width=0.75\linewidth]{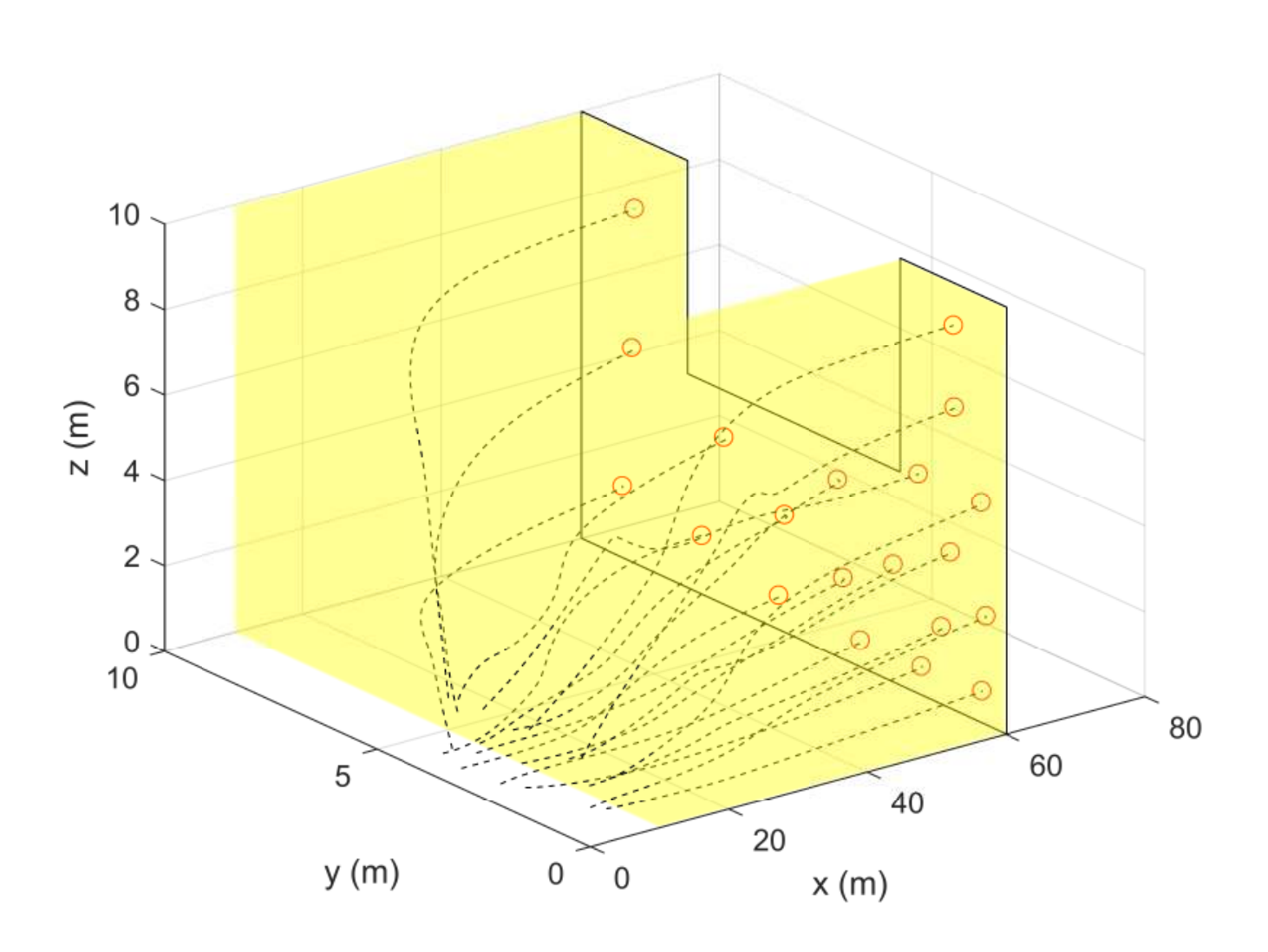} 
	\caption{Sensors trajectories for a sweeping coverage case using bounded control laws \eqref{equ:ch9u_bounded} (Simulation Case 4)} \label{fig:ch9sim7}
\end{figure}

\begin{figure}[!htb]
	\centering
	\includegraphics[width=0.65\linewidth]{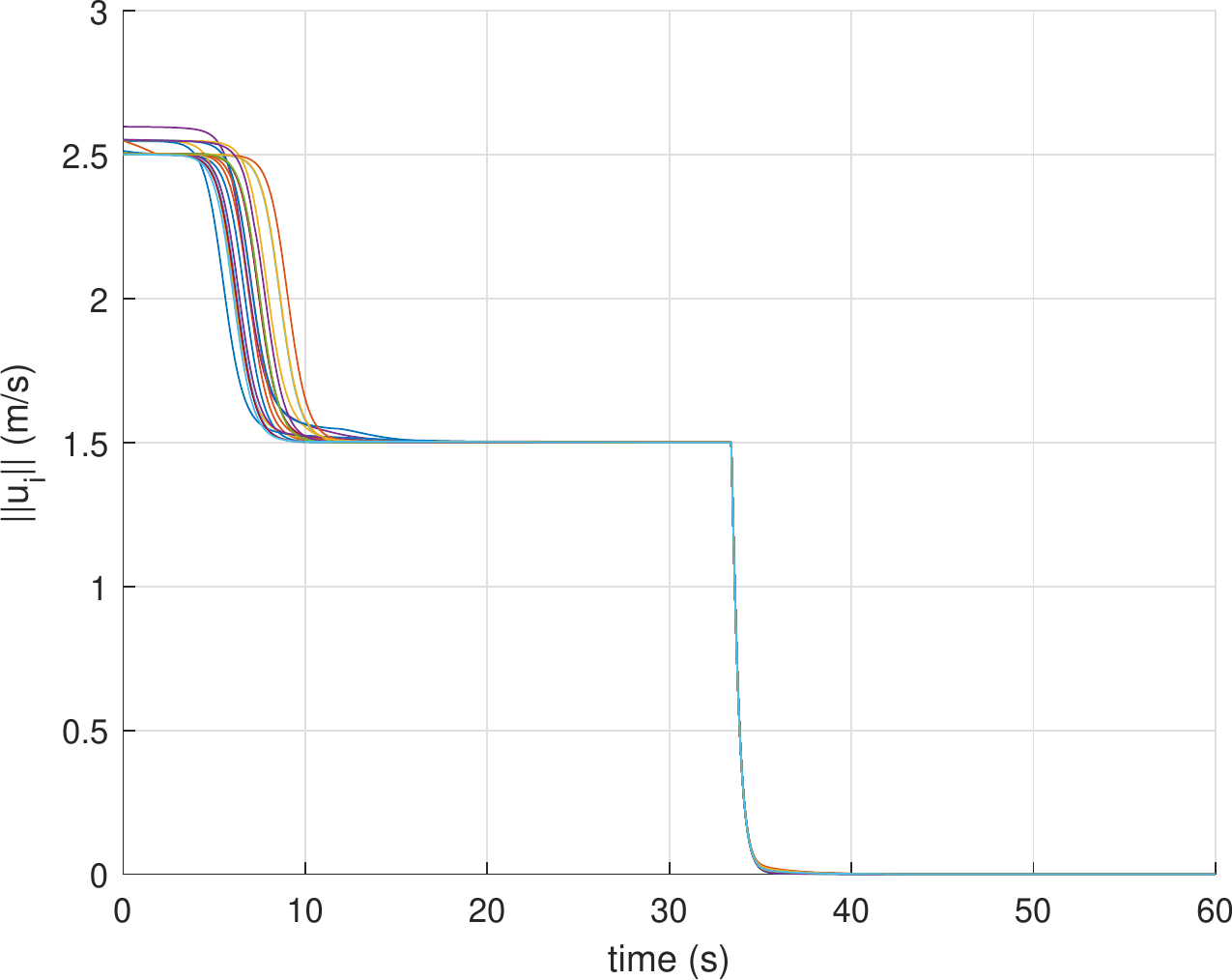} 
	\caption{Sensors control inputs norms for a sweeping coverage case using bounded control laws \eqref{equ:ch9u_bounded} (Simulation Case 4)} \label{fig:ch9sim8}
\end{figure}

\subsection{Simulation Case 5: Robustness}

Another simulation case was considered to show how the proposed method perform very well in situations where the number of active UAVs within the multi-UAV system change over time during a sweeping coverage mission.
For example, when a number of UAVs fail, the whole group should still be able to continue their coverage task as long as there is enough number of active UAVs to finish the mission.
This goes the same way when adding new vehicles to the system during the mission; however, this case was not considered here for brevity.

We considered a similar environment $\mathcal{S}$ and sweeping plane choice $\mathcal{F}(t)$ as in simulations cases 3 and 4.
Different time instants of the simulation are shown in \cref{fig:ch9:simRobustness1}.
The Voronoi regions generated by the UAVs over $\mathcal{F}(t)$ with their centroids are clearly highlighted to show how they change over time as the vehicles move.
Once centroidal Voronoi configurations are reached, the vehicles move acording to the plane's movement as was discussed earlier.
However, some vehicles fail and become inactive at certain time instants as in \cref{fig:ch9:simRobustness1}~(e,g,i,k).
Whenever this occurs, the remaining active vehicles quickly adjust their distribution over $\mathcal{F}(t)$ while still moving in accordance with $\dot{\mathcal{F}}(t)$.
For example, at $t=21s$, one vehicle fail as shown in \cref{fig:ch9:simRobustness1}~(e) which directly indicates a change of the Voronoi partition of $\mathcal{F}(t)$.
This results in a change of the centroidal Voronoi configurations which causes the vehicles to quickly adapt to the situation in a robust way by moving to the new centroidal Voronoi configurations.
It can be noticed that once a vehicle fail, only the vehicles in the neighborhood of that vehicle will be affected (i.e. their Voronoi regions will be extended).

Initially, the multi-UAV system had a size of $n=9$.
The complete collision-free trajectories of all vehicles are shown in \cref{fig:ch9:simRobustness2} where 4 vehicles have failed during the mission (inactive UAVs), and the coverage task was completed efficiently by the remaining 5 vehicles (active UAVs).
This clearly shows how robust and scalable our method can be.
It is also worth mentioning that changes to $\mathcal{F}(t)$ can be applied in real-time if its size becomes larger/smaller than what the remaining vehicles can cover based on their combined sensing FOV.
Such a decision can be autonomously made by the vehicles and shared among the connected network.
The next simulation case shows how the proposed control laws work when such changes to $\mathcal{F}(t)$ are applied which is really important when considering obstacle avoidance.

\begin{figure}[!htb]
	\centering
	\begin{adjustbox}{minipage=\linewidth,scale=1.0}
		\begin{subfigure}[t]{0.32\textwidth}
			\centering
			\includegraphics[clip, width=\linewidth]{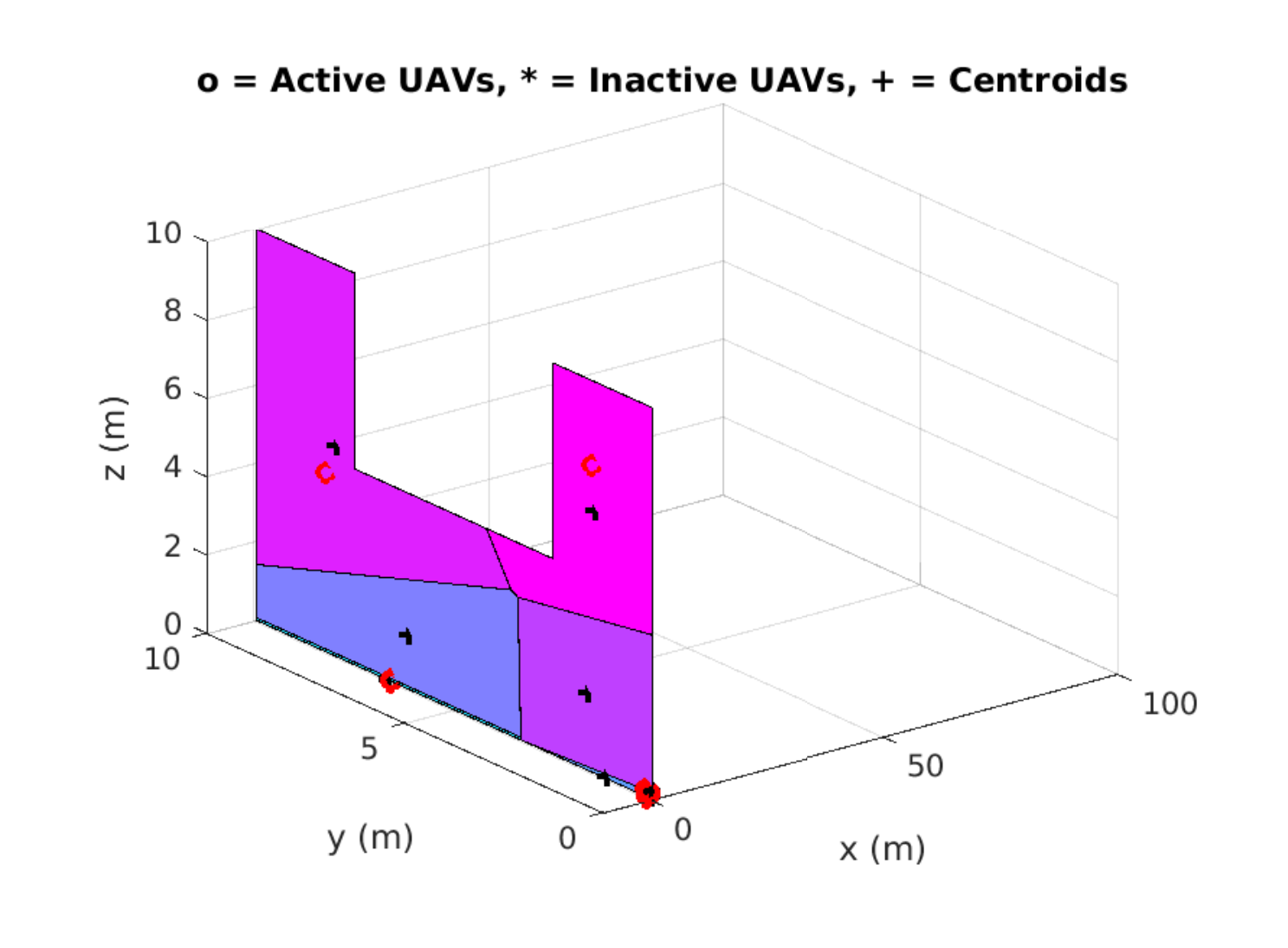} 
			\caption{t=3s}
		\end{subfigure}
		\hfill
		\begin{subfigure}[t]{0.32\textwidth}
			\centering
			\includegraphics[trim={0 0 0 0}, clip, width=\linewidth]{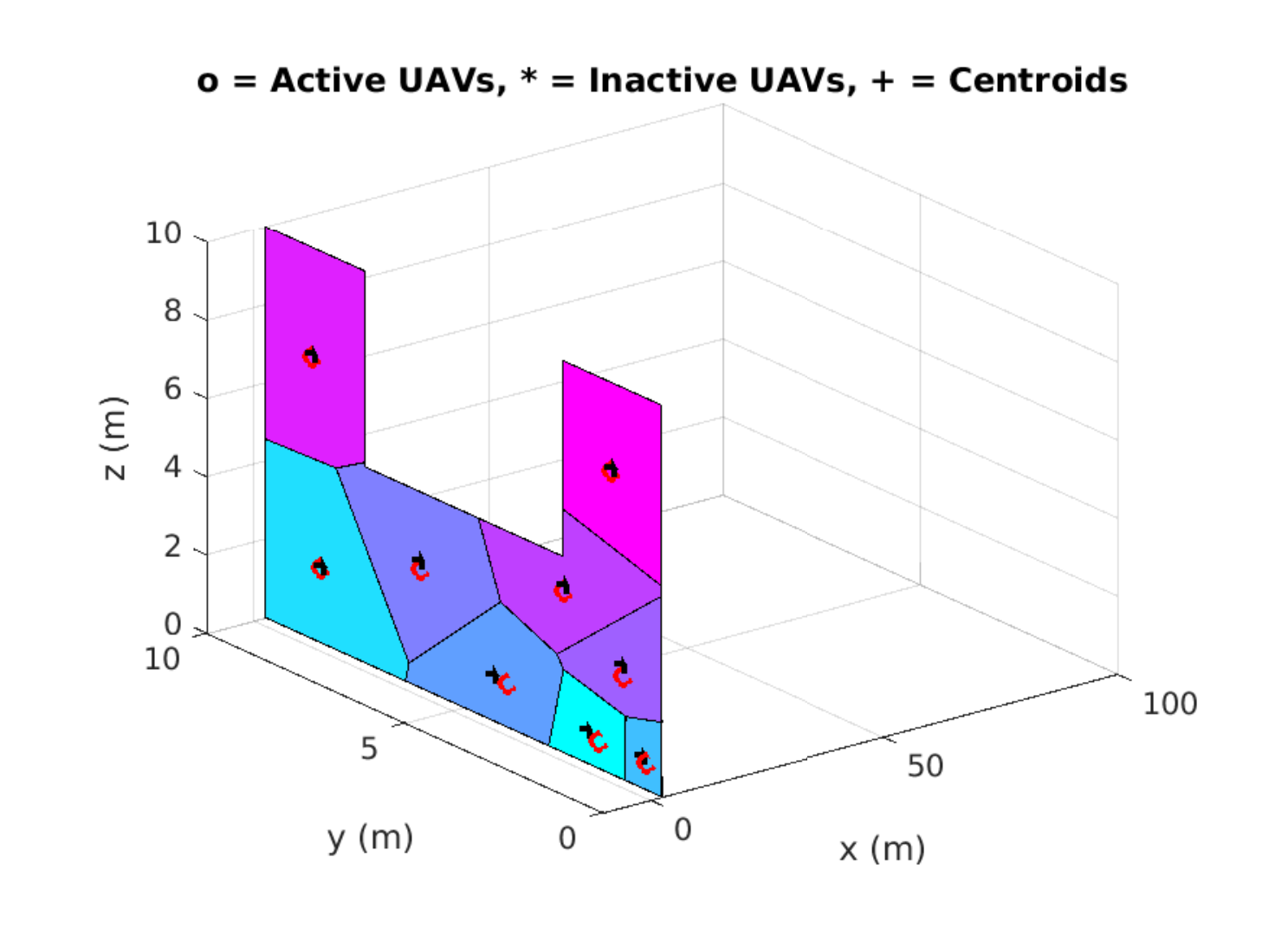} 
			\caption{t=7s}
		\end{subfigure}
		\hfill
		\begin{subfigure}[t]{0.32\textwidth}
			\centering
			\includegraphics[clip, width=\linewidth]{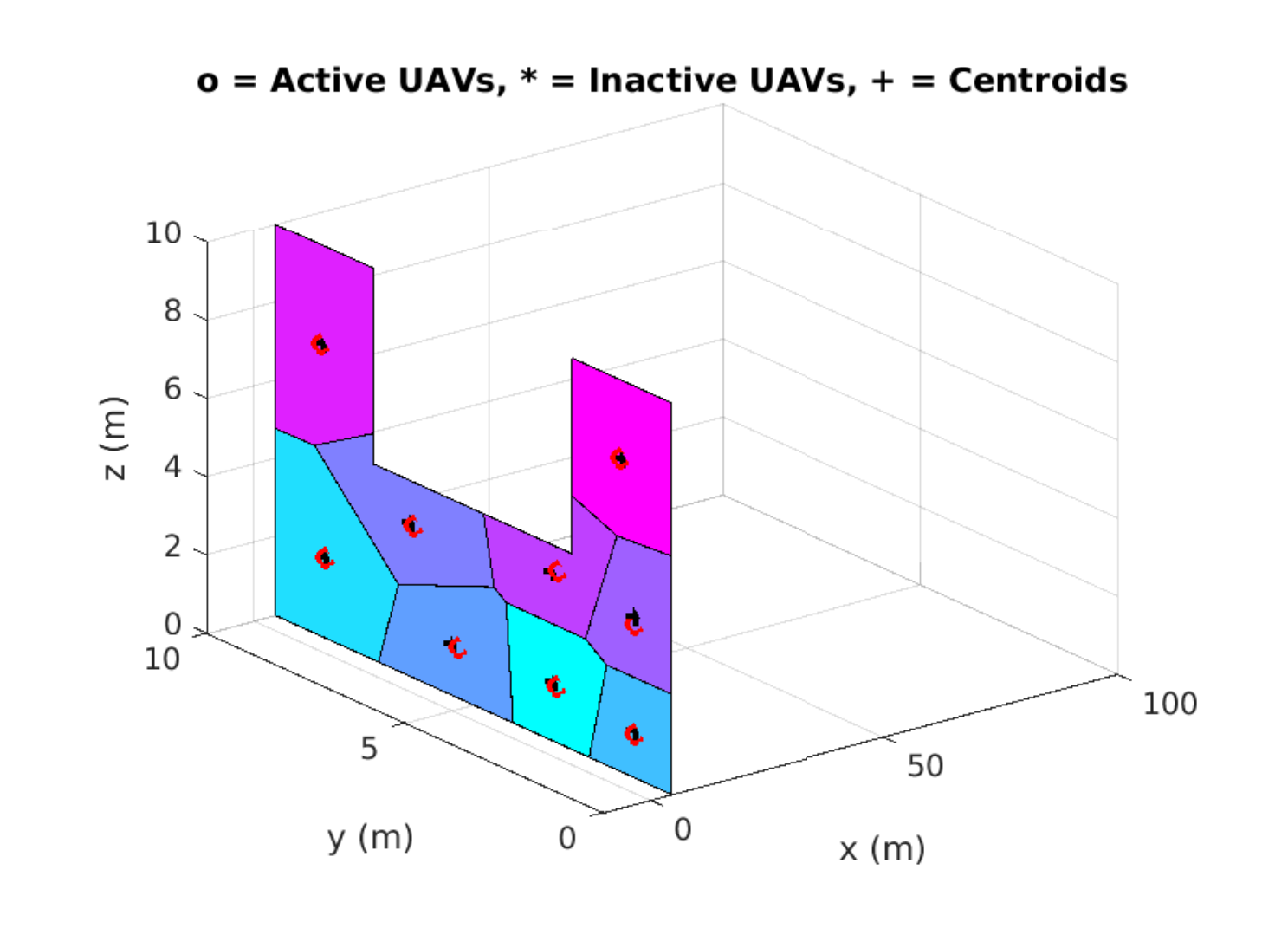} 
			\caption{t=11s}
		\end{subfigure}

		\begin{subfigure}[t]{0.32\textwidth}
			\centering
			\includegraphics[trim={0 0 0 1.25cm}, clip, width=\linewidth]{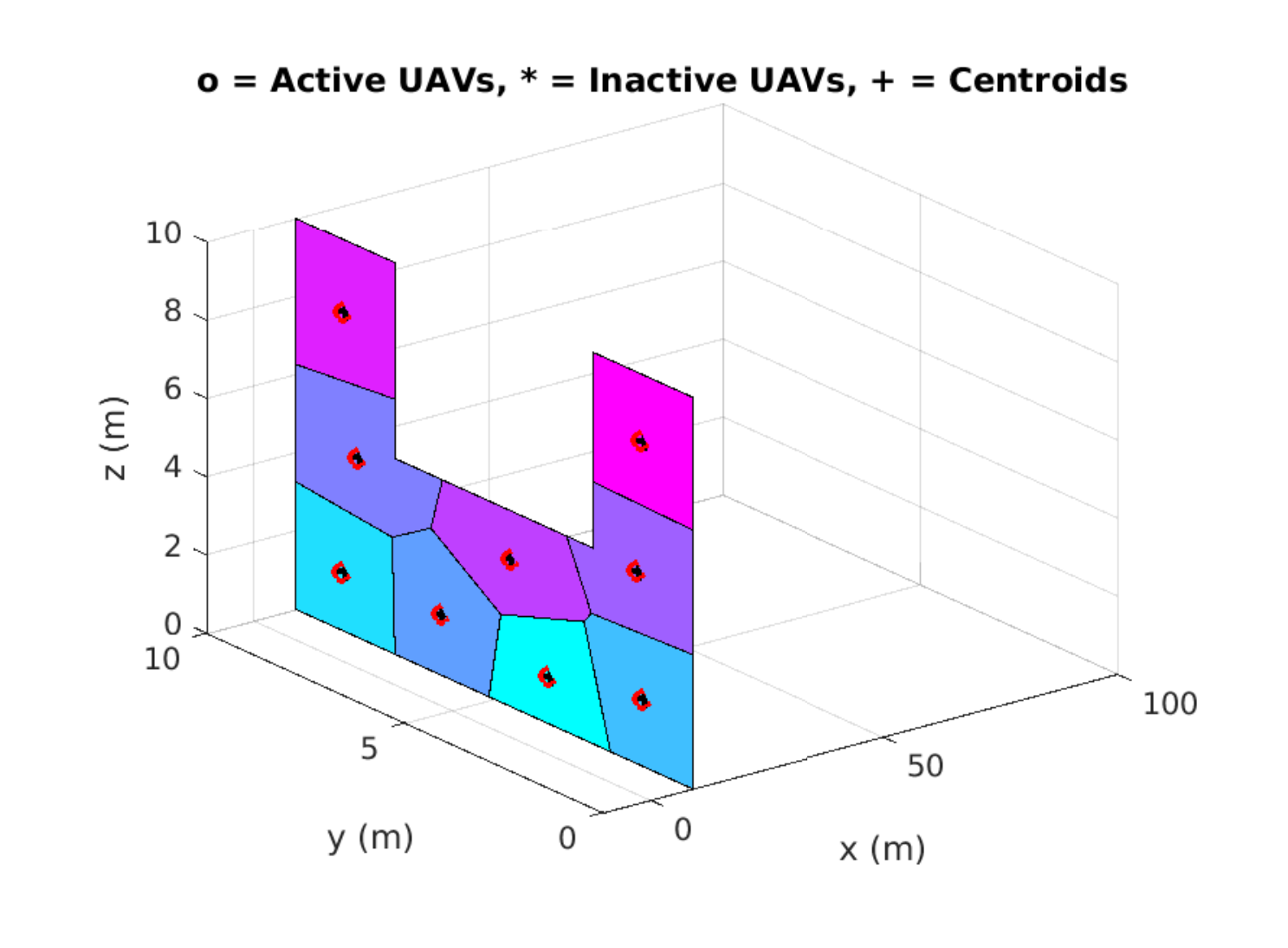} 
			\caption{t=20s}
		\end{subfigure}
		\hfill
		\begin{subfigure}[t]{0.32\textwidth}
			\centering
			\includegraphics[trim={0 0 0 1.25cm}, clip, width=\linewidth]{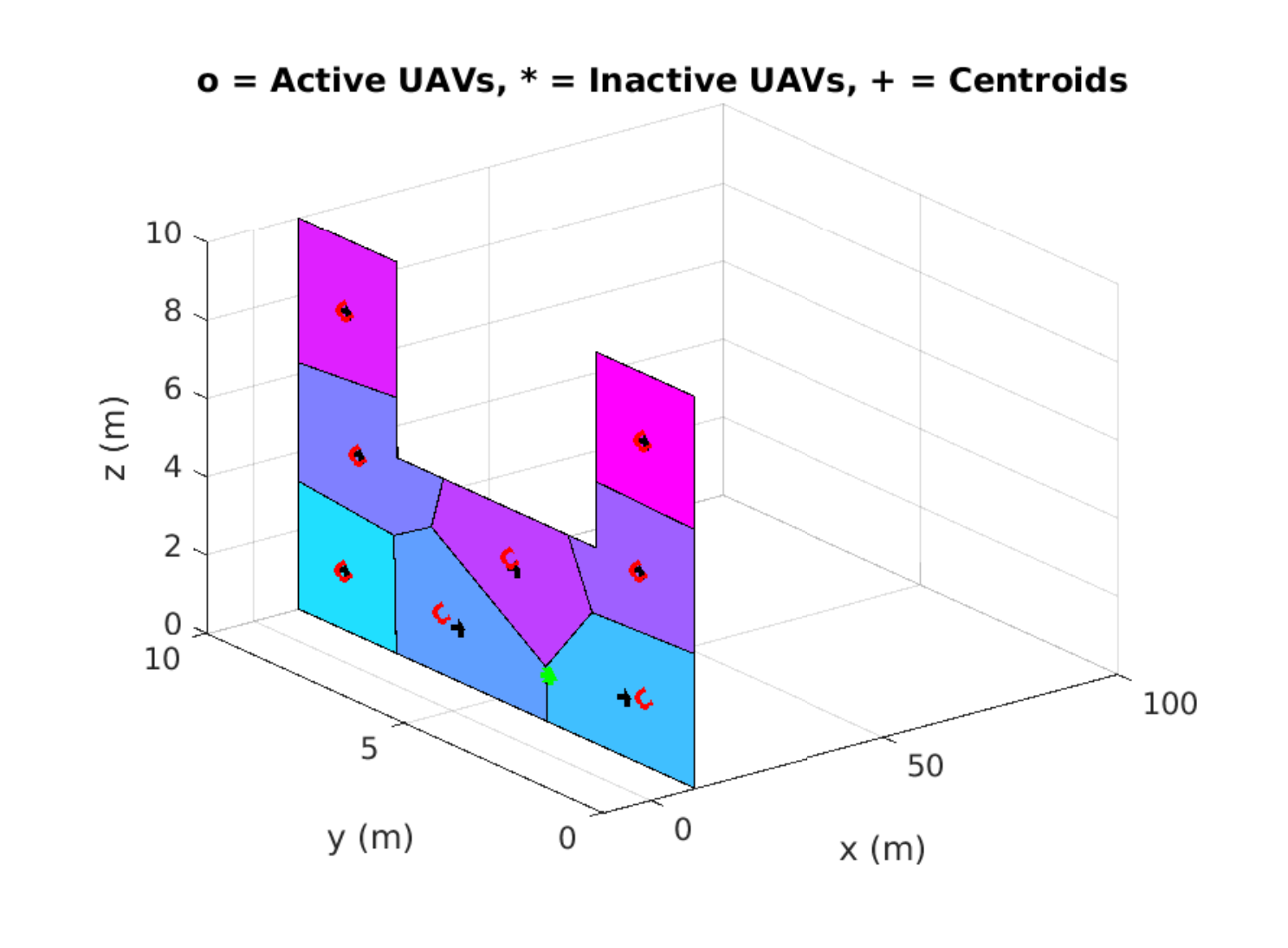} 
			\caption{t=21s}
		\end{subfigure}
		\hfill
		\begin{subfigure}[t]{0.32\textwidth}
			\centering
			\includegraphics[trim={0 0 0 1.25cm}, clip, width=\linewidth]{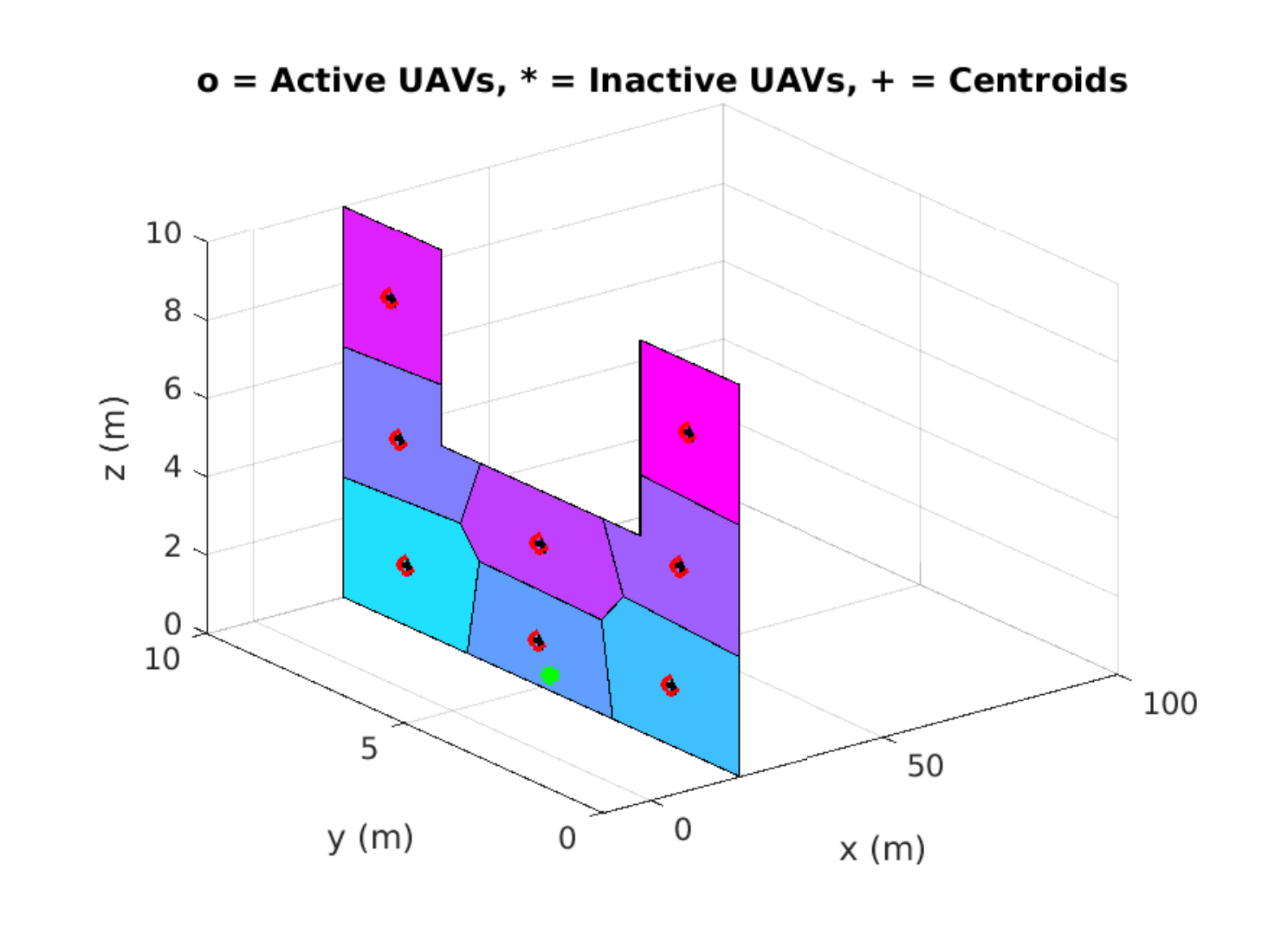} 
			\caption{t=40s}
		\end{subfigure}
	
		\begin{subfigure}[t]{0.32\textwidth}
			\centering
			\includegraphics[trim={0 0 0 1.25cm}, clip, width=\linewidth]{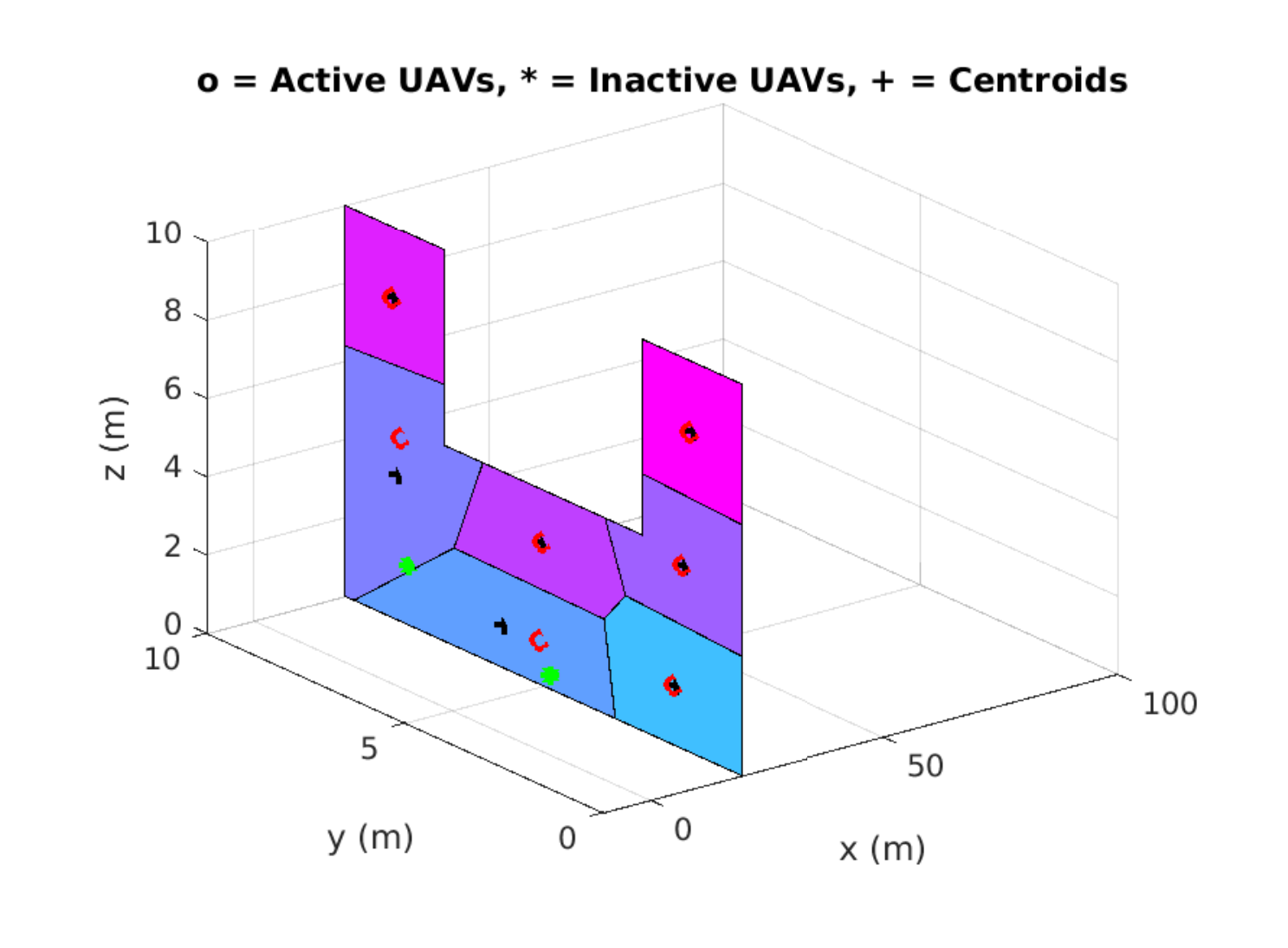} 
			\caption{t=41s}
		\end{subfigure}
		\hfill
		\begin{subfigure}[t]{0.32\textwidth}
			\centering
			\includegraphics[trim={0 0 0 1.25cm}, clip, width=\linewidth]{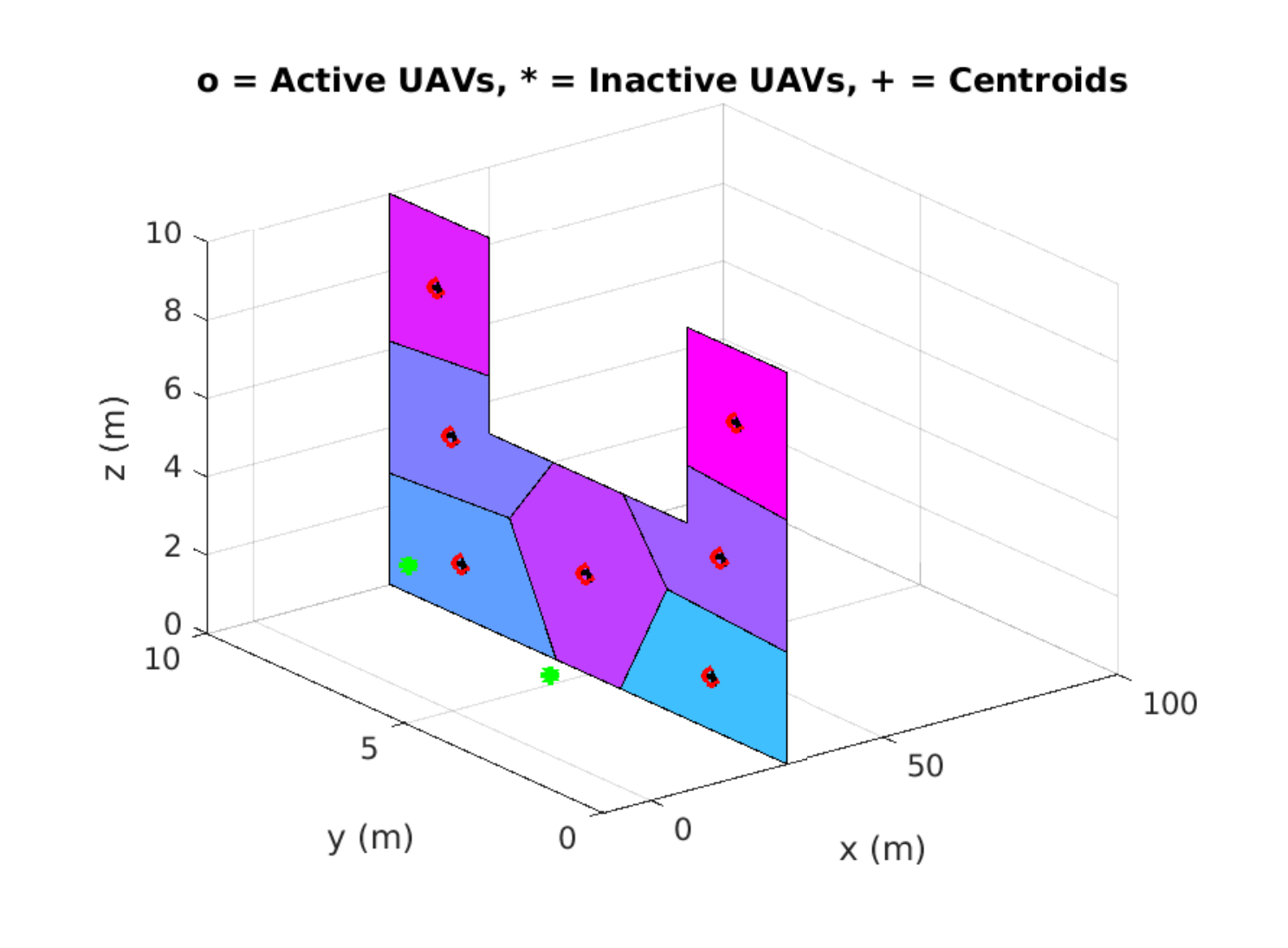} 
			\caption{t=60s}
		\end{subfigure}
		\hfill
		\begin{subfigure}[t]{0.32\textwidth}
			\centering
			\includegraphics[trim={0 0 0 1.25cm}, clip, width=\linewidth]{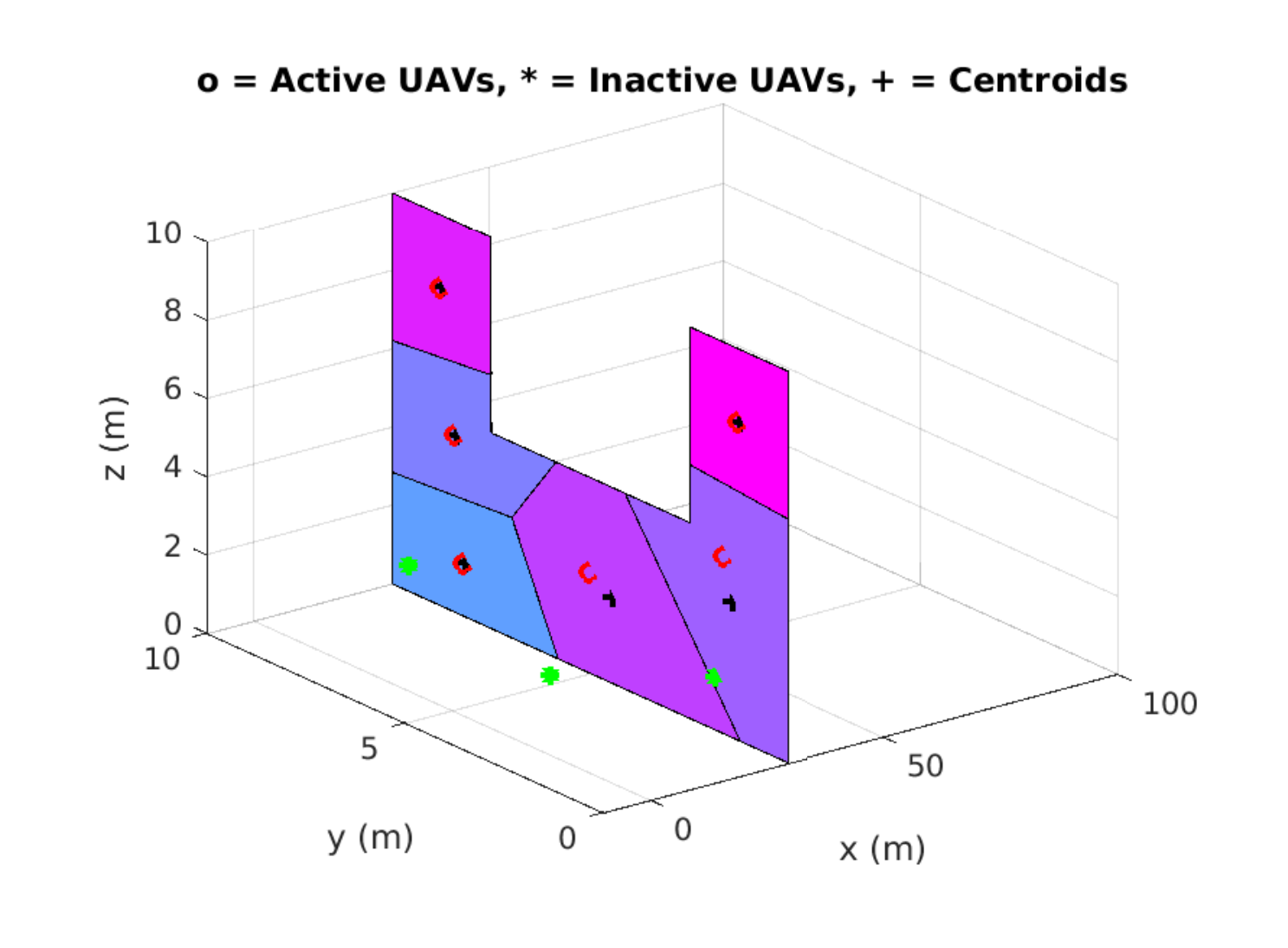} 
			\caption{t=61s}
		\end{subfigure}
	
		\begin{subfigure}[t]{0.32\textwidth}
			\centering
			\includegraphics[trim={0 0 0 1.25cm}, clip, width=\linewidth]{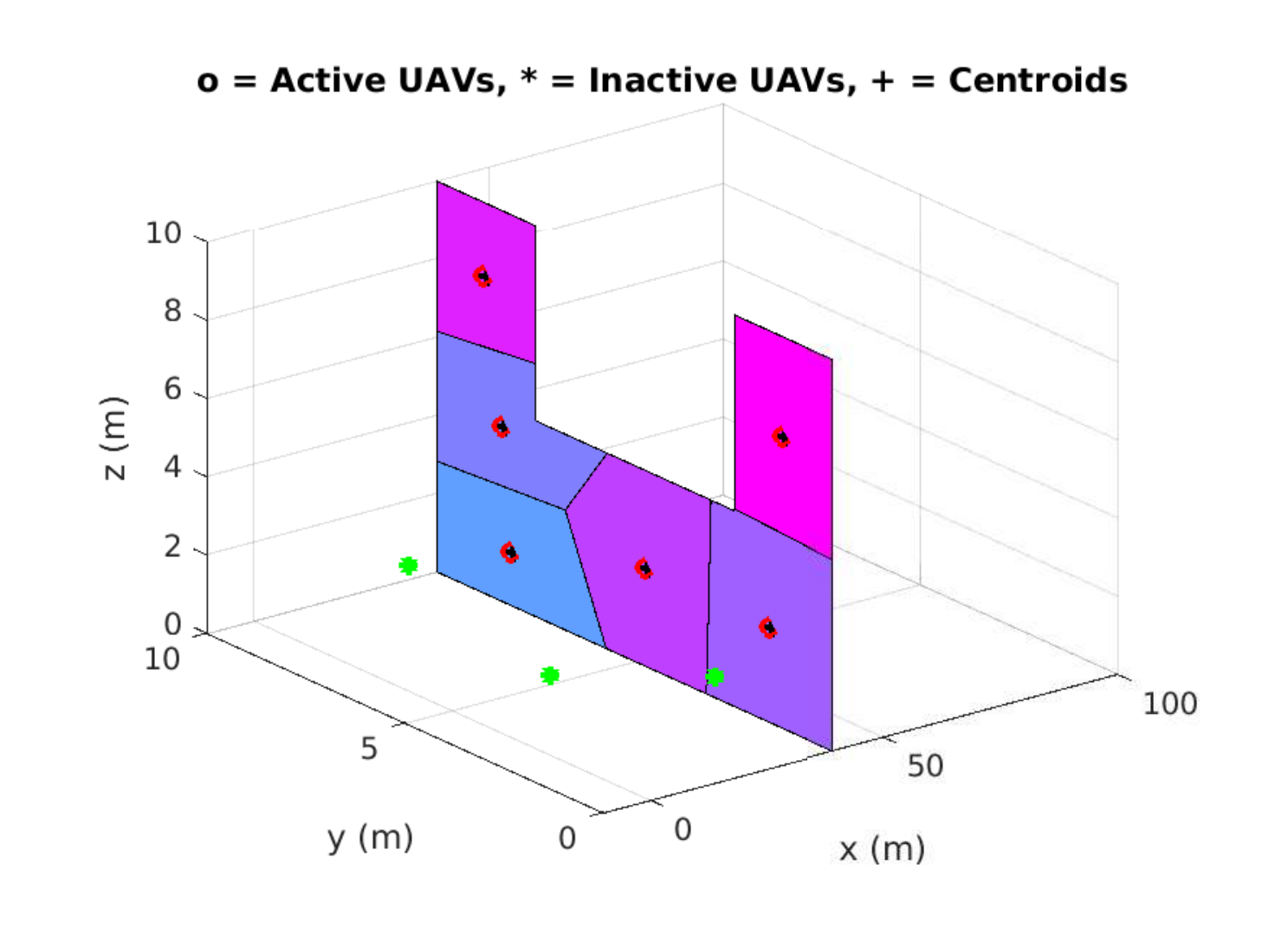} 
			\caption{t=80s}
		\end{subfigure}
		\hfill
		\begin{subfigure}[t]{0.32\textwidth}
			\centering
			\includegraphics[trim={0 0 0 1.25cm}, clip, width=\linewidth]{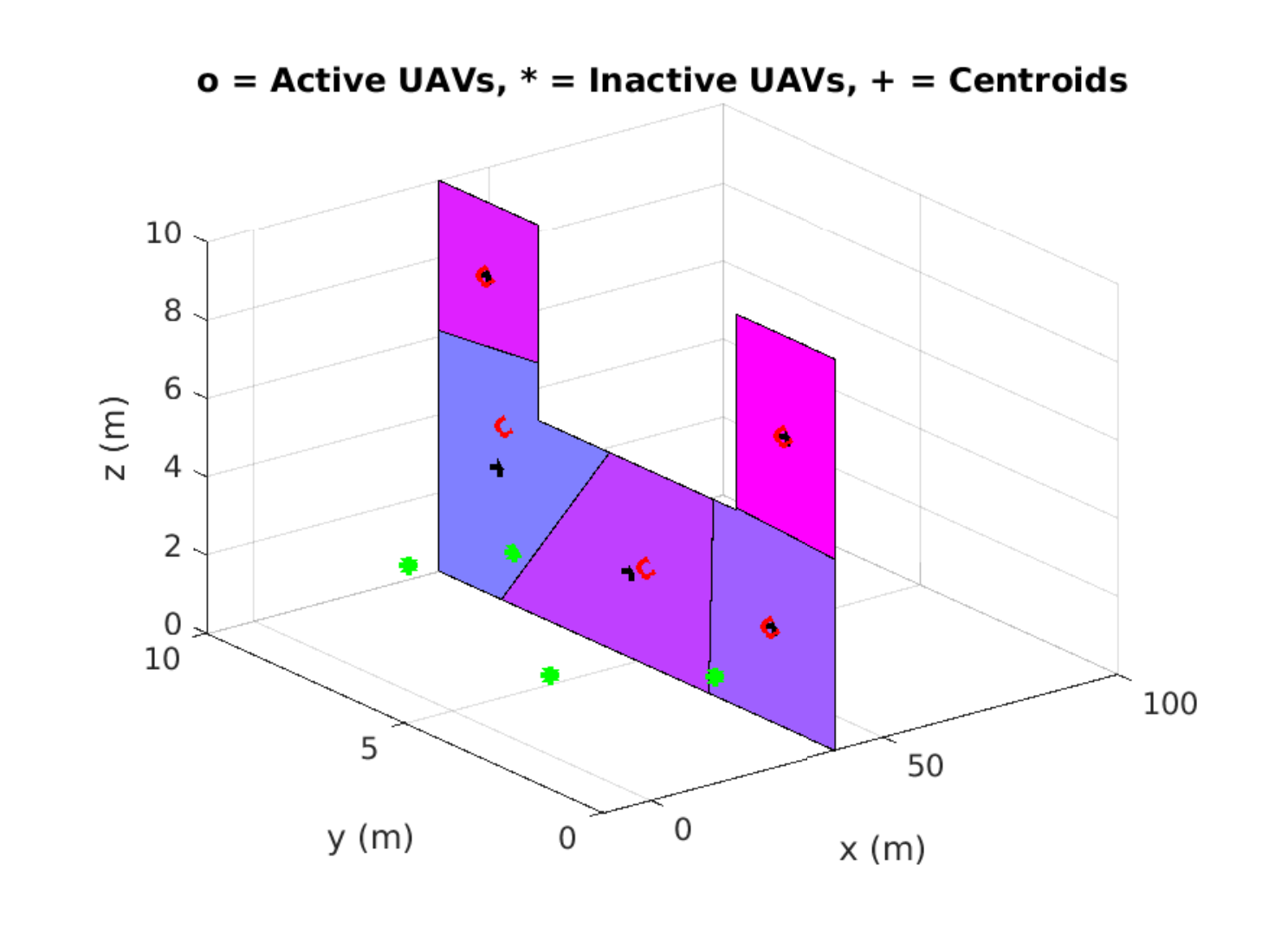} 
			\caption{t=81s}
		\end{subfigure}
		\hfill
		\begin{subfigure}[t]{0.32\textwidth}
			\centering
			\includegraphics[trim={0 0 0 1.25cm}, clip, width=\linewidth]{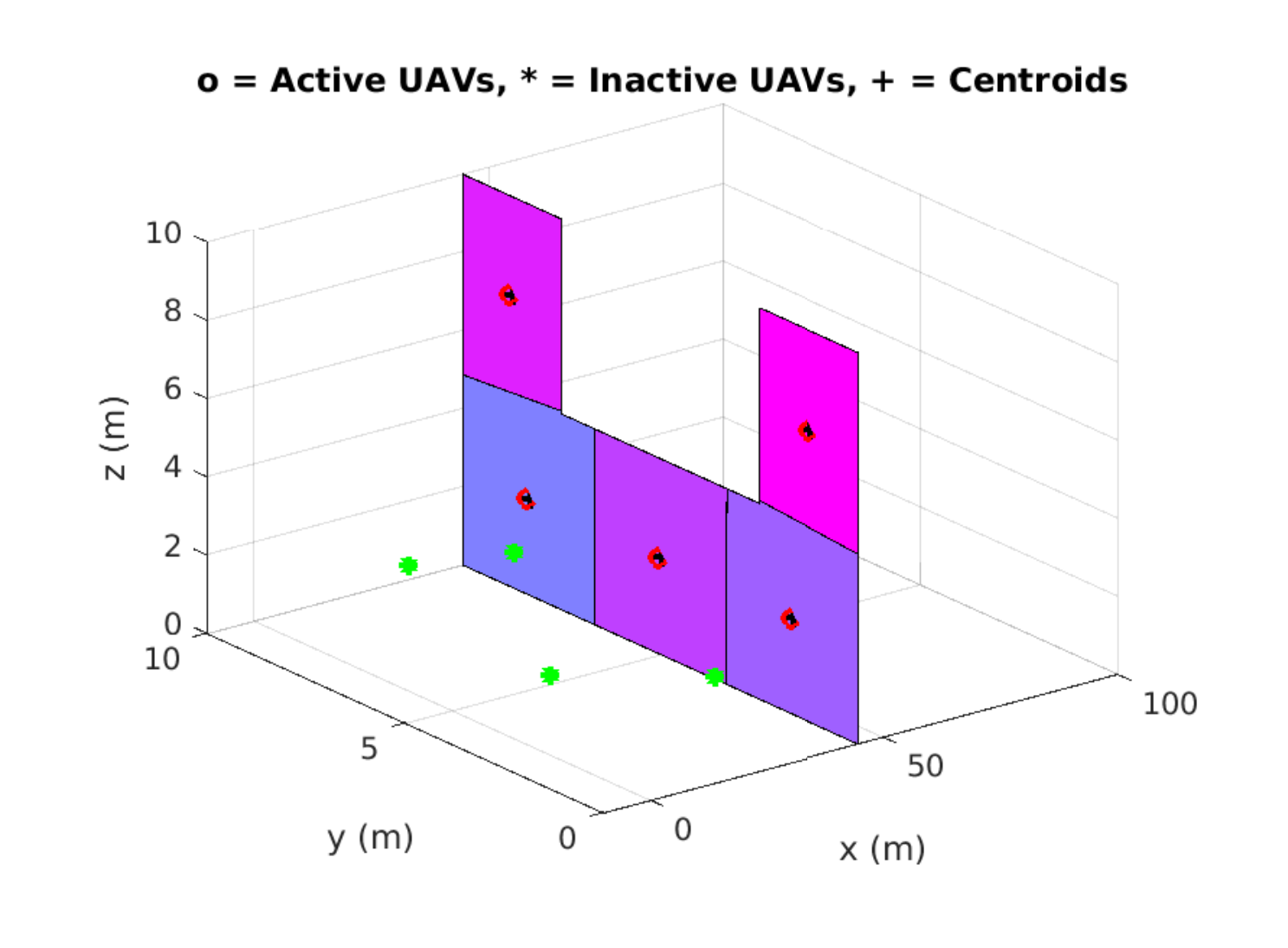} 
			\caption{t=91s}
		\end{subfigure}
		\caption{UAV locations at different time instants showing the robustness of the proposed approach against UAVs failures when they get removed from the group (Simulation Case 5)}
		\label{fig:ch9:simRobustness1}
	\end{adjustbox}
\end{figure}

\begin{figure}[!htb]
	\centering
	\includegraphics[width=0.7\linewidth]{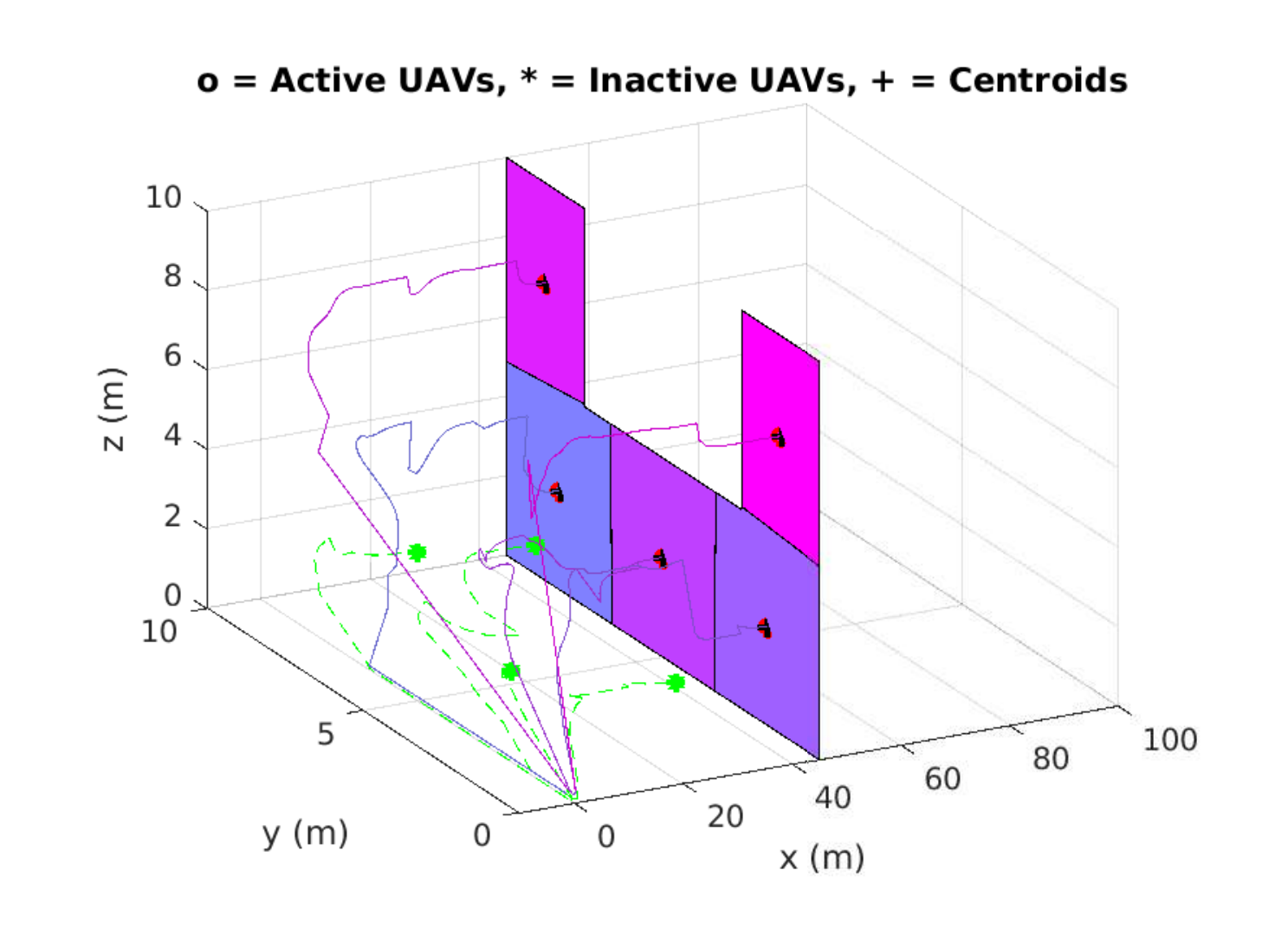} 
	\caption{Complete trajectories of the Multi-UAV system showing both active and inactive UAVs (Simulation Case 5)} \label{fig:ch9:simRobustness2}
\end{figure}

\subsection{Simulation Case 6: Obstacle Avoidance}

An additional simulation case was performed to show one way of incorporating obstacle avoidance capabilities within the proposed method.
The simulation results for this case are given in \cref{fig:ch9:simObsAvoidance1,fig:ch9:simObsAvoidance2,fig:ch9:simObsAvoidance3}.
In this case, 6 vehicles were initially deployed, and a rectangular region was considered to be the sweeping plane $\mathcal{F}(t)$.
This particular choice of $\mathcal{F}(t)$ can be used in scenarios where the vehicles are equipped with some downward-facing sensors (ex. cameras), and their sensing FOV are simply a footprint on the ground.
Even though using 6 vehicles may be redundant in this case as there will be overlapping in the vehicles' sensing FOV, the aim here is only to show a particular way for the multi-UAV system to avoid obstacles.

The adopted approach is simply to manipulate the shape, size, orientation and/or velocity of the sweeping plane $\mathcal{F}(t)$ to safely avoid detected obstacles.
Note that in this case, only vehicles within a sensing range from obstacles can determine such required action which then sent to the remaining vehicles over the connected network.
Obstacle avoidance is ensured in this case because vehicles are guaranteed to maintain their motion within $\mathcal{F}(t)$ once its reached using the proposed control laws (i.e. $\bm{p}_i(t)\in \mathcal{F}(t),\ \forall t>t_{*}$ where $t_*$ is the time at which the vehicles reach $\mathcal{F}(t)$).

It can be seen from \cref{fig:ch9:simObsAvoidance1}(a) and \cref{fig:ch9:simObsAvoidance1}(b) that the vehicles are approaching $\mathcal{F}(t)$ since they were initially deployed away from it.
Simultaneously, the vehicles move along $\mathcal{F}(t)$ to reach centroidal Voronoi configurations.
At $t = 15s$, two obstacles are detected along the way of the sweeping plane's movement by two nearby vehicles.
The vehicles decide to dynamically change the size of $\mathcal{F}(t)$ to avoid both obstacles as shown in \cref{fig:ch9:simObsAvoidance1}(g)-(j).
A different decision is also made by the multi-UAV system when detecting another obstacle at $t=49.5s$ (\cref{fig:ch9:simObsAvoidance1}(k)).
At this time, the sweeping plane is tilted in the $z$-direction as can be seen from \cref{fig:ch9:simObsAvoidance1}(l)-(n).

Overall, this approach has good potential for motion coordination of multi-vehicle systems especially when considering obstacle avoidance compared to artificial potential based approaches where sometimes there will be conflicts between achieving obstacle avoidance and avoiding collision with other vehicles.

\begin{figure}[!htb]
	\centering
	\begin{adjustbox}{minipage=\linewidth,scale=1.0}
		\begin{subfigure}[t]{0.32\textwidth}
			\centering
			\includegraphics[clip, width=\linewidth]{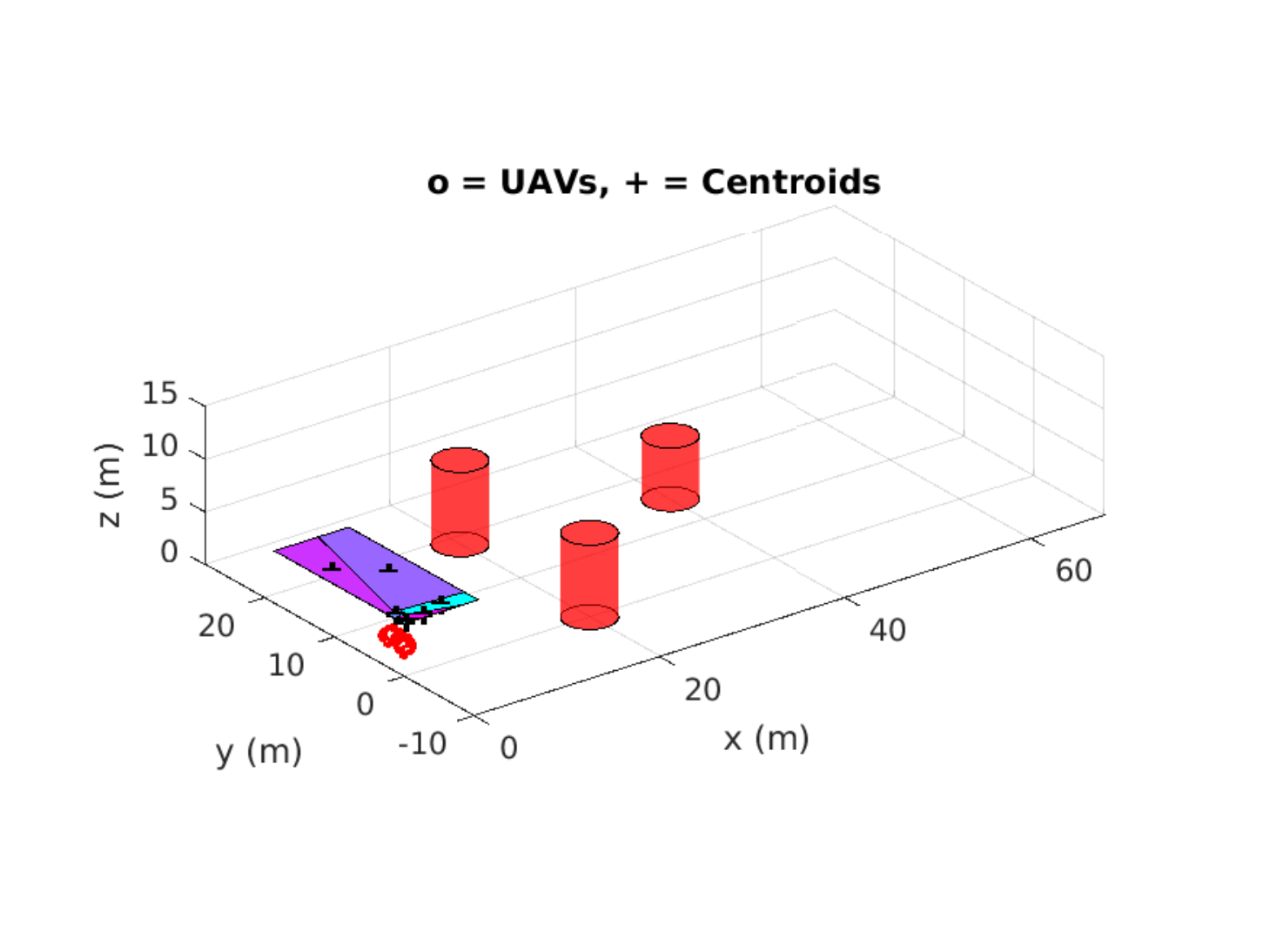} 
			\caption{t=1s}
		\end{subfigure}
		\hfill
		\begin{subfigure}[t]{0.32\textwidth}
			\centering
			\includegraphics[trim={0 0 0 0}, clip, width=\linewidth]{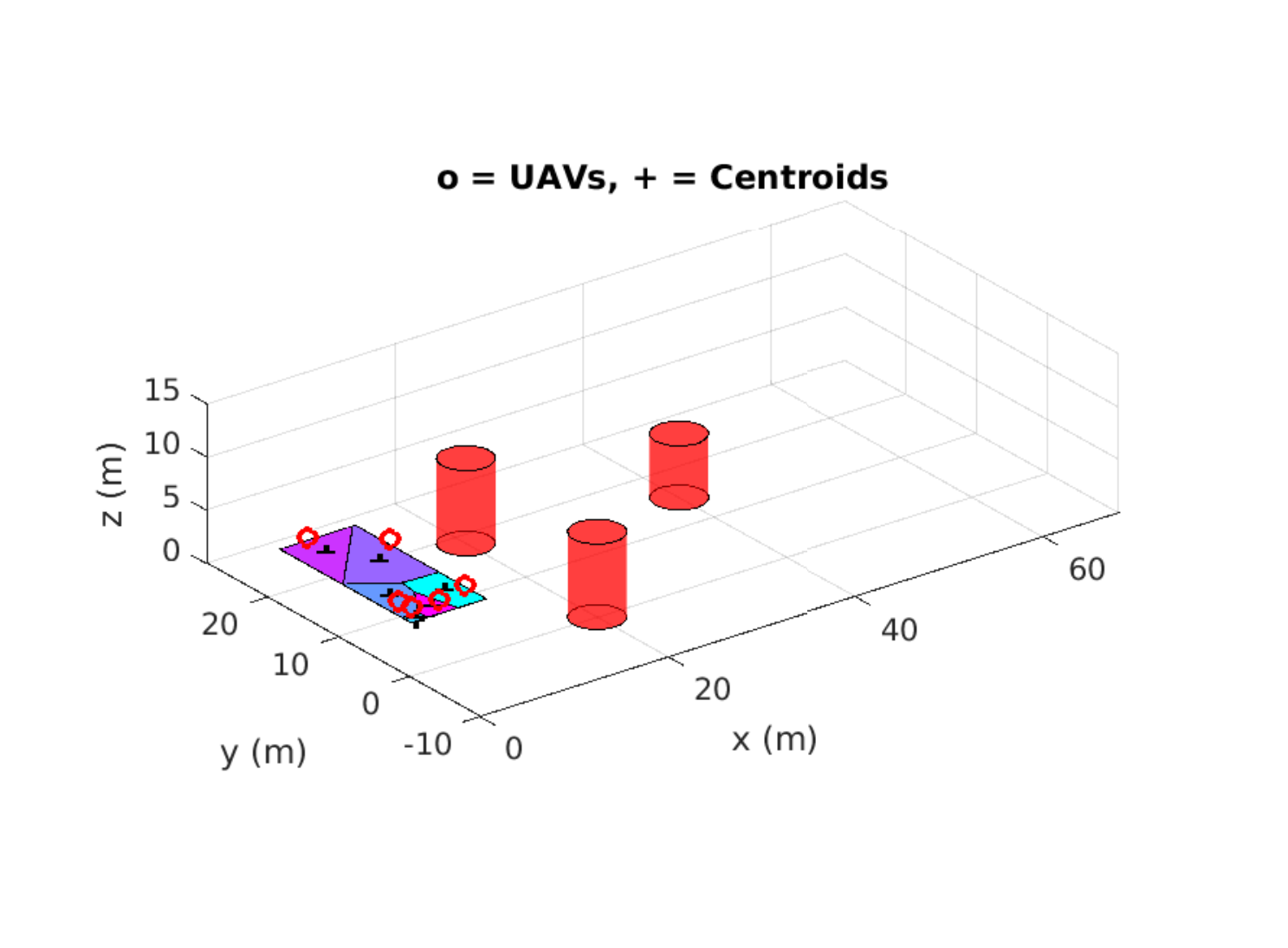} 
			\caption{t=1.5s}
		\end{subfigure}
		\hfill
		\begin{subfigure}[t]{0.32\textwidth}
			\centering
			\includegraphics[clip, width=\linewidth]{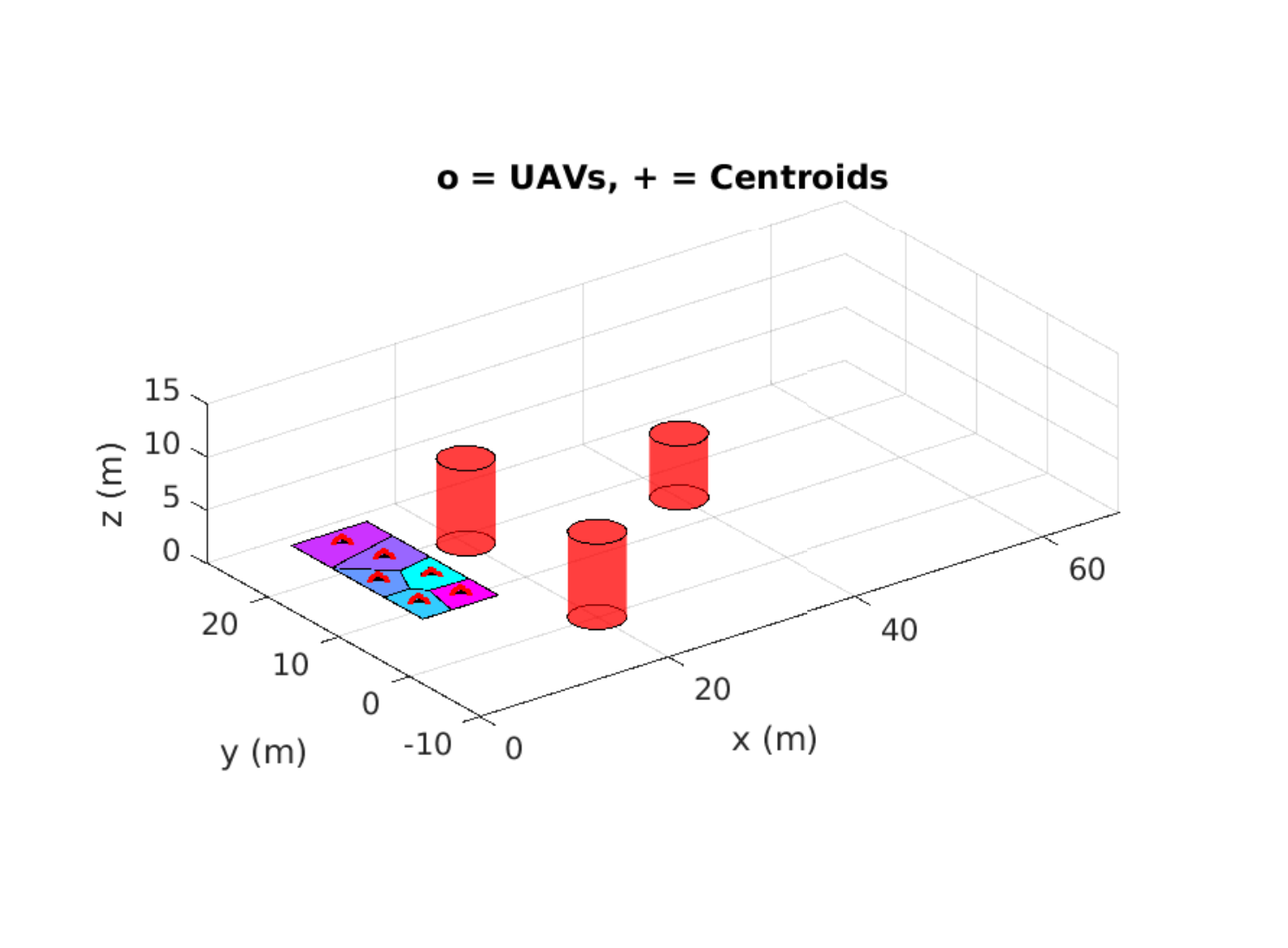} 
			\caption{t=4s}
		\end{subfigure}
		
		\begin{subfigure}[t]{0.32\textwidth}
			\centering
			\includegraphics[trim={0 0 0 2.5cm}, clip, width=\linewidth]{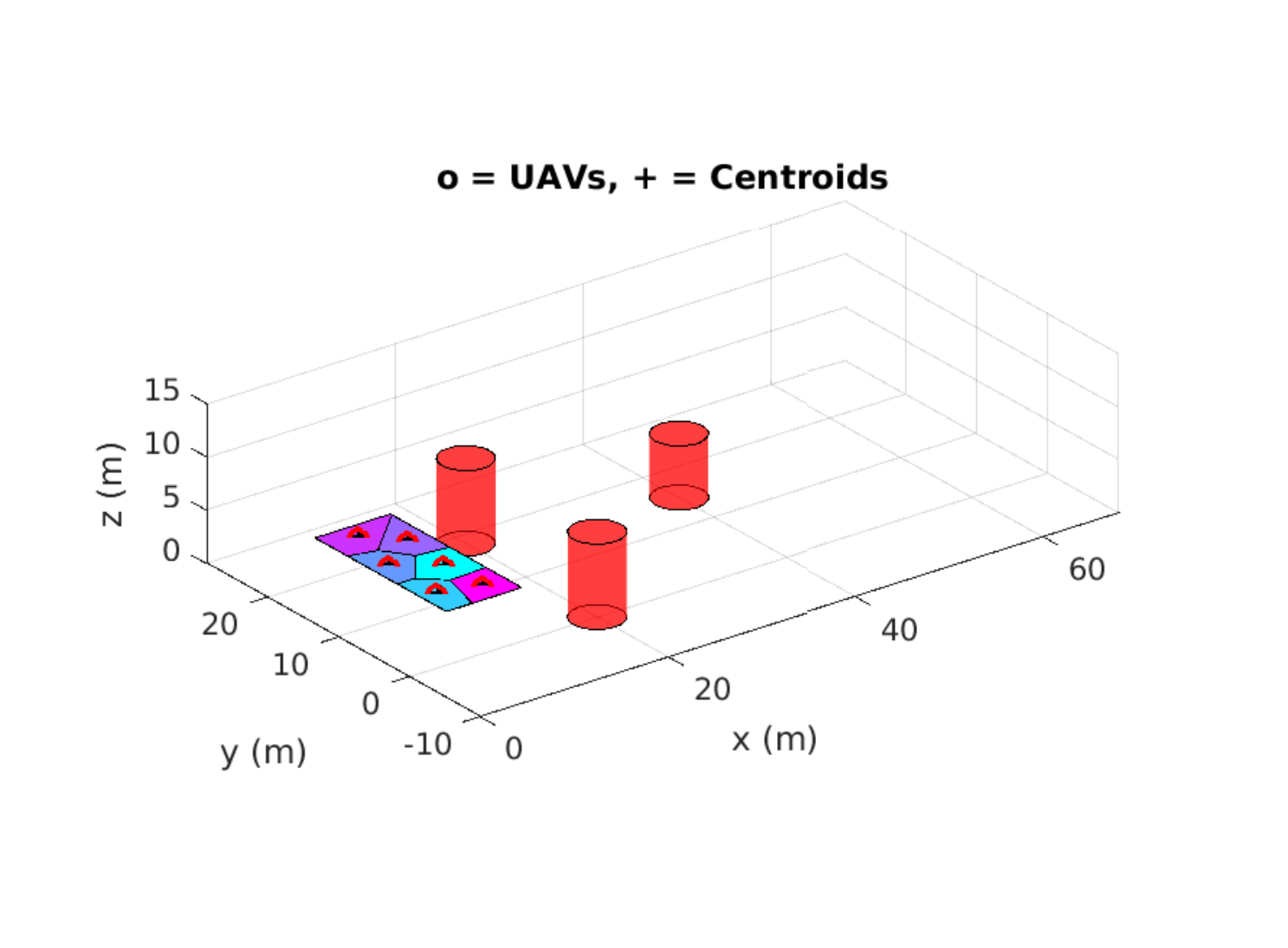} 
			\caption{t=9s}
		\end{subfigure}
		\hfill
		\begin{subfigure}[t]{0.32\textwidth}
			\centering
			\includegraphics[trim={0 0 0 2.5cm}, clip, width=\linewidth]{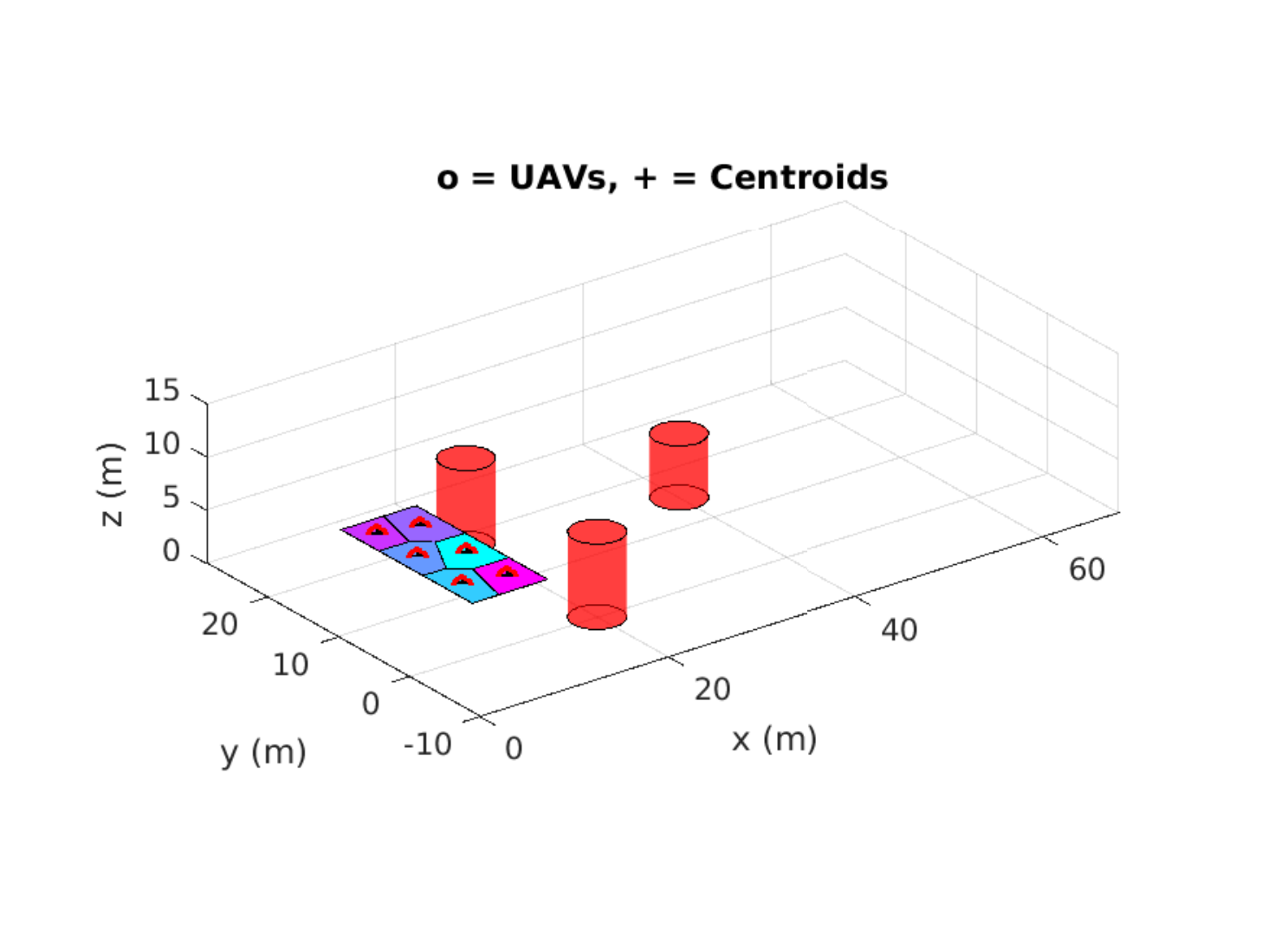} 
			\caption{t=14.5s}
		\end{subfigure}
		\hfill
		\begin{subfigure}[t]{0.32\textwidth}
			\centering
			\includegraphics[trim={0 0 0 2.5cm}, clip, width=\linewidth]{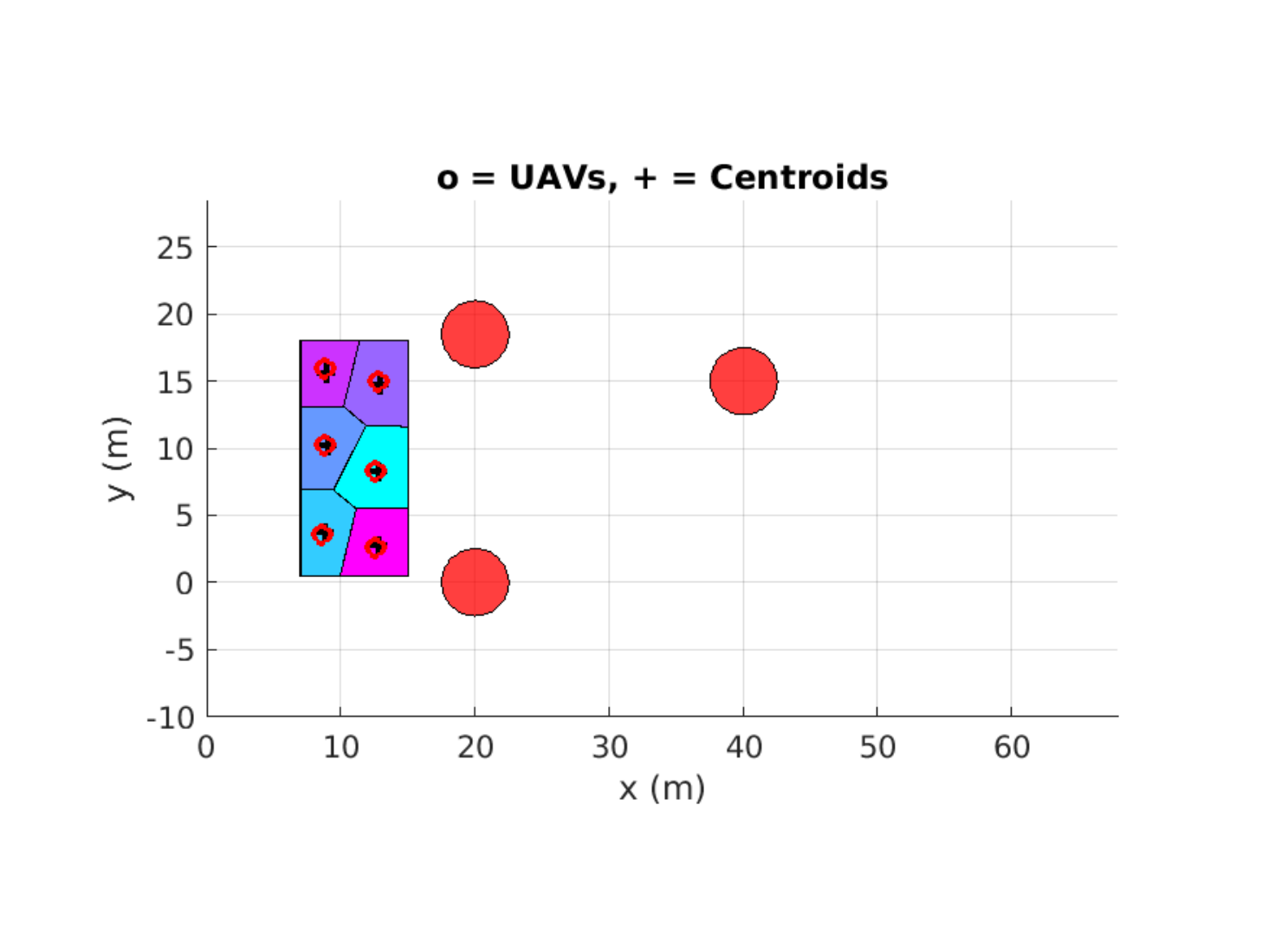} 
			\caption{t=15s}
		\end{subfigure}
		
		\begin{subfigure}[t]{0.32\textwidth}
			\centering
			\includegraphics[trim={0 0 0 2.5cm}, clip, width=\linewidth]{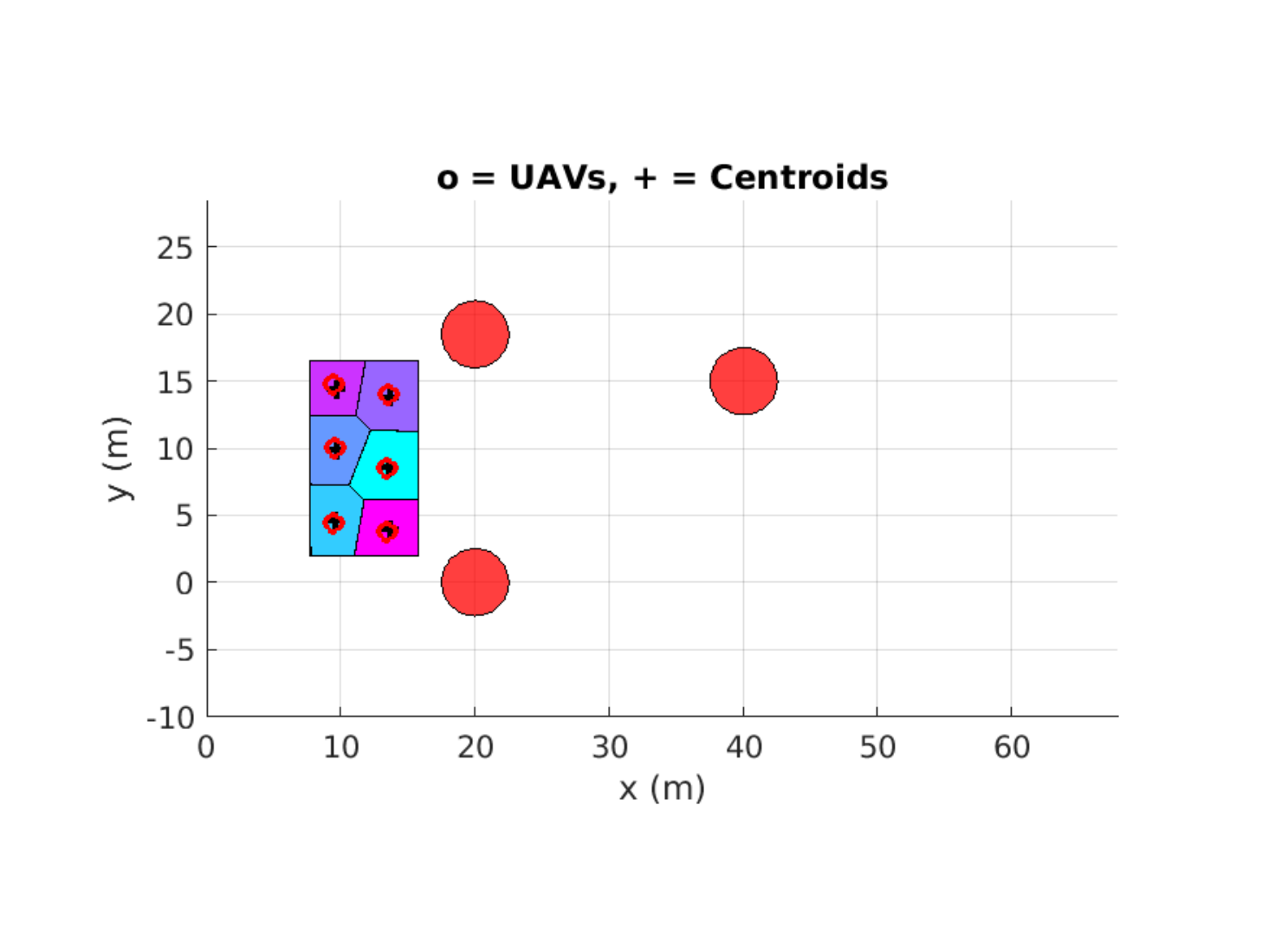} 
			\caption{t=16.5s}
		\end{subfigure}
		\hfill
		\begin{subfigure}[t]{0.32\textwidth}
			\centering
			\includegraphics[trim={0 0 0 2.5cm}, clip, width=\linewidth]{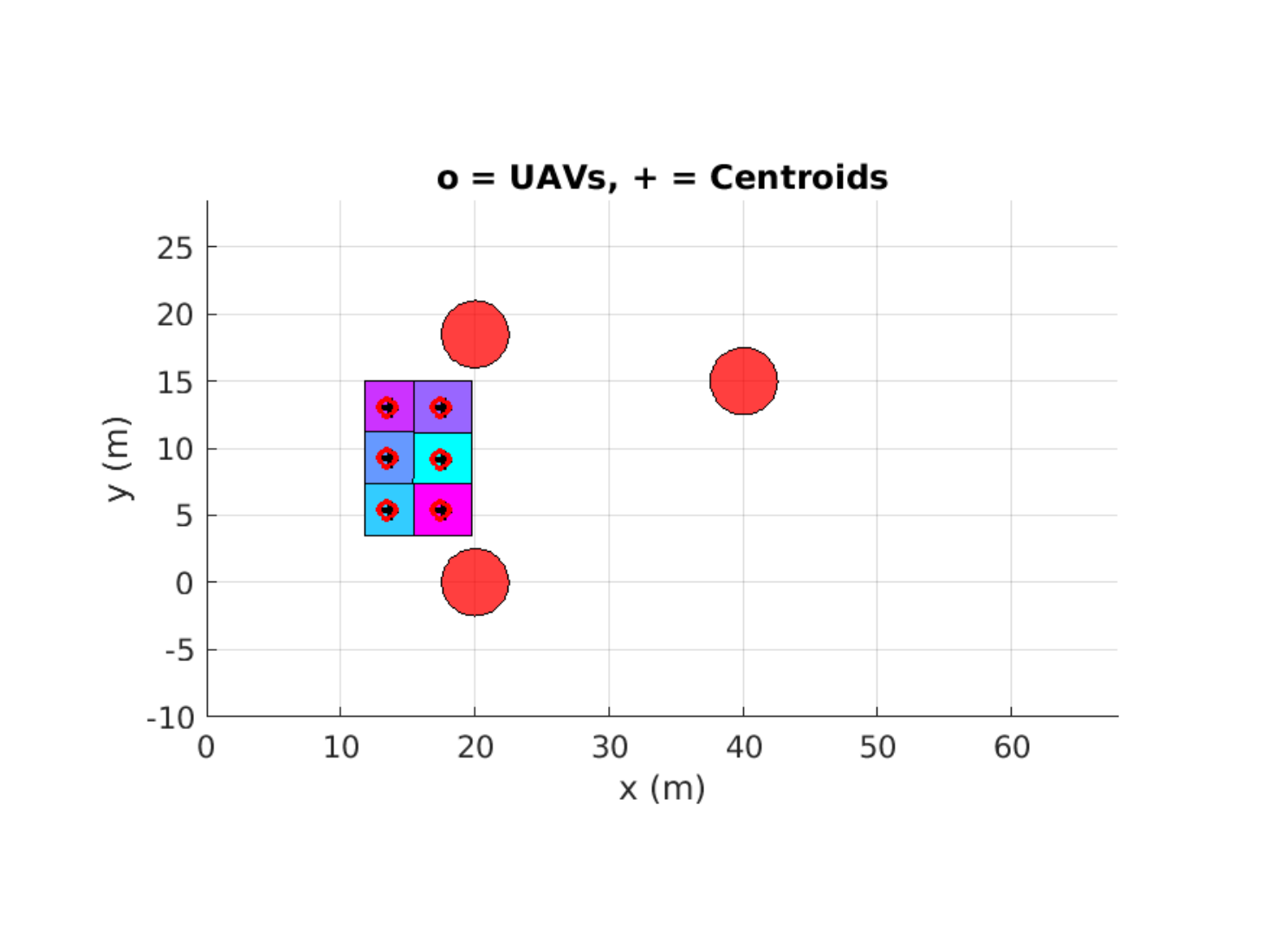} 
			\caption{t=24.5s}
		\end{subfigure}
		\hfill
		\begin{subfigure}[t]{0.32\textwidth}
			\centering
			\includegraphics[trim={0 0 0 2.5cm}, clip, width=\linewidth]{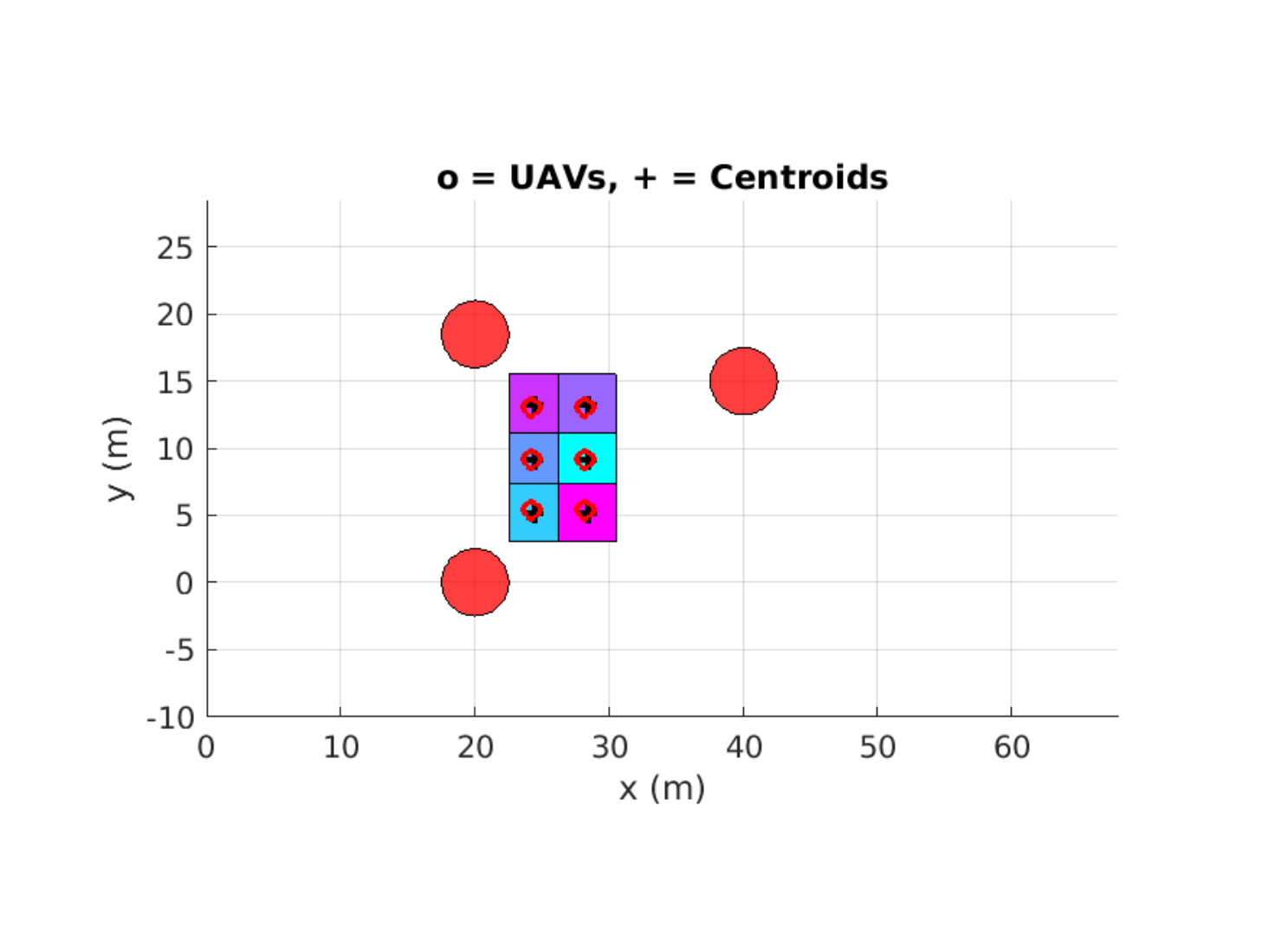} 
			\caption{t=46s}
		\end{subfigure}
		
		\begin{subfigure}[t]{0.32\textwidth}
			\centering
			\includegraphics[trim={0 0 0 2.5cm}, clip, width=\linewidth]{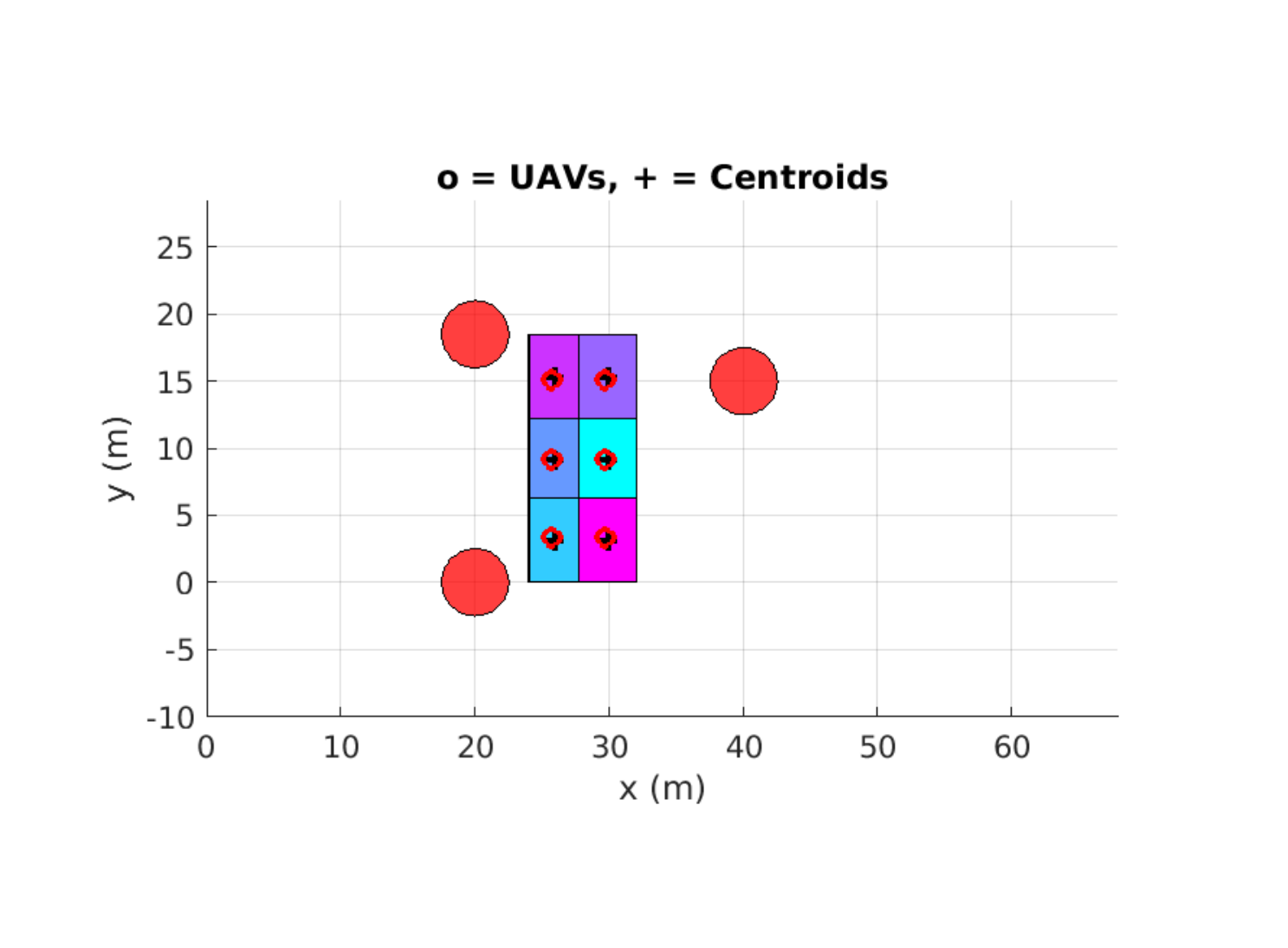} 
			\caption{t=49s}
		\end{subfigure}
		\hfill
		\begin{subfigure}[t]{0.32\textwidth}
			\centering
			\includegraphics[trim={0 0 0 2.5cm}, clip, width=\linewidth]{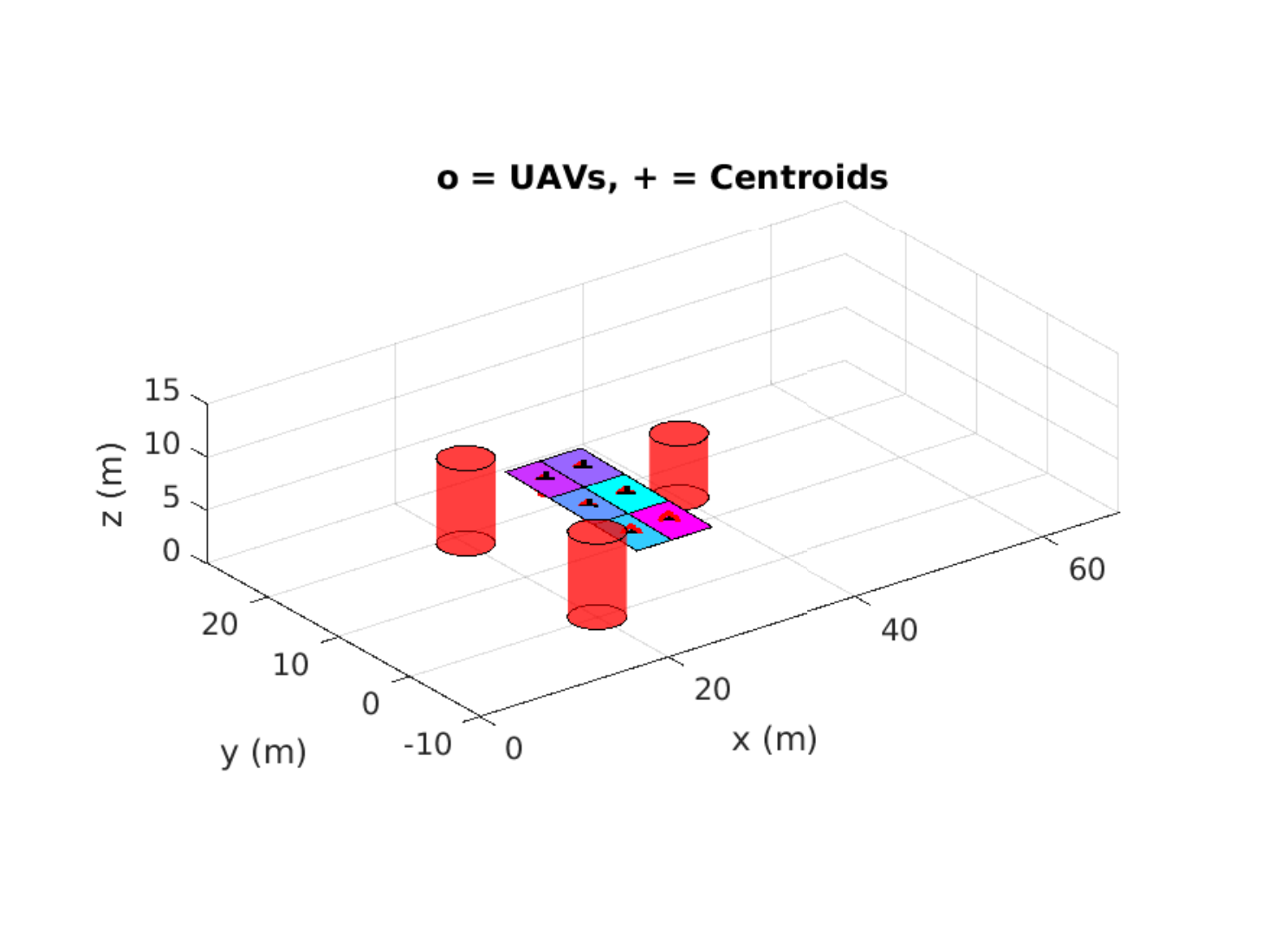} 
			\caption{t=49.5s}
		\end{subfigure}
		\hfill
		\begin{subfigure}[t]{0.32\textwidth}
			\centering
			\includegraphics[trim={0 0 0 2.5cm}, clip, width=\linewidth]{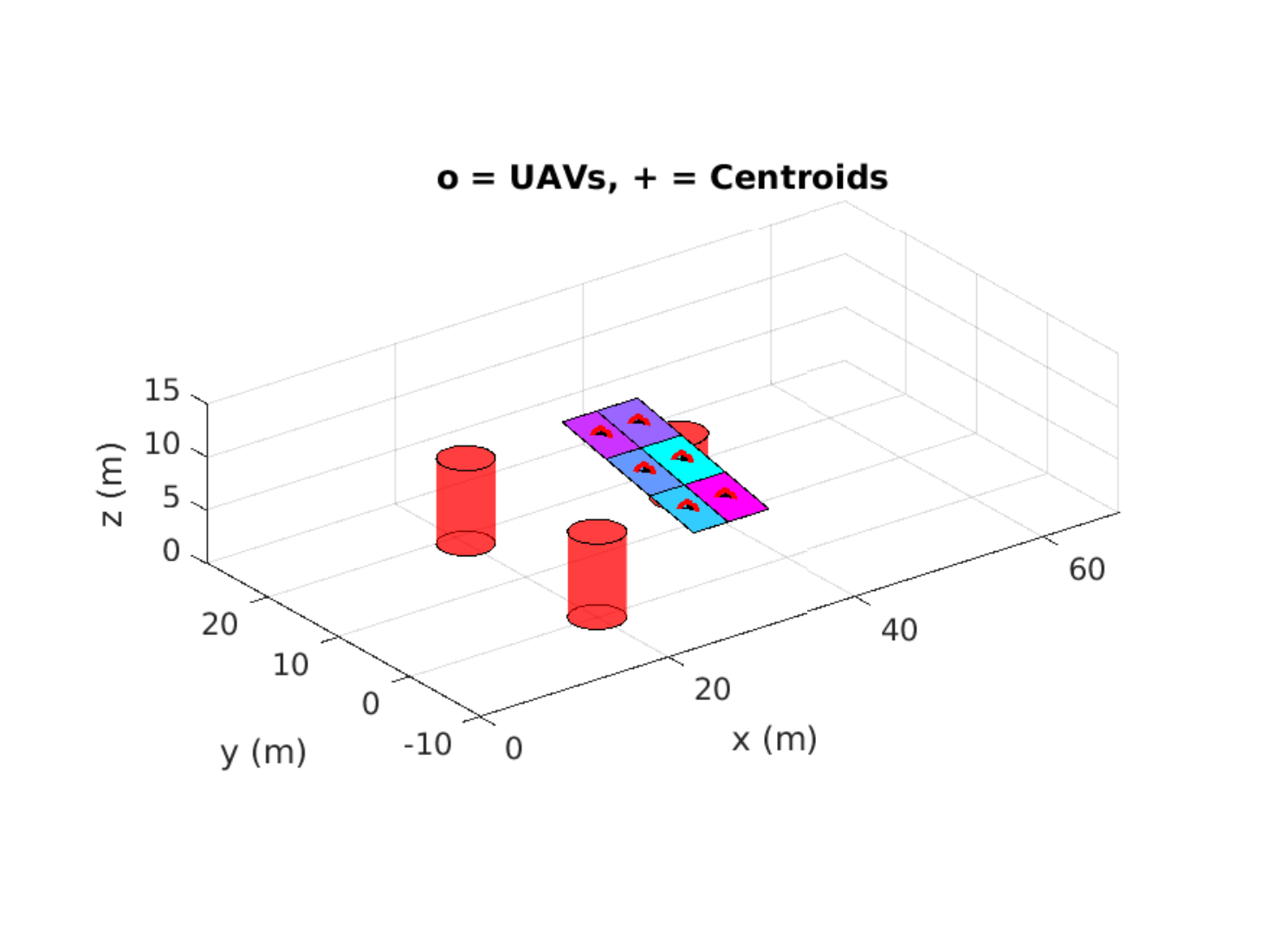} 
			\caption{t=61.5s}
		\end{subfigure}
	
		\begin{subfigure}[t]{0.32\textwidth}
			\centering
			\includegraphics[trim={0 0 0 2.5cm}, clip, width=\linewidth]{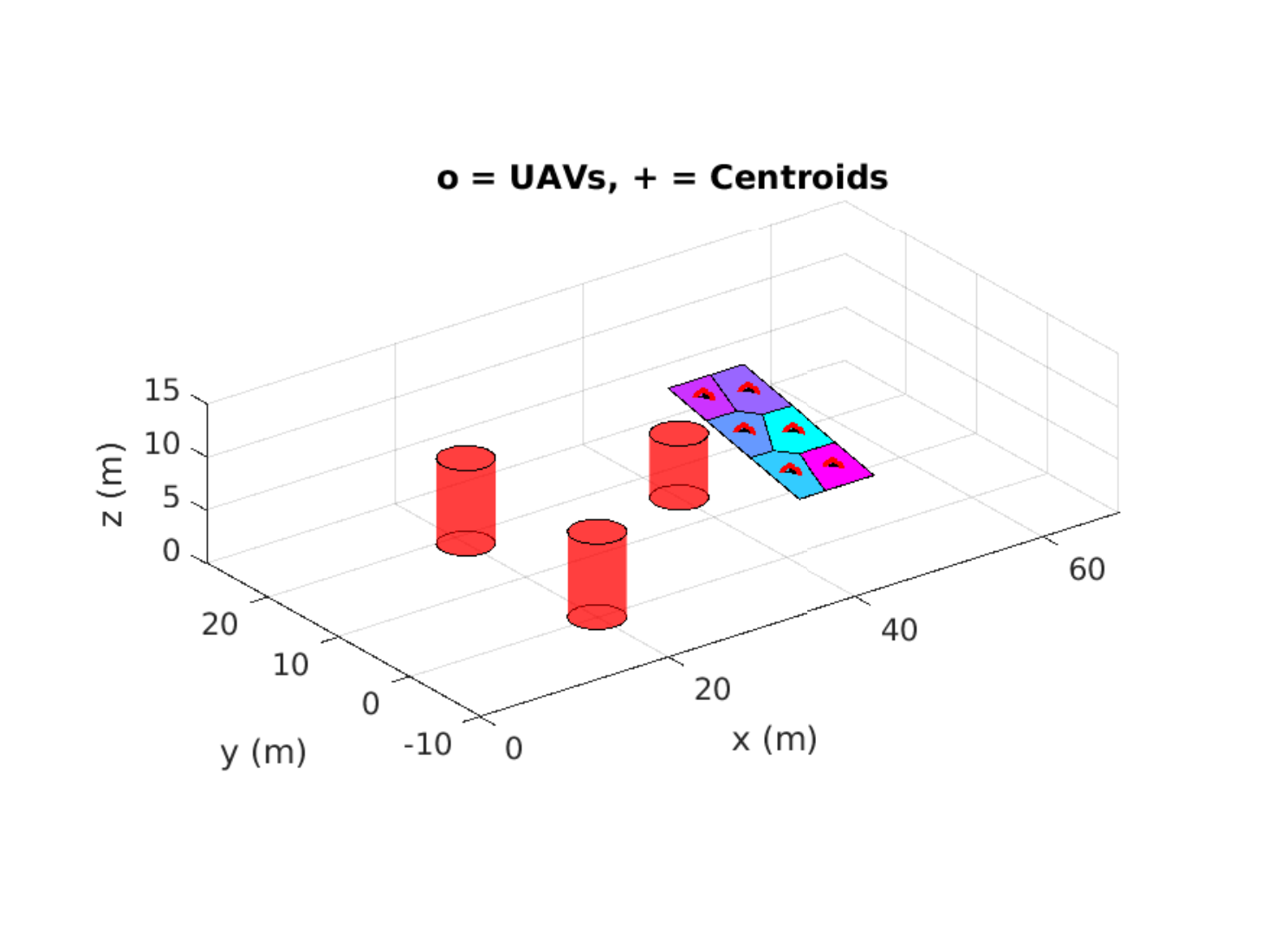} 
			\caption{t=84s}
		\end{subfigure}
		\begin{subfigure}[t]{0.32\textwidth}
			\centering
			\includegraphics[trim={0 0 0 2.5cm}, clip, width=\linewidth]{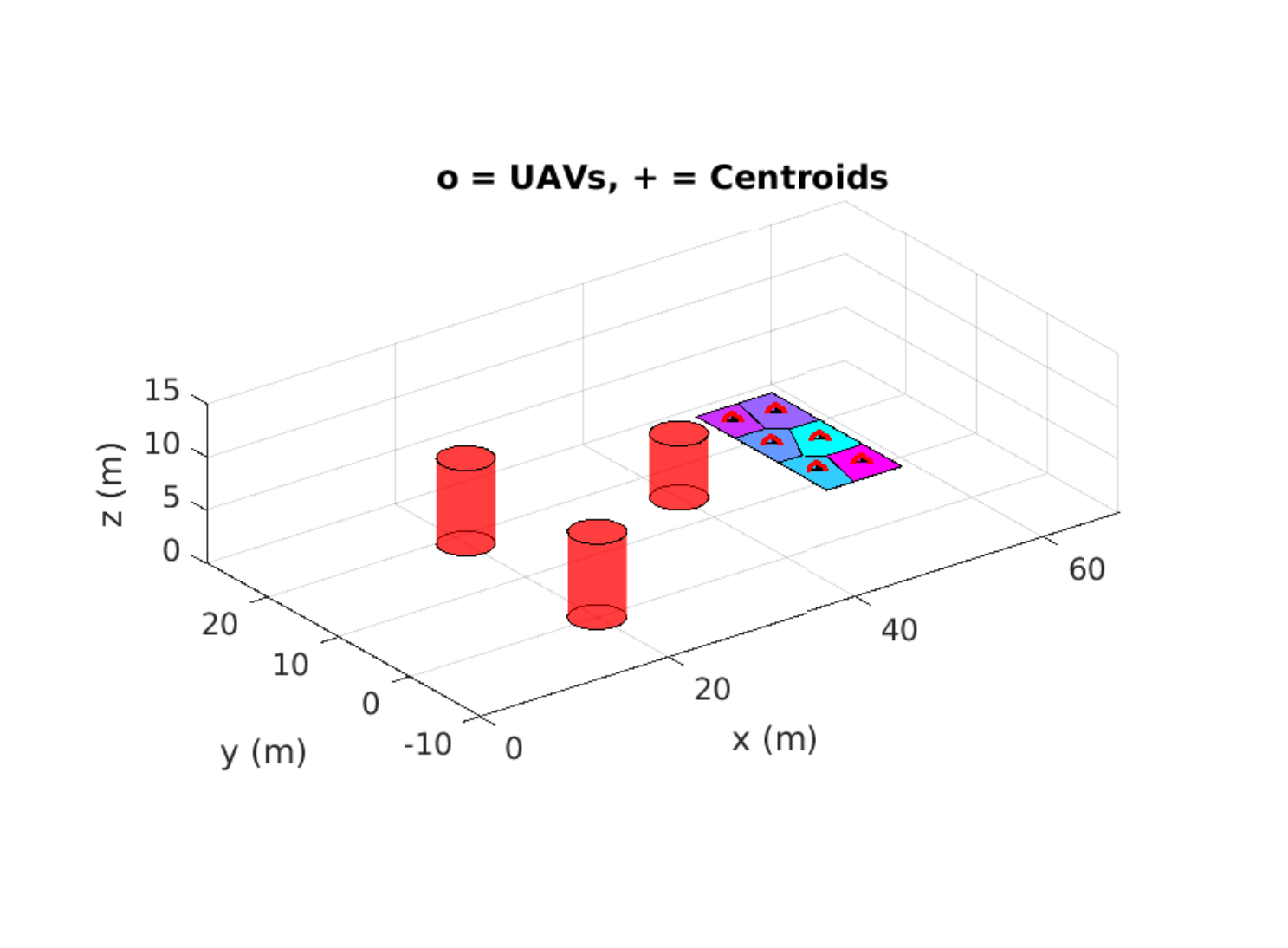} 
			\caption{t=90s}
		\end{subfigure}
		\hfill
		\caption{UAV locations at different time instants showing the obstacle avoidance capability of the proposed approach by dynamically manipulating the sweeping plane $\mathcal{F}(t)$ (Simulation Case 6)}
		\label{fig:ch9:simObsAvoidance1}
	\end{adjustbox}
\end{figure}

\begin{figure}[!htb]
	\centering
	\includegraphics[width=0.7\linewidth]{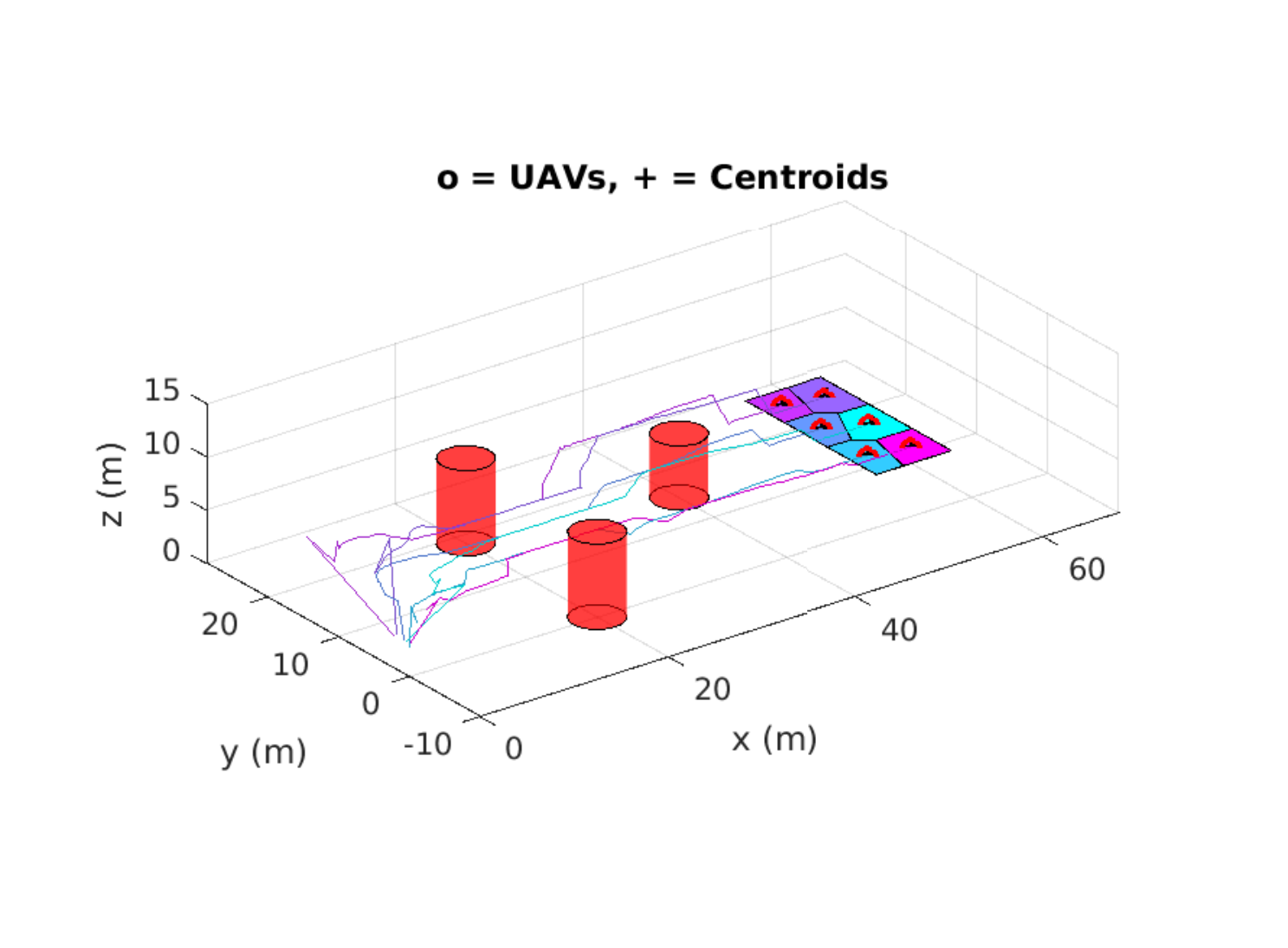} 
	\caption{Complete trajectories of the Multi-UAV system showing the obstacle avoidance capability of the proposed approach [3D View] (Simulation Case 6)} \label{fig:ch9:simObsAvoidance2}
\end{figure}

\begin{figure}[!htb]
	\centering
	\begin{adjustbox}{minipage=\linewidth,scale=1.0}
		\begin{subfigure}[t]{0.48\textwidth}
			\centering
			\includegraphics[clip, width=\linewidth]{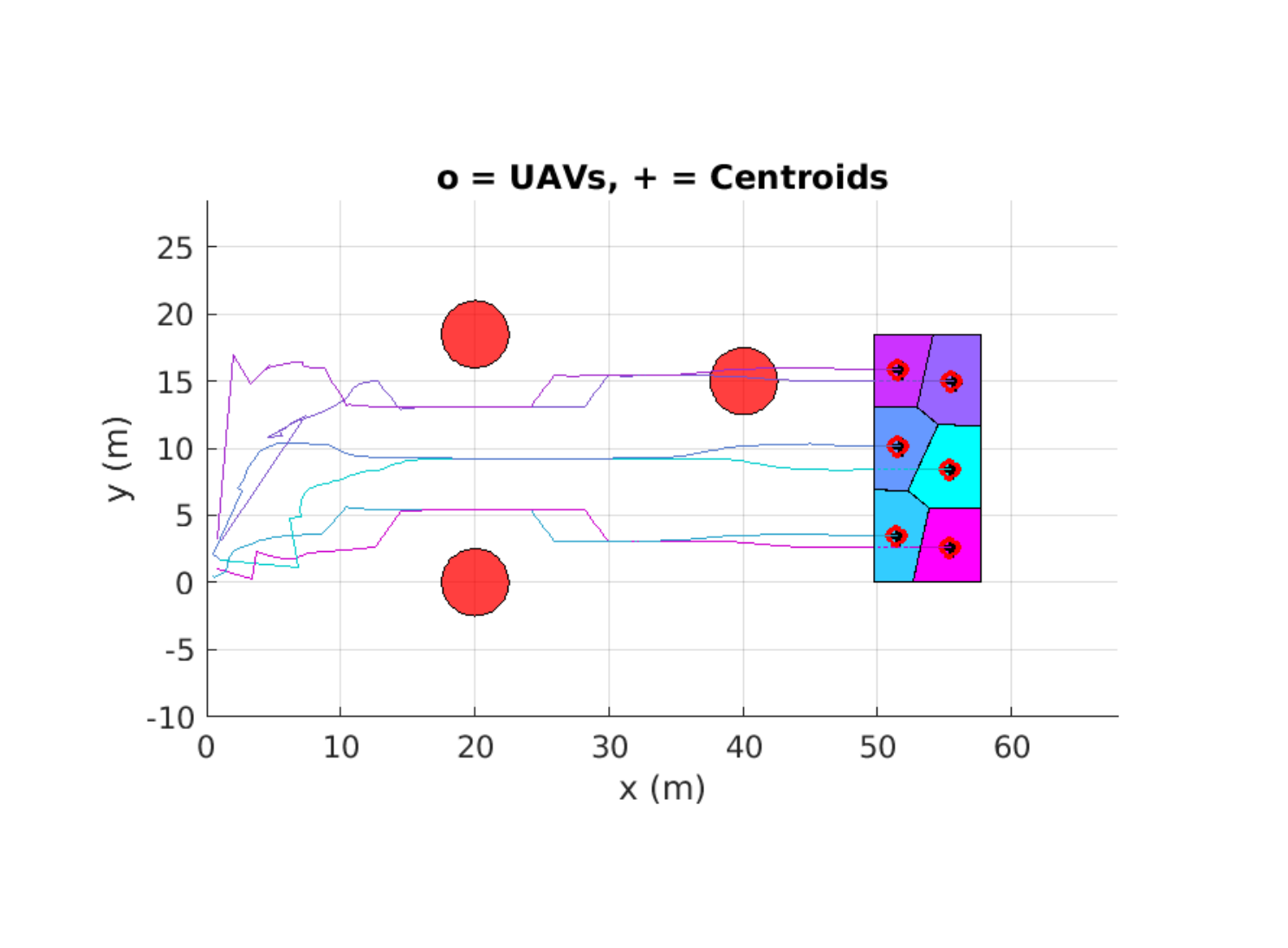} 
			\caption{XY View}
		\end{subfigure}
		\hfill
		\begin{subfigure}[t]{0.48\textwidth}
			\centering
			\includegraphics[trim={0 0 0 0}, clip, width=\linewidth]{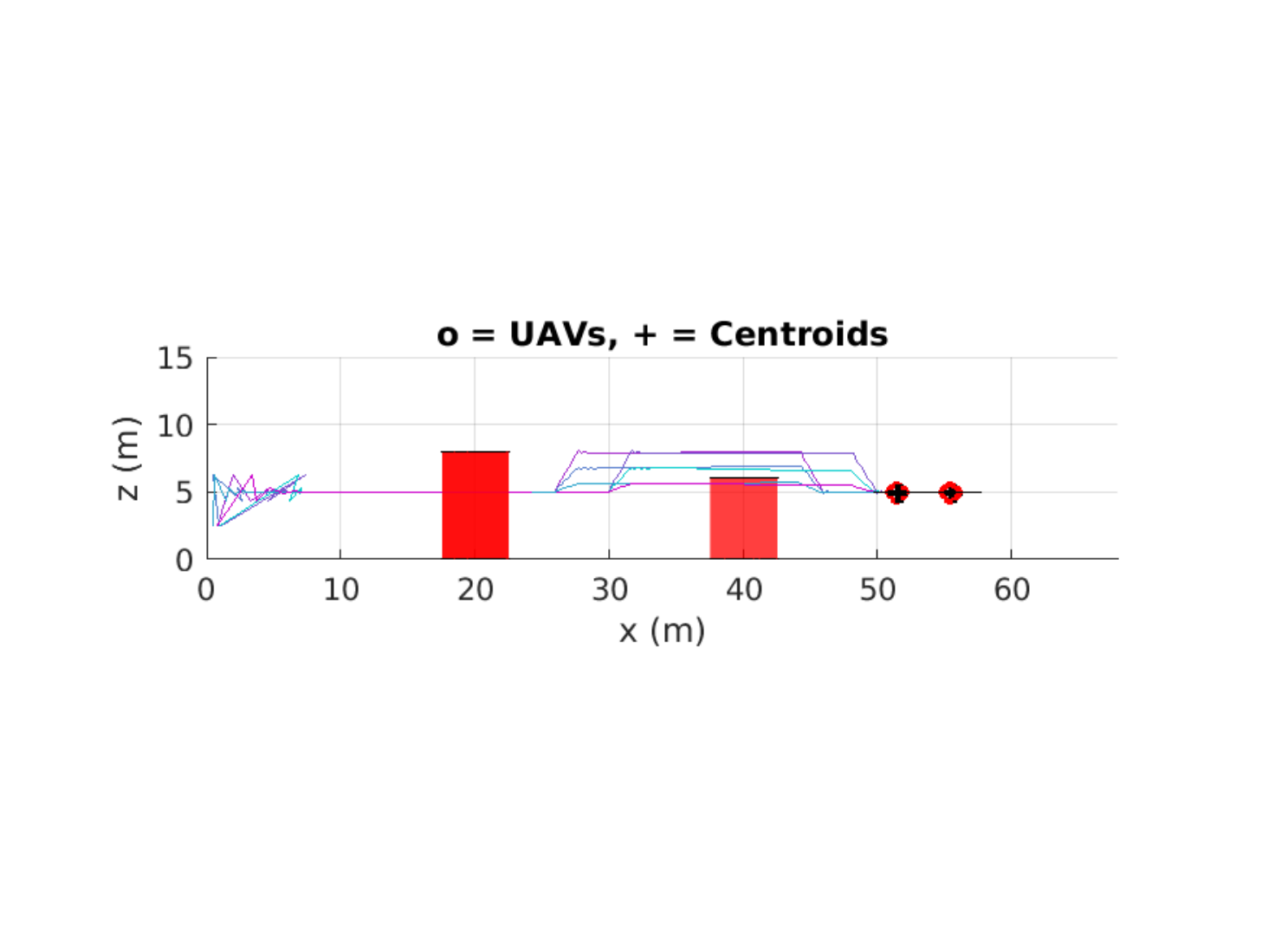} 
			\caption{XZ View}
		\end{subfigure}
		\caption{Complete trajectories of the Multi-UAV system showing the obstacle avoidance capability of the proposed approach [Planar Views] (Simulation Case 6)}
		\label{fig:ch9:simObsAvoidance3}
	\end{adjustbox}
\end{figure}

\section{Implementation using a Multi-Quadrotor System}\label{sec:ch9:impl}

The proposed barrier and sweep coverage control methods are developed based on the general kinematic model in \eqref{equ:ch9model} which is applicable to different UAV types and AUVs.
An example way in implementing the sweep coverage control method is presented in this section for a particular application in precision agriculture using a group of quadrotor UAVs.

\subsection{Quadrotor Dynamics}

Now, we will extend the model in \eqref{equ:ch9model} to include the dynamics of quadrotor-type UAVs according to the following model:
\begin{eqnarray}
\bm{\dot{p}}_i &=& \bm{v}_i \label{equ:ch9:pdot} \\
\bm{\dot{v}}_i &=& -g \bm{e}_3 + \frac{1}{m_i} T_i \bm{R}_i \bm{e}_3 \label{equ:ch9:vdot} \\
\bm{\dot{R}}_i &=& \bm{R}_i skew(\bm{\Omega}_i) \label{equ:ch9:Rdot} \\
\bm{I}_i\bm{\dot{\Omega}}_i &=& -\bm{\Omega}_i \times \bm{I}_i \bm{\Omega}_i + \bm{\tau}_i \label{equ:ch9:Omegadot}
\end{eqnarray}
where the velocity vector $\bm{u}_i$ is replaced by $\bm{v}_i$ for convenience, $g$ is the gravitational acceleration, $e_3 = [1\ 1\ 1]^T$, $m_i$ is the UAV mass, and $\bm{I}_i$ is the inertia matrix.
Recall that the vehicle's position and velocities are expressed in the inertial coordinate frame $\{\mathcal{I}\}$.
Consider also another reference frame attached to the UAV body with its origin at the center of mass.
This frame is referred to as the body-fixed frame $\{B^{i}\}$.
The orientation of the UAV is represented using a rotation matrix $\bm{R}_i$ between $\{B^{i}\}$ and $\{\mathcal{I}\}$.
The rate of change in the vehicle's orientation is denoted by $\Omega_i$ (i.e. angular velocity) which is expressed in the $\{B^{i}\}$ frame.
The control inputs for this model are the collective thrust $T_i \in \R$ and the body-torques vector $\bm{\tau}_i \in \R^3$.
For more details about quadrators modelling, the reader is referred to \cite{hamel2002dynamic,mellinger2011minimum}.

\subsection{Tracking Control}

There are different ways to implement the proposed sweeping control strategy.
One possible approach is to directly couple the velocity commands in \eqref{equ:ch9u_unbounded} (or \eqref{equ:ch9u_bounded}) into attitude and thrust control inputs.
Another possible direction is to use the determined centroidal Voronoi configurations by the algorithm \textbf{S1}-\textbf{S4} as goal positions with a position tracking controller.
The latter approach is considered here for a control design based on the differential-flatness property of the quadrotor dynamics and the sliding mode control technique.

Let's redefine the position and velocity tracking errors as follows:
\begin{eqnarray}
\bm{e}_{\bm{p},i} &=& \bm{p}_i - \bm{c}_{V_i} \\
\bm{e}_{\bm{v},i} &=& \bm{v}_i - \bm{\dot{c}}_{V_i}
\end{eqnarray}
Also, consider sliding surfaces $\bm{\sigma}_i$ such that:
\begin{equation}
\bm{\sigma}_i = \bm{e}_{\bm{v},i} + \bm{L}_{1,i} \tanh(\mu_1\bm{e}_{\bm{p},i})
\end{equation}
where $\mu_1>0$, and $\bm{L}_{1,i}$ is a positive definite diagonal matrix.
Note that this choice will ensure that the velocities can remain bounded by $\lambda_{max}(\bm{L}_{1,i})$ to account for physical limits.

A desired acceleration command can then be computed using:
\begin{equation}\label{equ:ch9:commandAcc}
\begin{aligned}
\frac{1}{m_i}\bm{a}_{des,i} =& g \bm{e}_3 + \bm{L}_{2,i} \tanh(\mu_2\bm{\sigma}_i) \\
&+ \sum_{j \in N_i} \bm{L}_{ij}\max(0,d_{ij}^2 - ||\bm{p}_i - \bm{p}_j||^2) \frac{\bm{p}_i - \bm{p}_j}{||\bm{p}_i - \bm{p}_j||}
\end{aligned}
\end{equation}
where $\bm{L}_{2,i}$ and $\bm{L}_{ij}$ are positive definite diagonal matrices.
The additional term in \eqref{equ:ch9:commandAcc} provides more safety as it represents a repelling force from vehicles closer than some critical distance $d_{ij}>0$ in the neighbourhood of vehicle $i$.

The differential-flatness property is then utilized to achieve the acceleration command in \eqref{equ:ch9:commandAcc} while maintaining some desired yaw angle $\psi_{ref,i}$ using the following equations:
\begin{eqnarray}\label{equ:ch9:thrustInput}
T_i &=& \bm{a}_{des,i}^T \bm{R}_i e_3 \\
\bm{R}_{des,i} &=& [\bm{x}^i_{\mathcal{B}^i,des}, \bm{y}^i_{\mathcal{B}^i,des}, \bm{z}^i_{\mathcal{B}^i,des}] \\
\bm{z}^i_{\mathcal{B}^i,des} &=& \frac{\bm{a}_{des,i}}{||\bm{a}_{des,i}||} \\
\bm{x}^i_{\mathcal{B}^i,des} &=& \frac{\bm{y}^i_{C^i}\times \bm{z}^i_{\mathcal{B}^i,des}}{||\bm{y}^i_{C^i}\times \bm{z}^i_{\mathcal{B}^i,des}||} \\
\bm{y}^i_{\mathcal{B}^i,des} &=& \bm{z}^i_{\mathcal{B}^i,des} \times \bm{x}^i_{\mathcal{B}^i,des} \\
\bm{y}^i_{C^i} &=& [-\sin(\psi_{ref,i})\ \cos(\psi_{ref,i})\ 0]^T
\end{eqnarray}
A low-level control is then used to compute $\bm{\tau}_i$ to ensure that the vehicle can achieve the desired attitude $\bm{R}_{des,i}$. 
Further details about the differential-flatness property of quadrotor dynamics can be found in \cite{mellinger2011minimum,faessler2017differential}.

\subsection{Software-in-the-Loop Simulations}

The performance of the suggested sweeping coverage control for multi-quadrotor systems was evaluated using Software-in-the-Loop (SITL) simulations using \textit{Gazebo} and the \textit{Robot Operating System (ROS)} framework.
This allows us to test the computational performance of the production code which can be used directly on the real vehicle's onboard computer to implement our algorithms and control strategies in real-time.
A simple scenario was considered where a group of 3 quadrotors were needed to survey a crops field $\tilde{\mathcal{S}}=\{x,y,z \in \R:  0 \leq x \leq 32, 0 \leq y \leq 40, z=0\}$.
Each vehicle is assumed to have an onboard camera providing a sensing footprint $\mathcal{A}_i \subset \mathcal{S}$ dependant on its characteristics and the UAV's altitude.
For simplicity, an obstacle-free environment was used as shown in \cref{fig:ch9:gazebo_env}.

We used the open-source PX4 flight stack as an implementation for the low-level controller and an extended Kalman filter for states estimation considering noisy measurements.
Our tracking control logic was implemented using C++ providing thrust and attitude inputs at a rate of 100Hz.
Furthermore, centroidal Voronoi configurations computations were implemented using Python utilizing some of the available tools from the "scipy.spatial" Python module to determine Voronoi cells as a convex polygon in accordance with \cref{assm:voronoiCells}.
Furthermore, the algorithm in \textbf{S1}-\textbf{S4} was used to obtain the centroids.
This turned out to be very computationally efficient since the computations rely on closed-form expressions.

In order to achieve the coverage task in hand, an initial rectangular sweeping plane was selected as: $\mathcal{F}(0) = \{x,y,z \in \R: 0 \leq x \leq 8, 0 \leq y \leq 2, z=2\}$ where it was required for the UAVs to fly at a fixed altitude of $2m$.
The classical lawn-mower pattern was considered to generate the trajectory of $\mathcal{F}(t)$ to completely scan $\tilde{\mathcal{S}}$ while keeping the yaw angle at $\psi_{ref}=0$ without loss of generality.
The following parameters were used in the simulation: $\bm{L}_{1,i}=diag\{2,2.5,3\}$, $\bm{L}_{2,i}=diag\{2.5,3,4\}$, $\bm{L}_{ij}=diag\{1,1,1\}$, $d_{ij}=0.5m$, $\mu_1=2$ and $\mu_2=1.5$ (where $diag\{\cdot\}$ represents a diagonal matrix).

\begin{figure}[ht]
	\centering
	\includegraphics[width=0.75\linewidth]{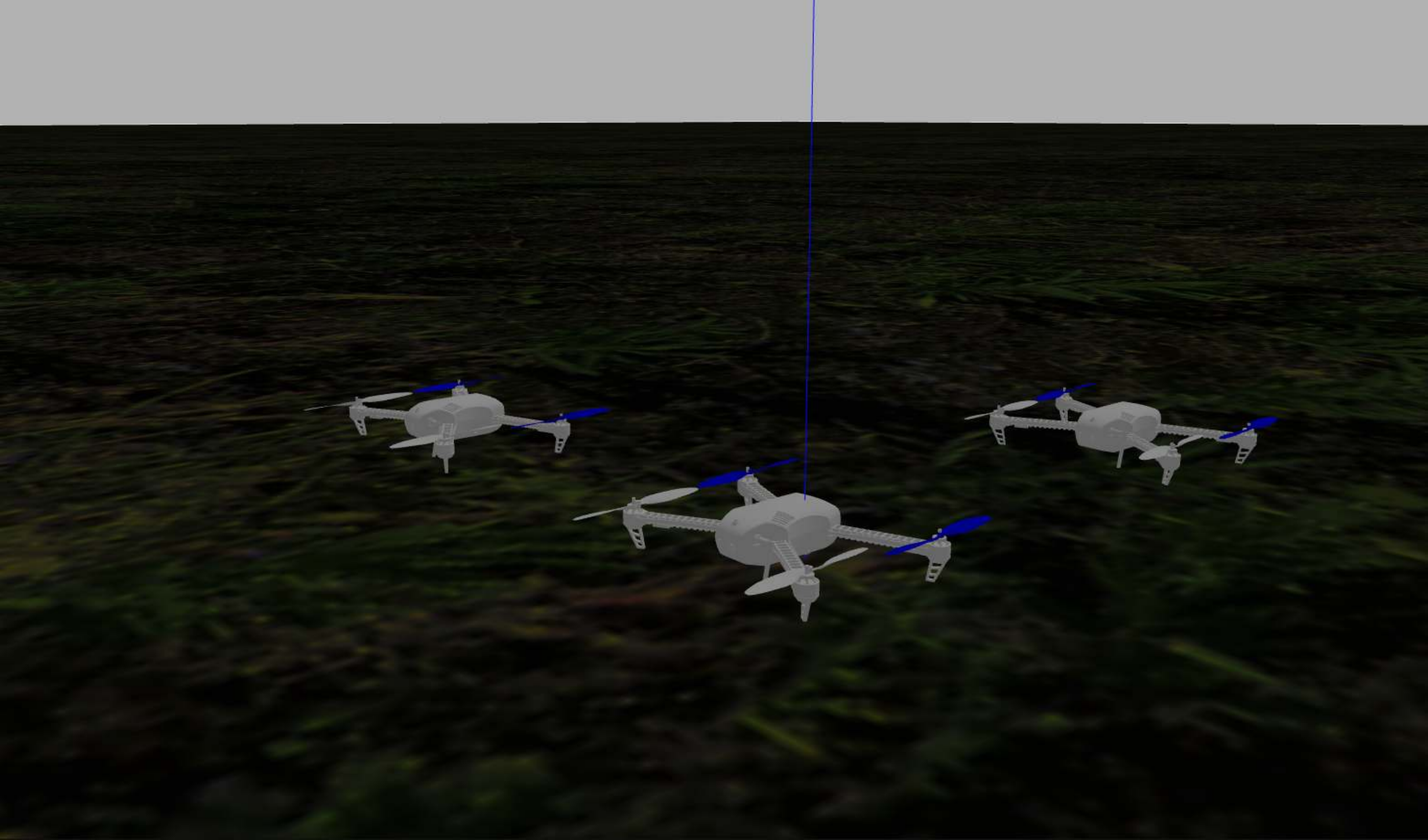}
	\caption{UAVs and environment used in simulations}
	\label{fig:ch9:gazebo_env}
\end{figure}

During the simulation, all measurements and data were captured as a bag file using available ROS recording tools.
The results are plotted using MATLAB which are shown in \cref{fig:ch9:simPos,fig:ch9:simPos3D,fig:ch9:simPos3D,fig:ch9:simVel,fig:ch9:simRPY,fig:ch9:simThrust}.
In \cref{fig:ch9:simPos}, coordinates of all vehicles are plotted with respect to time.
The overall paths taken by the vehicles are shown in \cref{fig:ch9:simPos3D,fig:ch9:simPos3D} where the highlighted rectangular area represents the crops field to be surveyed (i.e. $\tilde{\mathcal{S}}$). 
Note that in this time, the vehicles were moving along some other region $\mathcal{S}$ whose projection on the ground floor is $\tilde{\mathcal{S}}$.
Thus, the trajectory of $\mathcal{F}(t)$ was designed to cover $\tilde{\mathcal{S}}$ given that the equipped sensors have a collective footprint which ensures the optimal coverage of $\tilde{\mathcal{S}}$.  
It is clear from these results that the overall motion is collision-free, and the lawn-mower pattern decided for $\mathcal{F}(t)$ was followed by the whole group.
Velocity norms are shown in \cref{fig:ch9:simVel}.
Overall the vehicles are moving with a constant speed $g_0=1 m/s$ except for some changes in velocity at times where the vehicles are turning ($t\approx 60s,110s,160s$) in order to avoid collisions with each others.
Overall, velocities remain within the desired limits of $\|v\| \leq 3.5 m/s$ except at the start when vehicles were taking off which is based on a different control law.
The high-level control commands, namely roll, pitch, yaw and thrust, are shown in \cref{fig:ch9:simRPY,fig:ch9:simThrust} where the thrust is normalized in the range $[0,1]$.
This simulation case shows that our control can be successfully implemented using multi-quadrotor systems.

\begin{figure}[ht]
	\centering
	\includegraphics[width=0.8\linewidth]{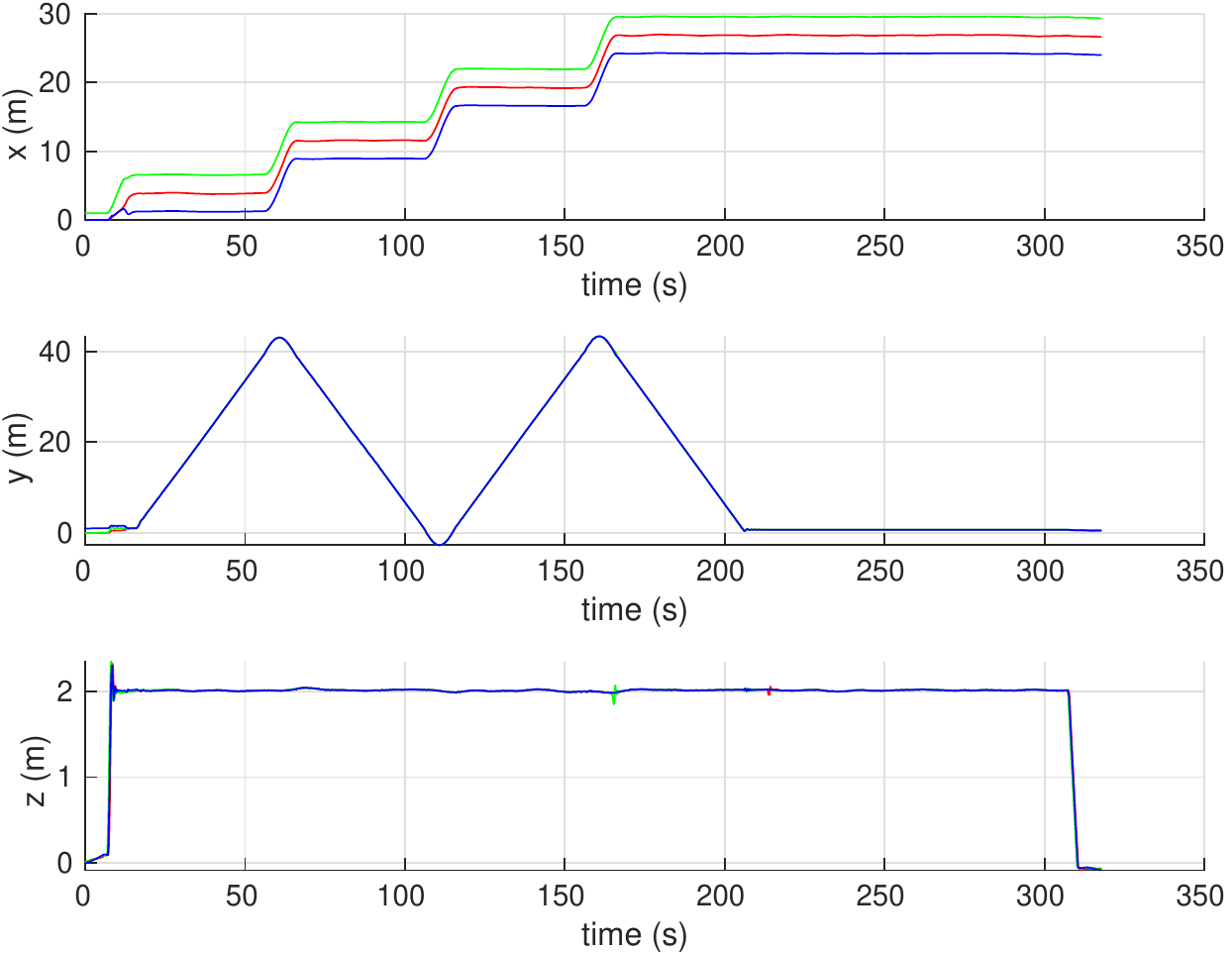}
	\caption{UAVs coordinates with respect to time}
	\label{fig:ch9:simPos}
\end{figure}

\begin{figure}[ht]
	\centering
	\includegraphics[width=0.8\linewidth]{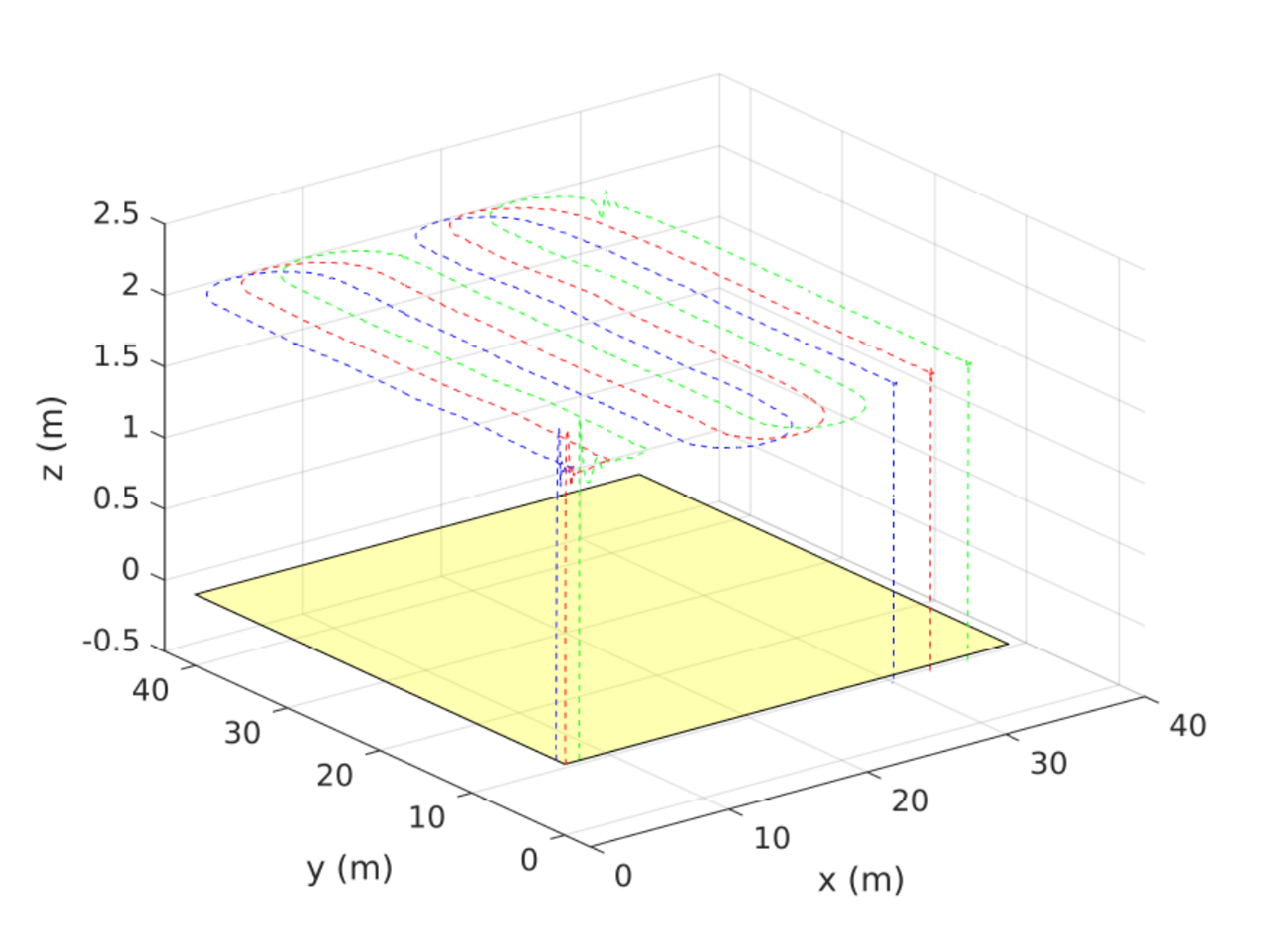}
	\caption{UAVs actual trajectories (3D prospective)}
	\label{fig:ch9:simPos3D}
\end{figure}

\begin{figure}[ht]
	\centering
	\includegraphics[width=0.7\linewidth]{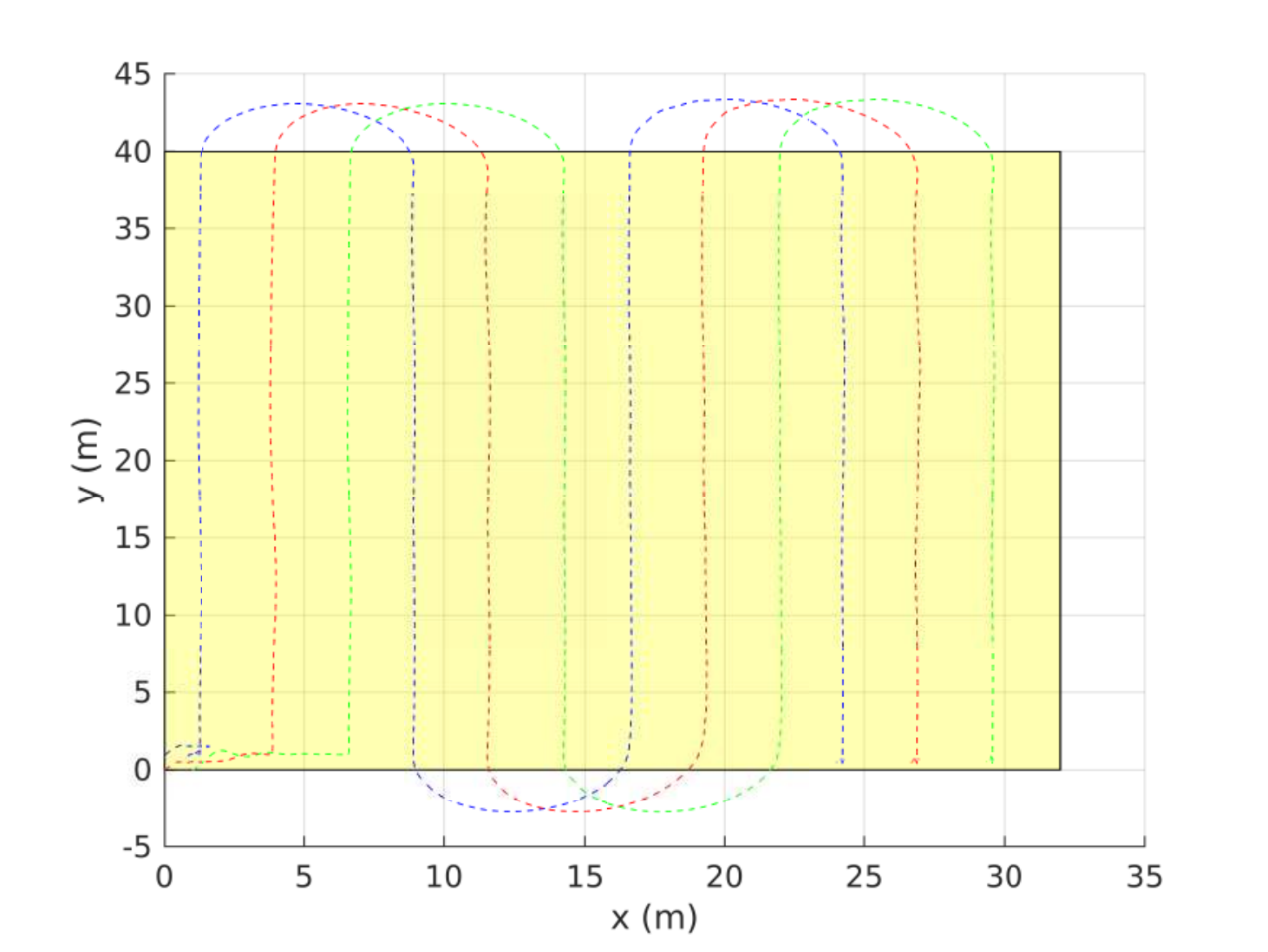}
	\caption{UAVs actual trajectories (2D prospective)}
	\label{fig:ch9:simPos2D}
\end{figure}

\begin{figure}[ht]
	\centering
	\includegraphics[width=0.7\linewidth]{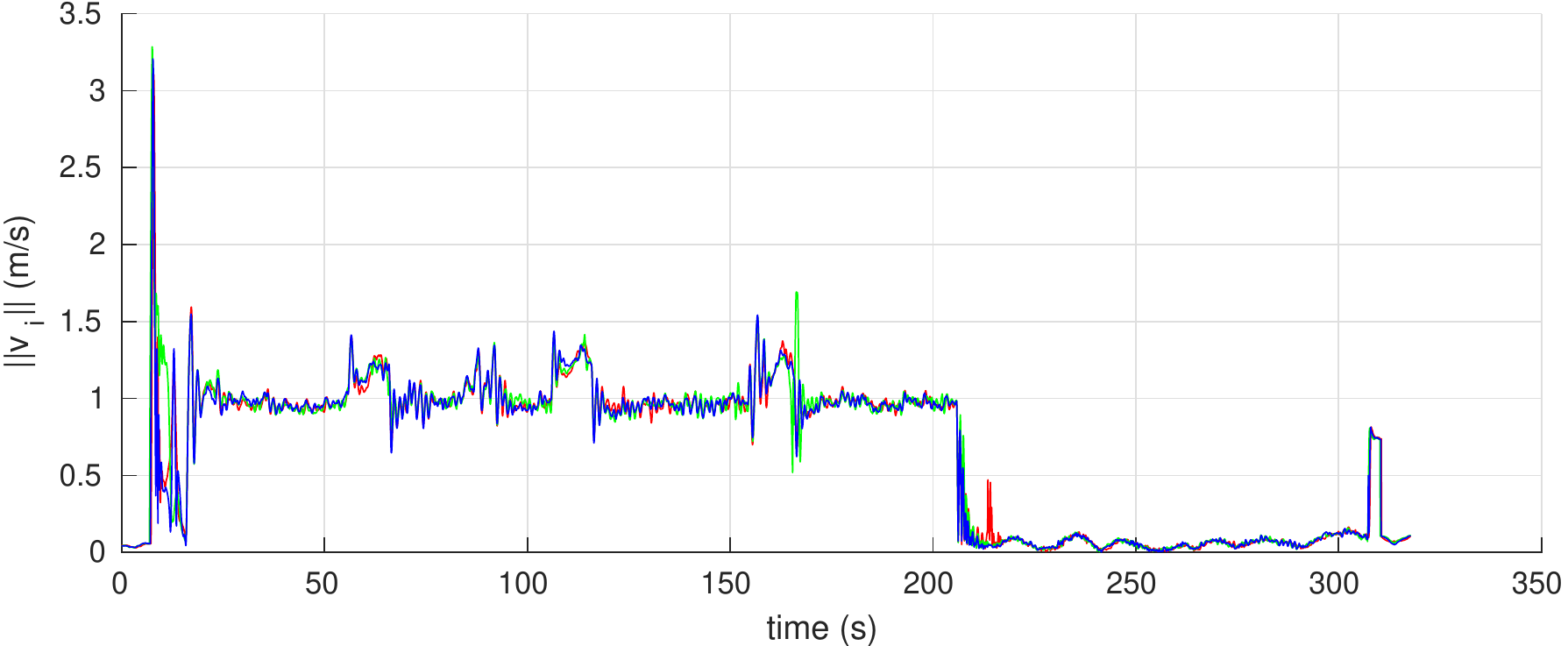}
	\caption{UAVs linear velocities with respect to time}
	\label{fig:ch9:simVel}
\end{figure}

\begin{figure}[ht]
	\centering
	\includegraphics[width=0.7\linewidth]{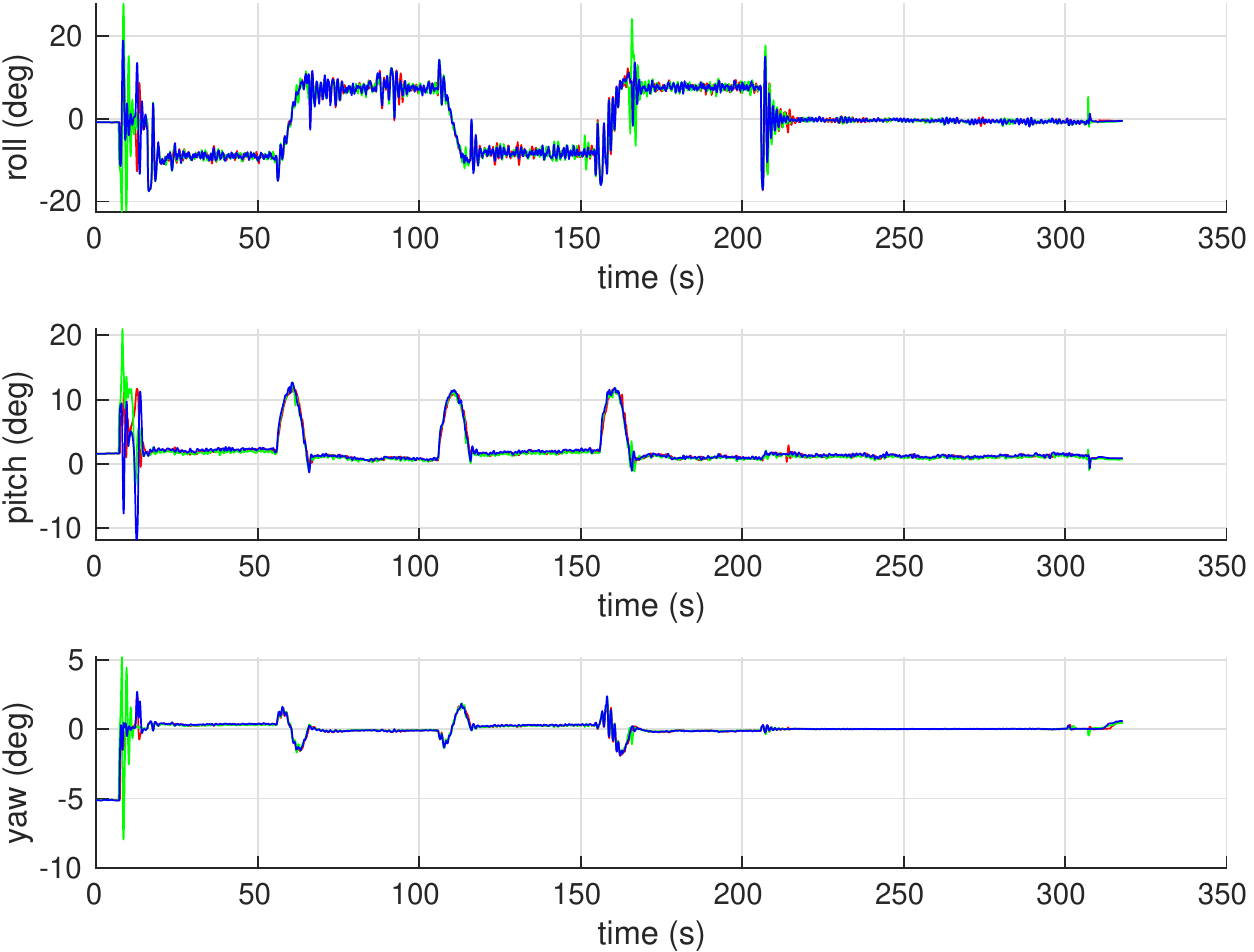}
	\caption{UAVs attitude (roll, pitch and yaw) with respect to time}
	\label{fig:ch9:simRPY}
\end{figure}

\begin{figure}[ht]
	\centering
	\includegraphics[width=0.7\linewidth]{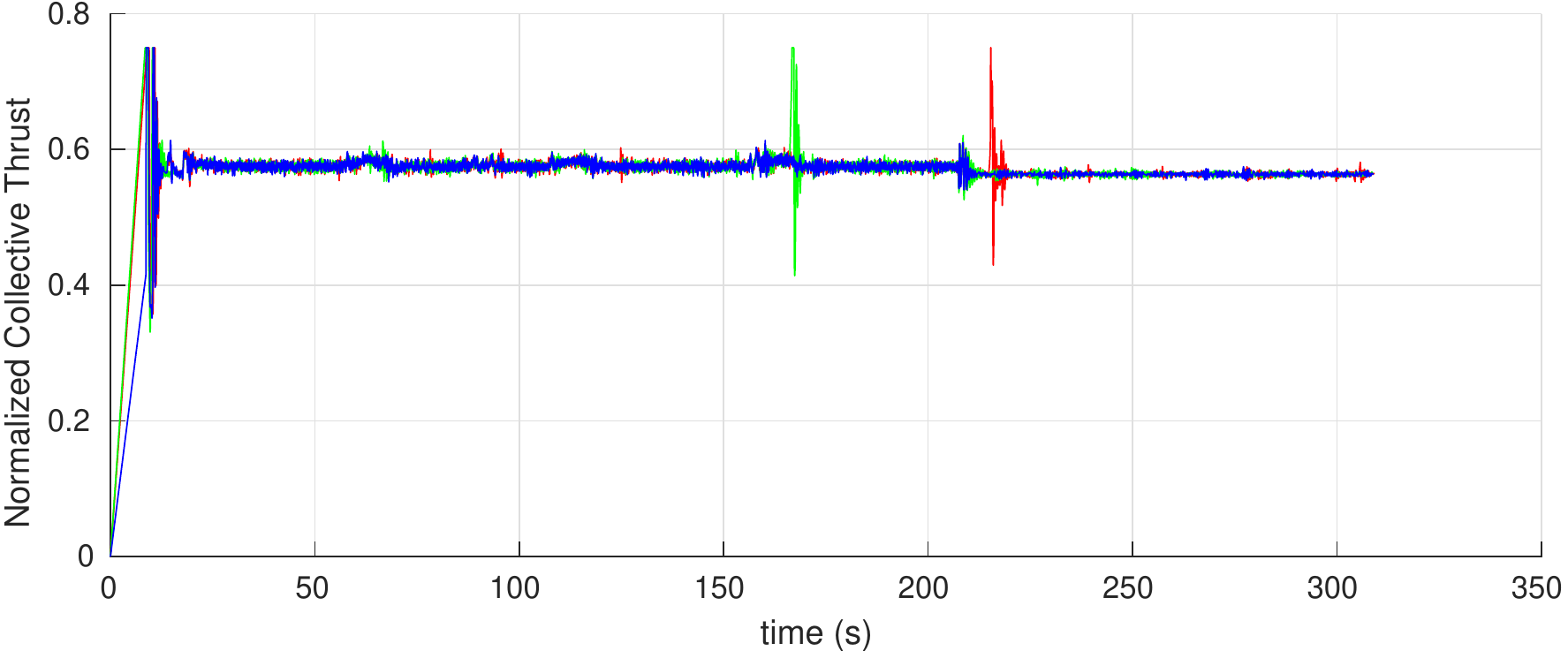}
	\caption{UAVs normalized collective thrust input with respect to time}
	\label{fig:ch9:simThrust}
\end{figure}

\section{Conclusion \& Future Work}\label{sec:ch9conclusion}

Distributed coordinate control method was proposed in this chapter for multi-UAV systems to address coverage problems in 3D spaces including barrier and sweeping problems.
The development of the control laws was based on general kinematic model which make them applicable to autonomous underwater vehicles as well.
A region-based control mechanism was adopted to constrain the movement of the vehicles with some desired region proving a computationally good solution that is both robust and scalable.
This has been shown through several simulations which have confirmed the performance and robustness of our method.
Moreover, a case was presented to show possible ways of avoiding obstacles using the suggested scheme.
Implementation details were also presented for a special case of using a multi-quadrotor system in precision agriculture.
For that case, software-in-the-loop simulations were carried out using Gazebo and ROS showing good results for real-time implementations.
Future work can consider practical implementation and developing proper ways to autonomously decide the shape of the sweeping region based on the multi-UAV system characteristics and the targeted sensing region size.

\bookmarksetup{startatroot}
\addtocontents{toc}{\bigskip}

\chapter{Conclusion\label{cha:conclusion}}

UAVs have been rapidly emerging in various fields ought to the advances in technologies related to sensing, computing, power and communication.
They provide more agile solutions to reach hard places where typical ground vehicles cannot reach.
One of the critical aspects for UAV developments is to be able to navigate safely in unknown environments.
Many of the existing solutions in the literature consider planar approaches when addressing the collision avoidance in a reactive manner.
Other solutions that treats the 3D collision avoidance problem from a path planning prospective may have expensive computational cost with high-latency reactions to obstacles.
This can be critical when moving at high speeds especially around dynamic obstacles.
The work presented in this report addressed such problem by providing reactive solutions with low computational complexity.
Moreover, the safe navigation problem of multi-vehicle systems was also addressed to develop different distributed motion-coordination control strategies for more advanced collective behaviors. 

\section{Contributions Summary}

This report made several contributions to the state-of-the-art methods for safe navigation of UAVs and motion coordination of multi-UAV systems.
A summary of these contributions is provided next.

{\bfseries\Cref{cha:litreivew}} presented a general overview of recent developments in unmanned aerial vehicle technologies towards achieving fully autonomous navigation.
The main focus of the survey was to highlight some of the recent methods proposed in the last decade targeting the problems of control and 3D motion planning. %
Additionally, a list of open-source projects and useful tools for developing autonomous applications were also provided.
This can aid researchers in this field to quickly develop and test more complex and advanced approaches through integration with existing frameworks.

{\bfseries \Cref{cha:methods_hybrid}} suggested a general framework for autonomous navigation in partially-known dynamic environments.
It adopts a hybrid-based approach by combining path planning methods with reactive guidance laws to provide a more advanced solution coping with the drawbacks of relying solely on either path planning or reactive navigation.
As a general framework, it can be adopted considering 3D methods even though it was proposed considering 2D models as a proof-of-concept.

A novel reactive navigation strategy for UAVs was then proposed in {\bfseries \cref{cha:methods_reactive3D}} to address the navigation problem in 3D unknown environments.
Guidance laws based on sliding mode control were developed using a general 3D kinematic model applicable to different UAV types and AUVs.
The overall strategy has low computational cost compared to search-based and optimization-based motion planning methods as it can rely only on local information about the environment as seen by onboard sensors.
A main differentiator of the suggested approach relative to many of the existing reactive navigation methods is that it can produce 3D obstacle avoidance maneuvers utilizing the full capabilities of UAVs.
The approach was further extended in {\bfseries \cref{cha:reactive_impl}} showing a possible implementation with multirotors.
The vehicles dynamics were considered in the low-level control design which was based on the differential-flatness property of the system and the sliding mode control technique.
Experiments were also carried out using a quadrotor to validate the overall performance. 

Another 3D navigation approach was suggested in {\bfseries \cref{cha:deforming_approach}} based on the idea of real-time deformable paths similar to elastic bands.
The deformation process has a very low computational complexity such that the overall method can be considered reactive following the "sense and avoid" paradigm.
This feature makes the approach suitable for navigation in unknown and dynamic environments.
It also makes it more appealing for vehicles with limited computing power compared to search-based and optimization-based methods.

In {\bfseries \cref{cha:tunnel_navigation}}, a special case of the UAV safe navigation problem was considered to allow autonomous navigation in unknown tunnel-like environments with a computationally-light solution. 
Novel control laws were developed to guide the UAV through the environment by relying directly on sensors measurements without the need for accurate localization.
The novelty of the approach lies in its ability to handle very complex 3D tunnel-like environments rather than assuming planar motions similar to many of the existing reactive methods.
Further implementation details with multirotors were provided considering the system dynamics.
Several real experiments and simulation cases were performed to show how well the proposed method can handle navigation in various 3D environments with different shapes.
A complete and robust perception pipeline was also suggested to be used by UAVs with limited sensing field-of-view.

{\bfseries \Cref{cha:flocking_control}} proposed bounded distributed control laws for multi-vehicle systems to address the flocking problem in 2D and 3D unknown environments.
The design adopted a Null-space-based approach with different priorities for motion objectives.
Local objectives include avoiding collisions with other vehicles and surrounding obstacles while global objectives are reaching a target region and forming a geometric structure to achieve consensus in motion.
The control laws were evaluated using several simulations with different number of vehicles.

Finally, novel distributed control strategies were proposed in {\bfseries \cref{cha:coverage_control}} to address 3D coverage problems using multi-UAV systems.
The idea relies on constraining the vehicles' movement to a virtual region whose shape and dynamics can be manipulated in real-time in a distributed fashion to achieve coverage tasks with collision/obstacle avoidance.
The vehicles maintain their positions according to the centroidal Voronoi configurations of the virtual region where a method for computing such locations were presented.
The performance of the suggested methods were evaluated through several simulations.
Additional simulation cases were also considered to show how robust and scalable these methods are.
The idea was developed using a general kinematic model to be applicable to different UAV types.
Furthermore, a dynamic controller was proposed considering quadrotor dynamics along with further implementation details.
Software-in-the-loop simulations of the full quadrotor dynamical model were performed using Gazebo and ROS to evaluate the computational complexity of the overall method.

\section{Future Work}

Developing fully autonomous UAVs remains a very complex and challenging problem as the complexity of their applications increase and more demanding tasks are needed.
To achieve a fully autonomous operation, each component of the commonly used modular approach, to implement the overall navigation stack as interconnected subsystems or modules, needs to be well developed where the performance of one module affects the performance of the whole system.
This report focused mainly on the high level control part for collision avoidance and motion coordination with an objective to have solutions with low computational demands to compensate for latencies introduced by other modules such as the sensing and perception subsystem.
Some of the suggested methods were tested considering a worst-case scenario of limited field-of-view cameras.
However, implementations with different sensors configuration is considered as future work to provide more insights in developing the suggested approaches further for more robustness.
Furthermore, future work will also consider extending the developed navigation strategies in \cref{cha:methods_reactive3D,cha:reactive_impl,cha:deforming_approach} to adopt a perception-aware design (ex. see \cite{zhang2018perception,tordesillas2021panther,Falanga2018pampc,murali2019perception,spasojevic2020perception,sheckells2016optimal}) in a way to maintain information about obstacles during the avoidance maneuver which can be important when using sensors with narrow FOV.

The design of low-level control for multirotors can also be improved to account for flapping blades, drag effects and other forms of disturbances.
This can increase the performance of the overall solution especially for the developed navigation strategy in tunnel-like environments where it's very challenging to fly near to surfaces in confined areas.
Thus, future research will also include supplementing the developed safe navigation and motion coordination control strategies by lower level controllers and filters designed by modern tools of robust control and robust state estimation such as $H_{\infty}$ control \cite{van19922,petersen2000robust,raffo2010integral}, sliding mode control \cite{utkin1999sliding,utkin2013sliding,utkin1978sliding}, robust Kalman filtering \cite{savkin1998robust,petersen1999robust,pathirana2005node,pathirana2004location,rigatos2012nonlinear,inoue2016extended}, and switched controllers \cite{savkin2002hybrid,skafidas1999stability,van2000introduction,liberzon1999basic}.
Additionally, current ongoing research includes investigating the design of low-level velocity tracking controllers to directly couple the designed kinematic reactive control laws with the low-level control using one of the previously mentioned robust control methods.

The proposed distributed control laws for multi-vehicle systems can be further analyzed mathematically to show that they can maintain the connectivity of the networked vehicles during the motion under some conditions on the design parameters.
It is also possible to extend the analysis of the designed flocking control laws in \cref{cha:flocking_control} to consider switching network topologies.
However, the system dynamics may exhibit discontinuities in this case whenever the connectivity graph changes.
Thus, the stability of the overall system can be analyzed using non-smooth analysis methods \cite{clarke1990optimization} and differential inclusions methods \cite{filippov2013differential}.
Field testing with multiple UAVs will also help to further evaluate the performance of the suggested flocking and coverage methods.
Thus, future work will also include practical implementation and developing advanced algorithms to autonomously decide the shape of the sweeping region based on the multi-UAV system characteristics and the targeted sensing region size.
It is also critical to investigate other aspects affecting the applicability of these approaches such as imperfections in communication channels which impacts the information exchanged among the vehicles and needed by the control.
Such problems attracted interest of several researchers in networked control systems \cite{tipsuwan2003control,hespanha2007survey,wang2008networked,matveev2009estimation}.

    \clearpage
    \backmatter
   
   \pagestyle{noHeading}
    \appendix
    \chapter*{Appendix A: Smooth Trapezoidal Velocity Trajectory Generation}\label{ch:appA:trapezoidal_velocity}

The considered profile consists of 7 time intervals satisfying some boundary conditions, and it is given by the following equations:
\begin{equation}
	\dddot{s}(t) = \left\{\begin{array}{lc}
	j_{m},    & 0 \leq t < t_1 \\
	0,          & t_1 \leq t < t_2 \\
	-j_{m},   & t_2 \leq t < t_3 \\
	0,          & t_3 \leq t < t_4 \\
	-j_{m},   & t_4 \leq t < t_5 \\
	0,          & t_5 \leq t < t_6 \\
	j_{m},    & t_6 \leq t < t_f
	\end{array}\right.
\end{equation}
\begin{equation}
\ddot{s}(t) = \left\{\begin{array}{lc}
j_{m}t + a_0,            & 0 \leq t < t_1 \\
a_{m},                   & t_1 \leq t < t_2 \\
-j_{m}\bar{t} + a_{m}, & t_2 \leq t < t_3 \\
0,                         & t_3 \leq t < t_4 \\
-j_{m}\bar{t},           & t_4 \leq t < t_5 \\
-a_{m},                  & t_5 \leq t < t_6 \\
j_{m}\bar{t} - a_{m},  & t_6 \leq t < t_f
\end{array}\right.
\end{equation}
\begin{equation}
\dot{s}(t) = \left\{\begin{array}{lc}
\frac{1}{2}j_{m}t^2 + a_0 t + v_0,            & 0 \leq t < t_1 \\
a_{m}\bar{t} + c_1,                   & t_1 \leq t < t_2 \\
-\frac{1}{2}j_{m}\bar{t}^2 + a_{m}\bar{t} + c_2, & t_2 \leq t < t_3 \\
v_{m},                         & t_3 \leq t < t_4 \\
-\frac{1}{2}j_{m}\bar{t}^2 + v_{m},           & t_4 \leq t < t_5 \\
-a_{m}\bar{t} + c_3,                  & t_5 \leq t < t_6 \\
\frac{1}{2}j_{m}\bar{t}^2 - a_{m}\bar{t} + c_4,  & t_6 \leq t < t_f
\end{array}\right.
\end{equation}
\begin{equation}
s(t) = \left\{\begin{array}{lc}
\frac{1}{6}j_{m}t^3 + \frac{1}{2}a_0 t^2 + v_0 t + p_0,            & 0 \leq t < t_1 \\
\frac{1}{2}a_{m}\bar{t}^2 + c_1 \bar{t} + k_1,                   & t_1 \leq t < t_2 \\
-\frac{1}{6}j_{m}\bar{t}^3 + \frac{1}{2}a_{m}\bar{t}^2 + c_2 \bar{t} + k_2, & t_2 \leq t < t_3 \\
v_{m}\bar{t} + k_3,                         & t_3 \leq t < t_4 \\
-\frac{1}{6}j_{m}\bar{t}^3 + v_{m}\bar{t} + k_4,           & t_4 \leq t < t_5 \\
-\frac{1}{2}a_{m}\bar{t}^2 + c_3 \bar{t} + k_5,                  & t_5 \leq t < t_6 \\
\frac{1}{6}j_{m}\bar{t}^3 - \frac{1}{2}a_{m}\bar{t}^2 + c_4 \bar{t} + k_6,  & t_6 \leq t < t_f
\end{array}\right.
\end{equation}
where $\bar{t}\in[0,t_b-t_a]$ is defined as $\bar{t} = t - t_a$ such that $t_a$ and $t_b$ are the start and end times of the corresponding interval respectively (i.e. $t \in [t_a,t_b]$), $\{p_0,v_0,a_0\}$ correspond to initial conditions $\{s(0),\dot{s}(0),\ddot{s}(0)\}$, and $\{v_m,a_m,j_m\}$ are maximum values to achieve for $\{\dot{s}(t),\ddot{s}(t),\dddot{s}(t)\}$ during the corresponding intervals with constant values.
Also, let $\{p_f,v_f,a_f\}$ correspond to final conditions $\{s(t_f),\dot{s}(t_f),\ddot{s}(t_f)\}$.
Then, the parameters $c_i,\ i\in\{1,\cdots,4\}$ are computed using:
\begin{equation}
	\begin{aligned}
		c_1 &= \frac{a_m^2 - a_0^2}{2j_m} + v_0\\ 
		c_2 &= -\frac{a_m^2}{2j_m} + v_{m}\\ 
		c_3 &= -\frac{a_m^2}{2j_{m}} + v_{m}\\ 
		c_4 &= v_f +\frac{a_m^2 - j_m^2a_f^2}{2j_m} \\
	\end{aligned}
\end{equation}
The time duration $\triangle t_i = t_i - t_{i-1}$ of each interval $i$ can be determined using:
\begin{equation}
	\begin{array}{ccc}
	\triangle t_1 = \frac{a_m - a_0}{j_m}, & 
	\triangle t_2 = \frac{c_2 - c_1}{a_m}, & 
	\triangle t_3 = \frac{a_m}{j_m}, \\[0.5cm]
	\triangle t_5 = \frac{a_m}{j_m}, & 
	\triangle t_6 = \frac{c_3 - c_4}{a_m}, & 
	\triangle t_7 = \frac{a_m}{j_m} + a_f
	\end{array}
\end{equation}
where $\triangle t_4$ is given later.
The parameters $k_j,\ j\in\{1,\cdots,6\}$ are computed using
\begin{equation}
	\begin{aligned}
	k_1 &= \frac{1}{6}j_m (\triangle t_1)^3 + \frac{1}{2} a_0 (\triangle t_1)^2 + v_0 \triangle t_1 + p_0\\
	k_2 &= k_1 + \frac{1}{2} a_m (\triangle t_2)^2 + c_1 \triangle t_2\\
	k_3 &= k_2 - \frac{1}{6}j_m (\triangle t_3)^3 + \frac{1}{2} a_m (\triangle t_3)^2 + c_2 \triangle t_3\\
	k_6 &= p_f - \frac{1}{6}j_m (\triangle t_7)^3 + \frac{1}{2} a_m (\triangle t_7)^2 - c_4 \triangle t_7\\
	k_5 &= k_6 + \frac{1}{2} a_m (\triangle t_6)^2 + a_m \triangle t_6\\
	k_4 &= k_5 + \frac{1}{6} a_m (\triangle t_5)^3
	\end{aligned}
\end{equation}
and finally $\triangle t_4$ is determined using:
\begin{equation}
\triangle t_4 = \frac{k_4 - k_3}{v_m}
\end{equation}

    \clearpage
    
    \bibliographystyle{IEEEtran}
    \bibliography{references}

\end{document}